\documentclass[11pt,reqno,english]{article}
\usepackage[T1]{fontenc}
\usepackage[latin9]{inputenc}
\usepackage{color}
\usepackage{babel}
\usepackage{refstyle}
\usepackage{mathrsfs}
\usepackage{enumitem}
\usepackage{tipa}
\usepackage{tipx}
\usepackage{amsmath}
\usepackage{amsthm}
\usepackage{amssymb}
\usepackage{graphicx}
\usepackage{wasysym}
\usepackage[unicode=true,pdfusetitle,
 bookmarks=false,
 breaklinks=false,pdfborder={0 0 0},pdfborderstyle={},backref=section,colorlinks=true]
 {hyperref}
\hypersetup{
 linkcolor=blue, citecolor=blue}

\makeatletter

\AtBeginDocument{\providecommand\enuref[1]{\ref{enu:#1}}}
\RS@ifundefined{subsecref}
  {\newref{subsec}{name = \RSsectxt}}
  {}
\RS@ifundefined{thmref}
  {\def\RSthmtxt{theorem~}\newref{thm}{name = \RSthmtxt}}
  {}
\RS@ifundefined{lemref}
  {\def\RSlemtxt{lemma~}\newref{lem}{name = \RSlemtxt}}
  {}

\theoremstyle{plain}
\newtheorem{thm}{\protect\theoremname}
\theoremstyle{plain}
\newtheorem{res}[thm]{\protect\resultname}
\theoremstyle{remark}
\newtheorem{rem}[thm]{\protect\remarkname}
\theoremstyle{plain}
\newtheorem{lem}[thm]{\protect\lemmaname}
\theoremstyle{plain}
\newtheorem{prop}[thm]{\protect\propositionname}

\usepackage{fullpage}
\usepackage[nottoc]{tocbibind}

\allowdisplaybreaks[3]

\@ifundefined{showcaptionsetup}{}{%
 \PassOptionsToPackage{caption=false}{subfig}}
\usepackage{subfig}
\makeatother

\providecommand{\lemmaname}{Lemma}
\providecommand{\propositionname}{Proposition}
\providecommand{\remarkname}{Remark}
\providecommand{\resultname}{Result}
\providecommand{\theoremname}{Theorem}

\begin{document}
\title{Analysis of feature learning in weight-tied autoencoders\\via the
mean field lens}
\author{Phan-Minh Nguyen\thanks{The Voleon Group. The work was done while the author was a Ph.D. candidate
at the department of Electrical Engineering, Stanford University.}}
\maketitle
\begin{abstract}
Autoencoders are among the earliest introduced nonlinear models for
unsupervised learning. Although they are widely adopted beyond research,
it has been a longstanding open problem to understand mathematically
the feature extraction mechanism that trained nonlinear autoencoders
provide.

In this work, we make progress in this problem by analyzing a class
of two-layer weight-tied nonlinear autoencoders in the mean field
framework. Upon a suitable scaling, in the regime of a large number
of neurons, the models trained with stochastic gradient descent are
shown to admit a mean field limiting dynamics. This limiting description
reveals an asymptotically precise picture of feature learning by these
models: their training dynamics exhibit different phases that correspond
to the learning of different principal subspaces of the data, with
varying degrees of nonlinear shrinkage dependent on the $\ell_{2}$-regularization
and stopping time. While we prove these results under an idealized
assumption of (correlated) Gaussian data, experiments on real-life
data demonstrate an interesting match with the theory.

The autoencoder setup of interests poses a nontrivial mathematical
challenge to proving these results. In this setup, the ``Lipschitz''
constants of the models grow with the data dimension $d$. Consequently
an adaptation of previous analyses requires a number of neurons $N$
that is at least exponential in $d$. Our main technical contribution
is a new argument which proves that the required $N$ is only polynomial
in $d$. We conjecture that $N\gg d$ is sufficient and that $N$
is necessarily larger than a data-dependent intrinsic dimension, a
behavior that is fundamentally different from previously studied setups.
\end{abstract}
\tableofcontents{}

\section{Introduction}

The recent surging interest in neural networks and the field deep
learning arguably started with the creation of a training technique
\cite{hinton2006fast,hinton2006reducing,ranzato2007efficient,bengio2007greedy}.
Underlying this technique was a class of nonlinear unsupervised learning
models, known as autoencoders \cite{rumelhart1985feature,ackley1985learning,rumelhart1985learning}.
During those early days of deep learning, this class of models again
played a key role in another major milestone, the famous ``Google
cat'' result \cite{le2012building}, where autoencoders were shown
to be able to ``detect'' high-level concepts such as cat faces from
a large unlabeled data set of images downloaded from the Internet.
As the field has become more mature, autoencoders are still found
to be useful in applications such as image processing \cite{mousavi2015deep}
and channel coding \cite{jiang2019turbo}. The models are also found
to display biological plausibility: when applied to natural movies,
they show certain resemblances with monkeys' retina after training
\cite{ocko2018emergence}. Yet despite more than a decade of progresses,
a solid mathematical foundation to understand the behavior during
training of these models is still missing. How do their training dynamics
look like? What data representation is being captured over the course
of training? These questions are challenging due to the complex, highly
non-convex nature of the training process, but an answer may give
a hint at how deep learning works and beyond.

In this paper, we study one such model in an analytically tractable
setting, while maintaining several important features of these models.
Namely, we consider a weight-tied two-layer autoencoder of the following
form:
\[
\hat{\boldsymbol{x}}\left(\boldsymbol{x};\boldsymbol{W}\right)=\frac{1}{N}\boldsymbol{W}^{\top}\sigma\left(\boldsymbol{W}\boldsymbol{x}\right),
\]
where $\boldsymbol{x}$ is the input, $\boldsymbol{W}\in\mathbb{R}^{N\times d}$
is the weight matrix, and $\sigma$ is the entry-wise nonlinear activation
function. Here $N$ is known as the width, or the number of neurons.
The weight-tying constraint is enforced by making the second layer's
weight the transposition of $\boldsymbol{W}$ the first layer's weight.
The model is trained by a stochastic gradient descent rule on the
$\ell_{2}$-regularized autoencoding problem of the following form:
\[
\min_{\boldsymbol{W}}\sum_{\boldsymbol{x}\in\text{ training set}}\left\Vert \boldsymbol{x}-\hat{\boldsymbol{x}}\left(\boldsymbol{x};\boldsymbol{W}\right)\right\Vert _{2}^{2}+\lambda_{{\rm reg}}\left\Vert \boldsymbol{W}\right\Vert _{{\rm F}}^{2},
\]
i.e. minimization of the squared loss with $\ell_{2}$-regularization,
where $\left\Vert \cdot\right\Vert _{{\rm F}}$ denotes the Frobenius
norm. We refer to Section \ref{sec:Overview} for the exact forms
of the model and its training algorithm. The training process learns
$\boldsymbol{W}$ and forms an encoding mapping $\boldsymbol{x}\mapsto\sigma\left(\boldsymbol{W}\boldsymbol{x}\right)$,
which gives a representation for each data point $\boldsymbol{x}$.
It is easy to see that the above autoencoding problem is non-convex.
In the special case where $\lambda_{{\rm reg}}=0$, one potential
solution is the identity mapping $\hat{\boldsymbol{x}}\left(\boldsymbol{x}\right)=\boldsymbol{x}$.
However even in that case, it is unclear from the optimization point
of view whether the training dynamics can find this solution. More
generally, from a representation learning point of view, there is
an interest to understand what $\boldsymbol{W}$ is learned in the
process.

To analyze the training dynamics of this model, we draw insights from
a recent theoretical advance, namely the mean field theory \cite{mei2018mean,mei2019mean,nguyen2019mean,nguyen2020rigorous}.
In particular, we consider over-complete autoencoders, which are ones
with very large $N$. It is crucial to note that the class of over-complete
weight-tied autoencoders is a standard architecture and has been found
to learn interesting features with appropriate training \cite{vincent2010stacked}.
When $N\to\infty$, under suitable scaling, the training dynamics
is shown to be precisely captured by a meaningful limit, known as
the mean field limit. This limit reveals interesting insights into
the inner-workings of the model. Indeed we shall see that the trained
autoencoder can exhibit a spectrum of behaviors: with suitable regularization,
the learned mapping $\boldsymbol{x}\mapsto\boldsymbol{W}\boldsymbol{x}$
performs a form of principal subspace selection via shrinkage with
a cut-off effect, whereas an unregularized autoencoder learns \textit{almost}
the identity mapping without any representation learning. Furthermore
the training dynamics exhibits a separation in time: the model progressively
learns from subspaces with higher importance -- relative to regularization
-- to less important ones. These are shown to hold for various nonlinear
activations $\sigma$, including the popular rectified linear unit
(ReLU) activation. While our theory builds up on an idealized setting
where the data $\boldsymbol{x}$ is drawn from a correlated zero-mean
Gaussian source, experiments on real-life data demonstrate a striking
agreement between the theory and empirical results. We conjecture
that a universality phenomenon takes place: in our autoencoder setup,
several properties of the learning dynamics are asymptotically the
same across a wide array of data distributions that have zero mean
and share the same covariance structure.

The mean field limit, roughly speaking, is an infinite-$N$ approximation
of the model. An important question is: how large should the number
of neurons $N$ be? It is known that under certain assumptions, one
only requires $N\gg O\left(1\right)$ independent of the data dimension
$d$ \cite{mei2019mean}. Unfortunately those assumptions fail to
hold in the present setting. A key fact is, unlike previous works,
here the ``Lipschitz'' constant of the model\footnote{Strictly speaking, our autoencoder model is non-Lipschitz in the parameter,
and neither is its initialization chosen to make the model effectively
Lipschitz over any finite training period as done in \cite{mei2019mean}.
This adds more complications to the analysis. The statement may be
interpreted as that the model is locally Lipschitz with a constant
that grows with $d$. Without taking the statement in the strict sense,
we stress on the underlying difficulty dealing with the dependency
on $d$.} grows with $d$. This not only poses a major mathematical challenge
but also leads to a fundamentally different result. A naive adaption
of previous analyses would lead to $N\gg\exp\left(d\right)$ undesirably.
A major technical feat of the paper is to show that one only requires
$N\gg{\rm poly}\left(d\right)$. Proving this result necessitates
a new argument which, unlike previous analyses, crucially exploits
the structure of the gradient flow learning dynamics. In fact, we
prove so in a more general framework of a broader class of two-layer
neural networks. Furthermore we believe that on one hand, $N\gg d$
is generally sufficient, and under special circumstances, so is $N\gg d_{{\rm eff}}$,
where the quantity $d_{{\rm eff}}$ is characteristic of the data
distribution. In general, $d_{{\rm eff}}$ can be on the same order
of or much smaller than $d$. On the other hand, we also conjecture
that $N\gg d_{{\rm eff}}$ is necessary, and hence unlike previous
settings \cite{mei2018mean,mei2019mean}, here it is generally insufficient
to have $N\gg O\left(1\right)$.

It has been known for a long time that under-complete weight-untied
autoencoders with a linear activation essentially perform principal
component analysis, if optimized with the squared loss \cite{bourlard1988auto,baldi1989neural}.
In our setting, at a high level, the autoencoder after training has
a similar effect with nonlinear shrinkage. Furthermore when the activation
function is the ReLU, the model also tends to learn a linear mapping,
and in the absence of regularization, this linear mapping is precisely
the identity mapping. This latter point may seem at odds with the
expectation that the weight-tying constraint will force the autoencoder
to learn a nonlinear mapping by discouraging it from ``\textit{stay(ing)
in the linear regime of its nonlinearity without paying a high price
in reconstruction error}'' -- quoted from the influential work \cite{vincent2010stacked}.
In fact, the role of the training dynamics, typically missing from
the discussions in those works, is important in our case. A key lesson
from our analysis is the following: the over-complete weight-tied
autoencoder, trained with the random weight initialization as in the
usual practice and the $\ell_{2}$-regularized squared loss, has the
tendency to maintain rotational invariance along its gradient descent
trajectory. Even though the pre-activation values of individual neurons
substantially occupy the nonlinear region of the activation function
$\sigma$, due to rotational invariance, the resultant model nevertheless
tends to favor less complex mappings. When the activation is the ReLU
which is a homogeneous function, the result is then a linear mapping.
When a generic nonlinear activation is used, the result is in general
a mildly nonlinear one. This situation is to be contrasted with under-complete
linear autoencoders, in which case the optimization landscape is benign
with essentially one unique (local and also global) minimizer \cite{baldi1989neural}
and therefore the training dynamics is not a crucial factor. In short,
while the resultant unsupervised learning effects are similar, the
causes are drastically different in nature. Of course, even this relatively
simple story has not been shown before for nonlinear over-complete
autoencoders. We note that the more challenging bulk of the work is
actually to prove that rotational invariance is maintained under the
requirement $N\gg{\rm poly}\left(d\right)$.

Finally let us mention two important directions for future studies:
(i) The effect of regularization methods beyond $\ell_{2}$-regularization.
We have focused on $\ell_{2}$-regularization, given the amount of
technical works that go into proving the results. Technical ideas
in this work should be applicable to setups with more sophisticated
regularizations. (ii) The learning dynamics of over-complete autoencoders
with more than two layers. New ideas and advances in the mean field
theory for multilayer networks \cite{nguyen2020rigorous,pham2020note}
could be useful in this direction.

\subsection{Relation with the literature}

\paragraph{Theoretical studies of autoencoders.}

Autoencoders and related architectures have been studied from a variety
of angles: representational power \cite{le2008representational,montufar2011refinements},
optimal autoencoding mappings in vanishing regularization \cite{alain2014regularized},
sparsity properties \cite{arpit2015regularized}, landscape properties
\cite{rangamani2018sparse,kunin2019loss}, initialization with random
weights \cite{li2018on}, memorization \cite{radhakrishnan2018memorization,zhang2019identity,radhakrishnan2020overparameterized}.
Closely related to our work are the recent works on the training dynamics
of autoencoders \cite{nguyen2019dynamics,nguyen2019benefits,gidel2019implicit,bao2020regularized}.
In particular, \cite{nguyen2019dynamics} studies the gradient descent
dynamics of weight-tied shallow under-complete autoencoders that are
initialized in a local neighborhood of certain assumed ground truth
models; \cite{nguyen2019benefits} studies weight-untied shallow over-complete
autoencoders in the lazy training regime \cite{chizat2019lazy} in
which the weights hardly evolve during training; \cite{gidel2019implicit}
establishes the exact solution to the gradient descent dynamics of
unregularized shallow autoencoders with a linear activation; \cite{bao2020regularized}
studies the task of recovering the underlying data structure with
suitably regularized shallow under-complete linear autoencoders and
gradient-based algorithms. Unlike these works, our work studies the
stochastic gradient descent training of weight-tied over-complete
autoencoders with random initializations and nonlinear activations
in a regime where the weights evolve nonlinearly. Our theoretical
finding, that the autoencoder can perform from some to zero degree
of representation learning depending on how it is regularized, complements
the recent literature on memorization in autoencoders \cite{radhakrishnan2018memorization,zhang2019identity,radhakrishnan2020overparameterized}.

Several features of the learning dynamics that we show for our autoencoder
setups resemble the behaviors of linear neural networks \cite{saxe2013exact,advani2017high,saxe2019mathematical,gidel2019implicit}
and nonlinear networks under very strong assumptions \cite{combes2018learning}.
Given the strong recent interest in analyses of the learning trajectory
of neural networks, our work solidifies and furthers understanding
in this research area.

\paragraph{Mean field theory of neural networks.}

The mean field view on the training dynamics of neural networks has
enjoyed numerous efforts from multiple groups of authors, firstly
with two-layer networks \cite{nitanda2017stochastic,mei2018mean,chizat2018,rotskoff2018neural,sirignano2018mean}
and more recently with multilayer ones \cite{nguyen2019mean,araujo2019mean,nguyen2020rigorous}.
This view has found successes in proving global convergence guarantees
\cite{mei2018mean,chizat2018,rotskoff2018neural,javanmard2019analysis,nguyen2020rigorous,pham2020note,wojtowytsch2020convergence,fang2020modeling},
inspiring new training algorithms \cite{wei2019regularization,rotskoff2019global},
studying stability properties of the trained networks \cite{shevchenko2019landscape},
other architectures which are compositions of multiple mean field
neural networks \cite{ma2019machine,lu2020mean} and other machine
learning contexts \cite{agazzi2020global}. It is associated with
a particular choice of scaling as one allows the number of neurons
to tend to infinity. The matter of scaling turns out to be important,
as found by several recent works \cite{chizat2019lazy,geiger2019disentangling,ghorbani2020neural,ma2020quenching}.
A key feature of the mean field scaling is that the parameters are
able to evolve in a nonlinear non-degenerate fashion and the network
is expected to enjoy meaningful learning. On the other hand, the analysis
of the mean field limit is typically challenging.

Our work follows this long line of works with two new contributions.
Firstly in these previous works, the mean field limit is typically
described as the solution of a certain differential equation, and
no specific high-dimensional setup has been found with an explicit
closed-form solution. The weight-tied ReLU autoencoder we study provides
one such example: its completely explicit solution allows to demonstrate
properties that are previously unproven for nonlinear neural networks
in the mean field limit. Secondly we provide a framework for a class
of two-layer networks with structural assumptions that are not covered
by previous works. These assumptions pose a highly nontrivial technical
challenge. We overcome it with a new argument on top of the usual
propagation of chaos argument \cite{sznitman1991topics} that has
been routinely used in previous analyses \cite{mei2018mean,mei2019mean,nguyen2020rigorous}.
We also differ by answering a different set of questions in unsupervised
learning. For example, previous studies take a keen interest in the
optimization aspects of the training process of neural networks, in
particular global convergence guarantees and convergence rates (see
e.g. \cite{chizat2018,mei2018mean,nguyen2020rigorous,pham2020note,javanmard2019analysis,chizat2019sparse}).
In our specific setting, these questions are straightforwards thanks
to the explicit solution to the mean field limit, but are not the
focus of our study.

\subsection{Organization}

We give an overview of our main contributions and their analyses in
Section \ref{sec:Overview}. This section is the more conceptual part
of the paper. As introduced, our work presents two main contributions:
a mean field limit result for a class of two-layer neural networks,
and its application to the weight-tied autoencoders. We formally state
and prove the first contribution in Section \ref{sec:MF_two-layers}
and the second contribution in Section \ref{sec:Autoencoders}. These
latter two sections are the more technical part of the paper.

\subsection{Notations}

Dimensions play an important role in this work. We shall routinely
mention a dimension vector $\mathfrak{Dim}=\left(D,D_{{\rm in}},D_{{\rm out}}\right)$
in the context of more general two-layer neural networks (Sections
\ref{subsec:contrib-MF} and \ref{sec:MF_two-layers}), in which $D$,
$D_{{\rm in}}$ and $D_{{\rm out}}$ are some dimension quantities.
When specialized to the specific context of autoencoders which involves
only one dimension quantity $d$ (Sections \ref{subsec:contrib-Gaussian-data},
\ref{subsec:contrib-real-data} and \ref{sec:Autoencoders}), $\mathfrak{Dim}=\left(D,D_{{\rm in}},D_{{\rm out}}\right)=\left(d,d,d\right)$.
We reserve the notations $\kappa$, $\kappa_{*}$, $\kappa_{1}$,
$\kappa_{2}$, etc for constant parameters that depend exclusively
on $\mathfrak{Dim}$.

We use $C$ for different constants which may differ at different
instances of use and do not depend on the number of neurons $N$,
the learning rate $\epsilon$, and the dimension vector $\mathfrak{Dim}=\left(D,D_{{\rm in}},D_{{\rm out}}\right)$.
The exact dependency of $C$ shall be clarified in the specific contexts.
We shall also write $a\lesssim b$, $a\simeq b$ and $a\gtrsim b$
as shorthands for $a\leq Cb$, $a=Cb$ and $a\geq Cb$ respectively
for such constants $C$.

For a positive integer $n$, we let $\left[n\right]$ denote the set
$\left\{ 1,2,...,n\right\} $. For a set $S$, we use $\text{Unif}\left(S\right)$
to denote the uniform distribution over $S$. We use $\left\Vert \cdot\right\Vert _{2}$
to denote the usual Euclidean norm for a vector, and $\left\Vert \cdot\right\Vert _{{\rm op}}$
and $\left\Vert \cdot\right\Vert _{{\rm F}}$ for the operator norm
and the Frobenius norm of a matrix. For a matrix $\boldsymbol{A}$,
we let ${\rm Proj}_{\boldsymbol{A}}$ be the projection onto the subspace
spanned by columns of $\boldsymbol{A}$, and ${\rm Proj}_{\boldsymbol{A}}^{\perp}=\boldsymbol{I}-{\rm Proj}_{\boldsymbol{A}}$
its orthogonal projection. For three vectors $\boldsymbol{u}$, $\boldsymbol{a}$
and $\boldsymbol{b}$, We write $\boldsymbol{u}\in\left[\boldsymbol{a},\boldsymbol{b}\right]$
to mean that $\boldsymbol{u}$ lies on the segment between $\boldsymbol{a}$
and $\boldsymbol{b}$, i.e. $\boldsymbol{u}=c\boldsymbol{a}+\left(1-c\right)\boldsymbol{b}$
for some $c\in\left[0,1\right]$. We let ${\cal B}_{d}\left(r\right)$
denote the ball $\left\{ \boldsymbol{u}\in\mathbb{R}^{d}:\;\left\Vert \boldsymbol{u}\right\Vert _{2}\leq r\right\} $.

For a topological space $S$, we use $\mathscr{P}\left(S\right)$
to denote the set of probability measures over $S$ (with its associated
Borel sigma-algebra being implicitly defined). We reserve the letter
$g$ for a standard Gaussian random variable $g\sim\mathsf{N}\left(0,1\right)$.
We use ${\cal P}$ to denote the data distribution, and $\mathbb{E}_{{\cal P}}$
to denote the expectation with respect to (w.r.t.) ${\cal P}$. For
sub-Gaussian and sub-exponential random variables, we use $\left\Vert \cdot\right\Vert _{\psi_{2}}$
and $\left\Vert \cdot\right\Vert _{\psi_{1}}$ to denote their respective
Orlicz norms (see Appendix \ref{subsec:Sub-Gaussian-RV} for definitions).

For a function $f\left(u_{1},...,u_{k}\right)$, we use $\partial_{j}f$
or $\partial_{u_{j}}f$ (respectively, $\nabla_{j}f$ or $\nabla_{u_{j}}f$)
to denote the partial derivative (respectively, gradient) w.r.t. the
$j$-th variable $u_{j}$. For a function $f:\;\mathbb{R}^{n}\times\mathbb{R}^{m}\to\mathbb{R}$
and its partial gradient $\nabla_{1}f$ w.r.t. the first variable,
with an abuse of notations, we let $\nabla_{111}^{3}f$ be the second-order
Fr\textipa{\' e}chet partial derivative of $\nabla_{1}f$ w.r.t.
the first variable, i.e. $\nabla_{111}^{3}f\equiv\nabla_{11}^{2}\left(\nabla_{1}f\right)$.
For each $\boldsymbol{u}_{1}\in\mathbb{R}^{n}$ and $\boldsymbol{u}_{2}\in\mathbb{R}^{m}$,
we define the operator norm of $\nabla_{111}^{3}f\left[\boldsymbol{u}_{1},\boldsymbol{u}_{2}\right]:\;\mathbb{R}^{n}\times\mathbb{R}^{n}\to\mathbb{R}^{n}$
-- which is a linear operator -- as follows:
\begin{align*}
\left\Vert \nabla_{111}^{3}f\left[\boldsymbol{u}_{1},\boldsymbol{u}_{2}\right]\right\Vert _{{\rm op}} & =\sup_{\boldsymbol{a},\boldsymbol{b},\boldsymbol{c}\in\mathbb{S}^{n-1}}\left\langle \boldsymbol{c},\nabla_{111}^{3}f\left[\boldsymbol{u}_{1},\boldsymbol{u}_{2}\right]\left(\boldsymbol{a},\boldsymbol{b}\right)\right\rangle .
\end{align*}
With an abuse of notations, we also use $\nabla_{111}^{3}f\left[\boldsymbol{u}_{1},\boldsymbol{u}_{2}\right]$
to denote a tensor in $\left(\mathbb{R}^{n}\right)^{\otimes3}$ such
that
\[
\left\langle \boldsymbol{c},\nabla_{111}^{3}f\left[\boldsymbol{u}_{1},\boldsymbol{u}_{2}\right]\left(\boldsymbol{a},\boldsymbol{b}\right)\right\rangle =\left\langle \nabla_{111}^{3}f\left[\boldsymbol{u}_{1},\boldsymbol{u}_{2}\right],\boldsymbol{a}\otimes\boldsymbol{b}\otimes\boldsymbol{c}\right\rangle .
\]
We define similarly: $\nabla_{121}^{3}f$ is the Fr\textipa{\' e}chet
cross partial derivative of $\nabla_{1}f$ w.r.t. the second variable
and then the first variable, and $\nabla_{122}^{3}f$ is the second-order
Fr\textipa{\' e}chet partial derivative of $\nabla_{1}f$ w.r.t.
the second variable, i.e. $\nabla_{121}^{3}f\equiv\nabla_{21}^{2}\left(\nabla_{1}f\right)$
and $\nabla_{122}^{3}f\equiv\nabla_{22}^{2}\left(\nabla_{1}f\right)$.

\subsection*{Acknowledgment}

The work was partially supported by grants NSF CCF-1714305 and ONR
N00014-18-1-2729. We would like to thank Andrea Montanari for initiating
the research and many helpful discussions, Marco Mondelli for several
exploratory discussions in the early stage of the project, and Huy
Tuan Pham for brilliant suggestions, one of which is the reference
\cite{goodman1990convergence}. This work forms a part of the author's
Ph.D. dissertation submitted to Stanford University in June 2020.

\section{Main contributions: An overview\label{sec:Overview}}

\subsection{Dynamics of weight-tied autoencoders: Gaussian data\label{subsec:contrib-Gaussian-data}}

We consider a weight-tied autoencoder with the following form:
\begin{equation}
\hat{\boldsymbol{x}}_{N}\left(\boldsymbol{x};\Theta\right)=\frac{1}{N}\sum_{i=1}^{N}\kappa\boldsymbol{\theta}_{i}\sigma\left(\left\langle \kappa\boldsymbol{\theta}_{i},\boldsymbol{x}\right\rangle \right),\qquad\kappa=\sqrt{d},\label{eq:ae_simplified}
\end{equation}
where $\boldsymbol{x}\in\mathbb{R}^{d}$ is the input, $\Theta=\left(\boldsymbol{\theta}_{i}\right)_{i\leq N}$
is the collection of weights $\boldsymbol{\theta}_{i}\in\mathbb{R}^{d}$.
Here $N$ is the number of neurons and $d$ is the dimension. This
is the usual weight-tied autoencoder without the bias. The factor
$\kappa=\sqrt{d}$ represents a scaling w.r.t. the dimension $d$,
which we shall clarify later. The data $\boldsymbol{x}$ is distributed
according to $\boldsymbol{x}\sim{\cal P}$. We train the network with
stochastic gradient descent (SGD). At each SGD iteration $k$, we
draw independently the data $\boldsymbol{x}^{k}\sim{\cal P}$. Let
$\Theta^{k}=\left(\boldsymbol{\theta}_{i}^{k}\right)_{i=1}^{N}$ be
the collection of weights at iteration $k$. Given an initialization
$\Theta^{0}$, we perform the SGD update w.r.t. the squared loss with
$\ell_{2}$-regularization:
\[
\boldsymbol{\theta}_{i}^{k+1}=\boldsymbol{\theta}_{i}^{k}-\epsilon N\nabla_{\boldsymbol{\theta}_{i}}{\rm Loss}\left(\boldsymbol{x}^{k};\Theta^{k}\right),\qquad i=1,...,N,
\]
with the training loss being
\[
{\rm Loss}\left(\boldsymbol{x};\Theta\right)=\frac{1}{2}\left\Vert \hat{\boldsymbol{x}}_{N}\left(\boldsymbol{x};\Theta\right)-\boldsymbol{x}\right\Vert _{2}^{2}+\frac{\lambda}{N}\sum_{i=1}^{N}\left\Vert \boldsymbol{\theta}_{i}\right\Vert _{2}^{2}.
\]
Here $\epsilon>0$ is the learning rate and $\lambda\geq0$ is the
regularization strength. We shall concern with the population squared
loss as a measure of reconstruction quality (which we shall call the
\textit{reconstruction error}):
\[
{\rm RecErr}\left(\Theta\right)=\mathbb{E}_{{\cal P}}\left\{ \frac{1}{2}\left\Vert \hat{\boldsymbol{x}}_{N}\left(\boldsymbol{x};\Theta\right)-\boldsymbol{x}\right\Vert _{2}^{2}\right\} ,
\]
although the training loss additionally includes the $\ell_{2}^{2}$-regularization
penalty.

We note two key differences that set the mean field regime apart from
the usual scalings: the factor $1/N$ in $\hat{\boldsymbol{x}}_{N}\left(\boldsymbol{x};\Theta\right)$,
and the factor $N$ being multiplied to the gradient update of $\boldsymbol{\theta}_{i}^{k+1}$.

\subsubsection{Setting with ReLU activation: SGD dynamics\label{subsec:contrib-AE-ReLU-SGD}}

Our first result concerns with the SGD dynamics in the case of ReLU
activation.
\begin{res}[Autoencoder with ReLU -- Informal and simplified]
\label{res:ReLU_setting_simplified}Consider the autoencoder, as
described in Section \ref{subsec:contrib-Gaussian-data}, in the following
setting. The data $\boldsymbol{x}$ assumes a Gaussian distribution
with the following mean and covariance:
\[
\mathbb{E}\left\{ \boldsymbol{x}\right\} =\boldsymbol{0},\qquad\mathbb{E}\left\{ \boldsymbol{x}\boldsymbol{x}^{\top}\right\} =\frac{1}{d}\boldsymbol{R}{\rm diag}\left(\Sigma_{1}^{2},...,\Sigma_{d}^{2}\right)\boldsymbol{R}^{\top},
\]
where $\boldsymbol{R}$ is an orthogonal matrix, $\Sigma_{1}\geq...\geq\Sigma_{d}>0$,
$\Sigma_{1}\leq C$ and $\Sigma_{d}\geq C\kappa_{*}$ for some $\kappa_{*}=1/{\rm poly}\left(d\right)$.
The activation $\sigma$ is the ReLU: $\sigma\left(a\right)=\max\left(0,a\right)$.
The regularization strength $0\leq\lambda\leq C$. The initialization
$\Theta^{0}=\left(\boldsymbol{\theta}_{i}^{0}\right)_{i\leq N}\sim_{{\rm i.i.d.}}\mathsf{N}\left(\boldsymbol{0},r_{0}^{2}\boldsymbol{I}_{d}/d\right)$
for a non-negative constant $r_{0}\leq C$.

Then for $N\gg{\rm poly}\left(d\right)$, $\epsilon\ll1/{\rm poly}\left(d\right)$
and a finite $t\in\mathbb{N}\epsilon$, $t\leq C$, with high probability,
\begin{align}
\frac{1}{N}\sum_{i=1}^{N}\delta_{\boldsymbol{\theta}_{i}^{t/\epsilon}} & \approx\mathsf{N}\left(\boldsymbol{0},\frac{1}{d}\boldsymbol{R}{\rm diag}\left(r_{1,t}^{2},...,r_{d,t}^{2}\right)\boldsymbol{R}^{\top}\right),\label{eq:thm_ReLU_simplified_1}\\
{\rm RecErr}\left(\Theta^{t/\epsilon}\right) & \approx\frac{1}{2d}\sum_{i=1}^{d}\Sigma_{i}^{2}\left(1-\frac{1}{2}r_{i,t}^{2}\right)^{2}.\label{eq:thm_ReLU_simplified_2}
\end{align}
Here $r_{i,t}\geq0$ satisfies
\begin{equation}
r_{i,t}^{2}=\frac{2r_{0}^{2}\eta_{i}}{r_{0}^{2}\Sigma_{i}^{2}-\left(r_{0}^{2}\Sigma_{i}^{2}-2\eta_{i}\right)e^{-2\eta_{i}t}},\qquad\eta_{i}=\Sigma_{i}^{2}-2\lambda.\label{eq:thm_ReLU_simplified_3}
\end{equation}
In the above, the constants $C$ do not depend on $N$, $\epsilon$
or $d$.
\end{res}

Exact details can be found in the statement of Theorem \ref{thm:ReLU_setting}.

Result \ref{res:ReLU_setting_simplified} describes the behavior of
the weights, as well as the reconstruction error, of the autoencoder
with ReLU activation under Gaussian data (with non-identity covariance).
These are governed by the continuous-time dynamics of the quantities
$\left(r_{i,t}\right)_{i\leq N}$. Observe that $r_{i,t}=O\left(1\right)$,
and hence the right-hand side of Eq. (\ref{eq:thm_ReLU_simplified_1})
suggests that $\left\Vert \boldsymbol{\theta}_{i}^{t/\epsilon}\right\Vert _{2}=O\left(1\right)$.
This is the effect of the scaling by $\kappa$ (see also Section \ref{subsec:Autoencoder-example}).
Likewise Eq. (\ref{eq:thm_ReLU_simplified_2}) suggests that the reconstruction
error remains $O\left(1\right)$ throughout the training dynamics.
Notably the requirement on $N$ and $\epsilon$ is relatively mild:
we only require $N\gg{\rm poly}\left(d\right)$ and $\epsilon\ll1/{\rm poly}\left(d\right)$.
We believe that the requirement $\kappa_{*}=1/{\rm poly}\left(d\right)$
could be relaxed (for instance, $\kappa_{*}$ could decay faster than
a polynomial rate while still allowing $N\gg{\rm poly}\left(d\right)$
and $\epsilon\ll1/{\rm poly}\left(d\right)$), but proving this is
not possible with our current analysis.

Eq. (\ref{eq:thm_ReLU_simplified_1}) further elucidates the role
of the weights: roughly speaking, each $\boldsymbol{\theta}_{i}^{t/\epsilon}$
performs a \textit{random rescaled projection} onto the principal
subspaces of the data distribution ${\cal P}$. Here we recall each
1-dimensional principal subspace aligns with the direction of a column
of $\boldsymbol{R}$, the matrix of eigenvectors of the data covariance.
As such, $r_{i,t}$ indicates the \textit{rescaling factor} at iteration
$t/\epsilon$, corresponding to the $i$-th principal subspace.

We now make several more detailed observations from Result \ref{res:ReLU_setting_simplified}:

\paragraph*{Independent evolution of the rescaling factors.}

We observe from Eq. (\ref{eq:thm_ReLU_simplified_3}) that for each
$i$, the evolution of $r_{i,t}$ does not depend on other indices.
As such, the evolution of one principal subspace is decoupled from
others. This fact is particular to the ReLU and does not hold for
generic nonlinear activations, as discussed in Section \ref{subsec:contrib-bounded}.

\paragraph*{Bad stationary point at the origin.}

If $r_{0}=0$, $r_{i,t}=0$ for all $t$. Hence the origin is a bad
stationary point, which one must initialize away from in order for
meaningful learning to take place. This situation is drastically different
from 1-hidden-layer autoencoders\footnote{More specifically, the work \cite{radhakrishnan2018memorization}
considers an autoencoder of the form $\hat{\boldsymbol{x}}=\sigma\left(\boldsymbol{W}\boldsymbol{x}\right)$,
where $\boldsymbol{x}\in\mathbb{R}^{d}$ is the input, $\boldsymbol{W}\in\mathbb{R}^{d\times d}$
is the weight matrix, $\sigma$ is the activation function and $\hat{\boldsymbol{x}}\in\mathbb{R}^{d}$
is the output.} \cite{radhakrishnan2018memorization}.

\paragraph*{Sigmoidal evolution.}

The evolution curve of $r_{i,t}$ takes a sigmoidal shape, since $r_{i,t}$
changes exponentially with $t$ according to Eq. (\ref{eq:thm_ReLU_simplified_3}).
This suggests that the reconstruction error displays a shape that
superimposes several sigmoidal curves of different changing speeds
and magnitudes. See Fig. \ref{fig:ReLU_twoblks_noreg_loss_theta}
for illustration.

\begin{figure}
\begin{centering}
\subfloat[]{\begin{centering}
\includegraphics[width=0.45\columnwidth]{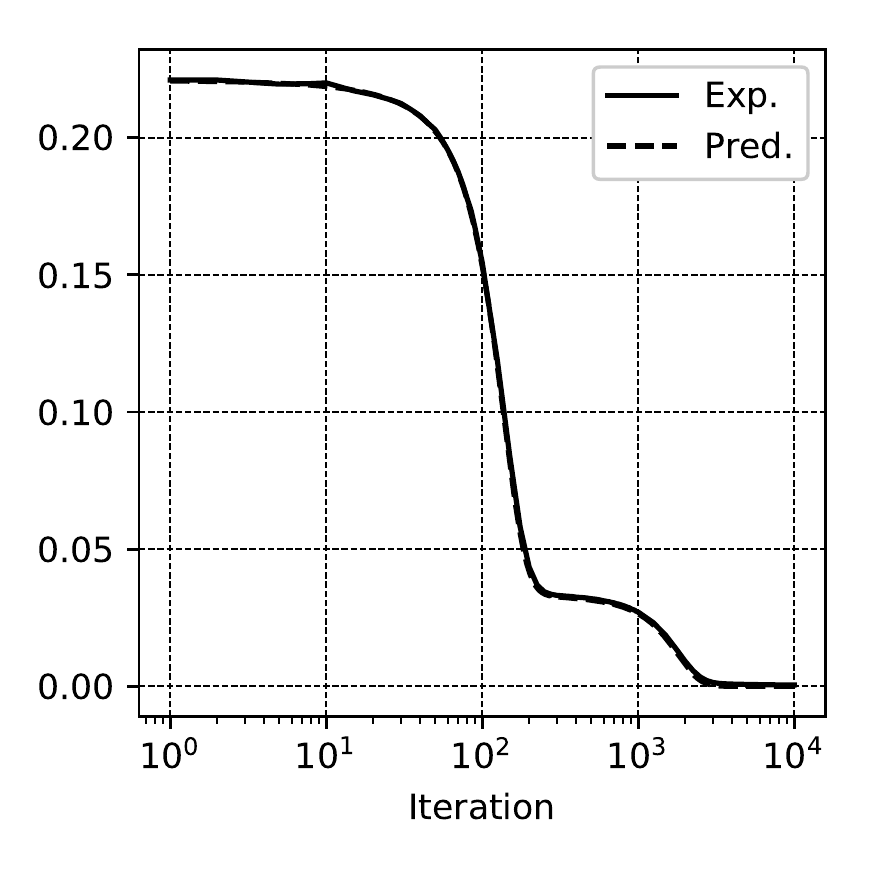}
\par\end{centering}
}\subfloat[]{\begin{centering}
\includegraphics[width=0.45\columnwidth]{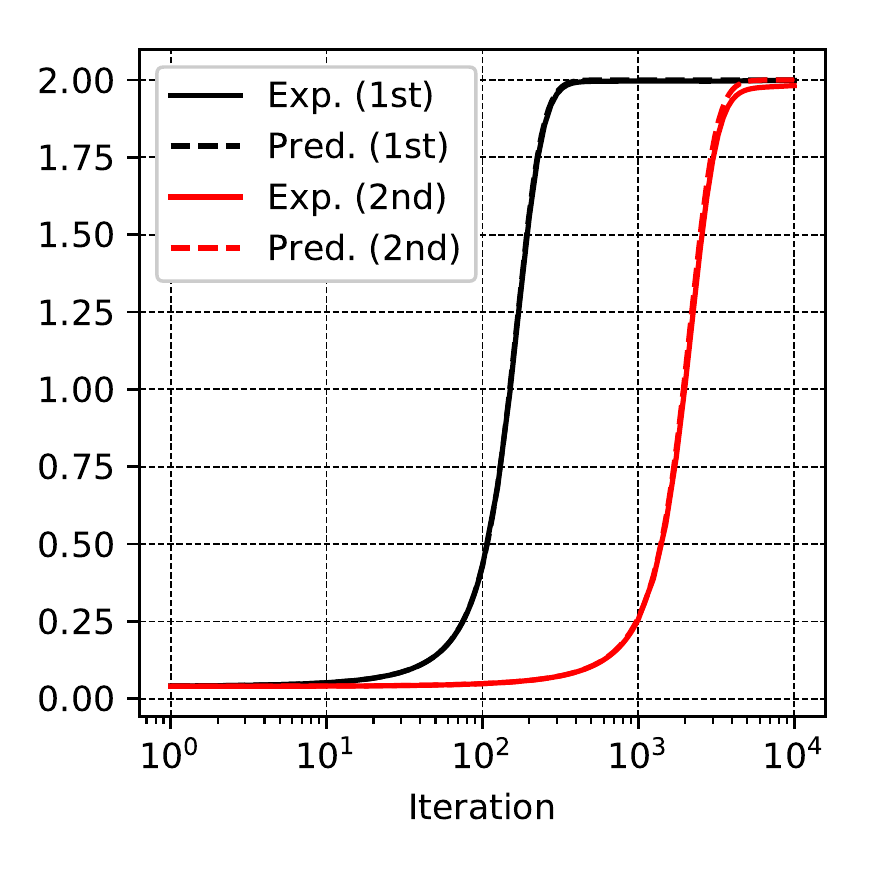}
\par\end{centering}
}
\par\end{centering}
\caption{Autoencoder with ReLU activation and Gaussian data, no regularization
(Result \ref{res:ReLU_setting_simplified}). Setup: $d=200$, $\Sigma_{1}^{2}=...=\Sigma_{60}^{2}=1.3$
and $\Sigma_{61}^{2}=...=\Sigma_{200}^{2}=0.1$, $\boldsymbol{R}=\boldsymbol{I}_{d}$,
$\lambda=0$, $r_{0}=0.2$, $\epsilon=0.01$ and $N=10000$. (a):
the reconstruction error versus the SGD iteration. (b): the normalized
squared norm of the first 60-dimensional subspace's weight (tagged
``1st'') and the second 140-dimensional subspace's weight (tagged
``2nd''). Here ``Exp.'' indicates the simulation results, and
``Pred.'' indicates the theoretical prediction. For more details,
see Appendix \ref{sec:Simulation-details}. We observe that the eventual
reconstruction error is almost zero, and the normalized squared norms
of the two subspaces' weights both tend to $2$ eventually. We also
observe that the reconstruction error, as a function of time, displays
a shape of two sigmoids that are superimposed onto each other, have
different magnitudes, have some time lag between each other and evolve
correspondingly to the normalized squared norm of the subspaces. The
learning speed of the second subspace is slower, since it has smaller
$\Sigma_{i}$. \\Early stopping can perform representation learning
in this example. A reasonable choice for early stopping is to stop
at the iteration $5\times10^{2}$. In particular, the first subspace
would then be reconstructed, whereas the second subspace has its corresponding
weight norm being small and hence is suppressed.}

\label{fig:ReLU_twoblks_noreg_loss_theta}
\end{figure}

\paragraph*{No regularization equals (efficient) learning of the identity.}

In the case $\lambda=0$ (no regularization) and $r_{0}>0$, Result
\ref{res:ReLU_setting_simplified} shows that as $t\to\infty$, we
have $r_{i,t}\to\sqrt{2}$ for any $i\in\left[d\right]$ and the reconstruction
error tending to $0$. In other words, the autoencoder is able to
reconstruct the Gaussian data source ${\cal P}$ to arbitrary precision,
with sufficiently large $N$ and sufficiently small $\epsilon$. This
holds for any finite $d$.

What is the required sample complexity w.r.t. the data dimension $d$?
Assume that $\kappa_{*}=C>0$, which implies we need $t\gg\max_{i}1/\Sigma_{i}^{2}=\Theta\left(1\right)$
in order for $r_{i,t}\approx\sqrt{2}$ for all $i\in\left[N\right]$.
Recall from Result \ref{res:ReLU_setting_simplified} that $\epsilon\ll1/{\rm poly}\left(d\right)$.
As such, the required number of SGD data samples -- which is $t/\epsilon$
-- is then only about ${\rm poly}\left(d\right)$. Note that this
sample complexity is independent of the number of neurons $N$, as
a consequence of the mean field scaling.

Interestingly, since ${\cal P}$ is a non-degenerate Gaussian source
and hence supported on $\mathbb{R}^{d}$, in this case, the fact that
the reconstruction error tends to $0$ implies the autoencoder is
bound to learn the identity function. This is the extreme of perfect
reconstruction but no representation learning. We also note that since
$\lambda=0$, the reconstruction error equals the training loss and
hence is non-increasing with time, as a simple consequence of gradient
flow evolution. See Fig. \ref{fig:ReLU_twoblks_noreg_loss_theta}
for illustration.

\paragraph*{Regularization equals principal subspace selection via shrinkage.}

In the case $\lambda>0$ and $r_{0}>0$, a critical phenomenon takes
place: as $t\to\infty$, $r_{i,t}\to\sqrt{2\left(1-2\lambda/\Sigma_{i}^{2}\right)}$
if $\Sigma_{i}^{2}>2\lambda$, $r_{i,t}\to0$ if $\Sigma_{i}^{2}<2\lambda$
and $r_{i,t}=r_{0}$ otherwise. In other words, $\ell_{2}$-regularization
performs a form of nonlinear shrinkage, controlled by $\lambda$,
and hence induces feature selection: the principal subspace $i$ with
sufficiently small $\Sigma_{i}$ is shrunk to zero and hence eliminated,
whereas the subspace with sufficiently large $\Sigma_{i}$ is selected.
The trade-off is that all selected principal subspaces are also shrunk.
This is one way the autoencoder performs representation learning.
We also note that since $\lambda>0$, the reconstruction error does
not equal the training loss and hence is not necessarily monotonic
with time, unlike the unregularized case; its time dependency is in
general complex. See Fig. \ref{fig:ReLU_twoblks_smallreg_loss_theta}
and \ref{fig:ReLU_twoblks_largereg_loss_theta} for illustration.

\begin{figure}
\begin{centering}
\subfloat[]{\begin{centering}
\includegraphics[width=0.45\columnwidth]{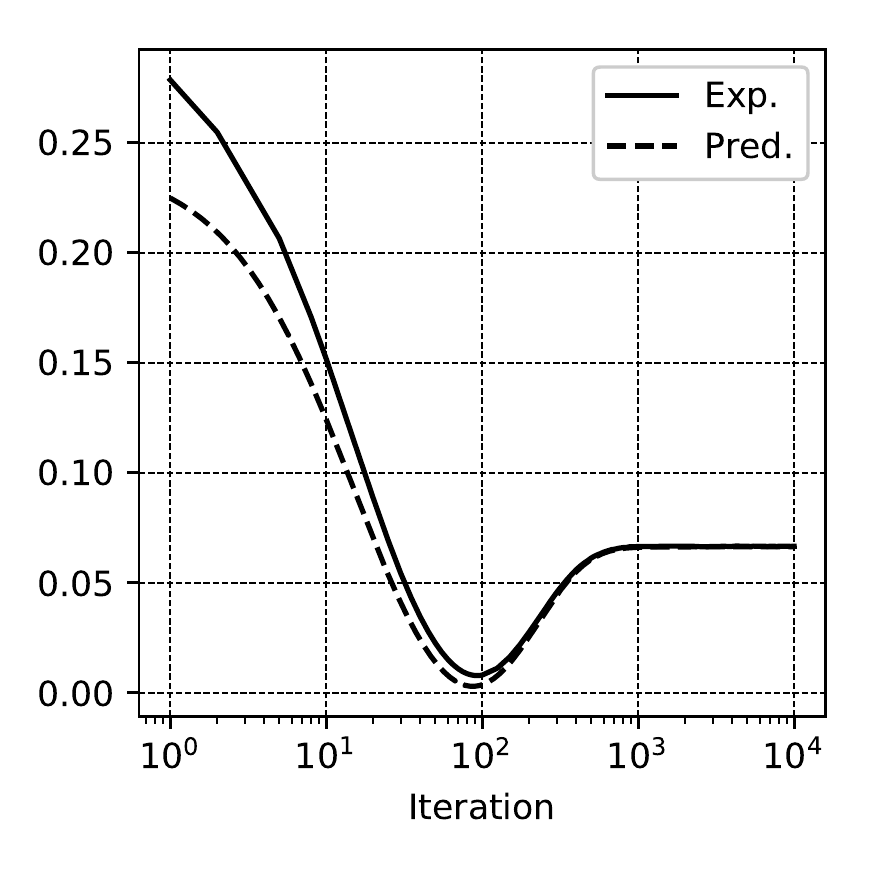}
\par\end{centering}
}\subfloat[]{\begin{centering}
\includegraphics[width=0.45\columnwidth]{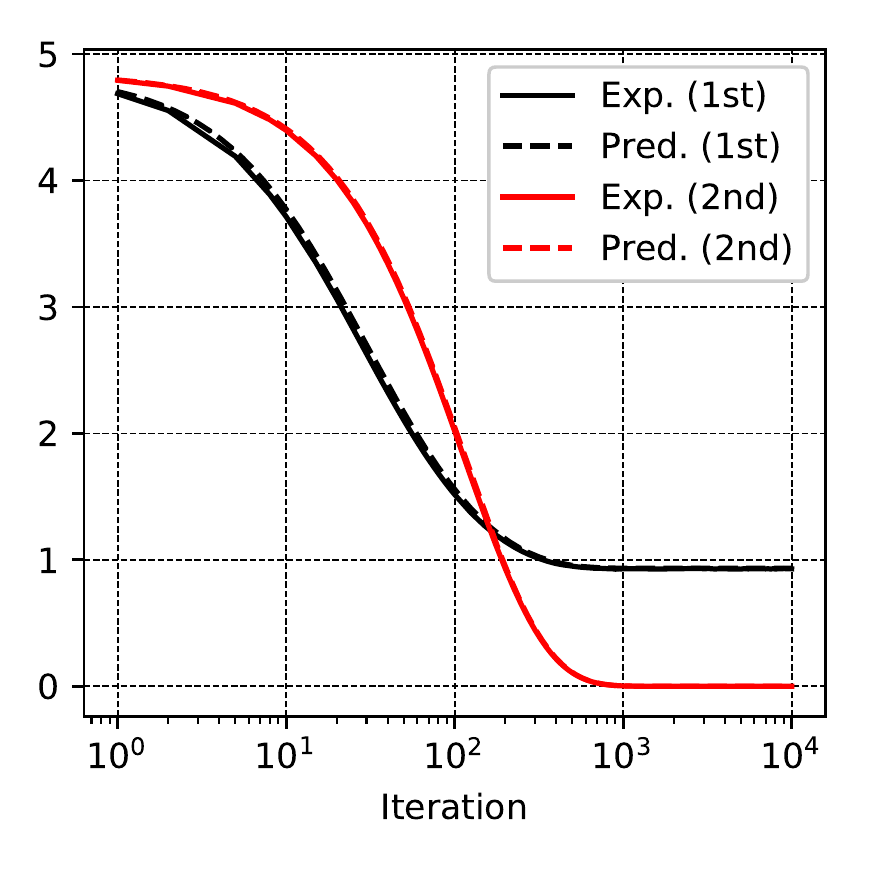}
\par\end{centering}
}
\par\end{centering}
\caption{Autoencoder with ReLU activation and Gaussian data, with moderate
regularization (Result \ref{res:ReLU_setting_simplified}). Setup:
$d=500$, $\Sigma_{1}^{2}=...=\Sigma_{50}^{2}=1.5$ and $\Sigma_{51}^{2}=...=\Sigma_{500}^{2}=0.1$,
$\boldsymbol{R}=\boldsymbol{I}_{d}$, $\lambda=0.4$, $r_{0}=2.2$,
$\epsilon=0.005$ and $N=10000$. (a): the reconstruction error versus
the SGD iteration. (b): the normalized squared norm of the first 50-dimensional
subspace's weight (tagged ``1st'') and the second 450-dimensional
subspace's weight (tagged ``2nd''). Here ``Exp.'' indicates the
simulation results, and ``Pred.'' indicates the theoretical prediction.
For more details, see Appendix \ref{sec:Simulation-details}. We observe
that the first subspace is selected (its weight remains non-zero eventually),
while the second subspace is eliminated (its weight becomes zero eventually).
The first subspace is shrunk owing to the regularization: the normalized
squared norm of its weight converges to a value smaller than $2$.
The learning speed of the second subspace is slower, since it has
smaller $\left|\Sigma_{i}^{2}-2\lambda\right|$. The reconstruction
error is non-monotonic with time, exhibiting a first phase of learning
to reconstruct (where the reconstruction error is decreasing) followed
by a second phase of learning the representation (where the reconstruction
error is increasing). In this second phase, the weight of the first
subspace has almost stopped evolving, whereas the weight of the second
subspace continues to shrink down to zero.}

\label{fig:ReLU_twoblks_smallreg_loss_theta}
\end{figure}

\begin{figure}
\begin{centering}
\subfloat[]{\begin{centering}
\includegraphics[width=0.45\columnwidth]{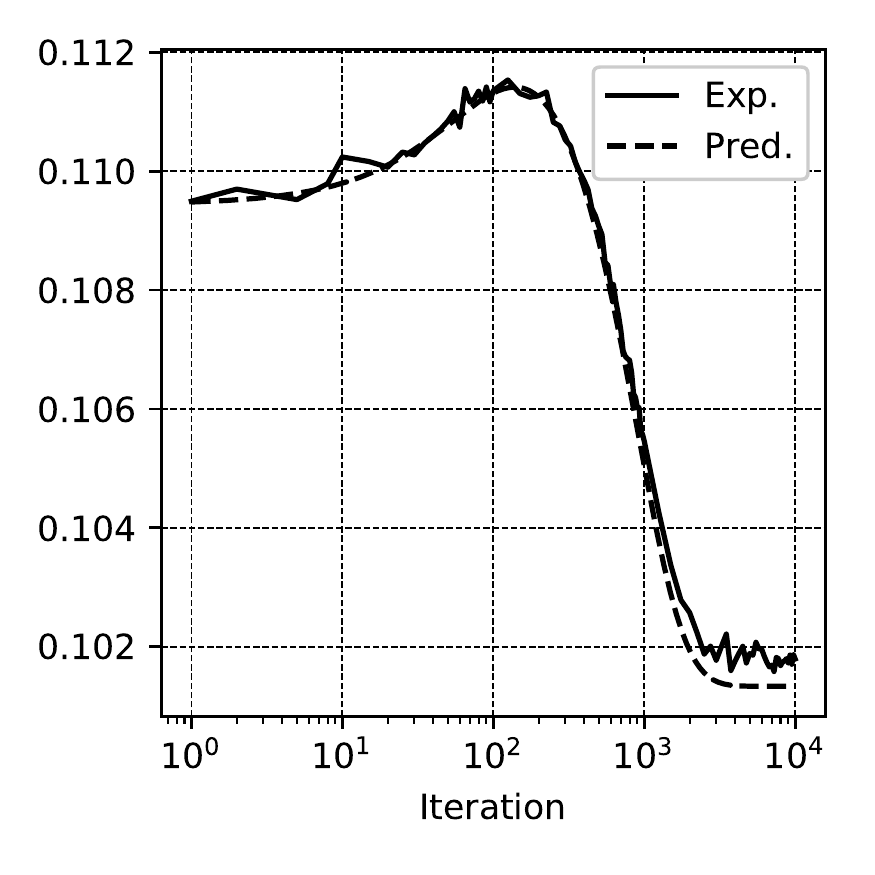}
\par\end{centering}
}\subfloat[]{\begin{centering}
\includegraphics[width=0.45\columnwidth]{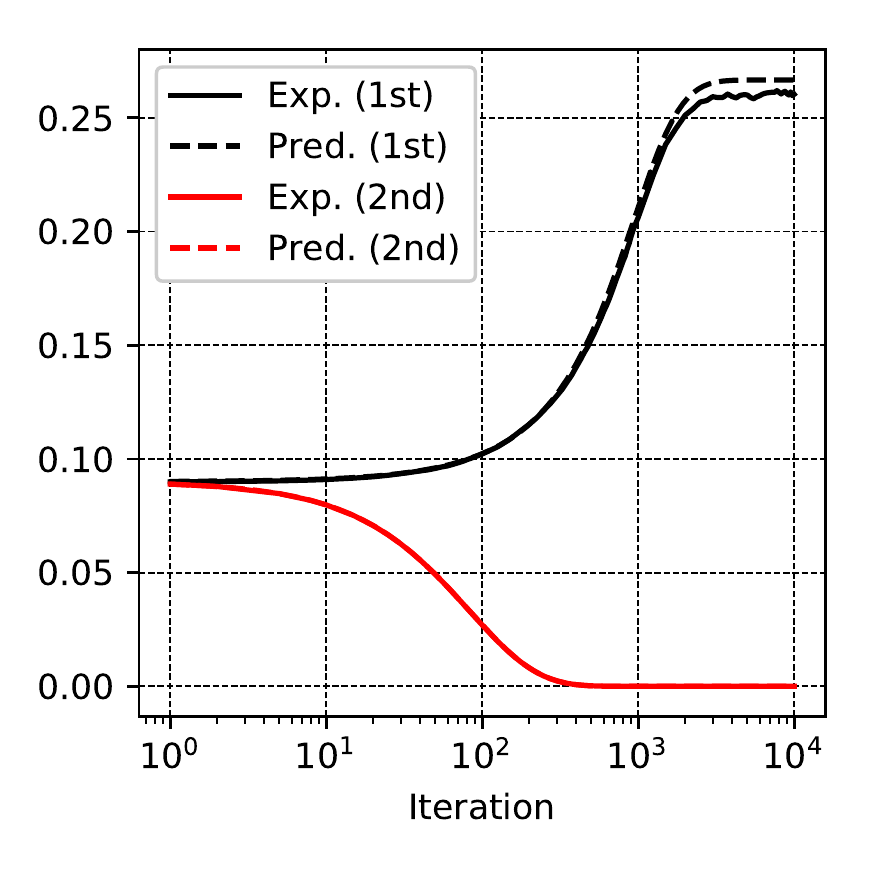}
\par\end{centering}
}
\par\end{centering}
\caption{Autoencoder with ReLU activation and Gaussian data, with large regularization
and small initialization (Result \ref{res:ReLU_setting_simplified}).
Setup: $d=500$, $\Sigma_{1}^{2}=...=\Sigma_{50}^{2}=1.5$ and $\Sigma_{51}^{2}=...=\Sigma_{500}^{2}=0.1$,
$\boldsymbol{R}=\boldsymbol{I}_{d}$, $\lambda=0.65$, $r_{0}=0.3$,
$\epsilon=0.005$ and $N=10000$. (a): the reconstruction error versus
the SGD iteration. (b): the normalized squared norm of the first 50-dimensional
subspace's weight (tagged ``1st'') and the second 450-dimensional
subspace's weight (tagged ``2nd''). Here ``Exp.'' indicates the
simulation results, and ``Pred.'' indicates the theoretical prediction.
For more details, see Appendix \ref{sec:Simulation-details}. The
properties at convergence are similar to Fig. \ref{fig:ReLU_twoblks_smallreg_loss_theta},
but the two phases of learning are different: the phase of learning
the representation (where the reconstruction error is increasing)
is followed by the phase of learning to reconstruct (where the reconstruction
error is decreasing). The learning speed of the first subspace is
slower, since it has smaller $\left|\Sigma_{i}^{2}-2\lambda\right|$.}

\label{fig:ReLU_twoblks_largereg_loss_theta}
\end{figure}

\paragraph*{Early stopping can perform representation learning.}

Instead of the infinite time limit, by considering finite time behaviors,
we observe a separation in time where the subspaces are learned and
selected (or eliminated) at different rates:
\begin{itemize}
\item If $\Sigma_{i}^{2}<2\lambda$, $r_{i,t}$ decreases from $r_{0}$
to $0$ exponentially and monotonically in $t$, at a rate of $\left|\Sigma_{i}^{2}-2\lambda\right|$.
\item Likewise, if $\Sigma_{i}^{2}>2\lambda$, $r_{i,t}$ converges to a
non-zero value exponentially and monotonically in $t$, at a rate
of $\Sigma_{i}^{2}-2\lambda$.
\item If $\Sigma_{i}^{2}=2\lambda$, $r_{i,t}=r_{0}$ unchanged.
\end{itemize}
For $\Sigma_{i}^{2}>2\lambda$, the principal subspaces with higher
$\Sigma_{i}$ are thus learned at a faster rate. On the other hand,
$r_{i,t}\leq r_{0}$ at all $t\geq0$ if $\Sigma_{i}^{2}\leq2\lambda$.
This suggests a second strategy for representation learning: one can
choose small initialization $r_{0}$ and perform early stopping. This
strategy is especially useful when $\lambda=0$. See Fig. \ref{fig:ReLU_twoblks_noreg_loss_theta}
for illustration.

\paragraph*{Maintenance of rotational invariance.}

Eq. (\ref{eq:thm_ReLU_simplified_1}) suggests that the ensemble of
weight vectors, initialized with a rotationally invariant distribution,
maintains a form of rotational invariance throughout the course of
training. To understand this effect, suppose we look at an ``infinite-$N$''
autoencoder whose weight vectors are i.i.d. copies of the random vector
\[
\boldsymbol{\theta}\sim\mathsf{N}\left(\boldsymbol{0},\frac{1}{d}\boldsymbol{R}{\rm diag}\left(b_{1}^{2},...,b_{d}^{2}\right)\boldsymbol{R}^{\top}\right),
\]
for some constants $b_{1},...,b_{d}$. For a given input $\boldsymbol{x}\neq\boldsymbol{0}$,
this idealized autoencoder then outputs the following:
\begin{align*}
\hat{\boldsymbol{x}}_{{\rm inf}}\left(\boldsymbol{x}\right) & =\mathbb{E}_{\boldsymbol{\theta}}\left\{ \kappa\boldsymbol{\theta}\sigma\left(\left\langle \kappa\boldsymbol{\theta},\boldsymbol{x}\right\rangle \right)\right\} =\gamma_{\boldsymbol{x}}\boldsymbol{R}{\rm diag}\left(b_{1}^{2},...,b_{d}^{2}\right)\boldsymbol{R}^{\top}\boldsymbol{x},\\
\gamma_{\boldsymbol{x}} & =\mathbb{E}_{g\sim\mathsf{N}\left(0,1\right)}\left\{ \sigma'\left(\left\Vert {\rm diag}\left(b_{1},...,b_{d}\right)\boldsymbol{R}^{\top}\boldsymbol{x}\right\Vert _{2}g\right)\right\} ,
\end{align*}
as an application of Stein's lemma. For ReLU activation $\sigma$,
$\gamma_{\boldsymbol{x}}=1/2$ a constant. As such, the model tends
to become a linear mapping. This happens despite the fact that the
pre-activation $\left\langle \kappa\boldsymbol{\theta},\boldsymbol{x}\right\rangle $
is a real-valued random variable that typically takes a $\Theta\left(1\right)$
value, has unbounded support and hence does not occupy only a single
linear branch of the ReLU.

\subsubsection{Setting with ReLU activation: Two-staged process\label{subsec:contrib_two_stage}}

Our second result concerns the compression efficiency of the autoencoder
in the setting with ReLU activation via a two-staged process.
\begin{res}[Autoencoder with ReLU, two-staged process -- Informal and simplified]
\label{res:ReLU_setting_2stage_simplified}Consider the same setting
as Result \ref{res:ReLU_setting_simplified}. Form a set of $M$ vectors
$\left(\boldsymbol{w}_{i}^{t}\right)_{i\leq M}$ such that for each
$i\in\left[M\right]$, $\boldsymbol{w}_{i}^{t}=\boldsymbol{w}_{i}^{t}\left(N,t,\epsilon\right)$
is drawn independently at random from the set of $N$ neurons $\left(\boldsymbol{\theta}_{i}^{t/\epsilon}\right)_{i\leq N}$,
trained with SGD. Construct a new autoencoder with $M$ neurons $\left(\boldsymbol{w}_{i}^{t}\right)_{i\leq M}$:
\[
\hat{\boldsymbol{x}}_{M}^{t}\left(\boldsymbol{x}\right)\equiv\hat{\boldsymbol{x}}_{M}^{t}\left(\boldsymbol{x};N,t,\epsilon\right)=\frac{1}{M}\sum_{i=1}^{M}\kappa\boldsymbol{w}_{i}^{t}\sigma\left(\left\langle \kappa\boldsymbol{w}_{i}^{t},\boldsymbol{x}\right\rangle \right).
\]
Suppose that $M=\mu d$ for some fixed $\mu>0$. We then have, for
any $t\geq0$, in the limit $N\to\infty$, $\epsilon\to0$ then $M\to\infty$,
with high probability,
\begin{equation}
{\rm RecErr}\left(\left(\boldsymbol{w}_{i}^{t}\right)_{i\leq M}\right)\approx\underbrace{\frac{1}{2d}\sum_{i=1}^{d}\Sigma_{i}^{2}\left(1-\frac{1}{2}r_{i,t}^{2}\right)^{2}}_{\text{Training}}+\underbrace{\frac{1}{4\mu d^{2}}\sum_{i=1}^{d}r_{i,t}^{2}\sum_{i=1}^{d}r_{i,t}^{2}\Sigma_{i}^{2}}_{\text{Sampling}}.\label{eq:thm_ReLU_2stage_simplified}
\end{equation}
\end{res}

Exact details can be found in the statement of Theorem \ref{thm:ReLU_setting}.
In essence, Result \ref{res:ReLU_setting_2stage_simplified} states
that if we perform a two-staged process where we construct a new autoencoder
by randomly sampling neurons from a trained autoencoder, in the high-dimensional
asymptotic regime (i.e. $M,d\to\infty$ with the sampling ratio $\mu=M/d$
fixed), its reconstruction error is a sum of two components: one is
by the training process of the original autoencoder (comparing the
first term in Eq. (\ref{eq:thm_ReLU_2stage_simplified}) with Eq.
(\ref{eq:thm_ReLU_simplified_2})), and the other is by the sampling
process. The training component is independent of $\mu$, whereas
the sampling component is decreasing and strictly convex in $\mu$.
Note that the reconstruction error of the derived autoencoder tends
to that of the original one as $\mu\to\infty$, while no training
is performed on the derived autoencoder. This is a particular consequence
of the mean field scaling. See Fig. \ref{fig:ReLU_twoblks_subsampled}
for illustration.

To gain further insights, let us analyze Eq. (\ref{eq:thm_ReLU_2stage_simplified})
in a specific scenario:
\[
\Sigma_{d_{0}}^{2}=2,\qquad\Sigma_{d_{0}+1}^{2}=\alpha^{99},\qquad2\lambda=1,
\]
for $d_{0}=\alpha d$ and some positive $\alpha\ll1$. (Here we recall
$C\geq\Sigma_{1}\geq...\geq\Sigma_{d}>0$.) In particular, the power
of the data $\boldsymbol{x}$ highly concentrates in the first $d_{0}$
principal subspaces. We have also chosen $\lambda$ appropriately
such that the trained ReLU-activated autoencoder eliminates the last
$d_{0}+1$ principal subspaces, while maintaining that $1-r_{i,t}^{2}/2\to2\lambda/\Sigma_{i}^{2}=\Theta\left(1\right)$
for all $i\leq d_{0}$ as $t\to\infty$. One easily finds that at
a large learning time $t$,
\[
\text{training component}\sim\frac{d_{0}}{d},\qquad\text{sampling component}\sim\frac{1}{\mu}\left(\frac{d_{0}}{d}\right)^{2}.
\]
Hence in order that eventually the sampling component is much smaller
than the training component, one only requires $\mu\gg d_{0}/d$ (equivalently,
$M\gg d_{0}$), instead of $\mu\gg1$ (equivalently, $M\gg d$). This
highlights the following more general observation: under suitable
circumstances, the number of sampled neurons $M$ only needs to be
larger than some effective dimension $d_{{\rm eff}}$ that is characteristic
of the data distribution, even though it could be the case that $d_{{\rm eff}}\ll d$.
See again Fig. \ref{fig:ReLU_twoblks_subsampled} for illustration.

The above discussion lends us some insight into the compression efficiency
at some large $t$ in a favorable scenario. What if we require good
compression on the whole time horizon $t\in[0,\infty)$? Let us consider
the same scenario but without regularization $\lambda=0$. Let us
further assume an initialization $r_{1,0}^{2}=...=r_{d,0}^{2}=\Theta\left(1\right)>0$.
We know that $r_{i,t}^{2}\to2$ as $t\to\infty$ monotonically for
any $i\in\left[d\right]$, and hence $r_{i,t}^{2}=\Theta\left(1\right)$
for all $t\geq0$. In this case, at any $t\geq0$,
\[
\text{training component}\lesssim\frac{d_{0}}{d},\qquad\text{sampling component}\sim\frac{1}{\mu}\frac{d_{0}}{d}=\frac{d_{0}}{M}.
\]
As $t\to\infty$, the training component tends to zero. In particular,
if $r_{1,0}^{2}=...=r_{d,0}^{2}=2$ and consequently $r_{i,t}^{2}=2$
for all $t\geq0$, then the training component is precisely zero at
all $t\geq0$. We see that on the whole time horizon, the sampling
component cannot be driven to be comparably small unless $M\gg d_{0}$,
and in general, unless $M\gg d$. This simple scenario suggests that
it is unrealistic to expect $M\gg1$ to be sufficient to have a negligible
sampling component. In other words, $M\gg d_{{\rm eff}}$ is necessary.

\begin{figure}
\begin{centering}
\includegraphics[width=0.7\columnwidth]{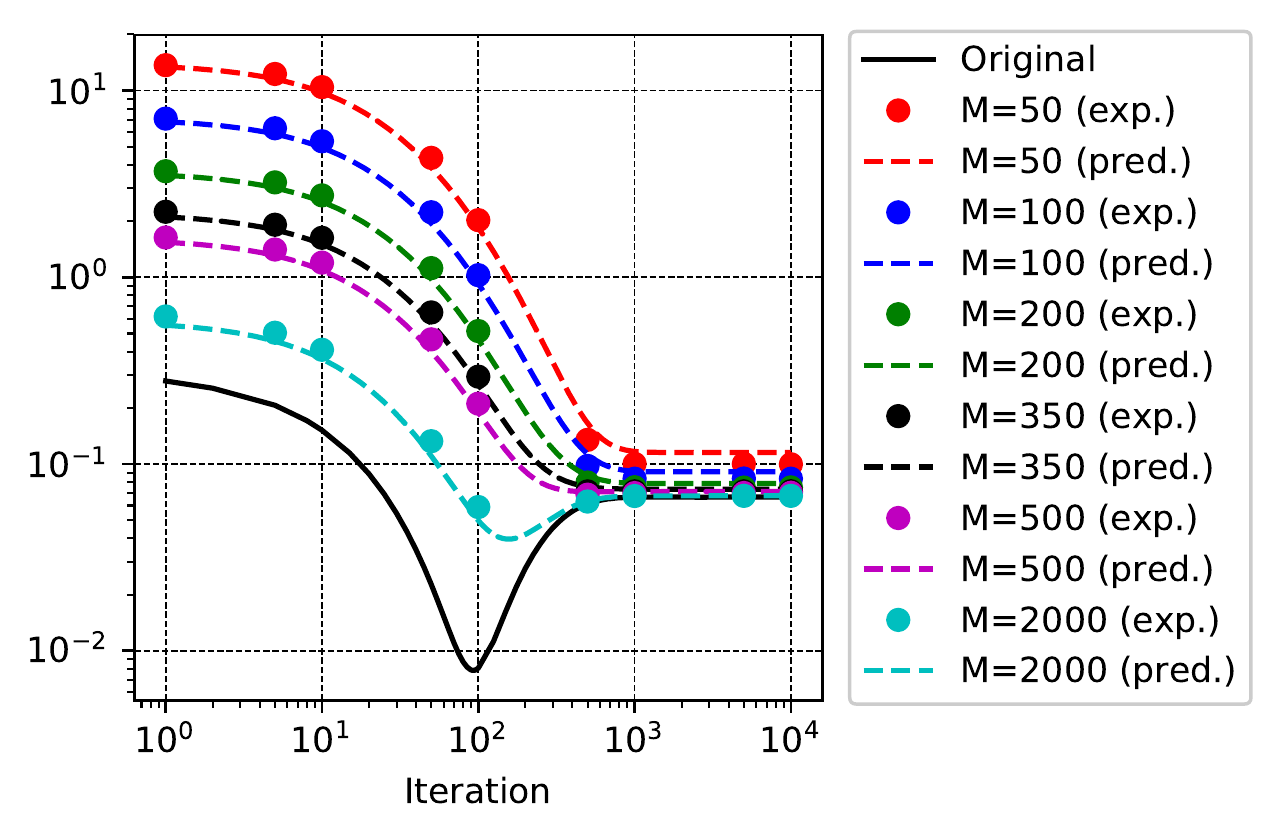}
\par\end{centering}
\caption{Autoencoder with ReLU activation and Gaussian data, with regularization
-- two-staged process (Result \ref{res:ReLU_setting_2stage_simplified}).
The setup is the same as Fig. \ref{fig:ReLU_twoblks_smallreg_loss_theta}.
The reconstruction error is plotted against the SGD iteration, for
the original autoencoder (tagged as ``original''), as well as several
derived autoencoders constructed by the two-staged process with different
numbers of sampled neurons $M$ at different SGD iterations. Here
``exp.'' indicates the simulation results, and ``pred.'' indicates
the theoretical prediction. For more details, see Appendix \ref{sec:Simulation-details}.
Observe that the curve with larger $M$ moves closer to the original
curve. Furthermore at convergence, the performance loss due to sampling
is negligible already for $M=200$, which is a significant reduction
from the data dimension $d=500$. Here we recall that in this setup,
the data $\boldsymbol{x}$ concentrates most of its power in the first
50-dimensional principal subspace.}

\label{fig:ReLU_twoblks_subsampled}
\end{figure}

\subsubsection{Setting with bounded activation\label{subsec:contrib-bounded}}

The previous results apply specifically to the ReLU activation. Our
next result extends to a broad class of bounded activations.
\begin{res}[Autoencoder with bounded activation -- Informal and simplified]
\label{res:Bdd_act_setting_simplified}Consider the autoencoder,
as described in Section \ref{subsec:contrib-Gaussian-data}, in the
following setting. The data $\boldsymbol{x}$ assumes a Gaussian distribution
with the following mean and covariance:
\[
\mathbb{E}\left\{ \boldsymbol{x}\right\} =\boldsymbol{0},\qquad\mathbb{E}\left\{ \boldsymbol{x}\boldsymbol{x}^{\top}\right\} =\frac{1}{d}{\rm diag}(\underbrace{\Sigma_{1}^{2},...,\Sigma_{1}^{2}}_{d_{1}\text{ entries}},\underbrace{\Sigma_{2}^{2},...,\Sigma_{2}^{2}}_{d_{2}\text{ entries}}),
\]
where $0<C\leq\Sigma_{1},\Sigma_{2}\leq C$, and $d_{1}=\alpha d$,
$d_{2}=\left(1-\alpha\right)d$ for some $\alpha\in\left(0,1\right)$
such that $d_{1}$ and $d_{2}$ are positive integers, and $\alpha$
does not depend on $d$. The activation $\sigma$ is bounded and sufficiently
regular. The regularization strength $\lambda\leq C$. The initialization
$\Theta^{0}=\left(\boldsymbol{\theta}_{i}^{0}\right)_{i\leq N}\sim_{{\rm i.i.d.}}\mathsf{N}\left(\boldsymbol{0},r_{0}^{2}\boldsymbol{I}_{d}/d\right)$
for a non-negative constant $r_{0}\leq C$.

Then for $N\gg{\rm poly}\left(d\right)$, $\epsilon\ll1/{\rm poly}\left(d\right)$
and a finite $t\in\mathbb{N}\epsilon$, $t\leq C$, with high probability,
\[
\frac{1}{N}\sum_{i=1}^{N}\delta_{\boldsymbol{\theta}_{i}^{t/\epsilon}}\approx{\rm Law}\left(r_{1,t}\boldsymbol{\omega}_{1},r_{1,t}\boldsymbol{\omega}_{2}\right),\qquad{\rm RecErr}\left(\Theta^{t/\epsilon}\right)\approx{\rm RecErr}_{*}\left(\rho_{r}^{t}\right).
\]
Here $\boldsymbol{\omega}_{1}\sim\text{Unif}\left(\mathbb{S}^{d_{1}-1}\right)$
and $\boldsymbol{\omega}_{2}\sim\text{Unif}\left(\mathbb{S}^{d_{2}-1}\right)$
independently and independent of $\left(r_{1,t},r_{2,t}\right)$,
$\rho_{r}^{t}={\rm Law}\left(r_{1,t},r_{2,t}\right)\in\mathscr{P}\left(\mathbb{R}_{\geq0}^{2}\right)$
is described by a system of two ODEs with random initialization and
${\rm RecErr}_{*}\left(\rho_{r}^{t}\right)$ has an explicit formula.

In the above, the constants $C$ do not depend on $N$, $\epsilon$
or $d$.
\end{res}

Exact details can be found in the statement of Theorem \ref{thm:bdd_act_setting}.
This setting covers the case $\sigma=\tanh$, a common activation.
The result can be extended easily to more general structures of the
covariance; we consider the simple two-blocks diagonal structure mainly
for simplicity. Similar to the ReLU setting, we stress that the requirement
is again mild: $N\gg{\rm poly}\left(d\right)$ and $\epsilon\ll1/{\rm poly}\left(d\right)$.

As suggested by Result \ref{res:Bdd_act_setting_simplified}, $r_{1,t}$
governs the first $d_{1}$ coordinates of $\left(\boldsymbol{\theta}_{i}^{t/\epsilon}\right)_{i\leq N}$,
and $r_{2,t}$ corresponds to the last $d_{2}$ coordinates. In other
words, $r_{1,t}$ and $r_{2,t}$ indicate the rescaling factors of
the first $d_{1}$-dimensional and second $d_{2}$-dimensional principal
subspaces, respectively. See Fig. \ref{fig:tanh_twoblks_loss_theta}
for illustration. We observe several qualitative features similar
to the ReLU setting. We note that some of these features, such as
the sigmoidal learning curve and the different learning speeds for
different principal subspaces, have been previously shown for linear
(weight-untied) neural networks \cite{saxe2013exact,advani2017high,saxe2019mathematical,gidel2019implicit}
and nonlinear networks under very strong assumptions \cite{combes2018learning}.
Our results give a theoretically solid piece of evidence towards the
remarkable observation that these features could continue to hold
more generally for neural networks with nonlinear activations in a
natural setting.

On the other hand, there are also some differences, which arise primarily
from the fact that the activation is not homogenous like the ReLU.
In particular:

\paragraph*{Joint evolution of the rescaling factors.}

In this present setting, $r_{1,t}$ and $r_{2,t}$ evolve jointly,
as seen from Fig. \ref{fig:tanh_twoblks_loss_theta}. This is a stark
contrast with the ReLU setting in Result \ref{res:ReLU_setting_simplified}
where each principal subspace's rescaling factor evolves independently
of each other. Such decoupling effect in the case of ReLU activation
allows for more analytical tractability than the present setting.

\paragraph*{No regularization does not equal learning the identity.}

We observe from Fig. \ref{fig:tanh_twoblks_loss_theta}.(a) that when
$\lambda=0$, with sufficiently large $d$, the reconstruction error
converges to zero, i.e. that the unregularized autoencoder is able
to reconstruct any vector $\boldsymbol{x}$ drawn from the data distribution
${\cal P}$. Note that in high dimension, ${\cal P}$ is almost the
same as the distribution of $\left(\Sigma_{1}\sqrt{\alpha}\boldsymbol{\omega}_{1},\Sigma_{2}\sqrt{1-\alpha}\boldsymbol{\omega}_{2}\right)$
for $\boldsymbol{\omega}_{1}\sim\text{Unif}\left(\mathbb{S}^{d_{1}-1}\right)$
and $\boldsymbol{\omega}_{2}\sim\text{Unif}\left(\mathbb{S}^{d_{2}-1}\right)$
independently. As such, the support of ${\cal P}$ concentrates in
a small region of $\mathbb{R}^{d}$. This suggests that the autoencoder
in this case does not learn the identity, unlike the unregularized
ReLU autoencoder. This is indeed confirmed in Fig. \ref{fig:tanh_twoblks_loss_outsample}.(a),
which shows that the reconstruction error of a vector $\boldsymbol{x}$
drawn from a certain distribution ${\cal Q}\neq{\cal P}$ does not
converge to zero.

On the other hand, Fig. \ref{fig:tanh_twoblks_loss_outsample}.(b)
shows that there are certain other distributions, different from ${\cal P}$,
such that the reconstruction error converges to zero. In fact, in
the next point, we shall argue that the unregularized autoencoder
can nevertheless ``almost'' learn the identity mapping.

\paragraph*{Maintenance of rotational invariance.}

Similar to the ReLU case, here there is also a form of rotational
invariance being preserved throughout training. In particular, let
us consider the effect in high dimension. For large $d$, one can
approximate $\boldsymbol{\omega}_{1}\approx\left(\alpha d\right)^{-1/2}\boldsymbol{z}_{1}$
and $\boldsymbol{\omega}_{2}\approx\left(\left(1-\alpha\right)d\right)^{-1/2}\boldsymbol{z}_{2}$
for $\boldsymbol{z}_{1}\sim\mathsf{N}\left(0,\boldsymbol{I}_{d_{1}}\right)$
and $\boldsymbol{z}_{2}\sim\mathsf{N}\left(0,\boldsymbol{I}_{d_{2}}\right)$
independently. Then similar to the ReLU case, considering Result \ref{res:Bdd_act_setting_simplified},
let us look at an ``infinite-$N$'' autoencoder whose weight vectors
are i.i.d. copies of the random vector
\[
\boldsymbol{\theta}\stackrel{{\rm d}}{=}\left(b_{1}\left(\alpha d\right)^{-1/2}\boldsymbol{z}_{1},\;b_{2}\left(\left(1-\alpha\right)d\right)^{-1/2}\boldsymbol{z}_{2}\right),
\]
for some constants $b_{1}$ and $b_{2}$. For a given input $\boldsymbol{x}\neq\boldsymbol{0}$,
this idealized autoencoder then outputs the following:
\begin{align*}
\hat{\boldsymbol{x}}_{{\rm inf}}\left(\boldsymbol{x}\right) & =\mathbb{E}_{\boldsymbol{\theta}}\left\{ \kappa\boldsymbol{\theta}\sigma\left(\left\langle \kappa\boldsymbol{\theta},\boldsymbol{x}\right\rangle \right)\right\} =\gamma_{\boldsymbol{x}}\left(b_{1}^{2}\alpha^{-1}\boldsymbol{x}_{\left[1\right]},\;b_{2}^{2}\left(1-\alpha\right)^{-1}\boldsymbol{x}_{\left[2\right]}\right),\\
\gamma_{\boldsymbol{x}} & =\mathbb{E}_{g\sim\mathsf{N}\left(0,1\right)}\left\{ \sigma'\left(\sqrt{b_{1}^{2}\alpha^{-1}\left\Vert \boldsymbol{x}_{\left[1\right]}\right\Vert _{2}^{2}+b_{2}^{2}\left(1-\alpha\right)^{-1}\left\Vert \boldsymbol{x}_{\left[2\right]}\right\Vert _{2}^{2}}g\right)\right\} ,
\end{align*}
as an application of Stein's lemma, where $\boldsymbol{x}_{\left[1\right]}$
indicates the vector of the first $d_{1}$ entries of $\boldsymbol{x}$
and $\boldsymbol{x}_{\left[2\right]}$ is the vector of all other
entries. Unlike the ReLU case, with a generic activation, $\gamma_{\boldsymbol{x}}$
is generally not a constant, even though it depends mildly on $\boldsymbol{x}$
via only the norms of the two components $\left\Vert \boldsymbol{x}_{\left[1\right]}\right\Vert _{2}$
and $\left\Vert \boldsymbol{x}_{\left[2\right]}\right\Vert _{2}$.

Motivated by the unregularized case $\lambda=0$ in which Fig. \ref{fig:tanh_twoblks_loss_theta}.(a)
suggests that at convergence $r_{1,t}^{2}\alpha^{-1}\approx r_{2,t}^{2}\left(1-\alpha\right)^{-1}$,
let us consider $b_{1}^{2}\alpha^{-1}=b_{2}^{2}\left(1-\alpha\right)^{-1}=c_{*}$.
In this scenario,
\[
\hat{\boldsymbol{x}}_{{\rm inf}}\left(\boldsymbol{x}\right)=\gamma_{\boldsymbol{x}}c_{*}\boldsymbol{x},\qquad\gamma_{\boldsymbol{x}}=\mathbb{E}_{g\sim\mathsf{N}\left(0,1\right)}\left\{ \sigma'\left(\sqrt{c_{*}}\left\Vert \boldsymbol{x}\right\Vert _{2}g\right)\right\} .
\]
One therefore does not expect $\gamma_{\boldsymbol{x}}$ to be independent
of $\boldsymbol{x}$ unless $\sigma$ is a homogeneous function. This
gives an explanation why the unregularized autoencoder does not learn
the identity and confirms the finding in Fig. \ref{fig:tanh_twoblks_loss_outsample}.(a).
On the other hand, we also see that the model learns a restricted
form of the identity mapping. In particular, $\boldsymbol{x}\mapsto\hat{\boldsymbol{x}}_{{\rm inf}}\left(\boldsymbol{x}\right)$
maps a sphere $S_{{\rm in}}$ to another sphere $S_{{\rm out}}$ by
preserving the direction of the input $\boldsymbol{x}\in S_{{\rm in}}$
and scaling the radius of $S_{{\rm in}}$ to that of $S_{{\rm out}}$.
A consequence is the following. Let $S=\left\{ \boldsymbol{x}\in\mathbb{R}^{d}:\;\left\Vert \boldsymbol{x}\right\Vert _{2}=\Sigma_{1}^{2}\alpha+\Sigma_{2}^{2}\left(1-\alpha\right)\right\} $.
Recall that on the data distribution ${\cal {\cal P}}$ with which
the autoencoder is trained, $\left\Vert \boldsymbol{x}\right\Vert _{2}\approx\Sigma_{1}^{2}\alpha+\Sigma_{2}^{2}\left(1-\alpha\right)$
in high dimension. Hence the support of ${\cal P}$ is essentially
a strict subset of $S$. Let us further assume $b_{1}$ and $b_{2}$
are equal to the values of $r_{1,t}$ and $r_{2,t}$ at convergence,
in which case we have $\gamma_{\boldsymbol{x}}c_{*}=1$ for any $\boldsymbol{x}$
drawn from ${\cal P}$ since the reconstruction error on ${\cal P}$
converges to zero as in Fig. \ref{fig:tanh_twoblks_loss_theta}.(a).
Now since $\gamma_{\boldsymbol{x}}$ only depends on $\left\Vert \boldsymbol{x}\right\Vert _{2}$,
for any $\boldsymbol{x}\in S$ not necessarily drawn from ${\cal P}$,
we also have $\gamma_{\boldsymbol{x}}c_{*}=1$, and equivalently,
$\hat{\boldsymbol{x}}_{{\rm inf}}\left(\boldsymbol{x}\right)=\boldsymbol{x}$.
This confirms the finding in Fig. \ref{fig:tanh_twoblks_loss_outsample}.(b).

In short, we see that rotational invariance results in the mild dependency
of $\gamma_{\boldsymbol{x}}$ on $\boldsymbol{x}$, and the lack of
homogeneity in the activation function results in a mildly nonlinear
mapping that is expressed by the autoencoder.

\paragraph*{Equivalence of activation functions.}

As a first note, we see that $\gamma_{\boldsymbol{x}}=0$ and $\hat{\boldsymbol{x}}_{{\rm inf}}\left(\boldsymbol{x}\right)=\boldsymbol{0}$
if $\sigma$ is an even function, which is therefore a bad design
choice.

Rotational invariance leads to another interesting consequence. From
the previous discussion (as well as Appendix \ref{subsec:Simplifications-for-Setting-bdd-act}),
we see that the influence of the activation $\sigma$ is via its derivative
$\sigma'$. In particular, for two activation functions $\sigma$
and $\tilde{\sigma}$, if
\[
\mathbb{E}_{g\sim\mathsf{N}\left(0,1\right)}\left\{ \sigma'\left(sg\right)\right\} =\mathbb{E}_{g\sim\mathsf{N}\left(0,1\right)}\left\{ \tilde{\sigma}'\left(sg\right)\right\} \qquad\forall s\in\mathbb{R},
\]
then it is expected that in high dimension, the dynamics of the $\sigma$-activated
autoencoder is the same as that of the $\tilde{\sigma}$-activated
one, provided the same data distribution, regularization strength
$\lambda$ and initialization parameter $r_{0}$. That is, $\sigma$
and $\tilde{\sigma}$ then belong to the same equivalence class of
activation functions. Given $\sigma$, one can obtain another activation
function $\tilde{\sigma}$ in its equivalence class by adding an even
function to it. Fig. \ref{fig:tanh_twoblks_equiv} confirms this expectation.
This holds even when the additional even function breaks monotonicity
of $\sigma$.

\begin{figure}
\begin{centering}
\subfloat[]{\begin{centering}
\includegraphics[width=0.45\columnwidth]{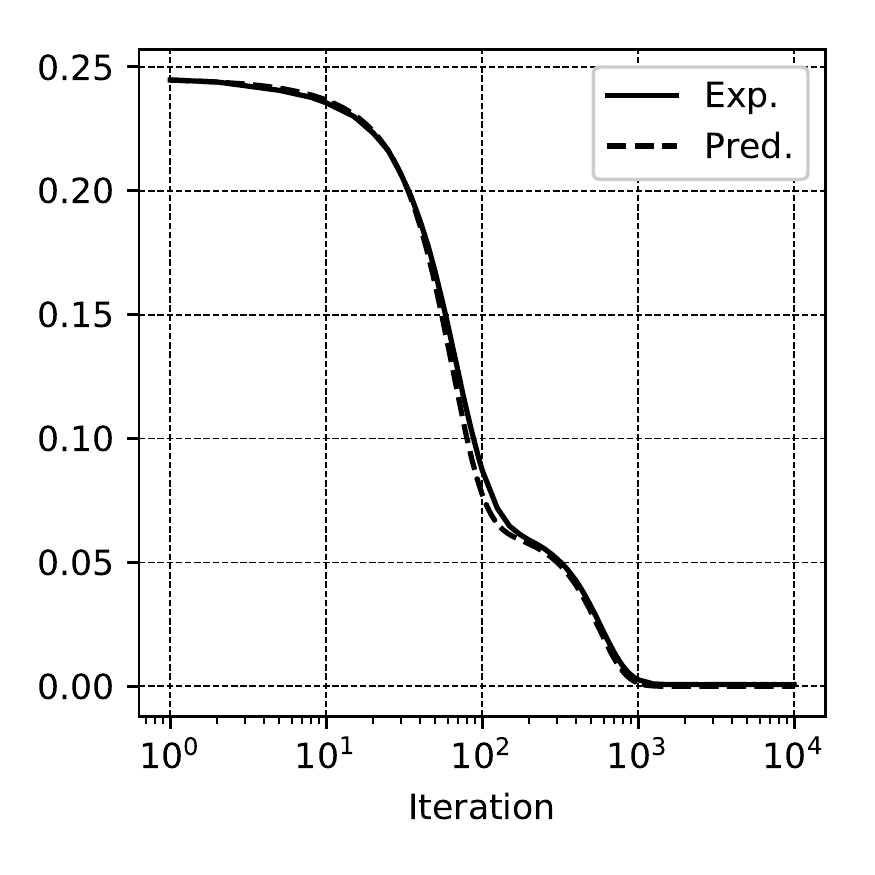}
\par\end{centering}
}\subfloat[]{\begin{centering}
\includegraphics[width=0.45\columnwidth]{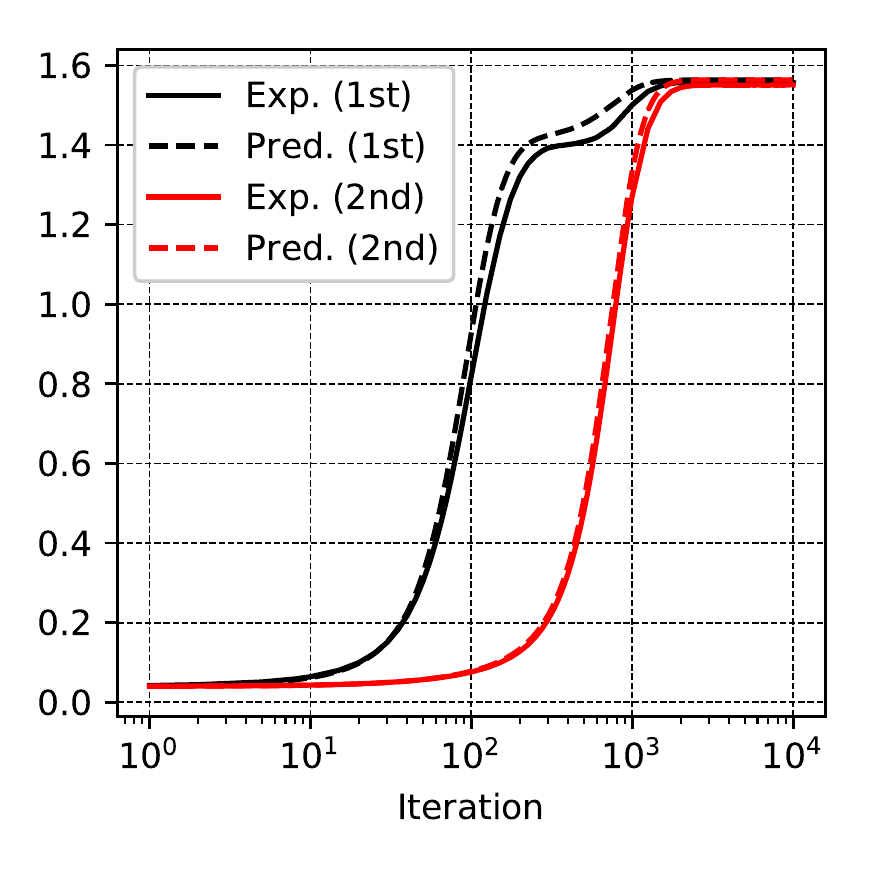}
\par\end{centering}
}
\par\end{centering}
\begin{centering}
\subfloat[]{\begin{centering}
\includegraphics[width=0.45\columnwidth]{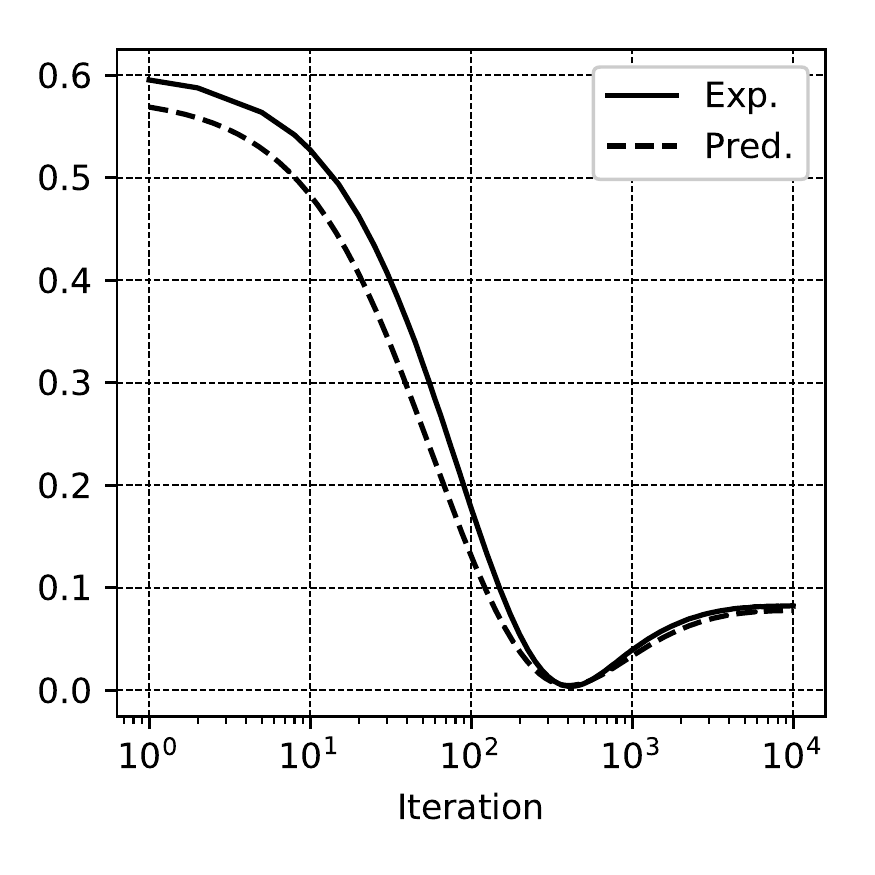}
\par\end{centering}
}\subfloat[]{\begin{centering}
\includegraphics[width=0.45\columnwidth]{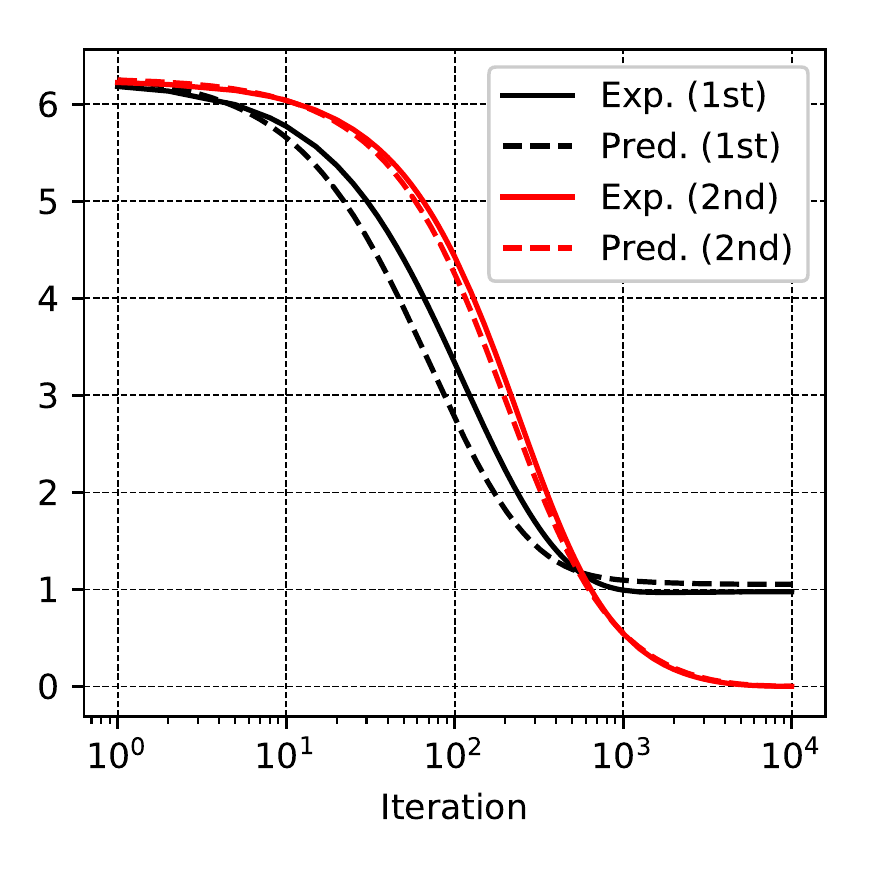}
\par\end{centering}
}
\par\end{centering}
\caption{Autoencoder with $\tanh$ activation and Gaussian data (Result \ref{res:Bdd_act_setting_simplified}).
Setup: $d=200$, $d_{1}=60$, $d_{2}=140$, $\Sigma_{1}^{2}=1.3$,
$\Sigma_{2}^{2}=0.2$, and $N=10000$. In (a) and (b), $\lambda=0$,
$r_{0}=0.2$, $\epsilon=0.01$. In (c) and (d), $\lambda=0.2$, $r_{0}=2.5$,
$\epsilon=0.003$. (a) and (c): the reconstruction error versus the
SGD iteration. (b) and (d): the normalized squared norm of the first
60-dimensional subspace's weight (tagged ``1st'') and the second
140-dimensional subspace's weight (tagged ``2nd''). Here ``Exp.''
indicates the simulation results, and ``Pred.'' indicates the theoretical
prediction. For more details, see Appendix \ref{sec:Simulation-details}.
We observe qualitative similarities between the plots and Fig. \ref{fig:ReLU_twoblks_noreg_loss_theta},
\ref{fig:ReLU_twoblks_smallreg_loss_theta} of the ReLU setting. We
also observe from plot (b) that unlike the ReLU setting, the normalized
squared norm of the first subspace no longer displays a simple sigmoidal
evolution. This indicates that the evolutions of the two subspaces
are coupled.}

\label{fig:tanh_twoblks_loss_theta}
\end{figure}

\begin{figure}
\begin{centering}
\subfloat[]{\begin{centering}
\includegraphics[width=0.45\columnwidth]{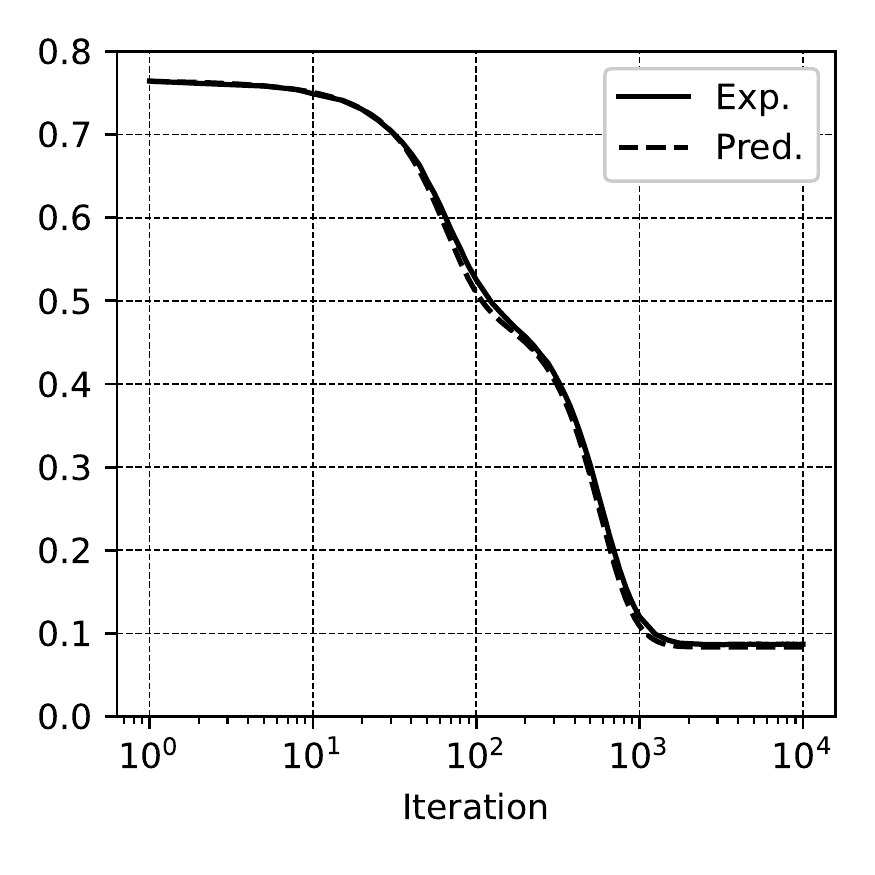}
\par\end{centering}
}\subfloat[]{\begin{centering}
\includegraphics[width=0.45\columnwidth]{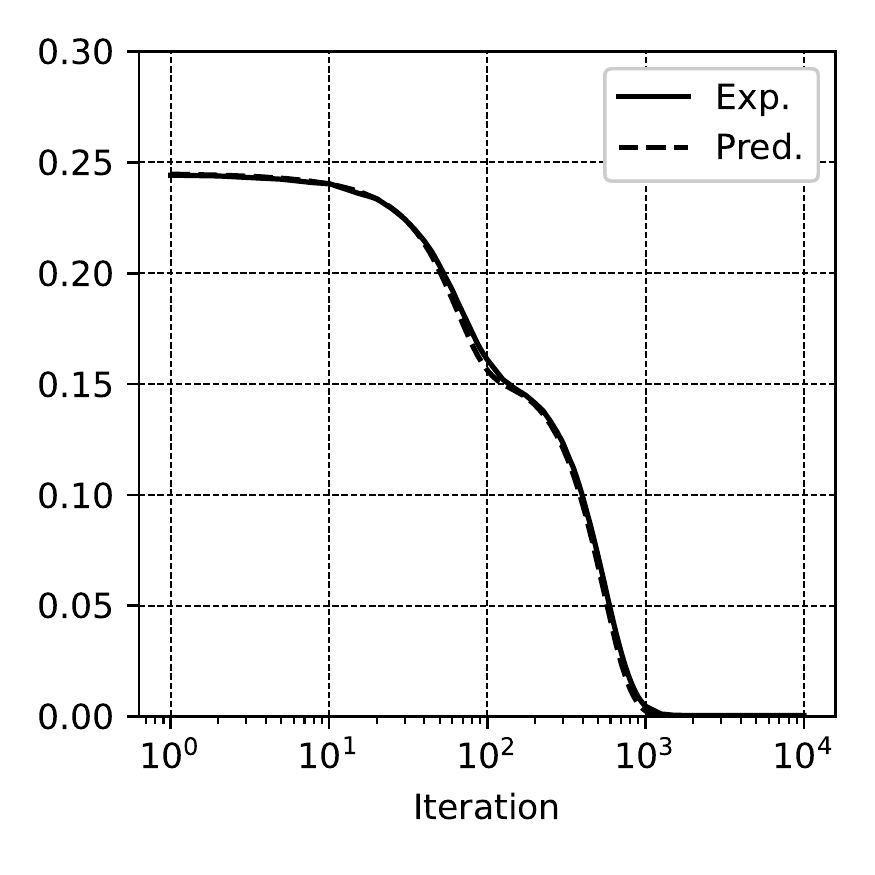}
\par\end{centering}
}
\par\end{centering}
\caption{Autoencoder with $\tanh$ activation and Gaussian data (Result \ref{res:Bdd_act_setting_simplified}),
with the same setup as Fig. \ref{fig:tanh_twoblks_loss_theta}.(a)
(no regularization $\lambda=0$). We plot the reconstruction error
$\mathbb{E}_{\boldsymbol{x}\sim{\cal Q}}\left\{ \frac{1}{2}\left\Vert \hat{\boldsymbol{x}}_{N}\left(\boldsymbol{x};\Theta\right)-\boldsymbol{x}\right\Vert _{2}^{2}\right\} $
of the autoencoder $\hat{\boldsymbol{x}}_{N}\left(\cdot;\Theta\right)$,
trained on the data $\left(\boldsymbol{x}^{k}\right)_{k\protect\geq0}\sim{\cal P}$,
with respect to another distribution ${\cal Q}$. Here ${\cal Q}$
is also a zero-mean Gaussian distribution with the same covariance
structure as ${\cal P}$, but in subfigure (a), it has $\Sigma_{1,{\cal Q}}^{2}=2$
and $\Sigma_{2,{\cal Q}}^{2}=1.5$, and in subfigure (b), it has $\Sigma_{1,{\cal Q}}^{2}=0.6$
and $\Sigma_{2,{\cal Q}}^{2}=0.5$ (whereas $\Sigma_{1,{\cal P}}^{2}=1.3$
and $\Sigma_{2,{\cal P}}^{2}=0.2$ for ${\cal P}$). In this figure,
``Exp.'' indicates the simulation results, and ``Pred.'' indicates
the theoretical prediction. For implementation details, see Appendix
\ref{sec:Simulation-details}. Observe that the reconstruction error
does not converge to zero in subfigure (a), in which case $\Sigma_{1,{\cal Q}}^{2}d_{1}+\Sigma_{2,{\cal Q}}^{2}d_{2}\protect\neq\Sigma_{1,{\cal P}}^{2}d_{1}+\Sigma_{2,{\cal P}}^{2}d_{2}$.
In subfigure (b), we have $\Sigma_{1,{\cal Q}}^{2}d_{1}+\Sigma_{2,{\cal Q}}^{2}d_{2}=\Sigma_{1,{\cal P}}^{2}d_{1}+\Sigma_{2,{\cal P}}^{2}d_{2}$
and the reconstruction error converges to zero.}

\label{fig:tanh_twoblks_loss_outsample}
\end{figure}

\begin{figure}
\begin{centering}
\subfloat[]{\begin{centering}
\includegraphics[width=0.45\columnwidth]{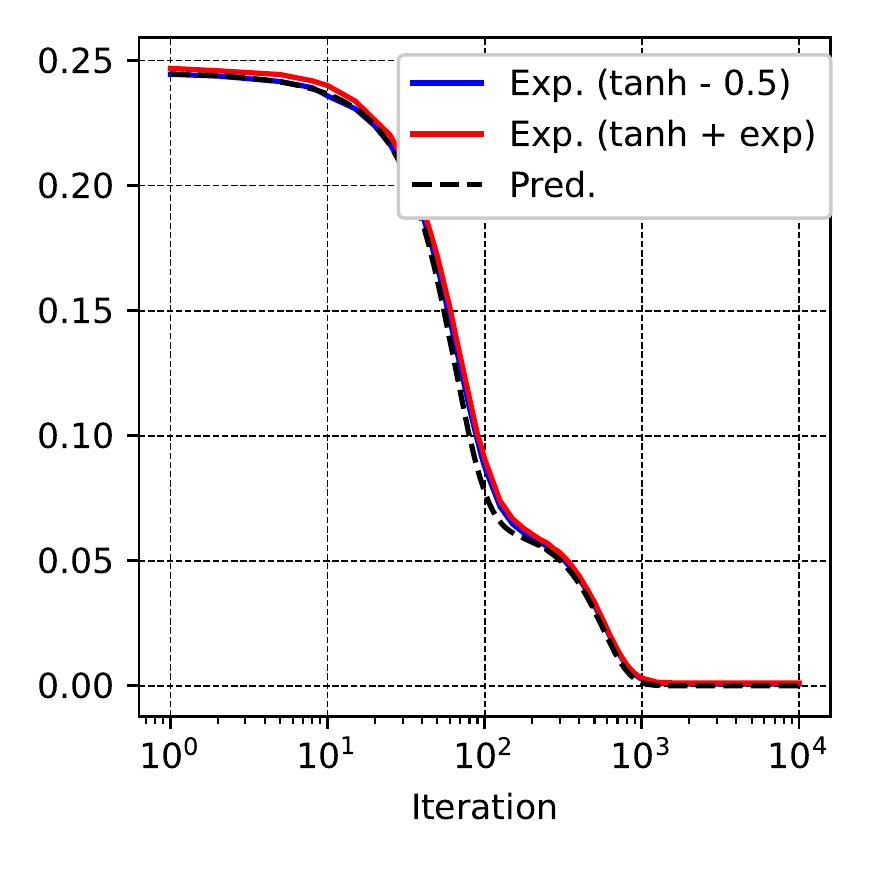}
\par\end{centering}
}\subfloat[]{\begin{centering}
\includegraphics[width=0.45\columnwidth]{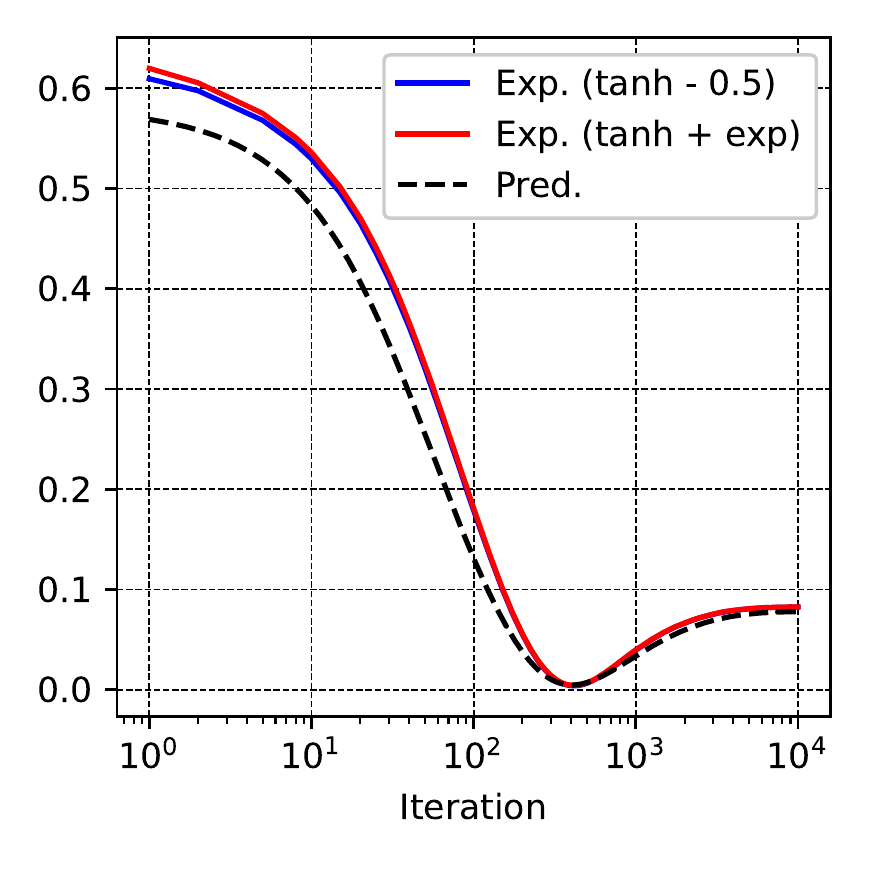}
\par\end{centering}
}
\par\end{centering}
\begin{centering}
\subfloat[]{\begin{centering}
\includegraphics[width=0.51\columnwidth]{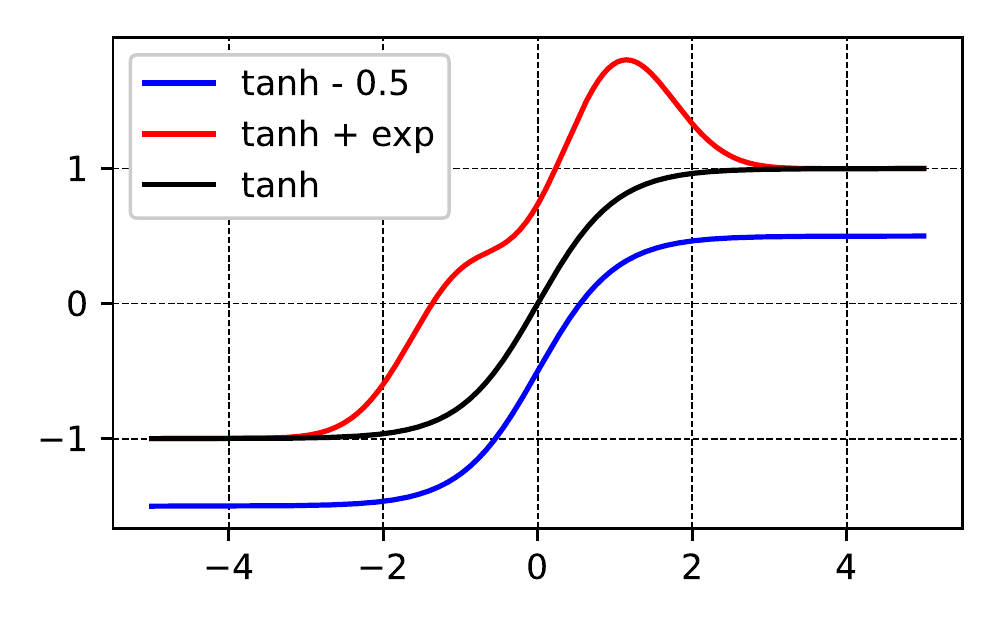}
\par\end{centering}
}
\par\end{centering}
\caption{Autoencoders with Gaussian data and activations in the same equivalence
class as $\tanh$ (Result \ref{res:Bdd_act_setting_simplified}).
In subfigures (a) and (b), we plot the evolution of the reconstruction
error in two different settings. In subfigure (c), we plot the activation
functions. The setup of (a) is the same as Fig. \ref{fig:tanh_twoblks_loss_theta}.(a),
and the setup of (b) is the same as Fig. \ref{fig:tanh_twoblks_loss_theta}.(c).
Here ``Exp.'' indicates the simulation results, ``$\tanh-0.5$''
indicates $\sigma\left(u\right)=\tanh\left(u\right)-0.5$, ``$\tanh+{\rm exp}$''
indicates $\sigma\left(u\right)=\tanh\left(u\right)+\exp(-\left(u-1\right)^{2})+\exp(-\left(u+1\right)^{2})$,
and ``Pred.'' indicates the theoretical prediction computed based
on $\sigma=\tanh$. For more details, see Appendix \ref{sec:Simulation-details}.}

\label{fig:tanh_twoblks_equiv}
\end{figure}

\subsection{Dynamics of weight-tied autoencoders: Real data\label{subsec:contrib-real-data}}

Our theoretical predictions so far have assumed Gaussian data. Here
we show experimentally that these predictions capture surprisingly
well the learning dynamics of the autoencoder on real data, in particular
the MNIST data, despite the fact that it is far from being Gaussian.
We show this for the particular setting with ReLU activation, since
Results \ref{res:ReLU_setting_simplified} and \ref{res:ReLU_setting_2stage_simplified}
allow for almost arbitrary spectrum of the data covariance matrix
and hence we can estimate this matrix and apply the given formulas.
We plot the results in Fig. \ref{fig:MNIST_reg_loss_theta}, \ref{fig:MNIST_reg_subsampled}
and \ref{fig:MNIST_noreg_loss_theta} for simulations on the MNIST
data. See also Appendix \ref{sec:Simulation-details} for the experimental
setups.

In Appendix \ref{sec:Simulation-details}, we plot the spectrum of
the MNIST data set's estimated covariance matrix. Observe the fast
decay of the spectrum, while we recall that Results \ref{res:ReLU_setting_simplified}
and \ref{res:ReLU_setting_2stage_simplified} require a sufficiently
slow decay. It is interesting that we can observe a reasonable fit
of the theoretical predictions with the experimental results in Fig.
\ref{fig:MNIST_reg_loss_theta}, \ref{fig:MNIST_reg_subsampled} and
\ref{fig:MNIST_noreg_loss_theta}.

Remarkably the agreement extends beyond the learning curves: our theory
predicts well what the autoencoder actually learns when it is trained
on MNIST. More specifically, as demonstrated in Fig. \ref{fig:MNIST_reg_loss_theta}
and \ref{fig:MNIST_noreg_loss_theta}, depending on the regularization,
the trained autoencoder exhibits a spectrum of behaviors: it can perform
a certain degree of representation learning when there is regularization,
and it can also learn an identity function and no representation at
the other extreme when there is no regularization. This agrees well
with our theoretical prediction.

This remarkable agreement leads us to the conjecture on a universality
phenomenon: our theory should extend to a broad class of data distributions
that have zero mean and share the same covariance. The work \cite{ng2004feature}
made a relevant observation -- without proof -- that for a variety
of machine learning models, including feedforward neural networks
trained with gradient descent and initialized with independent Gaussian
weights, the model output is generally insensitive w.r.t. rotational
transformations that act on the input. While it does not directly
prove our conjecture, it gives another encouraging piece of evidence
towards the conjecture.

We also refer to Appendix \ref{sec:Simulation-details}, where we
demonstrate that there is little loss in the reconstruction quality
incurred by the two-staged process.

\begin{figure}
\begin{centering}
\subfloat[]{\begin{centering}
\includegraphics[width=0.45\columnwidth]{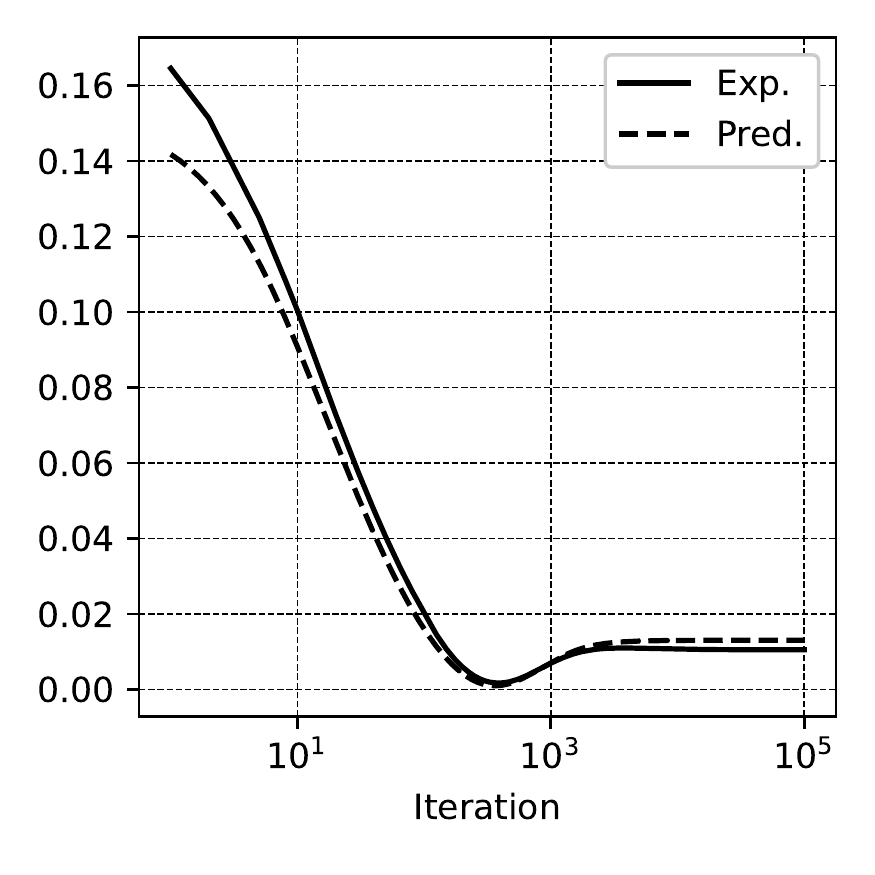}
\par\end{centering}
}\subfloat[]{\begin{centering}
\includegraphics[width=0.45\columnwidth]{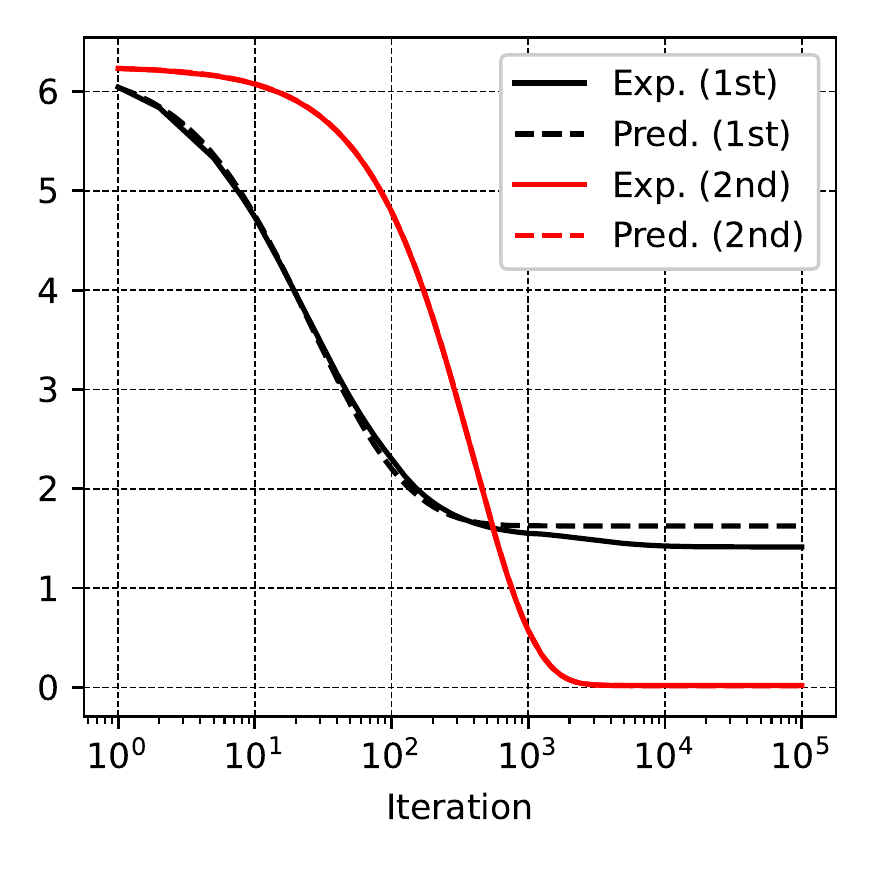}
\par\end{centering}
}
\par\end{centering}
\begin{centering}
\subfloat[]{\begin{centering}
\includegraphics[width=0.9\columnwidth]{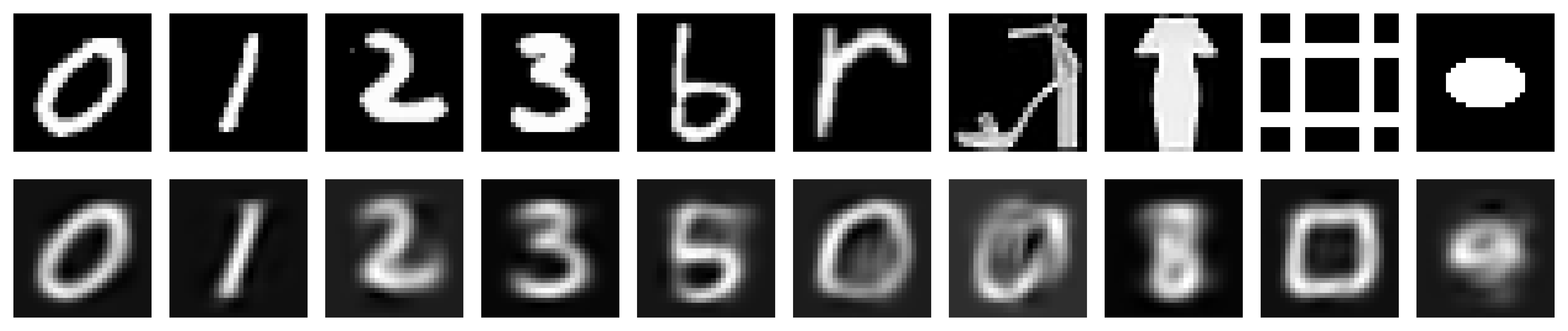}
\par\end{centering}
}
\par\end{centering}
\caption{Autoencoder with ReLU activation and MNIST data, with regularization.
Setup: $\lambda=0.2$, $r_{0}=2.5$, $\epsilon=0.003$ and $N=20000$.\\(a):
the reconstruction error versus the SGD iteration. Here ``Exp.''
indicates the simulation results, and ``Pred.'' indicates the theoretical
prediction computed using the formulas given in Result \ref{res:ReLU_setting_simplified}.
For more details, see Appendix \ref{sec:Simulation-details}.\\(b):
the normalized squared norm of the first 10-dimensional subspace's
weight (tagged ``1st'') and the second 774-dimensional subspace's
weight (tagged ``2nd''). Since the spectrum of MNIST data concentrates
in the first 10 principal subspaces, our theory predicts these subspaces
would not be removed by the regularization. This is reflected by plot
(b), where the normalized squared norm of the weight of these subspaces
converges to a non-zero value, whereas the other converges to zero.\\(c):
the first row shows four MNIST digit test samples and six non-digit
samples, and the second row shows their respective reconstructions
at iteration $10^{5}$. Note that the model is not trained with any
non-digit samples. Since only the projection onto the first few principal
subspaces of the MNIST spectrum is retained, the reconstructions of
the non-digit samples show several features of digits and are hardly
recognizable. The reconstructions of the digit samples are recognizable,
but blurry due to the shrinkage effect of the regularization.}

\label{fig:MNIST_reg_loss_theta}
\end{figure}

\begin{figure}
\begin{centering}
\includegraphics[width=0.7\columnwidth]{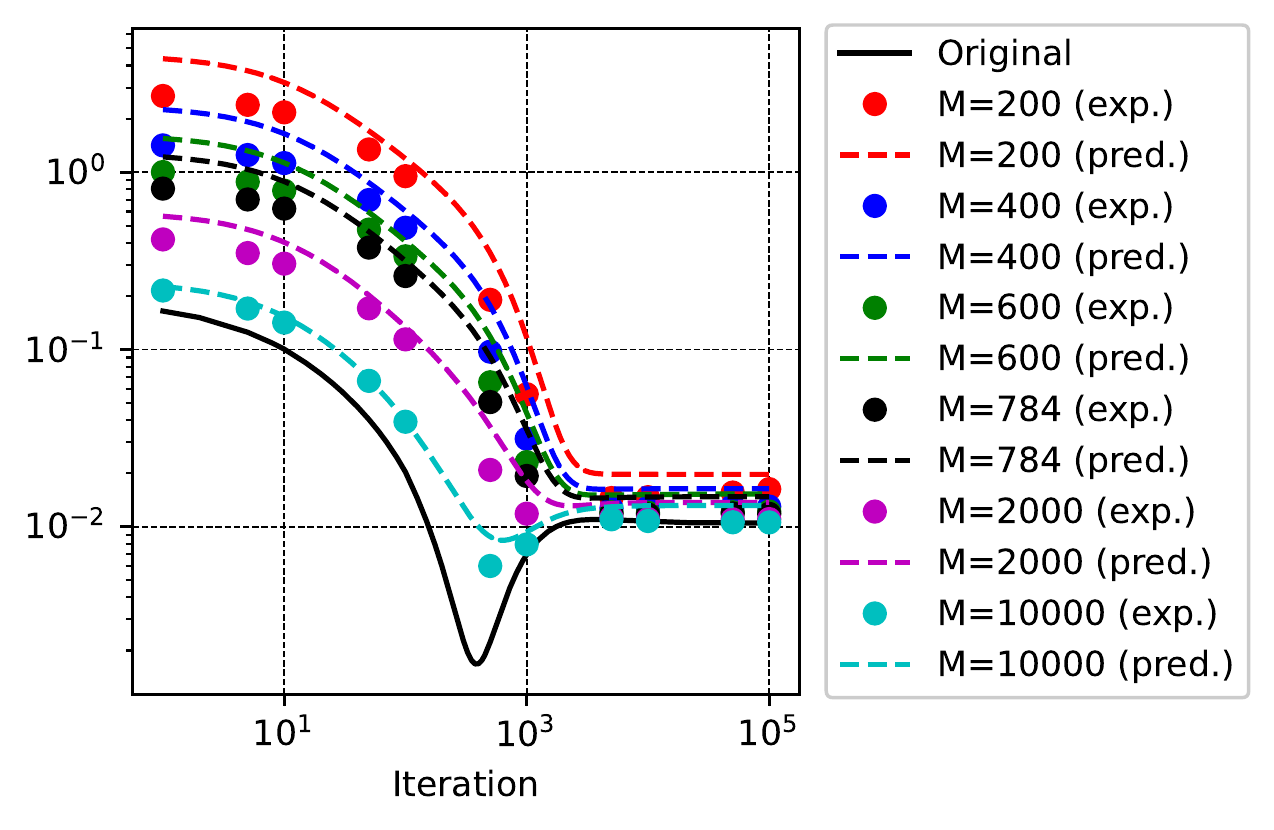}
\par\end{centering}
\caption{Autoencoder with ReLU activation and MNIST data, with regularization.
Same setup as Fig. \ref{fig:MNIST_reg_loss_theta}. The reconstruction
error is plotted against the SGD iteration, for the original autoencoder
(tagged as ``original''), as well as several derived autoencoders
constructed by the two-staged process with different numbers of sampled
neurons $M$ at different SGD iterations. Here ``exp.'' indicates
the simulation results, and ``pred.'' indicates the theoretical
prediction computed using the formulas given in Result \ref{res:ReLU_setting_2stage_simplified}.
For more details, see Appendix \ref{sec:Simulation-details}. At convergence,
the increase in the reconstruction error is negligible already at
$M=400$, which is a significant reduction from the image dimension
of $28\times28=784$.}

\label{fig:MNIST_reg_subsampled}
\end{figure}

\begin{figure}
\begin{centering}
\subfloat[]{\begin{centering}
\includegraphics[width=0.45\columnwidth]{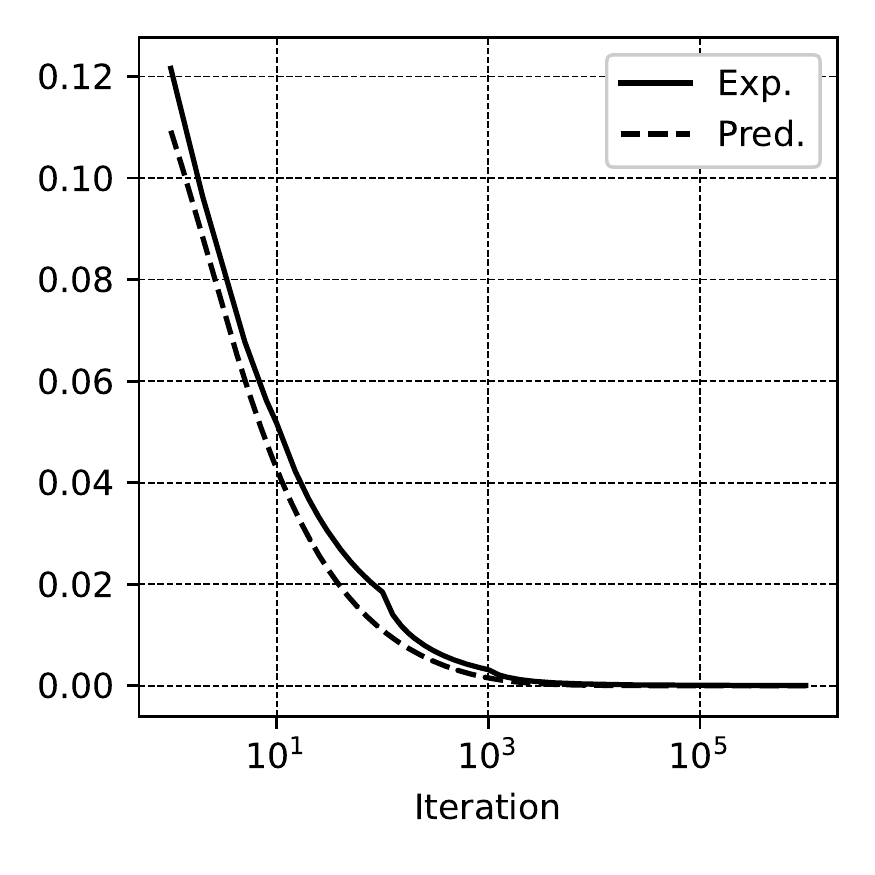}
\par\end{centering}
}\subfloat[]{\begin{centering}
\includegraphics[width=0.45\columnwidth]{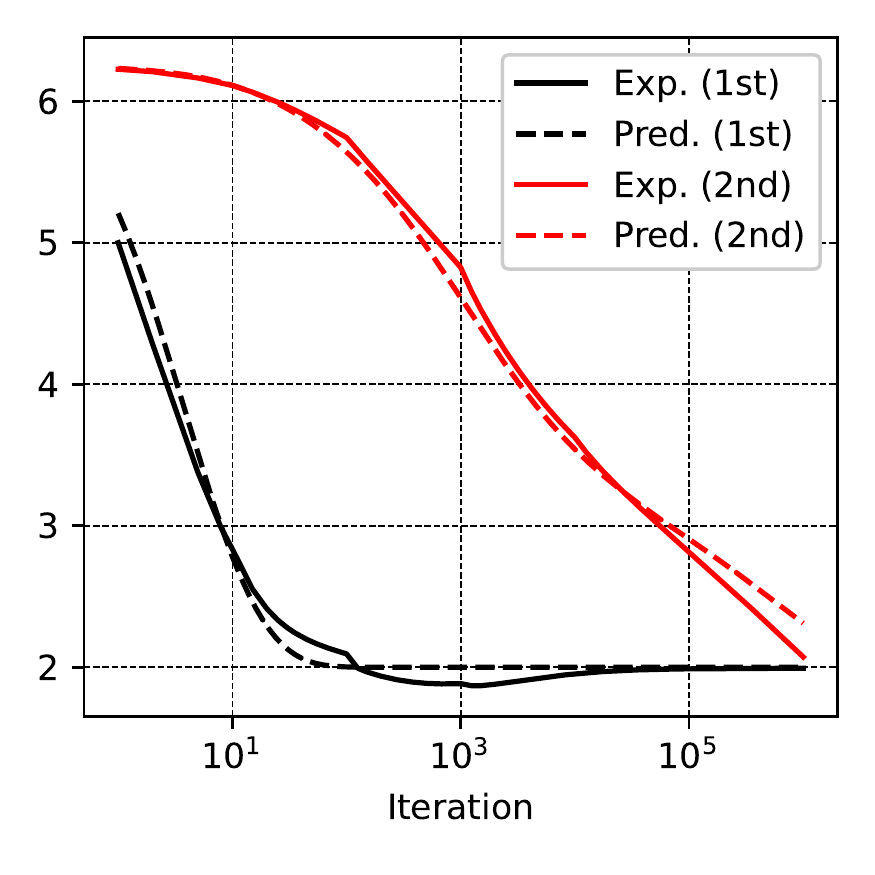}
\par\end{centering}
}
\par\end{centering}
\begin{centering}
\subfloat[]{\begin{centering}
\includegraphics[width=0.9\columnwidth]{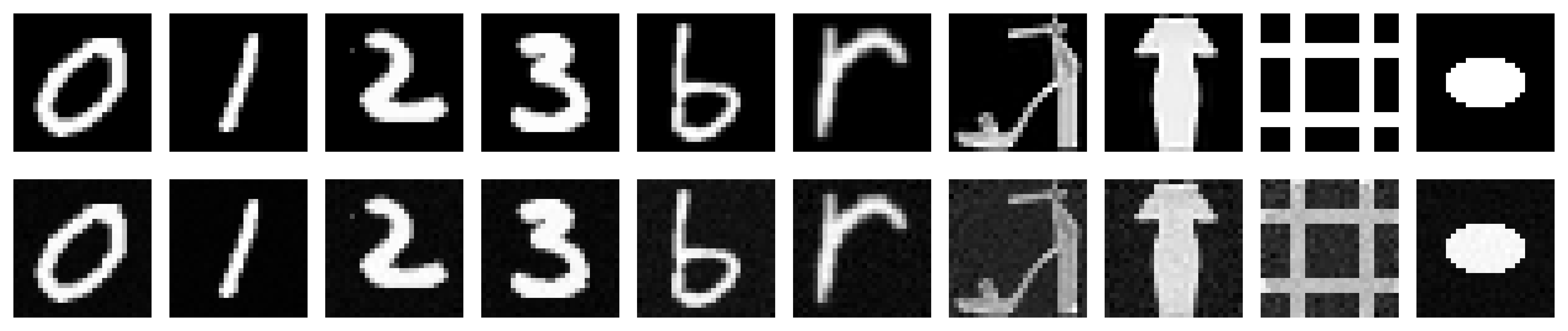}
\par\end{centering}
}
\par\end{centering}
\caption{Autoencoder with ReLU activation and MNIST data, no regularization.
Setup: $\lambda=0$, $r_{0}=2.5$, $\epsilon=0.02$ and $N=20000$.\\(a):
the reconstruction error versus the SGD iteration. Here ``Exp.''
indicates the simulation results, and ``Pred.'' indicates the theoretical
prediction computed using the formulas given in Result \ref{res:ReLU_setting_simplified}.
For more details, see Appendix \ref{sec:Simulation-details}.\\(b):
the normalized squared norm of the first 10-dimensional subspace's
weight (tagged ``1st'') and the second 774-dimensional subspace's
weight (tagged ``2nd''). Since the spectrum of MNIST data concentrates
in the first 10 principal subspaces, the learning speed of the second
subspace would be much slower, as predicted by our theory and demonstrated
by the plot.\\(c): the first row shows four MNIST digit test samples
and six non-digit samples, and the second row shows their respective
reconstructions at iteration $10^{6}$. As predicted by our theory,
the unregularized autoencoder has a tendency to learn an identity
function: the non-digit samples are well reconstructed, even though
the model is not trained with any non-digit samples and we stop training
when the learning of the second subspace has not fully converged.
This is a stark contrast with regularized autoencoders, as demonstrated
in Fig. \ref{fig:MNIST_reg_loss_theta}.}

\label{fig:MNIST_noreg_loss_theta}
\end{figure}

\subsection{Mean field limit for multi-output two-layer networks\label{subsec:contrib-MF}}

All theoretical results stated in Section \ref{subsec:contrib-Gaussian-data}
are, in fact, applications of a result which establishes the mean
field limit for multi-output two-layer neural networks. We first describe
the framework in the following.

\paragraph{Two-layer neural network.}

Given a dimension vector $\mathfrak{Dim}=\left(D,D_{{\rm in}},D_{{\rm out}}\right)$,
we consider the following two-layer network with $N$ neurons:
\begin{equation}
\hat{\boldsymbol{y}}_{N}\left(\boldsymbol{x};\Theta\right)=\frac{1}{N}\sum_{i=1}^{N}\sigma_{*}\left(\boldsymbol{x};\kappa\boldsymbol{\theta}_{i}\right),\label{eq:two_layers_nn}
\end{equation}
where $\Theta=\left(\boldsymbol{\theta}_{i}\right)_{i=1}^{N}$ is
the collection of weights $\theta_{i}\in\mathbb{R}^{D}$, $\boldsymbol{x}\in\mathbb{R}^{D_{{\rm in}}}$
is the input, $\hat{\boldsymbol{y}}_{N}\left(\boldsymbol{x};\Theta\right)\in\mathbb{R}^{D_{{\rm out}}}$
is the output and $\sigma_{*}:\;\mathbb{R}^{D_{{\rm in}}}\times\mathbb{R}^{D}\to\mathbb{R}^{D_{{\rm out}}}$
is the activation function. Let $\mathfrak{Dim}=\left(D,D_{{\rm in}},D_{{\rm out}}\right)$
the dimension vector. Here $\kappa=\kappa\left(\mathfrak{Dim}\right)\geq1$
is a factor that defines the scaling of the weights w.r.t. the dimension.
In order to obtain a non-trivial high-dimensional behavior, this scaling
has to be chosen in a suitable way, as to be discussed later (Section
\ref{subsec:Autoencoder-example}). We assume that the data is distributed
as $\boldsymbol{z}\equiv\left(\boldsymbol{x},\boldsymbol{y}\right)\sim{\cal P}\in\mathscr{P}\left(\mathbb{R}^{D_{{\rm in}}}\times\mathbb{R}^{D_{{\rm out}}}\right)$.
We train the network with stochastic gradient descent (SGD). At each
SGD iteration $k$, we draw independently the data $\boldsymbol{z}^{k}\equiv\left(\boldsymbol{x}^{k},\boldsymbol{y}^{k}\right)\sim{\cal P}$.
Let $\Theta^{k}=\left(\boldsymbol{\theta}_{i}^{k}\right)_{i=1}^{N}$
be the collection of weights at iteration $k$. Given an initialization
$\Theta^{0}$, we perform SGD w.r.t. the squared loss with regularization:
\begin{equation}
\boldsymbol{\theta}_{i}^{k+1}=\boldsymbol{\theta}_{i}^{k}-\epsilon\xi\left(k\epsilon\right)N\nabla_{\boldsymbol{\theta}_{i}}{\rm Loss}\left(\boldsymbol{z}^{k};\Theta^{k}\right),\qquad i=1,...,N,\label{eq:SGD}
\end{equation}
with the training loss being
\[
{\rm Loss}\left(\boldsymbol{z};\Theta\right)=\frac{1}{2}\left\Vert \hat{\boldsymbol{y}}_{N}\left(\boldsymbol{x};\Theta\right)-\boldsymbol{y}\right\Vert _{2}^{2}+\frac{1}{N}\sum_{i=1}^{N}\Lambda\left(\boldsymbol{\theta}_{i};\boldsymbol{z}\right).
\]
Here $\epsilon>0$ is the learning rate, $\xi:\;\mathbb{R}_{\geq0}\to\mathbb{R}_{\geq0}$
is the learning rate schedule, and $\Lambda:\;\mathbb{R}^{D}\times\mathbb{R}^{D_{{\rm in}}}\times\mathbb{R}^{D_{{\rm out}}}\to\mathbb{R}$
is the regularizer. We let $\rho_{N}^{k}$ denote the empirical distribution
of $\Theta^{k}$, i.e.
\[
\rho_{N}^{k}=\frac{1}{N}\sum_{i=1}^{N}\delta_{\boldsymbol{\theta}_{i}^{k}}.
\]

\paragraph{Mean field limit.}

We define the mean field risk, which is a measure of the performance,
as
\begin{align}
{\cal R}\left(\rho\right) & =\mathbb{E}_{{\cal P}}\left\{ \frac{1}{2}\left\Vert \boldsymbol{y}-\int\sigma_{*}\left(\boldsymbol{x};\kappa\boldsymbol{\theta}\right)\rho\left({\rm d}\boldsymbol{\theta}\right)\right\Vert _{2}^{2}\right\} ,\qquad\rho\in\mathscr{P}\left(\mathbb{R}^{D}\right).\label{eq:MF_risk}
\end{align}
We also consider the following continuous-time evolution, for a given
initialization $\rho^{0}\in\mathscr{P}\left(\mathbb{R}^{D}\right)$:
\[
\partial_{t}\rho^{t}\left(\boldsymbol{\theta}\right)=\xi\left(t\right){\rm div}_{\boldsymbol{\theta}}\left(\rho^{t}\left(\boldsymbol{\theta}\right)\nabla_{\boldsymbol{\theta}}\left[V\left(\boldsymbol{\theta}\right)+W\left(\boldsymbol{\theta};\rho^{t}\right)\right]\right),
\]
in which we define:
\begin{align*}
V\left(\boldsymbol{\theta}\right) & =\mathbb{E}_{{\cal P}}\left\{ -\left\langle \sigma_{*}\left(\boldsymbol{x};\kappa\boldsymbol{\theta}\right),\boldsymbol{y}\right\rangle +\Lambda\left(\boldsymbol{\theta},\boldsymbol{z}\right)\right\} ,\\
W\left(\boldsymbol{\theta};\rho\right) & =\int U\left(\boldsymbol{\theta},\boldsymbol{\theta}'\right)\rho\left({\rm d}\boldsymbol{\theta}'\right),\\
U\left(\boldsymbol{\theta},\boldsymbol{\theta}'\right) & =\mathbb{E}_{{\cal P}}\left\{ \left\langle \sigma_{*}\left(\boldsymbol{x};\kappa\boldsymbol{\theta}\right),\sigma_{*}\left(\boldsymbol{x};\kappa\boldsymbol{\theta}'\right)\right\rangle \right\} .
\end{align*}
The above evolution should be interpreted in weak sense, namely $\left(\rho^{t}\right)_{t\geq0}$
is a solution if for any bounded differentiable test function $\phi:\;\mathbb{R}^{D}\to\mathbb{R}$
with bounded gradient:
\[
\frac{{\rm d}}{{\rm d}t}\int\phi\left(\boldsymbol{\theta}\right)\rho^{t}\left({\rm d}\boldsymbol{\theta}\right)=-\xi\left(t\right)\int\left\langle \nabla\phi\left(\boldsymbol{\theta}\right),\nabla_{\boldsymbol{\theta}}\left[V\left(\boldsymbol{\theta}\right)+W\left(\boldsymbol{\theta};\rho^{t}\right)\right]\right\rangle \rho^{t}\left({\rm d}\boldsymbol{\theta}\right).
\]
We shall alternatively work with an equivalent definition of $\left(\rho^{t}\right)_{t\geq0}$,
described by the following nonlinear dynamics:
\begin{equation}
\frac{{\rm d}}{{\rm d}t}\hat{\boldsymbol{\theta}}^{t}=-\xi\left(t\right)\nabla_{\boldsymbol{\theta}}\left[V\left(\hat{\boldsymbol{\theta}}^{t}\right)+W\left(\hat{\boldsymbol{\theta}}^{t};\rho^{t}\right)\right],\qquad\rho^{t}={\rm Law}\left(\hat{\boldsymbol{\theta}}^{t}\right),\qquad\hat{\boldsymbol{\theta}}^{0}\sim\rho^{0}.\label{eq:ODE}
\end{equation}
This dynamics is self-contained, i.e. $\left(\rho^{t}\right)_{t\geq0}$
can be determined from solely Eq. (\ref{eq:ODE}). Observe that given
$\left(\rho^{t}\right)_{t\geq0}$, Eq. (\ref{eq:ODE}) also describes
a (randomly initialized) ODE for the trajectory $\left(\hat{\boldsymbol{\theta}}^{t}\right)_{t\geq0}$,
where $\hat{\boldsymbol{\theta}}^{0}$ is drawn at random according
to $\rho^{0}$. We shall refer to Eq. (\ref{eq:ODE}) as the \textit{nonlinear
dynamics} when discussing $\left(\rho^{t}\right)_{t\geq0}$ and as
the \textit{ODE} when discussing $\left(\hat{\boldsymbol{\theta}}^{t}\right)_{t\geq0}$
on $\left(\rho^{t}\right)_{t\geq0}$.

The basic idea of the mean field limit is that one can track the evolution
of the neural network with its mean field limit. See Section \ref{subsec:Prop-Chaos-result-simplified}
for the result statement. In certain cases, the mean field limit is
analytically tractable, hence aiding the study of the neural network.
This is the case for the autoencoders considered in Section \ref{subsec:contrib-Gaussian-data}.

\subsubsection{The autoencoder example\label{subsec:Autoencoder-example}}

We briefly revisit the $\ell_{2}$-regularized autoencoder described
in Section \ref{subsec:contrib-Gaussian-data}. It is easy to see
that it fits into the framework introduced above. Indeed, the dimensions
$D=D_{{\rm in}}=D_{{\rm out}}=d$ (hence $\mathfrak{Dim}=\left(d,d,d\right)$),
the data $\boldsymbol{y}=\boldsymbol{x}\sim{\cal P}$, the activation
is given by $\sigma_{*}\left(\boldsymbol{x};\kappa\boldsymbol{\theta}\right)=\kappa\boldsymbol{\theta}\sigma\left(\left\langle \kappa\boldsymbol{\theta},\boldsymbol{x}\right\rangle \right)$
with $\kappa=\sqrt{d}$, the regularizer $\Lambda\left(\boldsymbol{\theta};\cdot\right)=\left\Vert \boldsymbol{\theta}\right\Vert _{2}^{2}$
and the learning rate schedule $\xi\left(\cdot\right)=1$.

To make sense of the choice of the factor $\kappa$, we consider $\sigma$
being the ReLU with the following ansatz for the neurons: we generate
the neurons i.i.d. $\boldsymbol{\theta}_{i}\sim\mathsf{N}\left(0,\left(2/d\right)\boldsymbol{I}_{d}\right)$.
With large $N$, we have:
\[
\hat{\boldsymbol{y}}_{N}\left(\boldsymbol{x};\Theta\right)=\frac{1}{N}\sum_{i=1}^{N}\kappa\boldsymbol{\theta}_{i}\sigma\left(\left\langle \kappa\boldsymbol{\theta}_{i},\boldsymbol{x}\right\rangle \right)\approx\mathbb{E}_{\boldsymbol{\theta}_{i}}\left\{ \kappa\boldsymbol{\theta}_{i}\sigma\left(\left\langle \kappa\boldsymbol{\theta}_{i},\boldsymbol{x}\right\rangle \right)\right\} =2\mathbb{E}_{\boldsymbol{\theta}_{i}}\left\{ \sigma'\left(\left\langle \kappa\boldsymbol{\theta}_{i},\boldsymbol{x}\right\rangle \right)\right\} \boldsymbol{x}=\boldsymbol{x}
\]
for any $\boldsymbol{x}\in\mathbb{R}^{d}$, by Stein's lemma. On one
hand, under this ansatz, the autoencoder hence recovers the identity
function -- the same result as a trained unregularized autoencoder
in Section \ref{subsec:contrib-AE-ReLU-SGD}. On the other hand, we
also observe that $\left\Vert \boldsymbol{\theta}_{i}\right\Vert _{2}\leq C$
independent of $\mathfrak{Dim}$. The choice of $\kappa$ thus allows
reasonable functioning of the autoencoder, while maintaining $\left\Vert \boldsymbol{\theta}_{i}\right\Vert _{2}\leq C$.
More generally, this latter ``$\mathfrak{Dim}$-independent'' property
holds for the mean field limit: for $\boldsymbol{\theta}\sim\rho^{t}$,
we have $\left\Vert \boldsymbol{\theta}\right\Vert _{2}\leq C$ in
an appropriate sense.

\subsubsection{Main result\label{subsec:Prop-Chaos-result-simplified}}

We recall the mean field risk ${\cal R}\left(\rho\right)$ in (\ref{eq:MF_risk}),
the empirical distribution $\rho_{N}^{k}$ of the neural network's
collection of weights $\Theta^{k}$ at SGD iteration $k$ and note
that
\[
{\cal R}\left(\rho_{N}^{k}\right)=\mathbb{E}_{{\cal P}}\left\{ \frac{1}{2}\left\Vert \hat{\boldsymbol{y}}_{N}\left(\boldsymbol{x};\Theta^{k}\right)-\boldsymbol{y}\right\Vert _{2}^{2}\right\} .
\]
In general, the above identity holds for any collection of parameters
(replacing $\Theta^{k}$) and its respective empirical distribution
(replacing $\rho_{N}^{k}$). In the setting of the autoencoders (Section
\ref{subsec:contrib-Gaussian-data}), one easily recognizes that ${\rm RecErr}\left(\Theta^{k}\right)={\cal R}\left(\rho_{N}^{k}\right)$.

Our main result connects $\rho^{t}$ of the mean field limit with
$\Theta^{t/\epsilon}$ of the neural network.
\begin{res}[Two-layer network -- Informal and simplified]
\label{res:propChaos-simplified}Consider the two-layer neural network
and its mean field limit as described in Section \ref{subsec:contrib-MF}.
Suppose that we generate the SGD initialization $\Theta^{0}=\left(\boldsymbol{\theta}_{i}^{0}\right)_{i\leq N}\sim_{{\rm i.i.d.}}\rho^{0}$.
Also assume that $\kappa=O\left({\rm poly}\left(\mathfrak{Dim}\right)\right)$.

Under certain regularity conditions, for $N\gg{\rm poly}\left(\mathfrak{Dim}\right)$
and $\epsilon\ll1/{\rm poly}\left(\mathfrak{Dim}\right)$ and a finite
$t\in\mathbb{N}\epsilon$, $t\leq C$, with high probability,
\[
\rho_{N}^{t/\epsilon}\approx\rho^{t},\qquad{\cal R}\left(\rho_{N}^{t/\epsilon}\right)\approx{\cal R}\left(\rho^{t}\right).
\]
Furthermore, given a positive integer $M$, construct a set of indices
$\left(h\left(i\right)\right)_{i\leq M}$ by sampling independently
at random $h\left(i\right)$ from $\left[N\right]$, for each $i\in\left[M\right]$.
Then with high probability,
\[
{\cal R}\left(\nu_{M}^{t/\epsilon}\right)\approx{\cal R}\left(\bar{\nu}_{M}^{t}\right),
\]
where we define $\nu_{M}^{t/\epsilon}=\left(1/M\right)\cdot\sum_{i=1}^{M}\delta_{\boldsymbol{\theta}_{h\left(i\right)}^{t/\epsilon}}$
and $\bar{\nu}_{M}^{t}=\left(1/M\right)\cdot\sum_{i=1}^{M}\delta_{\bar{\boldsymbol{\theta}}_{h\left(i\right)}^{t}}$
for $\left(\bar{\boldsymbol{\theta}}_{i}^{t}\right)_{i\leq N}\sim_{{\rm i.i.d.}}\rho^{t}$.

In the above, the constants $C$ do not depend on $N$, $\epsilon$
or the dimension vector $\mathfrak{Dim}$.
\end{res}

Exact details can be found in the statement of Theorem \ref{thm:propChaos}.
It can be observed that the conclusions of Results \ref{res:ReLU_setting_simplified},
\ref{res:ReLU_setting_2stage_simplified} and \ref{res:Bdd_act_setting_simplified}
are reminiscent of, and indeed consequences of, Result \ref{res:propChaos-simplified}.
It should also be noted that the required regularity conditions of
Result \ref{res:propChaos-simplified} are non-trivial. Indeed a major
technical part of this work is devoted to verifying these conditions
for the autoencoder settings.

This result is in line with the previous works on two-layer networks
\cite{mei2018mean,mei2019mean}. A key difference with respect to
the work \cite{mei2019mean} is that in \cite{mei2019mean}, the number
of neurons $N$ can be independent of $\mathfrak{Dim}$, whereas here
we require $N\gg{\rm poly}\left(\mathfrak{Dim}\right)$. This difference
is due to the differences between the setups and poses an interesting,
yet highly non-trivial technical challenge, which requires a new proof
strategy. We delve into this issue in the next section.

\subsubsection{Technical challenge\label{subsec:Technical-challenge}}

We explain here the key technical challenge in our setting, compared
to the work \cite{mei2019mean}. Both \cite{mei2019mean} and our
work employ a propagation of chaos argument, following \cite{sznitman1991topics}.
To fix ideas, let us give a heuristic treatment of a simplified problem.
Consider the following continuous-time dynamics of $N$ particles
$\left(\boldsymbol{\theta}_{j}^{t}\right)_{j\leq N}$:
\[
\frac{{\rm d}}{{\rm d}t}\boldsymbol{\theta}_{i}^{t}=\boldsymbol{f}\left(\boldsymbol{\theta}_{i}^{t};\rho_{N}^{t}\right),\qquad\rho_{N}^{t}=\frac{1}{N}\sum_{j=1}^{N}\delta_{\boldsymbol{\theta}_{j}^{t}}.
\]
The mean field limit counterpart is given by the following nonlinear
dynamics:
\[
\frac{{\rm d}}{{\rm d}t}\hat{\boldsymbol{\theta}}^{t}=\boldsymbol{f}\left(\hat{\boldsymbol{\theta}}^{t};\rho^{t}\right),\qquad\rho^{t}={\rm Law}\left(\hat{\boldsymbol{\theta}}^{t}\right).
\]
The argument proceeds with the following coupling. We first generate
the initializations of the particles $\left(\boldsymbol{\theta}_{j}^{0}\right)_{j\leq N}\sim_{{\rm i.i.d.}}\rho^{0}$.
Then we obtain $N$ i.i.d. copies of the mean field dynamics:
\[
\frac{{\rm d}}{{\rm d}t}\bar{\boldsymbol{\theta}}_{i}^{t}=\boldsymbol{f}\left(\bar{\boldsymbol{\theta}}_{i}^{t};\rho^{t}\right),\qquad\bar{\boldsymbol{\theta}}_{i}^{0}=\boldsymbol{\theta}_{i}^{0},\qquad i=1,...,N.
\]
Note that $\left(\bar{\boldsymbol{\theta}}_{j}^{t}\right)_{j\leq N}\sim_{{\rm i.i.d.}}\rho^{t}$
for all time $t$. The goal is to approximate $\left(\boldsymbol{\theta}_{j}^{t}\right)_{j\leq N}$
with $\left(\bar{\boldsymbol{\theta}}_{j}^{t}\right)_{j\leq N}$.
The first step is to realize that
\[
\frac{{\rm d}}{{\rm d}t}\bar{\boldsymbol{\theta}}_{i}^{t}=\boldsymbol{f}\left(\bar{\boldsymbol{\theta}}_{i}^{t};\bar{\rho}_{N}^{t}\right)+\Theta\left(N^{-\gamma}\right),\qquad\bar{\rho}_{N}^{t}=\frac{1}{N}\sum_{j=1}^{N}\delta_{\bar{\boldsymbol{\theta}}_{j}^{t}},
\]
as a consequence of concentration of measure, for an absolute constant
$\gamma>0$. Next, the analysis of \cite{mei2018mean,mei2019mean}
compares $\boldsymbol{f}\left(\boldsymbol{\theta}_{i}^{t};\rho_{N}^{t}\right)$
with $\boldsymbol{f}\left(\bar{\boldsymbol{\theta}}_{i}^{t};\bar{\rho}_{N}^{t}\right)$:
\begin{equation}
\max_{i\leq N}\left\Vert \boldsymbol{f}\left(\boldsymbol{\theta}_{i}^{t};\rho_{N}^{t}\right)-\boldsymbol{f}\left(\bar{\boldsymbol{\theta}}_{i}^{t};\bar{\rho}_{N}^{t}\right)\right\Vert _{2}\leq L\max_{i\leq N}\left\Vert \boldsymbol{\theta}_{i}^{t}-\bar{\boldsymbol{\theta}}_{i}^{t}\right\Vert _{2},\label{eq:technical_challenge_L}
\end{equation}
for some constant $L>0$. Gronwall's lemma then yields the desired
approximation:
\[
\max_{i\leq N}\left\Vert \boldsymbol{\theta}_{i}^{t}-\bar{\boldsymbol{\theta}}_{i}^{t}\right\Vert _{2}\leq\Theta\left(N^{-\gamma}\right)\exp\left(Lt\right)\stackrel{N\to\infty}{\longrightarrow}0.
\]
In other words, this argument requires $N\gg\exp\left(CL\right)$.
In \cite{mei2019mean}, several structural assumptions are made so
that $L$ and thus the required $N$ are independent of the dimension
vector $\mathfrak{Dim}$. This is, however, not the case in our setting,
owing to the presence of $\kappa$ in Eq. (\ref{eq:two_layers_nn}).
In particular, a naive adaptation of the approach of \cite{mei2018mean,mei2019mean}
would result in $N\gg\exp\left(\mathfrak{Dim}^{O\left(1\right)}\right)$
even if $\kappa=O\left({\rm poly}\left(\mathfrak{Dim}\right)\right)$,
which is undesirable. Is it necessary that $N\gg\exp\left(\mathfrak{Dim}^{O\left(1\right)}\right)$
in our setting? Is it possible that $N$ can be made independent of
$\mathfrak{Dim}$?

Result \ref{res:propChaos-simplified} achieves the first positive
step in this quest, showing that $N\gg{\rm poly}\left(\mathfrak{Dim}\right)$
is sufficient. To that end, we take a different approach that is inspired
by analyses of vortex methods for Euler equations (see e.g. \cite{goodman1990convergence}).
The specific form of the gradient flow learning dynamics is important
for our analysis to hold. On the other hand, as observed in \cite{nguyen2020rigorous},
the analyses of \cite{mei2018mean,mei2019mean} are applicable to
more general $\boldsymbol{f}$ at the expense of certain stronger
structural assumptions.

We believe the requirement $N\gg{\rm poly}\left(\mathfrak{Dim}\right)$
is not a mere proof artifact. Recall that the collection of neurons
$\Theta^{t/\epsilon}$ is approximated by the measure $\rho^{t}$
of the mean field limit. Result \ref{res:ReLU_setting_2stage_simplified}
and the analysis in Section \ref{subsec:contrib_two_stage} show that,
in our autoencoder example with ReLU activation, already given knowledge
of $\rho^{t}$, we still need to sample $M\gg d$ neurons to guarantee
a good approximation, where we recall $d$ is the data dimension.
Indeed the sampling error component in Eq. (\ref{eq:thm_ReLU_2stage_simplified})
becomes significant if $M\ll d$. We conjecture that under a suitable
set of assumptions (in which $L$ from Eq. (\ref{eq:technical_challenge_L})
is still $\mathfrak{Dim}$-dependent and hence the main difficulty
is not artificially removed), the conclusions of Result \ref{res:propChaos-simplified}
can hold with $N\gg\mathfrak{Dim}$, a milder requirement than $N\gg{\rm poly}\left(\mathfrak{Dim}\right)$.
In fact, our analysis suggests an even bolder conjecture: $N\gg d_{{\rm eff}}$
is necessary, and under special circumstances, it is also sufficient,
where $d_{{\rm eff}}$ is a quantity characteristic of the data distribution
such that $d_{{\rm eff}}=O\left(\mathfrak{Dim}\right)$ generally
and $d_{{\rm eff}}=o\left(\mathfrak{Dim}\right)$ for certain data
distributions. It would be interesting to find a propagation of chaos
argument that proves the conjectures.

\section{Mean field limit of multi-output two-layer networks\label{sec:MF_two-layers}}

We recall the framework as described in Section \ref{subsec:contrib-MF}.
In particular, we recall the neural network (\ref{eq:two_layers_nn}),
its SGD learning dynamics (\ref{eq:SGD}) and its associated mean
field limit that is described via the nonlinear dynamics (\ref{eq:ODE}).

\subsection{Theorem statement}

In the following, we let the parameters $\kappa_{i}\geq1$, $i=1,2,...,6$,
to depend exclusively on $\mathfrak{Dim}=\left(D,D_{{\rm in}},D_{{\rm out}}\right)$.
We consider a finite terminal time $T$, and allow the constants $C$
(hidden in $\lesssim$) to depend on $T$ but not $N$, $\epsilon$
or $\mathfrak{Dim}$, such that $C$ is finite for finite $T$. Recalling
Eq. (\ref{eq:SGD}), we define:
\begin{align*}
\boldsymbol{F}_{i}\left(\Theta;\boldsymbol{z}\right) & =N\nabla_{\boldsymbol{\theta}_{i}}{\rm Loss}\left(\boldsymbol{z};\Theta\right)\\
 & =\kappa\nabla_{2}\sigma_{*}\left(\boldsymbol{x};\kappa\boldsymbol{\theta}_{i}\right)^{\top}\left(\hat{\boldsymbol{y}}_{N}\left(\boldsymbol{x};\Theta\right)-\boldsymbol{y}\right)+\nabla_{1}\Lambda\left(\boldsymbol{\theta}_{i},\boldsymbol{z}\right).
\end{align*}
We list below our assumptions:
\begin{enumerate}[{label=\textbf{[A.\arabic*]},ref=[A.\arabic*]}]
\item \label{enu:Assump_ODE}The initial law $\rho^{0}$ is such that for
$\boldsymbol{\theta}^{0}\sim\rho^{0}$, $\left\Vert \boldsymbol{\theta}^{0}\right\Vert _{2}$
is $C$-sub-Gaussian with $\mathbb{E}\left\{ \left\Vert \boldsymbol{\theta}^{0}\right\Vert _{2}\right\} \leq C$
and $C$ being $\mathfrak{Dim}$-independent constants. By this, we
mean $\mathbb{E}\left\{ \left\Vert \boldsymbol{\theta}^{0}\right\Vert _{2}^{p}\right\} ^{1/p}\leq C\sqrt{p}$
for all $p\geq1$. We assume that the nonlinear dynamics (\ref{eq:ODE})
has a weakly unique solution $\left(\rho^{t}\right)_{t\geq0}$.
\item \label{enu:Assump_lr}The learning rate schedule $\xi:\;\mathbb{R}_{\geq0}\to\mathbb{R}_{\geq0}$
satisfies: $\left|\xi\left(t\right)\right|\lesssim1$ and $\left|\xi\left(t_{1}\right)-\xi\left(t_{2}\right)\right|\lesssim\left|t_{1}-t_{2}\right|$.
\item \label{enu:Assump_growth}Given the solution $\left(\rho^{t}\right)_{t\geq0}$
to the nonlinear dynamics (\ref{eq:ODE}), the functions $V$, $W$
and $U$ satisfy the following growth conditions:
\begin{align*}
\left\Vert \nabla V\left(\boldsymbol{\theta}\right)\right\Vert _{2} & \lesssim\left\Vert \boldsymbol{\theta}\right\Vert _{2}+1,\\
\left\Vert \nabla V\left(\boldsymbol{\theta}_{1}\right)-\nabla V\left(\boldsymbol{\theta}_{2}\right)\right\Vert _{2} & \lesssim\left\Vert \boldsymbol{\theta}_{1}-\boldsymbol{\theta}_{2}\right\Vert _{2},\\
\left\Vert \nabla_{1}W\left(\boldsymbol{\theta};\rho\right)\right\Vert _{2} & \lesssim\left\Vert \boldsymbol{\theta}\right\Vert _{2}+1,\\
\left\Vert \nabla_{1}W\left(\boldsymbol{\theta}_{1};\rho\right)-\nabla_{1}W\left(\boldsymbol{\theta}_{2};\rho\right)\right\Vert _{2} & \lesssim\left\Vert \boldsymbol{\theta}_{1}-\boldsymbol{\theta}_{2}\right\Vert _{2},\\
\left\Vert \nabla_{1}U\left(\boldsymbol{\theta},\boldsymbol{\theta}'\right)\right\Vert _{2} & \lesssim\kappa_{1}\left(\left\Vert \boldsymbol{\theta}\right\Vert _{2}+1\right)\left(\left\Vert \boldsymbol{\theta}'\right\Vert _{2}^{2}+1\right),
\end{align*}
for any $\rho$ on the trajectory $\left(\rho^{t}\right)_{t\in\left[0,T\right]}$.
Furthermore,
\[
\left\Vert \nabla_{1}W\left(\boldsymbol{\theta};\rho^{t_{1}}\right)-\nabla_{1}W\left(\boldsymbol{\theta};\rho^{t_{2}}\right)\right\Vert _{2}\lesssim\left(\left\Vert \boldsymbol{\theta}\right\Vert _{2}+1\right)\left|t_{2}-t_{1}\right|,
\]
for $t_{1},t_{2}\leq T$.
\item \label{enu:Assump_opNorm}The function $U$ satisfies the following
operator norm bounds:
\begin{align*}
\left\Vert \nabla_{12}^{2}U\left(\boldsymbol{\theta},\boldsymbol{\theta}'\right)\right\Vert _{{\rm op}} & \lesssim\kappa_{2}\left(\left\Vert \boldsymbol{\theta}\right\Vert _{2}+1\right)\left(\left\Vert \boldsymbol{\theta}'\right\Vert _{2}+1\right),\\
\left\Vert \nabla_{121}^{3}U\left[\boldsymbol{\zeta},\boldsymbol{\theta}\right]\right\Vert _{{\rm op}} & \lesssim\kappa_{3}\left(\left\Vert \boldsymbol{\theta}\right\Vert _{2}+1\right),\\
\left\Vert \nabla_{122}^{3}U\left[\boldsymbol{\theta},\boldsymbol{\zeta}\right]\right\Vert _{{\rm op}} & \lesssim\kappa_{4}\left(\left\Vert \boldsymbol{\theta}\right\Vert _{2}+1\right).
\end{align*}
\item \label{enu:Assump_F}The SGD update $\boldsymbol{F}_{i}\left(\Theta;\boldsymbol{z}\right)$
is sub-exponential (w.r.t. $\boldsymbol{z}\sim{\cal P}$) with $\psi_{1}$-norm:
\[
\left\Vert \boldsymbol{F}_{i}\left(\Theta;\boldsymbol{z}\right)\right\Vert _{\psi_{1}}\lesssim\kappa_{5}\left(\left\Vert \boldsymbol{\theta}_{i}\right\Vert _{2}+1\right)\left(\frac{1}{N}\sum_{j=1}^{N}\left\Vert \boldsymbol{\theta}_{j}\right\Vert _{2}^{2}+1\right),
\]
where $\Theta=\left(\boldsymbol{\theta}_{i}\right)_{i\leq N}$.
\item \label{enu:Assump_nabla11U}Given the solution $\left(\rho^{t}\right)_{t\geq0}$
to the nonlinear dynamics (\ref{eq:ODE}), let $\left(\hat{\boldsymbol{\theta}}_{j}^{t}\right)_{t\leq T,\;j\leq N}$
be i.i.d. copies of the ODE (\ref{eq:ODE}) with initializations $\left(\hat{\boldsymbol{\theta}}_{j}^{0}\right)_{j\leq N}\sim_{{\rm i.i.d.}}\rho^{0}$.
We have for any $c>0$,
\[
\mathbb{P}\left\{ \sup_{t\leq T}\sup_{\boldsymbol{\zeta}\in{\cal B}_{D}\left(c\sqrt{N}\right)}\left\Vert \frac{1}{N}\sum_{j=1}^{N}\nabla_{11}^{2}U\left(\boldsymbol{\zeta},\hat{\boldsymbol{\theta}}_{j}^{t}\right)\right\Vert _{{\rm op}}\geq c_{\ref{enu:Assump_nabla11U}}\left(T,c\right)\right\} \leq\Xi\left(N;T,\kappa_{6}\right),
\]
for functions $\Xi$ and $c_{\ref{enu:Assump_nabla11U}}$ such that
$\Xi\left(N;T,\kappa_{6}\right)\to0$ as $N\to\infty$, and $c_{\ref{enu:Assump_nabla11U}}\left(T,c\right)$
is finite with finite $c$ and $T$. We emphasize that in the right-hand
side of the above event, $c_{\ref{enu:Assump_nabla11U}}$ is independent
of $\mathfrak{Dim}$, unlike those in Assumption \ref{enu:Assump_opNorm}.
\item \label{enu:Assump_additional_1}The regularizer $\Lambda$ satisfies
the growth condition:
\begin{align*}
\left\Vert \nabla_{\boldsymbol{\theta}}\mathbb{E}_{{\cal P}}\left\{ \Lambda\left(\boldsymbol{\theta},\boldsymbol{z}\right)\right\} \right\Vert _{2} & \lesssim\left\Vert \boldsymbol{\theta}\right\Vert _{2}+1.
\end{align*}
Furthermore, under Assumption \ref{enu:Assump_ODE}, given the solution
$\left(\rho^{t}\right)_{t\geq0}$ to the nonlinear dynamics (\ref{eq:ODE}),
$\left|V\left(\boldsymbol{0}\right)\right|$, $\left|\mathbb{E}_{{\cal P}}\left\{ \Lambda\left(\boldsymbol{0},\boldsymbol{z}\right)\right\} \right|$,
$\left|U\left(\boldsymbol{0},\boldsymbol{0}\right)\right|\leq C$
and $\left|W\left(\boldsymbol{0};\rho\right)\right|\leq C$ for any
$\rho$ on the trajectory $\left(\rho^{t}\right)_{t\in\left[0,T\right]}$.
(In fact, one can alternatively replace vector $\boldsymbol{0}$ in
the last condition with a constant vector $\boldsymbol{u}\in\mathbb{R}^{D}$
with $\left\Vert \boldsymbol{u}\right\Vert _{2}\leq C$.)
\end{enumerate}
\begin{rem}
\label{rem:ODE-existence}Let us remark that under $\left(\rho^{t}\right)_{t\geq0}$
the unique weak solution to the nonlinear dynamics (\ref{eq:ODE})
(Assumption \ref{enu:Assump_ODE}), the ODE (\ref{eq:ODE}) has a
unique solution $\left(\hat{\boldsymbol{\theta}}^{t}\right)_{t\in\left[0,T\right]}$.
Indeed, by Assumption \ref{enu:Assump_growth}, $\nabla V$ and $\nabla_{1}W\left(\cdot;\rho^{t}\right)$
are both $C$-Lipschitz uniformly in $t\in\left[0,T\right]$, and
similarly by Assumption \ref{enu:Assump_lr}, $\xi$ is bounded and
Lipschitz. The existence of a unique solution $\left(\hat{\boldsymbol{\theta}}^{t}\right)_{t\in\left[0,T\right]}$
then follows from a standard argument. In fact, there exists such
unique solution on $t\in[0,\infty)$. This shows that the trajectories
$\left(\hat{\boldsymbol{\theta}}_{j}^{t}\right)_{t\leq T,\;j\leq N}$
in Assumption \ref{enu:Assump_nabla11U} are well-defined.
\end{rem}

\begin{rem}
Although Assumption \ref{enu:Assump_nabla11U} requires the statement
to hold \textit{for all} $c>0$, we note that in fact it suffices
to alternatively assume a weaker condition, in which the same statement
holds for \textsl{some} sufficiently large constant $c$ that is independent
of $\mathfrak{Dim}$, $N$ and $\epsilon$. How large it is depends
on other constants hidden in other assumptions, and as such, we choose
to state Assumption \ref{enu:Assump_nabla11U} in the current form
only for ease of presentation.
\end{rem}

We again emphasize that $\kappa_{1},...,\kappa_{6}$ depend exclusively
on $\mathfrak{Dim}$. Even though we are primarily interested in dependencies
that are at most polynomial in $\mathfrak{Dim}$, the theorem we shall
prove holds for any dependency. We now state the main theorem.
\begin{thm}
\label{thm:propChaos}Suppose that we generate the SGD initialization
$\Theta^{0}=\left(\boldsymbol{\theta}_{i}^{0}\right)_{i\leq N}\sim_{{\rm i.i.d.}}\rho^{0}$.
Assume the conditions \ref{enu:Assump_ODE}-\ref{enu:Assump_nabla11U}
to hold. Given $\delta>1$ and a finite $T\in\mathbb{N}\epsilon$,
further assume that
\[
\epsilon\lesssim\frac{1}{\max\left\{ \kappa_{2}^{2},\;\left(\kappa_{3}+\kappa_{4}\right)^{2}\kappa_{5}^{2}D^{2}\delta^{2}\right\} },\qquad\left(\delta^{2}+\log^{5}\left(\frac{NT}{\epsilon}+1\right)\right)\frac{\kappa_{1}^{2}}{N}\lesssim\frac{1}{\left(\kappa_{3}+\kappa_{4}\right)^{2}}.
\]
Let $\left(\rho^{t}\right)_{t\geq0}$ be the unique weak solution
of the nonlinear dynamics (\ref{eq:ODE}). Also recall that $\rho_{N}^{k}$
denotes the empirical distribution of $\Theta^{k}$, namely $\rho_{N}^{k}=\left(1/N\right)\sum_{i=1}^{N}\delta_{\boldsymbol{\theta}_{i}^{k}}.$
Then:
\begin{enumerate}[{label=\textbf{[B.\arabic*]},ref=[B.\arabic*]}]
\item \label{enu:propChaos_thm_main}For each $i\in\left[N\right]$, let
$\left(\bar{\boldsymbol{\theta}}_{i}^{t}\right)_{t\geq0}$ be the
solution of the ODE (\ref{eq:ODE}) on $\left(\rho^{t}\right)_{t\geq0}$
with the initialization $\bar{\boldsymbol{\theta}}_{i}^{0}=\boldsymbol{\theta}_{i}^{0}$.
Then:
\[
\mathbb{P}\left\{ \max_{k\leq T/\epsilon}\frac{1}{N}\sum_{i=1}^{N}\left\Vert \boldsymbol{\theta}_{i}^{k}-\bar{\boldsymbol{\theta}}_{i}^{k\epsilon}\right\Vert _{2}^{2}\apprge\mathsf{err}\left(N,\epsilon,\delta\right)\right\} \lesssim\mathsf{prob}\left(N,\delta\right),
\]
in which we define
\begin{align*}
\mathsf{err}\left(N,\epsilon,\delta\right) & =\left(\delta^{2}+\log^{5}\left(\frac{NT}{\epsilon}+1\right)\right)\frac{\kappa_{1}^{2}}{N}+\sqrt{\epsilon}\frac{\kappa_{5}}{\kappa_{3}+\kappa_{4}}\delta+\epsilon D^{2}\kappa_{5}^{2}\delta,\\
\mathsf{prob}\left(N,\delta\right) & =\delta^{-2}+\Xi\left(N;T,\kappa_{6}\right)+\exp\left(-N^{1/8}\right).
\end{align*}
\item \label{enu:propChaos_thm_phi}For any $1$-Lipschitz function $\phi:\;\mathbb{R}^{D}\to\mathbb{R}$
and any $\epsilon_{0}>0$,
\[
\max_{t\in\mathbb{N}\epsilon\cap\left[0,T\right]}\left|\frac{1}{N}\sum_{i=1}^{N}\phi\left(\boldsymbol{\theta}_{i}^{t/\epsilon}\right)-\int\phi\left(\boldsymbol{\theta}\right)\rho^{t}\left({\rm d}\boldsymbol{\theta}\right)\right|\lesssim\epsilon_{0}+\sqrt{\mathsf{err}\left(N,\epsilon,\delta\right)},
\]
with probability at least 
\[
1-C\mathsf{prob}\left(N,\delta\right)-\frac{CT}{\epsilon}\exp\left(-CN\epsilon_{0}^{2}\right).
\]
\item \label{enu:propChaos_thm_risk}If we further assume condition \ref{enu:Assump_additional_1},
then
\[
\max_{t\in\mathbb{N}\epsilon\cap\left[0,T\right]}\left|{\cal R}\left(\rho_{N}^{t/\epsilon}\right)-{\cal R}\left(\rho^{t}\right)\right|\lesssim\kappa_{1}\sqrt{\mathsf{err}\left(N,\epsilon,\delta\right)}+\epsilon_{1},
\]
with probability at least 
\[
1-C\mathsf{prob}\left(N,\delta\right)-\frac{CNT}{\epsilon}\exp\left(-C\epsilon_{1}^{1/3}\left(\frac{N}{\kappa_{1}^{2}}\right)^{1/6}\right),
\]
for any $\epsilon_{1}\in\left(0,1\right)$.
\item \label{enu:propChaos_thm_sampling}Given a positive integer $M$,
construct a set of indices $\left(h\left(i\right)\right)_{i\leq M}$
by sampling independently at random $h\left(i\right)$ from $\left[N\right]$,
for each $i\in\left[M\right]$. If we further assume condition \ref{enu:Assump_additional_1},
then for any $\delta_{0}>0$ and $t\in\mathbb{N}\epsilon\cap\left[0,T\right]$,
\[
\mathbb{P}\left\{ \left|{\cal R}\left(\nu_{M}^{t/\epsilon}\right)-{\cal R}\left(\bar{\nu}_{M}^{t}\right)\right|\gtrsim\kappa_{1}\left(\delta_{0}^{2}+1\right)\sqrt{\mathsf{err}\left(N,\epsilon,\delta\right)}\right\} \lesssim\mathsf{prob}\left(N,\delta\right)+\delta_{0}^{-1}+e^{-M},
\]
where we define $\nu_{M}^{t/\epsilon}=\left(1/M\right)\cdot\sum_{i=1}^{M}\delta_{\boldsymbol{\theta}_{h\left(i\right)}^{t/\epsilon}}$
and $\bar{\nu}_{M}^{t}=\left(1/M\right)\cdot\sum_{i=1}^{M}\delta_{\bar{\boldsymbol{\theta}}_{h\left(i\right)}^{t}}$,
recalling the definition of $\left(\bar{\boldsymbol{\theta}}_{i}^{t}\right)_{i\leq N}$
in Claim \ref{enu:propChaos_thm_main}.
\end{enumerate}
In the above, the constants $C$ (hidden in $\lesssim$) depend on
$T$, but not $N$, $\epsilon$, the dimension vector $\mathfrak{Dim}$,
$\delta$, $\delta_{0}$, $\epsilon_{0}$ or $\epsilon_{1}$, such
that $C$ is finite for finite $T$.
\end{thm}

\subsection{Proof of Theorem \ref{thm:propChaos}\label{subsec:propChaos-main-proof}}

\subsubsection*{Step 0: Preliminaries}

We start with several preliminaries, some of which are restated for
ease of reading. We define $\boldsymbol{G}:\;\mathbb{R}^{D}\times\mathscr{P}\left(\mathbb{R}^{D}\right)\to\mathbb{R}^{D}$,
by
\[
\boldsymbol{G}\left(\boldsymbol{\theta};\rho\right)=\nabla V\left(\boldsymbol{\theta}\right)+\int\nabla_{1}U\left(\boldsymbol{\theta},\boldsymbol{\theta}'\right)\rho\left({\rm d}\boldsymbol{\theta}'\right)=\nabla V\left(\boldsymbol{\theta}\right)+\nabla_{1}W\left(\boldsymbol{\theta};\rho\right).
\]
Given an initial law $\rho^{0}$, we consider $N$ i.i.d. copies $\left(\bar{\boldsymbol{\theta}}_{i}^{t}\right)_{t\leq T,\;i\leq N}$
of the ODE (\ref{eq:ODE}) with initializations $\left(\bar{\boldsymbol{\theta}}_{i}^{0}\right)_{i\leq N}\sim_{{\rm i.i.d.}}\rho^{0}$:
\[
\bar{\boldsymbol{\theta}}_{i}^{t}=\bar{\boldsymbol{\theta}}_{i}^{0}-\int_{0}^{t}\xi\left(s\right)\boldsymbol{G}\left(\bar{\boldsymbol{\theta}}_{i}^{s};\rho^{s}\right){\rm d}s,\qquad\rho^{t}=\text{Law}\left(\bar{\boldsymbol{\theta}}_{i}^{t}\right).
\]
We note that $\left(\bar{\boldsymbol{\theta}}_{i}^{t}\right)_{t\leq T}$
is well-defined by Remark \ref{rem:ODE-existence}. We also remind
of the SGD dynamics $\Theta^{k}=\left(\boldsymbol{\theta}_{i}^{k}\right)_{i\leq N}$
with initialization $\boldsymbol{\theta}_{i}^{0}=\bar{\boldsymbol{\theta}}_{i}^{0}$:
\[
\boldsymbol{\theta}_{i}^{k}=\boldsymbol{\theta}_{i}^{0}-\epsilon\sum_{\ell=0}^{k-1}\xi\left(\ell\epsilon\right)\boldsymbol{F}_{i}\left(\Theta^{\ell};\boldsymbol{z}^{\ell}\right).
\]
Note that for each $i\in\left[N\right]$, the trajectories $\left(\bar{\boldsymbol{\theta}}_{i}^{t}\right)_{t\geq0}$
and $\left(\boldsymbol{\theta}_{i}^{k}\right)_{k\geq0}$ are coupled
since they share the same initialization $\bar{\boldsymbol{\theta}}_{i}^{0}$.
Let us introduce the notations for the empirical distributions:
\[
\bar{\rho}_{N}^{t}=\frac{1}{N}\sum_{i=1}^{N}\delta_{\bar{\boldsymbol{\theta}}_{i}^{t}},\qquad\rho_{N}^{k}=\frac{1}{N}\sum_{i=1}^{N}\delta_{\boldsymbol{\theta}_{i}^{k}}.
\]
For each $i=1,...,N$, we define
\[
\boldsymbol{\delta}_{i}^{k}=\boldsymbol{\theta}_{i}^{k}-\bar{\boldsymbol{\theta}}_{i}^{k\epsilon},\qquad\boldsymbol{\delta}^{k}=\left(\boldsymbol{\delta}_{1}^{k},...,\boldsymbol{\delta}_{N}^{k}\right)\in\mathbb{R}^{DN}.
\]
We note that $\boldsymbol{\delta}^{0}=\boldsymbol{0}$ since the two
trajectories are coupled by the same initialization. We are interested
in bounding the error quantity:
\[
\mathscr{E}_{k}=\frac{1}{N}\left\Vert \boldsymbol{\delta}^{k}\right\Vert _{2}^{2}=\frac{1}{N}\sum_{i=1}^{N}\left\Vert \boldsymbol{\delta}_{i}^{k}\right\Vert _{2}^{2}.
\]
In the proof, we consider a finite constant terminal time $T>0$.
For some threshold $\gamma_{{\rm st}}\in\left[0,1\right]$, we define
the stopping time 
\begin{equation}
T_{{\rm st}}=\inf\left\{ k\epsilon:\;\mathscr{E}_{k}>\gamma_{{\rm st}}\right\} .\label{eq:MF_proof_stopping_time}
\end{equation}
We also define the following event:
\begin{align*}
\mathsf{Ev} & =\left\{ \frac{1}{N}\sum_{i=1}^{N}\left\Vert \bar{\boldsymbol{\theta}}_{i}^{0}\right\Vert _{2}^{2}\leq C,\quad\frac{1}{N}\sum_{i=1}^{N}\left\Vert \bar{\boldsymbol{\theta}}_{i}^{0}\right\Vert _{2}^{6}\leq C\right\} ,
\end{align*}
for some sufficiently large $C$.

Before we proceed, let us prove a few simple facts:
\begin{itemize}
\item We bound $\mathbb{P}\left\{ \mathsf{Ev}\right\} $. Recall that $\left(\left\Vert \bar{\boldsymbol{\theta}}_{i}^{0}\right\Vert _{2}\right)_{i\leq N}$
are i.i.d. $C$-sub-Gaussian by Assumption \enuref{Assump_ODE}. As
such, by Lemma \ref{lem:concen_subgauss_sum_q_power}, $\mathbb{P}\left\{ \neg\mathsf{Ev}\right\} \lesssim\exp\left(-N^{1/8}\right)$.
\item We bound $\sup_{t\in\left[0,T\right]}\left\Vert \bar{\boldsymbol{\theta}}_{i}^{t}\right\Vert _{2}$
as a deterministic function of $\left\Vert \bar{\boldsymbol{\theta}}_{i}^{0}\right\Vert _{2}$,
for each $i\in\left[N\right]$. Using Assumptions \enuref{Assump_lr}
and \enuref{Assump_growth}, we have:
\begin{align*}
\frac{{\rm d}}{{\rm d}t}\left\Vert \bar{\boldsymbol{\theta}}_{i}^{t}\right\Vert _{2}^{2} & =2\left\langle \bar{\boldsymbol{\theta}}_{i}^{t},\frac{{\rm d}}{{\rm d}t}\bar{\boldsymbol{\theta}}_{i}^{t}\right\rangle =-2\xi\left(t\right)\left\langle \bar{\boldsymbol{\theta}}_{i}^{t},\nabla V\left(\bar{\boldsymbol{\theta}}_{i}^{t}\right)+\nabla_{1}W\left(\bar{\boldsymbol{\theta}}_{i}^{t};\rho^{t}\right)\right\rangle \\
 & \lesssim\left\Vert \bar{\boldsymbol{\theta}}_{i}^{t}\right\Vert _{2}\left(\left\Vert \nabla V\left(\bar{\boldsymbol{\theta}}_{i}^{t}\right)\right\Vert _{2}+\left\Vert \nabla_{1}W\left(\bar{\boldsymbol{\theta}}_{i}^{t};\rho^{t}\right)\right\Vert _{2}\right)\lesssim\left\Vert \bar{\boldsymbol{\theta}}_{i}^{t}\right\Vert _{2}\left(\left\Vert \bar{\boldsymbol{\theta}}_{i}^{t}\right\Vert _{2}+1\right),
\end{align*}
which implies $\frac{{\rm d}}{{\rm d}t}\left\Vert \bar{\boldsymbol{\theta}}_{i}^{t}\right\Vert _{2}\lesssim\left\Vert \bar{\boldsymbol{\theta}}_{i}^{t}\right\Vert _{2}+1$.
By Gronwall's lemma, 
\begin{equation}
\sup_{t\in\left[0,T\right]}\left\Vert \bar{\boldsymbol{\theta}}_{i}^{t}\right\Vert _{2}\lesssim\left\Vert \bar{\boldsymbol{\theta}}_{i}^{0}\right\Vert _{2}+1.\label{eq:propChaos_proof_thetaBar_bound}
\end{equation}
\item We bound $\left\Vert \bar{\boldsymbol{\theta}}_{i}^{t}-\bar{\boldsymbol{\theta}}_{i}^{t'}\right\Vert _{2}$
as a deterministic function of $\bar{\boldsymbol{\theta}}_{i}^{0}$
and $\left|t-t'\right|$, for each $i\in\left[N\right]$ and $t,t'\in\left[0,T\right]$.
Using Assumptions \enuref{Assump_lr} and \enuref{Assump_growth}
as well as Eq. (\ref{eq:propChaos_proof_thetaBar_bound}), we have:
\begin{align}
\left\Vert \bar{\boldsymbol{\theta}}_{i}^{t}-\bar{\boldsymbol{\theta}}_{i}^{t'}\right\Vert _{2} & =\left\Vert \int_{t}^{t'}\xi\left(s\right)\left[\nabla V\left(\bar{\boldsymbol{\theta}}_{i}^{s}\right)+\nabla_{1}W\left(\bar{\boldsymbol{\theta}}_{i}^{s};\rho^{s}\right)\right]{\rm d}s\right\Vert _{2}\nonumber \\
 & \lesssim\int_{t}^{t'}\left\Vert \nabla V\left(\bar{\boldsymbol{\theta}}_{i}^{s}\right)\right\Vert _{2}{\rm d}s+\int_{t}^{t'}\left\Vert \nabla_{1}W\left(\bar{\boldsymbol{\theta}}_{i}^{s};\rho^{s}\right)\right\Vert _{2}{\rm d}s\nonumber \\
 & \lesssim\int_{t}^{t'}\left(\left\Vert \bar{\boldsymbol{\theta}}_{i}^{s}\right\Vert _{2}+1\right){\rm d}s\nonumber \\
 & \lesssim\left(\left\Vert \bar{\boldsymbol{\theta}}_{i}^{0}\right\Vert _{2}+1\right)\left|t'-t\right|.\label{eq:propChaos_proof_thetaBar_tDiff_bound}
\end{align}
\item We also have a bound on $\left\Vert \boldsymbol{\theta}_{i}^{k}\right\Vert _{2}$
for each $i\in\left[N\right]$ and $k\leq T/\epsilon$:
\begin{equation}
\left\Vert \boldsymbol{\theta}_{i}^{k}\right\Vert _{2}\leq\left\Vert \boldsymbol{\delta}_{i}^{k}\right\Vert _{2}+\left\Vert \bar{\boldsymbol{\theta}}_{i}^{k\epsilon}\right\Vert _{2}\lesssim\left\Vert \boldsymbol{\delta}_{i}^{k}\right\Vert _{2}+\left\Vert \bar{\boldsymbol{\theta}}_{i}^{0}\right\Vert _{2}+1,\label{eq:propChaos_proof_theta_bound}
\end{equation}
where the second inequality is by Eq. (\ref{eq:propChaos_proof_thetaBar_bound}).
\end{itemize}
The agenda is as follows. We prove Claims \ref{enu:propChaos_thm_main},
\ref{enu:propChaos_thm_phi}, \ref{enu:propChaos_thm_risk} and \ref{enu:propChaos_thm_sampling}
in Steps 1-4 below. In fact, the latter claims are consequences of
Claim \ref{enu:propChaos_thm_main}. We defer the proofs of several
auxiliary lemmas that are used in the proof of Claim \ref{enu:propChaos_thm_main}
to Section \ref{subsec:propChaos-Auxiliary-Lemmas}.

\subsubsection*{Step 1: Claim \ref{enu:propChaos_thm_main}}

Let ${\cal F}^{k}$ be the sigma-algebra generated by $\left(\bar{\boldsymbol{\theta}}_{i}^{0}\right)_{i\leq N}$
and $\left(\boldsymbol{z}^{\ell}\right)_{\ell\leq k-1}$. Observe
that
\[
\mathbb{E}\left\{ \boldsymbol{F}_{i}\left(\Theta^{k};\boldsymbol{z}^{k}\right)\middle|{\cal F}^{k}\right\} =\boldsymbol{G}\left(\boldsymbol{\theta}_{i}^{k};\rho_{N}^{k}\right).
\]
As such, we have the following decomposition:
\[
\boldsymbol{\delta}_{i}^{k+1}-\boldsymbol{\delta}_{i}^{k}=\int_{k\epsilon}^{\left(k+1\right)\epsilon}\xi\left(s\right)\boldsymbol{G}\left(\bar{\boldsymbol{\theta}}_{i}^{s};\rho^{s}\right){\rm d}s-\epsilon\xi\left(k\epsilon\right)\boldsymbol{F}_{i}\left(\Theta^{k};\boldsymbol{z}^{k}\right)\equiv\epsilon\left(\boldsymbol{E}_{1,i}^{k}+\boldsymbol{E}_{2,i}^{k}-\boldsymbol{E}_{3,i}^{k}+\boldsymbol{E}_{4,i}^{k}\right),
\]
where we define the quantities:
\begin{align*}
\boldsymbol{E}_{1,i}^{k} & =\frac{1}{\epsilon}\int_{k\epsilon}^{\left(k+1\right)\epsilon}\left[\xi\left(s\right)\boldsymbol{G}\left(\bar{\boldsymbol{\theta}}_{i}^{s};\rho^{s}\right)-\xi\left(k\epsilon\right)\boldsymbol{G}\left(\bar{\boldsymbol{\theta}}_{i}^{k\epsilon};\rho^{k\epsilon}\right)\right]{\rm d}s,\\
\boldsymbol{E}_{2,i}^{k} & =\xi\left(k\epsilon\right)\left[\boldsymbol{G}\left(\bar{\boldsymbol{\theta}}_{i}^{k\epsilon};\rho^{k\epsilon}\right)-\boldsymbol{G}\left(\bar{\boldsymbol{\theta}}_{i}^{k\epsilon};\bar{\rho}_{N}^{k\epsilon}\right)\right],\\
\boldsymbol{E}_{3,i}^{k} & =\xi\left(k\epsilon\right)\left[\boldsymbol{G}\left(\boldsymbol{\theta}_{i}^{k};\rho_{N}^{k}\right)-\boldsymbol{G}\left(\bar{\boldsymbol{\theta}}_{i}^{k\epsilon};\bar{\rho}_{N}^{k\epsilon}\right)\right],\\
\boldsymbol{E}_{4,i}^{k} & =\xi\left(k\epsilon\right)\left[\mathbb{E}\left\{ \boldsymbol{F}_{i}\left(\Theta^{k};\boldsymbol{z}^{k}\right)\middle|{\cal F}^{k}\right\} -\boldsymbol{F}_{i}\left(\Theta^{k};\boldsymbol{z}^{k}\right)\right].
\end{align*}
Notice that $\boldsymbol{\delta}^{0}=\boldsymbol{0}$ and that
\begin{align*}
\left\Vert \boldsymbol{\delta}^{k+1}\right\Vert _{2}^{2}-\left\Vert \boldsymbol{\delta}^{k}\right\Vert _{2}^{2} & =2\left\langle \boldsymbol{\delta}^{k},\boldsymbol{\delta}^{k+1}-\boldsymbol{\delta}^{k}\right\rangle +\left\Vert \boldsymbol{\delta}^{k+1}-\boldsymbol{\delta}^{k}\right\Vert _{2}^{2}\\
 & \leq2\epsilon\sum_{i=1}^{N}\left(\left\Vert \boldsymbol{E}_{1,i}^{k}\right\Vert _{2}+\left\Vert \boldsymbol{E}_{2,i}^{k}\right\Vert _{2}\right)\left\Vert \boldsymbol{\delta}_{i}^{k}\right\Vert _{2}+2\epsilon\sum_{i=1}^{N}\left(-\left\langle \boldsymbol{\delta}_{i}^{k},\boldsymbol{E}_{3,i}^{k}\right\rangle +\left\langle \boldsymbol{\delta}_{i}^{k},\boldsymbol{E}_{4,i}^{k}\right\rangle \right)\\
 & \qquad+4\epsilon^{2}\sum_{i=1}^{N}\left(\left\Vert \boldsymbol{E}_{1,i}^{k}\right\Vert _{2}^{2}+\left\Vert \boldsymbol{E}_{2,i}^{k}\right\Vert _{2}^{2}+\left\Vert \boldsymbol{E}_{3,i}^{k}\right\Vert _{2}^{2}+\left\Vert \boldsymbol{E}_{4,i}^{k}\right\Vert _{2}^{2}\right).
\end{align*}
Considering $t\in\mathbb{N}\epsilon\cap\left[0,T\right]$, we thus
have:
\begin{align*}
\mathscr{E}_{t/\epsilon} & \leq\frac{2\epsilon}{N}\sum_{k=0}^{t/\epsilon-1}\sum_{i=1}^{N}\left(\left\Vert \boldsymbol{E}_{1,i}^{k}\right\Vert _{2}+\left\Vert \boldsymbol{E}_{2,i}^{k}\right\Vert _{2}\right)\left\Vert \boldsymbol{\delta}_{i}^{k}\right\Vert _{2}+\frac{2\epsilon}{N}\sum_{k=0}^{t/\epsilon-1}\sum_{i=1}^{N}\left(-\left\langle \boldsymbol{\delta}_{i}^{k},\boldsymbol{E}_{3,i}^{k}\right\rangle +\left\langle \boldsymbol{\delta}_{i}^{k},\boldsymbol{E}_{4,i}^{k}\right\rangle \right)\\
 & \qquad+\frac{4\epsilon^{2}}{N}\sum_{k=0}^{t/\epsilon-1}\sum_{i=1}^{N}\left(\left\Vert \boldsymbol{E}_{1,i}^{k}\right\Vert _{2}^{2}+\left\Vert \boldsymbol{E}_{2,i}^{k}\right\Vert _{2}^{2}+\left\Vert \boldsymbol{E}_{3,i}^{k}\right\Vert _{2}^{2}+\left\Vert \boldsymbol{E}_{4,i}^{k}\right\Vert _{2}^{2}\right).
\end{align*}
Hence we need to bound each of the terms.

We list here upper bounds for the terms, which are proven in the indicated
lemmas:
\begin{align*}
 & \text{[Lemma \ref{lem:propChaos-E1}]} & \frac{\epsilon}{N}\sum_{k=0}^{t/\epsilon-1}\sum_{i=1}^{N}\left\Vert \boldsymbol{E}_{1,i}^{k}\right\Vert _{2}\left\Vert \boldsymbol{\delta}_{i}^{k}\right\Vert _{2} & \lesssim\epsilon^{2}\sum_{k=0}^{t/\epsilon-1}\sqrt{\mathscr{E}_{k}},\\
 & \text{[Lemma \ref{lem:propChaos-E1}]} & \frac{\epsilon^{2}}{N}\sum_{k=0}^{t/\epsilon-1}\sum_{i=1}^{N}\left\Vert \boldsymbol{E}_{1,i}^{k}\right\Vert _{2}^{2} & \lesssim\epsilon^{3},\\
 & \text{[Lemma \ref{lem:propChaos-E2}]} & \frac{\epsilon}{N}\sum_{k=0}^{t/\epsilon-1}\sum_{i=1}^{N}\left\Vert \boldsymbol{E}_{2,i}^{k}\right\Vert _{2}\left\Vert \boldsymbol{\delta}_{i}^{k}\right\Vert _{2} & \lesssim\epsilon\mathfrak{E}_{{\rm [\ref{lem:propChaos-E2}]}}\sum_{k=0}^{t/\epsilon-1}\sqrt{\mathscr{E}_{k}},\\
 & \text{[Lemma \ref{lem:propChaos-E2}]} & \frac{\epsilon^{2}}{N}\sum_{k=0}^{t/\epsilon-1}\sum_{i=1}^{N}\left\Vert \boldsymbol{E}_{2,i}^{k}\right\Vert _{2}^{2} & \lesssim\epsilon\mathfrak{E}_{{\rm [\ref{lem:propChaos-E2}]}}^{2},\\
 & \text{[Lemma \ref{lem:propChaos-E4-prod}]} & \max_{k\leq T/\epsilon}\left|\epsilon\underline{Z}_{{\rm st}}^{k}\right| & \lesssim\sqrt{\epsilon}\kappa_{5}\left(\gamma_{{\rm st}}^{2}+\sqrt{\gamma_{{\rm st}}}\right)\delta_{{\rm [\ref{lem:propChaos-E4-prod}]}},\\
 & \text{[Lemma \ref{lem:propChaos-E4-norm}]} & \frac{\epsilon^{2}}{N}\sum_{k=0}^{t/\epsilon-1}\sum_{i=1}^{N}\left\Vert \boldsymbol{E}_{4,i}^{k}\right\Vert _{2}^{2} & \lesssim\epsilon D^{2}\kappa_{5}^{2}\delta_{{\rm [\ref{lem:propChaos-E4-norm}]}},\\
 & \text{[Lemma \ref{lem:propChaos-E3}]} & -\frac{\epsilon}{N}\sum_{k=0}^{t/\epsilon-1}\sum_{i=1}^{N}\left\langle \boldsymbol{\delta}_{i}^{k},\boldsymbol{E}_{3,i}^{k}\right\rangle  & \lesssim\epsilon\sum_{k=0}^{t/\epsilon-1}\left(\mathscr{E}_{k}+\left(\kappa_{3}+\kappa_{4}\right)\mathscr{E}_{k}^{3/2}\right),\\
 & \text{[Lemma \ref{lem:propChaos-E3}]} & \frac{\epsilon^{2}}{N}\sum_{k=0}^{t/\epsilon-1}\sum_{i=1}^{N}\left\Vert \boldsymbol{E}_{3,i}^{k}\right\Vert _{2}^{2} & \lesssim\epsilon^{2}\kappa_{2}^{2}\sum_{k=0}^{t/\epsilon-1}\mathscr{E}_{k},
\end{align*}
in which we define:
\begin{align*}
\mathfrak{E}_{{\rm [\ref{lem:propChaos-E2}]}} & =\frac{\kappa_{1}}{N}+\left(\delta_{{\rm [\ref{lem:propChaos-E2}]}}+\log^{5/2}\left(\frac{NT}{\epsilon}+1\right)\right)\frac{\kappa_{1}}{\sqrt{N}},\\
\underline{Z}_{{\rm st}}^{k} & =\frac{1}{N}\sum_{\ell=0}^{k\land\left(T_{{\rm st}}/\epsilon\right)-1}\sum_{i=1}^{N}\left\langle \boldsymbol{\delta}_{i}^{\ell},\boldsymbol{E}_{4,i}^{\ell}\right\rangle ,
\end{align*}
for some $\delta_{{\rm [\ref{lem:propChaos-E2}]}},\delta_{{\rm [\ref{lem:propChaos-E4-prod}]}},\delta_{{\rm [\ref{lem:propChaos-E4-norm}]}}>0$.
These bounds collectively hold for all $t\in\mathbb{N}\epsilon\cap\left[0,T\wedge T_{{\rm st}}\right]$,
with probability at least $1-C\exp\left(-\delta_{{\rm [\ref{lem:propChaos-E2}]}}^{2/5}\right)-2\exp\left(-\delta_{{\rm [\ref{lem:propChaos-E4-prod}]}}^{2}\right)-\delta_{{\rm [\ref{lem:propChaos-E4-norm}]}}^{-1}-\Xi\left(N;T,\kappa_{6}\right)$
on the event $\mathsf{Ev}$, provided $\delta_{{\rm [\ref{lem:propChaos-E4-prod}]}}\leq c_{{\rm [\ref{lem:propChaos-E4-prod}]}}/\sqrt{\epsilon}$
for some sufficiently small absolute constant $c_{{\rm [\ref{lem:propChaos-E4-prod}]}}>0$.
The proofs of these lemmas are deferred to Section \ref{subsec:propChaos-Auxiliary-Lemmas}.

Assuming these bounds and recalling the definition of $T_{{\rm st}}$,
we obtain for all $t\in\mathbb{N}\epsilon\cap\left[0,T\wedge T_{{\rm st}}\right]$:
\begin{equation}
\mathscr{E}_{t/\epsilon}\lesssim\mathfrak{E}+\left(\epsilon+\mathfrak{E}_{{\rm [\ref{lem:propChaos-E2}]}}\right)\epsilon\sum_{k=0}^{t/\epsilon-1}\sqrt{\mathscr{E}_{k}}+\underbrace{\left(1+\epsilon\kappa_{2}^{2}+\left(\kappa_{3}+\kappa_{4}\right)\sqrt{\gamma_{{\rm st}}}\right)}_{\text{Gronwall's exponent}}\epsilon\sum_{k=0}^{t/\epsilon-1}\mathscr{E}_{k},\label{eq:propChaos_proof_Gronwall_exp}
\end{equation}
in which 
\[
\mathfrak{E}=\epsilon^{3}+\epsilon\mathfrak{E}_{{\rm [\ref{lem:propChaos-E2}]}}^{2}+\sqrt{\epsilon}\kappa_{5}\sqrt{\gamma_{{\rm st}}}\delta_{{\rm [\ref{lem:propChaos-E4-prod}]}}+\epsilon D^{2}\kappa_{5}^{2}\delta_{{\rm [\ref{lem:propChaos-E4-norm}]}}.
\]
By Gronwall's lemma \cite{dragomir2003some}:
\begin{align*}
\mathscr{E}_{t/\epsilon} & \lesssim\left(\mathfrak{E}+\frac{\epsilon^{2}+\mathfrak{E}_{{\rm [\ref{lem:propChaos-E2}]}}^{2}}{\left(1+\epsilon\kappa_{2}^{2}+\left(\kappa_{3}+\kappa_{4}\right)\sqrt{\gamma_{{\rm st}}}\right)^{2}}\right)\exp\left(C\left(1+\epsilon\kappa_{2}^{2}+\left(\kappa_{3}+\kappa_{4}\right)\sqrt{\gamma_{{\rm st}}}\right)\right)\\
 & \lesssim\left(\mathfrak{E}+\epsilon^{2}+\mathfrak{E}_{{\rm [\ref{lem:propChaos-E2}]}}^{2}\right)\exp\left(C\left(1+\epsilon\kappa_{2}^{2}+\left(\kappa_{3}+\kappa_{4}\right)\sqrt{\gamma_{{\rm st}}}\right)\right).
\end{align*}
It is critical to ensure that Gronwall's exponent component in Eq.
(\ref{eq:propChaos_proof_Gronwall_exp}) is independent of the dimension
vector $\mathfrak{Dim}$ and hence $\kappa_{3}+\kappa_{4}$ and $\kappa_{2}$.
We do so by choosing $N$ and $\epsilon$ such that
\begin{align*}
\epsilon & \leq c/\max\left\{ \delta_{{\rm [\ref{lem:propChaos-E4-prod}]}}^{2}/c_{{\rm [\ref{lem:propChaos-E4-prod}]}}^{2},\;\kappa_{2}^{2},\;\left(\kappa_{3}+\kappa_{4}\right)^{2}\kappa_{5}^{2}\delta_{{\rm [\ref{lem:propChaos-E4-prod}]}}^{2},\;\left(\kappa_{3}+\kappa_{4}\right)^{2}D^{2}\kappa_{5}^{2}\delta_{{\rm [\ref{lem:propChaos-E4-norm}]}}\right\} ,\\
\mathfrak{E}_{{\rm [\ref{lem:propChaos-E2}]}} & \leq c'/\left(\kappa_{3}+\kappa_{4}\right),
\end{align*}
for two absolute constants $c$ and $c'$. With these constraints
and sufficiently small $c$ and $c'$, it is easy to see that with
$\gamma_{{\rm st}}=1/\left(\kappa_{3}+\kappa_{4}\right)^{2}\leq1$,
we have $\mathscr{E}_{t/\epsilon}\leq\gamma_{{\rm st}}$, and hence
$T\leq T_{{\rm st}}$. This, in particular, implies that with probability
at least
\[
1-C\exp\left(-\delta_{{\rm [\ref{lem:propChaos-E2}]}}^{2/5}\right)-2\exp\left(-\delta_{{\rm [\ref{lem:propChaos-E4-prod}]}}^{2}\right)-\delta_{{\rm [\ref{lem:propChaos-E4-norm}]}}^{-1}-\Xi\left(N;T,\kappa_{6}\right)-\exp\left(-N^{1/8}\right),
\]
for all $t\in\mathbb{N}\epsilon\cap\left[0,T\right]$, $\mathscr{E}_{t/\epsilon}\lesssim\mathfrak{E}+\epsilon^{2}+\mathfrak{E}_{{\rm [\ref{lem:propChaos-E2}]}}^{2}.$
By substituting $\delta_{{\rm [\ref{lem:propChaos-E2}]}}=\delta_{{\rm [\ref{lem:propChaos-E4-prod}]}}=\delta$
and $\delta_{{\rm [\ref{lem:propChaos-E4-norm}]}}=\delta^{2}$, Claim
\ref{enu:propChaos_thm_main} of the theorem can be established after
some algebraic manipulations, noticing that $\kappa_{1},...,\kappa_{6}\geq1$,
$D\geq1$ and $\delta>1$.

\subsubsection*{Step 2: Claim \ref{enu:propChaos_thm_phi}}

Claim \ref{enu:propChaos_thm_phi} is a corollary of Claim \ref{enu:propChaos_thm_main}
and is proven in the following.

We have from Claim \ref{enu:propChaos_thm_main} that with probability
at least $1-C\mathsf{prob}\left(N,\delta\right)$, for all $t\in\mathbb{N}\epsilon\cap\left[0,T\right]$,
for any $1$-Lipschitz test function $\phi:\;\mathbb{R}^{d}\to\mathbb{R}$,
\[
\left|\frac{1}{N}\sum_{i=1}^{N}\phi\left(\boldsymbol{\theta}_{i}^{t/\epsilon}\right)-\phi\left(\boldsymbol{\bar{\theta}}_{i}^{t}\right)\right|\lesssim\sqrt{\mathsf{err}\left(N,\epsilon,\delta\right)}.
\]
For a fixed $1$-Lipschitz $\phi$, let us define $X_{2,i}^{t}=\phi\left(\boldsymbol{\bar{\theta}}_{i}^{t}\right)-\int\phi\left(\boldsymbol{\theta}\right)\rho^{t}\left({\rm d}\boldsymbol{\theta}\right)$.
We have for any integer $p\geq1$, since $\left\Vert \boldsymbol{\bar{\theta}}_{i}^{0}\right\Vert _{2}$
is $C$-sub-Gaussian by Assumption \ref{enu:Assump_ODE} and by Eq.
(\ref{eq:propChaos_proof_thetaBar_bound}),
\begin{align*}
\mathbb{E}\left\{ \left|X_{2,i}^{t}\right|^{p}\right\}  & \leq2^{p}\mathbb{E}\left\{ \left|\phi\left(\boldsymbol{\bar{\theta}}_{i}^{t}\right)\right|^{p}\right\} \leq C^{p}\left(\mathbb{E}\left\{ \left\Vert \boldsymbol{\bar{\theta}}_{i}^{t}\right\Vert _{2}^{p}\right\} +1\right)\\
 & \leq C^{p}\left(\mathbb{E}\left\{ \left\Vert \boldsymbol{\bar{\theta}}_{i}^{0}\right\Vert _{2}^{p}\right\} +1\right)\leq C^{p}\left(p^{p/2}+1\right),
\end{align*}
which implies that $X_{2,i}^{t}$ is also $C$-sub-Gaussian. Since
$\left(X_{2,i}^{t}\right)_{i\leq N}$ are i.i.d. with zero mean, by
Lemma \ref{lem:subgauss_subexp_properties} and the union bound,
\[
\mathbb{P}\left\{ \max_{t\in\mathbb{N}\epsilon\cap\left[0,T\right]}\left|\frac{1}{N}\sum_{i=1}^{N}X_{2,i}^{t}\right|\geq\delta_{0}\right\} \lesssim\frac{T}{\epsilon}\exp\left(-CN\delta_{0}^{2}\right).
\]
This shows that with probability at least $1-C\left(\mathsf{prob}\left(N,\delta\right)+\left(T/\epsilon\right)\exp\left(-CN\delta_{0}^{2}\right)\right)$,
\[
\max_{t\in\mathbb{N}\epsilon\cap\left[0,T\right]}\left|\frac{1}{N}\sum_{i=1}^{N}\phi\left(\boldsymbol{\theta}_{i}^{t/\epsilon}\right)-\int\phi\left(\boldsymbol{\theta}\right)\rho^{t}\left({\rm d}\boldsymbol{\theta}\right)\right|\lesssim\delta_{0}+\sqrt{\mathsf{err}\left(N,\epsilon,\delta\right)}.
\]
This proves Claim \ref{enu:propChaos_thm_phi}.

\subsubsection*{Step 3: Claim \ref{enu:propChaos_thm_risk}}

Claim \ref{enu:propChaos_thm_risk} is again a corollary of Claim
\ref{enu:propChaos_thm_main} and is proven in the following.

Let us first consider $\left|{\cal R}\left(\rho_{N}^{t/\epsilon}\right)-{\cal R}\left(\bar{\rho}_{N}^{t}\right)\right|$.
Noticing that $\nabla_{1}U\left(\boldsymbol{\theta},\boldsymbol{\theta}'\right)=\nabla_{2}U\left(\boldsymbol{\theta}',\boldsymbol{\theta}\right)$,
we have from the mean value theorem: 
\begin{align*}
{\cal R}\left(\rho_{N}^{t/\epsilon}\right)-{\cal R}\left(\bar{\rho}_{N}^{t}\right) & =\frac{1}{N}\sum_{i=1}^{N}\left[V\left(\boldsymbol{\theta}_{i}^{t/\epsilon}\right)-V\left(\bar{\boldsymbol{\theta}}_{i}^{t}\right)-\mathbb{E}_{{\cal P}}\left\{ \Lambda\left(\boldsymbol{\theta}_{i}^{t/\epsilon},\boldsymbol{z}\right)-\Lambda\left(\bar{\boldsymbol{\theta}}_{i}^{t},\boldsymbol{z}\right)\right\} \right]\\
 & \qquad+\frac{1}{2N^{2}}\sum_{i,j\leq N}\left[U\left(\boldsymbol{\theta}_{i}^{t/\epsilon},\boldsymbol{\theta}_{j}^{t/\epsilon}\right)-U\left(\bar{\boldsymbol{\theta}}_{i}^{t},\bar{\boldsymbol{\theta}}_{j}^{t}\right)\right]\\
 & =\frac{1}{N}\sum_{i=1}^{N}\left\langle \nabla V\left(\boldsymbol{\zeta}_{1,i}^{t}\right)-\nabla_{1}\mathbb{E}_{{\cal P}}\left\{ \Lambda\left(\boldsymbol{\zeta}_{2,i}^{t},\boldsymbol{z}\right)\right\} ,\boldsymbol{\delta}_{i}^{t/\epsilon}\right\rangle \\
 & \qquad+\frac{1}{2N^{2}}\sum_{i,j\leq N}\left\langle \nabla_{1}U\left(\boldsymbol{\zeta}_{3,ij}^{t},\boldsymbol{\zeta}_{4,ij}^{t}\right),\boldsymbol{\delta}_{i}^{t/\epsilon}\right\rangle +\left\langle \nabla_{1}U\left(\boldsymbol{\zeta}_{4,ij}^{t},\boldsymbol{\zeta}_{3,ij}^{t}\right),\boldsymbol{\delta}_{j}^{t/\epsilon}\right\rangle ,
\end{align*}
for some $\boldsymbol{\zeta}_{1,i}^{t},\boldsymbol{\zeta}_{2,i}^{t},\boldsymbol{\zeta}_{3,ij}^{t}\in\left[\bar{\boldsymbol{\theta}}_{i}^{t},\boldsymbol{\theta}_{i}^{t/\epsilon}\right]$
and $\boldsymbol{\zeta}_{4,ij}^{t}\in\left[\bar{\boldsymbol{\theta}}_{j}^{t},\boldsymbol{\theta}_{j}^{t/\epsilon}\right]$.
Note that by Eq. (\ref{eq:propChaos_proof_thetaBar_bound}),
\begin{align*}
\left\Vert \boldsymbol{\zeta}_{r,i}^{t}\right\Vert _{2} & \leq\left\Vert \boldsymbol{\delta}_{i}^{t/\epsilon}\right\Vert _{2}+\left\Vert \bar{\boldsymbol{\theta}}_{i}^{t}\right\Vert _{2}\lesssim\left\Vert \boldsymbol{\delta}_{i}^{t/\epsilon}\right\Vert _{2}+\left\Vert \bar{\boldsymbol{\theta}}_{i}^{0}\right\Vert _{2}+1,\qquad r=1,2,\\
\left\Vert \boldsymbol{\zeta}_{3,ij}^{t}\right\Vert _{2} & \leq\left\Vert \boldsymbol{\delta}_{i}^{t/\epsilon}\right\Vert _{2}+\left\Vert \bar{\boldsymbol{\theta}}_{i}^{t}\right\Vert _{2}\lesssim\left\Vert \boldsymbol{\delta}_{i}^{t/\epsilon}\right\Vert _{2}+\left\Vert \bar{\boldsymbol{\theta}}_{i}^{0}\right\Vert _{2}+1,\\
\left\Vert \boldsymbol{\zeta}_{4,ij}^{t}\right\Vert _{2} & \leq\left\Vert \boldsymbol{\delta}_{j}^{t/\epsilon}\right\Vert _{2}+\left\Vert \bar{\boldsymbol{\theta}}_{j}^{t}\right\Vert _{2}\lesssim\left\Vert \boldsymbol{\delta}_{j}^{t/\epsilon}\right\Vert _{2}+\left\Vert \bar{\boldsymbol{\theta}}_{j}^{0}\right\Vert _{2}+1.
\end{align*}
Then by Assumptions \ref{enu:Assump_growth} and \ref{enu:Assump_additional_1},
under the event $\mathsf{Ev}$,
\begin{align}
\left|{\cal R}\left(\rho_{N}^{t/\epsilon}\right)-{\cal R}\left(\bar{\rho}_{N}^{t}\right)\right| & \lesssim\frac{1}{N}\sum_{i=1}^{N}\left(\left\Vert \nabla V\left(\boldsymbol{\zeta}_{1,i}^{t}\right)\right\Vert _{2}+\left\Vert \nabla_{1}\mathbb{E}_{{\cal P}}\left\{ \Lambda\left(\boldsymbol{\zeta}_{2,i}^{t},\boldsymbol{z}\right)\right\} \right\Vert _{2}\right)\left\Vert \boldsymbol{\delta}_{i}^{t/\epsilon}\right\Vert _{2}\nonumber \\
 & \qquad+\frac{1}{N^{2}}\sum_{i,j\leq N}\left\Vert \nabla_{1}U\left(\boldsymbol{\zeta}_{3,ij}^{t},\boldsymbol{\zeta}_{4,ij}^{t}\right)\right\Vert _{2}\left\Vert \boldsymbol{\delta}_{i}^{t/\epsilon}\right\Vert _{2}+\left\Vert \nabla_{1}U\left(\boldsymbol{\zeta}_{4,ij}^{t},\boldsymbol{\zeta}_{3,ij}^{t}\right)\right\Vert _{2}\left\Vert \boldsymbol{\delta}_{j}^{t/\epsilon}\right\Vert _{2}\nonumber \\
 & \lesssim\frac{1}{N}\sum_{i=1}^{N}\left(\left\Vert \boldsymbol{\zeta}_{1,i}^{t}\right\Vert _{2}+\left\Vert \boldsymbol{\zeta}_{2,i}^{t}\right\Vert _{2}+1\right)\left\Vert \boldsymbol{\delta}_{i}^{t/\epsilon}\right\Vert _{2}\nonumber \\
 & \qquad+\frac{\kappa_{1}}{N^{2}}\sum_{i,j\leq N}\left(\left\Vert \boldsymbol{\zeta}_{3,ij}^{t}\right\Vert _{2}+1\right)\left(\left\Vert \boldsymbol{\zeta}_{4,ij}^{t}\right\Vert _{2}^{2}+1\right)\left\Vert \boldsymbol{\delta}_{i}^{t/\epsilon}\right\Vert _{2}\nonumber \\
 & \qquad+\frac{\kappa_{1}}{N^{2}}\sum_{i,j\leq N}\left(\left\Vert \boldsymbol{\zeta}_{4,ij}^{t}\right\Vert _{2}+1\right)\left(\left\Vert \boldsymbol{\zeta}_{3,ij}^{t}\right\Vert _{2}^{2}+1\right)\left\Vert \boldsymbol{\delta}_{j}^{t/\epsilon}\right\Vert _{2}\nonumber \\
 & \lesssim\frac{1}{N}\sum_{i=1}^{N}\left(\left\Vert \boldsymbol{\delta}_{i}^{t/\epsilon}\right\Vert _{2}+\left\Vert \bar{\boldsymbol{\theta}}_{i}^{0}\right\Vert _{2}+1\right)\left\Vert \boldsymbol{\delta}_{i}^{t/\epsilon}\right\Vert _{2}\nonumber \\
 & \qquad+\frac{\kappa_{1}}{N^{2}}\sum_{i,j\leq N}\left(\left\Vert \boldsymbol{\delta}_{i}^{t/\epsilon}\right\Vert _{2}+\left\Vert \bar{\boldsymbol{\theta}}_{i}^{0}\right\Vert _{2}+1\right)\left(\left\Vert \boldsymbol{\delta}_{j}^{t/\epsilon}\right\Vert _{2}^{2}+\left\Vert \bar{\boldsymbol{\theta}}_{j}^{0}\right\Vert _{2}^{2}+1\right)\left\Vert \boldsymbol{\delta}_{i}^{t/\epsilon}\right\Vert _{2}\nonumber \\
 & \qquad+\frac{\kappa_{1}}{N^{2}}\sum_{i,j\leq N}\left(\left\Vert \boldsymbol{\delta}_{j}^{t/\epsilon}\right\Vert _{2}+\left\Vert \bar{\boldsymbol{\theta}}_{j}^{0}\right\Vert _{2}+1\right)\left(\left\Vert \boldsymbol{\delta}_{i}^{t/\epsilon}\right\Vert _{2}^{2}+\left\Vert \bar{\boldsymbol{\theta}}_{i}^{0}\right\Vert _{2}^{2}+1\right)\left\Vert \boldsymbol{\delta}_{j}^{t/\epsilon}\right\Vert _{2}\nonumber \\
 & \lesssim\mathscr{E}_{t/\epsilon}+\sqrt{\frac{1}{N}\sum_{i=1}^{N}\left\Vert \bar{\boldsymbol{\theta}}_{i}^{0}\right\Vert _{2}^{2}\mathscr{E}_{t/\epsilon}}+\sqrt{\mathscr{E}_{t/\epsilon}}\nonumber \\
 & \qquad+\kappa_{1}\left(\mathscr{E}_{t/\epsilon}+\sqrt{\frac{1}{N}\sum_{i=1}^{N}\left\Vert \bar{\boldsymbol{\theta}}_{i}^{0}\right\Vert _{2}^{2}\mathscr{E}_{t/\epsilon}}+\sqrt{\mathscr{E}_{t/\epsilon}}\right)\left(\mathscr{E}_{t/\epsilon}+\frac{1}{N}\sum_{i=1}^{N}\left\Vert \bar{\boldsymbol{\theta}}_{i}^{0}\right\Vert _{2}^{2}+1\right)\nonumber \\
 & \lesssim\mathscr{E}_{t/\epsilon}+\sqrt{\mathscr{E}_{t/\epsilon}}+\kappa_{1}\left(\mathscr{E}_{t/\epsilon}+\sqrt{\mathscr{E}_{t/\epsilon}}\right)\left(\mathscr{E}_{t/\epsilon}+1\right)\label{eq:propChaos_proof_s1p3}\\
 & \lesssim\kappa_{1}\sqrt{\mathscr{E}_{t/\epsilon}},\nonumber 
\end{align}
where in the last step, we use the fact that $\mathscr{E}_{t/\epsilon}\leq\gamma_{{\rm st}}\leq1$
for all $t\in\mathbb{N}\epsilon\cap\left[0,T\right]$, with probability
at least $1-C\mathsf{prob}\left(N,\delta\right)$.

Next, we consider $\left|{\cal R}\left(\bar{\rho}_{N}^{t}\right)-{\cal R}\left(\rho^{t}\right)\right|$,
for $t\in\left[0,T\right]$:
\begin{align*}
\left|{\cal R}\left(\bar{\rho}_{N}^{t}\right)-{\cal R}\left(\rho^{t}\right)\right| & \lesssim\left|\frac{1}{N}\sum_{i=1}^{N}\left[V\left(\bar{\boldsymbol{\theta}}_{i}^{t}\right)-\int V\left(\boldsymbol{\theta}\right)\rho^{t}\left({\rm d}\boldsymbol{\theta}\right)\right]\right|\\
 & \qquad+\left|\frac{1}{N}\sum_{i=1}^{N}\mathbb{E}_{{\cal P}}\left\{ \Lambda\left(\bar{\boldsymbol{\theta}}_{i}^{t},\boldsymbol{z}\right)\right\} -\int\mathbb{E}_{{\cal P}}\left\{ \Lambda\left(\boldsymbol{\theta},\boldsymbol{z}\right)\right\} \rho^{t}\left({\rm d}\boldsymbol{\theta}\right)\right|\\
 & \qquad+\left|\frac{1}{N^{2}}\sum_{i=1}^{N}\sum_{j\neq i}\left[U\left(\bar{\boldsymbol{\theta}}_{i}^{t},\bar{\boldsymbol{\theta}}_{j}^{t}\right)-\int U\left(\bar{\boldsymbol{\theta}}_{i}^{t},\boldsymbol{\theta}\right)\rho^{t}\left({\rm d}\boldsymbol{\theta}\right)\right]\right|\\
 & \qquad+\left|\frac{1}{N}\sum_{i=1}^{N}\left[W\left(\bar{\boldsymbol{\theta}}_{i}^{t};\rho^{t}\right)-\int W\left(\boldsymbol{\theta};\rho^{t}\right)\rho^{t}\left({\rm d}\boldsymbol{\theta}\right)\right]\right|\\
 & \qquad+\left|\frac{1}{N^{2}}\sum_{i=1}^{N}\left[U\left(\bar{\boldsymbol{\theta}}_{i}^{t},\bar{\boldsymbol{\theta}}_{i}^{t}\right)+W\left(\bar{\boldsymbol{\theta}}_{i}^{t};\rho^{t}\right)\right]\right|\\
 & \equiv A_{3,1}^{t}+A_{3,2}^{t}+A_{3,3}^{t}+A_{3,4}^{t}+A_{3,5}^{t}.
\end{align*}
Let us bound $A_{3,1}^{t}$. Denote $X_{3,i}^{t}=V\left(\bar{\boldsymbol{\theta}}_{i}^{t}\right)-\int V\left(\boldsymbol{\theta}\right)\rho^{t}\left({\rm d}\boldsymbol{\theta}\right)$.
We have from Assumptions \ref{enu:Assump_growth}, \ref{enu:Assump_ODE}
and \ref{enu:Assump_additional_1} and Eq. (\ref{eq:propChaos_proof_thetaBar_bound})
that, for any positive integer $p$,
\begin{align*}
\mathbb{E}\left\{ \left|X_{3,i}^{t}\right|^{p}\right\}  & \leq2^{p}\mathbb{E}\left\{ \left|V\left(\boldsymbol{\bar{\theta}}_{i}^{t}\right)\right|^{p}\right\} \leq C^{p}\mathbb{E}\left\{ \left\Vert \nabla V\left(\boldsymbol{\zeta}_{i}^{t}\right)\right\Vert _{2}^{p}\left\Vert \boldsymbol{\bar{\theta}}_{i}^{t}\right\Vert _{2}^{p}+\left|V\left(\boldsymbol{0}\right)\right|^{p}\right\} ,\\
 & \leq C^{p}\mathbb{E}\left\{ \left(\left\Vert \boldsymbol{\zeta}_{i}^{t}\right\Vert _{2}^{p}+1\right)\left\Vert \boldsymbol{\bar{\theta}}_{i}^{t}\right\Vert _{2}^{p}+\left|V\left(\boldsymbol{0}\right)\right|^{p}\right\} \leq C^{p}\mathbb{E}\left\{ \left(\left\Vert \bar{\boldsymbol{\theta}}_{i}^{t}\right\Vert _{2}^{p}+1\right)\left\Vert \boldsymbol{\bar{\theta}}_{i}^{t}\right\Vert _{2}^{p}+\left|V\left(\boldsymbol{0}\right)\right|^{p}\right\} \\
 & \leq C^{p}\left(\mathbb{E}\left\{ \left\Vert \boldsymbol{\bar{\theta}}_{i}^{0}\right\Vert _{2}^{2p}\right\} +1\right)\leq C^{p}\left(p^{p}+1\right),
\end{align*}
for some $\boldsymbol{\zeta}_{i}^{t}\in\left[\boldsymbol{0},\bar{\boldsymbol{\theta}}_{i}^{t}\right]$,
where we have applied the mean value theorem and we note $\left\Vert \boldsymbol{\zeta}_{i}^{t}\right\Vert _{2}\leq\left\Vert \bar{\boldsymbol{\theta}}_{i}^{t}\right\Vert _{2}$.
This implies that $X_{3,i}^{t}$ is $C$-sub-exponential. Since $\left(X_{3,i}^{t}\right)_{i\leq N}$
are i.i.d. with zero mean, by Lemma \ref{lem:subgauss_subexp_properties}
and the union bound, for $\delta\in\left(0,1\right)$,
\[
\mathbb{P}\left\{ \max_{t\in\mathbb{N}\epsilon\cap\left[0,T\right]}A_{3,1}^{t}\geq\delta\right\} \lesssim\left(T/\epsilon\right)\cdot\exp\left(-CN\delta^{2}\right).
\]
One has similar results for $A_{3,2}^{t}$ and $A_{3,4}^{t}$ by using
Assumptions \ref{enu:Assump_growth}, \ref{enu:Assump_ODE} and \ref{enu:Assump_additional_1}
and Eq. (\ref{eq:propChaos_proof_thetaBar_bound}), for $\delta\in\left(0,1\right)$:
\begin{align*}
\mathbb{P}\left\{ \max_{t\in\mathbb{N}\epsilon\cap\left[0,T\right]}A_{3,2}^{t}\geq\delta\right\}  & \lesssim\left(T/\epsilon\right)\cdot\exp\left(-CN\delta^{2}\right),\\
\mathbb{P}\left\{ \max_{t\in\mathbb{N}\epsilon\cap\left[0,T\right]}A_{3,4}^{t}\geq\delta\right\}  & \lesssim\left(T/\epsilon\right)\cdot\exp\left(-CN\delta^{2}\right).
\end{align*}
Let us bound $A_{3,3}^{t}$. Denote $Y_{3,ij}^{t}=U\left(\bar{\boldsymbol{\theta}}_{i}^{t},\bar{\boldsymbol{\theta}}_{j}^{t}\right)-\int U\left(\bar{\boldsymbol{\theta}}_{i}^{t},\boldsymbol{\theta}\right)\rho^{t}\left({\rm d}\boldsymbol{\theta}\right)$,
and consider $j\neq i$ for a fixed $i\in\left[N\right]$. Recalling
Assumptions \ref{enu:Assump_growth}, \ref{enu:Assump_ODE} and \ref{enu:Assump_additional_1}
and Eq. (\ref{eq:propChaos_proof_thetaBar_bound}), that $\left(\bar{\boldsymbol{\theta}}_{i}^{t}\right)_{i\leq N}$
are i.i.d. and that $\nabla_{1}U\left(\boldsymbol{\theta},\boldsymbol{\theta}'\right)=\nabla_{2}U\left(\boldsymbol{\theta}',\boldsymbol{\theta}\right)$,
we have for any positive integer $p$,
\begin{align*}
 & \mathbb{E}\left\{ \left|Y_{3,ij}^{t}\right|^{2p}\right\} \leq4^{p}\mathbb{E}\left\{ \left|U\left(\bar{\boldsymbol{\theta}}_{i}^{t},\bar{\boldsymbol{\theta}}_{j}^{t}\right)\right|^{2p}\right\} \\
 & \quad\leq4^{p}\mathbb{E}\left\{ \left\Vert \nabla_{1}U\left(\boldsymbol{\zeta}_{1,ij}^{t},\boldsymbol{\zeta}_{2,ij}^{t}\right)\right\Vert _{2}^{2p}\left\Vert \bar{\boldsymbol{\theta}}_{i}^{t}\right\Vert _{2}^{2p}+\left\Vert \nabla_{1}U\left(\boldsymbol{\zeta}_{2,ij}^{t},\boldsymbol{\zeta}_{1,ij}^{t}\right)\right\Vert _{2}^{2p}\left\Vert \bar{\boldsymbol{\theta}}_{j}^{t}\right\Vert _{2}^{2p}+\left|U\left(\boldsymbol{0},\boldsymbol{0}\right)\right|^{2p}\right\} \\
 & \quad\leq4^{p}\mathbb{E}\left\{ \kappa_{1}^{2p}\left(\left\Vert \boldsymbol{\zeta}_{1,ij}^{t}\right\Vert _{2}^{2p}+1\right)\left(\left\Vert \boldsymbol{\zeta}_{2,ij}^{t}\right\Vert _{2}^{4p}+1\right)\left\Vert \bar{\boldsymbol{\theta}}_{i}^{t}\right\Vert _{2}^{2p}+\kappa_{1}^{2p}\left(\left\Vert \boldsymbol{\zeta}_{2,ij}^{t}\right\Vert _{2}^{2p}+1\right)\left(\left\Vert \boldsymbol{\zeta}_{1,ij}^{t}\right\Vert _{2}^{4p}+1\right)\left\Vert \bar{\boldsymbol{\theta}}_{j}^{t}\right\Vert _{2}^{2p}+1\right\} \\
 & \quad\leq4^{p}\mathbb{E}\left\{ \kappa_{1}^{2p}\left(\left\Vert \bar{\boldsymbol{\theta}}_{i}^{t}\right\Vert _{2}^{2p}+1\right)\left(\left\Vert \bar{\boldsymbol{\theta}}_{j}^{t}\right\Vert _{2}^{4p}+1\right)\left\Vert \bar{\boldsymbol{\theta}}_{i}^{t}\right\Vert _{2}^{2p}+\kappa_{1}^{2p}\left(\left\Vert \bar{\boldsymbol{\theta}}_{j}^{t}\right\Vert _{2}^{2p}+1\right)\left(\left\Vert \bar{\boldsymbol{\theta}}_{i}^{t}\right\Vert _{2}^{4p}+1\right)\left\Vert \bar{\boldsymbol{\theta}}_{j}^{t}\right\Vert _{2}^{2p}+1\right\} \\
 & \quad\leq C^{p}\mathbb{E}\left\{ \kappa_{1}^{2p}\left(\left\Vert \bar{\boldsymbol{\theta}}_{i}^{0}\right\Vert _{2}^{4p}+1\right)\left(\left\Vert \bar{\boldsymbol{\theta}}_{j}^{0}\right\Vert _{2}^{4p}+1\right)+\kappa_{1}^{2p}\left(\left\Vert \bar{\boldsymbol{\theta}}_{j}^{0}\right\Vert _{2}^{4p}+1\right)\left(\left\Vert \bar{\boldsymbol{\theta}}_{i}^{0}\right\Vert _{2}^{4p}+1\right)+1\right\} \\
 & \quad\leq C^{p}\left(\kappa_{1}^{2p}\left(p^{2p}+1\right)^{2}+1\right)\\
 & \quad\leq C^{p}\kappa_{1}^{2p}p^{4p},
\end{align*}
for some $\boldsymbol{\zeta}_{1,ij}^{t}\in\left[\boldsymbol{0},\bar{\boldsymbol{\theta}}_{i}^{t}\right]$
and $\boldsymbol{\zeta}_{2,ij}^{t}\in\left[\boldsymbol{0},\bar{\boldsymbol{\theta}}_{j}^{t}\right]$,
where we have used the mean value theorem and we note $\left\Vert \boldsymbol{\zeta}_{1,ij}^{t}\right\Vert _{2}\leq\left\Vert \bar{\boldsymbol{\theta}}_{i}^{t}\right\Vert _{2}$,
$\left\Vert \boldsymbol{\zeta}_{2,ij}^{t}\right\Vert _{2}\leq\left\Vert \bar{\boldsymbol{\theta}}_{j}^{t}\right\Vert _{2}$.
Note that for a fixed $i$, $\left(Y_{3,ij}^{t}\right)_{j\neq i,\;j\leq N}$
are independent with zero mean, conditional on $\bar{\boldsymbol{\theta}}_{i}^{t}$.
By Lemma \ref{lem:bound_moment_sym}, for a fixed $i$,
\[
\mathbb{E}\left\{ \left|\frac{1}{N}\sum_{j\neq i,\;j\leq N}Y_{3,ij}^{t}\right|^{2p}\right\} \leq C^{p}\kappa_{1}^{2p}p^{6p}/N^{p}.
\]
This implies that $\left|\left(1/N\right)\cdot\sum_{j\neq i,\;j\leq N}Y_{3,ij}^{t}\right|^{1/3}$
is $\left(C\kappa_{1}^{1/3}N^{-1/6}\right)$-sub-exponential, and
therefore, by Lemma \ref{lem:subgauss_subexp_properties},
\[
\mathbb{P}\left\{ \left|\frac{1}{N}\sum_{j\neq i,\;j\leq N}Y_{3,ij}^{t}\right|\geq\delta\right\} \lesssim\exp\left(-C\delta^{1/3}\left(\frac{N}{\kappa_{1}^{2}}\right)^{1/6}\right).
\]
By the union bound,
\[
\mathbb{P}\left\{ \max_{t\in\mathbb{N}\epsilon\cap\left[0,T\right]}A_{3,3}^{t}\geq\delta\right\} \leq\mathbb{P}\left\{ \max_{t\in\mathbb{N}\epsilon\cap\left[0,T\right]}\max_{i\in\left[N\right]}\left|\frac{1}{N}\sum_{j\neq i}Y_{3,ij}^{t}\right|\geq\delta\right\} \lesssim\frac{NT}{\epsilon}\exp\left(-C\delta^{1/3}\left(\frac{N}{\kappa_{1}^{2}}\right)^{1/6}\right).
\]
Let us now turn to $A_{3,5}^{t}$. We have from Assumptions \ref{enu:Assump_growth},
\ref{enu:Assump_additional_1} and Eq. (\ref{eq:propChaos_proof_thetaBar_bound}),
and again the mean value theorem, on the event $\mathsf{Ev}$,
\begin{align*}
A_{3,5}^{t} & \lesssim\frac{1}{N^{2}}\sum_{i=1}^{N}\bigg[\left(\left\Vert \nabla_{1}U\left(\boldsymbol{\zeta}_{1,i}^{t},\boldsymbol{\zeta}_{2,i}^{t}\right)\right\Vert _{2}+\left\Vert \nabla_{1}U\left(\boldsymbol{\zeta}_{2,i}^{t},\boldsymbol{\zeta}_{1,i}^{t}\right)\right\Vert _{2}\right)\left\Vert \bar{\boldsymbol{\theta}}_{i}^{t}\right\Vert _{2}\\
 & \qquad\qquad+\left\Vert \nabla_{1}W\left(\boldsymbol{\zeta}_{i}^{t};\rho^{t}\right)\right\Vert _{2}\left\Vert \bar{\boldsymbol{\theta}}_{i}^{t}\right\Vert _{2}+\left|U\left(\boldsymbol{0},\boldsymbol{0}\right)\right|+W\left(\boldsymbol{0};\rho^{t}\right)\bigg]\\
 & \lesssim\frac{1}{N^{2}}\sum_{i=1}^{N}\bigg[\kappa_{1}\left(\left(\left\Vert \boldsymbol{\zeta}_{1,i}^{t}\right\Vert _{2}+1\right)\left(\left\Vert \boldsymbol{\zeta}_{2,i}^{t}\right\Vert _{2}^{2}+1\right)+\left(\left\Vert \boldsymbol{\zeta}_{2,i}^{t}\right\Vert _{2}+1\right)\left(\left\Vert \boldsymbol{\zeta}_{1,i}^{t}\right\Vert _{2}^{2}+1\right)\right)\left\Vert \bar{\boldsymbol{\theta}}_{i}^{t}\right\Vert _{2}\\
 & \qquad\qquad+\left(\left\Vert \boldsymbol{\zeta}_{i}^{t}\right\Vert _{2}+1\right)\left\Vert \bar{\boldsymbol{\theta}}_{i}^{t}\right\Vert _{2}+\left|U\left(\boldsymbol{0},\boldsymbol{0}\right)\right|+W\left(\boldsymbol{0};\rho^{t}\right)\bigg]\\
 & \lesssim\frac{1}{N^{2}}\sum_{i=1}^{N}\kappa_{1}\left(\left\Vert \bar{\boldsymbol{\theta}}_{i}^{t}\right\Vert _{2}^{3}+1\right)\left\Vert \bar{\boldsymbol{\theta}}_{i}^{t}\right\Vert _{2}+\frac{1}{N}\lesssim\frac{1}{N^{2}}\sum_{i=1}^{N}\kappa_{1}\left(\left\Vert \bar{\boldsymbol{\theta}}_{i}^{t}\right\Vert _{2}^{6}+1\right)+\frac{1}{N}\\
 & \lesssim\frac{\kappa_{1}}{N}\left(\frac{1}{N}\sum_{i=1}^{N}\left\Vert \bar{\boldsymbol{\theta}}_{i}^{0}\right\Vert _{2}^{6}+1\right)+\frac{1}{N}\lesssim\frac{\kappa_{1}}{N},
\end{align*}
for some $\boldsymbol{\zeta}_{1,i}^{t},\boldsymbol{\zeta}_{2,i}^{t},\boldsymbol{\zeta}_{i}^{t}\in\left[\boldsymbol{0},\bar{\boldsymbol{\theta}}_{i}^{t}\right]$,
where we note $\left\Vert \boldsymbol{\zeta}_{1,i}^{t}\right\Vert _{2},\left\Vert \boldsymbol{\zeta}_{2,i}^{t}\right\Vert _{2},\left\Vert \boldsymbol{\zeta}_{i}^{t}\right\Vert _{2}\leq\left\Vert \bar{\boldsymbol{\theta}}_{i}^{t}\right\Vert _{2}$.
Combining the bounds, we thus obtain for any $\delta\in\left(0,1\right)$,
\[
\mathbb{P}\left\{ \left\{ \max_{t\in\mathbb{N}\epsilon\cap\left[0,T\right]}\left|{\cal R}\left(\bar{\rho}_{N}^{t}\right)-{\cal R}\left(\rho^{t}\right)\right|\gtrsim\delta+\frac{\kappa_{1}}{N}\right\} \cap\mathsf{Ev}\right\} \lesssim\frac{NT}{\epsilon}\exp\left(-C\delta^{1/3}\left(\frac{N}{\kappa_{1}^{2}}\right)^{1/6}\right).
\]

Finally with the bounds on $\left|{\cal R}\left(\rho_{N}^{t/\epsilon}\right)-{\cal R}\left(\bar{\rho}_{N}^{t}\right)\right|$
and $\left|{\cal R}\left(\bar{\rho}_{N}^{t}\right)-{\cal R}\left(\rho^{t}\right)\right|$,
along with Claim \ref{enu:propChaos_thm_main}, we have:
\[
\max_{t\in\mathbb{N}\epsilon\cap\left[0,T\right]}\left|{\cal R}\left(\rho_{N}^{t/\epsilon}\right)-{\cal R}\left(\rho^{t}\right)\right|\lesssim\kappa_{1}\sqrt{\mathsf{err}\left(N,\epsilon,\delta\right)}+\epsilon_{0}+\frac{\kappa_{1}}{N},
\]
with probability at least 
\[
1-C\mathsf{prob}\left(N,\delta\right)-C\frac{NT}{\epsilon}\exp\left(-C\epsilon_{0}^{1/3}\left(\frac{N}{\kappa_{1}^{2}}\right)^{1/6}\right),
\]
for any $\epsilon_{0}\in\left(0,1\right)$. This completes the proof
of Claim \ref{enu:propChaos_thm_risk}.

\subsubsection*{Step 4: Claim \ref{enu:propChaos_thm_sampling}}

Let ${\cal G}$ denote the sigma-algebra generated by everything but
the random indices $\left(h\left(i\right)\right)_{i\leq M}$. Consider
$t\in\mathbb{N}\epsilon\cap\left[0,T\right]$. We have $\mathbb{E}\left\{ \left\Vert \boldsymbol{\delta}_{h\left(i\right)}^{t/\epsilon}\right\Vert _{2}^{2}\middle|{\cal G}\right\} =\mathscr{E}_{t/\epsilon}$.
Therefore, for any $\delta_{0}>0$,
\[
\mathbb{P}\left\{ \frac{1}{M}\sum_{i=1}^{M}\left\Vert \boldsymbol{\delta}_{h\left(i\right)}^{t/\epsilon}\right\Vert _{2}^{2}\geq\delta_{0}\mathscr{E}_{t/\epsilon}\middle|{\cal G}\right\} \leq\frac{1}{\delta_{0}\mathscr{E}_{t/\epsilon}}\mathbb{E}\left\{ \frac{1}{M}\sum_{i=1}^{M}\left\Vert \boldsymbol{\delta}_{h\left(i\right)}^{t/\epsilon}\right\Vert _{2}^{2}\middle|{\cal G}\right\} =\frac{1}{\delta_{0}}.
\]
We also have, by Assumption \ref{enu:Assump_ODE}, for any positive
integer $p$,
\[
\mathbb{E}\left\{ \left\Vert \bar{\boldsymbol{\theta}}_{h\left(i\right)}^{0}\right\Vert _{2}^{2p}\right\} =\mathbb{E}\left\{ \frac{1}{N}\sum_{j=1}^{N}\left\Vert \bar{\boldsymbol{\theta}}_{j}^{0}\right\Vert _{2}^{2p}\right\} =C^{p}p^{p},
\]
which means $\left\Vert \bar{\boldsymbol{\theta}}_{h\left(i\right)}^{0}\right\Vert _{2}^{2}-\left(1/N\right)\cdot\sum_{j=1}^{N}\left\Vert \bar{\boldsymbol{\theta}}_{j}^{0}\right\Vert _{2}^{2}$
is a zero-mean $C$-sub-exponential random variable. Therefore, by
Lemma \ref{lem:subgauss_subexp_properties},
\[
\mathbb{P}\left\{ \left|\frac{1}{M}\sum_{i=1}^{M}\left\Vert \bar{\boldsymbol{\theta}}_{h\left(i\right)}^{0}\right\Vert _{2}^{2}-\frac{1}{N}\sum_{i=1}^{N}\left\Vert \bar{\boldsymbol{\theta}}_{i}^{0}\right\Vert _{2}^{2}\right|\geq C\right\} \lesssim e^{-M}.
\]
Then proceeding similarly to the steps leading up to Eq. (\ref{eq:propChaos_proof_s1p3})
(proof of Claim \ref{enu:propChaos_thm_risk}), we obtain:
\[
\left|{\cal R}\left(\nu_{M}^{t/\epsilon}\right)-{\cal R}\left(\bar{\nu}_{M}^{t}\right)\right|\lesssim\kappa_{1}\left(\delta_{0}^{3/2}+1\right)\sqrt{\delta_{0}\mathscr{E}_{t/\epsilon}}\lesssim\kappa_{1}\left(\delta_{0}^{3/2}+1\right)\sqrt{\delta_{0}\mathsf{err}\left(N,\epsilon,\delta\right)},
\]
with probability at least $1-C\mathsf{prob}\left(N,\delta\right)-\delta_{0}^{-1}-Ce^{-M}$.

\subsection{Proofs of auxiliary lemmas\label{subsec:propChaos-Auxiliary-Lemmas}}

We state and prove the auxiliary lemmas that are used in the proof
of Claim \ref{enu:propChaos_thm_main} of Theorem \ref{thm:propChaos}
in Section \ref{subsec:propChaos-main-proof}. We reuse the notations
and setups that are introduced in that proof.
\begin{lem}[Control of $\boldsymbol{E}_{1,i}^{k}$]
\label{lem:propChaos-E1}Consider the same setting as Theorem \ref{thm:propChaos}.
We have:
\begin{align*}
\frac{\epsilon}{N}\sum_{k=0}^{t/\epsilon-1}\sum_{i=1}^{N}\left\Vert \boldsymbol{E}_{1,i}^{k}\right\Vert _{2}\left\Vert \boldsymbol{\delta}_{i}^{k}\right\Vert _{2} & \lesssim\epsilon^{2}\sum_{k=0}^{t/\epsilon-1}\sqrt{\mathscr{E}_{k}},\\
\frac{\epsilon^{2}}{N}\sum_{k=0}^{t/\epsilon-1}\sum_{i=1}^{N}\left\Vert \boldsymbol{E}_{1,i}^{k}\right\Vert _{2}^{2} & \lesssim\epsilon^{3},
\end{align*}
for all $t\in\mathbb{N}\epsilon\cap\left[0,T\right]$, under the event
$\mathsf{Ev}$.
\end{lem}

\begin{proof}
All of the following bounds use Assumptions \enuref{Assump_lr} and
\enuref{Assump_growth} and Eq. (\ref{eq:propChaos_proof_thetaBar_bound}).
We have:
\begin{align*}
\left\Vert \boldsymbol{E}_{1,i}^{k}\right\Vert _{2} & \leq\frac{1}{\epsilon}\int_{k\epsilon}^{\left(k+1\right)\epsilon}\left|\xi\left(s\right)-\xi\left(k\epsilon\right)\right|\left\Vert \boldsymbol{G}\left(\bar{\boldsymbol{\theta}}_{i}^{s};\rho^{s}\right)\right\Vert _{2}{\rm d}s\\
 & \qquad+\frac{1}{\epsilon}\xi\left(k\epsilon\right)\int_{k\epsilon}^{\left(k+1\right)\epsilon}\left\Vert \boldsymbol{G}\left(\bar{\boldsymbol{\theta}}_{i}^{s};\rho^{s}\right)-\boldsymbol{G}\left(\bar{\boldsymbol{\theta}}_{i}^{k\epsilon};\rho^{s}\right)\right\Vert _{2}{\rm d}s\\
 & \qquad+\frac{1}{\epsilon}\xi\left(k\epsilon\right)\int_{k\epsilon}^{\left(k+1\right)\epsilon}\left\Vert \boldsymbol{G}\left(\bar{\boldsymbol{\theta}}_{i}^{k\epsilon};\rho^{s}\right)-\boldsymbol{G}\left(\bar{\boldsymbol{\theta}}_{i}^{k\epsilon};\rho^{k\epsilon}\right)\right\Vert _{2}{\rm d}s\\
 & \equiv E_{i,1}^{k}+E_{i,2}^{k}+E_{i,3}^{k}.
\end{align*}
Consider $E_{i,1}^{k}$:
\[
E_{i,1}^{k}\lesssim\epsilon\left(\left\Vert \bar{\boldsymbol{\theta}}_{i}^{s}\right\Vert _{2}+1\right)\lesssim\epsilon\left(\left\Vert \bar{\boldsymbol{\theta}}_{i}^{0}\right\Vert _{2}+1\right).
\]
We also have from Eq. (\ref{eq:propChaos_proof_thetaBar_tDiff_bound}):
\begin{align*}
\left\Vert \boldsymbol{G}\left(\bar{\boldsymbol{\theta}}_{i}^{s};\rho^{s}\right)-\boldsymbol{G}\left(\bar{\boldsymbol{\theta}}_{i}^{k\epsilon};\rho^{s}\right)\right\Vert _{2} & \leq\left\Vert \nabla V\left(\bar{\boldsymbol{\theta}}_{i}^{s}\right)-\nabla V\left(\bar{\boldsymbol{\theta}}_{i}^{k\epsilon}\right)\right\Vert _{2}+\left\Vert \nabla_{1}W\left(\bar{\boldsymbol{\theta}}_{i}^{s};\rho^{s}\right)-\nabla_{1}W\left(\bar{\boldsymbol{\theta}}_{i}^{k\epsilon};\rho^{s}\right)\right\Vert _{2}\\
 & \lesssim\left\Vert \bar{\boldsymbol{\theta}}_{i}^{s}-\bar{\boldsymbol{\theta}}_{i}^{k\epsilon}\right\Vert _{2}\lesssim\epsilon\left(\left\Vert \bar{\boldsymbol{\theta}}_{i}^{0}\right\Vert _{2}+1\right),
\end{align*}
which yields:
\[
E_{i,2}^{k}\lesssim\epsilon\left(\left\Vert \bar{\boldsymbol{\theta}}_{i}^{0}\right\Vert _{2}+1\right).
\]
For the third term $E_{i,3}^{k}$:
\begin{align*}
E_{i,3}^{k} & =\frac{1}{\epsilon}\xi\left(k\epsilon\right)\int_{k\epsilon}^{\left(k+1\right)\epsilon}\left\Vert \nabla_{1}W\left(\bar{\boldsymbol{\theta}}_{i}^{k\epsilon};\rho^{s}\right)-\nabla_{1}W\left(\bar{\boldsymbol{\theta}}_{i}^{k\epsilon};\rho^{k\epsilon}\right)\right\Vert _{2}{\rm d}s\\
 & \lesssim\epsilon\left(\left\Vert \bar{\boldsymbol{\theta}}_{i}^{k\epsilon}\right\Vert _{2}+1\right)\lesssim\epsilon\left(\left\Vert \bar{\boldsymbol{\theta}}_{i}^{0}\right\Vert _{2}+1\right).
\end{align*}
Combining the terms, we then obtain that, under the event $\mathsf{Ev}$,
for all $t\in\mathbb{N}\epsilon\cap\left[0,T\right]$,
\begin{align*}
\frac{\epsilon}{N}\sum_{k=0}^{t/\epsilon-1}\sum_{i=1}^{N}\left\Vert \boldsymbol{E}_{1,i}^{k}\right\Vert _{2}\left\Vert \boldsymbol{\delta}_{i}^{k}\right\Vert _{2} & \lesssim\frac{\epsilon^{2}}{N}\sum_{k=0}^{t/\epsilon-1}\sum_{i=1}^{N}\left(\left\Vert \bar{\boldsymbol{\theta}}_{i}^{0}\right\Vert _{2}+1\right)\left\Vert \boldsymbol{\delta}_{i}^{k}\right\Vert _{2}\\
 & \lesssim\epsilon^{2}\sum_{k=0}^{t/\epsilon-1}\left(\sqrt{\frac{1}{N}\sum_{i=1}^{N}\left\Vert \bar{\boldsymbol{\theta}}_{i}^{0}\right\Vert _{2}^{2}}+1\right)\sqrt{\mathscr{E}_{k}}\lesssim\epsilon^{2}\sum_{k=0}^{t/\epsilon-1}\sqrt{\mathscr{E}_{k}},\\
\frac{\epsilon^{2}}{N}\sum_{k=0}^{t/\epsilon-1}\sum_{i=1}^{N}\left\Vert \boldsymbol{E}_{1,i}^{k}\right\Vert _{2}^{2} & \lesssim\frac{\epsilon^{4}}{N}\sum_{k=0}^{t/\epsilon-1}\sum_{i=1}^{N}\left(\left\Vert \bar{\boldsymbol{\theta}}_{i}^{0}\right\Vert _{2}^{2}+1\right)\lesssim\epsilon^{3}.
\end{align*}
This concludes the proof.
\end{proof}
\begin{lem}[Control of $\boldsymbol{E}_{2,i}^{k}$]
\label{lem:propChaos-E2}Consider the same setting as Theorem \ref{thm:propChaos}.
For any $\delta>0$, on the event $\mathsf{Ev}$, with probability
at least $1-C\exp\left(-\delta^{2/5}\right)$, for all $t\in\mathbb{N}\epsilon\cap\left[0,T\right]$:
\begin{align*}
\frac{\epsilon}{N}\sum_{k=0}^{t/\epsilon-1}\sum_{i=1}^{N}\left\Vert \boldsymbol{E}_{2,i}^{k}\right\Vert _{2}\left\Vert \boldsymbol{\delta}_{i}^{k}\right\Vert _{2} & \lesssim\epsilon\mathfrak{E}\sum_{k=0}^{t/\epsilon-1}\sqrt{\mathscr{E}_{k}},\\
\frac{\epsilon^{2}}{N}\sum_{k=0}^{t/\epsilon-1}\sum_{i=1}^{N}\left\Vert \boldsymbol{E}_{2,i}^{k}\right\Vert _{2}^{2} & \lesssim\epsilon\mathfrak{E}^{2},
\end{align*}
in which we define:
\[
\mathfrak{E}=\frac{\kappa_{1}}{N}+\left(\delta+\log^{5/2}\left(\frac{NT}{\epsilon}+1\right)\right)\frac{\kappa_{1}}{\sqrt{N}}.
\]
\end{lem}

\begin{proof}
We have from Assumption \ref{enu:Assump_lr}:
\begin{align*}
\left\Vert \boldsymbol{E}_{2,i}^{k}\right\Vert _{2} & \lesssim\frac{1}{N}\left\Vert \nabla_{1}W\left(\bar{\boldsymbol{\theta}}_{i}^{k\epsilon};\rho^{k\epsilon}\right)\right\Vert _{2}+\frac{1}{N}\left\Vert \nabla_{1}U\left(\bar{\boldsymbol{\theta}}_{i}^{k\epsilon};\bar{\boldsymbol{\theta}}_{i}^{k\epsilon}\right)\right\Vert _{2}\\
 & \qquad+\left\Vert \frac{1}{N}\sum_{j\neq i}\left[\nabla_{1}U\left(\bar{\boldsymbol{\theta}}_{i}^{k\epsilon};\bar{\boldsymbol{\theta}}_{j}^{k\epsilon}\right)-\int\nabla_{1}U\left(\bar{\boldsymbol{\theta}}_{i}^{k\epsilon};\boldsymbol{\theta}\right)\rho^{k\epsilon}\left({\rm d}\boldsymbol{\theta}\right)\right]\right\Vert _{2}\equiv E_{i,1}^{k}+E_{i,2}^{k}+E_{i,3}^{k}.
\end{align*}
By Assumption \ref{enu:Assump_growth} and Eq. (\ref{eq:propChaos_proof_thetaBar_bound}),
under the event $\mathsf{Ev}$,
\begin{align*}
\sum_{i=1}^{N}E_{i,1}^{k}\left\Vert \boldsymbol{\delta}_{i}^{k}\right\Vert _{2} & \lesssim\frac{1}{N}\sum_{i=1}^{N}\left(\left\Vert \bar{\boldsymbol{\theta}}_{i}^{0}\right\Vert _{2}+1\right)\left\Vert \boldsymbol{\delta}_{i}^{k}\right\Vert _{2}\lesssim\left(\sqrt{\frac{1}{N}\sum_{i=1}^{N}\left\Vert \bar{\boldsymbol{\theta}}_{i}^{0}\right\Vert _{2}}+1\right)\sqrt{\mathscr{E}_{k}}\lesssim\sqrt{\mathscr{E}_{k}},\\
\sum_{i=1}^{N}\left(E_{i,1}^{k}\right)^{2} & \lesssim\frac{1}{N^{2}}\sum_{i=1}^{N}\left(\left\Vert \bar{\boldsymbol{\theta}}_{i}^{0}\right\Vert _{2}^{2}+1\right)\lesssim\frac{1}{N},\\
\sum_{i=1}^{N}E_{i,2}^{k}\left\Vert \boldsymbol{\delta}_{i}^{k}\right\Vert _{2} & \lesssim\frac{\kappa_{1}}{N}\sum_{i=1}^{N}\left(\left\Vert \bar{\boldsymbol{\theta}}_{i}^{0}\right\Vert _{2}^{3}+1\right)\left\Vert \boldsymbol{\delta}_{i}^{k}\right\Vert _{2}\lesssim\kappa_{1}\left(\sqrt{\frac{1}{N}\sum_{i=1}^{N}\left\Vert \bar{\boldsymbol{\theta}}_{i}^{0}\right\Vert _{2}^{6}}+1\right)\sqrt{\mathscr{E}_{k}}\lesssim\kappa_{1}\sqrt{\mathscr{E}_{k}},\\
\sum_{i=1}^{N}\left(E_{i,2}^{k}\right)^{2} & \lesssim\frac{\kappa_{1}^{2}}{N^{2}}\sum_{i=1}^{N}\left(\left\Vert \bar{\boldsymbol{\theta}}_{i}^{0}\right\Vert _{2}^{6}+1\right)\lesssim\frac{\kappa_{1}^{2}}{N}.
\end{align*}
For the third term $E_{i,3}^{k}$, recall that $\left(\bar{\boldsymbol{\theta}}_{j}^{k\epsilon}\right)_{j\leq N}$
are i.i.d. according to $\rho^{k\epsilon}$ and that the randomness
comes from the initialization $\left(\bar{\boldsymbol{\theta}}_{j}^{0}\right)_{j\leq N}$.
For brevity, we define
\[
\boldsymbol{a}_{ij}^{k}=\nabla_{1}U\left(\bar{\boldsymbol{\theta}}_{i}^{k\epsilon};\bar{\boldsymbol{\theta}}_{j}^{k\epsilon}\right)-\int\nabla_{1}U\left(\bar{\boldsymbol{\theta}}_{i}^{k\epsilon};\boldsymbol{\theta}\right)\rho^{k\epsilon}\left({\rm d}\boldsymbol{\theta}\right).
\]
We then have for $j\neq i$ and a positive integer $p$, by Assumptions
\ref{enu:Assump_growth}, \ref{enu:Assump_ODE} and Eq. (\ref{eq:propChaos_proof_thetaBar_bound}):
\begin{align*}
\mathbb{E}\left\{ \left\Vert \boldsymbol{a}_{ij}^{k}\right\Vert _{2}^{2p}\right\}  & \leq2^{2p}\mathbb{E}\left\{ \left\Vert \nabla_{1}U\left(\bar{\boldsymbol{\theta}}_{i}^{k\epsilon};\bar{\boldsymbol{\theta}}_{j}^{k\epsilon}\right)\right\Vert _{2}^{2p}\right\} \leq C^{2p}\kappa_{1}^{2p}\mathbb{E}\left\{ \left(\left\Vert \bar{\boldsymbol{\theta}}_{i}^{0}\right\Vert _{2}^{2p}+1\right)\left(\left\Vert \bar{\boldsymbol{\theta}}_{j}^{0}\right\Vert _{2}^{4p}+1\right)\right\} \\
 & =C^{2p}\kappa_{1}^{2p}\left(p^{p}+1\right)\left(p^{2p}+1\right)\leq C^{2p}\kappa_{1}^{2p}p^{3p}.
\end{align*}
Note that for a fixed $i$, $\left(\boldsymbol{a}_{ij}^{k}\right)_{j\neq i,\;j\leq N}$
are independent with zero mean, conditional on $\bar{\boldsymbol{\theta}}_{i}^{k\epsilon}$.
Hence, by Lemma \ref{lem:bound_moment_sym},
\[
\mathbb{E}\left\{ \left(E_{i,3}^{k}\right)^{2p}\right\} =\mathbb{E}\left\{ \left\Vert \frac{1}{N}\sum_{j\neq i}\boldsymbol{a}_{ij}^{k}\right\Vert _{2}^{2p}\right\} \leq C^{p}\kappa_{1}^{2p}p^{5p}/N^{p}.
\]
It is then easy to see that $\left(E_{i,3}^{k}\right)^{2/5}$ is sub-exponential
with $\psi_{1}$-norm $\left\Vert \left(E_{i,3}^{k}\right)^{2/5}\right\Vert _{\psi_{1}}\lesssim\kappa_{1}^{2/5}/N^{1/5}.$
By Lemma \ref{lem:subgauss_subexp_properties} and the union bound,
on the event $\mathsf{Ev}$, with probability at least $1-C\exp\left(-\delta^{2/5}\right)$:
\[
\max_{k\leq T/\epsilon}\max_{i\leq N}E_{i,3}^{k}\lesssim\left(\delta+\log^{5/2}\left(\frac{NT}{\epsilon}+1\right)\right)\frac{\kappa_{1}}{\sqrt{N}}.
\]
Combining these bounds, we obtain the claim.
\end{proof}
\begin{lem}[Control of $\left\langle \boldsymbol{\delta}_{i}^{k},\boldsymbol{E}_{4,i}^{k}\right\rangle $]
\label{lem:propChaos-E4-prod}Consider the same setting as Theorem
\ref{thm:propChaos}. For a sufficiently small absolute constant $c$
and any $\delta\leq c/\sqrt{\epsilon}$, on the event $\mathsf{Ev}$,
with probability at least $1-2\exp\left(-\delta^{2}\right)$:
\[
\max_{k\leq T/\epsilon}\left|\epsilon\underline{Z}_{{\rm st}}^{k}\right|\lesssim\sqrt{\epsilon}\kappa_{5}\left(\gamma_{{\rm st}}^{2}+\sqrt{\gamma_{{\rm st}}}\right)\delta,
\]
in which we define:
\[
\underline{Z}_{{\rm st}}^{k}=\frac{1}{N}\sum_{\ell=0}^{k\land\left(T_{{\rm st}}/\epsilon\right)-1}\sum_{i=1}^{N}\left\langle \boldsymbol{\delta}_{i}^{\ell},\boldsymbol{E}_{4,i}^{\ell}\right\rangle .
\]
\end{lem}

\begin{proof}
Let us define:
\begin{align*}
Z_{i}^{k} & =\left\langle \boldsymbol{\delta}_{i}^{k},\boldsymbol{E}_{4,i}^{k}\right\rangle =\xi\left(k\epsilon\right)\left\langle \boldsymbol{\delta}_{i}^{k},\mathbb{E}\left\{ \boldsymbol{F}_{i}\left(\Theta^{k};\boldsymbol{z}^{k}\right)\middle|{\cal F}^{k}\right\} -\boldsymbol{F}_{i}\left(\Theta^{k};\boldsymbol{z}^{k}\right)\right\rangle ,\\
\underline{Z}^{k} & =\frac{1}{N}\sum_{\ell=0}^{k-1}\sum_{i=1}^{N}Z_{i}^{\ell},\qquad\underline{Z}^{0}=0.
\end{align*}
Recall that $\boldsymbol{\delta}_{i}^{k}=\boldsymbol{\theta}_{i}^{k}-\bar{\boldsymbol{\theta}}_{i}^{k\epsilon}$
and ${\cal F}^{k}$ is the sigma-algebra generated by $\left(\bar{\boldsymbol{\theta}}_{i}^{0}\right)_{i\leq N}$
and $\left(\boldsymbol{z}^{\ell}\right)_{\ell\leq k-1}$, and hence
$\boldsymbol{\delta}_{i}^{k}$ is ${\cal F}^{k}$-measurable. Therefore
$\left(\underline{Z}^{k}\right)_{k\geq0}$ is a martingale adapted
to the filtration $\left({\cal F}^{k}\right)_{k\geq0}$. Conditioning
on ${\cal F}^{k}$, on the event $\mathsf{Ev}$, we have by Assumptions
\ref{enu:Assump_lr}, \ref{enu:Assump_F} and Eq. (\ref{eq:propChaos_proof_theta_bound}):
\begin{align*}
\left\Vert \frac{1}{N}\sum_{i=1}^{N}\xi\left(k\epsilon\right)\left\langle \boldsymbol{\delta}_{i}^{k},\boldsymbol{F}_{i}\left(\Theta^{k};\boldsymbol{z}^{k}\right)\right\rangle \right\Vert _{\psi_{1}} & \lesssim\frac{1}{N}\sum_{i=1}^{N}\left\Vert \boldsymbol{\delta}_{i}^{k}\right\Vert _{2}\left\Vert \boldsymbol{F}_{i}\left(\Theta^{k};\boldsymbol{z}^{k}\right)\right\Vert _{\psi_{1}}\\
 & \lesssim\frac{\kappa_{5}}{N}\sum_{i=1}^{N}\left\Vert \boldsymbol{\delta}_{i}^{k}\right\Vert _{2}\left(\left\Vert \boldsymbol{\theta}_{i}^{k}\right\Vert _{2}+1\right)\left(\frac{1}{N}\sum_{j=1}^{N}\left\Vert \boldsymbol{\theta}_{j}^{k}\right\Vert _{2}^{2}+1\right)\\
 & \lesssim\frac{\kappa_{5}}{N}\sum_{i=1}^{N}\left\Vert \boldsymbol{\delta}_{i}^{k}\right\Vert _{2}\left(\left\Vert \bar{\boldsymbol{\theta}}_{i}^{0}\right\Vert _{2}+\left\Vert \boldsymbol{\delta}_{i}^{k}\right\Vert _{2}+1\right)\left(\frac{1}{N}\sum_{j=1}^{N}\left\Vert \bar{\boldsymbol{\theta}}_{j}^{0}\right\Vert _{2}^{2}+\mathscr{E}_{k}+1\right)\\
 & \lesssim\kappa_{5}\sqrt{\mathscr{E}_{k}}\left(\sqrt{\frac{1}{N}\sum_{i=1}^{N}\left\Vert \bar{\boldsymbol{\theta}}_{i}^{0}\right\Vert _{2}^{2}}+\sqrt{\mathscr{E}_{k}}+1\right)\left(\frac{1}{N}\sum_{j=1}^{N}\left\Vert \bar{\boldsymbol{\theta}}_{j}^{0}\right\Vert _{2}^{2}+\mathscr{E}_{k}+1\right)\\
 & \lesssim\kappa_{5}\left(\mathscr{E}_{k}^{2}+\sqrt{\mathscr{E}_{k}}\right),
\end{align*}
which implies
\[
\left\Vert \frac{1}{N}\sum_{i=1}^{N}Z_{i}^{k}\right\Vert _{\psi_{1}}\lesssim\kappa_{5}\left(\mathscr{E}_{k}^{2}+\sqrt{\mathscr{E}_{k}}\right).
\]
We now consider the martingale $\underline{Z}_{{\rm st}}^{k}=\underline{Z}^{k\land\left(T_{{\rm st}}/\epsilon\right)}$,
where we recall the stopping time is $T_{{\rm st}}$ defined in Eq.
(\ref{eq:MF_proof_stopping_time}). Then we have that conditioning
on ${\cal F}^{k}$, on the event $\mathsf{Ev}$, the martingale difference
$\underline{Z}_{{\rm st}}^{k+1}-\underline{Z}_{{\rm st}}^{k}$ is
sub-exponential with zero mean and $\psi_{1}$-norm upper-bounded
by $C\kappa_{5}\left(\gamma_{{\rm st}}^{2}+\sqrt{\gamma_{{\rm st}}}\right)$.
The thesis then follows from Lemma \ref{lem:azuma}.
\end{proof}
\begin{lem}[Control of $\left\Vert \boldsymbol{E}_{4,i}^{k}\right\Vert _{2}^{2}$]
\label{lem:propChaos-E4-norm}Consider the same setting as Theorem
\ref{thm:propChaos}. For any $\delta>0$, on the event $\mathsf{Ev}$,
with probability at least $1-\delta^{-1}$, for any $t\in\mathbb{N}\epsilon\cap\left[0,T\wedge T_{{\rm st}}\right]$,
\[
\frac{\epsilon^{2}}{N}\sum_{k=0}^{t/\epsilon-1}\sum_{i=1}^{N}\left\Vert \boldsymbol{E}_{4,i}^{k}\right\Vert _{2}^{2}\lesssim\epsilon D^{2}\kappa_{5}^{2}\delta.
\]
\end{lem}

\begin{proof}
To analyze the term $\left\Vert \boldsymbol{E}_{4,i}^{k}\right\Vert _{2}^{2}$,
recall that $\mathbb{E}\left\{ \boldsymbol{E}_{4,i}^{k}\middle|{\cal F}^{k}\right\} =\boldsymbol{0}$
and ${\cal F}^{k}$ is the sigma-algebra generated by $\left(\bar{\boldsymbol{\theta}}_{i}^{0}\right)_{i\leq N}$
and $\left(\boldsymbol{z}^{\ell}\right)_{\ell\leq k-1}$. Conditioning
on ${\cal F}^{k}$, on the event $\mathsf{Ev}$, we have by Assumptions
\ref{enu:Assump_lr}, \ref{enu:Assump_F} and Eq. (\ref{eq:propChaos_proof_theta_bound}):
\begin{align*}
\left\Vert \boldsymbol{E}_{4,i}^{k}\right\Vert _{\psi_{1}} & \lesssim\left\Vert \boldsymbol{F}_{i}\left(\Theta^{k};\boldsymbol{z}^{k}\right)\right\Vert _{\psi_{1}}\lesssim\kappa_{5}\left(\left\Vert \boldsymbol{\theta}_{i}^{k}\right\Vert _{2}+1\right)\left(\frac{1}{N}\sum_{j=1}^{N}\left\Vert \boldsymbol{\theta}_{j}^{k}\right\Vert _{2}^{2}+1\right)\\
 & \lesssim\kappa_{5}\left(\left\Vert \bar{\boldsymbol{\theta}}_{i}^{0}\right\Vert _{2}+\left\Vert \boldsymbol{\delta}_{i}^{k}\right\Vert _{2}+1\right)\left(\frac{1}{N}\sum_{j=1}^{N}\left\Vert \bar{\boldsymbol{\theta}}_{j}^{0}\right\Vert _{2}^{2}+\mathscr{E}_{k}+1\right)\lesssim\kappa_{5}\left(\left\Vert \bar{\boldsymbol{\theta}}_{i}^{0}\right\Vert _{2}+\left\Vert \boldsymbol{\delta}_{i}^{k}\right\Vert _{2}+1\right)\left(\mathscr{E}_{k}+1\right),
\end{align*}
and therefore, by Lemma \ref{lem:subexp_vector_norm_E}, on the event
$\mathsf{Ev}$,
\begin{align*}
\mathbb{E}\left\{ \frac{1}{N}\sum_{i=1}^{N}\left\Vert \boldsymbol{E}_{4,i}^{k}\right\Vert _{2}^{2}\middle|{\cal F}^{k}\right\}  & \lesssim\frac{1}{N}\sum_{i=1}^{N}D^{2}\kappa_{5}^{2}\left(\left\Vert \bar{\boldsymbol{\theta}}_{i}^{0}\right\Vert _{2}+\left\Vert \boldsymbol{\delta}_{i}^{k}\right\Vert _{2}+1\right)^{2}\left(\mathscr{E}_{k}+1\right)^{2}+1\\
 & \lesssim D^{2}\kappa_{5}^{2}\left(\frac{1}{N}\sum_{i=1}^{N}\left\Vert \bar{\boldsymbol{\theta}}_{i}^{0}\right\Vert _{2}^{2}+\mathscr{E}_{k}+1\right)\left(\mathscr{E}_{k}+1\right)^{2}+1\\
 & \lesssim D^{2}\kappa_{5}^{2}\left(\mathscr{E}_{k}^{3}+1\right)+1.
\end{align*}
The last inequality implies that
\[
\mathbb{E}\left\{ \frac{1}{N}\sum_{i=1}^{N}\left\Vert \boldsymbol{E}_{4,i}^{k\wedge\left(T_{{\rm st}}/\epsilon\right)}\right\Vert _{2}^{2}\middle|{\cal F}^{k}\right\} \mathbb{I}\left(\mathsf{Ev}\right)\lesssim D^{2}\kappa_{5}^{2},
\]
since $\gamma_{{\rm st}}\leq1$. Therefore,
\[
\mathbb{E}\left\{ \max_{t\in\mathbb{N}\epsilon\cap\left[0,T\wedge T_{{\rm st}}\right]}\frac{\epsilon^{2}}{N}\sum_{k=0}^{t/\epsilon-1}\sum_{i=1}^{N}\left\Vert \boldsymbol{E}_{4,i}^{k}\right\Vert _{2}^{2}\mathbb{I}\left(\mathsf{Ev}\right)\right\} \leq\mathbb{E}\left\{ \frac{\epsilon^{2}}{N}\sum_{k=0}^{\left(T\wedge T_{{\rm st}}\right)/\epsilon-1}\sum_{i=1}^{N}\left\Vert \boldsymbol{E}_{4,i}^{k}\right\Vert _{2}^{2}\mathbb{I}\left(\mathsf{Ev}\right)\right\} \lesssim D^{2}\kappa_{5}^{2}\epsilon.
\]
The thesis then follows from Markov's inequality.
\end{proof}
\begin{lem}[Control of $\boldsymbol{E}_{3,i}^{k}$]
\label{lem:propChaos-E3}Consider the same setting as Theorem \ref{thm:propChaos}.
On the event $\mathsf{Ev}$, with probability at least $1-\Xi\left(N;T,\kappa_{6}\right)$,
for all $t\in\mathbb{N}\epsilon\cap\left[0,T\wedge T_{{\rm st}}\right]$,
\begin{align*}
-\frac{\epsilon}{N}\sum_{k=0}^{t/\epsilon-1}\sum_{i=1}^{N}\left\langle \boldsymbol{\delta}_{i}^{k},\boldsymbol{E}_{3,i}^{k}\right\rangle  & \lesssim\epsilon\sum_{k=0}^{t/\epsilon-1}\left(\mathscr{E}_{k}+\left(\kappa_{3}+\kappa_{4}\right)\mathscr{E}_{k}^{3/2}\right),\\
\frac{\epsilon^{2}}{N}\sum_{k=0}^{t/\epsilon-1}\sum_{i=1}^{N}\left\Vert \boldsymbol{E}_{3,i}^{k}\right\Vert _{2}^{2} & \lesssim\epsilon^{2}\kappa_{2}^{2}\sum_{k=0}^{t/\epsilon-1}\mathscr{E}_{k}.
\end{align*}
\end{lem}

\begin{proof}
We decompose the proof into two steps.

\paragraph*{Step 1: Control of $-\left\langle \boldsymbol{\delta}_{i}^{k},\boldsymbol{E}_{3,i}^{k}\right\rangle $.}

We have:
\[
\boldsymbol{E}_{3,i}^{k}=\xi\left(k\epsilon\right)\left[\nabla V\left(\boldsymbol{\theta}_{i}^{k}\right)-\nabla V\left(\bar{\boldsymbol{\theta}}_{i}^{k\epsilon}\right)\right]+\xi\left(k\epsilon\right)\frac{1}{N}\sum_{j=1}^{N}\left[\nabla_{1}U\left(\boldsymbol{\theta}_{i}^{k},\boldsymbol{\theta}_{j}^{k}\right)-\nabla_{1}U\left(\bar{\boldsymbol{\theta}}_{i}^{k\epsilon},\bar{\boldsymbol{\theta}}_{j}^{k\epsilon}\right)\right].
\]
From Assumption \ref{enu:Assump_growth},
\[
\left\Vert \nabla V\left(\boldsymbol{\theta}_{i}^{k}\right)-\nabla V\left(\bar{\boldsymbol{\theta}}_{i}^{k\epsilon}\right)\right\Vert _{2}\lesssim\left\Vert \boldsymbol{\delta}_{i}^{k}\right\Vert _{2},
\]
which, by Assumption \ref{enu:Assump_lr}, gives
\[
-\frac{1}{N}\sum_{i=1}^{N}\left\langle \boldsymbol{\delta}_{i}^{k},\xi\left(k\epsilon\right)\left[\nabla V\left(\boldsymbol{\theta}_{i}^{k}\right)-\nabla V\left(\bar{\boldsymbol{\theta}}_{i}^{k\epsilon}\right)\right]\right\rangle \lesssim\mathscr{E}_{k}.
\]
We have from Taylor's theorem:
\begin{align}
 & \nabla_{1}U\left(\boldsymbol{\theta}_{i}^{k},\boldsymbol{\theta}_{j}^{k}\right)-\nabla_{1}U\left(\bar{\boldsymbol{\theta}}_{i}^{k\epsilon},\bar{\boldsymbol{\theta}}_{j}^{k\epsilon}\right)\nonumber \\
 & \quad=\left[\nabla_{1}U\left(\boldsymbol{\theta}_{i}^{k},\bar{\boldsymbol{\theta}}_{j}^{k\epsilon}\right)-\nabla_{1}U\left(\bar{\boldsymbol{\theta}}_{i}^{k\epsilon},\bar{\boldsymbol{\theta}}_{j}^{k\epsilon}\right)\right]+\left[\nabla_{1}U\left(\boldsymbol{\theta}_{i}^{k},\boldsymbol{\theta}_{j}^{k}\right)-\nabla_{1}U\left(\boldsymbol{\theta}_{i}^{k},\bar{\boldsymbol{\theta}}_{j}^{k\epsilon}\right)\right]\nonumber \\
 & \quad=\nabla_{11}^{2}U\left(\boldsymbol{\zeta}_{1,ij}^{k},\bar{\boldsymbol{\theta}}_{j}^{k\epsilon}\right)\boldsymbol{\delta}_{i}^{k}+\nabla_{12}^{2}U\left(\boldsymbol{\theta}_{i}^{k},\bar{\boldsymbol{\theta}}_{j}^{k\epsilon}\right)\boldsymbol{\delta}_{j}^{k}+\nabla_{122}^{3}U\left[\boldsymbol{\theta}_{i}^{k},\boldsymbol{\zeta}_{2,ij}^{k}\right]\left(\boldsymbol{\delta}_{j}^{k},\boldsymbol{\delta}_{j}^{k}\right)\nonumber \\
 & \quad=\nabla_{11}^{2}U\left(\boldsymbol{\zeta}_{1,ij}^{k},\bar{\boldsymbol{\theta}}_{j}^{k\epsilon}\right)\boldsymbol{\delta}_{i}^{k}+\nabla_{12}^{2}U\left(\bar{\boldsymbol{\theta}}_{i}^{k\epsilon},\bar{\boldsymbol{\theta}}_{j}^{k\epsilon}\right)\boldsymbol{\delta}_{j}^{k}+\nabla_{121}^{3}U\left[\boldsymbol{\zeta}_{3,ij}^{k},\bar{\boldsymbol{\theta}}_{j}^{k\epsilon}\right]\left(\boldsymbol{\delta}_{i}^{k},\boldsymbol{\delta}_{j}^{k}\right)\nonumber \\
 & \quad\qquad+\nabla_{122}^{3}U\left[\boldsymbol{\theta}_{i}^{k},\boldsymbol{\zeta}_{2,ij}^{k}\right]\left(\boldsymbol{\delta}_{j}^{k},\boldsymbol{\delta}_{j}^{k}\right),\label{eq:propChaos_proof_Taylor}
\end{align}
for some appropriate $\boldsymbol{\zeta}_{1,ij}^{k},\boldsymbol{\zeta}_{3,ij}^{k}\in\left[\bar{\boldsymbol{\theta}}_{i}^{k\epsilon},\boldsymbol{\theta}_{i}^{k}\right]$
and $\boldsymbol{\zeta}_{2,ij}^{k}\in\left[\bar{\boldsymbol{\theta}}_{j}^{k\epsilon},\boldsymbol{\theta}_{j}^{k}\right]$.
Notice that
\begin{align}
\sum_{i=1}^{N}\sum_{j=1}^{N}\left\langle \boldsymbol{\delta}_{i}^{k},\nabla_{12}^{2}U\left(\bar{\boldsymbol{\theta}}_{i}^{k\epsilon},\bar{\boldsymbol{\theta}}_{j}^{k\epsilon}\right)\boldsymbol{\delta}_{j}^{k}\right\rangle  & =\kappa^{2}\sum_{i=1}^{N}\sum_{j=1}^{N}\mathbb{E}_{{\cal P}}\left\{ \left\langle \boldsymbol{\delta}_{i}^{k},\nabla_{2}\sigma_{*}\left(\boldsymbol{x};\kappa\bar{\boldsymbol{\theta}}_{i}^{k\epsilon}\right)^{\top}\nabla_{2}\sigma_{*}\left(\boldsymbol{x};\kappa\bar{\boldsymbol{\theta}}_{j}^{k\epsilon}\right)\boldsymbol{\delta}_{j}^{k}\right\rangle \right\} \nonumber \\
 & =\kappa^{2}\mathbb{E}_{{\cal P}}\left\{ \left\Vert \sum_{i=1}^{N}\nabla_{2}\sigma_{*}\left(\boldsymbol{x};\kappa\bar{\boldsymbol{\theta}}_{i}^{k\epsilon}\right)\boldsymbol{\delta}_{i}^{k}\right\Vert _{2}^{2}\right\} \geq0.\label{eq:nabla_12_U}
\end{align}
Also recall $\xi\left(\cdot\right)\geq0$. Therefore we can remove
the quantity containing $\nabla_{12}^{2}U\left(\bar{\boldsymbol{\theta}}_{i}^{k\epsilon},\bar{\boldsymbol{\theta}}_{j}^{k\epsilon}\right)$
from the right-hand side upper bound and obtain the bound:
\begin{align*}
-\frac{1}{N}\sum_{i=1}^{N}\left\langle \boldsymbol{\delta}_{i}^{k},\boldsymbol{E}_{3,i}^{k}\right\rangle  & \lesssim\mathscr{E}_{k}+\frac{1}{N}\sum_{i=1}^{N}\left|\left\langle \boldsymbol{\delta}_{i}^{k},\left[\frac{1}{N}\sum_{j=1}^{N}\nabla_{11}^{2}U\left(\boldsymbol{\zeta}_{1,ij}^{k},\bar{\boldsymbol{\theta}}_{j}^{k\epsilon}\right)\right]\boldsymbol{\delta}_{i}^{k}\right\rangle \right|\\
 & \qquad+\frac{1}{N^{2}}\sum_{i=1}^{N}\sum_{j=1}^{N}\left\Vert \nabla_{121}^{3}U\left[\boldsymbol{\zeta}_{3,ij}^{k},\bar{\boldsymbol{\theta}}_{j}^{k\epsilon}\right]\right\Vert _{{\rm op}}\left\Vert \boldsymbol{\delta}_{i}^{k}\right\Vert _{2}^{2}\left\Vert \boldsymbol{\delta}_{j}^{k}\right\Vert _{2}\\
 & \qquad+\frac{1}{N^{2}}\sum_{i=1}^{N}\sum_{j=1}^{N}\left\Vert \nabla_{122}^{3}U\left[\boldsymbol{\theta}_{i}^{k},\boldsymbol{\zeta}_{2,ij}^{k}\right]\right\Vert _{{\rm op}}\left\Vert \boldsymbol{\delta}_{i}^{k}\right\Vert _{2}\left\Vert \boldsymbol{\delta}_{j}^{k}\right\Vert _{2}^{2}\\
 & \equiv\mathscr{E}_{k}+A_{1}^{k}+A_{2}^{k}+A_{3}^{k}.
\end{align*}
We have, by Assumption \ref{enu:Assump_opNorm} and Eq. (\ref{eq:propChaos_proof_thetaBar_bound}),
on the event $\mathsf{Ev}$,
\begin{align*}
A_{2}^{k} & \lesssim\frac{\kappa_{3}}{N^{2}}\sum_{i=1}^{N}\sum_{j=1}^{N}\left(\left\Vert \bar{\boldsymbol{\theta}}_{j}^{0}\right\Vert _{2}+1\right)\left\Vert \boldsymbol{\delta}_{i}^{k}\right\Vert _{2}^{2}\left\Vert \boldsymbol{\delta}_{j}^{k}\right\Vert _{2}=\frac{\kappa_{3}}{N}\mathscr{E}_{k}\sum_{j=1}^{N}\left(\left\Vert \bar{\boldsymbol{\theta}}_{j}^{0}\right\Vert _{2}+1\right)\left\Vert \boldsymbol{\delta}_{j}^{k}\right\Vert _{2}\\
 & \lesssim\kappa_{3}\mathscr{E}_{k}^{3/2}\left(\sqrt{\frac{1}{N}\sum_{j=1}^{N}\left\Vert \bar{\boldsymbol{\theta}}_{j}^{0}\right\Vert _{2}^{2}}+1\right)\lesssim\kappa_{3}\mathscr{E}_{k}^{3/2}.
\end{align*}
Likewise, by Assumption \ref{enu:Assump_opNorm} and Eq. (\ref{eq:propChaos_proof_theta_bound}),
on the event $\mathsf{Ev}$,
\begin{align*}
A_{3}^{k} & \lesssim\frac{\kappa_{4}}{N^{2}}\sum_{i=1}^{N}\sum_{j=1}^{N}\left(\left\Vert \boldsymbol{\theta}_{i}^{k}\right\Vert _{2}+1\right)\left\Vert \boldsymbol{\delta}_{i}^{k}\right\Vert _{2}\left\Vert \boldsymbol{\delta}_{j}^{k}\right\Vert _{2}^{2}\lesssim\frac{\kappa_{4}}{N^{2}}\sum_{i=1}^{N}\sum_{j=1}^{N}\left(\left\Vert \boldsymbol{\delta}_{i}^{k}\right\Vert _{2}+\left\Vert \bar{\boldsymbol{\theta}}_{i}^{0}\right\Vert _{2}+1\right)\left\Vert \boldsymbol{\delta}_{i}^{k}\right\Vert _{2}\left\Vert \boldsymbol{\delta}_{j}^{k}\right\Vert _{2}^{2}\\
 & =\frac{\kappa_{4}}{N}\mathscr{E}_{k}\sum_{i=1}^{N}\left(\left\Vert \boldsymbol{\delta}_{i}^{k}\right\Vert _{2}+\left\Vert \bar{\boldsymbol{\theta}}_{i}^{0}\right\Vert _{2}+1\right)\left\Vert \boldsymbol{\delta}_{i}^{k}\right\Vert _{2}\lesssim\kappa_{4}\mathscr{E}_{k}\left(\mathscr{E}_{k}+\left(\sqrt{\frac{1}{N}\sum_{i=1}^{N}\left\Vert \bar{\boldsymbol{\theta}}_{i}^{0}\right\Vert _{2}^{2}}+1\right)\sqrt{\mathscr{E}_{k}}\right)\\
 & \lesssim\kappa_{4}\left(\mathscr{E}_{k}^{3/2}+\mathscr{E}_{k}^{2}\right).
\end{align*}
We note that since $\boldsymbol{\zeta}_{1,ij}^{k}\in\left[\bar{\boldsymbol{\theta}}_{i}^{k\epsilon},\boldsymbol{\theta}_{i}^{k}\right]$,
we have $\left\Vert \boldsymbol{\zeta}_{1,ij}^{k}-\bar{\boldsymbol{\theta}}_{i}^{k\epsilon}\right\Vert _{2}\leq\left\Vert \boldsymbol{\delta}_{i}^{k}\right\Vert _{2}$.
Then on the event $\mathsf{Ev}$ and for $k\epsilon\leq T_{{\rm st}}$,
we have for any $i\in\left[N\right]$,
\begin{align*}
\left\Vert \boldsymbol{\zeta}_{1,ij}^{k}\right\Vert _{2} & \leq\left\Vert \bar{\boldsymbol{\theta}}_{i}^{k\epsilon}\right\Vert _{2}+\left\Vert \boldsymbol{\delta}_{i}^{k}\right\Vert _{2}\leq C\left(\left\Vert \bar{\boldsymbol{\theta}}_{i}^{0}\right\Vert _{2}+1\right)+\left\Vert \boldsymbol{\delta}_{i}^{k}\right\Vert _{2}\\
 & \leq C\sqrt{\sum_{i=1}^{N}\left\Vert \bar{\boldsymbol{\theta}}_{i}^{0}\right\Vert _{2}^{2}}+C+\sqrt{N}\mathscr{E}_{k}\leq C\sqrt{N}.
\end{align*}
By Assumption \ref{enu:Assump_nabla11U}, we have with probability
at least $1-\Xi\left(N;T,\kappa_{6}\right)$:
\[
\max_{k\leq T/\epsilon}\sup_{\boldsymbol{\zeta}\in{\cal B}_{D}\left(C\sqrt{N}\right)}\left\Vert \frac{1}{N}\sum_{j=1}^{N}\nabla_{11}^{2}U\left(\boldsymbol{\zeta},\bar{\boldsymbol{\theta}}_{j}^{k\epsilon}\right)\right\Vert _{{\rm op}}\leq c_{\ref{enu:Assump_nabla11U}}\left(T,C\right)\leq C.
\]
These imply that for $k\epsilon\leq T\wedge T_{{\rm st}}$, $A_{1}^{k}\lesssim\mathscr{E}_{k}$.
Combining all the bounds, we have on the event $\mathsf{Ev}$, with
probability at least $1-\Xi\left(N;T,\kappa_{6}\right)$,
\[
-\frac{\epsilon}{N}\sum_{k=0}^{t/\epsilon-1}\sum_{i=1}^{N}\left\langle \boldsymbol{\delta}_{i}^{k},\boldsymbol{E}_{3,i}^{k}\right\rangle \lesssim\epsilon\sum_{k=0}^{t/\epsilon-1}\left(\mathscr{E}_{k}+\left(\kappa_{3}+\kappa_{4}\right)\mathscr{E}_{k}^{3/2}\right),
\]
for all $t\in\mathbb{N}\epsilon\cap\left[0,T\wedge T_{{\rm st}}\right]$,
recalling $\mathscr{E}_{k}\leq\gamma_{{\rm st}}\leq1$ for $k\leq T_{{\rm st}}/\epsilon$.

\paragraph*{Step 2: Control of $\left\Vert \boldsymbol{E}_{3,i}^{k}\right\Vert _{2}^{2}$.}

We have by Assumption \ref{enu:Assump_lr}:
\begin{align*}
\left\Vert \boldsymbol{E}_{3,i}^{k}\right\Vert _{2}^{2} & \lesssim\left\Vert \nabla V\left(\boldsymbol{\theta}_{i}^{k}\right)-\nabla V\left(\bar{\boldsymbol{\theta}}_{i}^{k\epsilon}\right)\right\Vert _{2}^{2}+\left\Vert \frac{1}{N}\sum_{j=1}^{N}\left[\nabla_{1}U\left(\boldsymbol{\theta}_{i}^{k},\boldsymbol{\theta}_{j}^{k}\right)-\nabla_{1}U\left(\bar{\boldsymbol{\theta}}_{i}^{k\epsilon},\bar{\boldsymbol{\theta}}_{j}^{k\epsilon}\right)\right]\right\Vert _{2}^{2}.
\end{align*}
From Assumption \ref{enu:Assump_growth},
\[
\left\Vert \nabla V\left(\boldsymbol{\theta}_{i}^{k}\right)-\nabla V\left(\bar{\boldsymbol{\theta}}_{i}^{k\epsilon}\right)\right\Vert _{2}^{2}\lesssim\left\Vert \boldsymbol{\delta}_{i}^{k}\right\Vert _{2}^{2},
\]
which yields
\[
\frac{1}{N}\sum_{i=1}^{N}\left\Vert \nabla V\left(\boldsymbol{\theta}_{i}^{k}\right)-\nabla V\left(\bar{\boldsymbol{\theta}}_{i}^{k\epsilon}\right)\right\Vert _{2}^{2}\lesssim\mathscr{E}_{k}.
\]
Next, performing a Taylor expansion similar to Eq. (\ref{eq:propChaos_proof_Taylor})
in Step 6, we get:
\begin{align*}
 & \nabla_{1}U\left(\boldsymbol{\theta}_{i}^{k},\boldsymbol{\theta}_{j}^{k}\right)-\nabla_{1}U\left(\bar{\boldsymbol{\theta}}_{i}^{k\epsilon},\bar{\boldsymbol{\theta}}_{j}^{k\epsilon}\right)\\
 & \quad=\left[\nabla_{1}U\left(\boldsymbol{\theta}_{i}^{k},\bar{\boldsymbol{\theta}}_{j}^{k\epsilon}\right)-\nabla_{1}U\left(\bar{\boldsymbol{\theta}}_{i}^{k\epsilon},\bar{\boldsymbol{\theta}}_{j}^{k\epsilon}\right)\right]+\left[\nabla_{1}U\left(\boldsymbol{\theta}_{i}^{k},\boldsymbol{\theta}_{j}^{k}\right)-\nabla_{1}U\left(\boldsymbol{\theta}_{i}^{k},\bar{\boldsymbol{\theta}}_{j}^{k\epsilon}\right)\right]\\
 & \quad=\nabla_{11}^{2}U\left(\boldsymbol{\zeta}_{1,ij}^{k},\bar{\boldsymbol{\theta}}_{j}^{k\epsilon}\right)\boldsymbol{\delta}_{i}^{k}+\nabla_{12}^{2}U\left(\boldsymbol{\theta}_{i}^{k},\boldsymbol{\zeta}_{4,ij}^{k}\right)\boldsymbol{\delta}_{j}^{k},
\end{align*}
for $\boldsymbol{\zeta}_{1,ij}^{k}\in\left[\bar{\boldsymbol{\theta}}_{i}^{k\epsilon},\boldsymbol{\theta}_{i}^{k}\right]$
and $\boldsymbol{\zeta}_{4,ij}^{k}\in\left[\bar{\boldsymbol{\theta}}_{j}^{k\epsilon},\boldsymbol{\theta}_{j}^{k}\right]$.
Notice that, by Eq. (\ref{eq:propChaos_proof_thetaBar_bound}),
\[
\left\Vert \boldsymbol{\zeta}_{4,ij}^{k}\right\Vert _{2}\leq\left\Vert \bar{\boldsymbol{\theta}}_{j}^{k\epsilon}\right\Vert _{2}+\left\Vert \boldsymbol{\delta}_{j}^{k}\right\Vert _{2}\lesssim\left\Vert \bar{\boldsymbol{\theta}}_{j}^{0}\right\Vert _{2}+\left\Vert \boldsymbol{\delta}_{j}^{k}\right\Vert _{2}+1.
\]
On the good events of the previous step, using Assumption \ref{enu:Assump_opNorm}
and Eq. (\ref{eq:propChaos_proof_theta_bound}) with $k\epsilon\leq T\wedge T_{{\rm st}}$,
we have:
\begin{align*}
 & \frac{1}{N}\sum_{i=1}^{N}\left\Vert \frac{1}{N}\sum_{j=1}^{N}\left[\nabla_{1}U\left(\boldsymbol{\theta}_{i}^{k},\boldsymbol{\theta}_{j}^{k}\right)-\nabla_{1}U\left(\bar{\boldsymbol{\theta}}_{i}^{k\epsilon},\bar{\boldsymbol{\theta}}_{j}^{k\epsilon}\right)\right]\right\Vert _{2}^{2}\\
 & \quad\lesssim\frac{1}{N}\sum_{i=1}^{N}\left\Vert \frac{1}{N}\sum_{j=1}^{N}\nabla_{11}^{2}U\left(\boldsymbol{\zeta}_{1,ij}^{k},\bar{\boldsymbol{\theta}}_{j}^{k\epsilon}\right)\right\Vert _{{\rm op}}^{2}\left\Vert \boldsymbol{\delta}_{i}^{k}\right\Vert _{2}^{2}+\frac{1}{N}\sum_{i=1}^{N}\left(\frac{1}{N}\sum_{j=1}^{N}\left\Vert \nabla_{12}^{2}U\left(\boldsymbol{\theta}_{i}^{k},\boldsymbol{\zeta}_{4,ij}^{k}\right)\right\Vert _{{\rm op}}\left\Vert \boldsymbol{\delta}_{j}^{k}\right\Vert _{2}\right)^{2}\\
 & \quad\lesssim\mathscr{E}_{k}+\frac{1}{N}\sum_{i=1}^{N}\left(\frac{1}{N}\sum_{j=1}^{N}\kappa_{2}\left(\left\Vert \boldsymbol{\theta}_{i}^{k}\right\Vert _{2}+1\right)\left(\left\Vert \boldsymbol{\zeta}_{4,ij}^{k}\right\Vert _{2}+1\right)\left\Vert \boldsymbol{\delta}_{j}^{k}\right\Vert _{2}\right)^{2}\\
 & \quad\lesssim\mathscr{E}_{k}+\frac{1}{N}\sum_{i=1}^{N}\left(\frac{1}{N}\sum_{j=1}^{N}\kappa_{2}\left(\left\Vert \bar{\boldsymbol{\theta}}_{i}^{0}\right\Vert _{2}+\left\Vert \boldsymbol{\delta}_{i}^{k}\right\Vert _{2}+1\right)\left(\left\Vert \bar{\boldsymbol{\theta}}_{j}^{0}\right\Vert _{2}+\left\Vert \boldsymbol{\delta}_{j}^{k}\right\Vert _{2}+1\right)\left\Vert \boldsymbol{\delta}_{j}^{k}\right\Vert _{2}\right)^{2}\\
 & \quad\lesssim\mathscr{E}_{k}+\frac{\kappa_{2}^{2}}{N}\sum_{i=1}^{N}\left(\left\Vert \bar{\boldsymbol{\theta}}_{i}^{0}\right\Vert _{2}^{2}+\left\Vert \boldsymbol{\delta}_{i}^{k}\right\Vert _{2}^{2}+1\right)\left(\frac{1}{N}\sum_{j=1}^{N}\left(\left\Vert \bar{\boldsymbol{\theta}}_{j}^{0}\right\Vert _{2}+\left\Vert \boldsymbol{\delta}_{j}^{k}\right\Vert _{2}+1\right)\left\Vert \boldsymbol{\delta}_{j}^{k}\right\Vert _{2}\right)^{2}\\
 & \quad\lesssim\mathscr{E}_{k}+\kappa_{2}^{2}\left(\mathscr{E}_{k}+1\right)\left(\left(\sqrt{\frac{1}{N}\sum_{j=1}^{N}\left\Vert \bar{\boldsymbol{\theta}}_{j}^{0}\right\Vert _{2}^{2}}+1\right)\sqrt{\mathscr{E}_{k}}+\mathscr{E}_{k}\right)^{2}\\
 & \quad\lesssim\kappa_{2}^{2}\mathscr{E}_{k},
\end{align*}
recalling $\mathscr{E}_{k}\leq\gamma_{{\rm st}}\leq1$ for $k\leq T_{{\rm st}}/\epsilon$.
We thus obtain from the bounds that on the event $\mathsf{Ev}$, with
probability at least $1-\Xi\left(N;T,\kappa_{6}\right)$,
\[
\frac{\epsilon^{2}}{N}\sum_{k=0}^{t/\epsilon-1}\sum_{i=1}^{N}\left\Vert \boldsymbol{E}_{3,i}^{k}\right\Vert _{2}^{2}\lesssim\epsilon^{2}\kappa_{2}^{2}\sum_{k=0}^{t/\epsilon-1}\mathscr{E}_{k},
\]
for all $t\in\mathbb{N}\epsilon\cap\left[0,T\wedge T_{{\rm st}}\right]$.
This completes the proof.
\end{proof}

\section{Application to autoencoders\label{sec:Autoencoders}}

We consider a weight-tied autoencoder of the form (\ref{eq:ae_simplified}).
In particular, it fits into our framework of two-layer neural networks
(\ref{eq:two_layers_nn}) by the following choice of activation function:
\begin{equation}
\sigma_{*}\left(\boldsymbol{x};\kappa\boldsymbol{\theta}\right)=\kappa\boldsymbol{\theta}\sigma\left(\left\langle \kappa\boldsymbol{\theta},\boldsymbol{x}\right\rangle \right),\qquad\kappa=\sqrt{d},\label{eq:weight-tied-AE}
\end{equation}
where $\boldsymbol{x},\boldsymbol{\theta}\in\mathbb{R}^{d}$, and
$\mathfrak{Dim}=\left(d,d,d\right)$ in this setting ($D_{{\rm in}}=D_{{\rm out}}=D=d$).
The rationale for the choice $\kappa=\sqrt{d}$ has been discussed
in Section \ref{subsec:Autoencoder-example}. The regularization $\Lambda$
represents a $\ell_{2}$-regularized autoencoder: $\Lambda\left(\boldsymbol{\theta},\boldsymbol{z}\right)=\lambda\left\Vert \boldsymbol{\theta}\right\Vert _{2}^{2}$,
where $\lambda\geq0$. Here we allow $\lambda$ to be dependent on
$\mathfrak{Dim}$, but impose a constraint that $\lambda\leq C$ for
some immaterial constant $C$ that is independent of $\mathfrak{Dim}$.
For simplicity, we have chosen a constant learning rate schedule $\xi\left(\cdot\right)=1$
in our autoencoder application; the extension to bounded Lipschitz
$\xi$ is straightforwards. We consider the following two scenarios:
\begin{enumerate}[{label=\textbf{[S.\arabic*]},ref=[S.\arabic*]}]
\item \label{enu:ReLU_setting}\textsl{(Setting with ReLU activation) }The
data $\boldsymbol{y}=\boldsymbol{x}\in\mathbb{R}^{d}$ follows a Gaussian
distribution with the following mean and covariance:
\[
\mathbb{E}\left\{ \boldsymbol{x}\right\} =\boldsymbol{0},\qquad\mathbb{E}\left\{ \boldsymbol{x}\boldsymbol{x}^{\top}\right\} =\frac{1}{d}\boldsymbol{R}{\rm diag}\left(\Sigma_{1}^{2},...,\Sigma_{d}^{2}\right)\boldsymbol{R}^{\top},
\]
for $\Sigma_{1}\geq...\geq\Sigma_{d}$ and $\boldsymbol{R}$ an orthogonal
matrix. In this case, let us define
\[
\boldsymbol{\Sigma}=\boldsymbol{R}{\rm diag}\left(\Sigma_{1},...,\Sigma_{d}\right)\boldsymbol{R}^{\top}.
\]
We assume $\sigma_{\min}\left(\boldsymbol{\Sigma}\right)=\Sigma_{d}\geq C\kappa_{*}$
and $\left\Vert \boldsymbol{\Sigma}\right\Vert _{2}=\Sigma_{1}\leq C$.
Here $\kappa_{*}>0$ depends uniquely on $d$ (and in general, may
decay with increasing $d$), and of course, $\kappa_{*}\leq C$. The
activation $\sigma$ is the ReLU: $\sigma\left(a\right)=\max\left(0,a\right)$.
\item \label{enu:bdd_act_setting}\textsl{(Setting with bounded activation)
}The data $\boldsymbol{y}=\boldsymbol{x}\in\mathbb{R}^{d}$ follows
a Gaussian distribution with the following mean and covariance:
\[
\mathbb{E}\left\{ \boldsymbol{x}\right\} =\boldsymbol{0},\qquad\mathbb{E}\left\{ \boldsymbol{x}\boldsymbol{x}^{\top}\right\} =\frac{1}{d}{\rm diag}(\underbrace{\Sigma_{1}^{2},...,\Sigma_{1}^{2}}_{d_{1}\text{ entries}},\underbrace{\Sigma_{2}^{2},...,\Sigma_{2}^{2}}_{d_{2}\text{ entries}}),
\]
where $0<C\leq\Sigma_{1},\Sigma_{2}\leq C$, and $d_{1}=\alpha d$,
$d_{2}=\left(1-\alpha\right)d$ for some $\alpha\in\left(0,1\right)$
such that $d_{1}$ and $d_{2}$ are positive integers, and $\alpha$
does not depend on $\mathfrak{Dim}$. In this case, let us define
\[
\boldsymbol{\Sigma}={\rm diag}(\underbrace{\Sigma_{1},...,\Sigma_{1}}_{d_{1}\text{ entries}},\underbrace{\Sigma_{2},...,\Sigma_{2}}_{d_{2}\text{ entries}}).
\]
The activation $\sigma$ is bounded and thrice differentiable with
bounded first two derivatives $\left\Vert \sigma\right\Vert _{\infty},\left\Vert \sigma'\right\Vert _{\infty},\left\Vert \sigma''\right\Vert _{\infty}\leq C$,
such that there exist an anti-derivative $\hat{\sigma}_{2}$ of $\left|\sigma''\right|$
with $\left\Vert \hat{\sigma}_{2}\right\Vert _{\infty}\leq C$ and
an anti-derivative $\hat{\sigma}_{3}$ of $\left|\sigma'''\right|$
with $\left\Vert \hat{\sigma}_{3}\right\Vert _{\infty}\leq C$. For
simplicity, we assume $d_{1},d_{2}>16$. The analysis could be extended
to scenarios where $\boldsymbol{\Sigma}$ is non-diagonal and the
spectrum of $\boldsymbol{\Sigma}$ contains more than two blocks.
\end{enumerate}
In setting \ref{enu:ReLU_setting}, we also recall the two-staged
process as described in Result \ref{res:ReLU_setting_2stage_simplified}:
\begin{enumerate}
\item Train an autoencoder with activation of the form (\ref{eq:weight-tied-AE})
and $N$ neurons for $t/\epsilon$ SGD steps.
\item Form a set of $M$ vectors $\left(\boldsymbol{w}_{i}^{t}\right)_{i\leq M}$
such that for each $i\in\left[M\right]$, $\boldsymbol{w}_{i}^{t}=\boldsymbol{w}_{i}^{t}\left(N,t,\epsilon\right)$
is drawn independently at random from the set of $N$ neurons $\left(\boldsymbol{\theta}_{i}^{t/\epsilon}\right)_{i\leq N}$.
Construct a new autoencoder with $M$ neurons $\left(\boldsymbol{w}_{i}^{t}\right)_{i\leq M}$:
\begin{equation}
\hat{\boldsymbol{x}}_{M}^{t}\left(\boldsymbol{x}\right)\equiv\hat{\boldsymbol{x}}_{M}^{t}\left(\boldsymbol{x};N,t,\epsilon\right)=\frac{1}{M}\sum_{i=1}^{M}\kappa\boldsymbol{w}_{i}^{t}\sigma\left(\left\langle \kappa\boldsymbol{w}_{i}^{t},\boldsymbol{x}\right\rangle \right).\label{eq:two-stage}
\end{equation}
\end{enumerate}
In the following, we shall state the main results for each of the
settings (Theorems \ref{thm:ReLU_setting} and \ref{thm:bdd_act_setting}
in Sections \ref{subsec:relu_main} and \ref{subsec:bdd_main} respectively).
Their proofs, as well as the proofs for auxiliary results, are presented
in Sections \ref{subsec:Proof_ReLU_act_thm}-\ref{subsec:Proof_bdd_act_aux}.

\subsection{Setting with ReLU activation: Main result\label{subsec:relu_main}}

We state the main result for the setting with ReLU activation (setting
\ref{enu:ReLU_setting}).
\begin{thm}
\label{thm:ReLU_setting}Consider setting \ref{enu:ReLU_setting}.
Suppose that the initialization $\rho^{0}=\mathsf{N}\left(\boldsymbol{0},r_{0}^{2}\boldsymbol{I}/d\right)$
for a non-negative constant $r_{0}\leq C$ and we generate the SGD
initialization $\Theta^{0}=\left(\boldsymbol{\theta}_{i}^{0}\right)_{i\leq N}\sim_{{\rm i.i.d.}}\rho^{0}$.
Given $\delta>1$, $\epsilon_{0}\in\left(0,1\right)$ and a finite
$T\in\mathbb{N}\epsilon$, assume
\[
\frac{d^{6}\delta^{2}}{\kappa_{*}^{2}}\epsilon\lesssim1,\qquad\left(\delta^{2}+\log^{5}\left(\frac{NT}{\epsilon}+1\right)\right)\frac{d^{4}}{\kappa_{*}^{2}N}\lesssim1,
\]
and define
\begin{align*}
\mathsf{err}\left(N,\epsilon,\delta\right) & =\left(\delta^{2}+\log^{5}\left(\frac{NT}{\epsilon}+1\right)\right)\frac{d^{2}}{N}+\sqrt{\epsilon}\kappa_{*}\delta+\epsilon d^{4}\delta,\\
\mathsf{prob}\left(N,\delta,\epsilon_{0}\right) & =\frac{1}{\delta^{2}}+\exp\left(Cd\log\left(\frac{\sqrt{d}}{\kappa_{*}}+e\right)-C\frac{N\kappa_{*}^{2}}{d}\right)+\exp\left(-N^{1/8}\right)+\frac{NT}{\epsilon}\exp\left(-C\epsilon_{0}^{1/3}\left(\frac{N}{d^{2}}\right)^{1/6}\right).
\end{align*}
The following statements hold:

\paragraph*{Properties of trained autoencoders.}

For any $1$-Lipschitz function $\phi:\;\mathbb{R}^{d}\to\mathbb{R}$,
with probability at least $1-C\mathsf{prob}\left(N,\delta,\epsilon_{0}\right)$,
the following properties hold:
\[
\max_{t\in\mathbb{N}\epsilon\cap\left[0,T\right]}\left|\frac{1}{N}\sum_{i=1}^{N}\phi\left(\boldsymbol{\theta}_{i}^{t/\epsilon}\right)-\mathbb{E}_{\boldsymbol{z}}\left\{ \phi\left(\boldsymbol{R}{\rm diag}\left(r_{1,t},...,r_{d,t}\right)\boldsymbol{z}\right)\right\} \right|\lesssim\epsilon_{0}+\sqrt{\mathsf{err}\left(N,\epsilon,\delta\right)},
\]
\[
\max_{t\in\mathbb{N}\epsilon\cap\left[0,T\right]}\left|{\cal R}\left(\rho_{N}^{t/\epsilon}\right)-\frac{1}{2d}\sum_{i=1}^{d}\Sigma_{i}^{2}\left(1-\frac{1}{2}r_{i,t}^{2}\right)^{2}\right|\lesssim d\sqrt{\mathsf{err}\left(N,\epsilon,\delta\right)}+\epsilon_{0},
\]
Here $\boldsymbol{z}\sim\mathsf{N}\left(\boldsymbol{0},\boldsymbol{I}_{d}/d\right)$
and we define
\[
r_{i,t}=\sqrt{\frac{\Sigma_{i}^{2}-2\lambda}{0.5r_{0}^{2}\Sigma_{i}^{2}-\left(0.5r_{0}^{2}\Sigma_{i}^{2}-\Sigma_{i}^{2}+2\lambda\right)\exp\left\{ -2\left(\Sigma_{i}^{2}-2\lambda\right)t\right\} }}r_{0}.
\]
(In the above, the immaterial constants $C$ may depend on $T$ and
$r_{0}$, but not $N$, $\epsilon$, $d$, $\delta$ or $\epsilon_{0}$.)

\paragraph*{Two-staged process.}

Given a positive integer $M$, perform the two-staged process in (\ref{eq:two-stage})
to obtain a new autoencoder with $M$ neurons $\left(\boldsymbol{w}_{i}^{t}\right)_{i\leq M}$.
Suppose that $M=\mu d$ for some $\mu>0$. We then have, for $\epsilon_{0}\in\left(0,1\right)$
and $t\geq0$,
\[
\lim_{\epsilon\downarrow0}\lim_{N\to\infty}\mathbb{P}\left\{ \left|{\cal R}\left(\nu_{M}^{t}\right)-{\cal R}_{*}^{t}\right|\geq\epsilon_{0}+\frac{C}{\sqrt{\mu M}}\right\} \leq C\exp\left(-C\epsilon_{0}^{1/6}\left(1+\frac{1}{\mu}\right)^{-1/6}M^{1/12}\right).
\]
where $\nu_{M}^{t}=\left(1/M\right)\cdot\sum_{i=1}^{M}\delta_{\boldsymbol{w}_{i}^{t}}$
and
\[
{\cal R}_{*}^{t}=\frac{1}{2d}\sum_{i=1}^{d}\Sigma_{i}^{2}\left(1-\frac{1}{2}r_{i,t}^{2}\right)^{2}+\frac{1}{4\mu d^{2}}\sum_{i=1}^{d}r_{i,t}^{2}\sum_{i=1}^{d}r_{i,t}^{2}\Sigma_{i}^{2}.
\]
(In the above, the immaterial constants $C$ may depend on $r_{0}$,
but not $M$, $d$, $\delta$, $\epsilon_{0}$, $t$ or $\mu$.)

\end{thm}

\begin{rem}
In Theorem \ref{thm:ReLU_setting}, a more quantitative statement
for the two-staged process could be made. Here we opt for the limits
$N\to\infty$, $\epsilon\to0$ for ease of presentation.
\end{rem}

\subsection{Setting with bounded activation: Main result\label{subsec:bdd_main}}

Given an activation $\sigma$, we define $q_{1}$ and $q_{2}$ on
the domain $\left(a,b\right)\in[0,\infty)\times[0,\infty)$:
\begin{align}
q_{1}\left(a,b\right) & =\mathbb{E}_{\boldsymbol{\omega}}\left\{ \kappa\omega_{11}\sigma\left(\kappa a\omega_{11}+\kappa b\omega_{21}\right)\right\} ,\label{eq:2nd_setting_q1}\\
q_{2}\left(a,b\right) & =\mathbb{E}_{\boldsymbol{\omega}}\left\{ \kappa\omega_{21}\sigma\left(\kappa a\omega_{11}+\kappa b\omega_{21}\right)\right\} ,\label{eq:2nd_setting_q2}
\end{align}
in which $\boldsymbol{\omega}_{1}\sim\text{Unif}\left(\mathbb{S}^{d_{1}-1}\right)$
and $\boldsymbol{\omega}_{2}\sim\text{Unif}\left(\mathbb{S}^{d_{2}-1}\right)$
independently, and $\omega_{11}$ and $\omega_{21}$ are their respective
first entries. From here onwards, we shall use $\boldsymbol{\omega}_{1}$
and $\boldsymbol{\omega}_{2}$ to indicate these respective random
vectors. For a vector $\boldsymbol{u}\in\mathbb{R}^{d}$, we shall
use $\boldsymbol{u}_{\left[1\right]}$ to denote a $d_{1}$-dimensional
vector of its first $d_{1}$ entries and $\boldsymbol{u}_{\left[2\right]}$
to denote a $d_{2}$-dimensional vector of its last $d_{2}$ entries.

We state the main result for the setting with bounded activation (setting
\ref{enu:bdd_act_setting}).
\begin{thm}
\label{thm:bdd_act_setting}Consider setting \ref{enu:bdd_act_setting}.
Suppose that the initialization $\rho^{0}=\mathsf{N}\left(\boldsymbol{0},r_{0}^{2}\boldsymbol{I}/d\right)$
for a non-negative constant $r_{0}\leq C$ and we generate the SGD
initialization $\Theta^{0}=\left(\boldsymbol{\theta}_{i}^{0}\right)_{i\leq N}\sim_{{\rm i.i.d.}}\rho^{0}$.
Given $\delta>1$, $\epsilon_{0}\in\left(0,1\right)$ and a finite
$T\in\mathbb{N}\epsilon$, assume
\[
d^{6}\delta^{2}\epsilon\lesssim1,\qquad\left(\delta^{2}+\log^{5}\left(\frac{NT}{\epsilon}+1\right)\right)\frac{d^{4}}{N}\lesssim1,
\]
and define
\begin{align*}
\mathsf{err}\left(N,\epsilon,\delta\right) & =\left(\delta^{2}+\log^{5}\left(\frac{NT}{\epsilon}+1\right)\right)\frac{d^{2}}{N}+\sqrt{\epsilon}\delta+\epsilon d^{4}\delta,\\
\mathsf{prob}\left(N,\delta,\epsilon_{0}\right) & =\frac{1}{\delta^{2}}+\exp\left(Cd\log\left(d\sqrt{N}+e\right)-CN/d^{2}\right)+\exp\left(-N^{1/8}\right)+\frac{NT}{\epsilon}\exp\left(-C\epsilon_{0}^{1/3}\left(\frac{N}{d^{2}}\right)^{1/6}\right).
\end{align*}
Let us also define two non-negative (random) processes $\left(r_{1,t}\right)_{t\geq0}$
and $\left(r_{2,t}\right)_{t\geq0}$ which satisfy the following self-contained
(randomly initialized) ODEs:
\begin{align}
\frac{{\rm d}}{{\rm d}t}r_{j,t} & =-\mathbb{E}_{\chi}\left\{ \Delta_{j}\left(\chi,\rho_{r}^{t}\right)\left[q_{j}\left(\chi_{1}r_{1,t},\chi_{2}r_{2,t}\right)+\chi_{j}r_{j,t}\partial_{j}q_{j}\left(\chi_{1}r_{1,t},\chi_{2}r_{2,t}\right)\right]\right\} \nonumber \\
 & \qquad-\mathbb{E}_{\chi}\left\{ \Delta_{\neg j}\left(\chi,\rho_{r}^{t}\right)\chi_{j}r_{\neg j,t}\partial_{j}q_{\neg j}\left(\chi_{1}r_{1,t},\chi_{2}r_{2,t}\right)\right\} -2\lambda r_{j,t},\nonumber \\
\rho_{r}^{t} & ={\rm Law}\left(r_{1,t},r_{2,t}\right),\label{eq:2nd_setting_ODE_r}
\end{align}
for $j=1,2$, and $\neg j=2$ if $j=1$, $\neg j=1$ if $j=2$. In
the above:
\begin{itemize}
\item $q_{1}$ and $q_{2}$ are functions defined in Eq. (\ref{eq:2nd_setting_q1})
and (\ref{eq:2nd_setting_q2}),
\item the initialization is $r_{1,0}\stackrel{{\rm d}}{=}r_{0}d^{-1/2}Z_{1}$
and $r_{2,0}\stackrel{{\rm d}}{=}r_{0}d^{-1/2}Z_{2}$ independently,
with $Z_{1}$ and $Z_{2}$ being respectively $\chi$-random variables
of degrees of freedom $d_{1}$ and $d_{2}$,
\item $\chi_{1}\stackrel{{\rm d}}{=}\Sigma_{1}d^{-1/2}Z_{1}$ and $\chi_{2}\stackrel{{\rm d}}{=}\Sigma_{2}d^{-1/2}Z_{2}$
are two independent random variables, which are also independent of
everything else, and $\chi=\left(\chi_{1},\chi_{2}\right)$,
\item the quantity $\Delta_{j}\left(\chi,\rho_{r}^{t}\right)$ is defined
as:
\[
\Delta_{j}\left(\chi,\rho_{r}^{t}\right)=\int\bar{r}_{j}q_{j}\left(\chi_{1}\bar{r}_{1},\chi_{2}\bar{r}_{2}\right)\rho_{r}^{t}\left({\rm d}\bar{r}_{1},{\rm d}\bar{r}_{2}\right)-\chi_{j},\qquad j=1,2.
\]
\end{itemize}
Then for any $1$-Lipschitz function $\phi:\;\mathbb{R}^{d}\to\mathbb{R}$,
with probability at least $1-C\mathsf{prob}\left(N,\delta,\epsilon_{0}\right)$,
\[
\max_{t\in\mathbb{N}\epsilon\cap\left[0,T\right]}\left|\frac{1}{N}\sum_{i=1}^{N}\phi\left(\boldsymbol{\theta}_{i}^{t/\epsilon}\right)-\int\mathbb{E}_{\boldsymbol{\omega}}\left\{ \phi\left(\left(\bar{r}_{1}\boldsymbol{\omega}_{1},\bar{r}_{2}\boldsymbol{\omega}_{2}\right)\right)\right\} \rho_{r}^{t}\left({\rm d}\bar{r}_{1},{\rm d}\bar{r}_{2}\right)\right|\lesssim\epsilon_{0}+\sqrt{\mathsf{err}\left(N,\epsilon,\delta\right)},
\]
\[
\max_{t\in\mathbb{N}\epsilon\cap\left[0,T\right]}\left|{\cal R}\left(\rho_{N}^{t/\epsilon}\right)-\mathbb{E}_{\chi}\left\{ \frac{1}{2}\sum_{j\in\left\{ 1,2\right\} }\Delta_{j}\left(\chi,\rho_{r}^{t}\right)^{2}\right\} \right|\lesssim d\sqrt{\mathsf{err}\left(N,\epsilon,\delta\right)}+\epsilon_{0}.
\]
(In the above, the immaterial constants $C$ may depend on $T$ and
$r_{0}$, but not $N$, $\epsilon$, $d$, $\delta$ or $\epsilon_{0}$.)
\end{thm}

\subsection{Setting with ReLU activation: Proof of Theorem \ref{thm:ReLU_setting}\label{subsec:Proof_ReLU_act_thm}}

We prove Theorem \ref{thm:ReLU_setting}. Our proof uses several auxiliary
results, which are stated and proven in Section \ref{subsec:Proof_relu_act_aux}.
\begin{proof}[Proof of Theorem \ref{thm:ReLU_setting}]
We decompose the proof into several parts.

\subsubsection*{Proof of the first statement: Properties of trained autoencoders.}

The first statement follows from Theorem \ref{thm:propChaos}, Propositions
\ref{prop:1st_setting_grow_bound}, \ref{prop:1st_setting_grow_bound_Wrho},
\ref{prop:1st_setting_op_bound}, \ref{prop:ReLU_setting_ODE}, \ref{prop:1st_setting_F_bound}
and \ref{prop:1st_setting_U_bound}. In particular, we have that
\[
\hat{\boldsymbol{\theta}}^{t}=\boldsymbol{R}{\rm diag}\left(\frac{r_{1,t}}{r_{0}},...,\frac{r_{d,t}}{r_{0}}\right)\boldsymbol{R}^{\top}\hat{\boldsymbol{\theta}}^{0},\qquad\rho^{t}=\mathsf{N}\left(\boldsymbol{0},\boldsymbol{R}{\rm diag}\left(r_{1,t}^{2},...,r_{d,t}^{2}\right)\boldsymbol{R}^{\top}/d\right)
\]
form the (weakly) unique solution to the ODE (\ref{eq:ODE}) with
initialization $\hat{\boldsymbol{\theta}}^{0}\sim\rho^{0}$ and $\rho^{0}$.
We also observe that 
\[
r_{i,t}\leq\max\left\{ r_{0},\sqrt{2\max\left(1-2\lambda/\Sigma_{i}^{2},0\right)}\right\} \leq\max\left\{ r_{0},\sqrt{2}\right\} \leq C,
\]
for all $i\in\left[d\right]$ and all $t\geq0$. Furthermore we have
that
\[
\left|\frac{{\rm d}}{{\rm d}t}r_{i,t}\right|=r_{i,t}\left|0.5\Sigma_{i}^{2}r_{i,t}^{2}-\left(\Sigma_{i}^{2}-2\lambda\right)\right|\leq C,
\]
for all $i\in\left[d\right]$ and all $t\geq0$. These verify Assumption
\ref{enu:Assump_ODE} and allow Propositions \ref{prop:1st_setting_grow_bound},
\ref{prop:1st_setting_grow_bound_Wrho} and \ref{prop:1st_setting_U_bound}
to verify Assumptions \ref{enu:Assump_growth} and \ref{enu:Assump_nabla11U}.
Finally, by Stein's lemma, we have:
\[
\int\kappa\boldsymbol{\theta}\sigma\left(\left\langle \kappa\boldsymbol{\theta},\boldsymbol{x}\right\rangle \right)\rho^{t}\left({\rm d}\boldsymbol{\theta}\right)=\frac{1}{2}\boldsymbol{R}{\rm diag}\left(r_{1,t}^{2},...,r_{d,t}^{2}\right)\boldsymbol{R}^{\top}\boldsymbol{x},
\]
and therefore,
\begin{align*}
{\cal R}\left(\rho^{t}\right) & =\mathbb{E}_{{\cal P}}\left\{ \frac{1}{2}\left\Vert \boldsymbol{x}-\int\kappa\boldsymbol{\theta}\sigma\left(\left\langle \kappa\boldsymbol{\theta},\boldsymbol{x}\right\rangle \right)\rho^{t}\left({\rm d}\boldsymbol{\theta}\right)\right\Vert _{2}^{2}\right\} \\
 & =\mathbb{E}_{{\cal P}}\left\{ \frac{1}{2}\left\Vert \boldsymbol{x}-\frac{1}{2}\boldsymbol{R}{\rm diag}\left(r_{1,t}^{2},...,r_{d,t}^{2}\right)\boldsymbol{R}^{\top}\boldsymbol{x}\right\Vert _{2}^{2}\right\} \\
 & =\frac{1}{2d}\sum_{i=1}^{d}\Sigma_{i}^{2}\left(1-\frac{1}{2}r_{i,t}^{2}\right)^{2}.
\end{align*}
This concludes the proof of the first statement.

\subsubsection*{Proof of the second statement: Two-staged process.}

We let $\boldsymbol{z}_{i}=\left(\sqrt{d}/r_{0}\right)\boldsymbol{R}^{\top}\boldsymbol{\theta}_{i}^{0}$
and hence $\left(\boldsymbol{z}_{i}\right)_{i\leq M}\sim_{{\rm i.i.d.}}\mathsf{N}\left(\boldsymbol{0},\boldsymbol{I}_{d}\right)$.
We let $\bar{\nu}_{M}^{t}$ denote the empirical distribution of $\left(\boldsymbol{R}\boldsymbol{D}_{t}\boldsymbol{z}_{i}/\sqrt{d}\right)_{i\leq M}$
for $\boldsymbol{D}_{t}={\rm diag}\left(r_{1,t},...,r_{d,t}\right)$.
By Theorem \ref{thm:propChaos} and Proposition \ref{prop:ReLU_setting_ODE},
we have that for any $\delta>0$,
\[
\lim_{\epsilon\downarrow0}\lim_{N\to\infty}\mathbb{P}\left\{ \left|{\cal R}\left(\nu_{M}^{t}\right)-{\cal R}\left(\bar{\nu}_{M}^{t}\right)\right|\geq\delta\right\} \leq Ce^{-M}.
\]
We claim that for all $t\geq0$,
\[
\left|\mathbb{E}\left\{ {\cal R}\left(\bar{\nu}_{M}^{t}\right)\right\} -{\cal R}_{*}^{t}\right|\leq C\frac{\sqrt{d}}{M},
\]
and for $\delta\in\left(0,1\right)$,
\[
\mathbb{P}\left\{ \left|{\cal R}\left(\bar{\nu}_{M}^{t}\right)-\mathbb{E}\left\{ {\cal R}\left(\bar{\nu}_{M}^{t}\right)\right\} \right|\geq\delta\right\} \leq C\exp\left(-C\delta^{1/6}\left(1+\sqrt{d/M}\right)^{-1/6}M^{1/12}\right).
\]
Using these claims, we then obtain for $\delta\in\left(0,1\right)$
and all $t\geq0$,
\[
\lim_{\epsilon\downarrow0}\lim_{N\to\infty}\mathbb{P}\left\{ \left|{\cal R}\left(\nu_{M}^{t}\right)-{\cal R}_{*}^{t}\right|\geq\delta+C\frac{\sqrt{d}}{M}\right\} \leq C\exp\left(-C\delta^{1/6}\left(1+\sqrt{d/M}\right)^{-1/6}M^{1/12}\right).
\]
Hence we are left with verifying the claims. Before we proceed, let
$\boldsymbol{Z}=\left(\boldsymbol{z}_{1},...,\boldsymbol{z}_{M}\right)^{\top}\in\mathbb{R}^{M\times d}$.
Then:
\begin{align*}
{\cal R}\left(\bar{\nu}_{M}^{t}\right) & =\mathbb{E}_{{\cal P}}\left\{ \frac{1}{2}\left\Vert \boldsymbol{x}-\frac{1}{M}\boldsymbol{R}\boldsymbol{D}_{t}\boldsymbol{Z}^{\top}\sigma\left(\boldsymbol{Z}\boldsymbol{D}_{t}\boldsymbol{R}^{\top}\boldsymbol{x}\right)\right\Vert _{2}^{2}\right\} \\
 & =\mathbb{E}_{\boldsymbol{u}}\left\{ \frac{1}{2}\left\Vert \boldsymbol{u}-\frac{1}{M}\boldsymbol{D}_{t}\boldsymbol{Z}^{\top}\sigma\left(\boldsymbol{Z}\boldsymbol{D}_{t}\boldsymbol{u}\right)\right\Vert _{2}^{2}\right\} \\
 & =\frac{1}{2}\mathbb{E}_{\boldsymbol{u}}\left\{ \left\Vert \boldsymbol{u}\right\Vert _{2}^{2}\right\} -\frac{1}{M}\mathbb{E}_{\boldsymbol{u}}\left\{ \left\langle \boldsymbol{u},\boldsymbol{D}_{t}\boldsymbol{Z}^{\top}\sigma\left(\boldsymbol{Z}\boldsymbol{D}_{t}\boldsymbol{u}\right)\right\rangle \right\} +\frac{1}{2M^{2}}\mathbb{E}_{\boldsymbol{u}}\left\{ \left\Vert \boldsymbol{D}_{t}\boldsymbol{Z}^{\top}\sigma\left(\boldsymbol{Z}\boldsymbol{D}_{t}\boldsymbol{u}\right)\right\Vert _{2}^{2}\right\} \\
 & \equiv\frac{1}{2d}\left\Vert \boldsymbol{D}_{\boldsymbol{\Sigma}}\right\Vert _{{\rm F}}^{2}-A_{1}+\frac{1}{2}A_{2},
\end{align*}
for $\boldsymbol{u}=\boldsymbol{R}^{\top}\boldsymbol{x}\sim\mathsf{N}\left(\boldsymbol{0},\boldsymbol{D}_{\boldsymbol{\Sigma}}^{2}/d\right)$
and $\boldsymbol{D}_{\boldsymbol{\Sigma}}={\rm diag}\left(\Sigma_{1},...,\Sigma_{d}\right)$.
We recall that $\left\Vert \boldsymbol{D}_{t}\right\Vert _{2}\leq C$
since $r_{i,t}\leq C$ for any $i\in\left[d\right]$ and $t\geq0$.

\paragraph{Step 1 - Calculation of $\mathbb{E}\left\{ {\cal R}\left(\bar{\nu}_{M}^{t}\right)\right\} $.}

We compute $\mathbb{E}\left\{ {\cal R}\left(\bar{\nu}_{M}^{t}\right)\right\} $.
By Stein's lemma, we have:
\[
\mathbb{E}\left\{ A_{1}\right\} =\mathbb{E}_{\boldsymbol{u}}\left\{ \left\langle \boldsymbol{u},\boldsymbol{D}_{t}\mathbb{E}_{\boldsymbol{Z}}\left\{ \frac{1}{M}\boldsymbol{Z}^{\top}\sigma\left(\boldsymbol{Z}\boldsymbol{D}_{t}\boldsymbol{u}\right)\right\} \right\rangle \right\} =\frac{1}{2}\mathbb{E}_{\boldsymbol{u}}\left\{ \left\langle \boldsymbol{u},\boldsymbol{D}_{t}^{2}\boldsymbol{u}\right\rangle \right\} =\frac{1}{2d}\left\Vert \boldsymbol{D}_{t}\boldsymbol{D}_{\boldsymbol{\Sigma}}\right\Vert _{{\rm F}}^{2}.
\]
Next, notice that for a fixed $\boldsymbol{u}$ and $\boldsymbol{a}=\boldsymbol{Z}\boldsymbol{D}_{t}\boldsymbol{u}\sim\mathsf{N}\left(\boldsymbol{0},\left\Vert \boldsymbol{D}_{t}\boldsymbol{u}\right\Vert _{2}^{2}\boldsymbol{I}_{M}\right)$,
\[
\left(\boldsymbol{a},\boldsymbol{Z}\right)\stackrel{{\rm d}}{=}\left(\boldsymbol{a},\tilde{\boldsymbol{Z}}{\rm Proj}_{\boldsymbol{D}_{t}\boldsymbol{u}}^{\perp}+\frac{\boldsymbol{a}\boldsymbol{u}^{\top}\boldsymbol{D}_{t}}{\left\Vert \boldsymbol{D}_{t}\boldsymbol{u}\right\Vert _{2}^{2}}\right),
\]
where $\tilde{\boldsymbol{Z}}\in\mathbb{R}^{M\times d}$ comprises
of i.i.d. $\mathsf{N}\left(0,1\right)$ entries independent of $\boldsymbol{a}$.
We apply this observation:
\begin{align*}
\mathbb{E}\left\{ A_{2}\right\}  & =\frac{1}{M^{2}}\mathbb{E}_{\boldsymbol{u}}\left\{ \mathbb{E}_{\boldsymbol{Z}}\left\{ \left\Vert \boldsymbol{D}_{t}\boldsymbol{Z}^{\top}\sigma\left(\boldsymbol{Z}\boldsymbol{D}_{t}\boldsymbol{u}\right)\right\Vert _{2}^{2}\right\} \right\} \\
 & =\frac{1}{M^{2}}\mathbb{E}_{\boldsymbol{u}}\left\{ \mathbb{E}_{\boldsymbol{a},\tilde{\boldsymbol{Z}}}\left\{ \left\Vert \boldsymbol{D}_{t}\left({\rm Proj}_{\boldsymbol{D}_{t}\boldsymbol{u}}^{\perp}\tilde{\boldsymbol{Z}}^{\top}\sigma\left(\boldsymbol{a}\right)+\frac{\boldsymbol{D}_{t}\boldsymbol{u}}{\left\Vert \boldsymbol{D}_{t}\boldsymbol{u}\right\Vert _{2}^{2}}\left\langle \boldsymbol{a},\sigma\left(\boldsymbol{a}\right)\right\rangle \right)\right\Vert _{2}^{2}\right\} \right\} \\
 & \stackrel{\left(a\right)}{=}\frac{1}{M^{2}}\mathbb{E}_{\boldsymbol{u}}\left\{ \mathbb{E}_{\boldsymbol{a},\tilde{\boldsymbol{Z}}}\left\{ \left\Vert \boldsymbol{D}_{t}{\rm Proj}_{\boldsymbol{D}_{t}\boldsymbol{u}}^{\perp}\tilde{\boldsymbol{Z}}^{\top}\sigma\left(\boldsymbol{a}\right)\right\Vert _{2}^{2}\right\} \right\} +\frac{1}{M^{2}}\mathbb{E}_{\boldsymbol{u}}\left\{ \mathbb{E}_{\boldsymbol{a}}\left\{ \frac{\left\Vert \boldsymbol{D}_{t}^{2}\boldsymbol{u}\right\Vert _{2}^{2}\left\langle \boldsymbol{a},\sigma\left(\boldsymbol{a}\right)\right\rangle ^{2}}{\left\Vert \boldsymbol{D}_{t}\boldsymbol{u}\right\Vert _{2}^{4}}\right\} \right\} \\
 & \equiv A_{2,1}+A_{2,2},
\end{align*}
where step $\left(a\right)$ is because $\mathbb{E}\left\{ \tilde{\boldsymbol{Z}}\right\} =0$.
To compute $A_{2,2}$, recall that $\boldsymbol{a}\sim\mathsf{N}\left(\boldsymbol{0},\left\Vert \boldsymbol{D}_{t}\boldsymbol{u}\right\Vert _{2}^{2}\boldsymbol{I}_{M}\right)$
and that $\sigma$ is homogenous:
\begin{align*}
A_{2,2} & =\mathbb{E}_{\boldsymbol{u}}\left\{ \left\Vert \boldsymbol{D}_{t}^{2}\boldsymbol{u}\right\Vert _{2}^{2}\left(\frac{1}{M}\mathbb{E}_{g}\left\{ g^{2}\sigma\left(g\right)^{2}\right\} +\frac{M\left(M-1\right)}{M^{2}}\mathbb{E}_{g}\left\{ g\sigma\left(g\right)\right\} ^{2}\right)\right\} \\
 & =\left(\frac{1}{4d}+\frac{5}{Md}\right)\left\Vert \boldsymbol{D}_{t}^{2}\boldsymbol{D}_{\boldsymbol{\Sigma}}\right\Vert _{{\rm F}}^{2}=\frac{1}{4d}\left\Vert \boldsymbol{D}_{t}^{2}\boldsymbol{D}_{\boldsymbol{\Sigma}}\right\Vert _{{\rm F}}^{2}+O\left(\frac{1}{M}\right).
\end{align*}
To compute $A_{2,1}$, let $\tilde{\boldsymbol{z}}_{i}$ be the $i$-th
row of $\tilde{\boldsymbol{Z}}$ and $a_{i}$ be the $i$-th entry
of $\boldsymbol{a}$:
\begin{align*}
A_{2,1} & =\frac{1}{M^{2}}\mathbb{E}_{\boldsymbol{u}}\left\{ \mathbb{E}_{\boldsymbol{a},\tilde{\boldsymbol{Z}}}\left\{ \left\Vert \sum_{i=1}^{M}\boldsymbol{D}_{t}{\rm Proj}_{\boldsymbol{D}_{t}\boldsymbol{u}}^{\perp}\tilde{\boldsymbol{z}}_{i}\sigma\left(a_{i}\right)\right\Vert _{2}^{2}\right\} \right\} \\
 & =\frac{1}{M^{2}}\mathbb{E}_{\boldsymbol{u}}\left\{ \sum_{i=1}^{M}\mathbb{E}_{\tilde{\boldsymbol{Z}}}\left\{ \left\Vert \boldsymbol{D}_{t}\left(\boldsymbol{I}_{d}-{\rm Proj}_{\boldsymbol{D}_{t}\boldsymbol{u}}\right)\tilde{\boldsymbol{z}}_{i}\right\Vert _{2}^{2}\right\} \mathbb{E}_{\boldsymbol{a}}\left\{ \sigma\left(a_{i}\right)^{2}\right\} \right\} \\
 & =\frac{1}{M^{2}}\mathbb{E}_{\boldsymbol{u}}\left\{ \sum_{i=1}^{M}\mathbb{E}_{\tilde{\boldsymbol{Z}}}\left\{ \left\Vert \boldsymbol{D}_{t}\tilde{\boldsymbol{z}}_{i}\right\Vert _{2}^{2}-2\left\langle \boldsymbol{D}_{t}\tilde{\boldsymbol{z}}_{i},\boldsymbol{D}_{t}{\rm Proj}_{\boldsymbol{D}_{t}\boldsymbol{u}}\tilde{\boldsymbol{z}}_{i}\right\rangle +\left\Vert \boldsymbol{D}_{t}{\rm Proj}_{\boldsymbol{D}_{t}\boldsymbol{u}}\tilde{\boldsymbol{z}}_{i}\right\Vert _{2}^{2}\right\} \frac{1}{2}\left\Vert \boldsymbol{D}_{t}\boldsymbol{u}\right\Vert _{2}^{2}\right\} \\
 & \equiv A_{2,1,1}+A_{2,1,2}+A_{2,1,3}.
\end{align*}
We compute $A_{2,1,1}$:
\[
A_{2,1,1}=\frac{1}{2M}\left\Vert \boldsymbol{D}_{t}\right\Vert _{{\rm F}}^{2}\mathbb{E}_{\boldsymbol{u}}\left\{ \left\Vert \boldsymbol{D}^{t}\boldsymbol{u}\right\Vert _{2}^{2}\right\} =\frac{1}{2dM}\left\Vert \boldsymbol{D}_{t}\right\Vert _{{\rm F}}^{2}\left\Vert \boldsymbol{D}_{t}\boldsymbol{D}_{\boldsymbol{\Sigma}}\right\Vert _{{\rm F}}^{2}.
\]
We give a bound on $A_{2,1,2}$:
\begin{align*}
\left|A_{2,1,2}\right| & \leq\frac{1}{M^{2}}\mathbb{E}_{\boldsymbol{u}}\left\{ \sum_{i=1}^{M}\left\Vert \boldsymbol{D}_{t}\right\Vert _{{\rm op}}^{4}\mathbb{E}_{\tilde{\boldsymbol{Z}}}\left\{ \left\Vert {\rm Proj}_{\boldsymbol{D}_{t}\boldsymbol{u}}\tilde{\boldsymbol{z}}_{i}\right\Vert _{2}\left\Vert \tilde{\boldsymbol{z}}_{i}\right\Vert _{2}\right\} \left\Vert \boldsymbol{u}\right\Vert _{2}^{2}\right\} \\
 & \leq\frac{C}{M^{2}}\mathbb{E}_{\boldsymbol{u}}\left\{ \sum_{i=1}^{M}\sqrt{\mathbb{E}_{\tilde{\boldsymbol{Z}}}\left\{ \left\Vert {\rm Proj}_{\boldsymbol{D}_{t}\boldsymbol{u}}\tilde{\boldsymbol{z}}_{i}\right\Vert _{2}^{2}\right\} \mathbb{E}_{\tilde{\boldsymbol{Z}}}\left\{ \left\Vert \tilde{\boldsymbol{z}}_{i}\right\Vert _{2}^{2}\right\} }\left\Vert \boldsymbol{u}\right\Vert _{2}^{2}\right\} \\
 & =\frac{C}{M^{2}}\mathbb{E}_{\boldsymbol{u}}\left\{ \sum_{i=1}^{M}\sqrt{\mathbb{E}_{\tilde{\boldsymbol{Z}}}\left\{ \frac{\left\langle \boldsymbol{D}_{t}\boldsymbol{u},\tilde{\boldsymbol{z}}_{i}\right\rangle ^{2}}{\left\Vert \boldsymbol{D}_{t}\boldsymbol{u}\right\Vert _{2}^{2}}\right\} d}\left\Vert \boldsymbol{u}\right\Vert _{2}^{2}\right\} =\frac{C}{M^{2}}\mathbb{E}_{\boldsymbol{u}}\left\{ \sum_{i=1}^{M}\sqrt{d}\left\Vert \boldsymbol{u}\right\Vert _{2}^{2}\right\} \leq C\frac{\sqrt{d}}{M}.
\end{align*}
Likewise, we obtain a bound on $A_{2,1,3}$:
\begin{align*}
\left|A_{2,1,3}\right| & \leq\frac{1}{2M^{2}}\mathbb{E}_{\boldsymbol{u}}\left\{ \sum_{i=1}^{M}\left\Vert \boldsymbol{D}_{t}\right\Vert _{{\rm op}}^{4}\mathbb{E}_{\tilde{\boldsymbol{Z}}}\left\{ \left\Vert {\rm Proj}_{\boldsymbol{D}_{t}\boldsymbol{u}}\tilde{\boldsymbol{z}}_{i}\right\Vert _{2}^{2}\right\} \left\Vert \boldsymbol{u}\right\Vert _{2}^{2}\right\} \\
 & \leq\frac{C}{M^{2}}\mathbb{E}_{\boldsymbol{u}}\left\{ \sum_{i=1}^{M}\mathbb{E}_{\tilde{\boldsymbol{Z}}}\left\{ \frac{\left\langle \boldsymbol{D}_{t}\boldsymbol{u},\tilde{\boldsymbol{z}}_{i}\right\rangle ^{2}}{\left\Vert \boldsymbol{D}_{t}\boldsymbol{u}\right\Vert _{2}^{2}}\right\} \left\Vert \boldsymbol{u}\right\Vert _{2}^{2}\right\} =\frac{C}{M^{2}}\mathbb{E}_{\boldsymbol{u}}\left\{ \sum_{i=1}^{M}\left\Vert \boldsymbol{u}\right\Vert _{2}^{2}\right\} \leq\frac{C}{M}.
\end{align*}
Combining these calculations, we obtain an estimate on $\mathbb{E}\left\{ {\cal R}\left(\bar{\nu}_{M}^{t}\right)\right\} $:
\[
\left|\mathbb{E}\left\{ {\cal R}\left(\bar{\nu}_{M}^{t}\right)\right\} -{\cal R}_{*}^{t}\right|\leq C\frac{\sqrt{d}}{M},
\]
since, recalling the definition of ${\cal R}_{*}^{t}$,
\begin{align*}
 & \frac{1}{d}\left\Vert \boldsymbol{D}_{\boldsymbol{\Sigma}}\right\Vert _{{\rm F}}^{2}-\frac{1}{d}\left\Vert \boldsymbol{D}_{t}\boldsymbol{D}_{\boldsymbol{\Sigma}}\right\Vert _{{\rm F}}^{2}+\frac{1}{4d}\left\Vert \boldsymbol{D}_{t}^{2}\boldsymbol{D}_{\boldsymbol{\Sigma}}\right\Vert _{{\rm F}}^{2}+\frac{1}{2dM}\left\Vert \boldsymbol{D}_{t}\right\Vert _{{\rm F}}^{2}\left\Vert \boldsymbol{D}_{t}\boldsymbol{D}_{\boldsymbol{\Sigma}}\right\Vert _{{\rm F}}^{2}\\
 & \quad=\frac{1}{d}\sum_{i=1}^{d}\Sigma_{i}^{2}\left(1-\frac{1}{2}r_{i,t}^{2}\right)^{2}+\frac{1}{2dM}\sum_{i=1}^{d}r_{i,t}^{2}\sum_{i=1}^{d}r_{i,t}^{2}\Sigma_{i}^{2}=2{\cal R}_{*}^{t}.
\end{align*}

\paragraph*{Step 2 - Concentration.}

We show that ${\cal R}\left(\bar{\nu}_{M}^{t}\right)$ concentrates
around $\mathbb{E}\left\{ {\cal R}\left(\bar{\nu}_{M}^{t}\right)\right\} $.
We first consider $A_{1}$:
\[
A_{1}-\mathbb{E}\left\{ A_{1}\right\} =\frac{1}{M}\sum_{i=1}^{M}X_{1,i}-\mathbb{E}\left\{ X_{1,i}\right\} ,
\]
in which
\[
X_{1,i}=\mathbb{E}_{\boldsymbol{u}}\left\{ \left\langle \boldsymbol{u},\boldsymbol{D}_{t}\boldsymbol{z}_{i}\sigma\left(\left\langle \boldsymbol{z}_{i},\boldsymbol{D}_{t}\boldsymbol{u}\right\rangle \right)\right\rangle \right\} =\frac{1}{d}\left\Vert \boldsymbol{D}_{\boldsymbol{\Sigma}}\boldsymbol{D}_{t}\boldsymbol{z}_{i}\right\Vert _{2}^{2}\mathbb{E}\left\{ g\sigma\left(g\right)\right\} =\frac{1}{2d}\left\Vert \boldsymbol{D}_{\boldsymbol{\Sigma}}\boldsymbol{D}_{t}\boldsymbol{z}_{i}\right\Vert _{2}^{2}.
\]
For any positive integer $p$,
\[
\mathbb{E}\left\{ \left|X_{1,i}\right|^{p}\right\} =\frac{C^{p}}{d^{p}}\mathbb{E}\left\{ \left\Vert \boldsymbol{D}_{\boldsymbol{\Sigma}}\boldsymbol{D}_{t}\boldsymbol{z}_{i}\right\Vert _{2}^{2p}\right\} \leq\frac{C^{p}}{d^{p}}\mathbb{E}\left\{ \left\Vert \boldsymbol{z}_{i}\right\Vert _{2}^{p}\right\} \leq C^{p}p^{p}.
\]
This implies that $X_{1,i}$ is $C$-sub-exponential, and hence, by
Lemma \ref{lem:subgauss_subexp_properties}, for $\delta\in\left(0,1\right)$,
\[
\mathbb{P}\left\{ \left|A_{1}-\mathbb{E}\left\{ A_{1}\right\} \right|\geq\delta\right\} \leq Ce^{-C\delta^{2}M},
\]
which shows concentration for $A_{1}$.

Next we consider concentration of $A_{2}$. To do so, we bound its
``central'' $p$-moment, for an even number $p$, recalling the
random variables $\boldsymbol{a}$ and $\tilde{\boldsymbol{Z}}$ as
defined in the previous step and applying Jensen's inequality:
\begin{align*}
 & \mathbb{E}\left\{ \left|A_{2}-\frac{1}{4}\mathbb{E}_{\boldsymbol{u}}\left\{ \left\Vert \boldsymbol{D}_{t}^{2}\boldsymbol{u}\right\Vert _{2}^{2}\right\} -\frac{1}{2M}\left\Vert \boldsymbol{D}_{t}\right\Vert _{{\rm F}}^{2}\mathbb{E}_{\boldsymbol{u}}\left\{ \left\Vert \boldsymbol{D}_{t}\boldsymbol{u}\right\Vert _{2}^{2}\right\} \right|^{p}\right\} \\
 & \quad\leq\mathbb{E}_{\boldsymbol{u},\boldsymbol{Z}}\left\{ \left|\frac{1}{M^{2}}\left\Vert \boldsymbol{D}_{t}\boldsymbol{Z}^{\top}\sigma\left(\boldsymbol{Z}\boldsymbol{D}_{t}\boldsymbol{u}\right)\right\Vert _{2}^{2}-\frac{1}{4}\left\Vert \boldsymbol{D}_{t}^{2}\boldsymbol{u}\right\Vert _{2}^{2}-\frac{1}{2M}\left\Vert \boldsymbol{D}_{t}\right\Vert _{{\rm F}}^{2}\left\Vert \boldsymbol{D}_{t}\boldsymbol{u}\right\Vert _{2}^{2}\right|^{p}\right\} \\
 & \quad=\mathbb{E}_{\boldsymbol{u},\boldsymbol{a},\tilde{\boldsymbol{Z}}}\left\{ \left|\frac{1}{M^{2}}\left\Vert \boldsymbol{D}_{t}{\rm Proj}_{\boldsymbol{D}_{t}\boldsymbol{u}}^{\perp}\tilde{\boldsymbol{Z}}^{\top}\sigma\left(\boldsymbol{a}\right)+\frac{\boldsymbol{D}_{t}^{2}\boldsymbol{u}}{\left\Vert \boldsymbol{D}_{t}\boldsymbol{u}\right\Vert _{2}^{2}}\left\langle \boldsymbol{a},\sigma\left(\boldsymbol{a}\right)\right\rangle \right\Vert _{2}^{2}-\frac{1}{4}\left\Vert \boldsymbol{D}_{t}^{2}\boldsymbol{u}\right\Vert _{2}^{2}-\frac{1}{2M}\left\Vert \boldsymbol{D}_{t}\right\Vert _{{\rm F}}^{2}\left\Vert \boldsymbol{D}_{t}\boldsymbol{u}\right\Vert _{2}^{2}\right|^{p}\right\} \\
 & \quad\leq C^{p}\left(A_{2,1,p}+A_{2,2,p}+A_{2,3,p}+\sqrt{A_{2,2,p}A_{2,4,p}}+\sqrt{A_{2,2,p}A_{2,5,p}}+A_{2,6,p}\right),
\end{align*}
in which we define:
\begin{align*}
A_{2,1,p} & =\frac{1}{M^{2p}}\mathbb{E}_{\boldsymbol{u},\boldsymbol{a},\tilde{\boldsymbol{Z}}}\left\{ \left|\left\Vert \boldsymbol{D}_{t}\tilde{\boldsymbol{Z}}^{\top}\sigma\left(\boldsymbol{a}\right)\right\Vert _{2}^{2}-\frac{M}{2}\left\Vert \boldsymbol{D}_{t}\right\Vert _{{\rm F}}^{2}\left\Vert \boldsymbol{D}_{t}\boldsymbol{u}\right\Vert _{2}^{2}\right|^{p}\right\} ,\\
A_{2,2,p} & =\frac{1}{M^{2p}}\mathbb{E}_{\boldsymbol{u},\boldsymbol{a},\tilde{\boldsymbol{Z}}}\left\{ \left\Vert \boldsymbol{D}_{t}{\rm Proj}_{\boldsymbol{D}_{t}\boldsymbol{u}}\tilde{\boldsymbol{Z}}^{\top}\sigma\left(\boldsymbol{a}\right)\right\Vert _{2}^{2p}\right\} ,\\
A_{2,3,p} & =\frac{1}{M^{2p}}\mathbb{E}_{\boldsymbol{u},\boldsymbol{a},\tilde{\boldsymbol{Z}}}\left\{ \left|\left\Vert \frac{\boldsymbol{D}_{t}^{2}\boldsymbol{u}}{\left\Vert \boldsymbol{D}_{t}\boldsymbol{u}\right\Vert _{2}^{2}}\left\langle \boldsymbol{a},\sigma\left(\boldsymbol{a}\right)\right\rangle \right\Vert _{2}^{2}-\frac{M^{2}}{4}\left\Vert \boldsymbol{D}_{t}^{2}\boldsymbol{u}\right\Vert _{2}^{2}\right|^{p}\right\} ,\\
A_{2,4,p} & =\frac{1}{M^{2p}}\mathbb{E}_{\boldsymbol{u},\boldsymbol{a},\tilde{\boldsymbol{Z}}}\left\{ \left\Vert \boldsymbol{D}_{t}\tilde{\boldsymbol{Z}}^{\top}\sigma\left(\boldsymbol{a}\right)\right\Vert _{2}^{2p}\right\} ,\\
A_{2,5,p} & =\frac{1}{M^{2p}}\mathbb{E}_{\boldsymbol{u},\boldsymbol{a},\tilde{\boldsymbol{Z}}}\left\{ \left\Vert \frac{\boldsymbol{D}_{t}^{2}\boldsymbol{u}}{\left\Vert \boldsymbol{D}_{t}\boldsymbol{u}\right\Vert _{2}^{2}}\left\langle \boldsymbol{a},\sigma\left(\boldsymbol{a}\right)\right\rangle \right\Vert _{2}^{2p}\right\} ,\\
A_{2,6,p} & =\frac{1}{M^{2p}}\mathbb{E}_{\boldsymbol{u},\boldsymbol{a},\tilde{\boldsymbol{Z}}}\left\{ \left|\left\langle \boldsymbol{D}_{t}\tilde{\boldsymbol{Z}}^{\top}\sigma\left(\boldsymbol{a}\right),\frac{\boldsymbol{D}_{t}^{2}\boldsymbol{u}}{\left\Vert \boldsymbol{D}_{t}\boldsymbol{u}\right\Vert _{2}^{2}}\left\langle \boldsymbol{a},\sigma\left(\boldsymbol{a}\right)\right\rangle \right\rangle \right|^{p}\right\} .
\end{align*}
Here without loss of generality, we have defined $\left(\boldsymbol{u},\boldsymbol{a},\tilde{\boldsymbol{Z}}\right)$
on a joint space such that $\tilde{\boldsymbol{Z}}$ is independent
of $\boldsymbol{u}$ and $\boldsymbol{a}$, and $\boldsymbol{a}|\boldsymbol{u}\sim\mathsf{N}\left(\boldsymbol{0},\left\Vert \boldsymbol{D}_{t}\boldsymbol{u}\right\Vert _{2}^{2}\boldsymbol{I}_{M}\right)$.
For convenience, we shall also take $\boldsymbol{a}=\left\Vert \boldsymbol{D}_{t}\boldsymbol{u}\right\Vert _{2}\boldsymbol{g}$
for some $\boldsymbol{g}\sim\mathsf{N}\left(\boldsymbol{0},\boldsymbol{I}_{M}\right)$,
defined on the same joint space, independent of $\boldsymbol{u}$
and $\tilde{\boldsymbol{Z}}$. Below we shall let the $\left(i,j\right)$-th
entry and $i$-th row of $\tilde{\boldsymbol{Z}}$ be $\tilde{z}_{i,j}$
and $\tilde{\boldsymbol{z}}_{i}$ respectively, the $i$-th entry
of $\boldsymbol{a}$ (respectively, $\boldsymbol{u}$ and $\boldsymbol{g}$)
be $a_{i}$ (respectively, $u_{i}$ and $g_{i}$). We also note and
recall a few useful bounds:
\begin{itemize}
\item $\left\Vert \boldsymbol{D}_{\boldsymbol{\Sigma}}\right\Vert _{{\rm op}}\leq C$
and $\left\Vert \boldsymbol{D}_{t}\right\Vert _{{\rm op}}\leq\max_{i\in\left[d\right]}r_{i,t}\leq C$.
\item $\mathbb{E}\left\{ \left\Vert \boldsymbol{u}\right\Vert _{2}^{2p}\right\} \leq C^{p}d^{-p}\mathbb{E}\left\{ \left\Vert \sqrt{d}\boldsymbol{D}_{\boldsymbol{\Sigma}}^{-1}\boldsymbol{u}\right\Vert _{2}^{2p}\right\} \leq C^{p}\left(1+\left(p/d\right)^{p}\right)$,
since $\left\Vert \boldsymbol{D}_{\boldsymbol{\Sigma}}\right\Vert _{{\rm op}}\leq C$
and $\left\Vert \sqrt{d}\boldsymbol{D}_{\boldsymbol{\Sigma}}^{-1}\boldsymbol{u}\right\Vert _{2}^{2}$
is a $\chi^{2}$ random variable with degree of freedom $d$ and thus
has its $p$-moment bounded by $C^{p}\left(d^{p}+p^{p}\right)$.
\item $\mathbb{E}_{\tilde{\boldsymbol{Z}}}\left\{ \left\Vert \tilde{\boldsymbol{z}}_{i}\right\Vert _{2}^{2p}\right\} \leq C^{p}\left(d^{p}+p^{p}\right)$
and $\mathbb{E}\left\{ \sigma\left(g\right)^{2p}\right\} \leq\mathbb{E}\left\{ g^{2p}\right\} \leq C^{p}p^{p}$
for the same reason.
\item $\mathbb{E}\left\{ \left[g\sigma\left(g\right)\right]^{2p}\right\} \leq\mathbb{E}\left\{ g^{4p}\right\} \leq C^{p}p^{2p}$
by the above.
\end{itemize}
We proceed with several steps.

\paragraph*{Step 2.1 - Bounding $A_{2,1,p}$.}

We have:
\begin{align*}
A_{2,1,p} & =\frac{C^{p}}{M^{2p}}\mathbb{E}_{\boldsymbol{u},\boldsymbol{a},\tilde{\boldsymbol{Z}}}\left\{ \left|\sum_{i=1}^{M}\sum_{j=1}^{M}\left\langle \boldsymbol{D}_{t}\tilde{\boldsymbol{z}}_{i}\sigma\left(a_{i}\right),\boldsymbol{D}_{t}\tilde{\boldsymbol{z}}_{j}\sigma\left(a_{j}\right)\right\rangle -\frac{M}{2}\left\Vert \boldsymbol{D}_{t}\right\Vert _{{\rm F}}^{2}\left\Vert \boldsymbol{D}_{t}\boldsymbol{u}\right\Vert _{2}^{2}\right|^{p}\right\} \\
 & \leq\frac{C^{p}}{M^{2p}}\mathbb{E}_{\boldsymbol{u},\boldsymbol{a},\tilde{\boldsymbol{Z}}}\left\{ \left|\sum_{i=1}^{M}\left[\left\Vert \boldsymbol{D}_{t}\tilde{\boldsymbol{z}}_{i}\sigma\left(a_{i}\right)\right\Vert _{2}^{2}-\left\Vert \boldsymbol{D}_{t}\right\Vert _{{\rm F}}^{2}\sigma\left(a_{i}\right)^{2}\right]\right|^{p}\right\} \\
 & \qquad+\frac{C^{p}}{M^{2p}}\mathbb{E}_{\boldsymbol{u},\boldsymbol{a},\tilde{\boldsymbol{Z}}}\left\{ \left|\sum_{i=1}^{M}\left[\left\Vert \boldsymbol{D}_{t}\right\Vert _{{\rm F}}^{2}\sigma\left(a_{i}\right)^{2}-\frac{1}{2}\left\Vert \boldsymbol{D}_{t}\right\Vert _{{\rm F}}^{2}\left\Vert \boldsymbol{D}_{t}\boldsymbol{u}\right\Vert _{2}^{2}\right]\right|^{p}\right\} \\
 & \qquad+\frac{C^{p}}{M^{2p}}\mathbb{E}_{\boldsymbol{a},\tilde{\boldsymbol{Z}}}\left\{ \left|\sum_{i\neq j}\left\langle \boldsymbol{D}_{t}\tilde{\boldsymbol{z}}_{i}\sigma\left(a_{i}\right),\boldsymbol{D}_{t}\tilde{\boldsymbol{z}}_{j}\sigma\left(a_{j}\right)\right\rangle \right|^{p}\right\} \\
 & \equiv B_{1.1}+B_{1.2}+B_{1.3}.
\end{align*}
We then bound $B_{1.1}$, $B_{1.2}$ and $B_{1.3}$:
\begin{itemize}
\item To bound $B_{1.1}$, we rewrite:
\[
B_{1.1}=\frac{C^{p}}{M^{2p}}\mathbb{E}_{\boldsymbol{u},\boldsymbol{a},\tilde{\boldsymbol{Z}}}\left\{ \left|\sum_{i=1}^{M}\sum_{k=1}^{d}r_{k,t}^{2}\left(\tilde{z}_{i,k}^{2}-1\right)\sigma\left(a_{i}\right)^{2}\right|^{p}\right\} .
\]
Notice that $\left(r_{k,t}^{2}\left(\tilde{z}_{i,k}^{2}-1\right)\sigma\left(a_{i}\right)^{2}\right)_{i\leq M,\;k\leq d}$
are independent conditional on $\boldsymbol{a}$ and $\boldsymbol{u}$.
We also have $\mathbb{E}_{\tilde{\boldsymbol{Z}}}\left\{ r_{k,t}^{2}\left(\tilde{z}_{i,k}^{2}-1\right)\sigma\left(a_{i}\right)^{2}\middle|\boldsymbol{a},\boldsymbol{u}\right\} =0$,
and
\begin{align*}
 & \mathbb{E}_{\boldsymbol{u},\boldsymbol{a},\tilde{\boldsymbol{Z}}}\left\{ \left|r_{k,t}^{2}\left(\tilde{z}_{i,k}^{2}-1\right)\sigma\left(a_{i}\right)^{2}\right|^{p}\right\} =\mathbb{E}_{\boldsymbol{u},\boldsymbol{g},\tilde{\boldsymbol{Z}}}\left\{ r_{k,t}^{2p}\left\Vert \boldsymbol{D}_{t}\boldsymbol{u}\right\Vert _{2}^{2p}\sigma\left(g_{i}\right)^{2p}\left|\tilde{z}_{i,k}^{2}-1\right|^{p}\right\} \\
 & \quad\leq C^{p}\mathbb{E}_{\boldsymbol{u}}\left\{ \left\Vert \boldsymbol{u}\right\Vert _{2}^{2p}\right\} \mathbb{E}_{\boldsymbol{g}}\left\{ \sigma\left(g_{i}\right)^{2p}\right\} \mathbb{E}_{\tilde{\boldsymbol{Z}}}\left\{ \left|\tilde{z}_{i,k}^{2}-1\right|^{p}\right\} \\
 & \quad\leq C^{p}\left(1+\left(p/d\right)^{p}\right)p^{p}\left(p^{p}+1\right).
\end{align*}
By Lemma \ref{lem:bound_moment_sym},
\[
B_{1.1}\leq C^{p}p^{4p}\left(\frac{\sqrt{d}}{M^{3/2}}\right)^{p}.
\]
\item To bound $B_{1.2}$, notice that $\left(\left\Vert \boldsymbol{D}_{t}\right\Vert _{{\rm F}}^{2}\sigma\left(a_{i}\right)^{2}-\frac{1}{2}\left\Vert \boldsymbol{D}_{t}\right\Vert _{{\rm F}}^{2}\left\Vert \boldsymbol{D}_{t}\boldsymbol{u}\right\Vert _{2}^{2}\right)_{i\leq M}$
are independent conditional on $\boldsymbol{u}$, $\mathbb{E}_{\boldsymbol{a}}\left\{ \left\Vert \boldsymbol{D}_{t}\right\Vert _{{\rm F}}^{2}\sigma\left(a_{i}\right)^{2}-\frac{1}{2}\left\Vert \boldsymbol{D}_{t}\right\Vert _{{\rm F}}^{2}\left\Vert \boldsymbol{D}_{t}\boldsymbol{u}\right\Vert _{2}^{2}\middle|\boldsymbol{u}\right\} =0$,
and
\begin{align*}
 & \mathbb{E}_{\boldsymbol{u},\boldsymbol{a}}\left\{ \left|\left\Vert \boldsymbol{D}_{t}\right\Vert _{{\rm F}}^{2}\sigma\left(a_{i}\right)^{2}-\frac{1}{2}\left\Vert \boldsymbol{D}_{t}\right\Vert _{{\rm F}}^{2}\left\Vert \boldsymbol{D}_{t}\boldsymbol{u}\right\Vert _{2}^{2}\right|^{p}\right\} \leq C^{p}d^{p}\mathbb{E}_{\boldsymbol{u},\boldsymbol{a}}\left\{ \left|\sigma\left(a_{i}\right)^{2}-\frac{1}{2}\left\Vert \boldsymbol{D}_{t}\boldsymbol{u}\right\Vert _{2}^{2}\right|^{p}\right\} \\
 & \quad=C^{p}d^{p}\mathbb{E}_{\boldsymbol{g}}\left\{ \left|\sigma\left(g_{i}\right)^{2}-\frac{1}{2}\right|^{p}\right\} \mathbb{E}_{\boldsymbol{u}}\left\{ \left\Vert \boldsymbol{D}_{t}\boldsymbol{u}\right\Vert _{2}^{2p}\right\} \leq C^{p}d^{p}\mathbb{E}_{\boldsymbol{g}}\left\{ \sigma\left(g_{i}\right)^{2p}+1\right\} \mathbb{E}_{\boldsymbol{u}}\left\{ \left\Vert \boldsymbol{u}\right\Vert _{2}^{2p}\right\} \\
 & \quad\leq C^{p}d^{p}\left(p^{p}+1\right)\left(1+\left(p/d\right)^{p}\right).
\end{align*}
By Lemma \ref{lem:bound_moment_sym},
\[
B_{1.2}\leq C^{p}p^{3p}\left(\frac{d}{M^{3/2}}\right)^{p}.
\]
\item To bound $B_{1.3}$, let $\boldsymbol{B}_{1.3,i}=\boldsymbol{D}_{t}\tilde{\boldsymbol{z}}_{i}\sigma\left(a_{i}\right)$.
For any $k\leq d$ and $i\neq j$,
\[
\mathbb{E}_{\tilde{\boldsymbol{Z}}}\left\{ \left|\tilde{z}_{ik}\tilde{z}_{jk}\right|^{p}\right\} =\mathbb{E}_{\tilde{\boldsymbol{Z}}}\left\{ \left|\tilde{z}_{ik}\right|^{p}\right\} \mathbb{E}_{\tilde{\boldsymbol{Z}}}\left\{ \left|\tilde{z}_{jk}\right|^{p}\right\} \leq C^{p}p^{p}.
\]
So Lemma \ref{lem:bound_moment_sym} implies
\[
\mathbb{E}_{\tilde{\boldsymbol{Z}}}\left\{ \left|\left\langle \tilde{\boldsymbol{z}}_{i},\tilde{\boldsymbol{z}}_{j}\right\rangle \right|^{p}\right\} \leq C^{p}p^{2p}d^{p/2}.
\]
As such, we obtain for any $i\neq j$:
\begin{align*}
\mathbb{E}\left\{ \left|\left\langle \boldsymbol{B}_{1.3,i},\boldsymbol{B}_{1.3,j}\right\rangle \right|^{p}\right\}  & \leq C^{p}\mathbb{E}_{\boldsymbol{u},\boldsymbol{g},\tilde{\boldsymbol{Z}}}\left\{ \left\Vert \boldsymbol{D}_{t}\boldsymbol{u}\right\Vert _{2}^{2p}\left|\left\langle \tilde{\boldsymbol{z}}_{i},\tilde{\boldsymbol{z}}_{j}\right\rangle \sigma\left(g_{i}\right)\sigma\left(g_{j}\right)\right|^{p}\right\} \\
 & \leq C^{p}\mathbb{E}_{\boldsymbol{u}}\left\{ \left\Vert \boldsymbol{u}\right\Vert _{2}^{2p}\right\} \mathbb{E}_{\tilde{\boldsymbol{Z}}}\left\{ \left|\left\langle \tilde{\boldsymbol{z}}_{i},\tilde{\boldsymbol{z}}_{j}\right\rangle \right|^{p}\right\} \mathbb{E}_{\boldsymbol{g}}\left\{ \left|g_{i}\right|^{p}\right\} \mathbb{E}_{\boldsymbol{g}}\left\{ \left|g_{j}\right|^{p}\right\} \\
 & \leq C^{p}\left(1+\left(p/d\right)^{p}\right)p^{3p}d^{p/2}.
\end{align*}
To proceed, we follow an argument similar to the proof of Lemma \ref{lem:bound_moment_sym}.
We observe that $\left|\sum_{i\neq j}\left\langle \boldsymbol{B}_{1.3,i},\boldsymbol{B}_{1.3,j}\right\rangle \right|^{p}$
is a sum of terms of the form $H=\prod_{k=1}^{p}\left\langle \boldsymbol{b}_{k},\boldsymbol{b}_{2k}\right\rangle $,
where $\boldsymbol{b}_{k}\in\left\{ \boldsymbol{B}_{1.3,i}\right\} _{i\leq M}$
for $k=1,...,2p$ such that $\boldsymbol{b}_{k}\neq\boldsymbol{b}_{2k}$.
Suppose $H$ has $q_{i}$ repeats of $\boldsymbol{B}_{1.3,i}$, where
$\sum_{i=1}^{M}q_{i}=2p$. By Holder's inequality and the bound on
$\mathbb{E}\left\{ \left|\left\langle \boldsymbol{B}_{1.3,i},\boldsymbol{B}_{1.3,j}\right\rangle \right|^{p}\right\} $,
\begin{align*}
\mathbb{E}\left\{ \left|H\right|\right\}  & \leq\prod_{k=1}^{p}\mathbb{E}\left\{ \left|\left\langle \boldsymbol{b}_{k},\boldsymbol{b}_{2k}\right\rangle \right|^{p}\right\} ^{q_{i}/\left(2p\right)}\leq\prod_{k=1}^{p}\left[C^{p}\left(1+\left(p/d\right)^{p}\right)p^{3p}d^{p/2}\right]^{q_{i}/\left(2p\right)}\\
 & =C^{p}\left(1+\left(p/d\right)^{p}\right)p^{3p}d^{p/2}.
\end{align*}
Observe that $\mathbb{E}\left\{ H\right\} =0$ if there exists some
$i\in\left[M\right]$ such that $q_{i}$ is odd since $\tilde{\boldsymbol{z}}_{i}$
is symmetric. As proven in the proof of Lemma \ref{lem:bound_moment_sym},
the number of terms $H$ such that no $q_{i}$ is odd is upper-bounded
by $\left(2p\right)!M^{p}\leq4^{p}p^{2p}M^{p}$. Hence 
\[
B_{1.3}=\frac{C^{p}}{M^{2p}}\mathbb{E}\left\{ \left|\sum_{i\neq j}\left\langle \boldsymbol{B}_{1.3,i},\boldsymbol{B}_{1.3,j}\right\rangle \right|^{p}\right\} \leq C^{p}\left(1+\left(p/d\right)^{p}\right)p^{5p}\left(\frac{\sqrt{d}}{M}\right)^{p}.
\]
\end{itemize}
These bounds yield
\[
A_{2,1,p}\leq C^{p}p^{6p}\left(1+\frac{d}{M}\right)^{p}\frac{1}{M^{p/2}}.
\]

\paragraph*{Step 2.2 - Bounding $A_{2,2,p}$.}

We bound $A_{2,2,p}$:
\begin{align*}
A_{2,2,p} & \leq\frac{C^{p}}{M^{2p}}\mathbb{E}_{\boldsymbol{u},\boldsymbol{a},\tilde{\boldsymbol{Z}}}\left\{ \left\Vert \frac{\boldsymbol{D}_{t}^{2}\boldsymbol{u}\boldsymbol{u}^{\top}\boldsymbol{D}_{t}}{\left\Vert \boldsymbol{D}_{t}\boldsymbol{u}\right\Vert _{2}^{2}}\tilde{\boldsymbol{Z}}^{\top}\sigma\left(\boldsymbol{a}\right)\right\Vert _{2}^{2p}\right\} \leq\frac{C^{p}}{M^{2p}}\mathbb{E}_{\boldsymbol{u},\boldsymbol{a},\tilde{\boldsymbol{Z}}}\left\{ \left\langle \frac{\tilde{\boldsymbol{Z}}\boldsymbol{D}_{t}\boldsymbol{u}}{\left\Vert \boldsymbol{D}_{t}\boldsymbol{u}\right\Vert _{2}},\sigma\left(\boldsymbol{a}\right)\right\rangle ^{2p}\right\} \\
 & =\frac{C^{p}}{M^{2p}}\mathbb{E}_{\boldsymbol{u},\boldsymbol{g},\tilde{\boldsymbol{Z}}}\left\{ \left|\sum_{i=1}^{M}\left\langle \tilde{\boldsymbol{z}}_{i},\boldsymbol{D}_{t}\boldsymbol{u}\right\rangle \sigma\left(g_{i}\right)\right|^{2p}\right\} =\frac{C^{p}}{M^{2p}}\mathbb{E}_{\boldsymbol{u},\boldsymbol{g},\tilde{\boldsymbol{Z}}}\left\{ \left|\sum_{i=1}^{M}\sum_{k=1}^{d}r_{k,t}\tilde{z}_{i,k}u_{k}\sigma\left(g_{i}\right)\right|^{2p}\right\} .
\end{align*}
We have $\left(r_{k,t}\tilde{z}_{i,k}u_{k}\sigma\left(g_{i}\right)\right)_{i\leq M,\;k\leq d}$
are independent conditional on $\boldsymbol{u}$ and $\boldsymbol{g}$.
Furthermore we also have $\mathbb{E}_{\tilde{\boldsymbol{Z}}}\left\{ r_{k,t}\tilde{z}_{i,k}u_{k}\sigma\left(g_{i}\right)\middle|\boldsymbol{u},\boldsymbol{g}\right\} =0$
and, by recalling $\boldsymbol{u}\sim\mathsf{N}\left(\boldsymbol{0},\boldsymbol{D}_{\boldsymbol{\Sigma}}^{2}/d\right)$
with $\Sigma_{k}\leq C$ for all $k\in\left[d\right]$,
\[
\mathbb{E}_{\boldsymbol{u},\boldsymbol{g},\tilde{\boldsymbol{Z}}}\left\{ \left|r_{k,t}\tilde{z}_{i,k}u_{k}\sigma\left(g_{i}\right)\right|^{2p}\right\} \leq C^{p}\mathbb{E}_{\tilde{\boldsymbol{Z}}}\left\{ \tilde{z}_{i,k}^{2p}\right\} \mathbb{E}_{\boldsymbol{u}}\left\{ u_{k}^{2p}\right\} \mathbb{E}_{\boldsymbol{g}}\left\{ g_{i}^{2p}\right\} \leq\frac{C^{p}p^{3p}}{d^{p}}.
\]
Then by applying Lemma \ref{lem:bound_moment_sym}, we obtain:
\[
A_{2,2,p}\leq\frac{C^{p}}{M^{2p}}\frac{p^{5p}}{d^{p}}\left(Md\right)^{p}=\frac{C^{p}}{M^{p}}p^{5p}.
\]

\paragraph*{Step 2.3 - Bounding $A_{2,3,p}$.}

Note that
\[
\mathbb{E}_{\boldsymbol{u}}\left\{ \left\Vert \boldsymbol{D}_{t}^{2}\boldsymbol{u}\right\Vert _{2}^{2p}\right\} \leq C^{p}\mathbb{E}_{\boldsymbol{u}}\left\{ \left\Vert \boldsymbol{u}\right\Vert _{2}^{2p}\right\} \leq C^{p}\left(1+\left(p/d\right)^{p}\right).
\]
We then have a bound on $A_{2,3,p}$:
\begin{align*}
A_{2,3,p} & =\frac{1}{M^{2p}}\mathbb{E}_{\boldsymbol{u},\boldsymbol{g}}\left\{ \left|\left\Vert \boldsymbol{D}_{t}^{2}\boldsymbol{u}\right\Vert _{2}^{2}\left\langle \boldsymbol{g},\sigma\left(\boldsymbol{g}\right)\right\rangle ^{2}-\frac{M^{2}}{4}\left\Vert \boldsymbol{D}_{t}^{2}\boldsymbol{u}\right\Vert _{2}^{2}\right|^{p}\right\} \\
 & =\frac{1}{M^{2p}}\mathbb{E}_{\boldsymbol{u},\boldsymbol{g}}\left\{ \left\Vert \boldsymbol{D}_{t}^{2}\boldsymbol{u}\right\Vert _{2}^{2p}\left|\sum_{i=1}^{M}g_{i}^{2}\sigma\left(g_{i}\right)^{2}+\sum_{i\neq j\leq M}g_{i}g_{j}\sigma\left(g_{i}\right)\sigma\left(g_{j}\right)-\frac{M^{2}}{4}\right|^{p}\right\} \\
 & \leq\frac{C^{p}}{M^{2p}}\left(1+\left(p/d\right)^{p}\right)\mathbb{E}_{\boldsymbol{g}}\left\{ \left|\sum_{i=1}^{M}g_{i}^{2}\sigma\left(g_{i}\right)^{2}+\sum_{i\neq j\leq M}g_{i}g_{j}\sigma\left(g_{i}\right)\sigma\left(g_{j}\right)-\frac{M^{2}}{4}\right|^{p}\right\} \\
 & \leq\frac{C^{p}}{M^{2p}}\left(1+\left(p/d\right)^{p}\right)\Bigg(\mathbb{E}_{\boldsymbol{g}}\left\{ \left|\sum_{i=1}^{M}\left(g_{i}^{2}\sigma\left(g_{i}\right)^{2}-1.5\right)\right|^{p}\right\} +\mathbb{E}_{\boldsymbol{g}}\left\{ \left|\frac{M-1}{2}\sum_{i=1}^{M}\left(g_{i}\sigma\left(g_{i}\right)-0.5\right)\right|^{p}\right\} \\
 & \qquad+\mathbb{E}_{\boldsymbol{g}}\left\{ \left|\sum_{i=1}^{M}\sum_{j\leq M,\;j\neq i}g_{i}\sigma\left(g_{i}\right)\left(g_{j}\sigma\left(g_{j}\right)-0.5\right)\right|^{p}\right\} +M^{p}\Bigg)\\
 & \equiv C^{p}\left(1+\left(p/d\right)^{p}\right)\left(B_{3.1}+B_{3.2}+B_{3.3}+M^{-p}\right).
\end{align*}
We bound each term:
\begin{itemize}
\item To bound $B_{3.1}$, notice that $\left(g_{i}^{2}\sigma\left(g_{i}\right)^{2}-1.5\right)_{i\leq M}$
are independent, $\mathbb{E}_{\boldsymbol{g}}\left\{ g_{i}^{2}\sigma\left(g_{i}\right)^{2}-1.5\right\} =0$
and
\[
\mathbb{E}_{\boldsymbol{g}}\left\{ \left|g_{i}^{2}\sigma\left(g_{i}\right)^{2}-1.5\right|^{p}\right\} \leq\mathbb{E}_{\boldsymbol{g}}\left\{ g_{i}^{2p}\sigma\left(g_{i}\right)^{2p}+1.5^{p}\right\} \leq C^{p}p^{2p}.
\]
By Lemma \ref{lem:bound_moment_sym},
\[
B_{3.1}\leq\frac{C^{p}p^{3p}}{M^{1.5p}}.
\]
\item To bound $B_{3.2}$, notice that $\left(g_{i}\sigma\left(g_{i}\right)-0.5\right)_{i\leq M}$
are independent, $\mathbb{E}_{\boldsymbol{g}}\left\{ g_{i}\sigma\left(g_{i}\right)-0.5\right\} =0$,
and
\[
\mathbb{E}_{\boldsymbol{g}}\left\{ \left|g_{i}\sigma\left(g_{i}\right)-0.5\right|^{p}\right\} \leq C^{p}\left(\mathbb{E}_{\boldsymbol{g}}\left\{ \left|g_{i}\sigma\left(g_{i}\right)\right|^{p}\right\} +1\right)\leq C^{p}p^{p}.
\]
By Lemma \ref{lem:bound_moment_sym},
\[
B_{3.2}\leq\frac{C^{p}p^{2p}}{M^{p/2}}.
\]
\item To bound $B_{3.3}$, notice that for a fixed $i$, $\left(g_{i}\sigma\left(g_{i}\right)\left(g_{j}\sigma\left(g_{j}\right)-0.5\right)\right)_{j\leq M,\;j\neq i}$
are independent conditional on $g_{i}$, $\mathbb{E}_{\boldsymbol{g}}\left\{ g_{i}\sigma\left(g_{i}\right)\left(g_{j}\sigma\left(g_{j}\right)-0.5\right)\middle|g_{i}\right\} =0$
and
\[
\mathbb{E}_{\boldsymbol{g}}\left\{ \left|g_{i}\sigma\left(g_{i}\right)\left(g_{j}\sigma\left(g_{j}\right)-0.5\right)\right|^{p}\right\} \leq C^{p}\mathbb{E}_{\boldsymbol{g}}\left\{ \left|g_{i}\sigma\left(g_{i}\right)\right|^{p}\right\} \left(\mathbb{E}_{\boldsymbol{g}}\left\{ \left|g_{j}\sigma\left(g_{j}\right)\right|^{p}\right\} +1\right)\leq C^{p}p^{2p}.
\]
By Lemma \ref{lem:bound_moment_sym},
\[
B_{3.3}\leq\frac{C^{p}}{M^{p}}\sum_{i=1}^{M}\mathbb{E}_{\boldsymbol{g}}\left\{ \left|\sum_{j\leq M,\;j\neq i}g_{i}\sigma\left(g_{i}\right)\left(g_{j}\sigma\left(g_{j}\right)-0.5\right)\right|^{p}\right\} \leq\frac{C^{p}p^{3p}}{M^{p/2-1}}.
\]
\end{itemize}
We thus obtain:
\[
A_{2,3,p}\leq\frac{C^{p}p^{4p}}{M^{p/2-1}}.
\]

\paragraph*{Step 2.4 - Bounding $A_{2,4,p}$.}

We bound $A_{2,4,p}$:
\begin{align*}
A_{2,4,p} & \leq\frac{C^{p}}{M^{2p}}\mathbb{E}_{\boldsymbol{u},\boldsymbol{a},\tilde{\boldsymbol{Z}}}\left\{ \left\Vert \tilde{\boldsymbol{Z}}^{\top}\sigma\left(\boldsymbol{a}\right)\right\Vert _{2}^{2p}\right\} =\frac{C^{p}}{M^{2p}}\mathbb{E}_{\boldsymbol{u},\boldsymbol{g},\tilde{\boldsymbol{Z}}}\left\{ \left\Vert \boldsymbol{D}_{t}\boldsymbol{u}\right\Vert _{2}^{2p}\left\Vert \sum_{i=1}^{M}\tilde{\boldsymbol{z}}_{i}\sigma\left(g_{i}\right)\right\Vert _{2}^{2p}\right\} \\
 & \leq\frac{C^{p}}{M^{2p}}\mathbb{E}_{\boldsymbol{u},\boldsymbol{g},\tilde{\boldsymbol{Z}}}\left\{ \left\Vert \boldsymbol{u}\right\Vert _{2}^{2p}\left\Vert \sum_{i=1}^{M}\tilde{\boldsymbol{z}}_{i}\sigma\left(g_{i}\right)\right\Vert _{2}^{2p}\right\} \leq\frac{C^{p}}{M^{2p}}\left(1+\left(p/d\right)^{p}\right)\mathbb{E}_{\boldsymbol{g},\tilde{\boldsymbol{Z}}}\left\{ \left\Vert \sum_{i=1}^{M}\tilde{\boldsymbol{z}}_{i}\sigma\left(g_{i}\right)\right\Vert _{2}^{2p}\right\} .
\end{align*}
Notice that $\left(\tilde{\boldsymbol{z}}_{i}\sigma\left(g_{i}\right)\right)_{i\leq M}$
are independent, $\mathbb{E}_{\boldsymbol{g},\tilde{\boldsymbol{Z}}}\left\{ \tilde{\boldsymbol{z}}_{i}\sigma\left(g_{i}\right)\right\} =\boldsymbol{0}$,
and
\[
\mathbb{E}_{\boldsymbol{g},\tilde{\boldsymbol{Z}}}\left\{ \left\Vert \tilde{\boldsymbol{z}}_{i}\sigma\left(g_{i}\right)\right\Vert _{2}^{2p}\right\} =\mathbb{E}_{\tilde{\boldsymbol{Z}}}\left\{ \left\Vert \tilde{\boldsymbol{z}}_{i}\right\Vert _{2}^{2p}\right\} \mathbb{E}_{\boldsymbol{g}}\left\{ \sigma\left(g_{i}\right)^{2p}\right\} \leq C^{p}\left(d^{p}+p^{p}\right)p^{p},
\]
which yields, by Lemma \ref{lem:bound_moment_sym},
\[
A_{2,4,p}\leq\frac{C^{p}}{M^{p}}\left(d^{p}+p^{p}\right)p^{4p}.
\]

\paragraph*{Step 2.5 - Bounding $A_{2,5,p}$.}

We have:
\begin{align*}
A_{2,5,p} & =\frac{1}{M^{2p}}\mathbb{E}_{\boldsymbol{u},\boldsymbol{g}}\left\{ \left\Vert \boldsymbol{D}_{t}^{2}\boldsymbol{u}\right\Vert _{2}^{2p}\left\langle \boldsymbol{g},\sigma\left(\boldsymbol{g}\right)\right\rangle ^{2p}\right\} \leq\frac{C^{p}}{M^{2p}}\mathbb{E}_{\boldsymbol{u}}\left\{ \left\Vert \boldsymbol{u}\right\Vert _{2}^{2p}\right\} \mathbb{E}_{\boldsymbol{g}}\left\{ \left\Vert \boldsymbol{g}\right\Vert _{2}^{4p}\right\} \\
 & \leq\frac{C^{p}}{M^{2p}}\left(1+\left(p/d\right)^{p}\right)\left(M^{2p}+p^{2p}\right)\leq C^{p}p^{3p}.
\end{align*}

\paragraph*{Step 2.6 - Bounding $A_{2,6,p}$.}

We have:
\begin{align*}
A_{2,6,p} & =\frac{1}{M^{2p}}\mathbb{E}_{\boldsymbol{u},\boldsymbol{g},\tilde{\boldsymbol{Z}}}\left\{ \left|\sum_{i=1}^{M}\left\langle \tilde{\boldsymbol{z}}_{i},\boldsymbol{D}_{t}^{3}\boldsymbol{u}\right\rangle \left\Vert \boldsymbol{D}_{t}\boldsymbol{u}\right\Vert _{2}\sigma\left(g_{i}\right)\left\langle \boldsymbol{g},\sigma\left(\boldsymbol{g}\right)\right\rangle \right|^{p}\right\} \\
 & \leq C^{p}\mathbb{E}_{\boldsymbol{u},\boldsymbol{g},\tilde{\boldsymbol{Z}}}\left\{ \left|\frac{1}{M}\sum_{i=1}^{M}\left\langle \tilde{\boldsymbol{z}}_{i},\boldsymbol{D}_{t}^{3}\boldsymbol{u}\right\rangle \left\Vert \boldsymbol{D}_{t}\boldsymbol{u}\right\Vert _{2}\sigma\left(g_{i}\right)\right|^{p}\right\} \\
 & \qquad+\frac{C^{p}}{M^{2p}}\mathbb{E}_{\boldsymbol{u},\boldsymbol{g},\tilde{\boldsymbol{Z}}}\left\{ \left|\sum_{i=1}^{M}\left\langle \tilde{\boldsymbol{z}}_{i},\boldsymbol{D}_{t}^{3}\boldsymbol{u}\right\rangle \left\Vert \boldsymbol{D}_{t}\boldsymbol{u}\right\Vert _{2}\sigma\left(g_{i}\right)^{2}g_{i}\right|^{p}\right\} \\
 & \qquad+\frac{C^{p}}{M^{2p}}\mathbb{E}_{\boldsymbol{u},\boldsymbol{g},\tilde{\boldsymbol{Z}}}\left\{ \left|\sum_{i=1}^{M}\left\langle \tilde{\boldsymbol{z}}_{i},\boldsymbol{D}_{t}^{3}\boldsymbol{u}\right\rangle \left\Vert \boldsymbol{D}_{t}\boldsymbol{u}\right\Vert _{2}\sigma\left(g_{i}\right)\sum_{j\neq i,\;j\leq M}\left(g_{j}\sigma\left(g_{j}\right)-0.5\right)\right|^{p}\right\} \\
 & \equiv B_{6.1}+B_{6.2}+B_{6.3}.
\end{align*}
We bound each of the terms:
\begin{itemize}
\item To bound $B_{6.1}$, we have for a fixed $i$, $\left(\tilde{z}_{ij}u_{j}\right)_{j\leq d}$
are independent, $\mathbb{E}_{\boldsymbol{u},\tilde{\boldsymbol{Z}}}\left\{ \tilde{z}_{ij}u_{j}\right\} =0$
and $\mathbb{E}_{\boldsymbol{u},\tilde{\boldsymbol{Z}}}\left\{ \left|\sqrt{d}\tilde{z}_{ij}u_{j}\right|^{p}\right\} \leq C^{p}p^{p}$.
We thus get from Lemma \ref{lem:bound_moment_sym}:
\[
\mathbb{E}_{\boldsymbol{u},\tilde{\boldsymbol{Z}}}\left\{ \left|\left\langle \tilde{\boldsymbol{z}}_{i},\boldsymbol{u}\right\rangle \right|^{p}\right\} =d^{p/2}\mathbb{E}_{\tilde{\boldsymbol{Z}}}\left\{ \left|\frac{1}{d}\sum_{j=1}^{d}\sqrt{d}\tilde{z}_{ij}u_{j}\right|^{p}\right\} \leq C^{p}p^{2p}.
\]
Observe that $\left(\left\langle \tilde{\boldsymbol{z}}_{i},\boldsymbol{D}_{t}^{3}\boldsymbol{u}\right\rangle \left\Vert \boldsymbol{D}_{t}\boldsymbol{u}\right\Vert _{2}\sigma\left(g_{i}\right)\right)_{i\leq M}$
are independent conditional on $\boldsymbol{u}$. We also have $\mathbb{E}_{\boldsymbol{g},\tilde{\boldsymbol{Z}}}\left\{ \left\langle \tilde{\boldsymbol{z}}_{i},\boldsymbol{D}_{t}^{3}\boldsymbol{u}\right\rangle \left\Vert \boldsymbol{D}_{t}\boldsymbol{u}\right\Vert _{2}\sigma\left(g_{i}\right)\middle|\boldsymbol{u}\right\} =0$
and
\begin{align*}
\mathbb{E}_{\boldsymbol{u},\boldsymbol{g},\tilde{\boldsymbol{Z}}}\left\{ \left|\left\langle \tilde{\boldsymbol{z}}_{i},\boldsymbol{D}_{t}^{3}\boldsymbol{u}\right\rangle \left\Vert \boldsymbol{D}_{t}\boldsymbol{u}\right\Vert _{2}\sigma\left(g_{i}\right)\right|^{p}\right\}  & \leq C^{p}\sqrt{\mathbb{E}_{\boldsymbol{u},\tilde{\boldsymbol{Z}}}\left\{ \left|\left\langle \tilde{\boldsymbol{z}}_{i},\boldsymbol{D}_{t}^{3}\boldsymbol{u}\right\rangle \right|^{2p}\right\} \mathbb{E}_{\boldsymbol{u}}\left\{ \left\Vert \boldsymbol{u}\right\Vert _{2}^{2p}\right\} }\mathbb{E}_{\boldsymbol{g}}\left\{ \left|\sigma\left(g_{i}\right)\right|^{p}\right\} \\
 & =C^{p}\sqrt{\mathbb{E}_{\boldsymbol{u}}\left\{ \left\Vert \boldsymbol{D}_{t}^{3}\boldsymbol{u}\right\Vert _{2}^{2p}\right\} \mathbb{E}_{g}\left\{ \left|g\right|^{2p}\right\} \mathbb{E}_{\boldsymbol{u}}\left\{ \left\Vert \boldsymbol{u}\right\Vert _{2}^{2p}\right\} }\mathbb{E}_{\boldsymbol{g}}\left\{ \left|\sigma\left(g_{i}\right)\right|^{p}\right\} \\
 & \leq C^{p}\sqrt{\left(1+\left(p/d\right)^{p}\right)^{2}p^{p}}p^{p/2}\leq C^{p}p^{2p}.
\end{align*}
Then by Lemma \ref{lem:bound_moment_sym},
\[
B_{6.1}\leq\frac{C^{p}p^{3p}}{M^{p/2}}.
\]
\item We bound $B_{6.2}$:
\begin{align*}
B_{6.2} & \leq\frac{C^{p}}{M^{p}}\sum_{i=1}^{M}\mathbb{E}_{\boldsymbol{u},\boldsymbol{g},\tilde{\boldsymbol{Z}}}\left\{ \left|\left\langle \tilde{\boldsymbol{z}}_{i},\boldsymbol{D}_{t}^{3}\boldsymbol{u}\right\rangle \left\Vert \boldsymbol{D}_{t}\boldsymbol{u}\right\Vert _{2}\sigma\left(g_{i}\right)^{2}g_{i}\right|^{p}\right\} \\
 & \leq\frac{C^{p}}{M^{p}}\sum_{i=1}^{M}\sqrt{\mathbb{E}_{\boldsymbol{u},\tilde{\boldsymbol{Z}}}\left\{ \left|\left\langle \tilde{\boldsymbol{z}}_{i},\boldsymbol{D}_{t}^{3}\boldsymbol{u}\right\rangle \right|^{2p}\right\} \mathbb{E}_{\boldsymbol{u}}\left\{ \left\Vert \boldsymbol{u}\right\Vert _{2}^{2p}\right\} }\mathbb{E}_{\boldsymbol{g}}\left\{ \left|g_{i}\right|^{3p}\right\} \\
 & \leq\frac{C^{p}}{M^{p-1}}\left(1+\left(p/d\right)^{p}\right)p^{2p}\leq\frac{C^{p}}{M^{p-1}}p^{3p}.
\end{align*}
\item To bound $B_{6.3}$, let $B_{6.3,i,j}=\left\langle \tilde{\boldsymbol{z}}_{i},\boldsymbol{D}_{t}^{3}\boldsymbol{u}\right\rangle \left\Vert \boldsymbol{D}_{t}\boldsymbol{u}\right\Vert _{2}\sigma\left(g_{i}\right)\left(g_{j}\sigma\left(g_{j}\right)-0.5\right)$.
We have, for a fixed $i$, $\left(B_{6.3,i,j}\right)_{j\neq i,\;j\leq M}$
are independent conditional on $\tilde{\boldsymbol{Z}}$, $\boldsymbol{u}$
and $g_{i}$, and $\mathbb{E}\left\{ B_{6.3,i,j}\middle|\tilde{\boldsymbol{Z}},\boldsymbol{u},g_{i}\right\} =0$.
In addition,
\[
\mathbb{E}\left\{ \left|B_{6.3,i,j}\right|^{p}\right\} \leq C^{p}\sqrt{\mathbb{E}_{\boldsymbol{u},\tilde{\boldsymbol{Z}}}\left\{ \left|\left\langle \tilde{\boldsymbol{z}}_{i},\boldsymbol{D}_{t}^{3}\boldsymbol{u}\right\rangle \right|^{2p}\right\} \mathbb{E}_{\boldsymbol{u}}\left\{ \left\Vert \boldsymbol{u}\right\Vert _{2}^{2p}\right\} }\mathbb{E}_{\boldsymbol{g}}\left\{ \left|g_{i}\right|^{p}\right\} \left(\mathbb{E}_{\boldsymbol{g}}\left\{ \left|g_{j}\right|^{2p}\right\} +1\right)\leq C^{p}p^{3p}.
\]
Then by Lemma \ref{lem:bound_moment_sym},
\[
B_{6.3}\leq\frac{C^{p}}{M^{p}}\sum_{i=1}^{M}\mathbb{E}_{\boldsymbol{u},\boldsymbol{g},\tilde{\boldsymbol{Z}}}\left\{ \left|\sum_{j\neq i,\;j\leq M}B_{6.3,i,j}\right|^{p}\right\} \leq\frac{C^{p}p^{4p}}{M^{p/2-1}}.
\]
\end{itemize}
Combining the bounds, recalling $p$ is even, we thus get:
\[
A_{2,6,p}\leq\frac{C^{p}p^{4p}}{M^{p/2-1}}.
\]

\paragraph*{Step 2.7 - Finishing the concentration of ${\cal R}\left(\bar{\nu}_{M}^{t}\right)$.}

Collecting all the bounds in the previous steps, we then obtain:
\[
\mathbb{E}\left\{ \left|A_{2}-\frac{1}{4}\mathbb{E}_{\boldsymbol{u}}\left\{ \left\Vert \boldsymbol{D}_{t}^{2}\boldsymbol{u}\right\Vert _{2}^{2}\right\} -\frac{1}{2M}\left\Vert \boldsymbol{D}_{t}\right\Vert _{{\rm F}}^{2}\mathbb{E}_{\boldsymbol{u}}\left\{ \left\Vert \boldsymbol{D}_{t}\boldsymbol{u}\right\Vert _{2}^{2}\right\} \right|^{p}\right\} \leq\frac{C^{p}p^{6p}\left(1+d/M\right)^{p}}{M^{p/2-1}}.
\]
Recall that 
\begin{align*}
\mathbb{E}\left\{ A_{2}\right\}  & =\frac{1}{4d}\left\Vert \boldsymbol{D}_{t}^{2}\boldsymbol{D}_{\boldsymbol{\Sigma}}\right\Vert _{{\rm F}}^{2}+\frac{1}{2dM}\left\Vert \boldsymbol{D}_{t}\right\Vert _{{\rm F}}^{2}\left\Vert \boldsymbol{D}_{t}\boldsymbol{D}_{\boldsymbol{\Sigma}}\right\Vert _{{\rm F}}^{2}+O\left(\frac{\sqrt{d}}{M}\right)\\
 & =\frac{1}{4}\mathbb{E}_{\boldsymbol{u}}\left\{ \left\Vert \boldsymbol{D}_{t}^{2}\boldsymbol{u}\right\Vert _{2}^{2}\right\} +\frac{1}{2M}\left\Vert \boldsymbol{D}_{t}\right\Vert _{{\rm F}}^{2}\mathbb{E}_{\boldsymbol{u}}\left\{ \left\Vert \boldsymbol{D}_{t}\boldsymbol{u}\right\Vert _{2}^{2}\right\} +O\left(\frac{\sqrt{d}}{M}\right).
\end{align*}
We thus get
\[
\mathbb{E}\left\{ \left|A_{2}-\mathbb{E}\left\{ A_{2}\right\} \right|^{p}\right\} \leq\frac{C^{p}p^{6p}\left(1+d/M\right)^{p}}{M^{p/2-1}}.
\]
This bound applies to even $p$ and consequently odd $p$, since for
odd $p$:
\[
\mathbb{E}\left\{ \left|A_{2}-\mathbb{E}\left\{ A_{2}\right\} \right|^{p}\right\} \leq\mathbb{E}\left\{ \left|A_{2}-\mathbb{E}\left\{ A_{2}\right\} \right|^{p+1}\right\} ^{p/\left(p+1\right)}\leq\frac{C^{p}p^{6p}\left(1+d/M\right)^{p}}{M^{p/2-p/\left(p+1\right)}}\leq\frac{C^{p}p^{6p}\left(1+d/M\right)^{p}}{M^{p/2-1}}.
\]
With the same argument, for an arbitrary integer $m\geq1$, we have
for any $p\leq m$, 
\[
\mathbb{E}\left\{ \left|A_{2}-\mathbb{E}\left\{ A_{2}\right\} \right|^{p}\right\} \leq\mathbb{E}\left\{ \left|A_{2}-\mathbb{E}\left\{ A_{2}\right\} \right|^{m}\right\} ^{p/m}\leq\frac{C^{p}p^{6p}\left(1+d/M\right)^{p}}{M^{p/2-p/m}},
\]
and therefore,
\[
\max_{p\leq m,\;p\in\mathbb{N}_{>0}}\frac{1}{p}\mathbb{E}\left\{ \left|A_{2}-\mathbb{E}\left\{ A_{2}\right\} \right|^{p}\right\} ^{1/\left(6p\right)}\leq\max_{p\leq m,\;p\in\mathbb{N}_{>0}}\frac{C\left(1+d/M\right)^{1/6}}{M^{1/12-1/\left(6mp\right)}}=\frac{C\left(1+d/M\right)^{1/6}}{M^{1/12-1/\left(6m\right)}}.
\]
We also have:
\[
\sup_{p\geq m}\frac{1}{p}\mathbb{E}\left\{ \left|A_{2}-\mathbb{E}\left\{ A_{2}\right\} \right|^{p}\right\} ^{1/\left(6p\right)}\leq\sup_{p\geq m}\frac{C\left(1+d/M\right)^{1/6}}{M^{1/12-1/\left(6p\right)}}\leq\frac{C\left(1+d/M\right)^{1/6}}{M^{1/12-1/\left(6m\right)}}.
\]
Therefore,
\[
\sup_{p\in\mathbb{N}_{>0}}\frac{1}{p}\mathbb{E}\left\{ \left|A_{2}-\mathbb{E}\left\{ A_{2}\right\} \right|^{p}\right\} ^{1/\left(6p\right)}\leq\lim_{m\to\infty}\frac{C\left(1+d/M\right)^{1/6}}{M^{1/12-1/\left(6m\right)}}=\frac{C\left(1+d/M\right)^{1/6}}{M^{1/12}}.
\]
Hence $\left|A_{2}-\mathbb{E}\left\{ A_{2}\right\} \right|^{1/6}$
is sub-exponential with $\left\Vert \left|A_{2}-\mathbb{E}\left\{ A_{2}\right\} \right|^{1/6}\right\Vert _{\psi_{1}}\leq C\left(1+d/M\right)^{1/6}M^{-1/12}$,
which yields the following concentration bound by Lemma \ref{lem:subgauss_subexp_properties}:
\[
\mathbb{P}\left\{ \left|A_{2}-\mathbb{E}\left\{ A_{2}\right\} \right|\geq\delta\right\} \leq C\exp\left(-C\delta^{1/6}\left(1+d/M\right)^{-1/6}M^{1/12}\right).
\]
for any $\delta\in\left(0,1\right)$. Combining with the concentration
of $A_{1}$, we get:
\[
\mathbb{P}\left\{ \left|{\cal R}\left(\bar{\nu}_{M}^{t}\right)-\mathbb{E}\left\{ {\cal R}\left(\bar{\nu}_{M}^{t}\right)\right\} \right|\geq\delta\right\} \leq C\exp\left(-C\delta^{1/6}\left(1+d/M\right)^{-1/6}M^{1/12}\right).
\]
This completes the proof.
\end{proof}

\subsection{Setting with ReLU activation: Proofs of auxiliary results\label{subsec:Proof_relu_act_aux}}
\begin{prop}
\label{prop:1st_setting_grow_bound}Consider setting \ref{enu:ReLU_setting}.
The following hold:
\begin{align*}
\left\Vert \nabla V\left(\boldsymbol{\theta}\right)\right\Vert _{2} & \leq C\left\Vert \boldsymbol{\theta}\right\Vert _{2},\\
\left\Vert \nabla V\left(\boldsymbol{\theta}_{1}\right)-\nabla V\left(\boldsymbol{\theta}_{2}\right)\right\Vert _{2} & \leq C\left\Vert \boldsymbol{\theta}_{1}-\boldsymbol{\theta}_{2}\right\Vert _{2},\\
\left\Vert \nabla_{1}W\left(\boldsymbol{\theta};\rho\right)\right\Vert _{2} & \leq C\left\Vert \boldsymbol{\theta}\right\Vert _{2},\\
\left\Vert \nabla_{1}W\left(\boldsymbol{\theta}_{1};\rho\right)-\nabla_{1}W\left(\boldsymbol{\theta}_{2};\rho\right)\right\Vert _{2} & \leq C\left\Vert \boldsymbol{\theta}_{1}-\boldsymbol{\theta}_{2}\right\Vert _{2},\\
\left\Vert \nabla_{1}U\left(\boldsymbol{\theta},\boldsymbol{\theta}'\right)\right\Vert _{2} & \leq C\kappa^{2}\left\Vert \boldsymbol{\theta}\right\Vert _{2}\left\Vert \boldsymbol{\theta}'\right\Vert _{2}^{2},
\end{align*}
where $\rho=\mathsf{N}\left(0,\boldsymbol{R}{\rm diag}\left(r_{1}^{2},...,r_{d}^{2}\right)\boldsymbol{R}^{\top}/d\right)$
with $\max_{i\leq d}r_{i}^{2}\leq C$. Furthermore, $\left|V\left(\boldsymbol{0}\right)\right|=\left|U\left(\boldsymbol{0},\boldsymbol{0}\right)\right|=\left|W\left(\boldsymbol{0};\rho\right)\right|=0$
for any $\rho$.
\end{prop}

\begin{proof}
With the given $\rho$, we have from Stein's lemma:
\[
\int\kappa\boldsymbol{\theta}\sigma\left(\left\langle \kappa\boldsymbol{\theta},\boldsymbol{x}\right\rangle \right)\rho\left({\rm d}\boldsymbol{\theta}\right)=\frac{1}{2}\boldsymbol{R}{\rm diag}\left(r_{1}^{2},...,r_{d}^{2}\right)\boldsymbol{R}^{\top}\boldsymbol{x}.
\]
This yields, again by Stein's lemma,
\begin{align*}
W\left(\boldsymbol{\theta};\rho\right) & =\mathbb{E}_{{\cal P}}\left\{ \left\langle \kappa\boldsymbol{\theta}\sigma\left(\left\langle \kappa\boldsymbol{\theta},\boldsymbol{x}\right\rangle \right),\int\kappa\boldsymbol{\theta}'\sigma\left(\left\langle \kappa\boldsymbol{\theta}',\boldsymbol{x}\right\rangle \right)\rho\left({\rm d}\boldsymbol{\theta}'\right)\right\rangle \right\} \\
 & =\mathbb{E}_{{\cal P}}\left\{ \left\langle \kappa\boldsymbol{\theta}\sigma\left(\left\langle \kappa\boldsymbol{\theta},\boldsymbol{x}\right\rangle \right),\frac{1}{2}\boldsymbol{R}{\rm diag}\left(r_{1}^{2},...,r_{d}^{2}\right)\boldsymbol{R}^{\top}\boldsymbol{x}\right\rangle \right\} \\
 & =\frac{1}{4}\left\Vert {\rm diag}\left(r_{1}\Sigma_{1},...,r_{d}\Sigma_{d}\right)\boldsymbol{R}^{\top}\boldsymbol{\theta}\right\Vert _{2}^{2}.
\end{align*}
One can also compute $V\left(\boldsymbol{\theta}\right)$:
\[
\mathbb{E}_{{\cal P}}\left\{ \left\langle \kappa\boldsymbol{\theta},\boldsymbol{x}\right\rangle \sigma\left(\left\langle \kappa\boldsymbol{\theta},\boldsymbol{x}\right\rangle \right)\right\} =\frac{1}{2}\left\Vert {\rm diag}\left(\Sigma_{1},...,\Sigma_{d}\right)\boldsymbol{R}^{\top}\boldsymbol{\theta}\right\Vert _{2}^{2},
\]
which yields
\[
V\left(\boldsymbol{\theta}\right)=-\frac{1}{2}\left\Vert {\rm diag}\left(\Sigma_{1},...,\Sigma_{d}\right)\boldsymbol{R}^{\top}\boldsymbol{\theta}\right\Vert _{2}^{2}+\lambda\left\Vert \boldsymbol{\theta}\right\Vert _{2}^{2}=-\frac{1}{2}\left\Vert \boldsymbol{\Sigma}\boldsymbol{\theta}\right\Vert _{2}^{2}+\lambda\left\Vert \boldsymbol{\theta}\right\Vert _{2}^{2}.
\]
Therefore:
\begin{align*}
\nabla V\left(\boldsymbol{\theta}\right) & =-\boldsymbol{\Sigma}^{2}\boldsymbol{\theta}+2\lambda\boldsymbol{\theta},\\
\nabla_{1}W\left(\boldsymbol{\theta};\rho\right) & =\frac{1}{2}\boldsymbol{R}{\rm diag}\left(r_{1}^{2}\Sigma_{1}^{2},...,r_{d}^{2}\Sigma_{d}^{2}\right)\boldsymbol{R}^{\top}\boldsymbol{\theta}.
\end{align*}
Since $\left\Vert \boldsymbol{\Sigma}\right\Vert _{{\rm op}}\leq C$,
one easily deduces the claims on $\nabla V$ and $\nabla_{1}W$.

Next we consider $U$:
\[
\nabla_{1}U\left(\boldsymbol{\theta},\boldsymbol{\theta}'\right)=\mathbb{E}_{{\cal P}}\left\{ \kappa^{2}\boldsymbol{\theta}'\sigma\left(\left\langle \kappa\boldsymbol{\theta},\boldsymbol{x}\right\rangle \right)\sigma\left(\left\langle \kappa\boldsymbol{\theta}',\boldsymbol{x}\right\rangle \right)\right\} +\mathbb{E}_{{\cal P}}\left\{ \kappa^{3}\left\langle \boldsymbol{\theta},\boldsymbol{\theta}'\right\rangle \sigma'\left(\left\langle \kappa\boldsymbol{\theta},\boldsymbol{x}\right\rangle \right)\sigma\left(\left\langle \kappa\boldsymbol{\theta}',\boldsymbol{x}\right\rangle \right)\boldsymbol{x}\right\} .
\]
We give a bound on $\left\Vert \nabla_{1}U\left(\boldsymbol{\theta},\boldsymbol{\theta}'\right)\right\Vert _{2}$.
For the first term:
\begin{align*}
\left\Vert \mathbb{E}_{{\cal P}}\left\{ \kappa^{2}\boldsymbol{\theta}'\sigma\left(\left\langle \kappa\boldsymbol{\theta},\boldsymbol{x}\right\rangle \right)\sigma\left(\left\langle \kappa\boldsymbol{\theta}',\boldsymbol{x}\right\rangle \right)\right\} \right\Vert _{2} & \leq\kappa^{2}\sqrt{\mathbb{E}_{{\cal P}}\left\{ \sigma\left(\left\langle \kappa\boldsymbol{\theta},\boldsymbol{x}\right\rangle \right)^{2}\right\} \mathbb{E}_{{\cal P}}\left\{ \sigma\left(\left\langle \kappa\boldsymbol{\theta}',\boldsymbol{x}\right\rangle \right)^{2}\right\} }\left\Vert \boldsymbol{\theta}'\right\Vert \\
 & =\kappa^{2}\sqrt{\mathbb{E}\left\{ \sigma\left(\left\Vert \boldsymbol{\Sigma}\boldsymbol{\theta}\right\Vert _{2}g\right)^{2}\right\} \mathbb{E}\left\{ \sigma\left(\left\Vert \boldsymbol{\Sigma}\boldsymbol{\theta}'\right\Vert _{2}g\right)^{2}\right\} }\left\Vert \boldsymbol{\theta}'\right\Vert \\
 & \leq C\kappa^{2}\left\Vert \boldsymbol{\theta}\right\Vert _{2}\left\Vert \boldsymbol{\theta}'\right\Vert _{2}^{2}.
\end{align*}
Denoting the second term $\boldsymbol{v}$, we have:
\begin{align*}
\left\Vert \boldsymbol{v}\right\Vert _{2}^{2} & =\mathbb{E}_{{\cal P}}\left\{ \kappa^{2}\left\langle \boldsymbol{\theta},\boldsymbol{\theta}'\right\rangle \sigma'\left(\left\langle \kappa\boldsymbol{\theta},\boldsymbol{x}\right\rangle \right)\sigma\left(\left\langle \kappa\boldsymbol{\theta}',\boldsymbol{x}\right\rangle \right)\left\langle \kappa\boldsymbol{v},\boldsymbol{x}\right\rangle \right\} \\
 & \leq\kappa^{2}\left\Vert \boldsymbol{\theta}\right\Vert _{2}\left\Vert \boldsymbol{\theta}'\right\Vert _{2}\left(\mathbb{P}\left\{ \left\langle \kappa\boldsymbol{\theta},\boldsymbol{x}\right\rangle \geq0\right\} \mathbb{E}_{{\cal P}}\left\{ \sigma\left(\left\langle \kappa\boldsymbol{\theta}',\boldsymbol{x}\right\rangle \right)^{3}\right\} \mathbb{E}_{{\cal P}}\left\{ \left|\left\langle \kappa\boldsymbol{v},\boldsymbol{x}\right\rangle \right|^{3}\right\} \right)^{1/3}\\
 & =\kappa^{2}\left\Vert \boldsymbol{\theta}\right\Vert _{2}\left\Vert \boldsymbol{\theta}'\right\Vert _{2}\left(\frac{1}{2}\mathbb{E}\left\{ \sigma\left(\left\Vert \boldsymbol{\Sigma}\boldsymbol{\theta}'\right\Vert _{2}g\right)^{3}\right\} \mathbb{E}\left\{ \left|\left\Vert \boldsymbol{\Sigma}\boldsymbol{v}\right\Vert _{2}g\right|^{3}\right\} \right)^{1/3}\\
 & \leq C\kappa^{2}\left\Vert \boldsymbol{\theta}\right\Vert _{2}\left\Vert \boldsymbol{\theta}'\right\Vert _{2}^{2}\left\Vert \boldsymbol{v}\right\Vert _{2},
\end{align*}
which then yields
\[
\left\Vert \nabla_{1}U\left(\boldsymbol{\theta},\boldsymbol{\theta}'\right)\right\Vert _{2}\leq C\kappa^{2}\left\Vert \boldsymbol{\theta}\right\Vert _{2}\left\Vert \boldsymbol{\theta}'\right\Vert _{2}^{2}.
\]

Lastly, it is easy to see that $V\left(\boldsymbol{0}\right)=U\left(\boldsymbol{0},\boldsymbol{0}\right)=W\left(\boldsymbol{0};\rho\right)=0$
for any $\rho$.
\end{proof}
\begin{prop}
\label{prop:1st_setting_grow_bound_Wrho}Consider setting \ref{enu:ReLU_setting}.
Then:
\[
\left\Vert \nabla_{1}W\left(\boldsymbol{\theta};\rho_{1}\right)-\nabla_{1}W\left(\boldsymbol{\theta};\rho_{2}\right)\right\Vert _{2}\leq C\left\Vert \boldsymbol{\theta}\right\Vert _{2}\max_{i\in\left[d\right]}\left|r_{i,1}-r_{i,2}\right|
\]
where $\rho_{j}=\mathsf{N}\left(0,\boldsymbol{R}{\rm diag}\left(r_{1,j}^{2},...,r_{d,j}^{2}\right)\boldsymbol{R}^{\top}/d\right)$,
$j=1,2$, with $\max_{i\leq d,\;j\in\left\{ 1,2\right\} }r_{i,j}^{2}\leq C$.
\end{prop}

\begin{proof}
The claim follows easily from the following formula given in the proof
of Proposition \ref{prop:1st_setting_grow_bound}:
\[
\nabla_{1}W\left(\boldsymbol{\theta};\rho_{j}\right)=\frac{1}{2}\boldsymbol{R}{\rm diag}\left(r_{1,j}^{2}\Sigma_{1}^{2},...,r_{d,j}^{2}\Sigma_{d}^{2}\right)\boldsymbol{R}^{\top}\boldsymbol{\theta},\qquad j=1,2,
\]
along with the fact $\left\Vert \boldsymbol{\Sigma}\right\Vert _{{\rm op}}\leq C$.
\end{proof}
\begin{prop}
\label{prop:1st_setting_op_bound}Consider setting \ref{enu:ReLU_setting}.
We have:
\begin{align*}
\left\Vert \nabla_{121}^{3}U\left[\boldsymbol{\zeta},\boldsymbol{\theta}\right]\right\Vert _{{\rm op}},\left\Vert \nabla_{122}^{3}U\left[\boldsymbol{\theta},\boldsymbol{\zeta}\right]\right\Vert _{{\rm op}} & \leq C\frac{\kappa^{2}}{\kappa_{*}}\left\Vert \boldsymbol{\theta}\right\Vert _{2},\\
\left\Vert \nabla_{12}^{2}U\left(\boldsymbol{\theta},\boldsymbol{\theta}'\right)\right\Vert _{{\rm op}} & \leq C\kappa^{2}\left\Vert \boldsymbol{\theta}\right\Vert _{2}\left\Vert \boldsymbol{\theta}'\right\Vert _{2},\\
\left\Vert \nabla_{11}^{2}U\left(\boldsymbol{\theta},\boldsymbol{\theta}'\right)\right\Vert _{{\rm op}} & \leq C\frac{\kappa^{2}}{\kappa_{*}}\left\Vert \boldsymbol{\theta}'\right\Vert _{2}^{2},
\end{align*}
for any $\boldsymbol{\zeta},\boldsymbol{\theta},\boldsymbol{\theta}'\in\mathbb{R}^{d}$.
\end{prop}

\begin{proof}
We have:
\begin{align*}
\nabla_{12}^{2}U\left(\boldsymbol{\theta},\boldsymbol{\theta}'\right) & =\kappa^{2}\mathbb{E}_{{\cal P}}\left\{ \sigma\left(\left\langle \kappa\boldsymbol{\theta},\boldsymbol{x}\right\rangle \right)\sigma\left(\left\langle \kappa\boldsymbol{\theta}',\boldsymbol{x}\right\rangle \right)\right\} \boldsymbol{I}_{d}\\
 & \qquad+\kappa^{3}\mathbb{E}_{{\cal P}}\left\{ \sigma\left(\left\langle \kappa\boldsymbol{\theta},\boldsymbol{x}\right\rangle \right)\sigma'\left(\left\langle \kappa\boldsymbol{\theta}',\boldsymbol{x}\right\rangle \right)\boldsymbol{x}\boldsymbol{\theta}'^{\top}\right\} \\
 & \qquad+\kappa^{3}\mathbb{E}_{{\cal P}}\left\{ \sigma'\left(\left\langle \kappa\boldsymbol{\theta},\boldsymbol{x}\right\rangle \right)\sigma\left(\left\langle \kappa\boldsymbol{\theta}',\boldsymbol{x}\right\rangle \right)\boldsymbol{\theta}\boldsymbol{x}^{\top}\right\} \\
 & \qquad+\kappa^{4}\mathbb{E}_{{\cal P}}\left\{ \left\langle \boldsymbol{\theta},\boldsymbol{\theta}'\right\rangle \sigma'\left(\left\langle \kappa\boldsymbol{\theta},\boldsymbol{x}\right\rangle \right)\sigma'\left(\left\langle \kappa\boldsymbol{\theta}',\boldsymbol{x}\right\rangle \right)\boldsymbol{x}\boldsymbol{x}^{\top}\right\} ,\\
\nabla_{11}^{2}U\left(\boldsymbol{\theta},\boldsymbol{\theta}'\right) & =\kappa^{3}\mathbb{E}_{{\cal P}}\left\{ \sigma'\left(\left\langle \kappa\boldsymbol{\theta},\boldsymbol{x}\right\rangle \right)\sigma\left(\left\langle \kappa\boldsymbol{\theta}',\boldsymbol{x}\right\rangle \right)\left(\boldsymbol{\theta}'\boldsymbol{x}^{\top}+\boldsymbol{x}\boldsymbol{\theta}'^{\top}\right)\right\} \\
 & \qquad+\kappa^{4}\mathbb{E}_{{\cal P}}\left\{ \left\langle \boldsymbol{\theta},\boldsymbol{\theta}'\right\rangle \sigma''\left(\left\langle \kappa\boldsymbol{\theta},\boldsymbol{x}\right\rangle \right)\sigma\left(\left\langle \kappa\boldsymbol{\theta}',\boldsymbol{x}\right\rangle \right)\boldsymbol{x}\boldsymbol{x}^{\top}\right\} .
\end{align*}
Therefore, for $\boldsymbol{a},\boldsymbol{b},\boldsymbol{c}\in\mathbb{R}^{d}$,
\begin{align*}
\left\langle \nabla_{121}^{3}U\left[\boldsymbol{\zeta},\boldsymbol{\theta}\right],\boldsymbol{a}\otimes\boldsymbol{b}\otimes\boldsymbol{c}\right\rangle  & =\kappa^{3}\mathbb{E}_{{\cal P}}\left\{ \sigma'\left(\left\langle \kappa\boldsymbol{\zeta},\boldsymbol{x}\right\rangle \right)\sigma\left(\left\langle \kappa\boldsymbol{\theta},\boldsymbol{x}\right\rangle \right)\left\langle \boldsymbol{a},\boldsymbol{x}\right\rangle \left\langle \boldsymbol{b},\boldsymbol{c}\right\rangle \right\} \\
 & \qquad+\kappa^{4}\mathbb{E}_{{\cal P}}\left\{ \sigma'\left(\left\langle \kappa\boldsymbol{\zeta},\boldsymbol{x}\right\rangle \right)\sigma'\left(\left\langle \kappa\boldsymbol{\theta},\boldsymbol{x}\right\rangle \right)\left\langle \boldsymbol{a},\boldsymbol{x}\right\rangle \left\langle \boldsymbol{b},\boldsymbol{x}\right\rangle \left\langle \boldsymbol{c},\boldsymbol{\theta}\right\rangle \right\} \\
 & \qquad+\kappa^{4}\mathbb{E}_{{\cal P}}\left\{ \sigma''\left(\left\langle \kappa\boldsymbol{\zeta},\boldsymbol{x}\right\rangle \right)\sigma\left(\left\langle \kappa\boldsymbol{\theta},\boldsymbol{x}\right\rangle \right)\left\langle \boldsymbol{a},\boldsymbol{x}\right\rangle \left\langle \boldsymbol{b},\boldsymbol{\zeta}\right\rangle \left\langle \boldsymbol{c},\boldsymbol{x}\right\rangle \right\} \\
 & \qquad+\kappa^{3}\mathbb{E}_{{\cal P}}\left\{ \sigma'\left(\left\langle \kappa\boldsymbol{\zeta},\boldsymbol{x}\right\rangle \right)\sigma\left(\left\langle \kappa\boldsymbol{\theta},\boldsymbol{x}\right\rangle \right)\left\langle \boldsymbol{a},\boldsymbol{b}\right\rangle \left\langle \boldsymbol{c},\boldsymbol{x}\right\rangle \right\} \\
 & \qquad+\kappa^{4}\mathbb{E}_{{\cal P}}\left\{ \sigma'\left(\left\langle \kappa\boldsymbol{\zeta},\boldsymbol{x}\right\rangle \right)\sigma'\left(\left\langle \kappa\boldsymbol{\theta},\boldsymbol{x}\right\rangle \right)\left\langle \boldsymbol{a},\boldsymbol{\theta}\right\rangle \left\langle \boldsymbol{b},\boldsymbol{x}\right\rangle \left\langle \boldsymbol{c},\boldsymbol{x}\right\rangle \right\} \\
 & \qquad+\kappa^{5}\mathbb{E}_{{\cal P}}\left\{ \left\langle \boldsymbol{\zeta},\boldsymbol{\theta}\right\rangle \sigma''\left(\left\langle \kappa\boldsymbol{\zeta},\boldsymbol{x}\right\rangle \right)\sigma'\left(\left\langle \kappa\boldsymbol{\theta},\boldsymbol{x}\right\rangle \right)\left\langle \boldsymbol{a},\boldsymbol{x}\right\rangle \left\langle \boldsymbol{b},\boldsymbol{x}\right\rangle \left\langle \boldsymbol{c},\boldsymbol{x}\right\rangle \right\} \\
 & \equiv A_{1}+A_{2}+A_{3}+A_{4}+A_{5}+A_{6},\\
\left\langle \nabla_{122}^{3}U\left[\boldsymbol{\theta},\boldsymbol{\zeta}\right],\boldsymbol{a}\otimes\boldsymbol{b}\otimes\boldsymbol{c}\right\rangle  & =\kappa^{3}\mathbb{E}_{{\cal P}}\left\{ \sigma\left(\left\langle \kappa\boldsymbol{\theta},\boldsymbol{x}\right\rangle \right)\sigma'\left(\left\langle \kappa\boldsymbol{\zeta},\boldsymbol{x}\right\rangle \right)\left\langle \boldsymbol{a},\boldsymbol{x}\right\rangle \left\langle \boldsymbol{b},\boldsymbol{c}\right\rangle \right\} \\
 & \qquad+\kappa^{4}\mathbb{E}_{{\cal P}}\left\{ \sigma\left(\left\langle \kappa\boldsymbol{\theta},\boldsymbol{x}\right\rangle \right)\sigma''\left(\left\langle \kappa\boldsymbol{\zeta},\boldsymbol{x}\right\rangle \right)\left\langle \boldsymbol{a},\boldsymbol{x}\right\rangle \left\langle \boldsymbol{b},\boldsymbol{x}\right\rangle \left\langle \boldsymbol{c},\boldsymbol{\zeta}\right\rangle \right\} \\
 & \qquad+\kappa^{3}\mathbb{E}_{{\cal P}}\left\{ \sigma\left(\left\langle \kappa\boldsymbol{\theta},\boldsymbol{x}\right\rangle \right)\sigma'\left(\left\langle \kappa\boldsymbol{\zeta},\boldsymbol{x}\right\rangle \right)\left\langle \boldsymbol{a},\boldsymbol{c}\right\rangle \left\langle \boldsymbol{b},\boldsymbol{x}\right\rangle \right\} \\
 & \qquad+\kappa^{4}\mathbb{E}_{{\cal P}}\left\{ \sigma'\left(\left\langle \kappa\boldsymbol{\theta},\boldsymbol{x}\right\rangle \right)\sigma'\left(\left\langle \kappa\boldsymbol{\zeta},\boldsymbol{x}\right\rangle \right)\left\langle \boldsymbol{a},\boldsymbol{x}\right\rangle \left\langle \boldsymbol{b},\boldsymbol{\theta}\right\rangle \left\langle \boldsymbol{c},\boldsymbol{x}\right\rangle \right\} \\
 & \qquad+\kappa^{4}\mathbb{E}_{{\cal P}}\left\{ \sigma'\left(\left\langle \kappa\boldsymbol{\theta},\boldsymbol{x}\right\rangle \right)\sigma'\left(\left\langle \kappa\boldsymbol{\zeta},\boldsymbol{x}\right\rangle \right)\left\langle \boldsymbol{a},\boldsymbol{\theta}\right\rangle \left\langle \boldsymbol{b},\boldsymbol{x}\right\rangle \left\langle \boldsymbol{c},\boldsymbol{x}\right\rangle \right\} \\
 & \qquad+\kappa^{5}\mathbb{E}_{{\cal P}}\left\{ \left\langle \boldsymbol{\theta},\boldsymbol{\zeta}\right\rangle \sigma'\left(\left\langle \kappa\boldsymbol{\theta},\boldsymbol{x}\right\rangle \right)\sigma''\left(\left\langle \kappa\boldsymbol{\zeta},\boldsymbol{x}\right\rangle \right)\left\langle \boldsymbol{a},\boldsymbol{x}\right\rangle \left\langle \boldsymbol{b},\boldsymbol{x}\right\rangle \left\langle \boldsymbol{c},\boldsymbol{x}\right\rangle \right\} \\
 & \equiv B_{1}+B_{2}+B_{3}+B_{4}+B_{5}+B_{6},\\
\left\langle \boldsymbol{a},\nabla_{12}^{2}U\left(\boldsymbol{\theta},\boldsymbol{\theta}'\right)\boldsymbol{b}\right\rangle  & =\kappa^{2}\mathbb{E}_{{\cal P}}\left\{ \sigma\left(\left\langle \kappa\boldsymbol{\theta},\boldsymbol{x}\right\rangle \right)\sigma\left(\left\langle \kappa\boldsymbol{\theta}',\boldsymbol{x}\right\rangle \right)\right\} \left\langle \boldsymbol{a},\boldsymbol{b}\right\rangle \\
 & \qquad+\kappa^{3}\mathbb{E}_{{\cal P}}\left\{ \sigma\left(\left\langle \kappa\boldsymbol{\theta},\boldsymbol{x}\right\rangle \right)\sigma'\left(\left\langle \kappa\boldsymbol{\theta}',\boldsymbol{x}\right\rangle \right)\left\langle \boldsymbol{a},\boldsymbol{x}\right\rangle \left\langle \boldsymbol{b},\boldsymbol{\theta}'\right\rangle \right\} \\
 & \qquad+\kappa^{3}\mathbb{E}_{{\cal P}}\left\{ \sigma'\left(\left\langle \kappa\boldsymbol{\theta},\boldsymbol{x}\right\rangle \right)\sigma\left(\left\langle \kappa\boldsymbol{\theta}',\boldsymbol{x}\right\rangle \right)\left\langle \boldsymbol{a},\boldsymbol{\theta}\right\rangle \left\langle \boldsymbol{b},\boldsymbol{x}\right\rangle \right\} \\
 & \qquad+\kappa^{4}\mathbb{E}_{{\cal P}}\left\{ \left\langle \boldsymbol{\theta},\boldsymbol{\theta}'\right\rangle \sigma'\left(\left\langle \kappa\boldsymbol{\theta},\boldsymbol{x}\right\rangle \right)\sigma'\left(\left\langle \kappa\boldsymbol{\theta}',\boldsymbol{x}\right\rangle \right)\left\langle \boldsymbol{a},\boldsymbol{x}\right\rangle \left\langle \boldsymbol{b},\boldsymbol{x}\right\rangle \right\} \\
 & \equiv F_{1}+F_{2}+F_{3}+F_{4},\\
\left\langle \boldsymbol{a},\nabla_{11}^{2}U\left(\boldsymbol{\theta},\boldsymbol{\theta}'\right)\boldsymbol{b}\right\rangle  & =\kappa^{3}\mathbb{E}_{{\cal P}}\left\{ \sigma'\left(\left\langle \kappa\boldsymbol{\theta},\boldsymbol{x}\right\rangle \right)\sigma\left(\left\langle \kappa\boldsymbol{\theta}',\boldsymbol{x}\right\rangle \right)\left(\left\langle \boldsymbol{a},\boldsymbol{\theta}'\right\rangle \left\langle \boldsymbol{b},\boldsymbol{x}\right\rangle +\left\langle \boldsymbol{b},\boldsymbol{\theta}'\right\rangle \left\langle \boldsymbol{a},\boldsymbol{x}\right\rangle \right)\right\} \\
 & \qquad+\kappa^{4}\mathbb{E}_{{\cal P}}\left\{ \left\langle \boldsymbol{\theta},\boldsymbol{\theta}'\right\rangle \sigma''\left(\left\langle \kappa\boldsymbol{\theta},\boldsymbol{x}\right\rangle \right)\sigma\left(\left\langle \kappa\boldsymbol{\theta}',\boldsymbol{x}\right\rangle \right)\left\langle \boldsymbol{a},\boldsymbol{x}\right\rangle \left\langle \boldsymbol{b},\boldsymbol{x}\right\rangle \right\} \\
 & \equiv H_{1}+H_{2}.
\end{align*}
Let us consider $A_{1}$:
\begin{align*}
\left|A_{1}\right| & \leq\kappa^{2}\mathbb{E}_{{\cal P}}\left\{ \left|\sigma'\left(\left\langle \kappa\boldsymbol{\zeta},\boldsymbol{x}\right\rangle \right)\right|^{3}\right\} ^{1/3}\mathbb{E}_{{\cal P}}\left\{ \sigma\left(\left\langle \kappa\boldsymbol{\theta},\boldsymbol{x}\right\rangle \right)^{3}\right\} ^{1/3}\mathbb{E}_{{\cal P}}\left\{ \left|\left\langle \kappa\boldsymbol{a},\boldsymbol{x}\right\rangle \right|^{3}\right\} ^{1/3}\left|\left\langle \boldsymbol{b},\boldsymbol{c}\right\rangle \right|\\
 & \leq C\kappa^{2}\left\Vert \boldsymbol{\Sigma}\boldsymbol{\theta}\right\Vert _{2}\left\Vert \boldsymbol{\Sigma}\boldsymbol{a}\right\Vert _{2}\left\Vert \boldsymbol{b}\right\Vert _{2}\left\Vert \boldsymbol{c}\right\Vert _{2}\\
 & \leq C\kappa^{2}\left\Vert \boldsymbol{\theta}\right\Vert _{2}\left\Vert \boldsymbol{a}\right\Vert _{2}\left\Vert \boldsymbol{b}\right\Vert _{2}\left\Vert \boldsymbol{c}\right\Vert _{2}.
\end{align*}
One can perform similar calculations to obtain:
\begin{align*}
\left|A_{1}\right|,\left|A_{2}\right|,\left|A_{4}\right|,\left|A_{5}\right|,\left|B_{1}\right|,\left|B_{3}\right|,\left|B_{4}\right|,\left|B_{5}\right| & \leq C\kappa^{2}\left\Vert \boldsymbol{\theta}\right\Vert _{2}\left\Vert \boldsymbol{a}\right\Vert _{2}\left\Vert \boldsymbol{b}\right\Vert _{2}\left\Vert \boldsymbol{c}\right\Vert _{2},\\
\left|F_{1}\right|,\left|F_{2}\right|,\left|F_{3}\right|,\left|F_{4}\right| & \leq C\kappa^{2}\left\Vert \boldsymbol{\theta}\right\Vert _{2}\left\Vert \boldsymbol{\theta}'\right\Vert _{2}\left\Vert \boldsymbol{a}\right\Vert _{2}\left\Vert \boldsymbol{b}\right\Vert _{2},\\
\left|H_{1}\right| & \leq C\kappa^{2}\left\Vert \boldsymbol{\theta}'\right\Vert _{2}^{2}\left\Vert \boldsymbol{a}\right\Vert _{2}\left\Vert \boldsymbol{b}\right\Vert _{2},
\end{align*}
for a suitable constant $C$. We are left with $A_{3}$, $A_{6}$,
$B_{2}$, $B_{6}$ and $H_{2}$. Consider $A_{3}$:
\[
A_{3}=\kappa^{2}\mathbb{E}_{\boldsymbol{z}}\left\{ \sigma''\left(\left\langle \boldsymbol{\Sigma}\boldsymbol{\zeta},\boldsymbol{z}\right\rangle \right)\sigma\left(\left\langle \boldsymbol{\Sigma}\boldsymbol{\theta},\boldsymbol{z}\right\rangle \right)\left\langle \boldsymbol{\Sigma}\boldsymbol{a},\boldsymbol{z}\right\rangle \left\langle \boldsymbol{b},\boldsymbol{\zeta}\right\rangle \left\langle \boldsymbol{\Sigma}\boldsymbol{c},\boldsymbol{z}\right\rangle \right\} ,
\]
for $\boldsymbol{z}\sim\mathsf{N}\left(0,\boldsymbol{I}_{d}\right)$.
Notice that for $w=\left\langle \boldsymbol{\Sigma}\boldsymbol{\zeta},\boldsymbol{z}\right\rangle \sim\mathsf{N}\left(0,\left\Vert \boldsymbol{\Sigma}\boldsymbol{\zeta}\right\Vert _{2}^{2}\right)$,
\[
\left(w,\boldsymbol{z}\right)\stackrel{{\rm d}}{=}\left(w,{\rm Proj}_{\boldsymbol{\Sigma}\boldsymbol{\zeta}}^{\perp}\tilde{\boldsymbol{z}}+\frac{w}{\left\Vert \boldsymbol{\Sigma}\boldsymbol{\zeta}\right\Vert _{2}^{2}}\boldsymbol{\Sigma}\boldsymbol{\zeta}\right),
\]
for $\tilde{\boldsymbol{z}}\sim\mathsf{N}\left(0,\boldsymbol{I}_{d}\right)$
independent of $w$. Therefore, using the fact $\sigma''\left(\cdot\right)=\delta\left(\cdot\right)$,
it is easy to see that:
\begin{align*}
A_{3} & =\kappa^{2}\left\langle \boldsymbol{b},\boldsymbol{\zeta}\right\rangle \mathbb{E}_{w,\tilde{\boldsymbol{z}}}\left\{ \sigma''\left(w\right)\sigma\left(\left\langle \boldsymbol{\Sigma}\boldsymbol{\theta},{\rm Proj}_{\boldsymbol{\Sigma}\boldsymbol{\zeta}}^{\perp}\tilde{\boldsymbol{z}}+\frac{w}{\left\Vert \boldsymbol{\Sigma}\boldsymbol{\zeta}\right\Vert _{2}^{2}}\boldsymbol{\Sigma}\boldsymbol{\zeta}\right\rangle \right)\left\langle \boldsymbol{\Sigma}\boldsymbol{a},{\rm Proj}_{\boldsymbol{\Sigma}\boldsymbol{\zeta}}^{\perp}\tilde{\boldsymbol{z}}\right\rangle \left\langle \boldsymbol{\Sigma}\boldsymbol{c},{\rm Proj}_{\boldsymbol{\Sigma}\boldsymbol{\zeta}}^{\perp}\tilde{\boldsymbol{z}}\right\rangle \right\} \\
 & \qquad+\kappa^{2}\left\langle \boldsymbol{b},\boldsymbol{\zeta}\right\rangle \mathbb{E}_{w,\tilde{\boldsymbol{z}}}\left\{ \sigma''\left(w\right)\sigma\left(\left\langle \boldsymbol{\Sigma}\boldsymbol{\theta},{\rm Proj}_{\boldsymbol{\Sigma}\boldsymbol{\zeta}}^{\perp}\tilde{\boldsymbol{z}}+\frac{w}{\left\Vert \boldsymbol{\Sigma}\boldsymbol{\zeta}\right\Vert _{2}^{2}}\boldsymbol{\Sigma}\boldsymbol{\zeta}\right\rangle \right)\left\langle \boldsymbol{\Sigma}\boldsymbol{a},\frac{w}{\left\Vert \boldsymbol{\Sigma}\boldsymbol{\zeta}\right\Vert _{2}^{2}}\boldsymbol{\Sigma}\boldsymbol{\zeta}\right\rangle \left\langle \boldsymbol{\Sigma}\boldsymbol{c},{\rm Proj}_{\boldsymbol{\Sigma}\boldsymbol{\zeta}}^{\perp}\tilde{\boldsymbol{z}}\right\rangle \right\} \\
 & \qquad+\kappa^{2}\left\langle \boldsymbol{b},\boldsymbol{\zeta}\right\rangle \mathbb{E}_{w,\tilde{\boldsymbol{z}}}\left\{ \sigma''\left(w\right)\sigma\left(\left\langle \boldsymbol{\Sigma}\boldsymbol{\theta},{\rm Proj}_{\boldsymbol{\Sigma}\boldsymbol{\zeta}}^{\perp}\tilde{\boldsymbol{z}}+\frac{w}{\left\Vert \boldsymbol{\Sigma}\boldsymbol{\zeta}\right\Vert _{2}^{2}}\boldsymbol{\Sigma}\boldsymbol{\zeta}\right\rangle \right)\left\langle \boldsymbol{\Sigma}\boldsymbol{a},{\rm Proj}_{\boldsymbol{\Sigma}\boldsymbol{\zeta}}^{\perp}\tilde{\boldsymbol{z}}\right\rangle \left\langle \boldsymbol{\Sigma}\boldsymbol{c},\frac{w}{\left\Vert \boldsymbol{\Sigma}\boldsymbol{\zeta}\right\Vert _{2}^{2}}\boldsymbol{\Sigma}\boldsymbol{\zeta}\right\rangle \right\} \\
 & \qquad+\kappa^{2}\left\langle \boldsymbol{b},\boldsymbol{\zeta}\right\rangle \mathbb{E}_{w,\tilde{\boldsymbol{z}}}\left\{ \sigma''\left(w\right)\sigma\left(\left\langle \boldsymbol{\Sigma}\boldsymbol{\theta},{\rm Proj}_{\boldsymbol{\Sigma}\boldsymbol{\zeta}}^{\perp}\tilde{\boldsymbol{z}}+\frac{w}{\left\Vert \boldsymbol{\Sigma}\boldsymbol{\zeta}\right\Vert _{2}^{2}}\boldsymbol{\Sigma}\boldsymbol{\zeta}\right\rangle \right)\left\langle \boldsymbol{\Sigma}\boldsymbol{a},\frac{w}{\left\Vert \boldsymbol{\Sigma}\boldsymbol{\zeta}\right\Vert _{2}^{2}}\boldsymbol{\Sigma}\boldsymbol{\zeta}\right\rangle \left\langle \boldsymbol{\Sigma}\boldsymbol{c},\frac{w}{\left\Vert \boldsymbol{\Sigma}\boldsymbol{\zeta}\right\Vert _{2}^{2}}\boldsymbol{\Sigma}\boldsymbol{\zeta}\right\rangle \right\} \\
 & =\frac{\kappa^{2}\left\langle \boldsymbol{b},\boldsymbol{\zeta}\right\rangle }{\sqrt{2\pi}\left\Vert \boldsymbol{\Sigma}\boldsymbol{\zeta}\right\Vert _{2}}\mathbb{E}_{\tilde{\boldsymbol{z}}}\left\{ \sigma\left(\left\langle \boldsymbol{S}\boldsymbol{\theta},\tilde{\boldsymbol{z}}\right\rangle \right)\left\langle \boldsymbol{S}\boldsymbol{a},\tilde{\boldsymbol{z}}\right\rangle \left\langle \boldsymbol{S}\boldsymbol{c},\tilde{\boldsymbol{z}}\right\rangle \right\} ,
\end{align*}
in which we let $\boldsymbol{S}={\rm Proj}_{\boldsymbol{\Sigma}\boldsymbol{\zeta}}^{\perp}\boldsymbol{\Sigma}$
for brevity. Since $\left\Vert \boldsymbol{\Sigma}\boldsymbol{\zeta}\right\Vert _{2}\geq\kappa_{*}\left\Vert \boldsymbol{\zeta}\right\Vert _{2}$
and $\left\Vert \boldsymbol{S}\right\Vert _{{\rm op}}\leq\left\Vert \boldsymbol{\Sigma}\right\Vert _{{\rm op}}\leq C$,
we have:
\[
\left|A_{3}\right|\leq C\frac{\kappa^{2}}{\kappa_{*}}\left\Vert \boldsymbol{\theta}\right\Vert _{2}\left\Vert \boldsymbol{a}\right\Vert _{2}\left\Vert \boldsymbol{b}\right\Vert _{2}\left\Vert \boldsymbol{c}\right\Vert _{2}.
\]
Similar calculations yield:
\begin{align*}
\left|A_{3}\right|,\left|A_{6}\right|,\left|B_{2}\right|,\left|B_{6}\right| & \leq C\frac{\kappa^{2}}{\kappa_{*}}\left\Vert \boldsymbol{\theta}\right\Vert _{2}\left\Vert \boldsymbol{a}\right\Vert _{2}\left\Vert \boldsymbol{b}\right\Vert _{2}\left\Vert \boldsymbol{c}\right\Vert _{2},\\
\left|H_{2}\right| & \leq C\frac{\kappa^{2}}{\kappa_{*}}\left\Vert \boldsymbol{\theta}'\right\Vert _{2}^{2}\left\Vert \boldsymbol{a}\right\Vert _{2}\left\Vert \boldsymbol{b}\right\Vert _{2}.
\end{align*}
We conclude that
\begin{align*}
\left\Vert \nabla_{121}^{3}U\left[\boldsymbol{\zeta},\boldsymbol{\theta}\right]\right\Vert _{{\rm op}},\left\Vert \nabla_{122}^{3}U\left[\boldsymbol{\theta},\boldsymbol{\zeta}\right]\right\Vert _{{\rm op}} & \leq C\frac{\kappa^{2}}{\kappa_{*}}\left\Vert \boldsymbol{\theta}\right\Vert _{2},\\
\left\Vert \nabla_{12}^{2}U\left(\boldsymbol{\theta},\boldsymbol{\theta}'\right)\right\Vert _{{\rm op}} & \leq C\kappa^{2}\left\Vert \boldsymbol{\theta}\right\Vert _{2}\left\Vert \boldsymbol{\theta}'\right\Vert _{2},\\
\left\Vert \nabla_{11}^{2}U\left(\boldsymbol{\theta},\boldsymbol{\theta}'\right)\right\Vert _{{\rm op}} & \leq C\frac{\kappa^{2}}{\kappa_{*}}\left\Vert \boldsymbol{\theta}'\right\Vert _{2}^{2},
\end{align*}
as claimed.
\end{proof}
\begin{prop}
\label{prop:ReLU_setting_ODE}Consider setting \ref{enu:ReLU_setting}.
Suppose that the initialization $\rho^{0}=\mathsf{N}\left(\boldsymbol{0},r_{0}^{2}\boldsymbol{I}_{d}/d\right)$
for $r_{0}\geq0$. Then the ODE (\ref{eq:ODE}) admits as solution
$\left(\hat{\boldsymbol{\theta}}^{t},\rho^{t}\right)_{t\geq0}$ with
\[
\hat{\boldsymbol{\theta}}^{t}=\boldsymbol{R}{\rm diag}\left(\frac{r_{1,t}}{r_{0}},...,\frac{r_{d,t}}{r_{0}}\right)\boldsymbol{R}^{\top}\hat{\boldsymbol{\theta}}^{0},\qquad\rho^{t}=\mathsf{N}\left(\boldsymbol{0},\boldsymbol{R}{\rm diag}\left(r_{1,t}^{2},...,r_{d,t}^{2}\right)\boldsymbol{R}^{\top}/d\right),
\]
in which $\hat{\boldsymbol{\theta}}^{0}\sim\rho^{0}$ and for each
$i\in\left[d\right]$,
\[
r_{i,t}=\sqrt{\frac{\Sigma_{i}^{2}-2\lambda}{0.5r_{0}^{2}\Sigma_{i}^{2}-\left(0.5r_{0}^{2}\Sigma_{i}^{2}-\Sigma_{i}^{2}+2\lambda\right)\exp\left\{ -2\left(\Sigma_{i}^{2}-2\lambda\right)t\right\} }}r_{0}.
\]
Here we take as a convention that if $r_{i,0}=0$ then $r_{i,t}=0$
and $r_{i,t}/r_{i,0}=1$. In fact, $\left(\rho^{t}\right)_{t\geq0}$
is the unique weak solution, and under $\left(\rho^{t}\right)_{t\geq0}$,
$\left(\hat{\boldsymbol{\theta}}^{t}\right)_{t\geq0}$ is the unique
solution to (\ref{eq:ODE}).
\end{prop}

\begin{proof}
We decompose the proof into two steps.

\paragraph*{Verification of the proposed solution and trajectorial uniqueness.}

It is easy to see that $\hat{\boldsymbol{\theta}}^{t}$ admits $\rho^{t}$
as the marginal and hence the claimed solution is consistent. We show
that $\left(\hat{\boldsymbol{\theta}}^{t}\right)_{t\geq0}$ is the
unique solution to the ODE under $\left(\rho^{t}\right)_{t\geq0}$,
which also shows $\left(\rho^{t}\right)_{t\geq0}$ is a solution.
As calculated in the proof of Proposition \ref{prop:1st_setting_grow_bound}:
\begin{align*}
W\left(\boldsymbol{\theta};\rho^{t}\right) & =\frac{1}{4}\left\Vert {\rm diag}\left(r_{1,t}\Sigma_{1},...,r_{d,t}\Sigma_{d}\right)\boldsymbol{R}^{\top}\boldsymbol{\theta}\right\Vert _{2}^{2},\\
V\left(\boldsymbol{\theta}\right) & =-\frac{1}{2}\left\Vert {\rm diag}\left(\Sigma_{1},...,\Sigma_{d}\right)\boldsymbol{R}^{\top}\boldsymbol{\theta}\right\Vert _{2}^{2}+\lambda\left\Vert \boldsymbol{\theta}\right\Vert _{2}^{2}.
\end{align*}
Then for any process $\left(\boldsymbol{\theta}^{t}\right)_{t\geq0}$
that satisfies the ODE (\ref{eq:ODE}) under $\left(\rho^{t}\right)_{t\geq0}$,
\[
\frac{{\rm d}}{{\rm d}t}\boldsymbol{\theta}^{t}=-\boldsymbol{R}{\rm diag}\left(\alpha_{1,t},...,\alpha_{d,t}\right)\boldsymbol{R}^{\top}\boldsymbol{\theta}^{t},\qquad\alpha_{i,t}=-\Sigma_{i}^{2}+\frac{1}{2}r_{i,t}^{2}\Sigma_{i}^{2}+2\lambda,
\]
or equivalently,
\[
\frac{{\rm d}}{{\rm d}t}\left(\boldsymbol{R}^{\top}\boldsymbol{\theta}^{t}\right)=-{\rm diag}\left(\alpha_{1,t},...,\alpha_{d,t}\right)\left(\boldsymbol{R}^{\top}\boldsymbol{\theta}^{t}\right).
\]
Noticing that $r_{i,t}\geq0$ obeys the following differential equation
with initialization $r_{i,0}$:
\[
\frac{{\rm d}}{{\rm d}t}r_{i,t}=-r_{i,t}\left(-\Sigma_{i}^{2}+2\lambda+\frac{1}{2}r_{i,t}^{2}\Sigma_{i}^{2}\right),
\]
we have:
\[
\frac{{\rm d}}{{\rm d}t}\left(\boldsymbol{R}^{\top}\hat{\boldsymbol{\theta}}^{t}\right)={\rm diag}\left(\frac{1}{r_{0}}\frac{{\rm d}}{{\rm d}t}r_{1,t},...,\frac{1}{r_{0}}\frac{{\rm d}}{{\rm d}t}r_{d,t}\right)\boldsymbol{R}^{\top}\hat{\boldsymbol{\theta}}^{0}=-{\rm diag}\left(\alpha_{1,t},...,\alpha_{d,t}\right)\boldsymbol{R}^{\top}\hat{\boldsymbol{\theta}}^{t}.
\]
Hence $\left(\hat{\boldsymbol{\theta}}^{t}\right)_{t\geq0}$ is a
solution. We now show that it is the only solution. It suffices to
show that for each $i\in\left[d\right]$, the solution to the ODE
$\left({\rm d}/{\rm d}t\right)u_{t}=-\alpha_{i,t}u_{t}$ is unique.
Note that $r_{i,t}\leq\max\left\{ r_{0},\sqrt{2\max\left(1-2\lambda/\Sigma_{i}^{2},0\right)}\right\} $
and hence $\left|\alpha_{i,t}\right|\leq c$ a constant for all $t\geq0$.
Let $u_{1,t}$ and $u_{2,t}$ be two solutions with $u_{1,0}=u_{2,0}$.
We have:
\[
\frac{{\rm d}}{{\rm d}t}\left(\left(u_{1,t}-u_{2,t}\right)^{2}\right)=-2\alpha_{i,t}\left(u_{1,t}-u_{2,t}\right)^{2}\leq2c\left(u_{1,t}-u_{2,t}\right)^{2}.
\]
Since $u_{1,0}=u_{2,0}$, Gronwall's lemma then implies that $u_{1,t}=u_{2,t}$,
and hence the solution must be unique.

\paragraph*{Uniqueness in law.}

We are left with proving that $\left(\rho^{t}\right)_{t\geq0}$ is
the unique weak solution with the initialization $\rho^{0}$. To that
end, we take a detour here. Let $\left(\bar{\rho}_{1}^{t}\right)_{t\geq0}$
and $\left(\bar{\rho}_{2}^{t}\right)_{t\geq0}$ be two solutions with
the same initialization $\bar{\rho}_{1}^{0}=\bar{\rho}_{2}^{0}=\bar{\rho}$
(with the equalities holding in the weak sense) for a generic $\bar{\rho}\in\mathscr{P}\left(\mathbb{R}^{d}\right)$
with finite second moment $B_{0}\left(\bar{\rho}\right)\equiv\int\left\Vert \boldsymbol{\theta}\right\Vert _{2}^{2}\bar{\rho}\left({\rm d}\boldsymbol{\theta}\right)<\infty$.
We define accordingly two coupled trajectories $\left(\boldsymbol{\theta}_{1}^{t}\right)_{t\geq0}$
and $\left(\boldsymbol{\theta}_{2}^{t}\right)_{t\geq0}$ with the
same initialization $\boldsymbol{\theta}_{1}^{0}=\boldsymbol{\theta}_{2}^{0}=\boldsymbol{\theta}^{0}\sim\bar{\rho}$:
\begin{align*}
\frac{{\rm d}}{{\rm d}t}\boldsymbol{\theta}_{1}^{t} & =-\nabla V\left(\boldsymbol{\theta}_{1}^{t}\right)-\nabla_{1}W\left(\boldsymbol{\theta}_{1}^{t};\bar{\rho}_{1}^{t}\right),\qquad\bar{\rho}_{1}^{t}={\rm Law}\left(\boldsymbol{\theta}_{1}^{t}\right),\\
\frac{{\rm d}}{{\rm d}t}\boldsymbol{\theta}_{2}^{t} & =-\nabla V\left(\boldsymbol{\theta}_{2}^{t}\right)-\nabla_{1}W\left(\boldsymbol{\theta}_{2}^{t};\bar{\rho}_{2}^{t}\right),\qquad\bar{\rho}_{2}^{t}={\rm Law}\left(\boldsymbol{\theta}_{2}^{t}\right).
\end{align*}
In the following, we let $c$ be generic positive constants that may
differ at different instances of use and may depend on the dimension
vector $\mathfrak{Dim}$, but not the time $t$ or the initialization
$\bar{\rho}$. We first obtain an a priori bound on $B_{1,t}\left(\bar{\rho}\right)=\mathbb{E}_{\boldsymbol{\theta}}\left\{ \left\Vert \boldsymbol{\theta}_{1}^{t}\right\Vert _{2}^{2}\right\} $.
By Proposition \ref{prop:1st_setting_grow_bound},
\begin{align*}
\frac{{\rm d}}{{\rm d}t}\left\Vert \boldsymbol{\theta}_{1}^{t}\right\Vert _{2} & \leq\left\Vert \nabla V\left(\boldsymbol{\theta}_{1}^{t}\right)\right\Vert _{2}+\left\Vert \nabla_{1}W\left(\boldsymbol{\theta}_{1}^{t};\bar{\rho}_{1}^{t}\right)\right\Vert _{2}\\
 & \leq\left\Vert \nabla V\left(\boldsymbol{\theta}_{1}^{t}\right)\right\Vert _{2}+\int\left\Vert \nabla_{1}U\left(\boldsymbol{\theta}_{1}^{t},\boldsymbol{\theta}\right)\right\Vert _{2}\bar{\rho}_{1}^{t}\left({\rm d}\boldsymbol{\theta}\right)\\
 & \leq c\left\Vert \boldsymbol{\theta}_{1}^{t}\right\Vert _{2}+c\left\Vert \boldsymbol{\theta}_{1}^{t}\right\Vert _{2}\int\left\Vert \boldsymbol{\theta}\right\Vert _{2}^{2}\bar{\rho}_{1}^{t}\left({\rm d}\boldsymbol{\theta}\right),
\end{align*}
from which we obtain
\[
\frac{{\rm d}}{{\rm d}t}B_{1,t}\left(\bar{\rho}\right)\leq c\left(1+B_{1,t}\left(\bar{\rho}\right)\right)B_{1,t}\left(\bar{\rho}\right)\leq c\left(1+eB_{0}\left(\bar{\rho}\right)\right)B_{1,t}\left(\bar{\rho}\right),
\]
for $t<t_{*}=\inf\left\{ t\geq0:\;B_{1,t}\left(\bar{\rho}\right)>eB_{0}\left(\bar{\rho}\right)\right\} $.
Gronwall's lemma then yields:
\[
B_{1,t}\left(\bar{\rho}\right)\leq B_{0}\left(\bar{\rho}\right)\exp\left\{ c\left(1+eB_{0}\left(\bar{\rho}\right)\right)t\right\} ,
\]
which holds for $t<t_{*}$. Therefore, with $1/T=c\left(1+eB_{0}\left(\bar{\rho}\right)\right)$,
we have $B_{1,t}\left(\bar{\rho}\right)\leq eB_{0}\left(\bar{\rho}\right)$
for all $t\leq T$. By the same procedure, we have the same result
for $B_{2,t}\left(\bar{\rho}\right)=\mathbb{E}_{\boldsymbol{\theta}}\left\{ \left\Vert \boldsymbol{\theta}_{2}^{t}\right\Vert _{2}^{2}\right\} $.
Next we bound the distance between the two trajectories:
\begin{align*}
\frac{{\rm d}}{{\rm d}t}\left\Vert \boldsymbol{\theta}_{1}^{t}-\boldsymbol{\theta}_{2}^{t}\right\Vert _{2} & \leq\left\Vert \nabla V\left(\boldsymbol{\theta}_{2}^{t}\right)-\nabla V\left(\boldsymbol{\theta}_{1}^{t}\right)\right\Vert _{2}+\left\Vert \nabla_{1}W\left(\boldsymbol{\theta}_{2}^{t};\bar{\rho}_{1}^{t}\right)-\nabla_{1}W\left(\boldsymbol{\theta}_{1}^{t};\bar{\rho}_{1}^{t}\right)\right\Vert _{2}\\
 & \qquad+\left\Vert \nabla_{1}W\left(\boldsymbol{\theta}_{2}^{t};\bar{\rho}_{2}^{t}\right)-\nabla_{1}W\left(\boldsymbol{\theta}_{2}^{t};\bar{\rho}_{1}^{t}\right)\right\Vert _{2}.
\end{align*}
Define $M_{t}\left(\bar{\rho}\right)=\mathbb{E}_{\boldsymbol{\theta}}\left\{ \left\Vert \boldsymbol{\theta}_{1}^{t}-\boldsymbol{\theta}_{2}^{t}\right\Vert _{2}^{2}\right\} $.
By Propositions \ref{prop:1st_setting_grow_bound} and \ref{prop:1st_setting_op_bound}
and the mean value theorem, for $t\leq T$:
\begin{align*}
\left\Vert \nabla V\left(\boldsymbol{\theta}_{2}^{t}\right)-\nabla V\left(\boldsymbol{\theta}_{1}^{t}\right)\right\Vert _{2} & \leq c\left\Vert \boldsymbol{\theta}_{2}^{t}-\boldsymbol{\theta}_{1}^{t}\right\Vert _{2},\\
\left\Vert \nabla_{1}W\left(\boldsymbol{\theta}_{2}^{t};\bar{\rho}_{1}^{t}\right)-\nabla_{1}W\left(\boldsymbol{\theta}_{1}^{t};\bar{\rho}_{1}^{t}\right)\right\Vert _{2} & \leq\int\left\Vert \nabla_{1}U\left(\boldsymbol{\theta}_{2}^{t},\boldsymbol{\theta}\right)-\nabla_{1}U\left(\boldsymbol{\theta}_{1}^{t},\boldsymbol{\theta}\right)\right\Vert _{2}\bar{\rho}_{1}^{t}\left({\rm d}\boldsymbol{\theta}\right)\\
 & \stackrel{\left(a\right)}{\leq}\int\left\Vert \nabla_{11}^{2}U\left(\boldsymbol{\zeta}_{1},\boldsymbol{\theta}\right)\right\Vert _{{\rm op}}\left\Vert \boldsymbol{\theta}_{2}^{t}-\boldsymbol{\theta}_{1}^{t}\right\Vert _{2}\bar{\rho}_{1}^{t}\left({\rm d}\boldsymbol{\theta}\right)\\
 & \leq c\left\Vert \boldsymbol{\theta}_{2}^{t}-\boldsymbol{\theta}_{1}^{t}\right\Vert _{2}\int\left\Vert \boldsymbol{\theta}\right\Vert _{2}^{2}\bar{\rho}_{1}^{t}\left({\rm d}\boldsymbol{\theta}\right)\\
 & \leq c\left\Vert \boldsymbol{\theta}_{2}^{t}-\boldsymbol{\theta}_{1}^{t}\right\Vert _{2}B_{0}\left(\bar{\rho}\right),\\
\left\Vert \nabla_{1}W\left(\boldsymbol{\theta}_{2}^{t};\bar{\rho}_{2}^{t}\right)-\nabla_{1}W\left(\boldsymbol{\theta}_{2}^{t};\bar{\rho}_{1}^{t}\right)\right\Vert _{2} & \stackrel{\left(b\right)}{=}\left\Vert \mathbb{E}_{\tilde{\boldsymbol{\theta}}}\left\{ \nabla_{1}U\left(\boldsymbol{\theta}_{2}^{t},\tilde{\boldsymbol{\theta}}_{2}\right)-\nabla_{1}U\left(\boldsymbol{\theta}_{2}^{t},\tilde{\boldsymbol{\theta}}_{1}\right)\right\} \right\Vert _{2}\\
 & \stackrel{\left(c\right)}{\leq}\mathbb{E}_{\tilde{\boldsymbol{\theta}}}\left\{ \left\Vert \nabla_{12}^{2}U\left(\boldsymbol{\theta}_{2}^{t},\boldsymbol{\zeta}_{2}\right)\right\Vert _{{\rm op}}\left\Vert \tilde{\boldsymbol{\theta}}_{2}-\tilde{\boldsymbol{\theta}}_{1}\right\Vert _{2}\right\} \\
 & \leq c\left\Vert \boldsymbol{\theta}_{2}^{t}\right\Vert _{2}\mathbb{E}_{\tilde{\boldsymbol{\theta}}}\left\{ \left\Vert \boldsymbol{\zeta}_{2}\right\Vert _{2}\left\Vert \tilde{\boldsymbol{\theta}}_{2}-\tilde{\boldsymbol{\theta}}_{1}\right\Vert _{2}\right\} \\
 & \leq c\left\Vert \boldsymbol{\theta}_{2}^{t}\right\Vert _{2}\mathbb{E}_{\tilde{\boldsymbol{\theta}}}\left\{ \left\Vert \tilde{\boldsymbol{\theta}}_{1}\right\Vert _{2}\left\Vert \tilde{\boldsymbol{\theta}}_{2}-\tilde{\boldsymbol{\theta}}_{1}\right\Vert _{2}+\left\Vert \tilde{\boldsymbol{\theta}}_{2}-\tilde{\boldsymbol{\theta}}_{1}\right\Vert _{2}^{2}\right\} \\
 & \leq c\left\Vert \boldsymbol{\theta}_{2}^{t}\right\Vert _{2}\left(\sqrt{\mathbb{E}_{\tilde{\boldsymbol{\theta}}}\left\{ \left\Vert \tilde{\boldsymbol{\theta}}_{1}\right\Vert _{2}^{2}\right\} \mathbb{E}_{\tilde{\boldsymbol{\theta}}}\left\{ \left\Vert \tilde{\boldsymbol{\theta}}_{2}-\tilde{\boldsymbol{\theta}}_{1}\right\Vert _{2}^{2}\right\} }+M_{t}\left(\bar{\rho}\right)\right)\\
 & \leq c\left\Vert \boldsymbol{\theta}_{2}^{t}\right\Vert _{2}\left(\sqrt{B_{0}\left(\bar{\rho}\right)M_{t}\left(\bar{\rho}\right)}+M_{t}\left(\bar{\rho}\right)\right),
\end{align*}
where in step $\left(a\right)$, $\boldsymbol{\zeta}_{1}\in\left[\boldsymbol{\theta}_{1}^{t},\boldsymbol{\theta}_{2}^{t}\right]$;
in step $\left(b\right)$, we define $\left(\tilde{\boldsymbol{\theta}}_{1},\tilde{\boldsymbol{\theta}}_{2}\right)\stackrel{{\rm d}}{=}\left(\boldsymbol{\theta}_{1}^{t},\boldsymbol{\theta}_{2}^{t}\right)$
and $\left(\tilde{\boldsymbol{\theta}}_{1},\tilde{\boldsymbol{\theta}}_{2}\right)$
is independent of $\left(\boldsymbol{\theta}_{1}^{t},\boldsymbol{\theta}_{2}^{t}\right)$;
in step $\left(c\right)$, $\boldsymbol{\zeta}_{2}\in\left[\tilde{\boldsymbol{\theta}}_{1},\tilde{\boldsymbol{\theta}}_{2}\right]$
and hence $\left\Vert \boldsymbol{\zeta}_{2}\right\Vert _{2}\leq\left\Vert \tilde{\boldsymbol{\theta}}_{1}\right\Vert _{2}+\left\Vert \tilde{\boldsymbol{\theta}}_{2}-\tilde{\boldsymbol{\theta}}_{1}\right\Vert _{2}$.
These bounds imply that
\[
\frac{{\rm d}}{{\rm d}t}\left\Vert \boldsymbol{\theta}_{1}^{t}-\boldsymbol{\theta}_{2}^{t}\right\Vert _{2}^{2}\leq c\left(1+B_{0}\left(\bar{\rho}\right)\right)\left\Vert \boldsymbol{\theta}_{1}^{t}-\boldsymbol{\theta}_{2}^{t}\right\Vert _{2}^{2}+c\left\Vert \boldsymbol{\theta}_{2}^{t}\right\Vert _{2}\left\Vert \boldsymbol{\theta}_{1}^{t}-\boldsymbol{\theta}_{2}^{t}\right\Vert _{2}\left(\sqrt{B_{0}\left(\bar{\rho}\right)M_{t}\left(\bar{\rho}\right)}+M_{t}\left(\bar{\rho}\right)\right).
\]
Taking expectation, we obtain:
\[
\frac{{\rm d}}{{\rm d}t}M_{t}\left(\bar{\rho}\right)\leq c\left(1+B_{0}\left(\bar{\rho}\right)\right)M_{t}\left(\bar{\rho}\right)+c\sqrt{B_{0}\left(\bar{\rho}\right)M_{t}\left(\bar{\rho}\right)}\left(\sqrt{B_{0}\left(\bar{\rho}\right)M_{t}\left(\bar{\rho}\right)}+M_{t}\left(\bar{\rho}\right)\right)\leq c\left(1+B_{0}\left(\bar{\rho}\right)\right)M_{t}\left(\bar{\rho}\right),
\]
for $t\leq T$ and $t<t_{*}'$ with $t_{*}'=\inf\left\{ t\geq0:\;M_{t}\left(\bar{\rho}\right)>1\right\} $.
Since $M_{0}\left(\bar{\rho}\right)=0$ and $M_{t}\left(\bar{\rho}\right)\geq0$,
Gronwall's lemma then implies that $t_{*}'>T$ and $M_{t}\left(\bar{\rho}\right)=0$
for $t\leq T$. Note that $M_{t}\left(\bar{\rho}\right)=0$ implies,
for any $1$-Lipschitz test function $\phi:\;\mathbb{R}^{d}\to\mathbb{R}$,
\[
\left|\int\phi\left(\boldsymbol{\theta}\right)\bar{\rho}_{1}^{t}\left({\rm d}\boldsymbol{\theta}\right)-\int\phi\left(\boldsymbol{\theta}\right)\bar{\rho}_{2}^{t}\left({\rm d}\boldsymbol{\theta}\right)\right|\leq\inf_{\boldsymbol{\theta}_{a}\sim\bar{\rho}_{1}^{t},\;\boldsymbol{\theta}_{b}\sim\bar{\rho}_{2}^{t}}\mathbb{E}\left\{ \left\Vert \boldsymbol{\theta}_{a}-\boldsymbol{\theta}_{b}\right\Vert _{2}\right\} \leq\mathbb{E}_{\boldsymbol{\theta}}\left\{ \left\Vert \boldsymbol{\theta}_{1}^{t}-\boldsymbol{\theta}_{2}^{t}\right\Vert _{2}\right\} \leq\sqrt{M_{t}\left(\bar{\rho}\right)}=0.
\]
Hence two solutions $\left(\bar{\rho}_{1}^{t}\right)_{t\geq0}$ and
$\left(\bar{\rho}_{2}^{t}\right)_{t\geq0}$ coincide (weakly) up to
time $T$.

Applying this result to our problem, we suppose that, for a fixed
$s\geq0$, two solutions $\left(\rho_{1}^{t}\right)_{t\geq0}$ and
$\left(\rho_{2}^{t}\right)_{t\geq0}$ coincide (weakly) with $\rho^{s}=\mathsf{N}\left(\boldsymbol{0},\boldsymbol{R}{\rm diag}\left(r_{1,s}^{2},...,r_{d,s}^{2}\right)\boldsymbol{R}^{\top}/d\right)$
at time $t=s$. Then the above result shows that they coincide (weakly)
on the time interval $\left[s,s+T_{s}\right]$, in which
\[
\frac{1}{T_{s}}=c\left(1+e\int\left\Vert \boldsymbol{\theta}\right\Vert _{2}^{2}\rho^{s}\left({\rm d}\boldsymbol{\theta}\right)\right)=c\left(1+\frac{e}{d}\sum_{i=1}^{d}r_{i,s}^{2}\right)\leq c\left(1+e\left(r_{0}^{2}+2\right)\right),
\]
using the observation $r_{i,s}\leq\max\left\{ r_{0},\sqrt{2\max\left(1-2\lambda/\Sigma_{i}^{2},0\right)}\right\} \leq\max\left\{ r_{0},\sqrt{2}\right\} \leq C$
which holds for all $i\in\left[d\right]$ and $s\geq0$. Since $T_{s}$
is lower-bounded by a strictly positive constant independent of $s\geq0$,
the solution $\left(\rho^{t}\right)_{t\geq0}$ must be the unique
weak solution on $t\in[0,\infty)$ with initialization $\rho^{0}$.

\end{proof}
\begin{prop}
\label{prop:1st_setting_F_bound}Consider setting \ref{enu:ReLU_setting}.
For a collection of vectors $\Theta=\left(\boldsymbol{\theta}_{i}\right)_{i\leq N}$
where $\boldsymbol{\theta}_{i}\in\mathbb{R}^{d}$, $\boldsymbol{x}\sim{\cal P}$
and $\boldsymbol{z}=\left(\boldsymbol{x},\boldsymbol{x}\right)$,
we have $\boldsymbol{F}_{i}\left(\Theta;\boldsymbol{z}\right)$ is
sub-exponential with $\psi_{1}$-norm: 
\[
\left\Vert \boldsymbol{F}_{i}\left(\Theta;\boldsymbol{z}\right)\right\Vert _{\psi_{1}}\leq C\kappa^{2}\left\Vert \boldsymbol{\theta}_{i}\right\Vert _{2}\left(\frac{1}{N}\sum_{j=1}^{N}\left\Vert \boldsymbol{\theta}_{j}\right\Vert _{2}^{2}+1\right).
\]
\end{prop}

\begin{proof}
Consider a fixed vector $\boldsymbol{v}\in\mathbb{S}^{d-1}$:
\begin{align*}
\left\langle \boldsymbol{v},\boldsymbol{F}_{i}\left(\Theta;\boldsymbol{z}\right)\right\rangle  & =\kappa\left\langle \boldsymbol{v},\nabla_{2}\sigma_{*}\left(\boldsymbol{x};\kappa\boldsymbol{\theta}_{i}\right)^{\top}\left(\hat{\boldsymbol{y}}_{N}\left(\boldsymbol{x};\Theta\right)-\boldsymbol{x}\right)\right\rangle +\lambda\left\langle \boldsymbol{v},\nabla_{1}\Lambda\left(\boldsymbol{\theta}_{i},\boldsymbol{z}\right)\right\rangle \\
 & =\kappa\sigma\left(\left\langle \kappa\boldsymbol{\theta}_{i},\boldsymbol{x}\right\rangle \right)\left(\left\langle \boldsymbol{v},\hat{\boldsymbol{x}}\right\rangle -\left\langle \boldsymbol{v},\boldsymbol{x}\right\rangle \right)+\kappa^{2}\sigma'\left(\left\langle \kappa\boldsymbol{\theta}_{i},\boldsymbol{x}\right\rangle \right)\left(\left\langle \boldsymbol{\theta}_{i},\hat{\boldsymbol{x}}\right\rangle -\left\langle \boldsymbol{\theta}_{i},\boldsymbol{x}\right\rangle \right)\left\langle \boldsymbol{v},\boldsymbol{x}\right\rangle +2\lambda\left\langle \boldsymbol{v},\boldsymbol{\theta}_{i}\right\rangle \\
 & \equiv A_{1}+A_{2}+A_{3},
\end{align*}
where we denote $\hat{\boldsymbol{x}}=\left(1/N\right)\cdot\sum_{j=1}^{N}\kappa\boldsymbol{\theta}_{j}\sigma\left(\left\langle \kappa\boldsymbol{\theta}_{j},\boldsymbol{x}\right\rangle \right)$
for brevity. We examine each component in the above:
\begin{itemize}
\item For any $i\in\left[N\right]$, since $\sigma\left(\left\langle \kappa\boldsymbol{\theta}_{i},\boldsymbol{x}\right\rangle \right)\leq\left|\left\langle \kappa\boldsymbol{\theta}_{i},\boldsymbol{x}\right\rangle \right|$,
$\left\langle \kappa\boldsymbol{\theta}_{i},\boldsymbol{x}\right\rangle \sim\mathsf{N}\left(0,\left\Vert \boldsymbol{\Sigma}\boldsymbol{\theta}_{i}\right\Vert _{2}^{2}\right)$
and $\left\Vert \boldsymbol{\Sigma}\boldsymbol{\theta}_{i}\right\Vert _{2}\leq C\left\Vert \boldsymbol{\theta}_{i}\right\Vert _{2}$,
we have $\sigma\left(\left\langle \kappa\boldsymbol{\theta}_{i},\boldsymbol{x}\right\rangle \right)$
is sub-Gaussian with $\psi_{2}$-norm $\left\Vert \sigma\left(\left\langle \kappa\boldsymbol{\theta}_{i},\boldsymbol{x}\right\rangle \right)\right\Vert _{\psi_{2}}\leq C\left\Vert \boldsymbol{\theta}_{i}\right\Vert _{2}$.
Therefore for any $\boldsymbol{u}\in\mathbb{R}^{d}$, $\left\langle \boldsymbol{u},\hat{\boldsymbol{x}}\right\rangle $
is sub-Gaussian with $\psi_{2}$-norm 
\[
\left\Vert \left\langle \boldsymbol{u},\hat{\boldsymbol{x}}\right\rangle \right\Vert _{\psi_{2}}\leq\frac{\kappa}{N}\sum_{j=1}^{N}\left|\left\langle \boldsymbol{u},\boldsymbol{\theta}_{j}\right\rangle \right|\left\Vert \sigma\left(\left\langle \kappa\boldsymbol{\theta}_{j},\boldsymbol{x}\right\rangle \right)\right\Vert _{\psi_{2}}\leq C\kappa\mathsf{M}\left\Vert \boldsymbol{u}\right\Vert _{2},
\]
where $\mathsf{M}=\left(1/N\right)\cdot\sum_{j=1}^{N}\left\Vert \boldsymbol{\theta}_{j}\right\Vert _{2}^{2}$.
We have $\left\langle \kappa\boldsymbol{u},\boldsymbol{x}\right\rangle $
is sub-Gaussian with $\psi_{2}$-norm $\left\Vert \left\langle \kappa\boldsymbol{u},\boldsymbol{x}\right\rangle \right\Vert _{\psi_{2}}=\left\Vert \boldsymbol{\Sigma}\boldsymbol{u}\right\Vert _{2}\leq C\left\Vert \boldsymbol{u}\right\Vert _{2}$.
Therefore, $A_{1}$ is sub-exponential:
\begin{align*}
\left\Vert A_{1}\right\Vert _{\psi_{1}} & \leq\kappa\left\Vert \sigma\left(\left\langle \kappa\boldsymbol{\theta}_{i},\boldsymbol{x}\right\rangle \right)\right\Vert _{\psi_{2}}\left(\left\Vert \left\langle \boldsymbol{v},\hat{\boldsymbol{x}}\right\rangle \right\Vert _{\psi_{2}}+\left\Vert \left\langle \boldsymbol{v},\boldsymbol{x}\right\rangle \right\Vert _{\psi_{2}}\right)\\
 & \leq C\kappa\left\Vert \boldsymbol{\theta}_{i}\right\Vert _{2}\left(\kappa\mathsf{M}+\frac{1}{\kappa}\right)\leq C\kappa^{2}\left\Vert \boldsymbol{\theta}_{i}\right\Vert _{2}\left(\mathsf{M}+1\right).
\end{align*}
\item Recall that $\sigma'\left(u\right)=\mathbb{I}\left(u\geq0\right)$
and hence $\left\Vert \sigma'\right\Vert _{\infty}\leq1$. Then $A_{2}$
is sub-exponential:
\begin{align*}
\left\Vert A_{2}\right\Vert _{\psi_{1}} & \leq\kappa\left(\left\Vert \left\langle \boldsymbol{\theta}_{i},\hat{\boldsymbol{x}}\right\rangle \right\Vert _{\psi_{2}}+\left\Vert \left\langle \boldsymbol{\theta}_{i},\boldsymbol{x}\right\rangle \right\Vert _{\psi_{2}}\right)\left\Vert \left\langle \kappa\boldsymbol{v},\boldsymbol{x}\right\rangle \right\Vert _{\psi_{2}}\\
 & \leq C\kappa\left(\kappa\mathsf{M}\left\Vert \boldsymbol{\theta}_{i}\right\Vert _{2}+\frac{1}{\kappa}\left\Vert \boldsymbol{\theta}_{i}\right\Vert _{2}\right)\leq C\kappa^{2}\left\Vert \boldsymbol{\theta}_{i}\right\Vert _{2}\left(\mathsf{M}+1\right).
\end{align*}
\item $A_{3}$ is a constant and so it is also sub-exponential with $\psi_{1}$-norm
$\left\Vert A_{3}\right\Vert _{\psi_{1}}\leq C\left\Vert \boldsymbol{\theta}_{i}\right\Vert _{2}$.
\end{itemize}
We have $\left\langle \boldsymbol{v},\boldsymbol{F}_{i}\left(\Theta;\boldsymbol{z}\right)\right\rangle $
and hence $\boldsymbol{F}_{i}\left(\Theta;\boldsymbol{z}\right)$
are sub-exponential:
\[
\left\Vert \boldsymbol{F}_{i}\left(\Theta;\boldsymbol{z}\right)\right\Vert _{\psi_{1}}=\sup_{\boldsymbol{v}\in\mathbb{S}^{d-1}}\left\Vert \left\langle \boldsymbol{v},\boldsymbol{F}_{i}\left(\Theta;\boldsymbol{z}\right)\right\rangle \right\Vert _{\psi_{1}}\leq C\kappa^{2}\left\Vert \boldsymbol{\theta}_{i}\right\Vert _{2}\left(\mathsf{M}+1\right).
\]
This completes the proof.
\end{proof}
\begin{lem}
\label{lem:Unorm_firstSetting}Consider setting \ref{enu:ReLU_setting}.
We have, for some sufficiently large $C_{*}$, with probability at
least $1-C\exp\left(Cd-CN\kappa_{*}^{2}/\kappa^{2}\right)$,
\[
\left\Vert \frac{1}{N}\sum_{i=1}^{N}\nabla_{11}^{2}U\left(\boldsymbol{\zeta},\boldsymbol{D}\boldsymbol{\theta}_{i}\right)\right\Vert _{{\rm op}}\leq C_{*},
\]
in which $\boldsymbol{\zeta}$ is a fixed vector with $\left\Vert \boldsymbol{\zeta}\right\Vert _{2}<\infty$,
$\left(\boldsymbol{\theta}_{i}\right)_{i\leq N}\sim_{{\rm i.i.d.}}\mathsf{N}\left(0,\boldsymbol{I}_{d}/d\right)$
and $\boldsymbol{D}\in\mathbb{R}^{d\times d}$ with $\left\Vert \boldsymbol{D}\right\Vert _{2}\leq C$.
Here $C_{*}$ does not depend on $d$ or $N$.
\end{lem}

\begin{proof}
Let us decompose
\[
\frac{1}{N}\sum_{i=1}^{N}\nabla_{11}^{2}U\left(\boldsymbol{\zeta},\boldsymbol{D}\boldsymbol{\theta}_{i}\right)=\boldsymbol{M}_{1}+\boldsymbol{M}_{1}^{\top}+\boldsymbol{M}_{2}\in\mathbb{R}^{d\times d},
\]
for which
\begin{align*}
\boldsymbol{M}_{1} & =\frac{1}{N}\sum_{i=1}^{N}\kappa^{3}\mathbb{E}_{{\cal P}}\left\{ \sigma'\left(\left\langle \kappa\boldsymbol{\zeta},\boldsymbol{x}\right\rangle \right)\sigma\left(\left\langle \kappa\boldsymbol{D}\boldsymbol{\theta}_{i},\boldsymbol{x}\right\rangle \right)\boldsymbol{D}\boldsymbol{\theta}_{i}\boldsymbol{x}^{\top}\right\} ,\\
\boldsymbol{M}_{2} & =\frac{1}{N}\sum_{i=1}^{N}\kappa^{4}\mathbb{E}_{{\cal P}}\left\{ \left\langle \boldsymbol{\zeta},\boldsymbol{D}\boldsymbol{\theta}_{i}\right\rangle \sigma''\left(\left\langle \kappa\boldsymbol{\zeta},\boldsymbol{x}\right\rangle \right)\sigma\left(\left\langle \kappa\boldsymbol{D}\boldsymbol{\theta}_{i},\boldsymbol{x}\right\rangle \right)\boldsymbol{x}\boldsymbol{x}^{\top}\right\} .
\end{align*}
Below we bound $\left\Vert \boldsymbol{M}_{1}\right\Vert _{{\rm op}}$
and $\left\Vert \boldsymbol{M}_{2}\right\Vert _{{\rm op}}$ separately.
We shall use repeatedly the following simple fact: $\mathbb{E}_{{\cal P}}\left\{ \left|\sigma'\left(\left\langle \kappa\boldsymbol{\zeta},\boldsymbol{x}\right\rangle \right)\right|^{m}\right\} =0.5$
for any $m>0$, since $\sigma'\left(u\right)=\mathbb{I}\left(u\geq0\right)$.

\paragraph{Step 1: Bounding $\left\Vert \boldsymbol{M}_{1}\right\Vert _{{\rm op}}$.}

Define the quantity $A_{1}=\frac{1}{2}\kappa^{2}\left\Vert \mathbb{E}_{{\cal P}}\left\{ \sigma'\left(\left\langle \kappa\boldsymbol{\zeta},\boldsymbol{x}\right\rangle \right)\boldsymbol{x}\boldsymbol{x}^{\top}\right\} \right\Vert _{2}$.
Note that for any $\boldsymbol{u},\boldsymbol{v}\in\mathbb{R}^{d}$,
\begin{align*}
 & \left|\left\langle \boldsymbol{v},\kappa^{2}\mathbb{E}_{{\cal P}}\left\{ \sigma'\left(\left\langle \kappa\boldsymbol{\zeta},\boldsymbol{x}\right\rangle \right)\boldsymbol{D}\boldsymbol{D}^{\top}\boldsymbol{x}\boldsymbol{x}^{\top}\right\} \boldsymbol{u}\right\rangle \right|=\left|\kappa^{2}\mathbb{E}_{{\cal P}}\left\{ \sigma'\left(\left\langle \kappa\boldsymbol{\zeta},\boldsymbol{x}\right\rangle \right)\left\langle \boldsymbol{D}\boldsymbol{D}^{\top}\boldsymbol{v},\boldsymbol{x}\right\rangle \left\langle \boldsymbol{u},\boldsymbol{x}\right\rangle \right\} \right|\\
 & \quad\leq\mathbb{E}_{{\cal P}}\left\{ \left|\sigma'\left(\left\langle \kappa\boldsymbol{\zeta},\boldsymbol{x}\right\rangle \right)\right|^{3}\right\} ^{1/3}\mathbb{E}_{{\cal P}}\left\{ \left|\kappa\left\langle \boldsymbol{D}\boldsymbol{D}^{\top}\boldsymbol{v},\boldsymbol{x}\right\rangle \right|^{3}\right\} ^{1/3}\mathbb{E}_{{\cal P}}\left\{ \left|\kappa\left\langle \boldsymbol{u},\boldsymbol{x}\right\rangle \right|^{3}\right\} ^{1/3}\\
 & \quad=C\left\Vert \boldsymbol{\Sigma}\boldsymbol{D}\boldsymbol{D}^{\top}\boldsymbol{v}\right\Vert _{2}\left\Vert \boldsymbol{\Sigma}\boldsymbol{u}\right\Vert _{2}\\
 & \quad\leq C\left\Vert \boldsymbol{v}\right\Vert _{2}\left\Vert \boldsymbol{u}\right\Vert _{2},
\end{align*}
and therefore $A_{1}\leq C$. Furthermore, we have:
\begin{align*}
\left|\left\Vert \boldsymbol{M}_{1}\right\Vert _{{\rm op}}-A_{1}\right| & \leq\left\Vert \boldsymbol{M}_{1}-\frac{1}{2}\kappa^{2}\mathbb{E}_{{\cal P}}\left\{ \sigma'\left(\left\langle \kappa\boldsymbol{\zeta},\boldsymbol{x}\right\rangle \right)\boldsymbol{D}\boldsymbol{D}^{\top}\boldsymbol{x}\boldsymbol{x}^{\top}\right\} \right\Vert _{{\rm op}}\\
 & =\left\Vert \kappa^{2}\mathbb{E}_{{\cal P}}\left\{ \sigma'\left(\left\langle \kappa\boldsymbol{\zeta},\boldsymbol{x}\right\rangle \right)\boldsymbol{D}\left[\frac{1}{N}\sum_{i=1}^{N}\kappa\boldsymbol{\theta}_{i}\sigma\left(\left\langle \kappa\boldsymbol{D}\boldsymbol{\theta}_{i},\boldsymbol{x}\right\rangle \right)-\frac{1}{2}\boldsymbol{D}^{\top}\boldsymbol{x}\right]\boldsymbol{x}^{\top}\right\} \right\Vert _{{\rm op}}\equiv\left\Vert \boldsymbol{M}_{1,1}\right\Vert _{{\rm op}}.
\end{align*}
Here we making the following claim:
\[
\mathbb{P}\left\{ \left\Vert \boldsymbol{M}_{1,1}\right\Vert _{{\rm op}}\geq\delta\right\} \leq C\exp\left(Cd-C\delta^{2}N/\kappa^{2}\right),
\]
for $\delta\geq0$. Assuming this claim, we thus have for $\delta\geq0$
and some sufficiently large $C'$,
\[
\mathbb{P}\left\{ \left\Vert \boldsymbol{M}_{1}\right\Vert _{{\rm op}}\geq C'+\delta\right\} \leq C\exp\left(Cd-C\delta^{2}N/\kappa^{2}\right),
\]
which is the desired result.

We are left with proving the claim on $\left\Vert \boldsymbol{M}_{1,1}\right\Vert _{{\rm op}}$.
Given fixed $\boldsymbol{u},\boldsymbol{v}\in\mathbb{S}^{d-1}$,
\[
\left\langle \boldsymbol{u},\boldsymbol{M}_{1,1}\boldsymbol{v}\right\rangle =\frac{1}{N}\sum_{i=1}^{N}M_{1,1,i}^{\boldsymbol{u},\boldsymbol{v}},\qquad M_{1,1,i}^{\boldsymbol{u},\boldsymbol{v}}=\kappa\mathbb{E}_{{\cal P}}\left\{ \sigma'\left(\left\langle \kappa\boldsymbol{\zeta},\boldsymbol{x}\right\rangle \right)\left\langle \kappa\boldsymbol{\theta}_{i}\sigma\left(\left\langle \kappa\boldsymbol{D}\boldsymbol{\theta}_{i},\boldsymbol{x}\right\rangle \right)-\frac{1}{2}\boldsymbol{D}^{\top}\boldsymbol{x},\boldsymbol{D}^{\top}\boldsymbol{u}\right\rangle \left\langle \boldsymbol{x},\kappa\boldsymbol{v}\right\rangle \right\} .
\]
First notice that $\left(M_{1,1,i}^{\boldsymbol{u},\boldsymbol{v}}\right)_{i\leq N}$
are i.i.d. Furthermore, by Stein's lemma,
\[
\mathbb{E}_{\boldsymbol{\theta}}\left\{ \kappa\boldsymbol{\theta}_{i}\sigma\left(\left\langle \kappa\boldsymbol{D}\boldsymbol{\theta}_{i},\boldsymbol{x}\right\rangle \right)\right\} =\mathbb{E}_{\boldsymbol{\theta}}\left\{ \sigma'\left(\left\langle \kappa\boldsymbol{D}\boldsymbol{\theta}_{i},\boldsymbol{x}\right\rangle \right)\right\} \boldsymbol{D}^{\top}\boldsymbol{x}=\frac{1}{2}\boldsymbol{D}^{\top}\boldsymbol{x}.
\]
Therefore $\mathbb{E}\left\{ M_{1,1,i}^{\boldsymbol{u},\boldsymbol{v}}\right\} =0$.
For any positive integer $p\geq1$, 
\begin{align*}
\mathbb{E}\left\{ \left|M_{1,1,i}^{\boldsymbol{u},\boldsymbol{v}}\right|^{p}\right\}  & =\mathbb{E}\left\{ \left|\mathbb{E}_{{\cal P}}\left\{ \sigma'\left(\left\langle \kappa\boldsymbol{\zeta},\boldsymbol{x}\right\rangle \right)\left\langle \kappa\boldsymbol{\theta}_{i}\sigma\left(\left\langle \kappa\boldsymbol{D}\boldsymbol{\theta}_{i},\boldsymbol{x}\right\rangle \right)-\frac{1}{2}\boldsymbol{D}^{\top}\boldsymbol{x},\kappa\boldsymbol{D}^{\top}\boldsymbol{u}\right\rangle \left\langle \boldsymbol{x},\kappa\boldsymbol{v}\right\rangle \right\} \right|^{p}\right\} \\
 & \leq\mathbb{E}\left\{ \mathbb{E}_{{\cal P}}\left\{ \sigma'\left(\left\langle \kappa\boldsymbol{\zeta},\boldsymbol{x}\right\rangle \right)^{2}\left\langle \kappa\boldsymbol{\theta}_{i}\sigma\left(\left\langle \kappa\boldsymbol{D}\boldsymbol{\theta}_{i},\boldsymbol{x}\right\rangle \right)-\frac{1}{2}\boldsymbol{D}^{\top}\boldsymbol{x},\kappa\boldsymbol{D}^{\top}\boldsymbol{u}\right\rangle ^{2}\right\} ^{p/2}\mathbb{E}_{{\cal P}}\left\{ \left\langle \boldsymbol{x},\kappa\boldsymbol{v}\right\rangle ^{2}\right\} ^{p/2}\right\} \\
 & \leq\mathbb{E}\left\{ \mathbb{E}_{{\cal P}}\left\{ \left\langle \kappa\boldsymbol{\theta}_{i}\sigma\left(\left\langle \kappa\boldsymbol{D}\boldsymbol{\theta}_{i},\boldsymbol{x}\right\rangle \right)-\frac{1}{2}\boldsymbol{D}^{\top}\boldsymbol{x},\kappa\boldsymbol{D}^{\top}\boldsymbol{u}\right\rangle ^{2}\right\} ^{p/2}\mathbb{E}_{{\cal P}}\left\{ \left\langle \boldsymbol{x},\kappa\boldsymbol{v}\right\rangle ^{2}\right\} ^{p/2}\right\} \\
 & \leq C^{p}\mathbb{E}\left\{ \mathbb{E}_{{\cal P}}\left\{ \kappa^{2}\left\langle \kappa\boldsymbol{\theta}_{i},\boldsymbol{D}^{\top}\boldsymbol{u}\right\rangle ^{2}\sigma\left(\left\langle \kappa\boldsymbol{D}^{\top}\boldsymbol{\theta}_{i},\boldsymbol{x}\right\rangle \right)^{2}+\left\langle \boldsymbol{x},\kappa\boldsymbol{D}\boldsymbol{D}^{\top}\boldsymbol{u}\right\rangle ^{2}\right\} ^{p/2}\mathbb{E}_{{\cal P}}\left\{ \left\langle \boldsymbol{x},\kappa\boldsymbol{v}\right\rangle ^{2}\right\} ^{p/2}\right\} \\
 & \leq C^{p}\mathbb{E}\left\{ \left(\kappa^{2}\left\langle \kappa\boldsymbol{\theta}_{i},\boldsymbol{D}^{\top}\boldsymbol{u}\right\rangle ^{2}\left\Vert \boldsymbol{\Sigma}\boldsymbol{D}^{\top}\boldsymbol{\theta}_{i}\right\Vert _{2}^{2}+\left\Vert \boldsymbol{\Sigma}\boldsymbol{D}\boldsymbol{D}^{\top}\boldsymbol{u}\right\Vert _{2}^{2}\right)^{p/2}\left\Vert \boldsymbol{\Sigma}\boldsymbol{v}\right\Vert _{2}^{p}\right\} \\
 & \leq C^{p}\mathbb{E}\left\{ \left|\left\langle \kappa\boldsymbol{\theta}_{i},\boldsymbol{D}^{\top}\boldsymbol{u}\right\rangle \right|^{p}\left\Vert \kappa\boldsymbol{\theta}_{i}\right\Vert _{2}^{p}+1\right\} \\
 & \leq C^{p}\left(\sqrt{\mathbb{E}\left\{ \left\langle \kappa\boldsymbol{\theta}_{i},\boldsymbol{D}^{\top}\boldsymbol{u}\right\rangle ^{2p}\right\} \mathbb{E}\left\{ \left\Vert \kappa\boldsymbol{\theta}_{i}\right\Vert _{2}^{2p}\right\} }+1\right)\\
 & =C^{p}\left(\sqrt{\left\Vert \boldsymbol{D}^{\top}\boldsymbol{u}\right\Vert _{2}^{2p}\mathbb{E}_{g}\left\{ g^{2p}\right\} \mathbb{E}\left\{ \left\Vert \kappa\boldsymbol{\theta}_{i}\right\Vert _{2}^{2p}\right\} }+1\right)\\
 & \leq C^{p}\left(\sqrt{p^{p}\left(\kappa^{2p}+p^{p}\right)}+1\right)\\
 & \leq C^{p}\left(\kappa^{p}p^{p/2}+p^{p}\right),
\end{align*}
recalling that $\left\Vert \sigma'\right\Vert _{\infty}\leq1$ for
$\sigma$ being the ReLU, $\kappa\boldsymbol{\theta}_{i}\sim\mathsf{N}\left(0,\boldsymbol{I}_{d}\right)$,
$\left\Vert \boldsymbol{\Sigma}\right\Vert _{{\rm op}}\leq C$, $\left\Vert \boldsymbol{D}\right\Vert _{{\rm op}}\leq C$
and $\left\Vert \boldsymbol{u}\right\Vert _{2}=\left\Vert \boldsymbol{v}\right\Vert _{2}=1$.
Here we have used the fact that if $X$ is a $\chi^{2}$ random variable
with degree of freedom $\kappa^{2}$, then $\mathbb{E}\left\{ X^{p}\right\} \leq C^{p}\left(\kappa^{2}+2p\right)^{p}$.
It is easy to see that $M_{1,1,i}^{\boldsymbol{u},\boldsymbol{v}}$
is a sub-exponential random variable with $\psi_{1}$-norm $\left\Vert M_{1,1,i}^{\boldsymbol{u},\boldsymbol{v}}\right\Vert _{\psi_{1}}\leq C\kappa$.
Then by Lemma \ref{lem:subgauss_subexp_properties}, for $\delta\in\left(0,1\right)$,
with probability at most $C\exp\left(-C\delta^{2}N/\kappa^{2}\right)$,
\[
\left|\left\langle \boldsymbol{u},\boldsymbol{M}_{1,1}\boldsymbol{v}\right\rangle \right|=\left|\frac{1}{N}\sum_{i=1}^{N}M_{1,1,i}^{\boldsymbol{u},\boldsymbol{v}}\right|\geq\delta.
\]
Now we construct an epsilon-net ${\cal N}\subset\mathbb{S}^{d-1}$
such that for any $\boldsymbol{a}\in\mathbb{S}^{d-1}$, there exists
$\boldsymbol{a}'\in{\cal N}$ with $\left\Vert \boldsymbol{a}-\boldsymbol{a}'\right\Vert _{2}\leq1/3$.
There is such an epsilon-net ${\cal N}$ with size $\left|{\cal N}\right|\leq9^{d}$
\cite{vershynin2010introduction}. A standard argument yields
\[
\left\Vert \boldsymbol{M}_{1,1}\right\Vert _{{\rm op}}\leq3\max_{\boldsymbol{u},\boldsymbol{v}\in{\cal N}}\left\langle \boldsymbol{u},\boldsymbol{M}_{1,1}\boldsymbol{v}\right\rangle .
\]
Therefore, by the union bound, we obtain:
\[
\mathbb{P}\left\{ \left\Vert \boldsymbol{M}_{1,1}\right\Vert _{{\rm op}}\geq\delta\right\} \leq\mathbb{P}\left\{ \max_{\boldsymbol{u},\boldsymbol{v}\in{\cal N}}\left\langle \boldsymbol{u},\boldsymbol{M}_{1,1}\boldsymbol{v}\right\rangle \geq\delta/3\right\} \leq C\exp\left(Cd-C\delta^{2}N/\kappa^{2}\right).
\]
This proves the claim.

\paragraph{Step 2: Bounding $\left\Vert \boldsymbol{M}_{2}\right\Vert _{{\rm op}}$.}

The procedure is similar to the bounding of $\left\Vert \boldsymbol{M}_{1}\right\Vert _{{\rm op}}$,
with some tweaks. In particular, for $\sigma$ being the ReLU, $\sigma''\left(\cdot\right)=\delta\left(\cdot\right)$
the Dirac-delta function, which presents technical challenges that
we circumvent in the following. To lighten notations, define $\boldsymbol{Q}=\kappa\left(\boldsymbol{\theta}_{1},...,\boldsymbol{\theta}_{N}\right)^{\top}\in\mathbb{R}^{N\times d}$.
One can then rewrite:
\[
\frac{1}{N}\sum_{i=1}^{N}\kappa\boldsymbol{D}\boldsymbol{\theta}_{i}\sigma\left(\left\langle \kappa\boldsymbol{D}\boldsymbol{\theta}_{i},\boldsymbol{x}\right\rangle \right)=\frac{1}{N}\boldsymbol{D}\boldsymbol{Q}^{\top}\sigma\left(\boldsymbol{Q}\boldsymbol{D}^{\top}\boldsymbol{x}\right).
\]
We have:
\[
\boldsymbol{M}_{2}=\kappa\mathbb{E}_{\boldsymbol{z}}\left\{ \sigma''\left(\left\langle \boldsymbol{\Sigma}\boldsymbol{\zeta},\boldsymbol{z}\right\rangle \right)\left\langle \boldsymbol{D}^{\top}\boldsymbol{\zeta},\frac{1}{N}\boldsymbol{Q}^{\top}\sigma\left(\frac{1}{\kappa}\boldsymbol{Q}\boldsymbol{D}^{\top}\boldsymbol{\Sigma}\boldsymbol{z}\right)\right\rangle \boldsymbol{\Sigma}\boldsymbol{z}\boldsymbol{z}^{\top}\boldsymbol{\Sigma}\right\} ,
\]
where $\boldsymbol{z}\sim\mathsf{N}\left(0,\boldsymbol{I}_{d}\right)$.
Notice that for $w=\left\langle \boldsymbol{\Sigma}\boldsymbol{\zeta},\boldsymbol{z}\right\rangle \sim\mathsf{N}\left(0,\left\Vert \boldsymbol{\Sigma}\boldsymbol{\zeta}\right\Vert _{2}^{2}\right)$,
\[
\left(w,\boldsymbol{z}\right)\stackrel{{\rm d}}{=}\left(w,{\rm Proj}_{\boldsymbol{\Sigma}\boldsymbol{\zeta}}^{\perp}\tilde{\boldsymbol{z}}+\frac{w}{\left\Vert \boldsymbol{\Sigma}\boldsymbol{\zeta}\right\Vert _{2}^{2}}\boldsymbol{\Sigma}\boldsymbol{\zeta}\right),
\]
for $\tilde{\boldsymbol{z}}\sim\mathsf{N}\left(0,\boldsymbol{I}_{d}\right)$
independent of $w$. Therefore, using the fact $\sigma''\left(\cdot\right)=\delta\left(\cdot\right)$,
it is easy to see that:
\begin{align*}
\boldsymbol{M}_{2} & =\kappa\mathbb{E}_{w,\tilde{\boldsymbol{z}}}\left\{ \sigma''\left(w\right)\left\langle \boldsymbol{D}^{\top}\boldsymbol{\zeta},\frac{1}{N}\boldsymbol{Q}^{\top}\sigma\left(\frac{1}{\kappa}\boldsymbol{Q}\boldsymbol{D}^{\top}\boldsymbol{\Sigma}\left({\rm Proj}_{\boldsymbol{\Sigma}\boldsymbol{\zeta}}^{\perp}\tilde{\boldsymbol{z}}+\frac{w}{\left\Vert \boldsymbol{\Sigma}\boldsymbol{\zeta}\right\Vert _{2}^{2}}\boldsymbol{\Sigma}\boldsymbol{\zeta}\right)\right)\right\rangle \boldsymbol{\Sigma}{\rm Proj}_{\boldsymbol{\Sigma}\boldsymbol{\zeta}}^{\perp}\tilde{\boldsymbol{z}}\tilde{\boldsymbol{z}}^{\top}{\rm Proj}_{\boldsymbol{\Sigma}\boldsymbol{\zeta}}^{\perp}\boldsymbol{\Sigma}\right\} \\
 & \qquad+\kappa\mathbb{E}_{w,\tilde{\boldsymbol{z}}}\left\{ \sigma''\left(w\right)\left\langle \boldsymbol{D}^{\top}\boldsymbol{\zeta},\frac{1}{N}\boldsymbol{Q}^{\top}\sigma\left(\frac{1}{\kappa}\boldsymbol{Q}\boldsymbol{D}^{\top}\boldsymbol{\Sigma}\left({\rm Proj}_{\boldsymbol{\Sigma}\boldsymbol{\zeta}}^{\perp}\tilde{\boldsymbol{z}}+\frac{w}{\left\Vert \boldsymbol{\Sigma}\boldsymbol{\zeta}\right\Vert _{2}^{2}}\boldsymbol{\Sigma}\boldsymbol{\zeta}\right)\right)\right\rangle \frac{w}{\left\Vert \boldsymbol{\Sigma}\boldsymbol{\zeta}\right\Vert _{2}^{2}}\boldsymbol{\Sigma}^{2}\boldsymbol{\zeta}\tilde{\boldsymbol{z}}^{\top}{\rm Proj}_{\boldsymbol{\Sigma}\boldsymbol{\zeta}}^{\perp}\boldsymbol{\Sigma}\right\} \\
 & \qquad+\kappa\mathbb{E}_{w,\tilde{\boldsymbol{z}}}\left\{ \sigma''\left(w\right)\left\langle \boldsymbol{D}^{\top}\boldsymbol{\zeta},\frac{1}{N}\boldsymbol{Q}^{\top}\sigma\left(\frac{1}{\kappa}\boldsymbol{Q}\boldsymbol{D}^{\top}\boldsymbol{\Sigma}\left({\rm Proj}_{\boldsymbol{\Sigma}\boldsymbol{\zeta}}^{\perp}\tilde{\boldsymbol{z}}+\frac{w}{\left\Vert \boldsymbol{\Sigma}\boldsymbol{\zeta}\right\Vert _{2}^{2}}\boldsymbol{\Sigma}\boldsymbol{\zeta}\right)\right)\right\rangle \boldsymbol{\Sigma}{\rm Proj}_{\boldsymbol{\Sigma}\boldsymbol{\zeta}}^{\perp}\tilde{\boldsymbol{z}}\boldsymbol{\zeta}^{\top}\boldsymbol{\Sigma}^{2}\frac{w}{\left\Vert \boldsymbol{\Sigma}\boldsymbol{\zeta}\right\Vert _{2}^{2}}\right\} \\
 & \qquad+\kappa\mathbb{E}_{w,\tilde{\boldsymbol{z}}}\left\{ \sigma''\left(w\right)\left\langle \boldsymbol{D}^{\top}\boldsymbol{\zeta},\frac{1}{N}\boldsymbol{Q}^{\top}\sigma\left(\frac{1}{\kappa}\boldsymbol{Q}\boldsymbol{D}^{\top}\boldsymbol{\Sigma}\left({\rm Proj}_{\boldsymbol{\Sigma}\boldsymbol{\zeta}}^{\perp}\tilde{\boldsymbol{z}}+\frac{w}{\left\Vert \boldsymbol{\Sigma}\boldsymbol{\zeta}\right\Vert _{2}^{2}}\boldsymbol{\Sigma}\boldsymbol{\zeta}\right)\right)\right\rangle \frac{w^{2}}{\left\Vert \boldsymbol{\Sigma}\boldsymbol{\zeta}\right\Vert _{2}^{4}}\boldsymbol{\Sigma}^{2}\boldsymbol{\zeta}\boldsymbol{\zeta}^{\top}\boldsymbol{\Sigma}^{2}\right\} \\
 & =\frac{\kappa}{\sqrt{2\pi}\left\Vert \boldsymbol{\Sigma}\boldsymbol{\zeta}\right\Vert _{2}}\mathbb{E}_{\tilde{\boldsymbol{z}}}\left\{ \left\langle \boldsymbol{D}^{\top}\boldsymbol{\zeta},\frac{1}{N}\boldsymbol{Q}^{\top}\sigma\left(\frac{1}{\kappa}\boldsymbol{Q}\boldsymbol{D}^{\top}\boldsymbol{S}\tilde{\boldsymbol{z}}\right)\right\rangle \boldsymbol{S}\tilde{\boldsymbol{z}}\tilde{\boldsymbol{z}}^{\top}\boldsymbol{S}^{\top}\right\} ,
\end{align*}
in which we let $\boldsymbol{S}=\boldsymbol{\Sigma}{\rm Proj}_{\boldsymbol{\Sigma}\boldsymbol{\zeta}}^{\perp}$
for brevity.

After this simplification, the analysis of $\boldsymbol{M}_{2}$ is
similar to $\boldsymbol{M}_{1,1}$. Given fixed $\boldsymbol{u},\boldsymbol{v}\in\mathbb{S}^{d-1}$,
\[
\left\langle \boldsymbol{u},\boldsymbol{M}_{2}\boldsymbol{v}\right\rangle =\frac{1}{N}\sum_{i=1}^{N}M_{2,i}^{\boldsymbol{u},\boldsymbol{v}},\qquad M_{2,i}^{\boldsymbol{u},\boldsymbol{v}}=\frac{\kappa}{\sqrt{2\pi}\left\Vert \boldsymbol{\Sigma}\boldsymbol{\zeta}\right\Vert _{2}}\mathbb{E}_{\tilde{\boldsymbol{z}}}\left\{ \left\langle \boldsymbol{D}^{\top}\boldsymbol{\zeta},\kappa\boldsymbol{\theta}_{i}\right\rangle \sigma\left(\left\langle \boldsymbol{\theta}_{i},\boldsymbol{D}^{\top}\boldsymbol{S}\tilde{\boldsymbol{z}}\right\rangle \right)\left\langle \boldsymbol{u},\boldsymbol{S}\tilde{\boldsymbol{z}}\right\rangle \left\langle \boldsymbol{v},\boldsymbol{S}\tilde{\boldsymbol{z}}\right\rangle \right\} .
\]
First notice that $\left(M_{2,i}^{\boldsymbol{u},\boldsymbol{v}}\right)_{i\leq N}$
are i.i.d. By Stein's lemma,
\[
\mathbb{E}_{\boldsymbol{\theta}}\left\{ \kappa\boldsymbol{\theta}_{i}\sigma\left(\left\langle \boldsymbol{\theta}_{i},\boldsymbol{D}^{\top}\boldsymbol{S}\tilde{\boldsymbol{z}}\right\rangle \right)\right\} =\mathbb{E}_{\boldsymbol{\theta}}\left\{ \sigma'\left(\left\langle \kappa\boldsymbol{\theta}_{i},\boldsymbol{D}^{\top}\boldsymbol{S}\tilde{\boldsymbol{z}}\right\rangle \right)\right\} \frac{1}{\kappa}\boldsymbol{D}^{\top}\boldsymbol{S}\tilde{\boldsymbol{z}}=\frac{1}{2\kappa}\boldsymbol{D}^{\top}\boldsymbol{S}\tilde{\boldsymbol{z}}.
\]
This yields
\[
\mathbb{E}\left\{ M_{2,i}^{\boldsymbol{u},\boldsymbol{v}}\right\} =\frac{1}{2\sqrt{2\pi}\left\Vert \boldsymbol{\Sigma}\boldsymbol{\zeta}\right\Vert _{2}}\mathbb{E}_{\tilde{\boldsymbol{z}}}\left\{ \left\langle \boldsymbol{D}^{\top}\boldsymbol{\zeta},\boldsymbol{D}^{\top}\boldsymbol{S}\tilde{\boldsymbol{z}}\right\rangle \left\langle \boldsymbol{u},\boldsymbol{S}\tilde{\boldsymbol{z}}\right\rangle \left\langle \boldsymbol{v},\boldsymbol{S}\tilde{\boldsymbol{z}}\right\rangle \right\} =0,
\]
since $\tilde{\boldsymbol{z}}$ is symmetric. Next, for any positive
integer $p\geq1$, 
\begin{align*}
\mathbb{E}\left\{ \left|M_{2,i}^{\boldsymbol{u},\boldsymbol{v}}\right|^{p}\right\}  & =C^{p}\frac{\kappa^{p}}{\left\Vert \boldsymbol{\Sigma}\boldsymbol{\zeta}\right\Vert _{2}^{p}}\mathbb{E}\left\{ \left|\left\langle \boldsymbol{D}^{\top}\boldsymbol{\zeta},\kappa\boldsymbol{\theta}_{i}\right\rangle \right|^{p}\mathbb{E}_{\tilde{\boldsymbol{z}}}\left\{ \sigma\left(\left\langle \boldsymbol{\theta}_{i},\boldsymbol{D}^{\top}\boldsymbol{S}\tilde{\boldsymbol{z}}\right\rangle \right)\left\langle \boldsymbol{u},\boldsymbol{S}\tilde{\boldsymbol{z}}\right\rangle \left\langle \boldsymbol{v},\boldsymbol{S}\tilde{\boldsymbol{z}}\right\rangle \right\} ^{p}\right\} \\
 & \leq C^{p}\frac{\kappa^{p}}{\left\Vert \boldsymbol{\Sigma}\boldsymbol{\zeta}\right\Vert _{2}^{p}}\mathbb{E}\left\{ \left|\left\langle \boldsymbol{D}^{\top}\boldsymbol{\zeta},\kappa\boldsymbol{\theta}_{i}\right\rangle \right|^{p}\mathbb{E}_{\tilde{\boldsymbol{z}}}\left\{ \sigma\left(\left\langle \boldsymbol{\theta}_{i},\boldsymbol{D}^{\top}\boldsymbol{S}\tilde{\boldsymbol{z}}\right\rangle \right)^{3}\right\} ^{p/3}\mathbb{E}_{\tilde{\boldsymbol{z}}}\left\{ \left|\left\langle \boldsymbol{u},\boldsymbol{S}\tilde{\boldsymbol{z}}\right\rangle \right|^{3}\right\} ^{p/3}\mathbb{E}_{\tilde{\boldsymbol{z}}}\left\{ \left|\left\langle \boldsymbol{v},\boldsymbol{S}\tilde{\boldsymbol{z}}\right\rangle \right|^{3}\right\} ^{p/3}\right\} \\
 & \leq C^{p}\frac{\kappa^{p}}{\left\Vert \boldsymbol{\Sigma}\boldsymbol{\zeta}\right\Vert _{2}^{p}}\mathbb{E}\left\{ \left|\left\langle \boldsymbol{D}^{\top}\boldsymbol{\zeta},\kappa\boldsymbol{\theta}_{i}\right\rangle \right|^{p}\left\Vert \boldsymbol{D}^{\top}\boldsymbol{S}\boldsymbol{\theta}_{i}\right\Vert _{2}^{p}\left\Vert \boldsymbol{S}\boldsymbol{u}\right\Vert _{2}^{p}\left\Vert \boldsymbol{S}\boldsymbol{v}\right\Vert _{2}^{p}\right\} \\
 & \leq C^{p}\frac{1}{\left\Vert \boldsymbol{\Sigma}\boldsymbol{\zeta}\right\Vert _{2}^{p}}\mathbb{E}\left\{ \left|\left\langle \boldsymbol{D}^{\top}\boldsymbol{\zeta},\kappa\boldsymbol{\theta}_{i}\right\rangle \right|^{p}\left\Vert \kappa\boldsymbol{\theta}_{i}\right\Vert _{2}^{p}\right\} \\
 & \leq C^{p}\frac{1}{\left\Vert \boldsymbol{\Sigma}\boldsymbol{\zeta}\right\Vert _{2}^{p}}\sqrt{\mathbb{E}\left\{ \left|\left\langle \boldsymbol{D}^{\top}\boldsymbol{\zeta},\kappa\boldsymbol{\theta}_{i}\right\rangle \right|^{2p}\right\} \mathbb{E}\left\{ \left\Vert \kappa\boldsymbol{\theta}_{i}\right\Vert _{2}^{2p}\right\} }\\
 & =C^{p}\frac{\left\Vert \boldsymbol{D}^{\top}\boldsymbol{\zeta}\right\Vert _{2}^{p}}{\left\Vert \boldsymbol{\Sigma}\boldsymbol{\zeta}\right\Vert _{2}^{p}}\sqrt{\mathbb{E}_{g}\left\{ g^{2p}\right\} \mathbb{E}\left\{ \left\Vert \kappa\boldsymbol{\theta}_{i}\right\Vert _{2}^{2p}\right\} }\\
 & \leq C^{p}\kappa_{*}^{-p}\sqrt{p^{p}\left(\kappa^{2p}+p^{p}\right)}\\
 & \leq C^{p}\kappa_{*}^{-p}\left(\kappa^{p}p^{p/2}+p^{p}\right)
\end{align*}
following a reasoning similar to the bounding procedure of $\boldsymbol{M}_{1,1}$,
where we note we have used $\left\Vert \boldsymbol{\Sigma}\boldsymbol{\zeta}\right\Vert _{2}\geq C\kappa_{*}\left\Vert \boldsymbol{\zeta}\right\Vert _{2}$
and $\left\Vert \boldsymbol{D}^{\top}\boldsymbol{\zeta}\right\Vert _{2}\leq C\left\Vert \boldsymbol{\zeta}\right\Vert _{2}$.
We conclude that $M_{2,i}^{\boldsymbol{u},\boldsymbol{v}}$ is a sub-exponential
random variable with $\psi_{1}$-norm $\left\Vert M_{2,i}^{\boldsymbol{u},\boldsymbol{v}}\right\Vert _{\psi_{1}}\leq C\kappa/\kappa_{*}$.
Therefore, by Lemma \ref{lem:subgauss_subexp_properties}, for $\delta\in\left(0,1\right)$,
with probability at most $C\exp\left(-C\delta^{2}N\kappa_{*}^{2}/\kappa^{2}\right)$,
\[
\left|\left\langle \boldsymbol{u},\boldsymbol{M}_{2}\boldsymbol{v}\right\rangle \right|=\left|\frac{1}{N}\sum_{i=1}^{N}M_{2,i}^{\boldsymbol{u},\boldsymbol{v}}\right|\geq\delta.
\]
Now we can reuse the same epsilon-net argument in the analysis of
$\boldsymbol{M}_{1,1}$ to obtain:
\[
\mathbb{P}\left\{ \left\Vert \boldsymbol{M}_{2}\right\Vert _{{\rm op}}\geq\delta\right\} \leq C\exp\left(Cd-C\delta^{2}N\kappa_{*}^{2}/\kappa^{2}\right).
\]

\paragraph{Step 3: Putting all together.}

From the bounds on $\left\Vert \boldsymbol{M}_{1}\right\Vert _{{\rm op}}$
and $\left\Vert \boldsymbol{M}_{2}\right\Vert _{{\rm op}}$, we obtain:
\begin{align*}
\mathbb{P}\left\{ \left\Vert \frac{1}{N}\sum_{i=1}^{N}\nabla_{11}^{2}U\left(\boldsymbol{\zeta},\boldsymbol{\theta}_{i}\right)\right\Vert _{{\rm op}}\geq C_{*}\right\}  & \leq C\exp\left(Cd-C\delta^{2}N/\kappa^{2}\right)+C\exp\left(Cd-CN\kappa_{*}^{2}/\kappa^{2}\right)\\
 & \leq C\exp\left(Cd-CN\kappa_{*}^{2}/\kappa^{2}\right),
\end{align*}
for sufficiently large $C_{*}$, recalling $\kappa_{*}\leq C$ and
choosing suitable $\delta\leq C\kappa_{*}$. This completes the proof.

\end{proof}
\begin{prop}
\label{prop:1st_setting_U_bound}Consider setting \ref{enu:ReLU_setting}.
We have, for some sufficiently large $C_{*}$, with probability at
least $1-\exp\left(Cd\log\left(\kappa/\kappa_{*}+e\right)-CN\kappa_{*}^{2}/\kappa^{2}\right)$,
\[
\sup_{\left\Vert \boldsymbol{r}\right\Vert _{\infty}\leq r_{*}}\sup_{\boldsymbol{\zeta}\in\mathbb{R}^{d}}\left\Vert \frac{1}{N}\sum_{i=1}^{N}\nabla_{11}^{2}U\left(\boldsymbol{\zeta},\boldsymbol{R}{\rm diag}\left(\boldsymbol{r}\right)\boldsymbol{R}^{\top}\boldsymbol{\theta}_{i}\right)\right\Vert _{{\rm op}}\leq r_{*}^{2}C_{*},
\]
in which $\left(\boldsymbol{\theta}_{i}\right)_{i\leq N}\sim_{{\rm i.i.d.}}\mathsf{N}\left(0,\boldsymbol{I}_{d}/d\right)$
and $r_{*}\geq0$. Here $C_{*}$ does not depend on $d$, $N$ or
$r_{*}$.
\end{prop}

\begin{proof}
The proof leverages on Lemma \ref{lem:Unorm_firstSetting} and comprises
of several steps. First of all, we note that $\boldsymbol{R}^{\top}\boldsymbol{\theta}_{i}\stackrel{{\rm d}}{=}\boldsymbol{\theta}_{i}$
since $\boldsymbol{R}$ is orthogonal. Hence we can equivalently study
the quantity:
\[
Q=\sup_{\left\Vert \boldsymbol{r}\right\Vert _{\infty}\leq r_{*}}\sup_{\boldsymbol{\zeta}\in\mathbb{R}^{d}}\left\Vert \frac{1}{N}\sum_{i=1}^{N}\nabla_{11}^{2}U\left(\boldsymbol{\zeta},\boldsymbol{R}{\rm diag}\left(\boldsymbol{r}\right)\boldsymbol{\theta}_{i}\right)\right\Vert _{{\rm op}}.
\]

\paragraph*{Step 1: Reduction of the supremization set.}

First recall that
\[
\frac{1}{N}\sum_{i=1}^{N}\nabla_{11}^{2}U\left(\boldsymbol{\zeta},\boldsymbol{R}{\rm diag}\left(\boldsymbol{r}\right)\boldsymbol{\theta}_{i}\right)=\boldsymbol{M}_{1}\left(\boldsymbol{\zeta},\boldsymbol{r}\right)+\boldsymbol{M}_{1}\left(\boldsymbol{\zeta},\boldsymbol{r}\right)^{\top}+\boldsymbol{M}_{2}\left(\boldsymbol{\zeta},\boldsymbol{r}\right)\in\mathbb{R}^{d\times d},
\]
for which
\begin{align*}
\boldsymbol{M}_{1}\left(\boldsymbol{\zeta},\boldsymbol{r}\right) & =\frac{1}{N}\sum_{i=1}^{N}\kappa^{3}\mathbb{E}_{{\cal P}}\left\{ \sigma'\left(\left\langle \kappa\boldsymbol{\zeta},\boldsymbol{x}\right\rangle \right)\sigma\left(\left\langle \kappa\boldsymbol{R}{\rm diag}\left(\boldsymbol{r}\right)\boldsymbol{\theta}_{i},\boldsymbol{x}\right\rangle \right)\boldsymbol{R}{\rm diag}\left(\boldsymbol{r}\right)\boldsymbol{\theta}_{i}\boldsymbol{x}^{\top}\right\} ,\\
\boldsymbol{M}_{2}\left(\boldsymbol{\zeta},\boldsymbol{r}\right) & =\frac{1}{N}\sum_{i=1}^{N}\kappa^{4}\mathbb{E}_{{\cal P}}\left\{ \left\langle \boldsymbol{\zeta},\boldsymbol{R}{\rm diag}\left(\boldsymbol{r}\right)\boldsymbol{\theta}_{i}\right\rangle \sigma''\left(\left\langle \kappa\boldsymbol{\zeta},\boldsymbol{x}\right\rangle \right)\sigma\left(\left\langle \kappa\boldsymbol{R}{\rm diag}\left(\boldsymbol{r}\right)\boldsymbol{\theta}_{i},\boldsymbol{x}\right\rangle \right)\boldsymbol{x}\boldsymbol{x}^{\top}\right\} .
\end{align*}
We make a few observations. Firstly, for any $c>0$, since $\sigma$
is the ReLU, $\boldsymbol{M}_{1}\left(c\boldsymbol{\zeta},\boldsymbol{r}\right)=\boldsymbol{M}_{1}\left(\boldsymbol{\zeta},\boldsymbol{r}\right)$
and $\boldsymbol{M}_{1}\left(\boldsymbol{\zeta},c\boldsymbol{r}\right)=c^{2}\boldsymbol{M}_{1}\left(\boldsymbol{\zeta},\boldsymbol{r}\right)$.
Secondly, as shown in the proof of Lemma \ref{lem:Unorm_firstSetting},
\[
\boldsymbol{M}_{2}\left(\boldsymbol{\zeta},\boldsymbol{r}\right)=\frac{\kappa}{\sqrt{2\pi}\left\Vert \boldsymbol{\Sigma}\boldsymbol{\zeta}\right\Vert _{2}}\mathbb{E}_{\boldsymbol{z}}\left\{ \left\langle {\rm diag}\left(\boldsymbol{r}\right)\boldsymbol{R}^{\top}\boldsymbol{\zeta},\frac{1}{N}\boldsymbol{Q}^{\top}\sigma\left(\frac{1}{\kappa}\boldsymbol{Q}{\rm diag}\left(\boldsymbol{r}\right)\boldsymbol{R}^{\top}\boldsymbol{S}\boldsymbol{z}\right)\right\rangle \boldsymbol{S}\boldsymbol{z}\boldsymbol{z}^{\top}\boldsymbol{S}^{\top}\right\} ,
\]
for $\boldsymbol{z}\sim\mathsf{N}\left(0,\boldsymbol{I}_{d}\right)$
(see the proof of Lemma \ref{lem:Unorm_firstSetting} for the definitions
of $\boldsymbol{Q}$ and $\boldsymbol{S}$, which are unimportant
here). It is then easy to see that $\boldsymbol{M}_{2}\left(c\boldsymbol{\zeta},\boldsymbol{r}\right)=\boldsymbol{M}_{2}\left(\boldsymbol{\zeta},\boldsymbol{r}\right)$
and $\boldsymbol{M}_{2}\left(\boldsymbol{\zeta},c\boldsymbol{r}\right)=c^{2}\boldsymbol{M}_{2}\left(\boldsymbol{\zeta},\boldsymbol{r}\right)$.
Therefore, we obtain the following simplification:
\begin{equation}
Q=r_{*}^{2}\sup_{\left\Vert \boldsymbol{r}\right\Vert _{\infty}\leq1}\sup_{\boldsymbol{\zeta}\in{\cal S}}\left\Vert \frac{1}{N}\sum_{i=1}^{N}\nabla_{11}^{2}U\left(\boldsymbol{\zeta},\boldsymbol{R}{\rm diag}\left(\boldsymbol{r}\right)\boldsymbol{\theta}_{i}\right)\right\Vert _{{\rm op}}\label{eq:prop_firstSetting_supUnorm_set}
\end{equation}
for ${\cal S}={\cal B}_{d}\left(1\right)\backslash{\cal B}_{d}\left(1/2\right)$.
Here the exclusion of $\boldsymbol{\zeta}=\boldsymbol{0}$ from ${\cal S}$
can be easily reasoned by a continuity argument.

\paragraph*{Step 2: Epsilon-net argument.}

From here onwards, we focus on the supremization over $\boldsymbol{\zeta}\in{\cal S}$
and $\left\Vert \boldsymbol{r}\right\Vert _{\infty}\leq1$. Fix $\gamma\in\left(0,1/3\right)$.
Consider an epsilon-net ${\cal N}_{d}^{\infty}\left(\gamma\right)\subset\left\{ \boldsymbol{r}:\;\left\Vert \boldsymbol{r}\right\Vert _{\infty}\leq1\right\} $
such that for any $\boldsymbol{r}$ with $\left\Vert \boldsymbol{r}\right\Vert _{\infty}\leq1$,
there exists $\boldsymbol{r}'\in{\cal N}_{d}^{\infty}\left(\gamma\right)$
with $\left\Vert \boldsymbol{r}-\boldsymbol{r}'\right\Vert _{\infty}\leq\gamma$.
Likewise, consider an epsilon-net ${\cal N}_{d}^{2}\left(\gamma\right)\subset{\cal S}$
in which for any $\boldsymbol{\zeta}\in{\cal S}$, there exists $\boldsymbol{\zeta}'\in{\cal N}_{d}^{2}\left(\gamma\right)$
such that $\left\Vert \boldsymbol{\zeta}-\boldsymbol{\zeta}'\right\Vert _{2}\leq\gamma$.
Note that ${\cal N}_{d}^{2}\left(\gamma\right)\subset{\cal B}_{d}\left(1\right)$.
A standard volumetric argument \cite{vershynin2010introduction} shows
that there exist such epsilon-nets with sizes 
\[
\left|{\cal N}_{d}^{\infty}\left(\gamma\right)\right|,\left|{\cal N}_{d}^{2}\left(\gamma\right)\right|\leq\left(\frac{3}{\gamma}\right)^{d}.
\]
Consider $\boldsymbol{r}$ and $\boldsymbol{r}'\in{\cal N}_{d}^{\infty}\left(\gamma\right)$
such that $\left\Vert \boldsymbol{r}\right\Vert _{\infty}\leq1$ and
$\left\Vert \boldsymbol{r}-\boldsymbol{r}'\right\Vert _{\infty}\leq\gamma$,
and $\boldsymbol{\zeta}\in{\cal S}$ and $\boldsymbol{\zeta}'\in{\cal N}_{d}^{2}\left(\gamma\right)$
such that $\left\Vert \boldsymbol{\zeta}-\boldsymbol{\zeta}'\right\Vert _{2}\leq\gamma$.
We have from the mean value theorem:
\begin{align}
 & \left|\left\Vert \frac{1}{N}\sum_{i=1}^{N}\nabla_{11}^{2}U\left(\boldsymbol{\zeta},\boldsymbol{R}{\rm diag}\left(\boldsymbol{r}\right)\boldsymbol{\theta}_{i}\right)\right\Vert _{{\rm op}}-\left\Vert \frac{1}{N}\sum_{i=1}^{N}\nabla_{11}^{2}U\left(\boldsymbol{\zeta}',\boldsymbol{R}{\rm diag}\left(\boldsymbol{r}'\right)\boldsymbol{\theta}_{i}\right)\right\Vert _{{\rm op}}\right|\nonumber \\
 & \quad\leq\left\Vert \frac{1}{N}\sum_{i=1}^{N}\nabla_{11}^{2}U\left(\boldsymbol{\zeta},\boldsymbol{R}{\rm diag}\left(\boldsymbol{r}\right)\boldsymbol{\theta}_{i}\right)-\nabla_{11}^{2}U\left(\boldsymbol{\zeta}',\boldsymbol{R}{\rm diag}\left(\boldsymbol{r}\right)\boldsymbol{\theta}_{i}\right)\right\Vert _{{\rm op}}\nonumber \\
 & \quad\qquad+\left\Vert \frac{1}{N}\sum_{i=1}^{N}\nabla_{11}^{2}U\left(\boldsymbol{\zeta}',\boldsymbol{R}{\rm diag}\left(\boldsymbol{r}\right)\boldsymbol{\theta}_{i}\right)-\nabla_{11}^{2}U\left(\boldsymbol{\zeta}',\boldsymbol{R}{\rm diag}\left(\boldsymbol{r}'\right)\boldsymbol{\theta}_{i}\right)\right\Vert _{{\rm op}}\nonumber \\
 & \quad\stackrel{\left(a\right)}{\leq}\frac{1}{N}\sum_{i=1}^{N}\left\Vert \nabla_{111}^{3}U\left[\boldsymbol{u}_{i},\boldsymbol{R}{\rm diag}\left(\boldsymbol{r}\right)\boldsymbol{\theta}_{i}\right]\right\Vert _{{\rm op}}\left\Vert \boldsymbol{\zeta}-\boldsymbol{\zeta}'\right\Vert _{2}\nonumber \\
 & \quad\qquad+\frac{1}{N}\sum_{i=1}^{N}\left\Vert \nabla_{121}^{3}U\left[\boldsymbol{\zeta}',\boldsymbol{v}_{i}\right]\right\Vert _{{\rm op}}\left\Vert \boldsymbol{R}{\rm diag}\left(\boldsymbol{r}-\boldsymbol{r}'\right)\boldsymbol{\theta}_{i}\right\Vert _{2}\nonumber \\
 & \quad\stackrel{\left(b\right)}{\leq}\frac{1}{N}\sum_{i=1}^{N}\left\Vert \nabla_{111}^{3}U\left[\boldsymbol{u}_{i},\boldsymbol{R}{\rm diag}\left(\boldsymbol{r}\right)\boldsymbol{\theta}_{i}\right]\right\Vert _{{\rm op}}\gamma+\frac{1}{N}\sum_{i=1}^{N}C\frac{\kappa^{2}}{\kappa_{*}}\left\Vert \boldsymbol{v}_{i}\right\Vert _{2}\left\Vert \boldsymbol{\theta}_{i}\right\Vert _{2}\gamma\nonumber \\
 & \quad\stackrel{\left(c\right)}{\leq}\frac{1}{N}\sum_{i=1}^{N}\left\Vert \nabla_{111}^{3}U\left[\boldsymbol{u}_{i},\boldsymbol{R}{\rm diag}\left(\boldsymbol{r}\right)\boldsymbol{\theta}_{i}\right]\right\Vert _{{\rm op}}\gamma+\frac{1}{N}\sum_{i=1}^{N}C\frac{\kappa^{2}}{\kappa_{*}}\left\Vert \boldsymbol{\theta}_{i}\right\Vert _{2}^{2}\gamma,\label{eq:prop_firstSetting_supUnorm_gap_1}
\end{align}
where in step $\left(a\right)$, we have $\boldsymbol{u}_{i}\in\left[\boldsymbol{\zeta},\boldsymbol{\zeta}'\right]$
and $\boldsymbol{v}_{i}\in\left[\boldsymbol{R}{\rm diag}\left(\boldsymbol{r}\right)\boldsymbol{\theta}_{i},\boldsymbol{R}{\rm diag}\left(\boldsymbol{r}'\right)\boldsymbol{\theta}_{i}\right]$;
in step $\left(b\right)$, we apply Proposition \ref{prop:1st_setting_op_bound};
in step $\left(c\right)$, we use the fact that
\[
\left\Vert \boldsymbol{v}_{i}\right\Vert _{2}\leq\left\Vert \boldsymbol{R}{\rm diag}\left(\boldsymbol{r}\right)\boldsymbol{\theta}_{i}\right\Vert _{2}+\left\Vert \boldsymbol{R}{\rm diag}\left(\boldsymbol{r}-\boldsymbol{r}'\right)\boldsymbol{\theta}_{i}\right\Vert _{2}\leq\left\Vert \boldsymbol{\theta}_{i}\right\Vert _{2}+\gamma\left\Vert \boldsymbol{\theta}_{i}\right\Vert _{2}\leq2\left\Vert \boldsymbol{\theta}_{i}\right\Vert _{2}.
\]
We note that since $\boldsymbol{u}_{i}\in\left[\boldsymbol{\zeta},\boldsymbol{\zeta}'\right]$,
\begin{equation}
\left\Vert \boldsymbol{u}_{i}\right\Vert _{2}\geq\left\Vert \boldsymbol{\zeta}'\right\Vert _{2}-\left\Vert \boldsymbol{\zeta}-\boldsymbol{\zeta}'\right\Vert _{2}\geq1/2-1/3=1/6.\label{eq:prop_firstSetting_supUnorm_u}
\end{equation}
To simplify the notations, let $\tilde{\boldsymbol{\theta}}_{i}=\boldsymbol{R}{\rm diag}\left(\boldsymbol{r}\right)\boldsymbol{\theta}_{i}$,
and note that $\left\Vert \tilde{\boldsymbol{\theta}}_{i}\right\Vert _{2}\leq\left\Vert \boldsymbol{\theta}_{i}\right\Vert _{2}$.
We have:
\[
\nabla_{111}^{3}U\left[\boldsymbol{u}_{i},\tilde{\boldsymbol{\theta}}_{i}\right]=\boldsymbol{M}_{1,i}+\boldsymbol{M}_{2,i}+\boldsymbol{M}_{3,i}+\boldsymbol{M}_{4,i}\in\left(\mathbb{R}^{d}\right)^{\otimes3},
\]
for which
\begin{align*}
\boldsymbol{M}_{1,i} & =\kappa^{4}\mathbb{E}_{{\cal P}}\left\{ \sigma''\left(\left\langle \kappa\boldsymbol{u}_{i},\boldsymbol{x}\right\rangle \right)\sigma\left(\left\langle \kappa\tilde{\boldsymbol{\theta}}_{i},\boldsymbol{x}\right\rangle \right)\boldsymbol{x}\otimes\tilde{\boldsymbol{\theta}}_{i}\otimes\boldsymbol{x}\right\} ,\\
\boldsymbol{M}_{2,i} & =\kappa^{4}\mathbb{E}_{{\cal P}}\left\{ \sigma''\left(\left\langle \kappa\boldsymbol{u}_{i},\boldsymbol{x}\right\rangle \right)\sigma\left(\left\langle \kappa\tilde{\boldsymbol{\theta}}_{i},\boldsymbol{x}\right\rangle \right)\boldsymbol{x}\otimes\boldsymbol{x}\otimes\tilde{\boldsymbol{\theta}}_{i}\right\} ,\\
\boldsymbol{M}_{3,i} & =\kappa^{4}\mathbb{E}_{{\cal P}}\left\{ \sigma''\left(\left\langle \kappa\boldsymbol{u}_{i},\boldsymbol{x}\right\rangle \right)\sigma\left(\left\langle \kappa\tilde{\boldsymbol{\theta}}_{i},\boldsymbol{x}\right\rangle \right)\tilde{\boldsymbol{\theta}}_{i}\otimes\boldsymbol{x}\otimes\boldsymbol{x}\right\} ,\\
\boldsymbol{M}_{4,i} & =\kappa^{5}\mathbb{E}_{{\cal P}}\left\{ \left\langle \boldsymbol{u}_{i},\tilde{\boldsymbol{\theta}}_{i}\right\rangle \sigma'''\left(\left\langle \kappa\boldsymbol{u}_{i},\boldsymbol{x}\right\rangle \right)\sigma\left(\left\langle \kappa\tilde{\boldsymbol{\theta}}_{i},\boldsymbol{x}\right\rangle \right)\boldsymbol{x}\otimes\boldsymbol{x}\otimes\boldsymbol{x}\right\} .
\end{align*}
Note that $\left\Vert \boldsymbol{M}_{1,i}\right\Vert _{{\rm op}}=\left\Vert \boldsymbol{M}_{2,i}\right\Vert _{{\rm op}}=\left\Vert \boldsymbol{M}_{3,i}\right\Vert _{{\rm op}}$.
Then Eq. (\ref{eq:prop_firstSetting_supUnorm_set}) and (\ref{eq:prop_firstSetting_supUnorm_gap_1})
yield
\begin{equation}
\left|Q-Q_{\gamma}\right|\leq r_{*}^{2}\frac{1}{N}\sum_{i=1}^{N}\left(3\left\Vert \boldsymbol{M}_{1,i}\right\Vert _{{\rm op}}+\left\Vert \boldsymbol{M}_{4,i}\right\Vert _{{\rm op}}\right)\gamma+r_{*}^{2}\frac{1}{N}\sum_{i=1}^{N}C\frac{\kappa^{2}}{\kappa_{*}}\left\Vert \boldsymbol{\theta}_{i}\right\Vert _{2}^{2}\gamma,\label{eq:prop_firstSetting_supUnorm_gap_2}
\end{equation}
in which we define:
\[
Q_{\gamma}=r_{*}^{2}\max_{\boldsymbol{r}\in{\cal N}_{d}^{\infty}\left(\gamma\right)}\max_{\boldsymbol{\zeta}\in{\cal N}_{d}^{2}\left(\gamma\right)}\left\Vert \frac{1}{N}\sum_{i=1}^{N}\nabla_{11}^{2}U\left(\boldsymbol{\zeta},\boldsymbol{R}{\rm diag}\left(\boldsymbol{r}\right)\boldsymbol{\theta}_{i}\right)\right\Vert _{{\rm op}}.
\]
The next two steps are devoted to bounding $\left\Vert \boldsymbol{M}_{1,i}\right\Vert _{{\rm op}}$
and $\left\Vert \boldsymbol{M}_{4,i}\right\Vert _{{\rm op}}$.

\paragraph*{Step 3: Bounding $\left\Vert \boldsymbol{M}_{1,i}\right\Vert _{{\rm op}}$.}

To bound $\left\Vert \boldsymbol{M}_{1,i}\right\Vert _{{\rm op}}$,
we have for any $\boldsymbol{a},\boldsymbol{b},\boldsymbol{c}\in\mathbb{R}^{d}$:
\begin{align*}
\left\langle \boldsymbol{M}_{1,i},\boldsymbol{a}\otimes\boldsymbol{b}\otimes\boldsymbol{c}\right\rangle  & =\kappa^{4}\mathbb{E}_{{\cal P}}\left\{ \sigma''\left(\left\langle \kappa\boldsymbol{u}_{i},\boldsymbol{x}\right\rangle \right)\sigma\left(\left\langle \kappa\tilde{\boldsymbol{\theta}}_{i},\boldsymbol{x}\right\rangle \right)\left\langle \boldsymbol{a},\boldsymbol{x}\right\rangle \left\langle \boldsymbol{b},\tilde{\boldsymbol{\theta}}_{i}\right\rangle \left\langle \boldsymbol{c},\boldsymbol{x}\right\rangle \right\} \\
 & =\kappa^{2}\mathbb{E}_{\boldsymbol{z}}\left\{ \sigma''\left(\left\langle \boldsymbol{\Sigma}\boldsymbol{u}_{i},\boldsymbol{z}\right\rangle \right)\sigma\left(\left\langle \boldsymbol{\Sigma}\tilde{\boldsymbol{\theta}}_{i},\boldsymbol{z}\right\rangle \right)\left\langle \boldsymbol{\Sigma}\boldsymbol{a},\boldsymbol{z}\right\rangle \left\langle \boldsymbol{b},\tilde{\boldsymbol{\theta}}_{i}\right\rangle \left\langle \boldsymbol{\Sigma}\boldsymbol{c},\boldsymbol{z}\right\rangle \right\} ,
\end{align*}
where $\boldsymbol{z}\sim\mathsf{N}\left(0,\boldsymbol{I}_{d}\right)$.
Notice that for $w_{i}=\left\langle \boldsymbol{\Sigma}\boldsymbol{u}_{i},\boldsymbol{z}\right\rangle \sim\mathsf{N}\left(0,\left\Vert \boldsymbol{\Sigma}\boldsymbol{u}_{i}\right\Vert _{2}^{2}\right)$,
\[
\left(w_{i},\boldsymbol{z}\right)\stackrel{{\rm d}}{=}\left(w_{i},{\rm Proj}_{\boldsymbol{\Sigma}\boldsymbol{u}_{i}}^{\perp}\tilde{\boldsymbol{z}}+\frac{w_{i}}{\left\Vert \boldsymbol{\Sigma}\boldsymbol{u}_{i}\right\Vert _{2}^{2}}\boldsymbol{\Sigma}\boldsymbol{u}_{i}\right),
\]
in which $\tilde{\boldsymbol{z}}\sim\mathsf{N}\left(0,\boldsymbol{I}_{d}\right)$
independent of $w_{i}$. Therefore,
\begin{align*}
 & \left\langle \boldsymbol{M}_{1,i},\boldsymbol{a}\otimes\boldsymbol{b}\otimes\boldsymbol{c}\right\rangle \\
 & \quad=\kappa^{2}\mathbb{E}_{w_{i},\tilde{\boldsymbol{z}}}\Bigg\{\sigma''\left(w_{i}\right)\sigma\left(\left\langle \boldsymbol{\Sigma}\tilde{\boldsymbol{\theta}}_{i},{\rm Proj}_{\boldsymbol{\Sigma}\boldsymbol{u}_{i}}^{\perp}\tilde{\boldsymbol{z}}+\frac{w_{i}}{\left\Vert \boldsymbol{\Sigma}\boldsymbol{u}_{i}\right\Vert _{2}^{2}}\boldsymbol{\Sigma}\boldsymbol{u}_{i}\right\rangle \right)\left\langle \boldsymbol{b},\tilde{\boldsymbol{\theta}}_{i}\right\rangle \\
 & \qquad\times\Bigg[\left\langle \boldsymbol{\Sigma}\boldsymbol{a},{\rm Proj}_{\boldsymbol{\Sigma}\boldsymbol{u}_{i}}^{\perp}\tilde{\boldsymbol{z}}\right\rangle \left\langle \boldsymbol{\Sigma}\boldsymbol{c},{\rm Proj}_{\boldsymbol{\Sigma}\boldsymbol{u}_{i}}^{\perp}\tilde{\boldsymbol{z}}\right\rangle +\left\langle \boldsymbol{\Sigma}\boldsymbol{a},\frac{w_{i}}{\left\Vert \boldsymbol{\Sigma}\boldsymbol{u}_{i}\right\Vert _{2}^{2}}\boldsymbol{\Sigma}\boldsymbol{u}_{i}\right\rangle \left\langle \boldsymbol{\Sigma}\boldsymbol{c},{\rm Proj}_{\boldsymbol{\Sigma}\boldsymbol{u}_{i}}^{\perp}\tilde{\boldsymbol{z}}\right\rangle \\
 & \qquad\quad+\left\langle \boldsymbol{\Sigma}\boldsymbol{a},{\rm Proj}_{\boldsymbol{\Sigma}\boldsymbol{u}_{i}}^{\perp}\tilde{\boldsymbol{z}}\right\rangle \left\langle \boldsymbol{\Sigma}\boldsymbol{c},\frac{w_{i}}{\left\Vert \boldsymbol{\Sigma}\boldsymbol{u}_{i}\right\Vert _{2}^{2}}\boldsymbol{\Sigma}\boldsymbol{u}_{i}\right\rangle +\left\langle \boldsymbol{\Sigma}\boldsymbol{a},\frac{w_{i}}{\left\Vert \boldsymbol{\Sigma}\boldsymbol{u}_{i}\right\Vert _{2}^{2}}\boldsymbol{\Sigma}\boldsymbol{u}_{i}\right\rangle \left\langle \boldsymbol{\Sigma}\boldsymbol{c},\frac{w_{i}}{\left\Vert \boldsymbol{\Sigma}\boldsymbol{u}_{i}\right\Vert _{2}^{2}}\boldsymbol{\Sigma}\boldsymbol{u}_{i}\right\rangle \Bigg]\Bigg\}\\
 & \quad\stackrel{\left(a\right)}{=}\kappa^{2}\mathbb{E}_{w_{i},\tilde{\boldsymbol{z}}}\Bigg\{\sigma''\left(w_{i}\right)\sigma\left(\left\langle \boldsymbol{\Sigma}\tilde{\boldsymbol{\theta}}_{i},{\rm Proj}_{\boldsymbol{\Sigma}\boldsymbol{u}_{i}}^{\perp}\tilde{\boldsymbol{z}}+\frac{w_{i}}{\left\Vert \boldsymbol{\Sigma}\boldsymbol{u}_{i}\right\Vert _{2}^{2}}\boldsymbol{\Sigma}\boldsymbol{u}_{i}\right\rangle \right)\left\langle \boldsymbol{b},\tilde{\boldsymbol{\theta}}_{i}\right\rangle \\
 & \qquad\times\left\langle \boldsymbol{\Sigma}\boldsymbol{a},{\rm Proj}_{\boldsymbol{\Sigma}\boldsymbol{u}_{i}}^{\perp}\tilde{\boldsymbol{z}}\right\rangle \left\langle \boldsymbol{\Sigma}\boldsymbol{c},{\rm Proj}_{\boldsymbol{\Sigma}\boldsymbol{u}_{i}}^{\perp}\tilde{\boldsymbol{z}}\right\rangle \Bigg\}\\
 & \quad\stackrel{\left(b\right)}{=}\frac{\kappa^{2}}{\sqrt{2\pi}\left\Vert \boldsymbol{\Sigma}\boldsymbol{u}_{i}\right\Vert _{2}}\mathbb{E}_{\tilde{\boldsymbol{z}}}\left\{ \sigma\left(\left\langle \boldsymbol{\Sigma}\tilde{\boldsymbol{\theta}}_{i},{\rm Proj}_{\boldsymbol{\Sigma}\boldsymbol{u}_{i}}^{\perp}\tilde{\boldsymbol{z}}\right\rangle \right)\left\langle \boldsymbol{b},\tilde{\boldsymbol{\theta}}_{i}\right\rangle \left\langle \boldsymbol{\Sigma}\boldsymbol{a},{\rm Proj}_{\boldsymbol{\Sigma}\boldsymbol{u}_{i}}^{\perp}\tilde{\boldsymbol{z}}\right\rangle \left\langle \boldsymbol{\Sigma}\boldsymbol{c},{\rm Proj}_{\boldsymbol{\Sigma}\boldsymbol{u}_{i}}^{\perp}\tilde{\boldsymbol{z}}\right\rangle \right\} \\
 & \quad\stackrel{\left(c\right)}{=}\frac{\kappa^{2}}{\sqrt{2\pi}\left\Vert \boldsymbol{\Sigma}\boldsymbol{u}_{i}\right\Vert _{2}}\mathbb{E}_{\tilde{\boldsymbol{z}}}\left\{ \sigma\left(\left\langle \boldsymbol{S}_{i}\tilde{\boldsymbol{\theta}}_{i},\tilde{\boldsymbol{z}}\right\rangle \right)\left\langle \boldsymbol{b},\tilde{\boldsymbol{\theta}}_{i}\right\rangle \left\langle \boldsymbol{S}_{i}\boldsymbol{a},\tilde{\boldsymbol{z}}\right\rangle \left\langle \boldsymbol{S}_{i}\boldsymbol{c},\tilde{\boldsymbol{z}}\right\rangle \right\} \\
 & \quad\stackrel{\left(d\right)}{\leq}C\frac{\kappa^{2}}{\kappa_{*}}\left\Vert \boldsymbol{b}\right\Vert _{2}\left\Vert \tilde{\boldsymbol{\theta}}_{i}\right\Vert _{2}\mathbb{E}_{\tilde{\boldsymbol{z}}}\left\{ \sigma\left(\left\langle \boldsymbol{S}_{i}\tilde{\boldsymbol{\theta}}_{i},\tilde{\boldsymbol{z}}\right\rangle \right)^{3}\right\} ^{1/3}\mathbb{E}_{\tilde{\boldsymbol{z}}}\left\{ \left|\left\langle \boldsymbol{S}_{i}\boldsymbol{a},\tilde{\boldsymbol{z}}\right\rangle \right|^{3}\right\} ^{1/3}\mathbb{E}_{\tilde{\boldsymbol{z}}}\left\{ \left|\left\langle \boldsymbol{S}_{i}\boldsymbol{c},\tilde{\boldsymbol{z}}\right\rangle \right|^{3}\right\} ^{1/3}\\
 & \quad\stackrel{\left(e\right)}{\leq}C\frac{\kappa^{2}}{\kappa_{*}}\left\Vert \boldsymbol{b}\right\Vert _{2}\left\Vert \boldsymbol{\theta}_{i}\right\Vert _{2}^{2}\left\Vert \boldsymbol{a}\right\Vert _{2}\left\Vert \boldsymbol{c}\right\Vert _{2}.
\end{align*}
where in steps $\left(a\right)$ and $\left(b\right)$, we recall
that $\sigma''\left(\cdot\right)=\delta\left(\cdot\right)$ the Dirac-delta
function and that $w_{i}\sim\mathsf{N}\left(0,\left\Vert \boldsymbol{\Sigma}\boldsymbol{u}_{i}\right\Vert _{2}^{2}\right)$;
in step $\left(c\right)$, we have define $\boldsymbol{S}_{i}={\rm Proj}_{\boldsymbol{\Sigma}\boldsymbol{u}_{i}}^{\perp}\boldsymbol{\Sigma}$
for brevity; in step $\left(d\right)$, we use $\left\Vert \boldsymbol{\Sigma}\boldsymbol{u}_{i}\right\Vert _{2}\geq\kappa_{*}\left\Vert \boldsymbol{u}_{i}\right\Vert _{2}$
and $\left\Vert \boldsymbol{u}_{i}\right\Vert _{2}\geq1/6$ from Eq.
(\ref{eq:prop_firstSetting_supUnorm_u}); in step $\left(e\right)$,
we use $\left\Vert \boldsymbol{S}_{i}\right\Vert _{{\rm op}}\leq\left\Vert \boldsymbol{\Sigma}\right\Vert _{{\rm op}}\leq C$
and $\left\Vert \tilde{\boldsymbol{\theta}}_{i}\right\Vert _{2}\leq\left\Vert \boldsymbol{\theta}_{i}\right\Vert _{2}$.
Consequently we obtain:
\[
\left\Vert \boldsymbol{M}_{1,i}\right\Vert _{{\rm op}}\leq C\frac{\kappa^{2}}{\kappa_{*}}\left\Vert \boldsymbol{\theta}_{i}\right\Vert _{2}^{2}.
\]

\paragraph*{Step 4: Bounding $\left\Vert \boldsymbol{M}_{4,i}\right\Vert _{{\rm op}}$.}

Owing to the presence of $\sigma'''$ for $\sigma$ being the ReLU,
we need to treat the expectation in this term in the distributional
sense:
\begin{align*}
 & \int_{-\infty}^{+\infty}\sigma'''\left(w\right)f\left(w\right)\frac{1}{\sqrt{2\pi}\sigma_{w}}\exp\left(-\frac{w^{2}}{2\sigma_{w}^{2}}\right){\rm d}w=-\int_{-\infty}^{+\infty}\sigma''\left(w\right)\frac{{\rm d}}{{\rm d}w}\left[f\left(w\right)\frac{1}{\sqrt{2\pi}\sigma_{w}}\exp\left(-\frac{w^{2}}{2\sigma_{w}^{2}}\right)\right]{\rm d}w\\
 & \qquad=-\frac{{\rm d}}{{\rm d}w}\left[f\left(w\right)\frac{1}{\sqrt{2\pi}\sigma_{w}}\exp\left(-\frac{w^{2}}{2\sigma_{w}^{2}}\right)\right]_{w=0}=-\frac{1}{\sqrt{2\pi}\sigma_{w}}f'\left(0\right).
\end{align*}
In particular, reusing the same argument in the simplification of
$\boldsymbol{M}_{1,i}$, for $w_{i}=\left\langle \boldsymbol{\Sigma}\boldsymbol{u}_{i},\boldsymbol{z}\right\rangle \sim\mathsf{N}\left(0,\left\Vert \boldsymbol{\Sigma}\boldsymbol{u}_{i}\right\Vert _{2}^{2}\right)$,
we have:
\begin{align*}
 & \left\langle \boldsymbol{M}_{4,i},\boldsymbol{a}\otimes\boldsymbol{b}\otimes\boldsymbol{c}\right\rangle \\
 & \quad=\kappa^{2}\mathbb{E}_{{\cal P}}\left\{ \left\langle \boldsymbol{u}_{i},\tilde{\boldsymbol{\theta}}_{i}\right\rangle \sigma'''\left(\left\langle \kappa\boldsymbol{u}_{i},\boldsymbol{x}\right\rangle \right)\sigma\left(\left\langle \kappa\tilde{\boldsymbol{\theta}}_{i},\boldsymbol{x}\right\rangle \right)\left\langle \kappa\boldsymbol{a},\boldsymbol{x}\right\rangle \left\langle \kappa\boldsymbol{b},\boldsymbol{x}\right\rangle \left\langle \kappa\boldsymbol{c},\boldsymbol{x}\right\rangle \right\} \\
 & \quad=\kappa^{2}\mathbb{E}_{w_{i},\tilde{\boldsymbol{z}}}\Bigg\{\left\langle \boldsymbol{u}_{i},\tilde{\boldsymbol{\theta}}_{i}\right\rangle \sigma'''\left(w_{i}\right)\sigma\left(\left\langle \boldsymbol{\Sigma}\tilde{\boldsymbol{\theta}}_{i},{\rm Proj}_{\boldsymbol{\Sigma}\boldsymbol{u}_{i}}^{\perp}\tilde{\boldsymbol{z}}+\frac{w_{i}}{\left\Vert \boldsymbol{\Sigma}\boldsymbol{u}_{i}\right\Vert _{2}^{2}}\boldsymbol{\Sigma}\boldsymbol{u}_{i}\right\rangle \right)\left\langle \boldsymbol{\Sigma}\boldsymbol{a},{\rm Proj}_{\boldsymbol{\Sigma}\boldsymbol{u}_{i}}^{\perp}\tilde{\boldsymbol{z}}+\frac{w_{i}}{\left\Vert \boldsymbol{\Sigma}\boldsymbol{u}_{i}\right\Vert _{2}^{2}}\boldsymbol{\Sigma}\boldsymbol{u}_{i}\right\rangle \\
 & \quad\qquad\times\left\langle \boldsymbol{\Sigma}\boldsymbol{b},{\rm Proj}_{\boldsymbol{\Sigma}\boldsymbol{u}_{i}}^{\perp}\tilde{\boldsymbol{z}}+\frac{w_{i}}{\left\Vert \boldsymbol{\Sigma}\boldsymbol{u}_{i}\right\Vert _{2}^{2}}\boldsymbol{\Sigma}\boldsymbol{u}_{i}\right\rangle \left\langle \boldsymbol{\Sigma}\boldsymbol{c},{\rm Proj}_{\boldsymbol{\Sigma}\boldsymbol{u}_{i}}^{\perp}\tilde{\boldsymbol{z}}+\frac{w_{i}}{\left\Vert \boldsymbol{\Sigma}\boldsymbol{u}_{i}\right\Vert _{2}^{2}}\boldsymbol{\Sigma}\boldsymbol{u}_{i}\right\rangle \Bigg\}\\
 & \quad=-\frac{\kappa^{2}\left\langle \boldsymbol{u}_{i},\tilde{\boldsymbol{\theta}}_{i}\right\rangle }{\sqrt{2\pi}\left\Vert \boldsymbol{\Sigma}\boldsymbol{u}_{i}\right\Vert _{2}}\mathbb{E}_{\tilde{\boldsymbol{z}}}\Bigg\{\frac{\left\langle \tilde{\boldsymbol{\theta}}_{i},\boldsymbol{\Sigma}^{2}\boldsymbol{u}_{i}\right\rangle }{\left\Vert \boldsymbol{\Sigma}\boldsymbol{u}_{i}\right\Vert _{2}^{2}}\sigma'\left(\left\langle \boldsymbol{S}_{i}\tilde{\boldsymbol{\theta}}_{i},\tilde{\boldsymbol{z}}\right\rangle \right)\left\langle \boldsymbol{S}_{i}\boldsymbol{a},\tilde{\boldsymbol{z}}\right\rangle \left\langle \boldsymbol{S}_{i}\boldsymbol{b},\tilde{\boldsymbol{z}}\right\rangle \left\langle \boldsymbol{S}_{i}\boldsymbol{c},\tilde{\boldsymbol{z}}\right\rangle \\
 & \quad\qquad+\frac{\left\langle \boldsymbol{a},\boldsymbol{\Sigma}^{2}\boldsymbol{u}_{i}\right\rangle }{\left\Vert \boldsymbol{\Sigma}\boldsymbol{u}_{i}\right\Vert _{2}^{2}}\sigma\left(\left\langle \boldsymbol{S}_{i}\tilde{\boldsymbol{\theta}}_{i},\tilde{\boldsymbol{z}}\right\rangle \right)\left\langle \boldsymbol{S}_{i}\boldsymbol{b},\tilde{\boldsymbol{z}}\right\rangle \left\langle \boldsymbol{S}_{i}\boldsymbol{c},\tilde{\boldsymbol{z}}\right\rangle +\frac{\left\langle \boldsymbol{b},\boldsymbol{\Sigma}^{2}\boldsymbol{u}_{i}\right\rangle }{\left\Vert \boldsymbol{\Sigma}\boldsymbol{u}_{i}\right\Vert _{2}^{2}}\sigma\left(\left\langle \boldsymbol{S}_{i}\tilde{\boldsymbol{\theta}}_{i},\tilde{\boldsymbol{z}}\right\rangle \right)\left\langle \boldsymbol{S}_{i}\boldsymbol{a},\tilde{\boldsymbol{z}}\right\rangle \left\langle \boldsymbol{S}_{i}\boldsymbol{c},\tilde{\boldsymbol{z}}\right\rangle \\
 & \quad\qquad+\frac{\left\langle \boldsymbol{c},\boldsymbol{\Sigma}^{2}\boldsymbol{u}_{i}\right\rangle }{\left\Vert \boldsymbol{\Sigma}\boldsymbol{u}_{i}\right\Vert _{2}^{2}}\sigma\left(\left\langle \boldsymbol{S}_{i}\tilde{\boldsymbol{\theta}}_{i},\tilde{\boldsymbol{z}}\right\rangle \right)\left\langle \boldsymbol{S}_{i}\boldsymbol{a},\tilde{\boldsymbol{z}}\right\rangle \left\langle \boldsymbol{S}_{i}\boldsymbol{b},\tilde{\boldsymbol{z}}\right\rangle \Bigg\},
\end{align*}
where we define $\boldsymbol{S}_{i}={\rm Proj}_{\boldsymbol{\Sigma}\boldsymbol{u}_{i}}^{\perp}\boldsymbol{\Sigma}$
for brevity. Then proceeding in a similar fashion to the bounding
of $\boldsymbol{M}_{1}$, one can easily show that
\[
\left\langle \boldsymbol{M}_{4,i},\boldsymbol{a}\otimes\boldsymbol{b}\otimes\boldsymbol{c}\right\rangle \leq C\frac{\kappa^{2}}{\kappa_{*}^{2}}\left\Vert \boldsymbol{\theta}_{i}\right\Vert _{2}^{2}\left\Vert \boldsymbol{a}\right\Vert _{2}\left\Vert \boldsymbol{b}\right\Vert _{2}\left\Vert \boldsymbol{c}\right\Vert _{2}.
\]
In other words,
\[
\left\Vert \boldsymbol{M}_{4,i}\right\Vert _{{\rm op}}\leq C\frac{\kappa^{2}}{\kappa_{*}^{2}}\left\Vert \boldsymbol{\theta}_{i}\right\Vert _{2}^{2}.
\]

\paragraph*{Step 5: Finishing the proof.}

From the bounds on $\left\Vert \boldsymbol{M}_{1,i}\right\Vert _{{\rm op}}$
and $\left\Vert \boldsymbol{M}_{4,i}\right\Vert _{{\rm op}}$ and
Eq. (\ref{eq:prop_firstSetting_supUnorm_gap_2}) , we get:
\[
\left|Q-Q_{\gamma}\right|\leq r_{*}^{2}\frac{1}{N}\sum_{i=1}^{N}C\frac{\kappa^{2}}{\kappa_{*}^{2}}\left\Vert \boldsymbol{\theta}_{i}\right\Vert _{2}^{2}\gamma.
\]
Notice that $\sum_{i=1}^{N}\left\Vert \kappa\boldsymbol{\theta}_{i}\right\Vert _{2}^{2}$
is a $\chi^{2}$ random variable of degree of freedom $Nd=N\kappa^{2}$,
and therefore it is a standard concentration fact that for $\delta\in\left(0,1\right)$,
\[
\mathbb{P}\left\{ \sum_{i=1}^{N}\left\Vert \kappa\boldsymbol{\theta}_{i}\right\Vert _{2}^{2}\geq N\kappa^{2}\left(1+\delta\right)\right\} \leq C\exp\left(-CN\kappa^{2}\delta^{2}\right).
\]
Furthermore, using Lemma \ref{lem:Unorm_firstSetting} and the union
bound, we obtain for sufficiently large $C_{*}$,
\[
\mathbb{P}\left\{ Q_{\gamma}\geq r_{*}^{2}C_{*}\right\} \leq\left|{\cal N}_{d}^{\infty}\left(\gamma\right)\right|\left|{\cal N}_{d}^{2}\left(\gamma\right)\right|C\exp\left(C\left(d-N\kappa_{*}^{2}/\kappa^{2}\right)\right)\leq\left(\frac{3}{\gamma}\right)^{2d}C\exp\left(C\left(d-N\kappa_{*}^{2}/\kappa^{2}\right)\right).
\]
Let us choose $\gamma=\kappa_{*}^{2}/\left(C\kappa^{2}\right)<1/3$
and $\delta=0.5$. Then for sufficiently large $C_{*}$,
\begin{align*}
\mathbb{P}\left\{ Q\geq r_{*}^{2}C_{*}\right\}  & \leq C\exp\left(-CN\kappa^{2}\right)+\left(\frac{C\kappa^{2}}{\kappa_{*}^{2}}\right)^{d}C\exp\left(Cd-CN\kappa_{*}^{2}/\kappa^{2}\right)\\
 & \leq C\exp\left(-CN\kappa^{2}\right)+C\exp\left(Cd\log\left(\kappa/\kappa_{*}+e\right)-CN\kappa_{*}^{2}/\kappa^{2}\right)\\
 & \leq C\exp\left(Cd\log\left(\kappa/\kappa_{*}+e\right)-CN\kappa_{*}^{2}/\kappa^{2}\right),
\end{align*}
where we recall $\kappa_{*}\leq C$. This completes the proof.
\end{proof}

\subsection{Setting with bounded activation: Proof of Theorem \ref{thm:bdd_act_setting}\label{subsec:Proof_bdd_act_thm}}

We prove Theorem \ref{thm:bdd_act_setting}. Our proof uses several
auxiliary results, which are stated and proven in Section \ref{subsec:Proof_bdd_act_aux}.
\begin{proof}[Proof of Theorem \ref{thm:bdd_act_setting}]
The theorem follows from Propositions \ref{prop:2nd_setting_grow_bound},
\ref{prop:2nd_setting_grow_bound_Wrho}, \ref{prop:2nd_setting_op_bound},
\ref{prop:2nd_setting_ODE_r}, \ref{prop:2nd_setting_ODE}, \ref{prop:2nd_setting_F_bound}
and \ref{prop:2nd_setting_U_bound}. In particular, by Proposition
\ref{prop:2nd_setting_ODE_r}, the process $\left(r_{1,t},r_{2,t},\rho_{r}^{t}\right)_{t\geq0}$
as described exists and is (weakly) unique. By Proposition \ref{prop:2nd_setting_ODE},
we have $\left(\hat{\boldsymbol{\theta}}^{t},\rho^{t}\right)_{t\geq0}$
form the (weakly) unique solution to the ODE (\ref{eq:ODE}) with
initialization $\hat{\boldsymbol{\theta}}^{0}\sim\rho^{0}$ and $\rho^{0}$
respectively, where
\[
\hat{\boldsymbol{\theta}}^{t}=\left(r_{1,t}\hat{\boldsymbol{\theta}}_{\left[1\right]}^{0}/\left\Vert \hat{\boldsymbol{\theta}}_{\left[1\right]}^{0}\right\Vert _{2},\quad r_{2,t}\hat{\boldsymbol{\theta}}_{\left[2\right]}^{0}/\left\Vert \hat{\boldsymbol{\theta}}_{\left[2\right]}^{0}\right\Vert _{2}\right),\qquad\rho^{t}={\rm Law}\left(\hat{\boldsymbol{\theta}}^{t}\right),
\]
$\hat{\boldsymbol{\theta}}_{\left[1\right]}^{0}/\left\Vert \hat{\boldsymbol{\theta}}_{\left[1\right]}^{0}\right\Vert _{2}\stackrel{{\rm d}}{=}\boldsymbol{\omega}_{1}$
and $\hat{\boldsymbol{\theta}}_{\left[2\right]}^{0}/\left\Vert \hat{\boldsymbol{\theta}}_{\left[2\right]}^{0}\right\Vert _{2}\stackrel{{\rm d}}{=}\boldsymbol{\omega}_{2}$
are independent of each other and of $\left(r_{1,t},r_{2,t}\right)_{t\geq0}$.
We also have from Proposition \ref{prop:2nd_setting_ODE_r} that $r_{1,t}$
and $r_{2,t}$ are $C$-sub-Gaussian for any $t\leq T$, and $\left(r_{1,t},r_{2,t}\right)$
is a deterministic functions of their initialization $\left(r_{1,0},r_{2,0}\right)$,
i.e. $\left(r_{1,t},r_{2,t}\right)=\psi_{t}\left(r_{1,0},r_{2,0}\right)$,
such that $\left\Vert \partial_{t}\psi_{t}\left(r_{1},r_{2}\right)\right\Vert _{2}\leq C\left(1+t+r_{1}+r_{2}\right)$.
Using these facts and recalling the definition of the Wasserstein
distance $\mathscr{W}_{2}$ in the statement of Proposition \ref{prop:2nd_setting_grow_bound_Wrho},
we have for any $t_{1},t_{2}\leq T$:
\begin{align*}
\mathscr{W}_{2}\left(\rho_{r}^{t_{1}},\rho_{r}^{t_{2}}\right)^{2} & \leq\mathbb{E}_{r}\left\{ \sum_{j\in\left\{ 1,2\right\} }\left|\left(\psi_{t_{1}}\left(r_{1,0},r_{2,0}\right)\right)_{j}-\left(\psi_{t_{2}}\left(r_{1,0},r_{2,0}\right)\right)_{j}\right|^{2}\right\} \\
 & \leq C\mathbb{E}_{r}\left\{ 1+t_{1}^{2}+t_{2}^{2}+r_{1,0}^{2}+r_{2,0}^{2}\right\} \left|t_{2}-t_{1}\right|^{2}\leq C\left|t_{2}-t_{1}\right|^{2},
\end{align*}
where we let $\mathbb{E}_{r}$ denote the expectation over $\left(r_{1,0},r_{2,0}\right)$.
These verify Assumption \ref{enu:Assump_ODE} and allow Propositions
\ref{prop:2nd_setting_grow_bound}, \ref{prop:2nd_setting_grow_bound_Wrho}
and \ref{prop:2nd_setting_U_bound} to verify Assumptions \ref{enu:Assump_growth}
and \ref{enu:Assump_nabla11U}.

By Proposition \ref{prop:2nd_setting_grow_bound}, for any $\boldsymbol{x}\in\mathbb{R}^{d}$,
\begin{align*}
 & \int\kappa\boldsymbol{\theta}\sigma\left(\left\langle \kappa\boldsymbol{\theta},\boldsymbol{x}\right\rangle \right)\rho^{t}\left({\rm d}\boldsymbol{\theta}\right)\\
 & =\int\left(\bar{r}_{1}q_{1}\left(\left\Vert \boldsymbol{x}_{\left[1\right]}\right\Vert _{2}\bar{r}_{1},\left\Vert \boldsymbol{x}_{\left[2\right]}\right\Vert _{2}\bar{r}_{2}\right)\frac{\boldsymbol{x}_{\left[1\right]}}{\left\Vert \boldsymbol{x}_{\left[1\right]}\right\Vert _{2}},\quad\bar{r}_{2}q_{2}\left(\left\Vert \boldsymbol{x}_{\left[1\right]}\right\Vert _{2}\bar{r}_{1},\left\Vert \boldsymbol{x}_{\left[2\right]}\right\Vert _{2}\bar{r}_{2}\right)\frac{\boldsymbol{x}_{\left[2\right]}}{\left\Vert \boldsymbol{x}_{\left[2\right]}\right\Vert _{2}}\right)\rho_{r}^{t}\left({\rm d}\bar{r}_{1},{\rm d}\bar{r}_{2}\right).
\end{align*}
Note that for $\boldsymbol{x}\sim{\cal P}$, $\left\Vert \boldsymbol{x}_{\left[1\right]}\right\Vert _{2}\stackrel{{\rm d}}{=}\chi_{1}$
and $\left\Vert \boldsymbol{x}_{\left[2\right]}\right\Vert _{2}\stackrel{{\rm d}}{=}\chi_{2}$.
Therefore,
\begin{align*}
{\cal R}\left(\rho^{t}\right) & =\mathbb{E}_{{\cal P}}\left\{ \frac{1}{2}\left\Vert \boldsymbol{x}-\int\kappa\boldsymbol{\theta}\sigma\left(\left\langle \kappa\boldsymbol{\theta},\boldsymbol{x}\right\rangle \right)\rho^{t}\left({\rm d}\boldsymbol{\theta}\right)\right\Vert _{2}^{2}\right\} \\
 & =\mathbb{E}_{\chi}\left\{ \frac{1}{2}\sum_{j\in\left\{ 1,2\right\} }\left(\chi_{j}-\int\bar{r}_{j}q_{j}\left(\chi_{1}\bar{r}_{1},\chi_{2}\bar{r}_{2}\right)\rho_{r}^{t}\left({\rm d}\bar{r}_{1},{\rm d}\bar{r}_{2}\right)\right)^{2}\right\} .
\end{align*}
This concludes the proof.
\end{proof}

\subsection{Setting with bounded activation: Proofs of auxiliary results\label{subsec:Proof_bdd_act_aux}}
\begin{lem}
\label{lem:2nd_setting_unifVec_moments}Consider $\boldsymbol{\omega}\sim\text{Unif}\left(\mathbb{S}^{d-1}\right)$
and let $\omega_{1}$ be its first entry, for $d>16$. Then
\[
\mathbb{E}\left\{ \left(\kappa\omega_{1}\right)^{8}\right\} \leq C,\qquad\mathbb{E}\left\{ \left\langle \kappa\boldsymbol{\omega},\boldsymbol{v}\right\rangle ^{8}\right\} \leq C,
\]
for some constant $C$ independent of $d$ and any $\boldsymbol{v}\in\mathbb{S}^{d-1}$.
\end{lem}

\begin{proof}
We have, for $\left(g_{i}\right)_{i\leq d}\sim_{{\rm i.i.d.}}\mathsf{N}\left(0,1\right)$,
\[
\mathbb{E}\left\{ \left(\sum_{i=1}^{d}g_{i}^{2}\right)^{-8}\right\} =\frac{\Gamma\left(d/2-8\right)}{256\Gamma\left(d/2\right)}\leq\frac{1}{256}\left(\frac{d}{2}-8\right)^{-8}.
\]
Note that $\omega_{1}\stackrel{{\rm d}}{=}g_{1}/\sqrt{\sum_{i=1}^{d}g_{i}^{2}}$.
By Cauchy-Schwarz's inequality, for $d>16$,
\[
\mathbb{E}\left\{ \left(\kappa\omega_{1}\right)^{8}\right\} \leq d^{4}\sqrt{\mathbb{E}\left\{ g_{i}^{16}\right\} \mathbb{E}\left\{ \left(\sum_{i=1}^{d}g_{i}^{2}\right)^{-8}\right\} }\leq\frac{Cd^{4}}{\left(d/2-8\right)^{4}}\leq C,
\]
uniformly in $d$. Next, for any $\boldsymbol{v}\in\mathbb{S}^{d-1}$,
by choosing an orthogonal $Q$ such that $Q\boldsymbol{v}=\left(1,0,...,0\right)^{\top}$,
we get:
\[
\mathbb{E}\left\{ \left\langle \kappa\boldsymbol{\omega},\boldsymbol{v}\right\rangle ^{8}\right\} =\mathbb{E}\left\{ \left\langle \kappa Q\boldsymbol{\omega},Q\boldsymbol{v}\right\rangle ^{8}\right\} =\mathbb{E}\left\{ \left\langle \kappa\boldsymbol{\omega},Q\boldsymbol{v}\right\rangle ^{8}\right\} =\mathbb{E}\left\{ \left(\kappa\omega_{1}\right)^{8}\right\} \leq C,
\]
where we have used the fact $\boldsymbol{\omega}\stackrel{{\rm d}}{=}Q\boldsymbol{\omega}$
for any orthogonal $Q$.
\end{proof}
\begin{lem}
\label{lem:2nd_setting_qBounds}Consider $q_{1}$ and $q_{2}$ as
defined in (\ref{eq:2nd_setting_q1}) and (\ref{eq:2nd_setting_q2}).
The following quantities
\[
\left|q_{1}\left(a,b\right)\right|,\;\left|q_{2}\left(a,b\right)\right|,\;\left|\frac{1}{a}q_{1}\left(a,b\right)\right|,\;\left|\frac{1}{b}q_{2}\left(a,b\right)\right|,\;\left|\partial_{1}q_{1}\left(a,b\right)\right|,\;\left|\partial_{2}q_{2}\left(a,b\right)\right|,
\]
\[
\left|\partial_{2}q_{1}\left(a,b\right)\right|,\;\left|\partial_{1}q_{2}\left(a,b\right)\right|,\;\left|b\partial_{2}q_{1}\left(a,b\right)\right|,\;\left|a\partial_{1}q_{2}\left(a,b\right)\right|,\;\left|a\partial_{2}q_{1}\left(a,b\right)\right|,\;\left|b\partial_{1}q_{2}\left(a,b\right)\right|,
\]
\[
\left|a\partial_{1}q_{1}\left(a,b\right)\right|,\;\left|b\partial_{2}q_{2}\left(a,b\right)\right|,\;\left|\frac{a}{b}\partial_{2}q_{1}\left(a,b\right)\right|,\;\left|\frac{b}{a}\partial_{1}q_{2}\left(a,b\right)\right|,\;\left|\partial_{11}^{2}q_{1}\left(a,b\right)\right|,\;\left|\partial_{22}^{2}q_{2}\left(a,b\right)\right|,
\]
\[
\left|a\partial_{11}^{2}q_{1}\left(a,b\right)\right|,\;\left|b\partial_{22}^{2}q_{2}\left(a,b\right)\right|,\;\left|a\partial_{22}^{2}q_{1}\left(a,b\right)\right|,\;\left|b\partial_{11}^{2}q_{2}\left(a,b\right)\right|,\;\left|a\partial_{12}^{2}q_{1}\left(a,b\right)\right|,\;\left|b\partial_{12}^{2}q_{2}\left(a,b\right)\right|,
\]
are all bounded by some constant $C$ independent of $\mathfrak{Dim}$,
for any $a,b\geq0$, given that $d_{1},d_{2}>16$. (Here $\left|\left(1/a\right)\cdot f\left(a,b\right)\right|\leq C$
should be interpreted as that $\left|f\left(a,b\right)\right|\leq Ca$,
which holds for any $a\geq0$.)
\end{lem}

\begin{proof}
By Lemma \ref{lem:2nd_setting_unifVec_moments}, $\mathbb{E}\left\{ \left(\kappa\omega_{11}\right)^{8}\right\} ,\;\mathbb{E}\left\{ \left(\kappa\omega_{21}\right)^{8}\right\} \leq C$.
We shall repeatedly use this fact, along with $\left\Vert \sigma\right\Vert _{\infty},\left\Vert \sigma'\right\Vert _{\infty},\left\Vert \sigma''\right\Vert _{\infty}\leq C$,
without stating explicitly. We have $\left|q_{1}\left(a,b\right)\right|\leq\mathbb{E}_{\boldsymbol{\omega}}\left\{ \left|\kappa\omega_{11}\right|\right\} \leq C$.
One can perform similar arguments to deduce the bounds for $q_{2}\left(a,b\right)$,
$\partial_{1}q_{1}\left(a,b\right),$ $\partial_{2}q_{2}\left(a,b\right)$,
$\partial_{2}q_{1}\left(a,b\right)$, $\partial_{1}q_{2}\left(a,b\right)$,
$\partial_{11}^{2}q_{1}\left(a,b\right)$, $\partial_{22}^{2}q_{2}\left(a,b\right)$.

We consider $b\partial_{1}q_{2}\left(a,b\right)$. Let $f\left(\omega\right)$
be the probability density of $\omega_{21}$:
\[
f\left(\omega\right)=\frac{1}{Z}\left(1-\omega^{2}\right)^{\left(d_{2}-3\right)/2}\mathbb{I}\left(\left|\omega\right|\leq1\right),
\]
where $Z$ is a normalization factor. We state a few simple properties:$f$
is continuous and supported on $\left[-1,1\right]$ and differentiable
on $\left(-1,1\right)$, $f\left(1\right)=f\left(-1\right)=0$, $f$
is an even function, and $f$ is non-increasing on $\left[0,1\right]$.
Then by integration by parts,
\[
\int_{-1}^{1}\left|\omega f'\left(\omega\right)\right|{\rm d}\omega=-2\int_{0}^{1}\omega f'\left(\omega\right){\rm d}\omega=2\int_{0}^{1}f\left(\omega\right){\rm d}\omega=\int_{-1}^{1}f\left(\omega\right){\rm d}\omega=1.
\]
We also have, by integration by parts,
\[
\mathbb{E}_{\omega_{21}}\left\{ \kappa b\omega_{21}\sigma'\left(\kappa a\omega_{11}+\kappa b\omega_{21}\right)\right\} =-\int_{-1}^{1}\left(f\left(\omega\right)+\omega f'\left(\omega\right)\right)\sigma\left(\kappa a\omega_{11}+\kappa b\omega\right){\rm d}\omega.
\]
Therefore,
\begin{align*}
\left|b\partial_{1}q_{2}\left(a,b\right)\right| & =\left|\mathbb{E}_{\boldsymbol{\omega}}\left\{ \kappa^{2}b\omega_{21}\omega_{11}\sigma'\left(\kappa a\omega_{11}+\kappa b\omega_{21}\right)\right\} \right|\\
 & =\left|\mathbb{E}_{\boldsymbol{\omega}}\left\{ \kappa\omega_{11}\int_{-1}^{1}\left(f\left(\omega\right)+\omega f'\left(\omega\right)\right)\sigma\left(\kappa a\omega_{11}+\kappa b\omega\right){\rm d}\omega\right\} \right|\\
 & \leq\mathbb{E}_{\boldsymbol{\omega}}\left\{ \left|\kappa\omega_{11}\right|\right\} \int_{-1}^{1}\left(f\left(\omega\right)+\left|\omega f'\left(\omega\right)\right|\right){\rm d}\omega\leq C.
\end{align*}
A similar argument applies to $a\partial_{2}q_{1}\left(a,b\right)$,
$b\partial_{2}q_{1}\left(a,b\right)$, $a\partial_{1}q_{2}\left(a,b\right)$,
$a\partial_{22}^{2}q_{1}\left(a,b\right)$, $b\partial_{11}^{2}q_{2}\left(a,b\right)$.

Next we consider $\left(1/a\right)\cdot q_{1}\left(a,b\right)$:
\begin{align*}
\left|\frac{1}{a}q_{1}\left(a,b\right)\right| & =\left|\mathbb{E}_{\boldsymbol{\omega}}\left\{ \frac{1}{a}\kappa\omega_{11}\sigma\left(\kappa a\omega_{11}+\kappa b\omega_{21}\right)\right\} \right|\\
 & \stackrel{\left(a\right)}{=}\frac{1}{2}\left|\mathbb{E}_{\boldsymbol{\omega}}\left\{ \frac{1}{a}\kappa\omega_{11}\left(\sigma\left(\kappa a\omega_{11}+\kappa b\omega_{21}\right)-\sigma\left(-\kappa a\omega_{11}+\kappa b\omega_{21}\right)\right)\right\} \right|\\
 & \stackrel{\left(b\right)}{=}\left|\mathbb{E}_{\boldsymbol{\omega}}\left\{ \kappa^{2}\omega_{11}^{2}\sigma'\left(\kappa a\zeta+\kappa b\omega_{21}\right)\right\} \right|\leq C,
\end{align*}
where we have used the fact that $\omega_{11}\stackrel{{\rm d}}{=}-\omega_{11}$
independent of $\omega_{21}$ in step $\left(a\right)$ and the mean
value theorem, for some $\zeta$ that lies between $-\omega_{11}$
and $\omega_{11}$, in step $\left(b\right)$. The same argument applies
to $\left(1/b\right)\cdot q_{2}\left(a,b\right)$.

We consider $\left(b/a\right)\cdot\partial_{1}q_{2}\left(a,b\right)$,
whose treatment is a combination of previously used arguments. In
particular,
\begin{align*}
\left|\frac{b}{a}\partial_{1}q_{2}\left(a,b\right)\right| & =\left|\mathbb{E}_{\boldsymbol{\omega}}\left\{ \frac{b}{a}\kappa^{2}\omega_{11}\omega_{21}\sigma'\left(\kappa a\omega_{11}+\kappa b\omega_{21}\right)\right\} \right|\\
 & \stackrel{\left(a\right)}{=}\left|\mathbb{E}_{\boldsymbol{\omega}}\left\{ \frac{1}{a}\kappa\omega_{11}\int_{-1}^{1}\left(f\left(\omega\right)+\omega f'\left(\omega\right)\right)\sigma\left(\kappa a\omega_{11}+\kappa b\omega\right){\rm d}\omega\right\} \right|\\
 & \stackrel{\left(b\right)}{=}\frac{1}{2}\left|\mathbb{E}_{\boldsymbol{\omega}}\left\{ \frac{1}{a}\kappa\omega_{11}\int_{-1}^{1}\left(f\left(\omega\right)+\omega f'\left(\omega\right)\right)\left(\sigma\left(\kappa a\omega_{11}+\kappa b\omega\right)-\sigma\left(-\kappa a\omega_{11}+\kappa b\omega\right)\right){\rm d}\omega\right\} \right|\\
 & \stackrel{\left(c\right)}{=}\left|\mathbb{E}_{\boldsymbol{\omega}}\left\{ \kappa^{2}\omega_{11}^{2}\int_{-1}^{1}\left(f\left(\omega\right)+\omega f'\left(\omega\right)\right)\sigma'\left(\kappa a\zeta+\kappa b\omega\right){\rm d}\omega\right\} \right|\\
 & \stackrel{\left(d\right)}{\leq}C,
\end{align*}
where we use the integration-by-parts formula in step $\left(a\right)$,
the fact that $\omega_{11}\stackrel{{\rm d}}{=}-\omega_{11}$ independent
of $\omega_{21}$ in step $\left(b\right)$, the mean value theorem
in step $\left(c\right)$, and the same argument as in the bounding
of $\left|b\partial_{1}q_{2}\left(a,b\right)\right|$ in step $\left(d\right)$.
The same argument applies to $\left(a/b\right)\cdot\partial_{2}q_{1}\left(a,b\right)$.

Finally we consider $b\partial_{2}q_{2}\left(a,b\right)$. We have:
\begin{align*}
\left|b\partial_{2}q_{2}\left(a,b\right)\right| & =\left|\mathbb{E}_{\boldsymbol{\omega}}\left\{ \kappa^{2}b\omega_{21}^{2}\sigma'\left(\kappa a\omega_{11}+\kappa b\omega_{21}\right)\right\} \right|\\
 & \stackrel{\left(a\right)}{=}\left|-2q_{2}\left(a,b\right)+\mathbb{E}_{\boldsymbol{\omega}}\left\{ \int_{-1}^{1}\kappa\omega^{2}f'\left(\omega\right)\sigma\left(\kappa a\omega_{11}+\kappa b\omega\right)\right\} {\rm d}\omega\right|\\
 & \stackrel{\left(b\right)}{\leq}C+\left|\mathbb{E}_{\boldsymbol{\omega}}\left\{ \kappa^{3}\frac{\omega_{21}^{3}}{1-\omega_{21}^{2}}\sigma\left(\kappa a\omega_{11}+\kappa b\omega_{21}\right)\right\} \right|\\
 & \leq C+C\sqrt{\mathbb{E}_{\boldsymbol{\omega}}\left\{ \kappa^{6}\omega_{21}^{6}\right\} \mathbb{E}_{\boldsymbol{\omega}}\left\{ \left(1-\omega_{21}^{2}\right)^{-2}\right\} }\\
 & \stackrel{\left(c\right)}{\leq}C,
\end{align*}
where in step $\left(a\right)$, we apply integration by parts; in
step $\left(b\right)$, we use $f'\left(\omega\right)/f\left(\omega\right)=\left(d_{2}-3\right)\omega/\left[2\left(1-\omega^{2}\right)\right]$
for $\left|\omega\right|<1$ and that $\kappa=\sqrt{d}$; in step
$\left(c\right)$, we use the bound:
\begin{align*}
\mathbb{E}_{\boldsymbol{\omega}}\left\{ \left(1-\omega_{21}^{2}\right)^{-2}\right\}  & =\mathbb{E}_{\boldsymbol{g}}\left\{ \left(\sum_{i=1}^{d_{2}}g_{i}^{2}\right)^{2}\left(\sum_{i=2}^{d_{2}}g_{i}^{2}\right)^{-2}\right\} \leq\sqrt{\mathbb{E}_{\boldsymbol{g}}\left\{ \left(\sum_{i=1}^{d_{2}}g_{i}^{2}\right)^{4}\right\} \mathbb{E}_{\boldsymbol{g}}\left\{ \left(\sum_{i=2}^{d_{2}}g_{i}^{2}\right)^{-4}\right\} }\\
 & =\sqrt{\frac{\Gamma\left(d_{2}/2+4\right)}{\Gamma\left(d_{2}/2\right)}\times\frac{\Gamma\left(\left(d_{2}-1\right)/2-4\right)}{\Gamma\left(\left(d_{2}-1\right)/2\right)}}\leq C,
\end{align*}
for $\left(g_{i}\right)_{i\leq d_{2}}\sim_{{\rm i.i.d.}}\mathsf{N}\left(0,1\right)$
and $d_{2}>9$. Similar arguments apply to $a\partial_{1}q_{1}\left(a,b\right)$,
$a\partial_{11}^{2}q_{1}\left(a,b\right)$, $b\partial_{22}^{2}q_{2}\left(a,b\right)$,
$a\partial_{12}^{2}q_{1}\left(a,b\right)$ and $b\partial_{12}^{2}q_{2}\left(a,b\right)$
\end{proof}
\begin{prop}
\label{prop:2nd_setting_grow_bound}Consider setting \ref{enu:bdd_act_setting},
and $\rho={\rm Law}\left(r_{1}\boldsymbol{\omega}_{1},r_{2}\boldsymbol{\omega}_{2}\right)$
in which $\left(r_{1},r_{2}\right)$, $\boldsymbol{\omega}_{1}$ and
$\boldsymbol{\omega}_{2}$ are mutually independent, $\left(r_{1},r_{2}\right)\sim\rho_{r}$,
$r_{1},r_{2}\geq0$ and $\int\left(r_{1}+r_{2}\right){\rm d}\rho_{r}\leq C$.
Then:
\begin{itemize}
\item The following growth bounds hold:
\begin{align*}
\left\Vert \nabla V\left(\boldsymbol{\theta}\right)\right\Vert _{2} & \leq C\left\Vert \boldsymbol{\theta}\right\Vert _{2},\\
\left\Vert \nabla V\left(\boldsymbol{\theta}_{1}\right)-\nabla V\left(\boldsymbol{\theta}_{2}\right)\right\Vert _{2} & \leq C\left\Vert \boldsymbol{\theta}_{1}-\boldsymbol{\theta}_{2}\right\Vert _{2},\\
\left\Vert \nabla_{1}W\left(\boldsymbol{\theta};\rho\right)\right\Vert _{2} & \leq C,\\
\left\Vert \nabla_{1}W\left(\boldsymbol{\theta}_{1};\rho\right)-\nabla_{1}W\left(\boldsymbol{\theta}_{2};\rho\right)\right\Vert _{2} & \leq C\left\Vert \boldsymbol{\theta}_{1}-\boldsymbol{\theta}_{2}\right\Vert _{2},\\
\left\Vert \nabla_{1}U\left(\boldsymbol{\theta},\boldsymbol{\theta}'\right)\right\Vert _{2} & \leq C\kappa^{2}\left(1+\left\Vert \boldsymbol{\theta}\right\Vert _{2}\right)\left\Vert \boldsymbol{\theta}'\right\Vert _{2}.
\end{align*}
Furthermore, $\left|V\left(\boldsymbol{0}\right)\right|=\left|U\left(\boldsymbol{0},\boldsymbol{0}\right)\right|=\left|W\left(\boldsymbol{0};\rho'\right)\right|=0$
for any $\rho'$.
\item We also have:
\begin{align*}
 & \hat{\boldsymbol{x}}\left(\boldsymbol{x}\right)\equiv\int\kappa\boldsymbol{\theta}\sigma\left(\left\langle \kappa\boldsymbol{\theta},\boldsymbol{x}\right\rangle \right)\rho\left({\rm d}\boldsymbol{\theta}\right)\\
 & =\int\left(r_{1}q_{1}\left(\left\Vert \boldsymbol{x}_{\left[1\right]}\right\Vert _{2}r_{1},\left\Vert \boldsymbol{x}_{\left[2\right]}\right\Vert _{2}r_{2}\right)\frac{\boldsymbol{x}_{\left[1\right]}}{\left\Vert \boldsymbol{x}_{\left[1\right]}\right\Vert _{2}},\quad r_{2}q_{2}\left(\left\Vert \boldsymbol{x}_{\left[1\right]}\right\Vert _{2}r_{1},\left\Vert \boldsymbol{x}_{\left[2\right]}\right\Vert _{2}r_{2}\right)\frac{\boldsymbol{x}_{\left[2\right]}}{\left\Vert \boldsymbol{x}_{\left[2\right]}\right\Vert _{2}}\right)\rho_{r}\left({\rm d}r_{1},{\rm d}r_{2}\right),
\end{align*}
for any $\boldsymbol{x}=\left(\boldsymbol{x}_{\left[1\right]},\boldsymbol{x}_{\left[2\right]}\right)$,
and $q_{1}$ and $q_{2}$ are as defined in (\ref{eq:2nd_setting_q1})
and (\ref{eq:2nd_setting_q2}). Furthermore, for any $\boldsymbol{v}\in\mathbb{S}^{d-1}$,
$\mathbb{E}_{{\cal P}}\left\{ \left|\kappa\left\langle \hat{\boldsymbol{x}}\left(\boldsymbol{x}\right),\boldsymbol{v}\right\rangle \right|^{8}\right\} \leq C$.
\end{itemize}
\end{prop}

\begin{proof}
The proof comprises of several parts.

\paragraph*{Bounds for $V$.}

We have:
\[
V\left(\boldsymbol{\theta}\right)=\mathbb{E}_{{\cal P}}\left\{ -\left\langle \kappa\boldsymbol{\theta},\boldsymbol{x}\right\rangle \sigma\left(\left\langle \kappa\boldsymbol{\theta},\boldsymbol{x}\right\rangle \right)\right\} +\lambda\left\Vert \boldsymbol{\theta}\right\Vert _{2}^{2}=-\mathbb{E}_{g}\left\{ \left\Vert \boldsymbol{\Sigma}\boldsymbol{\theta}\right\Vert _{2}g\sigma\left(\left\Vert \boldsymbol{\Sigma}\boldsymbol{\theta}\right\Vert _{2}g\right)\right\} +\lambda\left\Vert \boldsymbol{\theta}\right\Vert _{2}^{2}.
\]
We calculate $\nabla V\left(\boldsymbol{\theta}\right)$ and $\nabla^{2}V\left(\boldsymbol{\theta}\right)$:
\begin{align*}
\nabla V\left(\boldsymbol{\theta}\right) & =-\boldsymbol{\Sigma}^{2}\boldsymbol{\theta}\mathbb{E}_{g}\left\{ \sigma'\left(\left\Vert \boldsymbol{\Sigma}\boldsymbol{\theta}\right\Vert _{2}g\right)+g^{2}\sigma'\left(\left\Vert \boldsymbol{\Sigma}\boldsymbol{\theta}\right\Vert _{2}g\right)\right\} +2\lambda\boldsymbol{\theta},\\
\nabla^{2}V\left(\boldsymbol{\theta}\right) & =-\boldsymbol{\Sigma}^{2}\mathbb{E}_{g}\left\{ \left(1+g^{2}\right)\sigma'\left(\left\Vert \boldsymbol{\Sigma}\boldsymbol{\theta}\right\Vert _{2}g\right)\right\} +2\lambda\boldsymbol{I}_{d}\\
 & \qquad-\frac{\boldsymbol{\Sigma}^{2}\boldsymbol{\theta}\boldsymbol{\theta}^{\top}\boldsymbol{\Sigma}^{2}}{\left\Vert \boldsymbol{\Sigma}\boldsymbol{\theta}\right\Vert _{2}^{2}}\mathbb{E}_{g}\left\{ \left(1+g^{2}\right)\left\Vert \boldsymbol{\Sigma}\boldsymbol{\theta}\right\Vert _{2}g\sigma''\left(\left\Vert \boldsymbol{\Sigma}\boldsymbol{\theta}\right\Vert _{2}g\right)\right\} .
\end{align*}
Since $\left\Vert \sigma'\right\Vert _{\infty}\leq C$ and $\left\Vert \boldsymbol{\Sigma}\right\Vert _{{\rm op}}\leq C$,
it is easy to see that $\left\Vert \nabla V\left(\boldsymbol{\theta}\right)\right\Vert _{2}\leq C\left\Vert \boldsymbol{\theta}\right\Vert _{2}$.
We also have from Stein's lemma:
\begin{align*}
 & \mathbb{E}_{g}\left\{ g\left(2g-g^{3}\right)\sigma'\left(\left\Vert \boldsymbol{\Sigma}\boldsymbol{\theta}\right\Vert _{2}g\right)\right\} \\
 & \qquad=\mathbb{E}_{g}\left\{ \left(2-3g^{2}\right)\sigma'\left(\left\Vert \boldsymbol{\Sigma}\boldsymbol{\theta}\right\Vert _{2}g\right)+\left\Vert \boldsymbol{\Sigma}\boldsymbol{\theta}\right\Vert _{2}\left(2g-g^{3}\right)\sigma''\left(\left\Vert \boldsymbol{\Sigma}\boldsymbol{\theta}\right\Vert _{2}g\right)\right\} \\
 & \qquad=\mathbb{E}_{g}\left\{ -\sigma'\left(\left\Vert \boldsymbol{\Sigma}\boldsymbol{\theta}\right\Vert _{2}g\right)-3\left\Vert \boldsymbol{\Sigma}\boldsymbol{\theta}\right\Vert _{2}g\sigma''\left(\left\Vert \boldsymbol{\Sigma}\boldsymbol{\theta}\right\Vert _{2}g\right)+\left\Vert \boldsymbol{\Sigma}\boldsymbol{\theta}\right\Vert _{2}\left(2g-g^{3}\right)\sigma''\left(\left\Vert \boldsymbol{\Sigma}\boldsymbol{\theta}\right\Vert _{2}g\right)\right\} \\
 & \qquad=\mathbb{E}_{g}\left\{ -\sigma'\left(\left\Vert \boldsymbol{\Sigma}\boldsymbol{\theta}\right\Vert _{2}g\right)-\left\Vert \boldsymbol{\Sigma}\boldsymbol{\theta}\right\Vert _{2}g\left(1+g^{2}\right)\sigma''\left(\left\Vert \boldsymbol{\Sigma}\boldsymbol{\theta}\right\Vert _{2}g\right)\right\} ,
\end{align*}
and thus, using the fact $\left\Vert \sigma'\right\Vert _{\infty}\leq C$:
\[
\left|\mathbb{E}_{g}\left\{ \left\Vert \boldsymbol{\Sigma}\boldsymbol{\theta}\right\Vert _{2}g\left(1+g^{2}\right)\sigma''\left(\left\Vert \boldsymbol{\Sigma}\boldsymbol{\theta}\right\Vert _{2}g\right)\right\} \right|=\left|\mathbb{E}_{g}\left\{ \left(g\left(2-g^{3}\right)+1\right)\sigma'\left(\left\Vert \boldsymbol{\Sigma}\boldsymbol{\theta}\right\Vert _{2}g\right)\right\} \right|\leq C.
\]
It is then easy to see that $\left\Vert \nabla^{2}V\left(\boldsymbol{\theta}\right)\right\Vert _{{\rm op}}\leq C$,
since $\left\Vert \boldsymbol{\Sigma}\right\Vert _{{\rm op}}\leq C$
and $\left\Vert \boldsymbol{\Sigma}\boldsymbol{\theta}\right\Vert _{2}\geq C\left\Vert \boldsymbol{\theta}\right\Vert _{2}$.
This in particular implies
\[
\left\Vert \nabla V\left(\boldsymbol{\theta}_{1}\right)-\nabla V\left(\boldsymbol{\theta}_{2}\right)\right\Vert _{2}\leq C\left\Vert \boldsymbol{\theta}_{1}-\boldsymbol{\theta}_{2}\right\Vert _{2},
\]
as desired.

\paragraph*{Bounds for $W$.}

Let us define $\chi_{1}\stackrel{{\rm d}}{=}\Sigma_{1}\sqrt{\alpha/d_{1}}Z_{1}$
and $\chi_{2}\stackrel{{\rm d}}{=}\Sigma_{2}\sqrt{\left(1-\alpha\right)/d_{2}}Z_{2}$
two independent random variables, which are independent of $\boldsymbol{\omega}_{1}$
and $\boldsymbol{\omega}_{2}$, where $Z_{1}$ and $Z_{2}$ are respectively
$\chi$-random variables of degrees of freedom $d_{1}$ and $d_{2}$.
For ease of presentation, let us introduce several notations, for
$j,i,k\in\left\{ 1,2\right\} $:
\begin{align*}
q_{j}^{r} & =q_{j}\left(r_{1}\chi_{1},r_{2}\chi_{2}\right), & q_{j}^{\theta} & =q_{j}\left(\left\Vert \boldsymbol{\theta}_{\left[1\right]}\right\Vert _{2}\chi_{1},\left\Vert \boldsymbol{\theta}_{\left[2\right]}\right\Vert _{2}\chi_{2}\right),\\
\partial_{i}q_{j}^{\theta} & =\partial_{i}q_{j}\left(\left\Vert \boldsymbol{\theta}_{\left[1\right]}\right\Vert _{2}\chi_{1},\left\Vert \boldsymbol{\theta}_{\left[2\right]}\right\Vert _{2}\chi_{2}\right), & \partial_{ik}^{2}q_{j}^{\theta} & =\partial_{ik}^{2}q_{j}\left(\left\Vert \boldsymbol{\theta}_{\left[1\right]}\right\Vert _{2}\chi_{1},\left\Vert \boldsymbol{\theta}_{\left[2\right]}\right\Vert _{2}\chi_{2}\right).
\end{align*}
The meaning of each particular quantity shall be clear in the context
it is used.

We first do a useful calculation. For a fixed vector $\boldsymbol{v}\in\mathbb{R}^{d_{1}}$
and any $a,b\in\mathbb{R}$, $a\geq0$, we have:
\begin{align}
\mathbb{E}_{\boldsymbol{\omega}}\left\{ \boldsymbol{\omega}_{1}\sigma\left(a\left\langle \boldsymbol{v},\boldsymbol{\omega}_{1}\right\rangle +b\right)\right\}  & =\mathbb{E}_{\boldsymbol{\omega}}\left\{ \left(\frac{\left\langle \boldsymbol{v},\boldsymbol{\omega}_{1}\right\rangle }{\left\Vert \boldsymbol{v}\right\Vert _{2}^{2}}\boldsymbol{v}+{\rm Proj}_{\boldsymbol{v}}^{\perp}\boldsymbol{\omega}_{1}\right)\sigma\left(a\left\langle \boldsymbol{v},\boldsymbol{\omega}_{1}\right\rangle +b\right)\right\} \nonumber \\
 & \stackrel{\left(a\right)}{=}\frac{\boldsymbol{v}}{\left\Vert \boldsymbol{v}\right\Vert _{2}^{2}}\mathbb{E}_{\boldsymbol{\omega}}\left\{ \left\langle \boldsymbol{v},\boldsymbol{\omega}_{1}\right\rangle \sigma\left(a\left\langle \boldsymbol{v},\boldsymbol{\omega}_{1}\right\rangle +b\right)\right\} \nonumber \\
 & \stackrel{\left(b\right)}{=}\frac{\boldsymbol{v}}{\left\Vert \boldsymbol{v}\right\Vert _{2}}\mathbb{E}\left\{ \omega_{11}\sigma\left(a\left\Vert \boldsymbol{v}\right\Vert _{2}\omega_{11}+b\right)\right\} ,\label{eq:2nd_setting_grow_bound_SteinEq}
\end{align}
where step $\left(a\right)$ is because conditioning on $\left\langle \boldsymbol{v},\boldsymbol{\omega}_{1}\right\rangle $,
we have ${\rm Proj}_{\boldsymbol{v}}^{\perp}\boldsymbol{\omega}_{1}\stackrel{{\rm d}}{=}-{\rm Proj}_{\boldsymbol{v}}^{\perp}\boldsymbol{\omega}_{1}$;
step $\left(b\right)$ follows from that $\boldsymbol{\omega}_{1}\stackrel{{\rm d}}{=}\boldsymbol{Q}\boldsymbol{\omega}_{1}$
for any orthogonal matrix $\boldsymbol{Q}$, and we choose $\boldsymbol{Q}$
such that $\boldsymbol{Q}^{\top}\boldsymbol{v}=\left(\left\Vert \boldsymbol{v}\right\Vert _{2},0,...,0\right)^{\top}$.
Using this calculation, we have for any $\boldsymbol{x}\in\mathbb{R}^{d}$:
\begin{align*}
 & \int\kappa\bar{\boldsymbol{\theta}}\sigma\left(\left\langle \kappa\bar{\boldsymbol{\theta}},\boldsymbol{x}\right\rangle \right)\rho\left({\rm d}\bar{\boldsymbol{\theta}}\right)\\
 & =\int\left(r_{1}q_{1}\left(\left\Vert \boldsymbol{x}_{\left[1\right]}\right\Vert _{2}r_{1},\left\Vert \boldsymbol{x}_{\left[2\right]}\right\Vert _{2}r_{2}\right)\frac{\boldsymbol{x}_{\left[1\right]}}{\left\Vert \boldsymbol{x}_{\left[1\right]}\right\Vert _{2}},\quad r_{2}q_{2}\left(\left\Vert \boldsymbol{x}_{\left[1\right]}\right\Vert _{2}r_{1},\left\Vert \boldsymbol{x}_{\left[2\right]}\right\Vert _{2}r_{2}\right)\frac{\boldsymbol{x}_{\left[2\right]}}{\left\Vert \boldsymbol{x}_{\left[2\right]}\right\Vert _{2}}\right)\rho_{r}\left({\rm d}r_{1},{\rm d}r_{2}\right).
\end{align*}
We then obtain, again by Eq. (\ref{eq:2nd_setting_grow_bound_SteinEq}),
for $\boldsymbol{\theta}\in\mathbb{R}^{d}$,
\begin{align*}
W\left(\boldsymbol{\theta};\rho\right) & =\mathbb{E}_{{\cal P}}\left\{ \left\langle \kappa\boldsymbol{\theta}\sigma\left(\left\langle \kappa\boldsymbol{\theta},\boldsymbol{x}\right\rangle \right),\int\kappa\boldsymbol{\theta}'\sigma\left(\left\langle \kappa\boldsymbol{\theta}',\boldsymbol{x}\right\rangle \right)\rho\left({\rm d}\boldsymbol{\theta}'\right)\right\rangle \right\} \\
 & =\sum_{j\in\left\{ 1,2\right\} }\int\mathbb{E}_{{\cal P}}\left\{ r_{j}q_{j}\left(\left\Vert \boldsymbol{x}_{\left[1\right]}\right\Vert _{2}r_{1},\left\Vert \boldsymbol{x}_{\left[2\right]}\right\Vert _{2}r_{2}\right)\frac{\left\langle \kappa\boldsymbol{\theta}_{\left[j\right]},\boldsymbol{x}_{\left[j\right]}\right\rangle }{\left\Vert \boldsymbol{x}_{\left[j\right]}\right\Vert _{2}}\sigma\left(\left\langle \kappa\boldsymbol{\theta},\boldsymbol{x}\right\rangle \right)\right\} \rho_{r}\left({\rm d}r_{1},{\rm d}r_{2}\right)\\
 & =\sum_{j\in\left\{ 1,2\right\} }\int\mathbb{E}_{\chi,\boldsymbol{\omega}}\left\{ r_{j}q_{j}^{r}\left\langle \kappa\boldsymbol{\theta}_{\left[j\right]},\boldsymbol{\omega}_{j}\right\rangle \sigma\left(\chi_{j}\left\langle \kappa\boldsymbol{\theta}_{\left[j\right]},\boldsymbol{\omega}_{j}\right\rangle +\chi_{\neg j}\left\langle \kappa\boldsymbol{\theta}_{\left[\neg j\right]},\boldsymbol{\omega}_{\neg j}\right\rangle \right)\right\} \rho_{r}\left({\rm d}r_{1},{\rm d}r_{2}\right)\\
 & =\sum_{j\in\left\{ 1,2\right\} }\int r_{j}\left\Vert \boldsymbol{\theta}_{\left[j\right]}\right\Vert _{2}\mathbb{E}_{\chi}\left\{ q_{j}^{r}q_{j}^{\theta}\right\} \rho_{r}\left({\rm d}r_{1},{\rm d}r_{2}\right),
\end{align*}
where we assume the convention $\neg j=2$ if $j=1$ and $\neg j=1$
if $j=2$. We calculate $\nabla_{1}W\left(\boldsymbol{\theta};\rho\right)$:
\begin{align*}
\nabla_{1}W\left(\boldsymbol{\theta};\rho\right) & =\left(\nabla_{1}W\left(\boldsymbol{\theta};\rho\right)_{\left[1\right]},\quad\nabla_{1}W\left(\boldsymbol{\theta};\rho\right)_{\left[2\right]}\right),\\
\nabla_{1}W\left(\boldsymbol{\theta};\rho\right)_{\left[j\right]} & =\frac{\boldsymbol{\theta}_{\left[j\right]}}{\left\Vert \boldsymbol{\theta}_{\left[j\right]}\right\Vert _{2}}\int r_{j}\mathbb{E}_{\chi}\left\{ q_{j}^{r}q_{j}^{\theta}\right\} \rho_{r}\left({\rm d}r_{1},{\rm d}r_{2}\right)+\boldsymbol{\theta}_{\left[j\right]}\int r_{j}\mathbb{E}_{\chi}\left\{ \chi_{j}q_{j}^{r}\partial_{j}q_{j}^{\theta}\right\} \rho_{r}\left({\rm d}r_{1},{\rm d}r_{2}\right)\\
 & \quad+\frac{\boldsymbol{\theta}_{\left[j\right]}}{\left\Vert \boldsymbol{\theta}_{\left[j\right]}\right\Vert _{2}}\int r_{\neg j}\left\Vert \boldsymbol{\theta}_{\left[\neg j\right]}\right\Vert _{2}\mathbb{E}_{\chi}\left\{ \chi_{j}q_{\neg j}^{r}\partial_{j}q_{\neg j}^{\theta}\right\} \rho_{r}\left({\rm d}r_{1},{\rm d}r_{2}\right),\qquad j=1,2.
\end{align*}
Note that $\mathbb{E}_{\chi}\left\{ \left|\chi_{j}\right|\right\} \leq\sqrt{\mathbb{E}_{\chi}\left\{ \chi_{j}^{2}\right\} }=\sqrt{\Sigma_{j}^{2}d_{j}/d}\leq C$
and 
\[
\mathbb{E}_{\chi}\left\{ \left|\frac{\chi_{j}}{\chi_{\neg j}}\right|\right\} \leq\sqrt{\mathbb{E}_{\chi}\left\{ \chi_{j}^{2}\right\} \mathbb{E}_{\chi}\left\{ \chi_{\neg j}^{-2}\right\} }\leq\sqrt{\Sigma_{j}^{2}d_{j}/d}\sqrt{\Sigma_{\neg j}^{-2}d/\left(d_{\neg j}-2\right)}\leq C.
\]
Then by Lemma \ref{lem:2nd_setting_qBounds}, along with the fact
$\int\left(r_{1}+r_{2}\right){\rm d}\rho_{r}\leq C$, we have:
\[
\left\Vert \nabla_{1}W\left(\boldsymbol{\theta};\rho\right)_{\left[j\right]}\right\Vert _{2}\leq C\int r_{j}\rho_{r}\left({\rm d}r_{1},{\rm d}r_{2}\right)+C\mathbb{E}_{\chi}\left\{ \left|\frac{\chi_{j}}{\chi_{\neg j}}\right|\right\} \int r_{\neg j}\rho_{r}\left({\rm d}r_{1},{\rm d}r_{2}\right)\leq C,
\]
which implies $\left\Vert \nabla_{1}W\left(\boldsymbol{\theta};\rho\right)\right\Vert _{2}\leq C$
as desired. Next we calculate $\nabla_{11}^{2}W\left(\boldsymbol{\theta};\rho\right)$:
\begin{align*}
\nabla_{11}^{2}W\left(\boldsymbol{\theta};\rho\right) & =\left(\begin{array}{cc}
\left[\nabla_{11}^{2}W\left(\boldsymbol{\theta};\rho\right)\right]_{11} & \left[\nabla_{11}^{2}W\left(\boldsymbol{\theta};\rho\right)\right]_{12}\\
\left[\nabla_{11}^{2}W\left(\boldsymbol{\theta};\rho\right)\right]_{12}^{\top} & \left[\nabla_{11}^{2}W\left(\boldsymbol{\theta};\rho\right)\right]_{22}
\end{array}\right),\\
\left[\nabla_{11}^{2}W\left(\boldsymbol{\theta};\rho\right)\right]_{jj} & =\left(\frac{\boldsymbol{I}}{\left\Vert \boldsymbol{\theta}_{\left[j\right]}\right\Vert _{2}}-\frac{\boldsymbol{\theta}_{\left[j\right]}\boldsymbol{\theta}_{\left[j\right]}^{\top}}{\left\Vert \boldsymbol{\theta}_{\left[j\right]}\right\Vert _{2}^{3}}\right)\int\left(r_{j}\mathbb{E}_{\chi}\left\{ q_{j}^{r}q_{j}^{\theta}\right\} +r_{\neg j}\left\Vert \boldsymbol{\theta}_{\left[\neg j\right]}\right\Vert _{2}\mathbb{E}_{\chi}\left\{ \chi_{j}q_{\neg j}^{r}\partial_{j}q_{\neg j}^{\theta}\right\} \right)\rho_{r}\left({\rm d}r_{1},{\rm d}r_{2}\right)\\
 & \qquad+\left(\boldsymbol{I}+\frac{\boldsymbol{\theta}_{\left[j\right]}\boldsymbol{\theta}_{\left[j\right]}^{\top}}{\left\Vert \boldsymbol{\theta}_{\left[j\right]}\right\Vert _{2}^{2}}\right)\int r_{j}\mathbb{E}_{\chi}\left\{ \chi_{j}q_{j}^{r}\partial_{j}q_{j}^{\theta}\right\} \rho_{r}\left({\rm d}r_{1},{\rm d}r_{2}\right)\\
 & \qquad+\frac{\boldsymbol{\theta}_{\left[j\right]}\boldsymbol{\theta}_{\left[j\right]}^{\top}}{\left\Vert \boldsymbol{\theta}_{\left[j\right]}\right\Vert _{2}}\int r_{j}\mathbb{E}_{\chi}\left\{ \chi_{j}^{2}q_{j}^{r}\partial_{jj}^{2}q_{j}^{\theta}\right\} \rho_{r}\left({\rm d}r_{1},{\rm d}r_{2}\right)\\
 & \qquad+\frac{\boldsymbol{\theta}_{\left[j\right]}\boldsymbol{\theta}_{\left[j\right]}^{\top}}{\left\Vert \boldsymbol{\theta}_{\left[j\right]}\right\Vert _{2}^{2}}\int r_{\neg j}\left\Vert \boldsymbol{\theta}_{\left[\neg j\right]}\right\Vert _{2}\mathbb{E}_{\chi}\left\{ \chi_{j}^{2}q_{\neg j}^{r}\partial_{jj}^{2}q_{\neg j}^{\theta}\right\} \rho_{r}\left({\rm d}r_{1},{\rm d}r_{2}\right),\\
\left[\nabla_{11}^{2}W\left(\boldsymbol{\theta};\rho\right)\right]_{12} & =\frac{\boldsymbol{\theta}_{\left[1\right]}\boldsymbol{\theta}_{\left[2\right]}^{\top}}{\left\Vert \boldsymbol{\theta}_{\left[1\right]}\right\Vert _{2}\left\Vert \boldsymbol{\theta}_{\left[2\right]}\right\Vert _{2}}\Bigg(\int\left(r_{1}\mathbb{E}_{\chi}\left\{ \chi_{2}q_{1}^{r}\partial_{2}q_{1}^{\theta}\right\} +r_{2}\mathbb{E}_{\chi}\left\{ \chi_{1}q_{2}^{r}\partial_{1}q_{2}^{\theta}\right\} \right)\rho_{r}\left({\rm d}r_{1},{\rm d}r_{2}\right)\\
 & \qquad+\int\mathbb{E}_{\chi}\left\{ \chi_{1}\chi_{2}\left(\left\Vert \boldsymbol{\theta}_{\left[1\right]}\right\Vert _{2}r_{1}q_{1}^{r}\partial_{12}^{2}q_{1}^{\theta}+\left\Vert \boldsymbol{\theta}_{\left[2\right]}\right\Vert _{2}r_{2}q_{2}^{r}\partial_{12}^{2}q_{2}^{\theta}\right)\right\} \rho_{r}\left({\rm d}r_{1},{\rm d}r_{2}\right)\Bigg).
\end{align*}
Then again by Lemma \ref{lem:2nd_setting_qBounds}, along with the
fact $\int\left(r_{1}+r_{2}\right){\rm d}\rho_{r}\leq C$, we have:
\[
\left|\left\langle \boldsymbol{a},\left[\nabla_{11}^{2}W\left(\boldsymbol{\theta};\rho\right)\right]_{jj}\boldsymbol{b}\right\rangle \right|\leq C\left\Vert \boldsymbol{a}\right\Vert _{2}\left\Vert \boldsymbol{b}\right\Vert _{2},\qquad\left|\left\langle \boldsymbol{a}_{1},\left[\nabla_{11}^{2}W\left(\boldsymbol{\theta};\rho\right)\right]_{12}\boldsymbol{a}_{2}\right\rangle \right|\leq C\left\Vert \boldsymbol{a}_{1}\right\Vert _{2}\left\Vert \boldsymbol{a}_{2}\right\Vert _{2},
\]
for any $\boldsymbol{a},\boldsymbol{b}\in\mathbb{R}^{d_{j}}$ and
$\boldsymbol{a}_{1}\in\mathbb{R}^{d_{1}}$, $\boldsymbol{a}_{2}\in\mathbb{R}^{d_{2}}$.
This implies $\left\Vert \nabla_{11}^{2}W\left(\boldsymbol{\theta};\rho\right)\right\Vert _{2}\leq C$,
which shows that
\[
\left\Vert \nabla_{1}W\left(\boldsymbol{\theta}_{1};\rho\right)-\nabla_{1}W\left(\boldsymbol{\theta}_{2};\rho\right)\right\Vert _{2}\leq C\left\Vert \boldsymbol{\theta}_{1}-\boldsymbol{\theta}_{2}\right\Vert _{2}.
\]

\paragraph*{Bounds for $U$.}

Now we consider $U$:
\begin{align*}
\left|\left\langle \nabla_{1}U\left(\boldsymbol{\theta},\boldsymbol{\theta}'\right),\boldsymbol{v}\right\rangle \right| & =\left|\mathbb{E}_{{\cal P}}\left\{ \kappa^{2}\left\langle \boldsymbol{\theta}',\boldsymbol{v}\right\rangle \sigma\left(\left\langle \kappa\boldsymbol{\theta},\boldsymbol{x}\right\rangle \right)\sigma\left(\left\langle \kappa\boldsymbol{\theta}',\boldsymbol{x}\right\rangle \right)+\kappa^{3}\left\langle \boldsymbol{\theta},\boldsymbol{\theta}'\right\rangle \sigma'\left(\left\langle \kappa\boldsymbol{\theta},\boldsymbol{x}\right\rangle \right)\sigma\left(\left\langle \kappa\boldsymbol{\theta}',\boldsymbol{x}\right\rangle \right)\left\langle \boldsymbol{x},\boldsymbol{v}\right\rangle \right\} \right|\\
 & \leq C\kappa^{2}\left\Vert \boldsymbol{\theta}'\right\Vert _{2}\left\Vert \boldsymbol{v}\right\Vert _{2}+\mathbb{E}_{{\cal P}}\left\{ \kappa^{3}\left|\left\langle \boldsymbol{x},\boldsymbol{v}\right\rangle \right|\right\} \left\Vert \boldsymbol{\theta}\right\Vert _{2}\left\Vert \boldsymbol{\theta}'\right\Vert _{2}\\
 & =C\kappa^{2}\left\Vert \boldsymbol{\theta}'\right\Vert _{2}\left\Vert \boldsymbol{v}\right\Vert _{2}+\kappa^{2}\left\Vert \boldsymbol{\Sigma}\boldsymbol{v}\right\Vert _{2}\mathbb{E}_{g}\left\{ \left|g\right|\right\} \left\Vert \boldsymbol{\theta}\right\Vert _{2}\left\Vert \boldsymbol{\theta}'\right\Vert _{2}\\
 & \leq C\kappa^{2}\left(1+\left\Vert \boldsymbol{\theta}\right\Vert _{2}\right)\left\Vert \boldsymbol{\theta}'\right\Vert _{2}\left\Vert \boldsymbol{v}\right\Vert _{2}.
\end{align*}
This shows that $\left\Vert \nabla_{1}U\left(\boldsymbol{\theta},\boldsymbol{\theta}'\right)\right\Vert _{2}\leq C\kappa^{2}\left(1+\left\Vert \boldsymbol{\theta}\right\Vert _{2}\right)\left\Vert \boldsymbol{\theta}'\right\Vert _{2}$.

\paragraph*{Statement at $\boldsymbol{0}$.}

It is easy to see that $V\left(\boldsymbol{0}\right)=U\left(\boldsymbol{0},\boldsymbol{0}\right)=W\left(\boldsymbol{0};\rho'\right)=0$
for any $\rho'$.

\paragraph*{Statement on $\hat{\boldsymbol{x}}\left(\boldsymbol{x}\right)$.}

The formula for $\hat{\boldsymbol{x}}\left(\boldsymbol{x}\right)$
is shown in the bounding for $W$. Defining
\[
s_{j}=\int r_{j}q_{j}\left(\left\Vert \boldsymbol{x}_{\left[1\right]}\right\Vert _{2}r_{1},\left\Vert \boldsymbol{x}_{\left[2\right]}\right\Vert _{2}r_{2}\right)\rho_{r}\left({\rm d}r_{1},{\rm d}r_{2}\right),\qquad j=1,2,
\]
we have $\hat{\boldsymbol{x}}\left(\boldsymbol{x}\right)=\left(s_{1}\boldsymbol{x}_{\left[1\right]}/\left\Vert \boldsymbol{x}_{\left[1\right]}\right\Vert _{2},\;s_{2}\boldsymbol{x}_{\left[2\right]}/\left\Vert \boldsymbol{x}_{\left[2\right]}\right\Vert _{2}\right)$.
By Lemma \ref{lem:2nd_setting_qBounds}, along with the fact $\int\left(r_{1}+r_{2}\right){\rm d}\rho_{r}\leq C$,
it is easy to see that $\left|s_{1}\right|,\left|s_{2}\right|\leq C$.
Hence for any $\boldsymbol{v}\in\mathbb{S}^{d-1}$,
\begin{align*}
\mathbb{E}_{{\cal P}}\left\{ \left|\kappa\left\langle \hat{\boldsymbol{x}}\left(\boldsymbol{x}\right),\boldsymbol{v}\right\rangle \right|^{8}\right\}  & \leq C\mathbb{E}_{{\cal P}}\left\{ \left|\kappa\left\langle \frac{\boldsymbol{x}_{\left[1\right]}}{\left\Vert \boldsymbol{x}_{\left[1\right]}\right\Vert _{2}},\boldsymbol{v}_{1}\right\rangle \right|^{8}+\left|\kappa\left\langle \frac{\boldsymbol{x}_{\left[2\right]}}{\left\Vert \boldsymbol{x}_{\left[2\right]}\right\Vert _{2}},\boldsymbol{v}_{2}\right\rangle \right|^{8}\right\} \\
 & =C\mathbb{E}_{\boldsymbol{\omega}}\left\{ \left|\kappa\left\langle \boldsymbol{\omega}_{1},\boldsymbol{v}_{1}\right\rangle \right|^{8}+\left|\kappa\left\langle \boldsymbol{\omega}_{2},\boldsymbol{v}_{2}\right\rangle \right|^{8}\right\} \leq C,
\end{align*}
by Lemma \ref{lem:2nd_setting_unifVec_moments}.

\end{proof}
\begin{prop}
\label{prop:2nd_setting_grow_bound_Wrho}Consider setting \ref{enu:bdd_act_setting}
and, for each $k=1,2$, consider $\rho_{k}={\rm Law}\left(r_{1,k}\boldsymbol{\omega}_{1},r_{2,k}\boldsymbol{\omega}_{2}\right)$
in which $\left(r_{1,k},r_{2,k}\right)$, $\boldsymbol{\omega}_{1}$
and $\boldsymbol{\omega}_{2}$ are mutually independent, $\left(r_{1,k},r_{2,k}\right)\sim\rho_{r,k}$,
$r_{1,k},r_{2,k}\geq0$ and $\int\left(r_{1}^{2}+r_{2}^{2}\right)\rho_{r,k}\left({\rm d}r_{1},{\rm d}r_{2}\right)\leq C$.
Then:
\[
\left\Vert \nabla_{1}W\left(\boldsymbol{\theta};\rho_{1}\right)-\nabla_{1}W\left(\boldsymbol{\theta};\rho_{2}\right)\right\Vert _{2}\leq C\mathscr{W}_{2}\left(\rho_{r,1},\rho_{r,2}\right),
\]
where $\mathscr{W}_{2}\left(\rho_{r,1},\rho_{r,2}\right)$ is the
Wasserstein distance given as:
\[
\mathscr{W}_{2}\left(\rho_{r,1},\rho_{r,2}\right)=\inf\left\{ \int\left\Vert \boldsymbol{r}_{1}-\boldsymbol{r}_{2}\right\Vert _{2}^{2}\nu\left({\rm d}\boldsymbol{r}_{1},{\rm d}\boldsymbol{r}_{2}\right):\;\boldsymbol{r}_{k}\sim\rho_{r,k},\;k=1,2,\;\nu\text{ a coupling of }\rho_{r,1}\text{ and }\rho_{r,2}\right\} ^{1/2}.
\]
\end{prop}

\begin{proof}
We have the following formula given in the proof of Proposition \ref{prop:2nd_setting_grow_bound},
for $k=1,2$:
\begin{align*}
\nabla_{1}W\left(\boldsymbol{\theta};\rho_{k}\right) & =\left(\nabla_{1}W\left(\boldsymbol{\theta};\rho_{k}\right)_{\left[1\right]},\quad\nabla_{1}W\left(\boldsymbol{\theta};\rho_{k}\right)_{\left[2\right]}\right),\\
\nabla_{1}W\left(\boldsymbol{\theta};\rho_{k}\right)_{\left[j\right]} & =\frac{\boldsymbol{\theta}_{\left[j\right]}}{\left\Vert \boldsymbol{\theta}_{\left[j\right]}\right\Vert _{2}}\int r_{j}\mathbb{E}_{\chi}\left\{ q_{j}^{r}q_{j}^{\theta}\right\} \rho_{r,k}\left({\rm d}r_{1},{\rm d}r_{2}\right)+\boldsymbol{\theta}_{\left[j\right]}\int r_{j}\mathbb{E}_{\chi}\left\{ \chi_{j}q_{j}^{r}\partial_{j}q_{j}^{\theta}\right\} \rho_{r,k}\left({\rm d}r_{1},{\rm d}r_{2}\right)\\
 & \quad+\frac{\boldsymbol{\theta}_{\left[j\right]}}{\left\Vert \boldsymbol{\theta}_{\left[j\right]}\right\Vert _{2}}\int r_{\neg j}\left\Vert \boldsymbol{\theta}_{\left[\neg j\right]}\right\Vert _{2}\mathbb{E}_{\chi}\left\{ \chi_{j}q_{\neg j}^{r}\partial_{j}q_{\neg j}^{\theta}\right\} \rho_{r,k}\left({\rm d}r_{1},{\rm d}r_{2}\right),\qquad j=1,2,
\end{align*}
where we recall the short-hand notations, for $j,i\in\left\{ 1,2\right\} $:
\[
q_{j}^{r}=q_{j}\left(r_{1}\chi_{1},r_{2}\chi_{2}\right),\quad q_{j}^{\theta}=q_{j}\left(\left\Vert \boldsymbol{\theta}_{\left[1\right]}\right\Vert _{2}\chi_{1},\left\Vert \boldsymbol{\theta}_{\left[2\right]}\right\Vert _{2}\chi_{2}\right),\quad\partial_{i}q_{j}^{\theta}=\partial_{i}q_{j}\left(\left\Vert \boldsymbol{\theta}_{\left[1\right]}\right\Vert _{2}\chi_{1},\left\Vert \boldsymbol{\theta}_{\left[2\right]}\right\Vert _{2}\chi_{2}\right).
\]
Here $\chi_{1}\stackrel{{\rm d}}{=}\Sigma_{1}\sqrt{\alpha/d_{1}}Z_{1}$
and $\chi_{2}\stackrel{{\rm d}}{=}\Sigma_{2}\sqrt{\left(1-\alpha\right)/d_{2}}Z_{2}$
are two independent random variables, which are independent of $\boldsymbol{\omega}_{1}$
and $\boldsymbol{\omega}_{2}$, where $Z_{1}$ and $Z_{2}$ are respectively
$\chi$-random variables of degrees of freedom $d_{1}$ and $d_{2}$.
By Lemma \ref{lem:2nd_setting_qBounds},
\begin{align*}
 & \left\Vert \nabla_{1}W\left(\boldsymbol{\theta};\rho_{1}\right)_{\left[1\right]}-\nabla_{1}W\left(\boldsymbol{\theta};\rho_{2}\right)_{\left[1\right]}\right\Vert _{2}\\
 & \quad\leq\left|\int r_{1}\mathbb{E}_{\chi}\left\{ q_{1}^{r}q_{1}^{\theta}\right\} \left(\rho_{r,1}-\rho_{r,2}\right)\left({\rm d}r_{1},{\rm d}r_{2}\right)\right|+\left\Vert \boldsymbol{\theta}_{\left[1\right]}\right\Vert _{2}\left|\int r_{1}\mathbb{E}_{\chi}\left\{ \chi_{1}q_{1}^{r}\partial_{1}q_{1}^{\theta}\right\} \left(\rho_{r,1}-\rho_{r,2}\right)\left({\rm d}r_{1},{\rm d}r_{2}\right)\right|\\
 & \quad\qquad+\left\Vert \boldsymbol{\theta}_{\left[2\right]}\right\Vert _{2}\left|\int r_{2}\mathbb{E}_{\chi}\left\{ \chi_{1}q_{2}^{r}\partial_{1}q_{2}^{\theta}\right\} \left(\rho_{r,1}-\rho_{r,2}\right)\left({\rm d}r_{1},{\rm d}r_{2}\right)\right|\\
 & \quad\leq C\mathbb{E}_{\chi}\left\{ \left|\int r_{1}q_{1}^{r}\left(\rho_{r,1}-\rho_{r,2}\right)\left({\rm d}r_{1},{\rm d}r_{2}\right)\right|+\frac{\chi_{1}}{\chi_{2}}\left|\int r_{2}q_{2}^{r}\left(\rho_{r,1}-\rho_{r,2}\right)\left({\rm d}r_{1},{\rm d}r_{2}\right)\right|\right\} .
\end{align*}
Let us consider $\left|\int r_{1}q_{1}^{r}\left(\rho_{r,1}-\rho_{r,2}\right)\left({\rm d}r_{1},{\rm d}r_{2}\right)\right|$.
Consider any coupling between $\rho_{r,1}$ and $\rho_{r,2}$ so that
we can place $\left(r_{1,1},r_{2,1}\right)\sim\rho_{r,1}$ and $\left(r_{1,2},r_{2,2}\right)\sim\rho_{r,2}$
on the same joint probability space. Let $\mathbb{E}_{\backslash\chi}$
denote the expectation w.r.t. these random variables, excluding $\chi_{1}$
and $\chi_{2}$. We have:
\begin{align*}
 & \left|\int r_{1}q_{1}^{r}\left(\rho_{r,1}-\rho_{r,2}\right)\left({\rm d}r_{1},{\rm d}r_{2}\right)\right|\\
 & \quad=\mathbb{E}_{\backslash\chi}\left\{ \left|r_{1,1}q_{1}\left(\chi_{1}r_{1,1},\chi_{2}r_{2,1}\right)-r_{1,2}q_{1}\left(\chi_{1}r_{1,2},\chi_{2}r_{2,2}\right)\right|\right\} \\
 & \quad\leq\mathbb{E}_{\backslash\chi}\left\{ \left|q_{1}\left(\chi_{1}r_{1,1},\chi_{2}r_{2,1}\right)\right|\left|r_{1,1}-r_{1,2}\right|\right\} +\mathbb{E}_{\backslash\chi}\left\{ r_{1,2}\left|q_{1}\left(\chi_{1}r_{1,1},\chi_{2}r_{2,1}\right)-q_{1}\left(\chi_{1}r_{1,2},\chi_{2}r_{2,1}\right)\right|\right\} \\
 & \quad\qquad+\mathbb{E}_{\backslash\chi}\left\{ r_{1,2}\left|q_{1}\left(\chi_{1}r_{1,2},\chi_{2}r_{2,1}\right)-q_{1}\left(\chi_{1}r_{1,2},\chi_{2}r_{2,2}\right)\right|\right\} \\
 & \quad\stackrel{\left(a\right)}{\leq}\mathbb{E}_{\backslash\chi}\left\{ \left|q_{1}\left(\chi_{1}r_{1,1},\chi_{2}r_{2,1}\right)\right|\left|r_{1,1}-r_{1,2}\right|\right\} +\chi_{1}\mathbb{E}_{\backslash\chi}\left\{ r_{1,2}\left|\partial_{1}q_{1}\left(\zeta_{1},\chi_{2}r_{2,1}\right)\right|\left|r_{1,1}-r_{1,2}\right|\right\} \\
 & \quad\qquad+\chi_{2}\mathbb{E}_{\backslash\chi}\left\{ r_{1,2}\left|\partial_{2}q_{1}\left(\chi_{1}r_{1,2},\zeta_{2}\right)\right|\left|r_{2,1}-r_{2,2}\right|\right\} \\
 & \quad\stackrel{\left(b\right)}{\leq}C\left(1+\left(\chi_{1}+\chi_{2}\right)\sqrt{\mathbb{E}_{\backslash\chi}\left\{ r_{1,2}^{2}\right\} }\right)\sqrt{\mathbb{E}_{\backslash\chi}\left\{ \left|r_{1,1}-r_{1,2}\right|^{2}+\left|r_{2,1}-r_{2,2}\right|^{2}\right\} }\\
 & \quad\stackrel{\left(c\right)}{\leq}C\left(1+\chi_{1}+\chi_{2}\right)\sqrt{\mathbb{E}_{\backslash\chi}\left\{ \left|r_{1,1}-r_{1,2}\right|^{2}+\left|r_{2,1}-r_{2,2}\right|^{2}\right\} }
\end{align*}
where in step $\left(a\right)$, we use the mean value theorem for
some $\zeta_{1}$ between $\chi_{1}r_{1,1}$ and $\chi_{1}r_{1,2}$
and some $\zeta_{2}$ between $\chi_{2}r_{2,1}$ and $\chi_{2}r_{2,2}$;
in step $\left(b\right)$, we apply Lemma \ref{lem:2nd_setting_qBounds};
in step $\left(c\right)$, we recall the assumption $\int\left(r_{1}^{2}+r_{2}^{2}\right)\rho_{r,k}\left({\rm d}r_{1},{\rm d}r_{2}\right)\leq C$
for $k=1,2$. Since the coupling is arbitrary, we have:
\[
\left|\int r_{1}q_{1}^{r}\left(\rho_{r,1}-\rho_{r,2}\right)\left({\rm d}r_{1},{\rm d}r_{2}\right)\right|\leq C\left(1+\chi_{1}+\chi_{2}\right)\mathscr{W}_{2}\left(\rho_{r,1},\rho_{r,2}\right).
\]
We treat $\left|\int r_{2}q_{2}^{r}\left(\rho_{r,1}-\rho_{r,2}\right)\left({\rm d}r_{1},{\rm d}r_{2}\right)\right|$
similarly and then obtain:
\[
\left\Vert \nabla_{1}W\left(\boldsymbol{\theta};\rho_{1}\right)_{\left[1\right]}-\nabla_{1}W\left(\boldsymbol{\theta};\rho_{2}\right)_{\left[1\right]}\right\Vert _{2}\leq C\mathscr{W}_{2}\left(\rho_{r,1},\rho_{r,2}\right).
\]
A similar bound holds for $\left\Vert \nabla_{1}W\left(\boldsymbol{\theta};\rho_{1}\right)_{\left[2\right]}-\nabla_{1}W\left(\boldsymbol{\theta};\rho_{2}\right)_{\left[2\right]}\right\Vert _{2}$.
The thesis then follows.
\end{proof}
\begin{lem}
\label{lem:2nd_setting_actGaussian}Assume an activation $\sigma$
as described in setting \ref{enu:bdd_act_setting}, and a bounded
function $\phi:\;\mathbb{R}\to\mathbb{R}$, $\left\Vert \phi\right\Vert _{\infty}\leq K$.
Let $w\sim\mathsf{N}\left(0,s^{2}\right)$. Then for any integer $m\geq0$
and any $a,b\in\mathbb{R}$,
\begin{align*}
\left|\mathbb{E}_{w}\left\{ w^{m}\sigma''\left(w\right)\phi\left(w\right)\right\} \right| & \leq KC\left(m+1\right)^{\left(m+1\right)/2}s^{m-1},\\
\left|\mathbb{E}_{w}\left\{ w^{m}\sigma'''\left(w\right)\phi\left(w\right)\right\} \right| & \leq KC\left(m+1\right)^{\left(m+1\right)/2}s^{m-1},
\end{align*}
where $C$ is a constant that is independent of $K$, $s$ and $m$.
\end{lem}

\begin{proof}
By assumption, there exists an anti-derivative $\hat{\sigma}_{2}$
of $\left|\sigma''\right|$ such that $\left\Vert \hat{\sigma}_{2}\right\Vert _{\infty}\leq C$.
Let $f$ be the standard Gaussian probability density function. For
any integer $m\geq0$,
\begin{align*}
 & \left|\mathbb{E}_{w}\left\{ w^{m}\sigma''\left(w\right)\phi\left(w\right)\right\} \right|\leq KC\mathbb{E}_{w}\left\{ \left|w\right|^{m}\left|\sigma''\left(w\right)\right|\right\} =KCs^{m}\int_{-\infty}^{+\infty}\left|u\right|^{m}\left|\sigma''\left(su\right)\right|f\left(u\right){\rm d}u\\
 & \qquad=KCs^{m-1}\left(\left[\left|u\right|^{m}\hat{\sigma}_{2}\left(su\right)f\left(u\right)\right]_{u=-\infty}^{+\infty}-\int_{-\infty}^{+\infty}\hat{\sigma}_{2}\left(su\right)\left(m\left|u\right|^{m-1}{\rm sign}\left(u\right)-u\left|u\right|^{m}\right)f\left(u\right){\rm d}u\right)\\
 & \qquad\leq KCs^{m-1}\mathbb{E}_{g}\left\{ m\left|g\right|^{m-1}+\left|g\right|^{m+1}\right\} \leq KC\left(m+1\right)^{\left(m+1\right)/2}s^{m-1}.
\end{align*}
The proof for the second statement is similar.
\end{proof}
\begin{prop}
\label{prop:2nd_setting_op_bound}Consider setting \ref{enu:bdd_act_setting}.
We have:
\begin{align*}
\left\Vert \nabla_{121}^{3}U\left[\boldsymbol{\zeta},\boldsymbol{\theta}\right]\right\Vert _{{\rm op}},\left\Vert \nabla_{122}^{3}U\left[\boldsymbol{\theta},\boldsymbol{\zeta}\right]\right\Vert _{{\rm op}} & \leq C\kappa^{2}\left(1+\left\Vert \boldsymbol{\theta}\right\Vert _{2}\right),\\
\left\Vert \nabla_{12}^{2}U\left(\boldsymbol{\theta},\boldsymbol{\theta}'\right)\right\Vert _{{\rm op}} & \leq C\kappa^{2}\left(1+\left\Vert \boldsymbol{\theta}\right\Vert _{2}\right)\left(1+\left\Vert \boldsymbol{\theta}'\right\Vert _{2}\right),\\
\left\Vert \nabla_{11}^{2}U\left(\boldsymbol{\theta},\boldsymbol{\theta}'\right)\right\Vert _{{\rm op}} & \leq C\kappa^{2}\left\Vert \boldsymbol{\theta}'\right\Vert _{2},
\end{align*}
for any $\boldsymbol{\zeta},\boldsymbol{\theta},\boldsymbol{\theta}'\in\mathbb{R}^{d}$.
\end{prop}

\begin{proof}
The proof is almost the same as that of Proposition \ref{prop:1st_setting_op_bound},
so we omit several similar calculations and refer to the proof of
Proposition \ref{prop:1st_setting_op_bound} for the definitions of
the quantities. In particular, we obtain:
\begin{align*}
\left|A_{1}\right|,\left|A_{2}\right|,\left|A_{4}\right|,\left|A_{5}\right|,\left|B_{1}\right|,\left|B_{3}\right|,\left|B_{4}\right|,\left|B_{5}\right| & \leq C\kappa^{2}\left(1+\left\Vert \boldsymbol{\theta}\right\Vert _{2}\right)\left\Vert \boldsymbol{a}\right\Vert _{2}\left\Vert \boldsymbol{b}\right\Vert _{2}\left\Vert \boldsymbol{c}\right\Vert _{2},\\
\left|F_{1}\right|,\left|F_{2}\right|,\left|F_{3}\right|,\left|F_{4}\right| & \leq C\kappa^{2}\left(1+\left\Vert \boldsymbol{\theta}\right\Vert _{2}\right)\left(1+\left\Vert \boldsymbol{\theta}'\right\Vert _{2}\right)\left\Vert \boldsymbol{a}\right\Vert _{2}\left\Vert \boldsymbol{b}\right\Vert _{2},\\
\left|H_{1}\right| & \leq C\kappa^{2}\left\Vert \boldsymbol{\theta}'\right\Vert _{2}\left\Vert \boldsymbol{a}\right\Vert _{2}\left\Vert \boldsymbol{b}\right\Vert _{2},
\end{align*}
for a suitable constant $C$. We are left with $A_{3}$, $A_{6}$,
$B_{2}$, $B_{6}$ and $H_{2}$. We consider $A_{3}$. Proceeding
as in the proof of Proposition \ref{prop:1st_setting_op_bound}, we
have:
\begin{align*}
A_{3} & =\kappa^{2}\left\langle \boldsymbol{b},\boldsymbol{\zeta}\right\rangle \mathbb{E}_{w,\tilde{\boldsymbol{z}}}\left\{ \sigma''\left(w\right)\sigma\left(\left\langle \boldsymbol{S}\boldsymbol{\theta},\tilde{\boldsymbol{z}}\right\rangle +\frac{\left\langle \boldsymbol{\Sigma}^{2}\boldsymbol{\theta},\boldsymbol{\zeta}\right\rangle }{\left\Vert \boldsymbol{\Sigma}\boldsymbol{\zeta}\right\Vert _{2}^{2}}w\right)\left\langle \boldsymbol{S}\boldsymbol{a},\tilde{\boldsymbol{z}}\right\rangle \left\langle \boldsymbol{S}\boldsymbol{c},\tilde{\boldsymbol{z}}\right\rangle \right\} \\
 & \qquad+\kappa^{2}\left\langle \boldsymbol{b},\boldsymbol{\zeta}\right\rangle \mathbb{E}_{w,\tilde{\boldsymbol{z}}}\left\{ \sigma''\left(w\right)\sigma\left(\left\langle \boldsymbol{S}\boldsymbol{\theta},\tilde{\boldsymbol{z}}\right\rangle +\frac{\left\langle \boldsymbol{\Sigma}^{2}\boldsymbol{\theta},\boldsymbol{\zeta}\right\rangle }{\left\Vert \boldsymbol{\Sigma}\boldsymbol{\zeta}\right\Vert _{2}^{2}}w\right)\frac{w}{\left\Vert \boldsymbol{\Sigma}\boldsymbol{\zeta}\right\Vert _{2}^{2}}\left\langle \boldsymbol{\Sigma}^{2}\boldsymbol{a},\boldsymbol{\zeta}\right\rangle \left\langle \boldsymbol{S}\boldsymbol{c},\tilde{\boldsymbol{z}}\right\rangle \right\} \\
 & \qquad+\kappa^{2}\left\langle \boldsymbol{b},\boldsymbol{\zeta}\right\rangle \mathbb{E}_{w,\tilde{\boldsymbol{z}}}\left\{ \sigma''\left(w\right)\sigma\left(\left\langle \boldsymbol{S}\boldsymbol{\theta},\tilde{\boldsymbol{z}}\right\rangle +\frac{\left\langle \boldsymbol{\Sigma}^{2}\boldsymbol{\theta},\boldsymbol{\zeta}\right\rangle }{\left\Vert \boldsymbol{\Sigma}\boldsymbol{\zeta}\right\Vert _{2}^{2}}w\right)\frac{w}{\left\Vert \boldsymbol{\Sigma}\boldsymbol{\zeta}\right\Vert _{2}^{2}}\left\langle \boldsymbol{S}\boldsymbol{a},\tilde{\boldsymbol{z}}\right\rangle \left\langle \boldsymbol{\Sigma}^{2}\boldsymbol{c},\boldsymbol{\zeta}\right\rangle \right\} \\
 & \qquad+\kappa^{2}\left\langle \boldsymbol{b},\boldsymbol{\zeta}\right\rangle \mathbb{E}_{w,\tilde{\boldsymbol{z}}}\left\{ \sigma''\left(w\right)\sigma\left(\left\langle \boldsymbol{S}\boldsymbol{\theta},\tilde{\boldsymbol{z}}\right\rangle +\frac{\left\langle \boldsymbol{\Sigma}^{2}\boldsymbol{\theta},\boldsymbol{\zeta}\right\rangle }{\left\Vert \boldsymbol{\Sigma}\boldsymbol{\zeta}\right\Vert _{2}^{2}}w\right)\frac{w^{2}}{\left\Vert \boldsymbol{\Sigma}\boldsymbol{\zeta}\right\Vert _{2}^{4}}\left\langle \boldsymbol{\Sigma}^{2}\boldsymbol{a},\boldsymbol{\zeta}\right\rangle \left\langle \boldsymbol{\Sigma}^{2}\boldsymbol{c},\boldsymbol{\zeta}\right\rangle \right\} ,
\end{align*}
for $\tilde{\boldsymbol{z}}\sim\mathsf{N}\left(0,\boldsymbol{I}_{d}\right)$
and $w\sim\mathsf{N}\left(0,\left\Vert \boldsymbol{\Sigma}\boldsymbol{\zeta}\right\Vert _{2}^{2}\right)$
independently, where $\boldsymbol{S}={\rm Proj}_{\boldsymbol{\Sigma}\boldsymbol{\zeta}}^{\perp}\boldsymbol{\Sigma}$.
Applying Lemma \ref{lem:2nd_setting_actGaussian}, recalling that
$\left\Vert \boldsymbol{\Sigma}\right\Vert _{{\rm op}}\leq C$, $\left\Vert \boldsymbol{\Sigma}\boldsymbol{\zeta}\right\Vert _{2}\geq C\left\Vert \boldsymbol{\zeta}\right\Vert _{2}$
and $\left\Vert \boldsymbol{S}\right\Vert _{{\rm op}}\leq\left\Vert \boldsymbol{\Sigma}\right\Vert _{{\rm op}}\leq C$,
we obtain:
\begin{align*}
\left|A_{3}\right| & \leq C\kappa^{2}\frac{\left|\left\langle \boldsymbol{b},\boldsymbol{\zeta}\right\rangle \right|}{\left\Vert \boldsymbol{\Sigma}\boldsymbol{\zeta}\right\Vert _{2}}\mathbb{E}_{\tilde{\boldsymbol{z}}}\left\{ \left|\left\langle \boldsymbol{S}\boldsymbol{a},\tilde{\boldsymbol{z}}\right\rangle \left\langle \boldsymbol{S}\boldsymbol{c},\tilde{\boldsymbol{z}}\right\rangle \right|\right\} \\
 & \qquad+C\kappa^{2}\frac{\left|\left\langle \boldsymbol{b},\boldsymbol{\zeta}\right\rangle \right|}{\left\Vert \boldsymbol{\Sigma}\boldsymbol{\zeta}\right\Vert _{2}^{2}}\left(\left|\left\langle \boldsymbol{\Sigma}^{2}\boldsymbol{a},\boldsymbol{\zeta}\right\rangle \right|\mathbb{E}_{\tilde{\boldsymbol{z}}}\left\{ \left|\left\langle \boldsymbol{S}\boldsymbol{c},\tilde{\boldsymbol{z}}\right\rangle \right|\right\} +\left|\left\langle \boldsymbol{\Sigma}^{2}\boldsymbol{c},\boldsymbol{\zeta}\right\rangle \right|\mathbb{E}_{\tilde{\boldsymbol{z}}}\left\{ \left|\left\langle \boldsymbol{S}\boldsymbol{a},\tilde{\boldsymbol{z}}\right\rangle \right|\right\} \right)\\
 & \qquad+C\kappa^{2}\frac{\left|\left\langle \boldsymbol{b},\boldsymbol{\zeta}\right\rangle \right|}{\left\Vert \boldsymbol{\Sigma}\boldsymbol{\zeta}\right\Vert _{2}^{3}}\left|\left\langle \boldsymbol{\Sigma}^{2}\boldsymbol{a},\boldsymbol{\zeta}\right\rangle \left\langle \boldsymbol{\Sigma}^{2}\boldsymbol{c},\boldsymbol{\zeta}\right\rangle \right|\\
 & \leq C\kappa^{2}\left\Vert \boldsymbol{a}\right\Vert _{2}\left\Vert \boldsymbol{b}\right\Vert _{2}\left\Vert \boldsymbol{c}\right\Vert _{2}.
\end{align*}
Similar calculations yield:
\begin{align*}
\left|A_{3}\right|,\left|A_{6}\right|,\left|B_{2}\right|,\left|B_{6}\right| & \leq C\kappa^{2}\left(\left\Vert \boldsymbol{\theta}\right\Vert _{2}+1\right)\left\Vert \boldsymbol{a}\right\Vert _{2}\left\Vert \boldsymbol{b}\right\Vert _{2}\left\Vert \boldsymbol{c}\right\Vert _{2},\\
\left|H_{2}\right| & \leq C\kappa^{2}\left\Vert \boldsymbol{\theta}'\right\Vert _{2}\left\Vert \boldsymbol{a}\right\Vert _{2}\left\Vert \boldsymbol{b}\right\Vert _{2}.
\end{align*}
The thesis then follows.
\end{proof}
\begin{prop}
\label{prop:2nd_setting_ODE_r}Consider setting \ref{enu:bdd_act_setting}.
Recall the process $\left(r_{1,t},r_{2,t},\rho_{r}^{t}\right)_{t\geq0}$
that is described as in the statement of Theorem \ref{thm:bdd_act_setting}
via the ODE (\ref{eq:2nd_setting_ODE_r}). This ODE has a (weakly)
unique solution on $t\in[0,\infty)$. Furthermore, this solution satisfies
a sub-Gaussian moment bound:
\[
\int\left(\bar{r}_{1}^{p}+\bar{r}_{2}^{p}\right)\rho_{r}^{t}\left({\rm d}\bar{r}_{1},{\rm d}\bar{r}_{2}\right)\leq C^{p}\left(1+t^{p}\right)p^{p/2},
\]
for any integer $p\geq1$, where the immaterial constant $C$ is independent
of $t$. We also have, $\left(r_{1,t},r_{2,t}\right)$ is a deterministic
function of $\left(r_{1,0},r_{2,0}\right)$, i.e. $\left(r_{1,t},r_{2,t}\right)=\psi_{t}\left(r_{1,0},r_{2,0}\right)$,
such that $\left\Vert \partial_{t}\psi_{t}\left(r_{1},r_{2}\right)\right\Vert _{2}\leq C\left(1+t+r_{1}+r_{2}\right)$.
\end{prop}

\begin{proof}
We decompose the proof into several steps. We first show existence
and uniqueness of the solution, via a Picard-type iteration argument,
by adapting the strategy of \cite{sznitman1991topics}. This is done
from Steps 1-3 below. Then we show the properties of the solution.
Before we proceed, let us define:
\begin{align*}
G_{j}\left(r_{1},r_{2},\rho\right) & =-\mathbb{E}_{\chi}\left\{ \Delta_{j}\left(\chi,\rho\right)\left[q_{j}\left(\chi_{1}r_{1},\chi_{2}r_{2}\right)+\chi_{j}r_{j}\partial_{j}q_{j}\left(\chi_{1}r_{1},\chi_{2}r_{2}\right)\right]\right\} \\
 & \qquad-\mathbb{E}_{\chi}\left\{ \Delta_{\neg j}\left(\chi,\rho\right)\chi_{j}r_{\neg j}\partial_{j}q_{\neg j}\left(\chi_{1}r_{1},\chi_{2}r_{2}\right)\right\} -2\lambda r_{j},\qquad j=1,2,
\end{align*}
where we recall the convention $\neg j=2$ if $j=1$ and $\neg j=1$
if $j=2$.

\paragraph*{Step 1: Setup.}

Fix a terminal time $T\geq0$ that is to be chosen later. Let ${\cal C}={\cal C}\left(\left[0,T\right];\mathbb{R}^{2}\right)$
be the set of continuous mappings from $\left[0,T\right]$ to $\mathbb{R}^{2}$,
and $\mathscr{P}\left({\cal C};K\right)$ the set of probability measures
on ${\cal C}$ such that if $\mu\in\mathscr{P}\left({\cal C};K\right)$,
$\mathbb{E}_{\chi}\left\{ \Delta_{j}\left(\chi,\mu^{t}\right)^{2}\right\} \leq K$
for $j=1,2$ and any $t\in\left[0,T\right]$, for a constant $K\geq0$
that is to be chosen later. We equip this space with the following
Wasserstein metric:
\[
\mathscr{W}_{T}\left(\mu_{1},\mu_{2}\right)=\inf\left\{ \int\sup_{t\leq T}\sum_{j\in\left\{ 1,2\right\} }\left(r_{j,t}^{\left(1\right)}-r_{j,t}^{\left(2\right)}\right)^{2}\nu\left({\rm d}\boldsymbol{r}^{\left(1\right)},{\rm d}\boldsymbol{r}^{\left(2\right)}\right):\;\text{\ensuremath{\nu} is a coupling of \ensuremath{\mu_{1}} and \ensuremath{\mu_{2}}}\right\} ^{1/2}.
\]
Note that this defines a complete metric on $\mathscr{P}\left({\cal C};\infty\right)$.
We also note $\mathscr{P}\left({\cal C};K\right)\subseteq\mathscr{P}\left({\cal C};\infty\right)$
for all $K\geq0$. We prove that $\mathscr{P}\left({\cal C};K\right)$
is still a complete metric space under $\mathscr{W}_{T}$. Observe
that, for any $\mu_{1},\mu_{2}\in\mathscr{P}\left({\cal C};\infty\right)$
and $t\in\left[0,T\right]$, 
\begin{align}
\left|\Delta_{1}\left(\chi,\mu_{1}^{t}\right)-\Delta_{1}\left(\chi,\mu_{2}^{t}\right)\right| & =\left|\mathbb{E}\left\{ r_{1,t}^{\left(1\right)}q_{1}\left(\chi_{1}r_{1,t}^{\left(1\right)},\chi_{2}r_{2,t}^{\left(1\right)}\right)-r_{1,t}^{\left(2\right)}q_{1}\left(\chi_{1}r_{1,t}^{\left(2\right)},\chi_{2}r_{2,t}^{\left(2\right)}\right)\right\} \right|\nonumber \\
 & \leq\sup_{u_{1},u_{2}\geq0}\left|q_{1}\left(\chi_{1}u_{1},\chi_{2}u_{2}\right)+\chi_{1}u_{1}\partial_{1}q_{1}\left(\chi_{1}u_{1},\chi_{2}u_{2}\right)\right|\mathbb{E}\left\{ \left|r_{1,t}^{\left(1\right)}-r_{1,t}^{\left(2\right)}\right|\right\} \nonumber \\
 & \qquad+\left(\chi_{2}/\chi_{1}\right)\sup_{u_{1},u_{2}\geq0}\left|\chi_{1}u_{1}\partial_{2}q_{1}\left(\chi_{1}u_{1},\chi_{2}u_{2}\right)\right|\mathbb{E}\left\{ \left|r_{2,t}^{\left(1\right)}-r_{2,t}^{\left(2\right)}\right|\right\} \nonumber \\
 & \stackrel{\left(a\right)}{\leq}C\left(1+\chi_{2}/\chi_{1}\right)\left(\mathbb{E}\left\{ \left|r_{1,t}^{\left(1\right)}-r_{1,t}^{\left(2\right)}\right|\right\} +\mathbb{E}\left\{ \left|r_{2,t}^{\left(1\right)}-r_{2,t}^{\left(2\right)}\right|\right\} \right)\nonumber \\
 & \leq C\left(1+\chi_{2}/\chi_{1}\right)\sqrt{\mathbb{E}\left\{ \left(r_{1,t}^{\left(1\right)}-r_{1,t}^{\left(2\right)}\right)^{2}+\left(r_{2,t}^{\left(1\right)}-r_{2,t}^{\left(2\right)}\right)^{2}\right\} },\label{eq:2nd_setting_lem_sol_rho_r_1}
\end{align}
where the expectation is taken over an arbitrary coupling between
$\left(r_{1,t}^{\left(1\right)},r_{2,t}^{\left(1\right)}\right)\sim\mu_{1}^{t}$
and $\left(r_{1,t}^{\left(2\right)},r_{2,t}^{\left(2\right)}\right)\sim\mu_{2}^{t}$,
and step $\left(a\right)$ is due to Lemma \ref{lem:2nd_setting_qBounds}.
Therefore,
\begin{align*}
 & \left|\mathbb{E}_{\chi}\left\{ \Delta_{1}\left(\chi,\mu_{1}^{t}\right)^{2}\right\} -\mathbb{E}_{\chi}\left\{ \Delta_{1}\left(\chi,\mu_{2}^{t}\right)^{2}\right\} \right|\\
 & \quad\leq\left(\sqrt{\mathbb{E}_{\chi}\left\{ \Delta_{1}\left(\chi,\mu_{1}^{t}\right)^{2}\right\} }+\sqrt{\mathbb{E}_{\chi}\left\{ \Delta_{1}\left(\chi,\mu_{2}^{t}\right)^{2}\right\} }\right)\sqrt{\mathbb{E}_{\chi}\left\{ \left|\Delta_{1}\left(\chi,\mu_{1}^{t}\right)-\Delta_{1}\left(\chi,\mu_{2}^{t}\right)\right|^{2}\right\} }\\
 & \quad\leq C\left(\sqrt{\mathbb{E}_{\chi}\left\{ \Delta_{1}\left(\chi,\mu_{1}^{t}\right)^{2}\right\} }+\sqrt{\mathbb{E}_{\chi}\left\{ \Delta_{1}\left(\chi,\mu_{2}^{t}\right)^{2}\right\} }\right)\sqrt{\mathbb{E}_{\chi}\left\{ 1+\chi_{2}^{2}/\chi_{1}^{2}\right\} }\mathscr{W}_{T}\left(\mu_{1},\mu_{2}\right)\\
 & \quad\leq C\left(\sqrt{\mathbb{E}_{\chi}\left\{ \Delta_{1}\left(\chi,\mu_{1}^{t}\right)^{2}\right\} }+\sqrt{\mathbb{E}_{\chi}\left\{ \Delta_{1}\left(\chi,\mu_{2}^{t}\right)^{2}\right\} }\right)\mathscr{W}_{T}\left(\mu_{1},\mu_{2}\right).
\end{align*}
Now we take a sequence $\left(\mu_{n}\right)_{n\in\mathbb{N}}$ such
that $\mu_{n}\in\mathscr{P}\left({\cal C};K\right)$ and $\mu_{n}\stackrel{\mathscr{W}_{T}}{\longrightarrow}\mu$,
and apply this result to $\mu_{n}$ and $\mu$:
\begin{align*}
\mathbb{E}_{\chi}\left\{ \Delta_{1}\left(\chi,\mu^{t}\right)^{2}\right\}  & \leq\mathbb{E}_{\chi}\left\{ \Delta_{1}\left(\chi,\mu_{n}^{t}\right)^{2}\right\} +C\left(\sqrt{\mathbb{E}_{\chi}\left\{ \Delta_{1}\left(\chi,\mu_{n}^{t}\right)^{2}\right\} }+\sqrt{\mathbb{E}_{\chi}\left\{ \Delta_{1}\left(\chi,\mu^{t}\right)^{2}\right\} }\right)\mathscr{W}_{T}\left(\mu_{n},\mu\right)\\
 & \leq K+C\left(\sqrt{K}+\sqrt{\mathbb{E}_{\chi}\left\{ \Delta_{1}\left(\chi,\mu^{t}\right)^{2}\right\} }\right)\mathscr{W}_{T}\left(\mu_{n},\mu\right),
\end{align*}
since $\mu_{n}\in\mathscr{P}\left({\cal C};K\right)$. Suppose that
$\mathbb{E}_{\chi}\left\{ \Delta_{1}\left(\chi,\mu^{t}\right)^{2}\right\} \geq K+\epsilon$
for an arbitrary $\epsilon>0$ and some $t\in\left[0,T\right]$. Then
the above implies,
\[
\sqrt{\mathbb{E}_{\chi}\left\{ \Delta_{1}\left(\chi,\mu^{t}\right)^{2}\right\} }\leq\lim_{n\to\infty}\frac{K+C\sqrt{K}\mathscr{W}_{T}\left(\mu_{n},\mu\right)}{\sqrt{K+\epsilon}-C\mathscr{W}_{T}\left(\mu_{n},\mu\right)}=\frac{K}{\sqrt{K+\epsilon}},
\]
which contradicts $\mathbb{E}_{\chi}\left\{ \Delta_{1}\left(\chi,\mu^{t}\right)^{2}\right\} \geq K+\epsilon$.
Hence $\mathbb{E}_{\chi}\left\{ \Delta_{1}\left(\chi,\mu^{t}\right)^{2}\right\} \leq K$.
We also have similarly $\mathbb{E}_{\chi}\left\{ \Delta_{2}\left(\chi,\mu^{t}\right)^{2}\right\} \leq K$.
That is, $\mu\in\mathscr{P}\left({\cal C};K\right)$, and hence $\mathscr{P}\left({\cal C};K\right)$
is closed. Since $\mathscr{P}\left({\cal C};K\right)\subseteq\mathscr{P}\left({\cal C};\infty\right)$
and $\mathscr{P}\left({\cal C};\infty\right)$ is complete, we have
that $\mathscr{P}\left({\cal C};K\right)$ is complete, as desired.

\paragraph*{Step 2: The iterating map $\Phi$.}

We shall depart from the initial law $\rho_{r}^{0}$ as given in the
ODE (\ref{eq:2nd_setting_ODE_r}), and consider a generic initial
law $\tilde{\rho}_{r}^{0}\in\mathscr{P}\left(\mathbb{R}^{2}\right)$
such that $M\left(\tilde{\rho}_{r}^{0}\right)<\infty$, where we define
\[
M\left(\tilde{\rho}_{r}^{0}\right)=\max\left(1,\;\int\left(\bar{r}_{1}^{2}+\bar{r}_{2}^{2}\right)\tilde{\rho}_{r}^{0}\left({\rm d}\bar{r}_{1},{\rm d}\bar{r}_{2}\right)\right).
\]
Define $\Phi:\;\mathscr{P}\left({\cal C};K\right)\to\mathscr{P}\left({\cal C};K\right)$
which associates $\mu\in\mathscr{P}\left({\cal C};K\right)$ to the
law of $\left(\tilde{r}_{1,t},\tilde{r}_{2,t}\right)_{t\in\left[0,T\right]}$,
which is the solution to
\[
\tilde{r}_{j,t}=\tilde{r}_{j,0}+\int_{s=0}^{t}G_{j}\left(\tilde{r}_{1,s},\tilde{r}_{2,s},\mu^{s}\right){\rm d}s,\qquad t\leq T,\quad j=1,2,\quad\left(\tilde{r}_{1,0},\tilde{r}_{2,0}\right)\sim\tilde{\rho}_{r}^{0}.
\]
If $\mu$ is a weak solution of the ODE (\ref{eq:2nd_setting_ODE_r})
with initialization $\tilde{\rho}_{r}^{0}$, then it is a fixed point
of $\Phi$, and vice versa -- assuming that this is well-defined.
That is, we need to check that firstly, the process $\left(\tilde{r}_{1,t},\tilde{r}_{2,t}\right)_{t\in\left[0,T\right]}$
under $\mu\in\mathscr{P}\left({\cal C};K\right)$ exists and is unique
under any initialization $\left(\tilde{r}_{1,0},\tilde{r}_{2,0}\right)\in[0,\infty)\times[0,\infty)$,
and secondly, $\Phi\left(\mu\right)\in\mathscr{P}\left({\cal C};K\right)$
for any $\mu\in\mathscr{P}\left({\cal C};K\right)$, for suitably
chosen $K$ and $T$. We remark that $\Phi\left(\mu\right)\in\mathscr{P}\left({\cal C};K\right)$
already implies $\mathbb{E}_{\chi}\left\{ \Delta_{j}\left(\chi,\tilde{\rho}_{r}^{0}\right)^{2}\right\} \leq K$
for $j\in\left\{ 1,2\right\} $.

We check the first condition. By Lemma \ref{lem:2nd_setting_qBounds},
for $\mu\in\mathscr{P}\left({\cal C};K\right)$ and any $t\leq T$,
\begin{align*}
\left|\partial_{1}G_{1}\left(r_{1},r_{2},\mu^{t}\right)\right| & =\Big|-\mathbb{E}_{\chi}\left\{ \Delta_{1}\left(\chi,\mu^{t}\right)\left[\chi_{1}\partial_{1}q_{1}\left(\chi_{1}r_{1},\chi_{2}r_{2}\right)+\chi_{1}^{2}r_{1}\partial_{11}^{2}q_{1}\left(\chi_{1}r_{1},\chi_{2}r_{2}\right)\right]\right\} \\
 & \qquad-\mathbb{E}_{\chi}\left\{ \Delta_{2}\left(\chi,\mu^{t}\right)\chi_{1}^{2}r_{2}\partial_{11}^{2}q_{2}\left(\chi_{1}r_{1},\chi_{2}r_{2}\right)\right\} -2\lambda\Big|\\
 & \leq C\mathbb{E}_{\chi}\left\{ \chi_{1}\left|\Delta_{1}\left(\chi,\mu^{t}\right)\right|+\frac{\chi_{1}^{2}}{\chi_{2}}\left|\Delta_{2}\left(\chi,\mu^{t}\right)\right|+1\right\} \\
 & \leq C\left(\mathbb{E}_{\chi}\left\{ \chi_{1}^{2}\right\} ^{1/2}\mathbb{E}_{\chi}\left\{ \left|\Delta_{1}\left(\chi,\mu^{t}\right)\right|^{2}\right\} ^{1/2}+\mathbb{E}_{\chi}\left\{ \chi_{1}^{4}/\chi_{2}^{2}\right\} ^{1/2}\mathbb{E}_{\chi}\left\{ \left|\Delta_{2}\left(\chi,\mu^{t}\right)\right|^{2}\right\} ^{1/2}+1\right)\\
 & \leq C\left(\sqrt{K}+1\right),\\
\left|\partial_{2}G_{1}\left(r_{1},r_{2},\mu^{t}\right)\right| & =\Big|-\mathbb{E}_{\chi}\left\{ \Delta_{1}\left(\chi,\mu^{t}\right)\left[\chi_{2}\partial_{2}q_{1}\left(\chi_{1}r_{1},\chi_{2}r_{2}\right)+\chi_{1}\chi_{2}r_{1}\partial_{12}^{2}q_{1}\left(\chi_{1}r_{1},\chi_{2}r_{2}\right)\right]\right\} \\
 & \qquad-\mathbb{E}_{\chi}\left\{ \Delta_{2}\left(\chi,\mu^{t}\right)\chi_{1}\chi_{2}r_{2}\partial_{12}^{2}q_{2}\left(\chi_{1}r_{1},\chi_{2}r_{2}\right)\right\} \Big|\\
 & \leq C\mathbb{E}_{\chi}\left\{ \chi_{2}\left|\Delta_{1}\left(\chi,\mu^{t}\right)\right|+\chi_{1}\left|\Delta_{2}\left(\chi,\mu^{t}\right)\right|\right\} \\
 & \leq C\left(\mathbb{E}_{\chi}\left\{ \chi_{2}^{2}\right\} ^{1/2}\mathbb{E}_{\chi}\left\{ \left|\Delta_{1}\left(\chi,\mu^{t}\right)\right|^{2}\right\} ^{1/2}+\mathbb{E}_{\chi}\left\{ \chi_{1}^{2}\right\} ^{1/2}\mathbb{E}_{\chi}\left\{ \left|\Delta_{2}\left(\chi,\mu^{t}\right)\right|^{2}\right\} ^{1/2}\right)\\
 & \leq C\sqrt{K}.
\end{align*}
Similarly $\left|\partial_{2}G_{2}\left(r_{1},r_{2},\mu^{t}\right)\right|,\;\left|\partial_{1}G_{2}\left(r_{1},r_{2},\mu^{t}\right)\right|\leq C\left(\sqrt{K}+1\right)$,
uniformly in $t\in\left[0,T\right]$. It is easy to see that $t\mapsto G_{j}\left(r_{1},r_{2},\mu^{t}\right)$
is continuous, for $j\in\left\{ 1,2\right\} $ and any $r_{1}$, $r_{2}$,
since $\mu\in\mathscr{P}\left({\cal C};K\right)$. Existence and uniqueness
of $\left(\tilde{r}_{1,t},\tilde{r}_{2,t}\right)_{t\in\left[0,T\right]}$
then follow upon choosing $K<\infty$.

We check the second condition. We have for $\mu\in\mathscr{P}\left({\cal C};K\right)$:
\begin{align}
\tilde{r}_{1,t} & =\tilde{r}_{1,0}+\int_{s=0}^{t}\left(G_{1}\left(\tilde{r}_{1,s},\tilde{r}_{2,s},\mu^{s}\right)+2\lambda\tilde{r}_{1,s}\right){\rm d}s-\int_{s=0}^{t}2\lambda\tilde{r}_{1,s}{\rm d}s\nonumber \\
 & \stackrel{\left(a\right)}{\leq}\tilde{r}_{1,0}+C\int_{s=0}^{t}\mathbb{E}_{\chi}\left\{ \left|\Delta_{1}\left(\chi,\mu^{s}\right)\right|+\left(\chi_{1}/\chi_{2}\right)\left|\Delta_{2}\left(\chi,\mu^{s}\right)\right|\right\} {\rm d}s\nonumber \\
 & \leq\tilde{r}_{1,0}+C\int_{s=0}^{t}\left(\sqrt{\mathbb{E}_{\chi}\left\{ \left|\Delta_{1}\left(\chi,\mu^{s}\right)\right|^{2}\right\} }+\sqrt{\mathbb{E}_{\chi}\left\{ \chi_{1}^{2}/\chi_{2}^{2}\right\} \mathbb{E}_{\chi}\left\{ \left|\Delta_{2}\left(\chi,\mu^{s}\right)\right|^{2}\right\} }\right){\rm d}s\nonumber \\
 & \leq\tilde{r}_{1,0}+C\sqrt{K}t\label{eq:2nd_setting_lem_sol_rho_r_2}
\end{align}
where step $\left(a\right)$ is due to Lemma \ref{lem:2nd_setting_qBounds}
and the fact $\lambda\tilde{r}_{1,s}\geq0$. Using this and recalling
that $\Phi\left(\mu\right)^{t}={\rm Law}\left(\tilde{r}_{1,t},\tilde{r}_{2,t}\right)$,
we get:
\begin{align*}
\mathbb{E}_{\chi}\left\{ \Delta_{1}\left(\chi,\Phi\left(\mu\right)^{t}\right)^{2}\right\}  & \leq C\int\bar{r}_{1}^{2}\Phi\left(\mu\right)^{t}\left({\rm d}\bar{r}_{1},{\rm d}\bar{r}_{2}\right)+2\mathbb{E}_{\chi}\left\{ \chi_{1}^{2}\right\} =C\mathbb{E}\left\{ \tilde{r}_{j,t}^{2}\right\} +C\\
 & \leq C\mathbb{E}\left\{ \tilde{r}_{1,0}^{2}\right\} +CKt^{2}+C\leq C\left(M\left(\tilde{\rho}_{r}^{0}\right)+KT^{2}\right),
\end{align*}
where we have used Lemma \ref{lem:2nd_setting_qBounds} in the first
inequality and the fact $M\left(\tilde{\rho}_{r}^{0}\right)\geq1$
by definition. One can obtain similarly:
\[
\max_{j\in\left\{ 1,2\right\} }\mathbb{E}_{\chi}\left\{ \Delta_{j}\left(\chi,\Phi\left(\mu\right)^{t}\right)^{2}\right\} \leq C_{*}\left(M\left(\tilde{\rho}_{r}^{0}\right)+KT^{2}\right),
\]
for some constant $C_{*}>0$ independent of $M\left(\tilde{\rho}_{r}^{0}\right)$,
$K$ and $T$. By choosing $K=2C_{*}M\left(\tilde{\rho}_{r}^{0}\right)<\infty$
and $T=1/\sqrt{2C_{*}}$, we get $\mathbb{E}_{\chi}\left\{ \Delta_{j}\left(\chi,\Phi\left(\mu\right)^{t}\right)^{2}\right\} \leq K$
for $j=1,2$. That is, $\Phi\left(\mu\right)\in\mathscr{P}\left({\cal C};K\right)$.

\paragraph*{Step 3: Contraction of $\Phi$.}

Now we show a contraction property of $\Phi$. Let us consider $\mu_{1},\mu_{2}\in\mathscr{P}\left({\cal C};K\right)$,
and a coupling:
\[
\tilde{r}_{j,t}^{\left(1\right)}=\tilde{r}_{j,0}+\int_{s=0}^{t}G_{j}\left(\tilde{r}_{1,s}^{\left(1\right)},\tilde{r}_{2,s}^{\left(1\right)},\mu_{1}^{s}\right){\rm d}s,\qquad\tilde{r}_{j,t}^{\left(2\right)}=\tilde{r}_{j,0}+\int_{s=0}^{t}G_{j}\left(\tilde{r}_{1,s}^{\left(2\right)},\tilde{r}_{2,s}^{\left(2\right)},\mu_{2}^{s}\right){\rm d}s,\qquad t\leq T,\quad j=1,2.
\]
We have for $t\leq T$:
\begin{align*}
\sup_{s\leq t}\sum_{j\in\left\{ 1,2\right\} }\left|\tilde{r}_{j,s}^{\left(1\right)}-\tilde{r}_{j,s}^{\left(2\right)}\right| & \leq\sum_{j\in\left\{ 1,2\right\} }\int_{s=0}^{t}\left|G_{j}\left(\tilde{r}_{1,s}^{\left(1\right)},\tilde{r}_{2,s}^{\left(1\right)},\mu_{1}^{s}\right)-G_{j}\left(\tilde{r}_{1,s}^{\left(2\right)},\tilde{r}_{2,s}^{\left(2\right)},\mu_{2}^{s}\right)\right|{\rm d}s\\
 & \leq\sum_{j\in\left\{ 1,2\right\} }\sum_{i\in\left\{ 1,2\right\} }\sup_{r_{1},r_{2}\geq0,\;\mu\in\mathscr{P}\left({\cal C};K\right),\;t\leq T}\left|\partial_{i}G_{j}\left(r_{1},r_{2},\mu^{t}\right)\right|\int_{s=0}^{t}\left|\tilde{r}_{i,s}^{\left(1\right)}-\tilde{r}_{i,s}^{\left(2\right)}\right|{\rm d}s\\
 & \qquad+\sum_{j\in\left\{ 1,2\right\} }\int_{s=0}^{t}\sup_{r_{1},r_{2}\geq0}\left|G_{j}\left(r_{1},r_{2},\mu_{1}^{s}\right)-G_{j}\left(r_{1},r_{2},\mu_{2}^{s}\right)\right|{\rm d}s.
\end{align*}
We recall $\left|\partial_{i}G_{j}\left(r_{1},r_{2},\mu^{t}\right)\right|\leq C\left(\sqrt{K}+1\right)$
for $i,j\in\left\{ 1,2\right\} $, $t\in\left[0,T\right]$ and $\mu\in\mathscr{P}\left({\cal C};K\right)$
as shown in the previous step. We also have from Eq. (\ref{eq:2nd_setting_lem_sol_rho_r_1})
and Lemma \ref{lem:2nd_setting_qBounds} that
\begin{align*}
 & \left|G_{1}\left(r_{1},r_{2},\mu_{1}^{s}\right)-G_{1}\left(r_{1},r_{2},\mu_{2}^{s}\right)\right|\\
 & \qquad\leq\mathbb{E}_{\chi}\left\{ \left|\Delta_{1}\left(\chi,\mu_{1}^{s}\right)-\Delta_{1}\left(\chi,\mu_{2}^{s}\right)\right|\right\} +\mathbb{E}_{\chi}\left\{ \left(\chi_{2}/\chi_{1}\right)\left|\Delta_{2}\left(\chi,\mu_{1}^{s}\right)-\Delta_{2}\left(\chi,\mu_{2}^{s}\right)\right|\right\} \\
 & \qquad\leq C\mathbb{E}_{\chi}\left\{ 1+\chi_{2}/\chi_{1}\right\} \mathscr{W}_{s}\left(\mu_{1},\mu_{2}\right)\leq C{\cal \mathscr{W}}_{s}\left(\mu_{1},\mu_{2}\right).
\end{align*}
Similarly $\left|G_{2}\left(r_{1},r_{2},\mu_{1}^{s}\right)-G_{2}\left(r_{1},r_{2},\mu_{2}^{s}\right)\right|\leq C\mathscr{W}_{s}\left(\mu_{1},\mu_{2}\right)$.
Combining these bounds, we then obtain:
\[
\sup_{s\leq t}\sum_{j\in\left\{ 1,2\right\} }\left|\tilde{r}_{j,s}^{\left(1\right)}-\tilde{r}_{j,s}^{\left(2\right)}\right|\leq C\left(\sqrt{K}+1\right)\int_{s=0}^{t}\sum_{i\in\left\{ 1,2\right\} }\left|\tilde{r}_{i,s}^{\left(1\right)}-\tilde{r}_{i,s}^{\left(2\right)}\right|{\rm d}s+C\int_{s=0}^{t}{\cal \mathscr{W}}_{s}\left(\mu_{1},\mu_{2}\right){\rm d}s.
\]
Using Gronwall's lemma:
\begin{equation}
\sup_{s\leq t}\sum_{j\in\left\{ 1,2\right\} }\left|\tilde{r}_{j,s}^{\left(1\right)}-\tilde{r}_{j,s}^{\left(2\right)}\right|\leq Ce^{C\left(\sqrt{K}+1\right)T}\int_{s=0}^{t}\mathscr{W}_{s}\left(\mu_{1},\mu_{2}\right){\rm d}s,\label{eq:2nd_setting_lem_sol_rho_r_3}
\end{equation}
which implies 
\[
\mathscr{W}_{t}\left(\Phi\left(\mu_{1}\right),\Phi\left(\mu_{2}\right)\right)\leq Ce^{C\left(\sqrt{K}+1\right)T}\int_{s=0}^{t}\mathscr{W}_{s}\left(\mu_{1},\mu_{2}\right){\rm d}s.
\]
Iterating this result, we have for $\mu\in\mathscr{P}\left({\cal C};K\right)$:
\[
\mathscr{W}_{T}\left(\Phi^{k}\left(\mu_{1}\right),\Phi^{k}\left(\mu_{2}\right)\right)\leq C_{T,K}^{k}\frac{T^{k}}{k!}{\cal \mathscr{W}}_{T}\left(\mu_{1},\mu_{2}\right),
\]
for any integer $k\geq1$. Since $\mathscr{P}\left({\cal C};K\right)$
is complete, by substituting $\mu_{2}=\Phi\left(\mu_{1}\right)$,
this shows that $\Phi^{k}\left(\mu_{1}\right)$ converges to a limit
point $\mu_{*}\in\mathscr{P}\left({\cal C};K\right)$ as $k\to\infty$.
This limit point $\mu_{*}$ is a fixed point of $\Phi$ and hence
is a solution up to time $T$. The weak uniqueness of this fixed point
also follows easily. In particular, if $\mu_{1}$ and $\mu_{2}$ are
fixed points, then $\Phi^{k}\left(\mu_{1}\right)=\mu_{1}$ and $\Phi^{k}\left(\mu_{2}\right)=\mu_{2}$.
Hence
\[
\mathscr{W}_{T}\left(\mu_{1},\mu_{2}\right)\leq C_{T,K}^{k}\frac{T^{k}}{k!}\mathscr{W}_{T}\left(\mu_{1},\mu_{2}\right),
\]
for arbitrary $k\geq1$. This implies ${\cal \mathscr{W}}_{T}\left(\mu_{1},\mu_{2}\right)=0$.
Since $\mathscr{W}_{T}$ induces the weak topology on $\mathscr{P}\left({\cal C};K\right)$,
weak uniqueness follows. Uniqueness of the solution $\left(\tilde{r}_{1,t},\tilde{r}_{2,t}\right)_{t\in\left[0,T\right]}$
under $\mu_{*}$ is immediate from Eq. (\ref{eq:2nd_setting_lem_sol_rho_r_3}).

We have shown the solution exists (weakly) uniquely for $t\leq T=1/\sqrt{2C_{*}}$
for $C_{*}>0$ independent of the initial law $\tilde{\rho}_{r}^{0}$.
By Eq. (\ref{eq:2nd_setting_lem_sol_rho_r_2}) and the fact $M\left(\tilde{\rho}_{r}^{0}\right)\geq1$,
substituting the choice of $K$ and $T$, we have:
\[
M\left({\rm Law}\left(\tilde{r}_{1,T},\tilde{r}_{2,T}\right)\right)=\max\left(1,\;\mathbb{E}\left\{ \tilde{r}_{1,T}^{2}+\tilde{r}_{2,T}^{2}\right\} \right)\leq CM\left(\tilde{\rho}_{r}^{0}\right)+CKT^{2}=CM\left(\tilde{\rho}_{r}^{0}\right),
\]
which is finite if $M\left(\tilde{\rho}_{r}^{0}\right)$ is finite.
Hence the existence and (weak) uniqueness of the solution can be extended
to $t\in[0,\infty)$. We now return to the original ODE (\ref{eq:2nd_setting_ODE_r}).
Recall that its initial law $\rho_{r}^{0}$ satisfies $M\left(\rho_{r}^{0}\right)\leq C$.
This proves the existence and (weak) uniqueness of the solution of
the ODE (\ref{eq:2nd_setting_ODE_r}) on $t\in[0,\infty)$.

\paragraph*{Step 4: Properties of $\rho_{r}^{t}$.}

The above existence and uniqueness proof only shows that the law solution
lies in $\mathscr{P}\left({\cal C}\left([0,\infty),\mathbb{R}^{2}\right);\infty\right)$.
To derive its properties, we shall appeal to another approach. Consider
the following energy functional:
\[
E\left(\rho\right)=\frac{1}{2}\sum_{j\in\left\{ 1,2\right\} }\mathbb{E}_{\chi}\left\{ \Delta_{j}\left(\chi,\rho\right)^{2}\right\} +\lambda\int\left(\bar{r}_{1}^{2}+\bar{r}_{2}^{2}\right)\rho\left({\rm d}\bar{r}_{1},{\rm d}\bar{r}_{2}\right).
\]
Recall that $\left({\rm d}/{\rm d}t\right)r_{j,t}=G_{j}\left(r_{1,t},r_{2,t},\rho_{r}^{t}\right)$.
We have:
\begin{align*}
\frac{{\rm d}}{{\rm d}t}E\left(\rho_{r}^{t}\right) & =\sum_{j\in\left\{ 1,2\right\} }\mathbb{E}_{\chi}\left\{ \Delta_{j}\left(\chi,\rho_{r}^{t}\right)\int\left[q_{j}\left(\chi_{1}\bar{r}_{1},\chi_{2}\bar{r}_{2}\right)+\chi_{j}\bar{r}_{j}\partial_{j}q_{j}\left(\chi_{1}\bar{r}_{1},\chi_{2}\bar{r}_{2}\right)\right]G_{j}\left(\bar{r}_{1},\bar{r}_{2},\rho_{r}^{t}\right)\rho_{r}^{t}\left({\rm d}\bar{r}_{1},{\rm d}\bar{r}_{2}\right)\right\} \\
 & \qquad+\sum_{j\in\left\{ 1,2\right\} }\mathbb{E}_{\chi}\left\{ \Delta_{j}\left(\chi,\rho_{r}^{t}\right)\int\chi_{\neg j}\bar{r}_{j}\partial_{\neg j}q_{j}\left(\chi_{1}\bar{r}_{1},\chi_{2}\bar{r}_{2}\right)G_{\neg j}\left(\bar{r}_{1},\bar{r}_{2},\rho_{r}^{t}\right)\rho_{r}^{t}\left({\rm d}\bar{r}_{1},{\rm d}\bar{r}_{2}\right)\right\} \\
 & \qquad+2\lambda\sum_{j\in\left\{ 1,2\right\} }\int\bar{r}_{j}G_{j}\left(\bar{r}_{1},\bar{r}_{2},\rho_{r}^{t}\right)\rho_{r}^{t}\left({\rm d}\bar{r}_{1},{\rm d}\bar{r}_{2}\right)\\
 & =-\sum_{j\in\left\{ 1,2\right\} }\int G_{j}\left(\bar{r}_{1},\bar{r}_{2},\rho_{r}^{t}\right)^{2}\rho_{r}^{t}\left({\rm d}\bar{r}_{1},{\rm d}\bar{r}_{2}\right)\leq0.
\end{align*}
That is, $E\left(\rho_{r}^{t}\right)$ is non-increasing with $t\in[0,\infty)$.
Therefore, $E\left(\rho_{r}^{t}\right)\leq E\left(\rho_{r}^{0}\right)$.
Notice that $\int\left(\bar{r}_{1}^{2}+\bar{r}_{2}^{2}\right){\rm d}\rho_{r}^{0}=r_{0}^{2}\leq C$.
By Lemma \ref{lem:2nd_setting_qBounds}, $\left\Vert q_{j}\right\Vert _{\infty}\leq C$
and hence:
\[
\mathbb{E}_{\chi}\left\{ \Delta_{j}\left(\chi,\rho_{r}^{0}\right)^{2}\right\} \leq2\int\bar{r}_{j}^{2}{\rm d}\rho_{r}^{0}+2\mathbb{E}_{\chi}\left\{ \chi_{j}^{2}\right\} \leq C.
\]
These show that $\mathbb{E}_{\chi}\left\{ \Delta_{j}\left(\chi,\rho_{r}^{t}\right)^{2}\right\} \leq C$
for $j=1,2$. Along with Lemma \ref{lem:2nd_setting_qBounds}, we
then have:
\begin{align*}
\left|G_{j}\left(r_{1},r_{2},\rho_{r}^{t}\right)+2\lambda r_{j}\right| & \leq\sqrt{\mathbb{E}_{\chi}\left\{ \Delta_{j}\left(\chi,\rho_{r}^{t}\right)^{2}\right\} \mathbb{E}_{\chi}\left\{ q_{j}\left(\chi_{1}r_{1},\chi_{2}r_{2}\right)^{2}+\left(\chi_{j}r_{j}\partial_{j}q_{j}\left(\chi_{1}r_{1},\chi_{2}r_{2}\right)\right)^{2}\right\} }\\
 & \quad+\sqrt{\mathbb{E}_{\chi}\left\{ \Delta_{\neg j}\left(\chi,\rho_{r}\right)^{2}\right\} \mathbb{E}_{\chi}\left\{ \chi_{j}^{4}/\chi_{\neg j}^{4}\right\} ^{1/2}\mathbb{E}_{\chi}\left\{ \left(\chi_{\neg j}r_{\neg j}\partial_{j}q_{\neg j}\left(\chi_{1}r_{1},\chi_{2}r_{2}\right)\right)^{4}\right\} ^{1/2}}\\
 & \leq C,
\end{align*}
for any $t\geq0$ and any $r_{1},r_{2}\geq0$.

We now bound $\int\bar{r}_{j}^{p}{\rm d}\rho_{r}^{t}$, for $j=1,2$.
Let $\mathbb{E}_{r}$ denote the expectation w.r.t. $\left(r_{1,0},r_{2,0}\right)\sim\rho_{r}^{0}$,
and notice that $\left(r_{1,t},r_{2,t}\right)$ is a deterministic
function of $\left(r_{1,0},r_{2,0}\right)$. We bound the growth of
$r_{j,t}$:
\[
r_{j,t}=r_{j,0}+\int_{s=0}^{t}\left(G_{j}\left(r_{1,s},r_{2,s},\rho_{r}^{s}\right)+2\lambda r_{j,s}\right){\rm d}s-2\lambda\int_{s=0}^{t}r_{j,s}{\rm d}s\leq r_{j,0}+Ct,
\]
since $r_{j,s}\geq0$. This yields:
\[
\int\bar{r}_{j}^{p}{\rm d}\rho_{r}^{t}\leq\mathbb{E}_{r}\left\{ \left(r_{j,0}+Ct\right)^{p}\right\} \leq C^{p}\left(\mathbb{E}_{r}\left\{ r_{j,0}^{p}\right\} +t^{p}\right)\leq C^{p}\left(p^{p/2}+t^{p}\right)\leq C^{p}\left(1+t^{p}\right)p^{p/2},
\]
giving the desired moment bound.

Next we note that with $\left(r_{1,t},r_{2,t}\right)=\psi_{t}\left(r_{1,0},r_{2,0}\right)$,
\begin{align*}
\left\Vert \partial_{t}\psi_{t}\left(r_{1,0},r_{2,0}\right)\right\Vert _{2}^{2} & =\sum_{j\in\left\{ 1,2\right\} }\left|\frac{{\rm d}}{{\rm d}t}r_{j,t}\right|^{2}=\sum_{j\in\left\{ 1,2\right\} }\left|G_{j}\left(r_{1,t},r_{2,t},\rho_{r}^{t}\right)\right|^{2}\leq C\sum_{j\in\left\{ 1,2\right\} }\left(1+r_{j,t}\right)^{2}\\
 & \leq C\left(1+r_{1,t}+r_{2,t}\right)^{2}\leq C\left(1+r_{1,0}+r_{2,0}+t\right)^{2},
\end{align*}
as desired.

\end{proof}
\begin{prop}
\label{prop:2nd_setting_ODE}Consider setting \ref{enu:bdd_act_setting}.
Suppose that the initialization $\rho^{0}=\mathsf{N}\left(\boldsymbol{0},r_{0}^{2}\boldsymbol{I}_{d}/d\right)$
for a non-negative constant $r_{0}\leq C$. Given a random vector
$\hat{\boldsymbol{\theta}}^{0}\sim\rho^{0}$, define the following:
\[
\hat{\boldsymbol{\theta}}^{t}=\left(r_{1,t}\hat{\boldsymbol{\theta}}_{\left[1\right]}^{0}/\left\Vert \hat{\boldsymbol{\theta}}_{\left[1\right]}^{0}\right\Vert _{2},\quad r_{2,t}\hat{\boldsymbol{\theta}}_{\left[2\right]}^{0}/\left\Vert \hat{\boldsymbol{\theta}}_{\left[2\right]}^{0}\right\Vert _{2}\right),\qquad\rho^{t}={\rm Law}\left(\hat{\boldsymbol{\theta}}^{t}\right),
\]
in which $\left(r_{1,t}\right)_{t\geq0}$ and $\left(r_{2,t}\right)_{t\geq0}$
are two non-negative (random) processes, which are independent of
$\hat{\boldsymbol{\theta}}_{\left[1\right]}^{0}/\left\Vert \hat{\boldsymbol{\theta}}_{\left[1\right]}^{0}\right\Vert _{2}$
and $\hat{\boldsymbol{\theta}}_{\left[2\right]}^{0}/\left\Vert \hat{\boldsymbol{\theta}}_{\left[2\right]}^{0}\right\Vert _{2}$,
that are described as in the statement of Theorem \ref{thm:bdd_act_setting}.
Then the ODE (\ref{eq:ODE}) admits $\left(\hat{\boldsymbol{\theta}}^{t},\rho^{t}\right)_{t\geq0}$
as a solution. In fact, $\left(\rho^{t}\right)_{t\geq0}$ is the unique
weak solution, and under $\left(\rho^{t}\right)_{t\geq0}$, $\left(\hat{\boldsymbol{\theta}}^{t}\right)_{t\geq0}$
is the unique solution to (\ref{eq:ODE}).
\end{prop}

\begin{proof}
We decompose the proof into several parts. In the following, we let
$c_{t}$ to be an immaterial positive constant, which may differ at
different instances of use, may depend on time $t$ and $\mathfrak{Dim}$,
and is finite with finite $t$. We shall also reuse several quantities
in the description of $\left(r_{1,t},r_{2,t}\right)_{t\geq0}$ from
the statement of Theorem \ref{thm:bdd_act_setting}. By Proposition
\ref{prop:2nd_setting_ODE_r}, the process $\left(r_{1,t},r_{2,t},\rho_{r}^{t}\right)_{t\geq0}$
exists and is (weakly) unique. Without loss of generality, let us
assume $r_{1,0}=\left\Vert \hat{\boldsymbol{\theta}}_{\left[1\right]}^{0}\right\Vert _{2}$
and $r_{2,0}=\left\Vert \hat{\boldsymbol{\theta}}_{\left[2\right]}^{0}\right\Vert _{2}$.

\paragraph*{Verification of the proposed solution.}

We first check that the constructed $\left(\hat{\boldsymbol{\theta}}^{t},\rho^{t}\right)_{t\geq0}$
is a solution of the ODE (\ref{eq:ODE}). For brevity, let $\boldsymbol{u}_{\left[j\right]}^{t}=\hat{\boldsymbol{\theta}}_{\left[j\right]}^{t}/\left\Vert \hat{\boldsymbol{\theta}}_{\left[j\right]}^{t}\right\Vert _{2}$,
for $j=1,2$. Firstly since $\rho^{0}=\mathsf{N}\left(\boldsymbol{0},r_{0}^{2}\boldsymbol{I}_{d}/d\right)$,
we have $r_{1,0}$, $r_{2,0}$, $\boldsymbol{u}_{\left[1\right]}^{0}$
and $\boldsymbol{u}_{\left[2\right]}^{0}$ are mutually independent.
Furthermore, $\boldsymbol{u}_{\left[1\right]}^{t}=\boldsymbol{u}_{\left[1\right]}^{0}$
and $\boldsymbol{u}_{\left[2\right]}^{t}=\boldsymbol{u}_{\left[2\right]}^{0}$
for all $t\geq0$. It is then easy to see from the dynamics of $r_{1,t}$
and $r_{2,t}$ that $\left(r_{1,t},r_{2,t}\right)_{t\geq0}$, $\left(\boldsymbol{u}_{\left[1\right]}^{t}\right)_{t\geq0}$
and $\left(\boldsymbol{u}_{\left[2\right]}^{t}\right)_{t\geq0}$ are
mutually independent. Note that $\boldsymbol{u}_{\left[1\right]}^{0}\stackrel{{\rm d}}{=}\boldsymbol{\omega}_{1}$
and $\boldsymbol{u}_{\left[2\right]}^{0}\stackrel{{\rm d}}{=}\boldsymbol{\omega}_{2}$
(where we recall $\boldsymbol{\omega}_{1}\sim\text{Unif}\left(\mathbb{S}^{d_{1}-1}\right)$
and $\boldsymbol{\omega}_{2}\sim\text{Unif}\left(\mathbb{S}^{d_{2}-1}\right)$
independently), and $r_{1,t}=\left\Vert \hat{\boldsymbol{\theta}}_{\left[1\right]}^{t}\right\Vert _{2}$,
$r_{2,t}=\left\Vert \hat{\boldsymbol{\theta}}_{\left[2\right]}^{t}\right\Vert _{2}$.
Using these facts, performing a calculation similar to the proof of
Proposition \ref{prop:2nd_setting_grow_bound} (in particular, using
Eq. (\ref{eq:2nd_setting_grow_bound_SteinEq})), we arrive at the
following:
\begin{align*}
\nabla_{1}W\left(\hat{\boldsymbol{\theta}}^{t};\rho^{t}\right) & =\left(\nabla_{1}W\left(\hat{\boldsymbol{\theta}}^{t};\rho^{t}\right)_{\left[1\right]},\quad\nabla_{1}W\left(\hat{\boldsymbol{\theta}}^{t};\rho^{t}\right)_{\left[2\right]}\right),\\
\nabla_{1}W\left(\hat{\boldsymbol{\theta}}^{t};\rho^{t}\right)_{\left[j\right]} & =\boldsymbol{u}_{\left[j\right]}^{0}\int\bar{r}_{j}\mathbb{E}_{\chi}\left\{ \bar{q}_{j}q_{j}^{t}\right\} \rho_{r}^{t}\left({\rm d}\bar{r}_{1},{\rm d}\bar{r}_{2}\right)+\boldsymbol{u}_{\left[j\right]}^{0}r_{j,t}\int\bar{r}_{j}\mathbb{E}_{\chi}\left\{ \chi_{j}\bar{q}_{j}\partial_{j}q_{j}^{t}\right\} \rho_{r}^{t}\left({\rm d}\bar{r}_{1},{\rm d}\bar{r}_{2}\right)\\
 & \quad+\boldsymbol{u}_{\left[j\right]}^{0}r_{\neg j,t}\int\bar{r}_{\neg j}\mathbb{E}_{\chi}\left\{ \chi_{j}\bar{q}_{\neg j}\partial_{j}q_{\neg j}^{t}\right\} \rho_{r}^{t}\left({\rm d}\bar{r}_{1},{\rm d}\bar{r}_{2}\right),\qquad j=1,2,
\end{align*}
Here we have introduced several shortening notations, for $i,j\in\left\{ 1,2\right\} $:
\[
\bar{q}_{j}=\bar{q}_{j}\left(\chi_{1}\bar{r}_{1},\chi_{2}\bar{r}_{2}\right),\qquad q_{j}^{t}=q_{j}\left(\chi_{1}r_{1,t},\chi_{2}r_{2,t}\right),\qquad\partial_{i}q_{j}^{t}=\partial_{i}q_{j}\left(\chi_{1}r_{1,t},\chi_{2}r_{2,t}\right).
\]
Next we derive a compatible form of $\nabla V\left(\boldsymbol{\theta}\right)$.
Notice that $\boldsymbol{x}\stackrel{{\rm d}}{=}\left(\chi_{1}\boldsymbol{\omega}_{1},\chi_{2}\boldsymbol{\omega}_{2}\right)$
where $\boldsymbol{\omega}_{1}$, $\boldsymbol{\omega}_{2}$, $\chi_{1}$
and $\chi_{2}$ are mutually independent. Therefore,
\begin{align*}
V\left(\boldsymbol{\theta}\right) & =\mathbb{E}_{{\cal P}}\left\{ -\left\langle \kappa\boldsymbol{\theta},\boldsymbol{x}\right\rangle \sigma\left(\left\langle \kappa\boldsymbol{\theta},\boldsymbol{x}\right\rangle \right)\right\} +\lambda\left\Vert \boldsymbol{\theta}\right\Vert _{2}^{2}\\
 & =-\mathbb{E}_{\chi,\boldsymbol{\omega}}\left\{ \left(\sum_{j\in\left\{ 1,2\right\} }\kappa\chi_{j}\left\langle \boldsymbol{\theta}_{\left[j\right]},\boldsymbol{\omega}_{j}\right\rangle \right)\sigma\left(\sum_{j\in\left\{ 1,2\right\} }\kappa\chi_{j}\left\langle \boldsymbol{\theta}_{\left[j\right]},\boldsymbol{\omega}_{j}\right\rangle \right)\right\} +\lambda\sum_{j\in\left\{ 1,2\right\} }\left\Vert \boldsymbol{\theta}_{\left[j\right]}\right\Vert _{2}^{2}\\
 & =-\mathbb{E}_{\chi}\left\{ \sum_{j\in\left\{ 1,2\right\} }\chi_{j}\left\Vert \boldsymbol{\theta}_{\left[j\right]}\right\Vert _{2}q_{j}\left(\chi_{1}\left\Vert \boldsymbol{\theta}_{\left[1\right]}\right\Vert _{2},\chi_{2}\left\Vert \boldsymbol{\theta}_{\left[2\right]}\right\Vert _{2}\right)\right\} +\lambda\sum_{j\in\left\{ 1,2\right\} }\left\Vert \boldsymbol{\theta}_{\left[j\right]}\right\Vert _{2}^{2},
\end{align*}
where in the last step, we have performed a calculation similar to
the proof of Proposition \ref{prop:2nd_setting_grow_bound} (in particular,
we use Eq. (\ref{eq:2nd_setting_grow_bound_SteinEq})). This yields:
\begin{align*}
\nabla V\left(\hat{\boldsymbol{\theta}}^{t}\right) & =\left(\nabla V\left(\hat{\boldsymbol{\theta}}^{t}\right)_{\left[1\right]},\quad\nabla V\left(\hat{\boldsymbol{\theta}}^{t}\right)_{\left[2\right]}\right),\\
\nabla V\left(\hat{\boldsymbol{\theta}}^{t}\right)_{\left[j\right]} & =-\mathbb{E}_{\chi}\left\{ \chi_{j}q_{j}^{t}+\chi_{j}^{2}r_{j,t}\partial_{j}q_{j}^{t}+\chi_{j}\chi_{\neg j}r_{\neg j,t}\partial_{j}q_{\neg j}^{t}\right\} \boldsymbol{u}_{\left[j\right]}^{0}+2\lambda r_{j,t}\boldsymbol{u}_{\left[j\right]}^{0},\qquad j=1,2.
\end{align*}
It is then easy to see that:
\[
\nabla V\left(\hat{\boldsymbol{\theta}}^{t}\right)_{\left[j\right]}+\nabla_{1}W\left(\hat{\boldsymbol{\theta}}^{t};\rho^{t}\right)_{\left[j\right]}=-\boldsymbol{u}_{\left[j\right]}^{0}\frac{{\rm d}}{{\rm d}t}r_{j,t}=-\frac{{\rm d}}{{\rm d}t}\hat{\boldsymbol{\theta}}_{\left[j\right]}^{t},\qquad j=1,2.
\]
Therefore $\left(\hat{\boldsymbol{\theta}}^{t},\rho^{t}\right)_{t\geq0}$
is a solution of the ODE (\ref{eq:ODE}).

\paragraph*{Trajectorial uniqueness.}

Next we prove that under the given path $\left(\rho^{t}\right)_{t\geq0}$,
the process $\left(\hat{\boldsymbol{\theta}}^{t}\right)_{t\geq0}$
is the unique trajectorial solution to the ODE (\ref{eq:ODE}) with
initialization $\hat{\boldsymbol{\theta}}^{0}$. By Proposition \ref{prop:2nd_setting_ODE_r},
we have $\int\left(\bar{r}_{1}+\bar{r}_{2}\right)\rho_{r}^{t}\left({\rm d}\bar{r}_{1},{\rm d}\bar{r}_{2}\right)\leq c_{t}$,
and hence by Proposition \ref{prop:2nd_setting_grow_bound}, $\nabla V$
and $\nabla_{1}W\left(\cdot;\rho^{t}\right)$ are both $c_{t}$-Lipschitz.
A standard argument then yields the desired uniqueness.

\paragraph*{Uniqueness in law.}

We now prove that $\left(\rho^{t}\right)_{t\geq0}$ is the unique
weak solution with the initialization $\rho^{0}$. Let $\left(\bar{\rho}^{t}\right)_{t\geq0}$
be another solution with the same initialization $\bar{\rho}^{0}=\rho^{0}$
(with the equalities holding in the weak sense). We define accordingly
two coupled trajectories $\left(\boldsymbol{\theta}^{t}\right)_{t\geq0}$
and $\left(\bar{\boldsymbol{\theta}}^{t}\right)_{t\geq0}$ with the
same initialization $\boldsymbol{\theta}^{0}=\bar{\boldsymbol{\theta}}^{0}\sim\rho^{0}$:
\begin{align*}
\frac{{\rm d}}{{\rm d}t}\boldsymbol{\theta}^{t} & =-\nabla V\left(\boldsymbol{\theta}^{t}\right)-\nabla_{1}W\left(\boldsymbol{\theta}^{t};\rho^{t}\right),\qquad\rho^{t}={\rm Law}\left(\boldsymbol{\theta}^{t}\right),\\
\frac{{\rm d}}{{\rm d}t}\bar{\boldsymbol{\theta}}^{t} & =-\nabla V\left(\bar{\boldsymbol{\theta}}^{t}\right)-\nabla_{1}W\left(\bar{\boldsymbol{\theta}}^{t};\bar{\rho}^{t}\right),\qquad\bar{\rho}^{t}={\rm Law}\left(\bar{\boldsymbol{\theta}}^{t}\right).
\end{align*}
We examine the distance between these two trajectories:
\begin{align*}
\frac{{\rm d}}{{\rm d}t}\left\Vert \boldsymbol{\theta}^{t}-\bar{\boldsymbol{\theta}}^{t}\right\Vert _{2} & \leq\left\Vert \nabla V\left(\boldsymbol{\theta}^{t}\right)-\nabla V\left(\bar{\boldsymbol{\theta}}^{t}\right)\right\Vert _{2}+\left\Vert \nabla_{1}W\left(\boldsymbol{\theta}^{t};\rho^{t}\right)-\nabla_{1}W\left(\bar{\boldsymbol{\theta}}^{t};\rho^{t}\right)\right\Vert _{2}\\
 & \qquad+\left\Vert \nabla_{1}W\left(\bar{\boldsymbol{\theta}}^{t};\rho^{t}\right)-\nabla_{1}W\left(\bar{\boldsymbol{\theta}}^{t};\bar{\rho}^{t}\right)\right\Vert _{2}.
\end{align*}
Define $M_{t}=\mathbb{E}_{\boldsymbol{\theta}}\left\{ \left\Vert \boldsymbol{\theta}^{t}-\bar{\boldsymbol{\theta}}^{t}\right\Vert _{2}^{2}\right\} $,
and note that $M_{0}=0$. By Propositions \ref{prop:2nd_setting_grow_bound},
\ref{prop:2nd_setting_op_bound} and \ref{prop:2nd_setting_ODE_r},
along with the mean value theorem,
\begin{align*}
\left\Vert \nabla V\left(\boldsymbol{\theta}^{t}\right)-\nabla V\left(\bar{\boldsymbol{\theta}}^{t}\right)\right\Vert _{2} & \leq C\left\Vert \boldsymbol{\theta}^{t}-\bar{\boldsymbol{\theta}}^{t}\right\Vert _{2},\\
\left\Vert \nabla_{1}W\left(\boldsymbol{\theta}^{t};\rho^{t}\right)-\nabla_{1}W\left(\bar{\boldsymbol{\theta}}^{t};\rho^{t}\right)\right\Vert _{2} & \leq\int\left\Vert \nabla_{1}U\left(\boldsymbol{\theta}^{t},\boldsymbol{\theta}\right)-\nabla_{1}U\left(\bar{\boldsymbol{\theta}}^{t},\boldsymbol{\theta}\right)\right\Vert _{2}\rho^{t}\left({\rm d}\boldsymbol{\theta}\right)\\
 & \stackrel{\left(a\right)}{\leq}\int\left\Vert \nabla_{11}^{2}U\left(\boldsymbol{\zeta}_{1},\boldsymbol{\theta}\right)\right\Vert _{{\rm op}}\left\Vert \boldsymbol{\theta}^{t}-\bar{\boldsymbol{\theta}}^{t}\right\Vert _{2}\rho^{t}\left({\rm d}\boldsymbol{\theta}\right)\\
 & \leq c_{t}\left\Vert \boldsymbol{\theta}^{t}-\bar{\boldsymbol{\theta}}^{t}\right\Vert _{2}\int\left\Vert \boldsymbol{\theta}\right\Vert _{2}\rho^{t}\left({\rm d}\boldsymbol{\theta}\right)\\
 & \leq c_{t}\left\Vert \boldsymbol{\theta}^{t}-\bar{\boldsymbol{\theta}}^{t}\right\Vert _{2}\int\left(\bar{r}_{1}+\bar{r}_{2}\right)\rho_{r}^{t}\left({\rm d}\bar{r}_{1},{\rm d}\bar{r}_{2}\right)\\
 & \leq c_{t}\left\Vert \boldsymbol{\theta}^{t}-\bar{\boldsymbol{\theta}}^{t}\right\Vert _{2},\\
\left\Vert \nabla_{1}W\left(\bar{\boldsymbol{\theta}}^{t};\rho^{t}\right)-\nabla_{1}W\left(\bar{\boldsymbol{\theta}}^{t};\bar{\rho}^{t}\right)\right\Vert _{2} & \stackrel{\left(b\right)}{=}\left\Vert \mathbb{E}_{\tilde{\boldsymbol{\theta}}}\left\{ \nabla_{1}U\left(\bar{\boldsymbol{\theta}}^{t},\tilde{\boldsymbol{\theta}}_{2}\right)-\nabla_{1}U\left(\bar{\boldsymbol{\theta}}^{t},\tilde{\boldsymbol{\theta}}_{1}\right)\right\} \right\Vert _{2}\\
 & \stackrel{\left(c\right)}{\leq}\mathbb{E}_{\tilde{\boldsymbol{\theta}}}\left\{ \left\Vert \nabla_{12}^{2}U\left(\bar{\boldsymbol{\theta}}^{t},\boldsymbol{\zeta}_{2}\right)\right\Vert _{{\rm op}}\left\Vert \tilde{\boldsymbol{\theta}}_{2}-\tilde{\boldsymbol{\theta}}_{1}\right\Vert _{2}\right\} \\
 & \leq c_{t}\left(1+\left\Vert \bar{\boldsymbol{\theta}}^{t}\right\Vert _{2}\right)\mathbb{E}_{\tilde{\boldsymbol{\theta}}}\left\{ \left(1+\left\Vert \boldsymbol{\zeta}_{2}\right\Vert _{2}\right)\left\Vert \tilde{\boldsymbol{\theta}}_{2}-\tilde{\boldsymbol{\theta}}_{1}\right\Vert _{2}\right\} \\
 & \leq c_{t}\left(1+\left\Vert \bar{\boldsymbol{\theta}}^{t}\right\Vert _{2}\right)\mathbb{E}_{\tilde{\boldsymbol{\theta}}}\left\{ \left(1+\left\Vert \tilde{\boldsymbol{\theta}}_{1}\right\Vert _{2}\right)\left\Vert \tilde{\boldsymbol{\theta}}_{2}-\tilde{\boldsymbol{\theta}}_{1}\right\Vert _{2}+\left\Vert \tilde{\boldsymbol{\theta}}_{2}-\tilde{\boldsymbol{\theta}}_{1}\right\Vert _{2}^{2}\right\} \\
 & \leq c_{t}\left(1+\left\Vert \bar{\boldsymbol{\theta}}^{t}\right\Vert _{2}\right)\left(\sqrt{\mathbb{E}_{\tilde{\boldsymbol{\theta}}}\left\{ 1+\left\Vert \tilde{\boldsymbol{\theta}}_{1}\right\Vert _{2}^{2}\right\} \mathbb{E}_{\tilde{\boldsymbol{\theta}}}\left\{ \left\Vert \tilde{\boldsymbol{\theta}}_{2}-\tilde{\boldsymbol{\theta}}_{1}\right\Vert _{2}^{2}\right\} }+M_{t}\right)\\
 & =c_{t}\left(1+\left\Vert \bar{\boldsymbol{\theta}}^{t}\right\Vert _{2}\right)\left(\sqrt{\left(1+\int\left(\bar{r}_{1}^{2}+\bar{r}_{2}^{2}\right)\rho_{r}^{t}\left({\rm d}\bar{r}_{1},{\rm d}\bar{r}_{2}\right)\right)M_{t}}+M_{t}\right)\\
 & \leq c_{t}\left(1+\left\Vert \bar{\boldsymbol{\theta}}^{t}\right\Vert _{2}\right)\left(\sqrt{M_{t}}+M_{t}\right),
\end{align*}
where in step $\left(a\right)$, $\boldsymbol{\zeta}_{1}\in\left[\boldsymbol{\theta}_{1}^{t},\boldsymbol{\theta}_{2}^{t}\right]$;
in step $\left(b\right)$, we define $\left(\tilde{\boldsymbol{\theta}}_{1},\tilde{\boldsymbol{\theta}}_{2}\right)\stackrel{{\rm d}}{=}\left(\boldsymbol{\theta}^{t},\bar{\boldsymbol{\theta}}^{t}\right)$
and $\left(\tilde{\boldsymbol{\theta}}_{1},\tilde{\boldsymbol{\theta}}_{2}\right)$
is independent of $\left(\boldsymbol{\theta}_{1}^{t},\boldsymbol{\theta}_{2}^{t}\right)$;
in step $\left(c\right)$, $\boldsymbol{\zeta}_{2}\in\left[\tilde{\boldsymbol{\theta}}_{1},\tilde{\boldsymbol{\theta}}_{2}\right]$
and hence $\left\Vert \boldsymbol{\zeta}_{2}\right\Vert _{2}\leq\left\Vert \tilde{\boldsymbol{\theta}}_{1}\right\Vert _{2}+\left\Vert \tilde{\boldsymbol{\theta}}_{2}-\tilde{\boldsymbol{\theta}}_{1}\right\Vert _{2}$.
These bounds imply that
\begin{align*}
\frac{{\rm d}}{{\rm d}t}\left\Vert \boldsymbol{\theta}_{1}^{t}-\boldsymbol{\theta}_{2}^{t}\right\Vert _{2}^{2} & \leq c_{t}\left\Vert \boldsymbol{\theta}^{t}-\bar{\boldsymbol{\theta}}^{t}\right\Vert _{2}^{2}+c_{t}\left(1+\left\Vert \bar{\boldsymbol{\theta}}^{t}\right\Vert _{2}\right)\left\Vert \boldsymbol{\theta}^{t}-\bar{\boldsymbol{\theta}}^{t}\right\Vert _{2}\left(\sqrt{M_{t}}+M_{t}\right)\\
 & \leq c_{t}\left\Vert \boldsymbol{\theta}^{t}-\bar{\boldsymbol{\theta}}^{t}\right\Vert _{2}^{2}+c_{t}\left(1+\left\Vert \boldsymbol{\theta}^{t}\right\Vert _{2}+\left\Vert \boldsymbol{\theta}^{t}-\bar{\boldsymbol{\theta}}^{t}\right\Vert _{2}\right)\left\Vert \boldsymbol{\theta}^{t}-\bar{\boldsymbol{\theta}}^{t}\right\Vert _{2}\left(\sqrt{M_{t}}+M_{t}\right).
\end{align*}
Taking expectation, by Proposition \ref{prop:2nd_setting_ODE_r},
we obtain that for any $T\geq0$,
\begin{align*}
\frac{{\rm d}}{{\rm d}t}M_{t} & \leq c_{t}M_{t}+c_{t}\sqrt{\left(1+\int\left\Vert \boldsymbol{\theta}\right\Vert _{2}^{2}\rho^{t}\left({\rm d}\boldsymbol{\theta}\right)+M_{t}\right)M_{t}}\left(\sqrt{M_{t}}+M_{t}\right)\\
 & =c_{t}M_{t}+c_{t}\sqrt{\left(1+\int\left(\bar{r}_{1}^{2}+\bar{r}_{2}^{2}\right)\rho_{r}^{t}\left({\rm d}\bar{r}_{1},{\rm d}\bar{r}_{2}\right)+M_{t}\right)M_{t}}\left(\sqrt{M_{t}}+M_{t}\right)\\
 & \leq c_{t}M_{t}+c_{t}\sqrt{\left(1+M_{t}\right)M_{t}}\left(\sqrt{M_{t}}+M_{t}\right)\\
 & \leq c_{T}M_{t}
\end{align*}
for $t\leq T$ and $t<t_{*}$ with $t_{*}=\inf\left\{ t\geq0:\;M_{t}>1\right\} $.
Since $M_{0}=0$ and $M_{t}\geq0$, Gronwall's lemma then implies
that $t_{*}>T$ and $M_{t}=0$ for all $t\leq T$. Since this is satisfied
for any $T\geq0$, we have $M_{t}=0$ for all $t\geq0$. Note that
$M_{t}=0$ implies, for any $1$-Lipschitz test function $\phi:\;\mathbb{R}^{d}\to\mathbb{R}$,
\[
\left|\int\phi\left(\boldsymbol{\theta}\right)\bar{\rho}^{t}\left({\rm d}\boldsymbol{\theta}\right)-\int\phi\left(\boldsymbol{\theta}\right)\rho^{t}\left({\rm d}\boldsymbol{\theta}\right)\right|\leq\inf_{\boldsymbol{\theta}_{a}\sim\bar{\rho}^{t},\;\boldsymbol{\theta}_{b}\sim\rho^{t}}\mathbb{E}\left\{ \left\Vert \boldsymbol{\theta}_{a}-\boldsymbol{\theta}_{b}\right\Vert _{2}\right\} \leq\mathbb{E}_{\boldsymbol{\theta}}\left\{ \left\Vert \boldsymbol{\theta}^{t}-\bar{\boldsymbol{\theta}}^{t}\right\Vert _{2}\right\} \leq\sqrt{M_{t}}=0.
\]
This proves weak uniqueness of the solution $\left(\rho^{t}\right)_{t\geq0}$
with initialization $\rho^{0}$.

\end{proof}
\begin{prop}
\label{prop:2nd_setting_F_bound}Consider setting \ref{enu:bdd_act_setting}.
For a collection of vectors $\Theta=\left(\boldsymbol{\theta}_{i}\right)_{i\leq N}$
where $\boldsymbol{\theta}_{i}\in\mathbb{R}^{d}$, $\boldsymbol{x}\sim{\cal P}$
and $\boldsymbol{z}=\left(\boldsymbol{x},\boldsymbol{x}\right)$,
we have $\boldsymbol{F}_{i}\left(\Theta;\boldsymbol{z}\right)$ is
sub-exponential with $\psi_{1}$-norm: 
\[
\left\Vert \boldsymbol{F}_{i}\left(\Theta;\boldsymbol{z}\right)\right\Vert _{\psi_{1}}\leq C\kappa^{2}\left(\left\Vert \boldsymbol{\theta}_{i}\right\Vert _{2}+1\right)\left(\sqrt{\frac{1}{N}\sum_{j=1}^{N}\left\Vert \boldsymbol{\theta}_{j}\right\Vert _{2}^{2}}+1\right).
\]
\end{prop}

\begin{proof}
Consider a fixed vector $\boldsymbol{v}\in\mathbb{S}^{d-1}$:
\begin{align*}
\left\langle \boldsymbol{v},\boldsymbol{F}_{i}\left(\Theta;\boldsymbol{z}\right)\right\rangle  & =\kappa\left\langle \boldsymbol{v},\nabla_{2}\sigma_{*}\left(\boldsymbol{x};\kappa\boldsymbol{\theta}_{i}\right)^{\top}\left(\hat{\boldsymbol{y}}_{N}\left(\boldsymbol{x};\Theta\right)-\boldsymbol{x}\right)\right\rangle +\lambda\left\langle \boldsymbol{v},\nabla_{1}\Lambda\left(\boldsymbol{\theta}_{i},\boldsymbol{z}\right)\right\rangle \\
 & =\kappa\sigma\left(\left\langle \kappa\boldsymbol{\theta}_{i},\boldsymbol{x}\right\rangle \right)\left(\left\langle \boldsymbol{v},\hat{\boldsymbol{x}}\right\rangle -\left\langle \boldsymbol{v},\boldsymbol{x}\right\rangle \right)+\kappa^{2}\sigma'\left(\left\langle \kappa\boldsymbol{\theta}_{i},\boldsymbol{x}\right\rangle \right)\left(\left\langle \boldsymbol{\theta}_{i},\hat{\boldsymbol{x}}\right\rangle -\left\langle \boldsymbol{\theta}_{i},\boldsymbol{x}\right\rangle \right)\left\langle \boldsymbol{v},\boldsymbol{x}\right\rangle +2\lambda\left\langle \boldsymbol{v},\boldsymbol{\theta}_{i}\right\rangle \\
 & \equiv A_{1}+A_{2}+A_{3},
\end{align*}
where we denote $\hat{\boldsymbol{x}}=\left(1/N\right)\cdot\sum_{j=1}^{N}\kappa\boldsymbol{\theta}_{j}\sigma\left(\left\langle \kappa\boldsymbol{\theta}_{j},\boldsymbol{x}\right\rangle \right)$
for brevity. We examine each component in the above:
\begin{itemize}
\item Since $\left\Vert \sigma\right\Vert _{\infty}\leq C$, for any $\boldsymbol{u}\in\mathbb{R}^{d}$,
$\left\langle \boldsymbol{u},\hat{\boldsymbol{x}}\right\rangle $
is sub-Gaussian with $\psi_{2}$-norm 
\[
\left\Vert \left\langle \boldsymbol{u},\hat{\boldsymbol{x}}\right\rangle \right\Vert _{\psi_{2}}\leq C\frac{\kappa}{N}\sum_{j=1}^{N}\left|\left\langle \boldsymbol{u},\boldsymbol{\theta}_{j}\right\rangle \right|\leq C\kappa\left\Vert \boldsymbol{u}\right\Vert _{2}\frac{1}{N}\sum_{j=1}^{N}\left\Vert \boldsymbol{\theta}_{j}\right\Vert _{2}.
\]
We have $\left\langle \kappa\boldsymbol{u},\boldsymbol{x}\right\rangle $
is sub-Gaussian with $\psi_{2}$-norm $\left\Vert \left\langle \kappa\boldsymbol{u},\boldsymbol{x}\right\rangle \right\Vert _{\psi_{2}}=\left\Vert \boldsymbol{\Sigma}\boldsymbol{u}\right\Vert _{2}\leq C\left\Vert \boldsymbol{u}\right\Vert _{2}$.
Therefore, $A_{1}$ is sub-Gaussian:
\[
\left\Vert A_{1}\right\Vert _{\psi_{2}}\leq C\kappa\left(\left\Vert \left\langle \boldsymbol{v},\hat{\boldsymbol{x}}\right\rangle \right\Vert _{\psi_{2}}+\left\Vert \left\langle \boldsymbol{v},\boldsymbol{x}\right\rangle \right\Vert _{\psi_{2}}\right)\leq C\kappa\left(\frac{1}{N}\sum_{j=1}^{N}\left\Vert \boldsymbol{\theta}_{j}\right\Vert _{2}+1\right).
\]
\item Since $\left\Vert \sigma'\right\Vert _{\infty}\leq C$, $A_{2}$ is
sub-exponential:
\begin{align*}
\left\Vert A_{2}\right\Vert _{\psi_{1}} & \leq C\kappa\left(\left\Vert \left\langle \boldsymbol{\theta}_{i},\hat{\boldsymbol{x}}\right\rangle \right\Vert _{\psi_{2}}+\left\Vert \left\langle \boldsymbol{\theta}_{i},\boldsymbol{x}\right\rangle \right\Vert _{\psi_{2}}\right)\left\Vert \left\langle \kappa\boldsymbol{v},\boldsymbol{x}\right\rangle \right\Vert _{\psi_{2}}\\
 & \leq C\kappa\left(\kappa\left\Vert \boldsymbol{\theta}_{i}\right\Vert _{2}\frac{1}{N}\sum_{j=1}^{N}\left\Vert \boldsymbol{\theta}_{j}\right\Vert _{2}+\frac{1}{\kappa}\left\Vert \boldsymbol{\theta}_{i}\right\Vert _{2}\right)\leq C\kappa^{2}\left\Vert \boldsymbol{\theta}_{i}\right\Vert _{2}\left(\frac{1}{N}\sum_{j=1}^{N}\left\Vert \boldsymbol{\theta}_{j}\right\Vert _{2}+1\right).
\end{align*}
\item $A_{3}$ is a constant and so it is sub-exponential with $\psi_{1}$-norm
$\left\Vert A_{3}\right\Vert _{\psi_{1}}\leq C\left\Vert \boldsymbol{\theta}_{i}\right\Vert _{2}$.
\end{itemize}
We have $\left\langle \boldsymbol{v},\boldsymbol{F}_{i}\left(\Theta;\boldsymbol{z}\right)\right\rangle $
and hence $\boldsymbol{F}_{i}\left(\Theta;\boldsymbol{z}\right)$
are sub-exponential:
\begin{align*}
\left\Vert \boldsymbol{F}_{i}\left(\Theta;\boldsymbol{z}\right)\right\Vert _{\psi_{1}} & =\sup_{\boldsymbol{v}\in\mathbb{S}^{d-1}}\left\Vert \left\langle \boldsymbol{v},\boldsymbol{F}_{i}\left(\Theta;\boldsymbol{z}\right)\right\rangle \right\Vert _{\psi_{1}}\leq C\kappa^{2}\left(\left\Vert \boldsymbol{\theta}_{i}\right\Vert _{2}+1\right)\left(\frac{1}{N}\sum_{j=1}^{N}\left\Vert \boldsymbol{\theta}_{j}\right\Vert _{2}+1\right)\\
 & \leq C\kappa^{2}\left(\left\Vert \boldsymbol{\theta}_{i}\right\Vert _{2}+1\right)\left(\sqrt{\frac{1}{N}\sum_{j=1}^{N}\left\Vert \boldsymbol{\theta}_{j}\right\Vert _{2}^{2}}+1\right).
\end{align*}
This completes the proof.
\end{proof}
\begin{lem}
\label{lem:Unorm_2ndSetting}Consider setting \ref{enu:bdd_act_setting}.
Let $\rho={\rm Law}\left(r_{1}\boldsymbol{\omega}_{1},r_{2}\boldsymbol{\omega}_{2}\right)$
in which $\left(r_{1},r_{2}\right)$, $\boldsymbol{\omega}_{1}$ and
$\boldsymbol{\omega}_{2}$ are mutually independent and $\left(r_{1},r_{2}\right)\sim\rho_{r}$
such that $r_{1}$ and $r_{2}$ are non-negative and marginally $C$-sub-Gaussian.
We have, for some sufficiently large $C_{*}$, with probability at
least $1-C\exp\left(Cd-CN/\kappa^{4}\right)$,
\[
\left\Vert \frac{1}{N}\sum_{i=1}^{N}\nabla_{11}^{2}U\left(\boldsymbol{\zeta},\boldsymbol{\theta}_{i}\right)\right\Vert _{{\rm op}}\leq C_{*},
\]
in which $\boldsymbol{\zeta}$ is a fixed vector with $\left\Vert \boldsymbol{\zeta}\right\Vert _{2}<\infty$,
and $\left(\boldsymbol{\theta}_{i}\right)_{i\leq N}\sim_{{\rm i.i.d.}}\rho$.
Here $C_{*}$ does not depend on $d$ or $N$.
\end{lem}

\begin{proof}
We proceed in a fashion similar to the proof of Lemma \ref{lem:Unorm_firstSetting}.
Let us decompose
\[
\frac{1}{N}\sum_{i=1}^{N}\nabla_{11}^{2}U\left(\boldsymbol{\zeta},\boldsymbol{\theta}_{i}\right)=\boldsymbol{M}_{1}+\boldsymbol{M}_{1}^{\top}+\boldsymbol{M}_{2}\in\mathbb{R}^{d\times d},
\]
for which
\begin{align*}
\boldsymbol{M}_{1} & =\frac{1}{N}\sum_{i=1}^{N}\kappa^{3}\mathbb{E}_{{\cal P}}\left\{ \sigma'\left(\left\langle \kappa\boldsymbol{\zeta},\boldsymbol{x}\right\rangle \right)\sigma\left(\left\langle \kappa\boldsymbol{\theta}_{i},\boldsymbol{x}\right\rangle \right)\boldsymbol{\theta}_{i}\boldsymbol{x}^{\top}\right\} ,\\
\boldsymbol{M}_{2} & =\frac{1}{N}\sum_{i=1}^{N}\kappa^{4}\mathbb{E}_{{\cal P}}\left\{ \left\langle \boldsymbol{\zeta},\boldsymbol{\theta}_{i}\right\rangle \sigma''\left(\left\langle \kappa\boldsymbol{\zeta},\boldsymbol{x}\right\rangle \right)\sigma\left(\left\langle \kappa\boldsymbol{\theta}_{i},\boldsymbol{x}\right\rangle \right)\boldsymbol{x}\boldsymbol{x}^{\top}\right\} .
\end{align*}
Below we bound $\left\Vert \boldsymbol{M}_{1}\right\Vert _{{\rm op}}$
and $\left\Vert \boldsymbol{M}_{2}\right\Vert _{{\rm op}}$ separately.

\paragraph{Step 1: Bounding $\left\Vert \boldsymbol{M}_{1}\right\Vert _{{\rm op}}$.}

For a given $\boldsymbol{x}\in\mathbb{R}^{d}$, let us define $\hat{\boldsymbol{x}}\equiv\hat{\boldsymbol{x}}\left(\boldsymbol{x}\right)$
as in the statement of Proposition \ref{prop:2nd_setting_grow_bound}.
Let us also define the quantity $A_{1}=\kappa^{2}\left\Vert \mathbb{E}_{{\cal P}}\left\{ \sigma'\left(\left\langle \kappa\boldsymbol{\zeta},\boldsymbol{x}\right\rangle \right)\hat{\boldsymbol{x}}\boldsymbol{x}^{\top}\right\} \right\Vert _{2}$.
We observe that for any $\boldsymbol{u},\boldsymbol{v}\in\mathbb{R}^{d}$,
\begin{align*}
 & \left|\left\langle \boldsymbol{v},\kappa^{2}\mathbb{E}_{{\cal P}}\left\{ \sigma'\left(\left\langle \kappa\boldsymbol{\zeta},\boldsymbol{x}\right\rangle \right)\hat{\boldsymbol{x}}\boldsymbol{x}^{\top}\right\} \boldsymbol{u}\right\rangle \right|=\kappa^{2}\left|\mathbb{E}_{{\cal P}}\left\{ \sigma'\left(\left\langle \kappa\boldsymbol{\zeta},\boldsymbol{x}\right\rangle \right)\left\langle \boldsymbol{v},\hat{\boldsymbol{x}}\right\rangle \left\langle \boldsymbol{u},\boldsymbol{x}\right\rangle \right\} \right|\\
 & \quad\leq\sqrt{\mathbb{E}_{{\cal P}}\left\{ \left|\kappa\left\langle \boldsymbol{v},\hat{\boldsymbol{x}}\right\rangle \right|^{2}\right\} \mathbb{E}_{{\cal P}}\left\{ \left|\kappa\left\langle \boldsymbol{u},\boldsymbol{x}\right\rangle \right|^{2}\right\} }\leq C\left\Vert \boldsymbol{v}\right\Vert _{2}\left\Vert \boldsymbol{u}\right\Vert _{2},
\end{align*}
by Proposition \ref{prop:2nd_setting_grow_bound} and the fact $\left\Vert \boldsymbol{\Sigma}\right\Vert _{{\rm op}}\leq C$,
and therefore $A_{1}\leq C$. Furthermore, we have:
\begin{align*}
\left|\left\Vert \boldsymbol{M}_{1}\right\Vert _{{\rm op}}-A_{1}\right| & \leq\left\Vert \boldsymbol{M}_{1}-\kappa^{2}\mathbb{E}_{{\cal P}}\left\{ \sigma'\left(\left\langle \kappa\boldsymbol{\zeta},\boldsymbol{x}\right\rangle \right)\hat{\boldsymbol{x}}\boldsymbol{x}^{\top}\right\} \right\Vert _{{\rm op}}\\
 & =\left\Vert \kappa^{2}\mathbb{E}_{{\cal P}}\left\{ \sigma'\left(\left\langle \kappa\boldsymbol{\zeta},\boldsymbol{x}\right\rangle \right)\left[\frac{1}{N}\sum_{i=1}^{N}\kappa\boldsymbol{\theta}_{i}\sigma\left(\left\langle \kappa\boldsymbol{\theta}_{i},\boldsymbol{x}\right\rangle \right)-\hat{\boldsymbol{x}}\right]\boldsymbol{x}^{\top}\right\} \right\Vert _{{\rm op}}\equiv\left\Vert \boldsymbol{M}_{1,1}\right\Vert _{{\rm op}}.
\end{align*}
Here we making the following claim:
\[
\mathbb{P}\left\{ \left\Vert \boldsymbol{M}_{1,1}\right\Vert _{{\rm op}}\geq\delta\right\} \leq C\exp\left(Cd-C\delta^{2}N/\kappa^{4}\right),
\]
for $\delta\geq0$. Assuming this claim, we thus have for $\delta\geq0$
and some sufficiently large $C'$,
\[
\mathbb{P}\left\{ \left\Vert \boldsymbol{M}_{1}\right\Vert _{{\rm op}}\geq C'+\delta\right\} \leq C\exp\left(Cd-C\delta^{2}N/\kappa^{4}\right),
\]
which is the desired result.

We are left with proving the claim on $\left\Vert \boldsymbol{M}_{1,1}\right\Vert _{{\rm op}}$.
Given fixed $\boldsymbol{u},\boldsymbol{v}\in\mathbb{S}^{d-1}$,
\[
\left\langle \boldsymbol{u},\boldsymbol{M}_{1,1}\boldsymbol{v}\right\rangle =\frac{1}{N}\sum_{i=1}^{N}M_{1,1,i}^{\boldsymbol{u},\boldsymbol{v}},\qquad M_{1,1,i}^{\boldsymbol{u},\boldsymbol{v}}=\kappa\mathbb{E}_{{\cal P}}\left\{ \sigma'\left(\left\langle \kappa\boldsymbol{\zeta},\boldsymbol{x}\right\rangle \right)\left\langle \kappa\boldsymbol{\theta}_{i}\sigma\left(\left\langle \kappa\boldsymbol{\theta}_{i},\boldsymbol{x}\right\rangle \right)-\hat{\boldsymbol{x}},\boldsymbol{u}\right\rangle \left\langle \boldsymbol{x},\kappa\boldsymbol{v}\right\rangle \right\} .
\]
First notice that $\left(M_{1,1,i}^{\boldsymbol{u},\boldsymbol{v}}\right)_{i\leq N}$
are i.i.d. Furthermore $\mathbb{E}_{\boldsymbol{\theta}}\left\{ \kappa\boldsymbol{\theta}_{i}\sigma\left(\left\langle \kappa\boldsymbol{\theta}_{i},\boldsymbol{x}\right\rangle \right)\right\} =\hat{\boldsymbol{x}}$
by Proposition \ref{prop:2nd_setting_grow_bound}. Therefore $\mathbb{E}\left\{ M_{1,1,i}^{\boldsymbol{u},\boldsymbol{v}}\right\} =0$.
For any positive integer $p\geq1$, 
\begin{align*}
\mathbb{E}\left\{ \left|M_{1,1,i}^{\boldsymbol{u},\boldsymbol{v}}\right|^{p}\right\}  & =\mathbb{E}_{\boldsymbol{\theta}}\left\{ \left|\mathbb{E}_{{\cal P}}\left\{ \sigma'\left(\left\langle \kappa\boldsymbol{\zeta},\boldsymbol{x}\right\rangle \right)\left\langle \kappa\boldsymbol{\theta}_{i}\sigma\left(\left\langle \kappa\boldsymbol{\theta}_{i},\boldsymbol{x}\right\rangle \right)-\hat{\boldsymbol{x}},\kappa\boldsymbol{u}\right\rangle \left\langle \boldsymbol{x},\kappa\boldsymbol{v}\right\rangle \right\} \right|^{p}\right\} \\
 & \stackrel{\left(a\right)}{\leq}C^{p}\mathbb{E}_{\boldsymbol{\theta}}\left\{ \mathbb{E}_{{\cal P}}\left\{ \left\langle \kappa\boldsymbol{\theta}_{i}\sigma\left(\left\langle \kappa\boldsymbol{\theta}_{i},\boldsymbol{x}\right\rangle \right)-\hat{\boldsymbol{x}},\kappa\boldsymbol{u}\right\rangle ^{2}\right\} ^{p/2}\mathbb{E}_{{\cal P}}\left\{ \left\langle \boldsymbol{x},\kappa\boldsymbol{v}\right\rangle ^{2}\right\} ^{p/2}\right\} \\
 & \stackrel{\left(b\right)}{\leq}C^{p}\mathbb{E}_{\boldsymbol{\theta}}\left\{ \mathbb{E}_{{\cal P}}\left\{ \kappa^{2}\left\langle \kappa\boldsymbol{\theta}_{i},\boldsymbol{u}\right\rangle ^{2}+\left\langle \hat{\boldsymbol{x}},\kappa\boldsymbol{u}\right\rangle ^{2}\right\} ^{p/2}\mathbb{E}_{{\cal P}}\left\{ \left\langle \boldsymbol{x},\kappa\boldsymbol{v}\right\rangle ^{2}\right\} ^{p/2}\right\} \\
 & \stackrel{\left(c\right)}{\leq}C^{p}\mathbb{E}_{\boldsymbol{\theta}}\left\{ \left(\kappa^{2}\left\langle \kappa\boldsymbol{\theta}_{i},\boldsymbol{u}\right\rangle ^{2}+\left\Vert \boldsymbol{u}\right\Vert _{2}^{2}\right)^{p/2}\left\Vert \boldsymbol{\Sigma}\boldsymbol{v}\right\Vert _{2}^{p}\right\} \\
 & \stackrel{\left(d\right)}{\leq}C^{p}\mathbb{E}_{\boldsymbol{\theta}}\left\{ \kappa^{2p}\left\Vert \boldsymbol{\theta}_{i}\right\Vert _{2}^{p}+1\right\} \\
 & \stackrel{\left(e\right)}{\leq}C^{p}\left(\kappa^{2p}\int\left(r_{1}^{p}+r_{2}^{p}\right){\rm d}\rho_{r}+1\right)\\
 & \stackrel{\left(f\right)}{\leq}C^{p}\left(\kappa^{2p}p^{p/2}+1\right),
\end{align*}
where we have use the fact that $\left\Vert \sigma\right\Vert _{\infty},\left\Vert \sigma'\right\Vert _{\infty}\leq C$
in steps $\left(a\right)$ and $\left(b\right)$, $\mathbb{E}_{{\cal P}}\left\{ \left\langle \hat{\boldsymbol{x}},\kappa\boldsymbol{u}\right\rangle ^{2}\right\} \leq C\left\Vert \boldsymbol{u}\right\Vert _{2}^{2}$
by Proposition \ref{prop:2nd_setting_grow_bound} in step $\left(c\right)$,
$\left\Vert \boldsymbol{\Sigma}\right\Vert _{{\rm op}}\leq C$ and
$\left\Vert \boldsymbol{u}\right\Vert _{2}=\left\Vert \boldsymbol{v}\right\Vert _{2}=1$
in step $\left(d\right)$, $\boldsymbol{\theta}_{i}\stackrel{{\rm d}}{=}\left(r_{1}\boldsymbol{\omega}_{1},r_{2}\boldsymbol{\omega}_{2}\right)$
and $\left\Vert \boldsymbol{\omega}_{1}\right\Vert _{2}=\left\Vert \boldsymbol{\omega}_{2}\right\Vert _{2}=1$
in step $\left(e\right)$, and $r_{1}$ and $r_{2}$ are $C$-sub-Gaussian
in step $\left(f\right)$. It is easy to see that $M_{1,1,i}^{\boldsymbol{u},\boldsymbol{v}}$
is a sub-Gaussian random variable with $\psi_{2}$-norm $\left\Vert M_{1,1,i}^{\boldsymbol{u},\boldsymbol{v}}\right\Vert _{\psi_{2}}\leq C\kappa^{2}$.
Then by Lemma \ref{lem:subgauss_subexp_properties}, for any $\delta>0$,
with probability at most $C\exp\left(-C\delta^{2}N/\kappa^{4}\right)$,
\[
\left|\left\langle \boldsymbol{u},\boldsymbol{M}_{1,1}\boldsymbol{v}\right\rangle \right|=\left|\frac{1}{N}\sum_{i=1}^{N}M_{1,1,i}^{\boldsymbol{u},\boldsymbol{v}}\right|\geq\delta.
\]
Now we construct an epsilon-net ${\cal N}\subset\mathbb{S}^{d-1}$
such that for any $\boldsymbol{a}\in\mathbb{S}^{d-1}$, there exists
$\boldsymbol{a}'\in{\cal N}$ with $\left\Vert \boldsymbol{a}-\boldsymbol{a}'\right\Vert _{2}\leq1/3$.
There is such an epsilon-net ${\cal N}$ with size $\left|{\cal N}\right|\leq9^{d}$
\cite{vershynin2010introduction}. A standard argument yields
\[
\left\Vert \boldsymbol{M}_{1,1}\right\Vert _{{\rm op}}\leq3\max_{\boldsymbol{u},\boldsymbol{v}\in{\cal N}}\left\langle \boldsymbol{u},\boldsymbol{M}_{1,1}\boldsymbol{v}\right\rangle .
\]
Therefore, by the union bound, we obtain:
\[
\mathbb{P}\left\{ \left\Vert \boldsymbol{M}_{1,1}\right\Vert _{{\rm op}}\geq\delta\right\} \leq\mathbb{P}\left\{ \max_{\boldsymbol{u},\boldsymbol{v}\in{\cal N}}\left\langle \boldsymbol{u},\boldsymbol{M}_{1,1}\boldsymbol{v}\right\rangle \geq\delta/3\right\} \leq C\exp\left(Cd-C\delta^{2}N/\kappa^{4}\right).
\]
This proves the claim.

\paragraph*{Step 2: Bounding $\left\Vert \boldsymbol{M}_{2}\right\Vert _{{\rm op}}$.}

Given fixed $\boldsymbol{u},\boldsymbol{v}\in\mathbb{S}^{d-1}$,
\[
\left\langle \boldsymbol{u},\boldsymbol{M}_{2}\boldsymbol{v}\right\rangle =\frac{1}{N}\sum_{i=1}^{N}M_{2,i}^{\boldsymbol{u},\boldsymbol{v}},\qquad M_{2,i}^{\boldsymbol{u},\boldsymbol{v}}=\kappa\mathbb{E}_{\boldsymbol{z}}\left\{ \left\langle \boldsymbol{\zeta},\kappa\boldsymbol{\theta}_{i}\right\rangle \sigma''\left(\left\langle \boldsymbol{\Sigma}\boldsymbol{\zeta},\boldsymbol{z}\right\rangle \right)\sigma\left(\left\langle \boldsymbol{\theta}_{i},\boldsymbol{\Sigma}\boldsymbol{z}\right\rangle \right)\left\langle \boldsymbol{\Sigma}\boldsymbol{z},\boldsymbol{u}\right\rangle \left\langle \boldsymbol{\Sigma}\boldsymbol{z},\boldsymbol{v}\right\rangle \right\} ,
\]
where $\boldsymbol{z}\sim\mathsf{N}\left(0,\boldsymbol{I}_{d}\right)$.
First we bound $\mathbb{E}\left\{ \left|M_{2,i}^{\boldsymbol{u},\boldsymbol{v}}\right|^{p}\right\} $
for an integer $p\geq1$. We note that for $w=\left\langle \boldsymbol{\Sigma}\boldsymbol{\zeta},\boldsymbol{z}\right\rangle \sim\mathsf{N}\left(0,\left\Vert \boldsymbol{\Sigma}\boldsymbol{\zeta}\right\Vert _{2}^{2}\right)$,
\[
\left(w,\boldsymbol{z}\right)\stackrel{{\rm d}}{=}\left(w,{\rm Proj}_{\boldsymbol{\Sigma}\boldsymbol{\zeta}}^{\perp}\tilde{\boldsymbol{z}}+\frac{w}{\left\Vert \boldsymbol{\Sigma}\boldsymbol{\zeta}\right\Vert _{2}^{2}}\boldsymbol{\Sigma}\boldsymbol{\zeta}\right),
\]
for $\tilde{\boldsymbol{z}}\sim\mathsf{N}\left(0,\boldsymbol{I}_{d}\right)$
independent of $w$. Therefore, letting $\boldsymbol{S}=\boldsymbol{\Sigma}{\rm Proj}_{\boldsymbol{\Sigma}\boldsymbol{\zeta}}^{\perp}$
for brevity, we obtain:
\begin{align*}
M_{2,i}^{\boldsymbol{u},\boldsymbol{v}} & =\kappa\mathbb{E}_{w,\tilde{\boldsymbol{z}}}\Bigg\{\left\langle \boldsymbol{\zeta},\kappa\boldsymbol{\theta}_{i}\right\rangle \sigma''\left(w\right)\sigma\left(\left\langle \boldsymbol{\theta}_{i},\boldsymbol{S}\tilde{\boldsymbol{z}}\right\rangle +\frac{w\left\langle \boldsymbol{\theta}_{i},\boldsymbol{\Sigma}^{2}\boldsymbol{\zeta}\right\rangle }{\left\Vert \boldsymbol{\Sigma}\boldsymbol{\zeta}\right\Vert _{2}^{2}}\right)\\
 & \qquad\times\Bigg[\left\langle \boldsymbol{S}\tilde{\boldsymbol{z}},\boldsymbol{u}\right\rangle \left\langle \boldsymbol{S}\tilde{\boldsymbol{z}},\boldsymbol{v}\right\rangle +\frac{w}{\left\Vert \boldsymbol{\Sigma}\boldsymbol{\zeta}\right\Vert _{2}^{2}}\left\langle \boldsymbol{\Sigma}^{2}\boldsymbol{\zeta},\boldsymbol{u}\right\rangle \left\langle \boldsymbol{S}\tilde{\boldsymbol{z}},\boldsymbol{v}\right\rangle \\
 & \qquad+\frac{w}{\left\Vert \boldsymbol{\Sigma}\boldsymbol{\zeta}\right\Vert _{2}^{2}}\left\langle \boldsymbol{\Sigma}^{2}\boldsymbol{\zeta},\boldsymbol{v}\right\rangle \left\langle \boldsymbol{S}\tilde{\boldsymbol{z}},\boldsymbol{u}\right\rangle +\frac{w^{2}}{\left\Vert \boldsymbol{\Sigma}\boldsymbol{\zeta}\right\Vert _{2}^{4}}\left\langle \boldsymbol{\Sigma}^{2}\boldsymbol{\zeta},\boldsymbol{u}\right\rangle \left\langle \boldsymbol{\Sigma}^{2}\boldsymbol{\zeta},\boldsymbol{v}\right\rangle \Bigg]\Bigg\}.
\end{align*}
Using Lemma \ref{lem:2nd_setting_actGaussian} along with the facts
$\left\Vert \boldsymbol{\Sigma}\boldsymbol{\zeta}\right\Vert _{2}\geq C\left\Vert \boldsymbol{\zeta}\right\Vert _{2}$,
$\left\Vert \boldsymbol{S}\right\Vert _{{\rm op}}\leq\left\Vert \boldsymbol{\Sigma}\right\Vert _{{\rm op}}\leq C$
and $\left\Vert \boldsymbol{u}\right\Vert _{2}=\left\Vert \boldsymbol{v}\right\Vert _{2}=1$,
we deduce that
\begin{align*}
\left|M_{2,i}^{\boldsymbol{u},\boldsymbol{v}}\right| & \leq C\kappa\mathbb{E}_{\tilde{\boldsymbol{z}}}\Bigg\{\left|\left\langle \boldsymbol{\zeta},\kappa\boldsymbol{\theta}_{i}\right\rangle \right|\Bigg[\frac{\left|\left\langle \boldsymbol{S}\tilde{\boldsymbol{z}},\boldsymbol{u}\right\rangle \left\langle \boldsymbol{S}\tilde{\boldsymbol{z}},\boldsymbol{v}\right\rangle \right|}{\left\Vert \boldsymbol{\Sigma}\boldsymbol{\zeta}\right\Vert _{2}}+\frac{\left|\left\langle \boldsymbol{\Sigma}^{2}\boldsymbol{\zeta},\boldsymbol{u}\right\rangle \left\langle \boldsymbol{S}\tilde{\boldsymbol{z}},\boldsymbol{v}\right\rangle \right|}{\left\Vert \boldsymbol{\Sigma}\boldsymbol{\zeta}\right\Vert _{2}^{2}}\\
 & \qquad+\frac{\left|\left\langle \boldsymbol{\Sigma}^{2}\boldsymbol{\zeta},\boldsymbol{v}\right\rangle \left\langle \boldsymbol{S}\tilde{\boldsymbol{z}},\boldsymbol{u}\right\rangle \right|}{\left\Vert \boldsymbol{\Sigma}\boldsymbol{\zeta}\right\Vert _{2}^{2}}+\frac{\left|\left\langle \boldsymbol{\Sigma}^{2}\boldsymbol{\zeta},\boldsymbol{u}\right\rangle \left\langle \boldsymbol{\Sigma}^{2}\boldsymbol{\zeta},\boldsymbol{v}\right\rangle \right|}{\left\Vert \boldsymbol{\Sigma}\boldsymbol{\zeta}\right\Vert _{2}^{3}}\Bigg]\Bigg\}\\
 & \leq C\kappa^{2}\left\Vert \boldsymbol{\theta}_{i}\right\Vert _{2}.
\end{align*}
Therefore, $\mathbb{E}\left\{ \left|M_{2,i}^{\boldsymbol{u},\boldsymbol{v}}\right|^{p}\right\} \leq C\kappa^{2p}\int\left(r_{1}^{p}+r_{2}^{p}\right){\rm d}\rho_{r}\leq C\kappa^{2p}p^{p/2}$.
That is, $M_{2,i}^{\boldsymbol{u},\boldsymbol{v}}$ is $C\kappa^{2}$-sub-Gaussian.

The above bound, however, does not give a satisfactory bound for the
quantity $\left|\mathbb{E}\left\{ \left\langle \boldsymbol{u},\boldsymbol{M}_{2}\boldsymbol{v}\right\rangle \right\} \right|=\left|\mathbb{E}\left\{ M_{2,i}^{\boldsymbol{u},\boldsymbol{v}}\right\} \right|$,
since it incurs a factor $\kappa^{2}$ in the bound. We give a more
careful treatment of this quantity here. By Proposition \ref{prop:2nd_setting_grow_bound}:
\[
\mathbb{E}_{\boldsymbol{\theta}}\left\{ \kappa\boldsymbol{\theta}_{i}\sigma\left(\left\langle \boldsymbol{\theta}_{i},\boldsymbol{S}\tilde{\boldsymbol{z}}\right\rangle +\frac{w\left\langle \boldsymbol{\theta}_{i},\boldsymbol{\Sigma}^{2}\boldsymbol{\zeta}\right\rangle }{\left\Vert \boldsymbol{\Sigma}\boldsymbol{\zeta}\right\Vert _{2}^{2}}\right)\right\} =\left(s_{1}\hat{\boldsymbol{\zeta}}_{\left[1\right]},\;s_{2}\hat{\boldsymbol{\zeta}}_{\left[2\right]}\right)
\]
in which we define
\[
\hat{\boldsymbol{\zeta}}=\frac{1}{\kappa}\left(\boldsymbol{S}\tilde{\boldsymbol{z}}+\frac{w\boldsymbol{\Sigma}^{2}\boldsymbol{\zeta}}{\left\Vert \boldsymbol{\Sigma}\boldsymbol{\zeta}\right\Vert _{2}^{2}}\right),\qquad s_{j}=\int\frac{r_{j}}{\left\Vert \hat{\boldsymbol{\zeta}}_{\left[j\right]}\right\Vert _{2}}q_{j}\left(\left\Vert \hat{\boldsymbol{\zeta}}_{\left[1\right]}\right\Vert _{2}r_{1},\left\Vert \hat{\boldsymbol{\zeta}}_{\left[2\right]}\right\Vert _{2}r_{2}\right)\rho_{r}\left({\rm d}r_{1},{\rm d}r_{2}\right),\qquad j=1,2,
\]
and $q_{1}$ and $q_{2}$ are defined in (\ref{eq:2nd_setting_q1})
and (\ref{eq:2nd_setting_q2}). This yields the formula:
\begin{align*}
\mathbb{E}\left\{ M_{2,i}^{\boldsymbol{u},\boldsymbol{v}}\right\}  & =\mathbb{E}_{w,\tilde{\boldsymbol{z}}}\Bigg\{\left(\sum_{j\in\left\{ 1,2\right\} }s_{j}\left\langle \boldsymbol{\zeta}_{\left[j\right]},\left(\boldsymbol{S}\tilde{\boldsymbol{z}}\right)_{\left[j\right]}\right\rangle +s_{j}w\frac{\left\langle \boldsymbol{\zeta}_{\left[j\right]},\boldsymbol{\Sigma}^{2}\boldsymbol{\zeta}_{\left[j\right]}\right\rangle }{\left\Vert \boldsymbol{\Sigma}\boldsymbol{\zeta}\right\Vert _{2}^{2}}\right)\sigma''\left(w\right)\\
 & \qquad\times\Bigg[\left\langle \boldsymbol{S}\tilde{\boldsymbol{z}},\boldsymbol{u}\right\rangle \left\langle \boldsymbol{S}\tilde{\boldsymbol{z}},\boldsymbol{v}\right\rangle +\frac{w}{\left\Vert \boldsymbol{\Sigma}\boldsymbol{\zeta}\right\Vert _{2}^{2}}\left\langle \boldsymbol{\Sigma}^{2}\boldsymbol{\zeta},\boldsymbol{u}\right\rangle \left\langle \boldsymbol{S}\tilde{\boldsymbol{z}},\boldsymbol{v}\right\rangle \\
 & \qquad+\frac{w}{\left\Vert \boldsymbol{\Sigma}\boldsymbol{\zeta}\right\Vert _{2}^{2}}\left\langle \boldsymbol{\Sigma}^{2}\boldsymbol{\zeta},\boldsymbol{v}\right\rangle \left\langle \boldsymbol{S}\tilde{\boldsymbol{z}},\boldsymbol{u}\right\rangle +\frac{w^{2}}{\left\Vert \boldsymbol{\Sigma}\boldsymbol{\zeta}\right\Vert _{2}^{4}}\left\langle \boldsymbol{\Sigma}^{2}\boldsymbol{\zeta},\boldsymbol{u}\right\rangle \left\langle \boldsymbol{\Sigma}^{2}\boldsymbol{\zeta},\boldsymbol{v}\right\rangle \Bigg]\Bigg\}.
\end{align*}
By Lemma \ref{lem:2nd_setting_qBounds} and the fact $\int\left(r_{1}^{2}+r_{2}^{2}\right){\rm d}\rho_{r}\leq C$,
we have $\left|s_{1}\right|,\left|s_{2}\right|\leq C$. Then applying
Lemma \ref{lem:2nd_setting_actGaussian} along with the facts $\left\Vert \boldsymbol{\Sigma}\boldsymbol{\zeta}\right\Vert _{2}\geq C\left\Vert \boldsymbol{\zeta}\right\Vert _{2}$,
$\left\Vert \boldsymbol{S}\right\Vert _{{\rm op}}\leq\left\Vert \boldsymbol{\Sigma}\right\Vert _{{\rm op}}\leq C$
and $\left\Vert \boldsymbol{u}\right\Vert _{2}=\left\Vert \boldsymbol{v}\right\Vert _{2}=1$,
we obtain:
\begin{align*}
\left|\mathbb{E}\left\{ M_{2,i}^{\boldsymbol{u},\boldsymbol{v}}\right\} \right| & \leq C\mathbb{E}_{\tilde{\boldsymbol{z}}}\Bigg\{\left(\sum_{j\in\left\{ 1,2\right\} }\left|\left\langle \boldsymbol{\zeta}_{\left[j\right]},\left(\boldsymbol{S}\tilde{\boldsymbol{z}}\right)_{\left[j\right]}\right\rangle \right|+\frac{\left|\left\langle \boldsymbol{\zeta}_{\left[j\right]},\boldsymbol{\Sigma}^{2}\boldsymbol{\zeta}_{\left[j\right]}\right\rangle \right|}{\left\Vert \boldsymbol{\Sigma}\boldsymbol{\zeta}\right\Vert _{2}}\right)\\
 & \qquad\times\Bigg[\frac{\left\langle \boldsymbol{S}\tilde{\boldsymbol{z}},\boldsymbol{u}\right\rangle \left\langle \boldsymbol{S}\tilde{\boldsymbol{z}},\boldsymbol{v}\right\rangle }{\left\Vert \boldsymbol{\Sigma}\boldsymbol{\zeta}\right\Vert _{2}}+\frac{\left|\left\langle \boldsymbol{\Sigma}^{2}\boldsymbol{\zeta},\boldsymbol{u}\right\rangle \left\langle \boldsymbol{S}\tilde{\boldsymbol{z}},\boldsymbol{v}\right\rangle \right|}{\left\Vert \boldsymbol{\Sigma}\boldsymbol{\zeta}\right\Vert _{2}^{2}}\\
 & \qquad+\frac{\left|\left\langle \boldsymbol{\Sigma}^{2}\boldsymbol{\zeta},\boldsymbol{v}\right\rangle \left\langle \boldsymbol{S}\tilde{\boldsymbol{z}},\boldsymbol{u}\right\rangle \right|}{\left\Vert \boldsymbol{\Sigma}\boldsymbol{\zeta}\right\Vert _{2}^{2}}+\frac{\left|\left\langle \boldsymbol{\Sigma}^{2}\boldsymbol{\zeta},\boldsymbol{u}\right\rangle \left\langle \boldsymbol{\Sigma}^{2}\boldsymbol{\zeta},\boldsymbol{v}\right\rangle \right|}{\left\Vert \boldsymbol{\Sigma}\boldsymbol{\zeta}\right\Vert _{2}^{3}}\Bigg]\Bigg\}\leq C.
\end{align*}
Let this upper-bounding constant be $C_{1}$.

To complete the present step, notice that $\left(M_{2,i}^{\boldsymbol{u},\boldsymbol{v}}\right)_{i\leq N}$
are i.i.d. Then by Lemma \ref{lem:subgauss_subexp_properties}, for
any $\delta>0$, with probability at most $C\exp\left(-C\delta^{2}N/\kappa^{4}\right)$,
\[
\left|\left\langle \boldsymbol{u},\boldsymbol{M}_{2}\boldsymbol{v}\right\rangle -\mathbb{E}\left\{ \left\langle \boldsymbol{u},\boldsymbol{M}_{2}\boldsymbol{v}\right\rangle \right\} \right|=\left|\frac{1}{N}\sum_{i=1}^{N}M_{2,i}^{\boldsymbol{u},\boldsymbol{v}}-\mathbb{E}\left\{ M_{2,i}^{\boldsymbol{u},\boldsymbol{v}}\right\} \right|\geq\delta,
\]
which also implies
\[
\left|\left\langle \boldsymbol{u},\boldsymbol{M}_{2}\boldsymbol{v}\right\rangle \right|\geq\delta-\left|\mathbb{E}\left\{ \left\langle \boldsymbol{u},\boldsymbol{M}_{2}\boldsymbol{v}\right\rangle \right\} \right|\geq\delta-C_{1},
\]
since $\left|\mathbb{E}\left\{ \left\langle \boldsymbol{u},\boldsymbol{M}_{2}\boldsymbol{v}\right\rangle \right\} \right|\leq C_{1}$.
We opt for $\delta=2C_{1}$. Now we can reuse the same epsilon-net
argument in the analysis of $\boldsymbol{M}_{1,1}$ to obtain:
\[
\mathbb{P}\left\{ \left\Vert \boldsymbol{M}_{2}\right\Vert _{{\rm op}}\geq C_{1}\right\} \leq C\exp\left(Cd-CC_{1}N/\kappa^{4}\right).
\]

\paragraph{Step 3: Putting all together.}

From the bounds on $\left\Vert \boldsymbol{M}_{1}\right\Vert _{{\rm op}}$
and $\left\Vert \boldsymbol{M}_{2}\right\Vert _{{\rm op}}$, we obtain:
\begin{align*}
\mathbb{P}\left\{ \left\Vert \frac{1}{N}\sum_{i=1}^{N}\nabla_{11}^{2}U\left(\boldsymbol{\zeta},\boldsymbol{\theta}_{i}\right)\right\Vert _{{\rm op}}\geq C_{*}\right\}  & \leq C\exp\left(Cd-CN/\kappa^{4}\right),
\end{align*}
for sufficiently large $C_{*}$. This completes the proof.

\end{proof}
\begin{prop}
\label{prop:2nd_setting_U_bound}Consider setting \ref{enu:bdd_act_setting}.
For each integer $i=1,...,N$, we draw independently $\boldsymbol{\omega}_{1,i}\sim{\rm Unif}\left(\mathbb{S}^{d_{1}-1}\right)$,
$\boldsymbol{\omega}_{2,i}\sim{\rm Unif}\left(\mathbb{S}^{d_{2}-1}\right)$,
$r_{1,i}$ and $r_{2,i}$, with $r_{1,i}$ and $r_{2,i}$ being non-negative
$C$-sub-Gaussian random variables. Let $\left(\psi_{t}\right)_{t\in\left[0,T\right]}$
be a collection of (deterministic) functions, which map from $\mathbb{R}_{\geq0}\times\mathbb{R}_{\geq0}$
to $\mathbb{R}_{\geq0}\times\mathbb{R}_{\geq0}$, such that:
\begin{itemize}
\item for any $t\in\left[0,T\right]$, each of the two entries in $\psi_{t}\left(r_{1,i},r_{2,i}\right)$
is marginally $C$-sub-Gaussian,
\item $\left\Vert \partial_{t}\psi_{t}\left(r_{1},r_{2}\right)\right\Vert _{2}\leq C\left(1+r_{1}+r_{2}\right)$
for any $t\in\left[0,T\right]$ and $\psi_{0}\left(r_{1},r_{2}\right)=\left(r_{1},r_{2}\right)$.
\end{itemize}
For each $i\leq N$ and $t\in\left[0,T\right]$, we form $\boldsymbol{\theta}_{i}^{t}=\left(\left(\psi_{t}\left(r_{1,i},r_{2,i}\right)\right)_{1}\boldsymbol{\omega}_{1,i},\;\left(\psi_{t}\left(r_{1,i},r_{2,i}\right)\right)_{2}\boldsymbol{\omega}_{2,i}\right)\in\mathbb{R}^{d}$,
where $\left(\psi_{t}\left(r_{1,i},r_{2,i}\right)\right)_{j}$ denotes
the $j$-th entry of $\psi_{t}\left(r_{1,i},r_{2,i}\right)$, for
$j=1,2$. Then for any $c>0$ and $T>0$, with probability at least
$1-C\exp\left(Cd\log\left(\kappa^{2}\sqrt{N}+e\right)-CN/\kappa^{4}\right)$,
\[
\sup_{t\in\left[0,T\right]}\sup_{\boldsymbol{\zeta}\in{\cal B}_{d}\left(c\sqrt{N}\right)}\left\Vert \frac{1}{N}\sum_{i=1}^{N}\nabla_{11}^{2}U\left(\boldsymbol{\zeta},\boldsymbol{\theta}_{i}^{t}\right)\right\Vert _{{\rm op}}\leq C_{*},
\]
for some sufficiently large constant $C_{*}$. (The constants $C$
and $C_{*}$ do not depend on $d$ or $N$, may depend on $c$ and
$T$ and are finite with finite $c$ and $T$.)
\end{prop}

\begin{proof}
The proof leverages on Lemma \ref{lem:Unorm_2ndSetting} and comprises
of several steps. Without loss of generality, let us assume $c=T=1$.
That is, we shall study the quantity
\[
Q=\sup_{t\in\left[0,1\right]}\sup_{\boldsymbol{\zeta}\in{\cal B}_{d}\left(\sqrt{N}\right)}\left\Vert \frac{1}{N}\sum_{i=1}^{N}\nabla_{11}^{2}U\left(\boldsymbol{\zeta},\boldsymbol{\theta}_{i}^{t}\right)\right\Vert _{{\rm op}}.
\]

\paragraph*{Step 1: Epsilon-net argument.}

Fix $\gamma\in\left(0,1/3\right)$. Consider an epsilon-net ${\cal N}_{d}\left(\gamma\right)\subset{\cal B}_{d}\left(\sqrt{N}\right)$
in which for any $\boldsymbol{\zeta}\in{\cal B}_{d}\left(\sqrt{N}\right)$,
there exists $\boldsymbol{\zeta}'\in{\cal N}_{d}\left(\gamma\right)$
such that $\left\Vert \boldsymbol{\zeta}-\boldsymbol{\zeta}'\right\Vert _{2}\leq\gamma\sqrt{N}$.
A standard volumetric argument \cite{vershynin2010introduction} shows
that there exists such epsilon-net with size $\left|{\cal N}_{d}\left(\gamma\right)\right|\leq\left(3/\gamma\right)^{d}$.
Likewise let ${\cal N}\left(\gamma\right)=\left\{ k\gamma:\;k\in\mathbb{N}_{\geq0},\;0\leq k\gamma\leq1\right\} $,
and note that $\left|{\cal N}\left(\gamma\right)\right|\leq1+1/\gamma$.
Consider $t\in\left[0,1\right]$ and $t'\in{\cal N}\left(\gamma\right)$
such that $\left|t-t'\right|\leq\gamma$, and $\boldsymbol{\zeta}\in{\cal B}_{d}\left(\sqrt{N}\right)$
and $\boldsymbol{\zeta}'\in{\cal N}_{d}\left(\gamma\right)$ such
that $\left\Vert \boldsymbol{\zeta}-\boldsymbol{\zeta}'\right\Vert _{2}\leq\gamma\sqrt{N}$.
We have:
\begin{align*}
\left\Vert \boldsymbol{\theta}_{i}^{t}-\boldsymbol{\theta}_{i}^{t'}\right\Vert _{2} & \leq\sum_{j\in\left\{ 1,2\right\} }\left|\left(\psi_{t}\left(r_{1,i},r_{2,i}\right)\right)_{j}-\left(\psi_{t'}\left(r_{1,i},r_{2,i}\right)\right)_{j}\right|\leq2\sup_{s\in\left[t,t'\right]}\left\Vert \partial_{s}\psi_{s}\left(r_{1,i},r_{2,i}\right)\right\Vert _{2}\left|t-t'\right|\\
 & \leq C\left(r_{1,i}+r_{2,i}+1\right)\gamma.
\end{align*}
Furthermore, for any $t\in\left[0,T\right]$,
\begin{align*}
\left\Vert \boldsymbol{\theta}_{i}^{t}\right\Vert _{2} & \leq\sum_{j\in\left\{ 1,2\right\} }\left|\left(\psi_{t}\left(r_{1,i},r_{2,i}\right)\right)_{j}\right|=\sum_{j\in\left\{ 1,2\right\} }\left|\left(\psi_{0}\left(r_{1,i},r_{2,i}\right)\right)_{j}+\int_{s=0}^{t}\partial_{s}\left(\psi_{s}\left(r_{1,i},r_{2,i}\right)\right)_{j}{\rm d}s\right|\\
 & \leq r_{1,i}+r_{2,i}+C\int_{s=0}^{t}\left(r_{1,i}+r_{2,i}+1\right){\rm d}s\leq C\left(r_{1,i}+r_{2,i}+1\right).
\end{align*}
We then have from the mean value theorem:
\begin{align*}
 & \left|\left\Vert \frac{1}{N}\sum_{i=1}^{N}\nabla_{11}^{2}U\left(\boldsymbol{\zeta},\boldsymbol{\theta}_{i}^{t}\right)\right\Vert _{{\rm op}}-\left\Vert \frac{1}{N}\sum_{i=1}^{N}\nabla_{11}^{2}U\left(\boldsymbol{\zeta}',\boldsymbol{\theta}_{i}^{t'}\right)\right\Vert _{{\rm op}}\right|\\
 & \quad\leq\left\Vert \frac{1}{N}\sum_{i=1}^{N}\nabla_{11}^{2}U\left(\boldsymbol{\zeta},\boldsymbol{\theta}_{i}^{t}\right)-\nabla_{11}^{2}U\left(\boldsymbol{\zeta}',\boldsymbol{\theta}_{i}^{t}\right)\right\Vert _{{\rm op}}+\left\Vert \frac{1}{N}\sum_{i=1}^{N}\nabla_{11}^{2}U\left(\boldsymbol{\zeta}',\boldsymbol{\theta}_{i}^{t}\right)-\nabla_{11}^{2}U\left(\boldsymbol{\zeta}',\boldsymbol{\theta}_{i}^{t'}\right)\right\Vert _{{\rm op}}\\
 & \quad\stackrel{\left(a\right)}{\leq}\frac{1}{N}\sum_{i=1}^{N}\left\Vert \nabla_{111}^{3}U\left[\boldsymbol{u}_{i},\boldsymbol{\theta}_{i}^{t}\right]\right\Vert _{{\rm op}}\left\Vert \boldsymbol{\zeta}-\boldsymbol{\zeta}'\right\Vert _{2}+\frac{1}{N}\sum_{i=1}^{N}\left\Vert \nabla_{121}^{3}U\left[\boldsymbol{\zeta}',\boldsymbol{v}_{i}\right]\right\Vert _{{\rm op}}\left\Vert \boldsymbol{\theta}_{i}^{t}-\boldsymbol{\theta}_{i}^{t'}\right\Vert _{2}\\
 & \quad\stackrel{\left(b\right)}{\leq}\frac{1}{N}\sum_{i=1}^{N}\left\Vert \nabla_{111}^{3}U\left[\boldsymbol{u}_{i},\boldsymbol{\theta}_{i}^{t}\right]\right\Vert _{{\rm op}}\gamma\sqrt{N}+\frac{1}{N}\sum_{i=1}^{N}C\kappa^{2}\left(1+\left\Vert \boldsymbol{v}_{i}\right\Vert _{2}\right)\left(r_{1,i}+r_{2,i}+1\right)\gamma\\
 & \quad\stackrel{\left(c\right)}{\leq}\frac{1}{N}\sum_{i=1}^{N}\left\Vert \nabla_{111}^{3}U\left[\boldsymbol{u}_{i},\boldsymbol{\theta}_{i}^{t}\right]\right\Vert _{{\rm op}}\gamma\sqrt{N}+\frac{1}{N}\sum_{i=1}^{N}C\kappa^{2}\left(r_{1,i}^{2}+r_{2,i}^{2}+1\right)\gamma,
\end{align*}
where in step $\left(a\right)$, we have $\boldsymbol{u}_{i}\in\left[\boldsymbol{\zeta},\boldsymbol{\zeta}'\right]$
and $\boldsymbol{v}_{i}\in\left[\boldsymbol{\theta}_{i}^{t},\boldsymbol{\theta}_{i}^{t'}\right]$;
in step $\left(b\right)$, we apply Proposition \ref{prop:2nd_setting_op_bound};
in step $\left(c\right)$, we use the fact that $\left\Vert \boldsymbol{v}_{i}\right\Vert _{2}\leq\left\Vert \boldsymbol{\theta}_{i}^{t}\right\Vert _{2}+\left\Vert \boldsymbol{\theta}_{i}^{t'}-\boldsymbol{\theta}_{i}^{t'}\right\Vert _{2}$.
We have:
\[
\nabla_{111}^{3}U\left[\boldsymbol{u}_{i},\boldsymbol{\theta}_{i}^{t}\right]=\boldsymbol{M}_{1,i}+\boldsymbol{M}_{2,i}+\boldsymbol{M}_{3,i}+\boldsymbol{M}_{4,i}\in\left(\mathbb{R}^{d}\right)^{\otimes3},
\]
for which
\begin{align*}
\boldsymbol{M}_{1,i} & =K\kappa^{4}\mathbb{E}_{{\cal P}}\left\{ \sigma''\left(\left\langle \kappa\boldsymbol{u}_{i},\boldsymbol{x}\right\rangle \right)\sigma\left(\left\langle \kappa\boldsymbol{\theta}_{i}^{t},\boldsymbol{x}\right\rangle \right)\boldsymbol{x}\otimes\boldsymbol{\theta}_{i}^{t}\otimes\boldsymbol{x}\right\} ,\\
\boldsymbol{M}_{2,i} & =K\kappa^{4}\mathbb{E}_{{\cal P}}\left\{ \sigma''\left(\left\langle \kappa\boldsymbol{u}_{i},\boldsymbol{x}\right\rangle \right)\sigma\left(\left\langle \kappa\boldsymbol{\theta}_{i}^{t},\boldsymbol{x}\right\rangle \right)\boldsymbol{x}\otimes\boldsymbol{x}\otimes\boldsymbol{\theta}_{i}^{t}\right\} ,\\
\boldsymbol{M}_{3,i} & =K\kappa^{4}\mathbb{E}_{{\cal P}}\left\{ \sigma''\left(\left\langle \kappa\boldsymbol{u}_{i},\boldsymbol{x}\right\rangle \right)\sigma\left(\left\langle \kappa\boldsymbol{\theta}_{i}^{t},\boldsymbol{x}\right\rangle \right)\boldsymbol{\theta}_{i}^{t}\otimes\boldsymbol{x}\otimes\boldsymbol{x}\right\} ,\\
\boldsymbol{M}_{4,i} & =\kappa^{5}\mathbb{E}_{{\cal P}}\left\{ \left\langle \boldsymbol{u}_{i},\boldsymbol{\theta}_{i}^{t}\right\rangle \sigma'''\left(\left\langle \kappa\boldsymbol{u}_{i},\boldsymbol{x}\right\rangle \right)\sigma\left(\left\langle \kappa\boldsymbol{\theta}_{i}^{t},\boldsymbol{x}\right\rangle \right)\boldsymbol{x}\otimes\boldsymbol{x}\otimes\boldsymbol{x}\right\} .
\end{align*}
Note that $\left\Vert \boldsymbol{M}_{1,i}\right\Vert _{{\rm op}}=\left\Vert \boldsymbol{M}_{2,i}\right\Vert _{{\rm op}}=\left\Vert \boldsymbol{M}_{3,i}\right\Vert _{{\rm op}}$.
We then have:
\begin{equation}
\left|Q-Q_{\gamma}\right|\leq\frac{1}{N}\sum_{i=1}^{N}\sup_{\boldsymbol{u}_{i}\in\mathbb{R}^{d}}\left(3\left\Vert \boldsymbol{M}_{1,i}\right\Vert _{{\rm op}}+\left\Vert \boldsymbol{M}_{4,i}\right\Vert _{{\rm op}}\right)\gamma\sqrt{N}+\frac{1}{N}\sum_{i=1}^{N}C\kappa^{2}\left(r_{1,i}^{2}+r_{2,i}^{2}+1\right)\gamma,\label{eq:prop_2ndSetting_supUnorm_gap}
\end{equation}
in which we define:
\[
Q_{\gamma}=\max_{t\in{\cal N}\left(\gamma\right)}\max_{\boldsymbol{\zeta}\in{\cal N}_{d}\left(\gamma\right)}\left\Vert \frac{1}{N}\sum_{i=1}^{N}\nabla_{11}^{2}U\left(\boldsymbol{\zeta},\boldsymbol{\theta}_{i}^{t}\right)\right\Vert _{{\rm op}}.
\]
The next two steps are devoted to bounding $\left\Vert \boldsymbol{M}_{1,i}\right\Vert _{{\rm op}}$
and $\left\Vert \boldsymbol{M}_{4,i}\right\Vert _{{\rm op}}$.

\paragraph*{Step 2: Bounding $\left\Vert \boldsymbol{M}_{1,i}\right\Vert _{{\rm op}}$.}

To bound $\left\Vert \boldsymbol{M}_{1,i}\right\Vert _{{\rm op}}$,
we have for any $\boldsymbol{a},\boldsymbol{b},\boldsymbol{c}\in\mathbb{R}^{d}$:
\begin{align*}
\left\langle \boldsymbol{M}_{1,i},\boldsymbol{a}\otimes\boldsymbol{b}\otimes\boldsymbol{c}\right\rangle  & =\kappa^{4}\mathbb{E}_{{\cal P}}\left\{ \sigma''\left(\left\langle \kappa\boldsymbol{u}_{i},\boldsymbol{x}\right\rangle \right)\sigma\left(\left\langle \kappa\boldsymbol{\theta}_{i}^{t},\boldsymbol{x}\right\rangle \right)\left\langle \boldsymbol{a},\boldsymbol{x}\right\rangle \left\langle \boldsymbol{b},\boldsymbol{\theta}_{i}^{t}\right\rangle \left\langle \boldsymbol{c},\boldsymbol{x}\right\rangle \right\} \\
 & =\kappa^{2}\mathbb{E}_{\boldsymbol{z}}\left\{ \sigma''\left(\left\langle \boldsymbol{\Sigma}\boldsymbol{u}_{i},\boldsymbol{z}\right\rangle \right)\sigma\left(\left\langle \boldsymbol{\Sigma}\boldsymbol{\theta}_{i}^{t},\boldsymbol{z}\right\rangle \right)\left\langle \boldsymbol{\Sigma}\boldsymbol{a},\boldsymbol{z}\right\rangle \left\langle \boldsymbol{b},\boldsymbol{\theta}_{i}^{t}\right\rangle \left\langle \boldsymbol{\Sigma}\boldsymbol{c},\boldsymbol{z}\right\rangle \right\} ,
\end{align*}
where $\boldsymbol{z}\sim\mathsf{N}\left(0,\boldsymbol{I}_{d}\right)$.
Recalling that $\left\Vert \sigma\right\Vert _{\infty},\left\Vert \sigma''\right\Vert _{\infty}\leq C$
and $\left\Vert \boldsymbol{\Sigma}\right\Vert _{2}\leq C$, we thus
have:
\[
\left|\left\langle \boldsymbol{M}_{1,i},\boldsymbol{a}\otimes\boldsymbol{b}\otimes\boldsymbol{c}\right\rangle \right|\leq C\kappa^{2}\left\Vert \boldsymbol{\theta}_{i}^{t}\right\Vert _{2}\left\Vert \boldsymbol{a}\right\Vert _{2}\left\Vert \boldsymbol{b}\right\Vert _{2}\left\Vert \boldsymbol{c}\right\Vert _{2}.
\]
That is, $\left\Vert \boldsymbol{M}_{1,i}\right\Vert _{{\rm op}}\leq C\kappa^{2}\left\Vert \boldsymbol{\theta}_{i}^{t}\right\Vert _{2}$.

\paragraph*{Step 3: Bounding $\left\Vert \boldsymbol{M}_{4,i}\right\Vert _{{\rm op}}$.}

Notice that for $w_{i}=\left\langle \boldsymbol{\Sigma}\boldsymbol{u}_{i},\boldsymbol{z}\right\rangle \sim\mathsf{N}\left(0,\left\Vert \boldsymbol{\Sigma}\boldsymbol{u}_{i}\right\Vert _{2}^{2}\right)$,
\[
\left(w_{i},\boldsymbol{z}\right)\stackrel{{\rm d}}{=}\left(w_{i},{\rm Proj}_{\boldsymbol{\Sigma}\boldsymbol{u}_{i}}^{\perp}\tilde{\boldsymbol{z}}+\frac{w_{i}}{\left\Vert \boldsymbol{\Sigma}\boldsymbol{u}_{i}\right\Vert _{2}^{2}}\boldsymbol{\Sigma}\boldsymbol{u}_{i}\right),
\]
in which $\tilde{\boldsymbol{z}}\sim\mathsf{N}\left(0,\boldsymbol{I}_{d}\right)$
independent of $w_{i}$. We then have:
\begin{align*}
 & \left\langle \boldsymbol{M}_{4,i},\boldsymbol{a}\otimes\boldsymbol{b}\otimes\boldsymbol{c}\right\rangle \\
 & \quad=\kappa^{2}\mathbb{E}_{{\cal P}}\left\{ \left\langle \boldsymbol{u}_{i},\boldsymbol{\theta}_{i}^{t}\right\rangle \sigma'''\left(\left\langle \kappa\boldsymbol{u}_{i},\boldsymbol{x}\right\rangle \right)\sigma\left(\left\langle \kappa\boldsymbol{\theta}_{i}^{t},\boldsymbol{x}\right\rangle \right)\left\langle \kappa\boldsymbol{a},\boldsymbol{x}\right\rangle \left\langle \kappa\boldsymbol{b},\boldsymbol{x}\right\rangle \left\langle \kappa\boldsymbol{c},\boldsymbol{x}\right\rangle \right\} \\
 & \quad=\kappa^{2}\mathbb{E}_{w_{i},\tilde{\boldsymbol{z}}}\Bigg\{\left\langle \boldsymbol{u}_{i},\boldsymbol{\theta}_{i}^{t}\right\rangle \sigma'''\left(w_{i}\right)\sigma\left(\left\langle \boldsymbol{\theta}_{i}^{t},\boldsymbol{S}\tilde{\boldsymbol{z}}+w_{i}\frac{\boldsymbol{\Sigma}^{2}\boldsymbol{u}_{i}}{\left\Vert \boldsymbol{\Sigma}\boldsymbol{u}_{i}\right\Vert _{2}^{2}}\right\rangle \right)\\
 & \qquad\times\Bigg[\left\langle \boldsymbol{S}\tilde{\boldsymbol{z}},\boldsymbol{a}\right\rangle \left\langle \boldsymbol{S}\tilde{\boldsymbol{z}},\boldsymbol{b}\right\rangle \left\langle \boldsymbol{S}\tilde{\boldsymbol{z}},\boldsymbol{c}\right\rangle +w_{i}\sum_{\left(\boldsymbol{v}_{1},\boldsymbol{v}_{2},\boldsymbol{v}_{3}\right)}\left\langle \boldsymbol{S}\tilde{\boldsymbol{z}},\boldsymbol{v}_{1}\right\rangle \left\langle \boldsymbol{S}\tilde{\boldsymbol{z}},\boldsymbol{v}_{2}\right\rangle \frac{\left\langle \boldsymbol{\Sigma}^{2}\boldsymbol{u}_{i},\boldsymbol{v}_{3}\right\rangle }{\left\Vert \boldsymbol{\Sigma}\boldsymbol{u}_{i}\right\Vert _{2}^{2}}\\
 & \qquad+w_{i}^{2}\sum_{\left(\boldsymbol{v}_{1},\boldsymbol{v}_{2},\boldsymbol{v}_{3}\right)}\left\langle \boldsymbol{S}\tilde{\boldsymbol{z}},\boldsymbol{v}_{1}\right\rangle \frac{\left\langle \boldsymbol{\Sigma}^{2}\boldsymbol{u}_{i},\boldsymbol{v}_{2}\right\rangle }{\left\Vert \boldsymbol{\Sigma}\boldsymbol{u}_{i}\right\Vert _{2}^{2}}\frac{\left\langle \boldsymbol{\Sigma}^{2}\boldsymbol{u}_{i},\boldsymbol{v}_{3}\right\rangle }{\left\Vert \boldsymbol{\Sigma}\boldsymbol{u}_{i}\right\Vert _{2}^{2}}+w_{i}^{3}\frac{\left\langle \boldsymbol{\Sigma}^{2}\boldsymbol{u}_{i},\boldsymbol{a}\right\rangle }{\left\Vert \boldsymbol{\Sigma}\boldsymbol{u}_{i}\right\Vert _{2}^{2}}\frac{\left\langle \boldsymbol{\Sigma}^{2}\boldsymbol{u}_{i},\boldsymbol{b}\right\rangle }{\left\Vert \boldsymbol{\Sigma}\boldsymbol{u}_{i}\right\Vert _{2}^{2}}\frac{\left\langle \boldsymbol{\Sigma}^{2}\boldsymbol{u}_{i},\boldsymbol{c}\right\rangle }{\left\Vert \boldsymbol{\Sigma}\boldsymbol{u}_{i}\right\Vert _{2}^{2}}\bigg]\Bigg\},
\end{align*}
where $\boldsymbol{S}_{i}=\boldsymbol{\Sigma}{\rm Proj}_{\boldsymbol{\Sigma}\boldsymbol{u}_{i}}^{\perp}$
for brevity and the summations are over $\boldsymbol{v}_{1},\boldsymbol{v}_{2},\boldsymbol{v}_{3}\in\left\{ \boldsymbol{a},\boldsymbol{b},\boldsymbol{c}\right\} $
with $\boldsymbol{v}_{1}$, $\boldsymbol{v}_{2}$, $\boldsymbol{v}_{3}$
being mutually different. Then by Lemma \ref{lem:2nd_setting_actGaussian},
along with the facts $\left\Vert \sigma\right\Vert _{\infty}\leq C$,
$\left\Vert \boldsymbol{S}\right\Vert _{{\rm op}}\leq\left\Vert \boldsymbol{\Sigma}\right\Vert _{{\rm op}}\leq C$
and $\left\Vert \boldsymbol{\Sigma}\boldsymbol{u}_{i}\right\Vert _{2}\geq C\left\Vert \boldsymbol{u}_{i}\right\Vert _{2}$,
we have:
\begin{align*}
 & \left|\left\langle \boldsymbol{M}_{4,i},\boldsymbol{a}\otimes\boldsymbol{b}\otimes\boldsymbol{c}\right\rangle \right|\\
 & \quad\leq C\kappa^{2}\mathbb{E}_{\tilde{\boldsymbol{z}}}\Bigg\{\frac{\left|\left\langle \boldsymbol{u}_{i},\boldsymbol{\theta}_{i}^{t}\right\rangle \right|}{\left\Vert \boldsymbol{\Sigma}\boldsymbol{u}_{i}\right\Vert _{2}}\Bigg[\left|\left\langle \boldsymbol{S}\tilde{\boldsymbol{z}},\boldsymbol{a}\right\rangle \left\langle \boldsymbol{S}\tilde{\boldsymbol{z}},\boldsymbol{b}\right\rangle \left\langle \boldsymbol{S}\tilde{\boldsymbol{z}},\boldsymbol{c}\right\rangle \right|+\sum_{\left(\boldsymbol{v}_{1},\boldsymbol{v}_{2},\boldsymbol{v}_{3}\right)}\left|\left\langle \boldsymbol{S}\tilde{\boldsymbol{z}},\boldsymbol{v}_{1}\right\rangle \left\langle \boldsymbol{S}\tilde{\boldsymbol{z}},\boldsymbol{v}_{2}\right\rangle \right|\frac{\left|\left\langle \boldsymbol{\Sigma}^{2}\boldsymbol{u}_{i},\boldsymbol{v}_{3}\right\rangle \right|}{\left\Vert \boldsymbol{\Sigma}\boldsymbol{u}_{i}\right\Vert _{2}}\\
 & \qquad+\sum_{\left(\boldsymbol{v}_{1},\boldsymbol{v}_{2},\boldsymbol{v}_{3}\right)}\left|\left\langle \boldsymbol{S}\tilde{\boldsymbol{z}},\boldsymbol{v}_{1}\right\rangle \right|\frac{\left|\left\langle \boldsymbol{\Sigma}^{2}\boldsymbol{u}_{i},\boldsymbol{v}_{2}\right\rangle \right|}{\left\Vert \boldsymbol{\Sigma}\boldsymbol{u}_{i}\right\Vert _{2}}\frac{\left|\left\langle \boldsymbol{\Sigma}^{2}\boldsymbol{u}_{i},\boldsymbol{v}_{3}\right\rangle \right|}{\left\Vert \boldsymbol{\Sigma}\boldsymbol{u}_{i}\right\Vert _{2}}+\frac{\left|\left\langle \boldsymbol{\Sigma}^{2}\boldsymbol{u}_{i},\boldsymbol{a}\right\rangle \right|}{\left\Vert \boldsymbol{\Sigma}\boldsymbol{u}_{i}\right\Vert _{2}}\frac{\left|\left\langle \boldsymbol{\Sigma}^{2}\boldsymbol{u}_{i},\boldsymbol{b}\right\rangle \right|}{\left\Vert \boldsymbol{\Sigma}\boldsymbol{u}_{i}\right\Vert _{2}}\frac{\left|\left\langle \boldsymbol{\Sigma}^{2}\boldsymbol{u}_{i},\boldsymbol{c}\right\rangle \right|}{\left\Vert \boldsymbol{\Sigma}\boldsymbol{u}_{i}\right\Vert _{2}}\bigg]\Bigg\}\\
 & \quad\leq C\kappa^{2}\left\Vert \boldsymbol{\theta}_{i}^{t}\right\Vert _{2}\left\Vert \boldsymbol{a}\right\Vert _{2}\left\Vert \boldsymbol{b}\right\Vert _{2}\left\Vert \boldsymbol{c}\right\Vert _{2}.
\end{align*}
That is, $\left\Vert \boldsymbol{M}_{4,i}\right\Vert _{{\rm op}}\leq C\kappa^{2}\left\Vert \boldsymbol{\theta}_{i}^{t}\right\Vert _{2}$.

\paragraph*{Step 4: Finishing the proof.}

From the bounds on $\left\Vert \boldsymbol{M}_{1,i}\right\Vert _{{\rm op}}$
and $\left\Vert \boldsymbol{M}_{4,i}\right\Vert _{{\rm op}}$ and
Eq. (\ref{eq:prop_2ndSetting_supUnorm_gap}) , we get:
\begin{align*}
\left|Q-Q_{\gamma}\right| & \leq\frac{C}{N}\sum_{i=1}^{N}\kappa^{2}\gamma\left\Vert \boldsymbol{\theta}_{i}^{t}\right\Vert _{2}\sqrt{N}+\frac{1}{N}\sum_{i=1}^{N}C\kappa^{2}\left(r_{1,i}^{2}+r_{2,i}^{2}+1\right)\gamma\\
 & \leq\frac{C}{N}\sum_{i=1}^{N}\kappa^{2}\gamma\left(r_{1,i}+r_{2,i}+1\right)\sqrt{N}+\frac{1}{N}\sum_{i=1}^{N}C\kappa^{2}\left(r_{1,i}^{2}+r_{2,i}^{2}+1\right)\gamma\leq C\kappa^{2}\gamma\left(\sqrt{NA}+A+\sqrt{N}\right),\\
A & =\frac{1}{N}\sum_{i=1}^{N}\left(r_{1,i}^{2}+r_{2,i}^{2}\right).
\end{align*}
Recall that $r_{1,i}^{2}+r_{2,i}^{2}$ is $C$-sub-exponential. Then
by Lemma \ref{lem:subgauss_subexp_properties}, for $\delta\in\left(0,1\right)$,
$\mathbb{P}\left\{ A\geq C_{1}\left(1+\delta\right)\right\} \leq C\exp\left(-CN\delta^{2}\right)$,
where $C_{1}=\int\left(r_{1}^{2}+r_{2}^{2}\right){\rm d}\rho_{r}\leq C$.
Furthermore, since $\left(\psi_{t}\left(r_{1,i},r_{2,i}\right)\right)_{1}$
and $\left(\psi_{t}\left(r_{1,i},r_{2,i}\right)\right)_{2}$ are $C$-sub-Gaussian,
using Lemma \ref{lem:Unorm_2ndSetting} and the union bound, we obtain
for sufficiently large $C_{*}$,
\[
\mathbb{P}\left\{ Q_{\gamma}\geq C_{*}\right\} \leq\left|{\cal N}_{d}\left(\gamma\right)\right|\left|{\cal N}\left(\gamma\right)\right|C\exp\left(Cd-CN/\kappa^{4}\right)\leq\left(\frac{3}{\gamma}\right)^{d+1}C\exp\left(Cd-CN/\kappa^{4}\right).
\]
Let us choose $\gamma=1/\left(4\kappa^{2}\sqrt{N}\right)<1/3$ and
$\delta=0.5$. Then for sufficiently large $C_{*}$,
\begin{align*}
\mathbb{P}\left\{ Q\geq C_{*}\right\}  & \leq C\exp\left(-CN\right)+\left(C\kappa^{2}\sqrt{N}\right)^{d+1}C\exp\left(Cd-CN/\kappa^{4}\right)\\
 & \leq C\exp\left(Cd\log\left(\kappa^{2}\sqrt{N}+e\right)-CN/\kappa^{4}\right).
\end{align*}
This completes the proof.
\end{proof}

\appendix

\section{Technical lemmas}

\subsection{Sub-Gaussian and sub-exponential random variables\label{subsec:Sub-Gaussian-RV}}

We recall the Orlicz norms for a real-valued random variable $X$:
\[
\left\Vert X\right\Vert _{\psi_{2}}=\sup_{p\geq1}\frac{1}{\sqrt{p}}\mathbb{E}\left\{ \left|X\right|^{p}\right\} ^{1/p},\qquad\left\Vert X\right\Vert _{\psi_{1}}=\sup_{p\geq1}\frac{1}{p}\mathbb{E}\left\{ \left|X\right|^{p}\right\} ^{1/p}.
\]
A real-valued random variable $X$ is $K$-sub-Gaussian if $K=\left\Vert X\right\Vert _{\psi_{2}}$
is finite. It is $K$-sub-exponential if $K=\left\Vert X\right\Vert _{\psi_{1}}$
is finite. A random vector $\boldsymbol{X}$ is $K$-sub-Gaussian
if $\left\langle \boldsymbol{v},\boldsymbol{X}\right\rangle $ is
sub-Gaussian for any $\boldsymbol{v}\in\mathbb{S}^{d-1}$, and in
particular, $K=\sup_{\boldsymbol{v}\in\mathbb{S}^{d-1}}\left\Vert \left\langle \boldsymbol{v},\boldsymbol{X}\right\rangle \right\Vert _{\psi_{2}}<\infty$.

We summarize the following well-known facts about sub-Gaussian and
sub-exponential random variables \cite{vershynin2010introduction}:
\begin{lem}
\label{lem:subgauss_subexp_properties}The following properties hold:
\begin{itemize}
\item $X$ is $K$-sub-Gaussian if and only if there exists a constant $K_{0}$
that differs from $K$ by at most an absolute constant factor, such
that $\mathbb{P}\left\{ \left|X\right|>t\right\} \leq\exp\left(1-t^{2}/K_{0}^{2}\right)$
for all $t\geq0$.
\item $X$ is $K$-sub-exponential if and only if there exists a constant
$K_{0}$ that differs from $K$ by at most an absolute constant factor,
such that $\mathbb{P}\left\{ \left|X\right|>t\right\} \leq\exp\left(1-t/K_{0}\right)$
for all $t\geq0$.
\item For two sub-Gaussian random variables $X$ and $Y$, their sum $X+Y$
is sub-Gaussian with $\psi_{2}$-norm $\left\Vert X+Y\right\Vert _{\psi_{2}}\leq\left\Vert X\right\Vert _{\psi_{2}}+\left\Vert Y\right\Vert _{\psi_{2}}$.
Likewise, if they are sub-exponential, their sum is sub-exponential
with norm $\left\Vert X+Y\right\Vert _{\psi_{1}}\leq\left\Vert X\right\Vert _{\psi_{1}}+\left\Vert Y\right\Vert _{\psi_{1}}$.
\item For two sub-Gaussian random variables $X$ and $Y$, their product
$XY$ is sub-exponential with $\psi_{1}$-norm $\left\Vert XY\right\Vert _{\psi_{1}}\leq\left\Vert X\right\Vert _{\psi_{2}}\left\Vert Y\right\Vert _{\psi_{2}}$.
\item If $X$ is sub-exponential with zero mean and $\left\Vert X\right\Vert _{\psi_{1}}\leq K$,
then for any $t$ such that $\left|t\right|\leq c/K$, $\mathbb{E}\left\{ e^{tX}\right\} \leq e^{Ct^{2}K^{2}}$
for some absolute constants $C,c>0$.
\item Let $X_{1},...,X_{n}$ be independent sub-Gaussian random variables
with zero mean, and let $K=\max_{i\in\left[n\right]}\left\Vert X_{i}\right\Vert _{\psi_{2}}$.
Then for any $t\geq0$,
\[
\mathbb{P}\left\{ \left|\sum_{i=1}^{n}X_{i}\right|\geq tn\right\} \leq e\cdot\exp\left(-\frac{cnt^{2}}{K^{2}}\right),
\]
for an absolute constant $c>0$.
\item Let $X_{1},...,X_{n}$ be independent sub-exponential random variables
with zero mean, and let $K=\max_{i\in\left[n\right]}\left\Vert X_{i}\right\Vert _{\psi_{1}}$.
Then for any $t\geq0$,
\[
\mathbb{P}\left\{ \left|\sum_{i=1}^{n}X_{i}\right|\geq tn\right\} \leq2\exp\left(-cn\min\left(\frac{t^{2}}{K^{2}},\frac{t}{K}\right)\right),
\]
for an absolute constant $c>0$.
\end{itemize}
\end{lem}

We also have the following martingale concentration result for sub-exponential
martingale difference:
\begin{lem}
\label{lem:azuma}Let $\left(X^{k}\right)_{k\geq0}$ be a real-valued
martingale w.r.t. the filtration $\left({\cal F}^{k}\right)_{k\geq0}$
with $X^{0}=0$. Suppose that the martingale difference $X^{k}-X^{k-1}$,
conditioned on ${\cal F}^{k-1}$, is $K$-sub-exponential with zero
mean. Then:
\[
\mathbb{P}\left\{ \max_{k\leq n}\left|X^{k}\right|\geq c_{1}K\sqrt{n}\delta\right\} \leq2\exp\left(-\delta^{2}\right),
\]
for $\delta\leq c_{2}\sqrt{n}$, for some $c_{1},c_{2}>0$ absolute
constants.
\end{lem}

\begin{proof}
We have for $t>0$ and $t$ such that $\left|t\right|\leq c/K$,
\[
\mathbb{E}\left\{ e^{t\left(X^{k}-X^{k-1}\right)}\middle|{\cal F}^{k-1}\right\} \leq e^{Ct^{2}K^{2}},
\]
for some absolute constants $C,c>0$ by Lemma \ref{lem:subgauss_subexp_properties}.
This results in the recursive relation:
\[
\mathbb{E}\left\{ e^{tX^{k}}\right\} =\mathbb{E}\left\{ e^{tX^{k-1}}\mathbb{E}\left\{ e^{t\left(X^{k}-X^{k-1}\right)}\middle|{\cal F}^{k-1}\right\} \right\} \leq\mathbb{E}\left\{ e^{tX^{k-1}}\right\} e^{Ct^{2}K^{2}},
\]
which implies
\[
\mathbb{E}\left\{ e^{tX^{n}}\right\} \leq e^{Ct^{2}K^{2}n}.
\]
A standard argument yields a tail bound on $\mathbb{P}\left\{ \left|X_{n}\right|\geq n\delta\right\} $.
In particular, by Markov's inequality, for $\delta>0$,
\[
\mathbb{P}\left\{ X^{n}\geq n\delta\right\} \leq\inf_{t\in\left[0,c/K\right]}e^{-n\delta t}\mathbb{E}\left\{ e^{tX^{n}}\right\} \leq\inf_{t\in\left[0,c/K\right]}e^{Ct^{2}K^{2}n-n\delta t}\leq\exp\left(-n\min\left(\frac{\delta^{2}}{4CK^{2}},\frac{c\delta}{K}\right)\right).
\]
The same argument yields the same bound for $\mathbb{P}\left\{ -X_{n}\geq n\delta\right\} $.
Then:
\[
\mathbb{P}\left\{ \left|X^{n}\right|\geq n\delta\right\} \leq2\exp\left(-n\min\left(\frac{\delta^{2}}{4CK^{2}},\frac{c\delta}{K}\right)\right).
\]
Define the stopping time $T=\min\left\{ k:\;\left|X^{k}\right|\geq n\delta\right\} $
and the martingale $\bar{X}^{k}=X^{k\wedge T}$. Since $\max_{k\leq n}\left|X^{k}\right|\geq n\delta$
if and only if $\bar{X}^{n}\geq n\delta$, the same bound applies
to $\max_{k\leq n}\left|X^{k}\right|$. Finally, defining $z=\sqrt{n\delta^{2}/\left(4CK^{2}\right)}$,
for $z\leq\sqrt{4nc^{2}C}$, we have:
\[
\mathbb{P}\left\{ \max_{k\leq n}\left|X^{k}\right|\geq\sqrt{4CK^{2}n}z\right\} \leq2\exp\left(-z^{2}\right).
\]
This completes the proof.
\end{proof}
The following lemma provides an estimate on the expected norm of sub-exponential
random vector:
\begin{lem}
\label{lem:subexp_vector_norm_E}Let $\boldsymbol{X}$ be a sub-exponential
random vector in $\mathbb{R}^{d}$ with $\left\Vert \boldsymbol{X}\right\Vert _{\psi_{1}}\leq K$
and $\mathbb{E}\left\{ \boldsymbol{X}\right\} =\boldsymbol{0}$. Then
for some sufficiently large constant $C$ that does not depend on
$d$ or $K$,
\[
\mathbb{E}\left\{ \left\Vert \boldsymbol{X}\right\Vert _{2}^{2}\right\} \leq C\left(d^{2}K^{2}+1\right).
\]
\end{lem}

\begin{proof}
To compute $\mathbb{E}\left\{ \left\Vert \boldsymbol{X}\right\Vert _{2}^{2}\right\} $,
we first provide a tail bound on $\mathbb{P}\left\{ \left\Vert \boldsymbol{X}\right\Vert _{2}\geq\delta\right\} $.
Consider an epsilon-net ${\cal N}\subset\mathbb{S}^{d-1}$ such that
for any $\boldsymbol{u}\in\mathbb{S}^{d-1}$, there exists $\boldsymbol{u}'\in{\cal N}$
with $\left\Vert \boldsymbol{u}-\boldsymbol{u}'\right\Vert _{2}\leq1/2$.
There exists such an epsilon-net \cite{vershynin2010introduction}
with size $\left|{\cal N}\right|\leq6^{d}$. For $\boldsymbol{u}\in\mathbb{S}^{d-1}$,
let $\hat{\boldsymbol{u}}\left(\boldsymbol{u}\right)\in{\cal N}$
be such that $\left\Vert \boldsymbol{u}-\hat{\boldsymbol{u}}\left(\boldsymbol{u}\right)\right\Vert _{2}\leq1/2$.
Then:
\begin{align*}
\left\Vert \boldsymbol{X}\right\Vert _{2} & =\sup_{\boldsymbol{u}\in\mathbb{S}^{d-1}}\left\langle \boldsymbol{u},\boldsymbol{X}\right\rangle =\sup_{\boldsymbol{u}\in\mathbb{S}^{d-1}}\left(\left\langle \boldsymbol{u}-\hat{\boldsymbol{u}}\left(\boldsymbol{u}\right),\boldsymbol{X}\right\rangle +\left\langle \hat{\boldsymbol{u}}\left(\boldsymbol{u}\right),\boldsymbol{X}\right\rangle \right)\\
 & \leq\frac{1}{2}\sup_{\boldsymbol{u}\in\mathbb{S}^{d-1}}\left\langle \boldsymbol{u},\boldsymbol{X}\right\rangle +\sup_{\boldsymbol{u}\in{\cal N}}\left\langle \boldsymbol{u},\boldsymbol{X}\right\rangle =\frac{1}{2}\left\Vert \boldsymbol{X}\right\Vert _{2}+\sup_{\boldsymbol{u}\in{\cal N}}\left\langle \boldsymbol{u},\boldsymbol{X}\right\rangle ,
\end{align*}
and hence $\left\Vert \boldsymbol{X}\right\Vert _{2}\leq2\sup_{\boldsymbol{u}\in{\cal N}}\left\langle \boldsymbol{u},\boldsymbol{X}\right\rangle $.
Now fix a vector $\boldsymbol{u}\in{\cal N}$. Since $\left\langle \boldsymbol{u},\boldsymbol{X}\right\rangle $
has zero mean and $\left\Vert \left\langle \boldsymbol{u},\boldsymbol{X}\right\rangle \right\Vert _{\psi_{1}}\leq K$,
by Lemma \ref{lem:subgauss_subexp_properties}, for any $t$ such
that $\left|t\right|\leq c_{1}/K$, $\mathbb{E}\left\{ e^{t\left\langle \boldsymbol{u},\boldsymbol{X}\right\rangle }\right\} \leq e^{c_{2}t^{2}K^{2}}$
for some absolute constants $c_{1},c_{2}>0$. By Markov's inequality,
for $\delta\geq0$,
\[
\mathbb{P}\left\{ \left\langle \boldsymbol{u},\boldsymbol{X}\right\rangle \geq\delta\right\} \leq\inf_{t\in\left[0,c_{1}/K\right]}e^{-\delta t}\mathbb{E}\left\{ e^{t\left\langle \boldsymbol{u},\boldsymbol{X}\right\rangle }\right\} \leq\inf_{t\in\left[0,c_{1}/K\right]}e^{c_{2}t^{2}K^{2}-\delta t}\leq\exp\left(-\min\left(\frac{\delta^{2}}{4c_{2}K^{2}},\frac{c_{1}\delta}{K}\right)\right).
\]
The same argument yields the same bound for $\mathbb{P}\left\{ -\left\langle \boldsymbol{u},\boldsymbol{X}\right\rangle \geq\delta\right\} $.
Then:
\[
\mathbb{P}\left\{ \left|\left\langle \boldsymbol{u},\boldsymbol{X}\right\rangle \right|\geq\delta\right\} \leq2\exp\left(-\min\left(\frac{\delta^{2}}{4c_{2}K^{2}},\frac{c_{1}\delta}{K}\right)\right).
\]
By the union bound,
\[
\mathbb{P}\left\{ \left\Vert \boldsymbol{X}\right\Vert _{2}\geq\delta\right\} \leq2\exp\left(d\log6-\min\left(\frac{\delta^{2}}{16c_{2}K^{2}},\frac{c_{1}\delta}{2K}\right)\right).
\]
Now to compute $\mathbb{E}\left\{ \left\Vert \boldsymbol{X}\right\Vert _{2}^{2}\right\} $,
observe that firstly $\mathbb{P}\left\{ \left\Vert \boldsymbol{X}\right\Vert _{2}\geq\delta\right\} \leq1$
trivially, and secondly, if $\delta\geq8\left(c_{3}dK\log6\right)/c_{1}$
for $c_{3}\geq\max\left\{ 1,c_{1}^{2}c_{2}\log6\right\} \geq\left(c_{1}^{2}c_{2}\log6\right)/d$,
then we have 
\[
\min\left(\frac{\delta^{2}}{16c_{2}K^{2}},\frac{c_{1}\delta}{2K}\right)=\frac{c_{1}\delta}{2K},\qquad d\log6-\frac{c_{1}\delta}{2K}\leq-\frac{3c_{1}\delta}{8K}.
\]
Therefore we have:
\begin{align*}
\mathbb{E}\left\{ \left\Vert \boldsymbol{X}\right\Vert _{2}^{2}\right\}  & =\int_{0}^{\infty}\mathbb{P}\left\{ \left\Vert \boldsymbol{X}\right\Vert _{2}^{2}\geq t\right\} {\rm d}t=2\int_{0}^{\infty}\mathbb{P}\left\{ \left\Vert \boldsymbol{X}\right\Vert _{2}\geq\delta\right\} \delta{\rm d}\delta\\
 & \leq2\int_{0}^{8\left(c_{3}dK\log6\right)/c_{1}}\delta{\rm d}\delta+4\int_{8\left(c_{3}dK\log6\right)/c_{1}}^{\infty}\exp\left(-\frac{3c_{1}\delta}{8K}\right)\delta{\rm d}\delta\\
 & =\frac{64c_{3}^{2}d^{2}K^{2}\log^{2}6}{c_{1}^{2}}+\frac{256K^{2}}{9c_{1}^{2}}\left(3c_{3}d\log6+1\right)e^{-3c_{3}d\log6}\\
 & \leq C\left(d^{2}K^{2}+1\right),
\end{align*}
for some sufficiently large $C$ that depends only on $c_{1}$ and
$c_{3}$.
\end{proof}

\subsection{Moment controls}

We have the following control on the moments of the norm of the average
of (almost) independent random vectors:
\begin{lem}
\label{lem:bound_moment_sym}Consider a random variable $X$ and a
sequence of random vectors $\left(\boldsymbol{a}_{j}^{X}\right)_{j\leq N}$.
Assume $\left(\boldsymbol{a}_{j}^{X}\right)_{j\leq N}$ are independent
conditionally on $X$, $\mathbb{E}\left\{ \boldsymbol{a}_{j}^{X}\middle|X\right\} =\boldsymbol{0}$,
and $\mathbb{E}\left\{ \left\Vert \boldsymbol{a}_{j}^{X}\right\Vert _{2}^{2p}\right\} \leq K$
for all $j\in\left[N\right]$, for some positive integer $p$ and
constant $K$. Then:
\[
\mathbb{E}\left\{ \left\Vert \frac{1}{N}\sum_{j=1}^{N}\boldsymbol{a}_{i}^{X}\right\Vert _{2}^{2p}\right\} \leq4^{p}\left(2p\right)!\frac{K}{N^{p}}\leq16^{p}p^{2p}\frac{K}{N^{p}}.
\]
In fact, the same statement holds for $\left(\boldsymbol{a}_{j}^{X}\right)_{j\leq N}$
defined on a Hilbert space, equipped with an inner product $\left\langle \cdot,\cdot\right\rangle $
and an induced norm $\left\Vert \cdot\right\Vert _{2}$.
\end{lem}

\begin{proof}
We use a symmetrization argument. Define $\left(\varepsilon_{j}\right)_{j\leq N}$
being i.i.d. Bernoulli $\pm1$ random variables, independent of everything
else. Since $\mathbb{E}\left\{ \boldsymbol{a}_{j}^{X}\middle|X\right\} =\boldsymbol{0}$
and $\left(\boldsymbol{a}_{j}^{X}\right)_{j\leq N}$ are independent
conditionally on $X$, we have the following symmetrization fact \cite[Lemma 6.3]{ledoux2013probability}:
\begin{equation}
\mathbb{E}\left\{ \left\Vert \sum_{j=1}^{N}\boldsymbol{a}_{j}^{X}\right\Vert _{2}^{2p}\right\} \leq4^{p}\mathbb{E}\left\{ \left\Vert \sum_{j=1}^{N}\boldsymbol{b}_{j}^{X}\right\Vert _{2}^{2p}\right\} ,\label{eq:lem_bound_moment_sym_1}
\end{equation}
in which $\boldsymbol{b}_{j}^{X}=\varepsilon_{j}\boldsymbol{a}_{j}^{X}$.
We note that $\left\Vert \sum_{j=1}^{N}\boldsymbol{b}_{j}^{X}\right\Vert _{2}^{2p}$
is a sum of $N^{2p}$ terms of the form $\prod_{h=1}^{p}\left\langle \boldsymbol{b}_{h},\boldsymbol{b}_{2h}\right\rangle $,
where $\boldsymbol{b}_{h}\in\left\{ \boldsymbol{b}_{j}^{X}\right\} _{j\leq N}$
for $h=1,...,2p$. Consider a term $H$ that has $q_{j}$ appearances
of $\boldsymbol{b}_{j}^{X}$ for $j\in J_{H}\subseteq\left[N\right]$,
where $\sum_{j\in J_{H}}q_{j}=2p$. We have by Holder's inequality,
\begin{align*}
\left|\mathbb{E}\left\{ H\right\} \right| & \leq\mathbb{E}\left\{ \prod_{j\in J_{H}}\left\Vert \boldsymbol{b}_{j}^{X}\right\Vert _{2}^{q_{j}}\right\} \leq\prod_{j\in J_{H}}\mathbb{E}\left\{ \left\Vert \boldsymbol{b}_{j}^{X}\right\Vert _{2}^{2p}\right\} ^{q_{j}/\left(2p\right)}\\
 & =\prod_{j\in J_{H}}\mathbb{E}\left\{ \left\Vert \boldsymbol{a}_{j}^{X}\right\Vert _{2}^{2p}\right\} ^{q_{j}/\left(2p\right)}\leq\prod_{j\in J_{H}}K^{q_{j}/\left(2p\right)}=K.
\end{align*}
Notice that the above upper bound is the same for all terms. Furthermore
if there is $j\in J_{H}$ such that $q_{j}$ is odd, then $\mathbb{E}\left\{ H\middle|X\right\} =0$,
thanks to the randomness of $\varepsilon_{j}$. Hence we only need
to upper bound the number of terms $H$ such that there is no $j\in J_{H}$
with odd $q_{j}$. Let us call this number $N_{*}$. To bound $N_{*}$,
we consider the following construction of each desired term. As the
first step, we select $\boldsymbol{b}_{h}$ from the set $\left\{ \boldsymbol{b}_{j}^{X}\right\} _{j\leq N}$
for $h=1,...,p$, and we set $\boldsymbol{b}_{2h}=\boldsymbol{b}_{h}$.
Then in the second step, we construct the desired term as $\prod_{h=1}^{p}\left\langle \boldsymbol{b}_{\Pi\left(h\right)},\boldsymbol{b}_{\Pi\left(2h\right)}\right\rangle $,
where $\Pi:\;\left[2p\right]\to\left[2p\right]$ is any permutation.
This procedure guarantees to construct all desired terms, with some
being repeated. Note that the number of possibilities for the first
step is $N^{p}$, and in the second step, the number of permutations
is $\left(2p\right)!$. Hence we obtain $N_{*}\leq\left(2p\right)!N^{p}$.
Therefore, by Eq. (\ref{eq:lem_bound_moment_sym_1}),
\[
\mathbb{E}\left\{ \left\Vert \sum_{j=1}^{N}\boldsymbol{a}_{j}^{X}\right\Vert _{2}^{2p}\right\} \leq4^{p}\left(2p\right)!KN^{p},
\]
which completes the proof.
\end{proof}
The above result presents a simple approach to concentration for powers
of sub-Gaussian random variables:
\begin{lem}
\label{lem:concen_subgauss_sum_q_power}Let $\left(X_{i}\right)_{i\geq0}$
be independent real-valued $K$-sub-Gaussian random variables. Then
for any $q\geq1$,
\[
\mathbb{P}\left\{ \left|\frac{1}{N}\sum_{i=1}^{N}\left|X_{i}\right|^{q}-\mathbb{E}\left\{ \left|X_{i}\right|^{q}\right\} \right|\geq\delta\right\} \leq C\exp\left(-\frac{C^{1/\left(2+q\right)}N^{1/\left(2+q\right)}\delta^{2/\left(2+q\right)}}{K^{2q/\left(2+q\right)}}\right),
\]
where the constant $C$ does not depend on $q$ or $K$.
\end{lem}

\begin{proof}
Let $Y_{i}=\left|X_{i}\right|^{q}-\mathbb{E}\left\{ \left|X_{i}\right|^{q}\right\} $
and $S=\left(1/N\right)\cdot\sum_{i=1}^{N}Y_{i}$. We have for any
positive integer $p$,
\[
\mathbb{E}\left\{ \left|Y_{i}\right|^{2p}\right\} \leq4^{p}\mathbb{E}\left\{ \left|X_{i}\right|^{2pq}\right\} \leq4^{p}K^{2pq}p^{pq}.
\]
By Lemma \ref{lem:bound_moment_sym}, $\mathbb{E}\left\{ \left|S\right|^{2p}\right\} \leq C^{p}K^{2pq}p^{\left(2+q\right)p}/N^{p}$,
which implies that $\left|S\right|^{2/\left(2+q\right)}$ is sub-exponential
with $\left\Vert \left|S\right|^{2/\left(2+q\right)}\right\Vert _{\psi_{1}}\leq C^{1/\left(2+q\right)}K^{2q/\left(2+q\right)}N^{-1/\left(2+q\right)}$.
Therefore, by Lemma \ref{lem:subgauss_subexp_properties},
\[
\mathbb{P}\left\{ \left|S\right|\geq\delta\right\} \leq C\exp\left(-\frac{C^{1/\left(2+q\right)}N^{1/\left(2+q\right)}\delta^{2/\left(2+q\right)}}{K^{2q/\left(2+q\right)}}\right).
\]
\end{proof}

\section{Simulation details\label{sec:Simulation-details}}

\subsection{Simplifications for the setting with bounded activation (Setting
\ref{enu:bdd_act_setting})\label{subsec:Simplifications-for-Setting-bdd-act}}

We make further simplifications of the ODEs (\ref{eq:2nd_setting_ODE_r}).
In particular, we consider large dimension $d\gg1$, while keeping
$\alpha$ a fixed constant. Let $\alpha_{1}=\alpha$ and $\alpha_{2}=1-\alpha$.
In this case, for $Z_{1}$ and $Z_{2}$ being respectively $\chi$-random
variables of degrees of freedom $d_{1}$ and $d_{2}$, we have $Z_{1}\approx\sqrt{d_{1}}$
and $Z_{2}\approx\sqrt{d_{2}}$. Consequently at initialization, $\rho_{r}^{0}\approx\delta_{\left(r_{0}\sqrt{\alpha_{1}},r_{0}\sqrt{\alpha_{2}}\right)}$,
which implies that $\rho_{r}^{t}\approx\delta_{\check{r}_{1,t},\check{r}_{2,t}}$
concentrating at a point mass at all time $t\geq0$. Hence instead
of solving for the exact distribution of $r_{1,t}$ and $r_{2,t}$,
we can make approximations by keeping track of two scalars $\check{r}_{1,t}$
and $\check{r}_{2,t}$. Their evolutions are given by the following:
\begin{align*}
\frac{{\rm d}}{{\rm d}t}\check{r}_{j,t} & =-\check{\Delta}_{j}\left(\check{r}_{1,t},\check{r}_{2,t}\right)\left[\check{q}_{j}\left(\Sigma_{1}\sqrt{\alpha_{1}}\check{r}_{1,t},\Sigma_{2}\sqrt{\alpha_{2}}\check{r}_{2,t}\right)+\Sigma_{j}\sqrt{\alpha_{j}}\check{r}_{j,t}\partial_{j}\check{q}_{j}\left(\Sigma_{1}\sqrt{\alpha_{1}}\check{r}_{1,t},\Sigma_{2}\sqrt{\alpha_{2}}\check{r}_{2,t}\right)\right]\\
 & \qquad-\check{\Delta}_{\neg j}\left(\check{r}_{1,t},\check{r}_{2,t}\right)\Sigma_{j}\sqrt{\alpha_{j}}\check{r}_{\neg j,t}\partial_{j}\check{q}_{\neg j}\left(\Sigma_{1}\sqrt{\alpha_{1}}\check{r}_{1,t},\Sigma_{2}\sqrt{\alpha_{2}}\check{r}_{2,t}\right)-2\lambda\check{r}_{j,t},\qquad j=1,2,
\end{align*}
in which we define:
\begin{align*}
\check{q}_{1}\left(a,b\right) & =\frac{a}{\alpha_{1}}\mathbb{E}_{g}\left\{ \sigma'\left(\sqrt{\frac{a^{2}}{\alpha_{1}}+\frac{b^{2}}{\alpha_{2}}}g\right)\right\} ,\\
\check{q}_{2}\left(a,b\right) & =\frac{b}{\alpha_{2}}\mathbb{E}_{g}\left\{ \sigma'\left(\sqrt{\frac{a^{2}}{\alpha_{1}}+\frac{b^{2}}{\alpha_{2}}}g\right)\right\} ,\\
\check{\Delta}_{j}\left(r_{1},r_{2}\right) & =r_{j}\check{q}_{j}\left(\Sigma_{1}\sqrt{\alpha_{1}}r_{1},\Sigma_{2}\sqrt{\alpha_{2}}r_{2}\right)-\Sigma_{j}\sqrt{\alpha_{j}},\qquad j=1,2,
\end{align*}
and we initialize $\check{r}_{j,0}=r_{0}\sqrt{\alpha_{j}}$. This
is a system of two deterministic ODEs and can be solved numerically.
We also obtain an approximation of ${\cal R}\left(\rho_{N}^{t/\epsilon}\right)$:
\[
{\cal R}\left(\rho_{N}^{t/\epsilon}\right)\approx\frac{1}{2}\sum_{j\in\left\{ 1,2\right\} }\left(\check{\Delta}_{j}\left(\check{r}_{1,t},\check{r}_{2,t}\right)\right)^{2}.
\]

To approximate the reconstruction error with respect to a different
distribution ${\cal Q}$ in Fig. \ref{fig:tanh_twoblks_loss_outsample},
one can do the same simplification and obtain:
\[
\mathbb{E}_{\boldsymbol{x}\sim{\cal Q}}\left\{ \frac{1}{2}\left\Vert \hat{\boldsymbol{x}}_{N}\left(\boldsymbol{x};\Theta^{t/\epsilon}\right)-\boldsymbol{x}\right\Vert _{2}^{2}\right\} \approx\frac{1}{2}\sum_{j\in\left\{ 1,2\right\} }\left(\check{\Delta}_{j}^{{\cal Q}}\left(\check{r}_{1,t},\check{r}_{2,t}\right)\right)^{2},
\]
in which
\begin{align*}
\check{\Delta}_{j}^{{\cal Q}}\left(r_{1},r_{2}\right) & =r_{j}\check{q}_{j}\left(\Sigma_{1,{\cal Q}}\sqrt{\alpha_{1}}r_{1},\Sigma_{2,{\cal Q}}\sqrt{\alpha_{2}}r_{2}\right)-\Sigma_{j,{\cal Q}}\sqrt{\alpha_{j}},\qquad j=1,2.
\end{align*}

\subsection{Further simulation details}

We describe several additional details that were omitted from the
captions of Fig. \ref{fig:ReLU_twoblks_noreg_loss_theta}-\ref{fig:MNIST_noreg_loss_theta}:
\begin{itemize}
\item In the settings of Fig. \ref{fig:ReLU_twoblks_noreg_loss_theta}-\ref{fig:tanh_twoblks_loss_theta},
the data covariance $\boldsymbol{\Sigma}^{2}$ has two subspaces of
dimensions $d_{1}$ and $d_{2}=d-d_{1}$, each corresponding to $\boldsymbol{\theta}_{i,1:d_{1}}^{k}\in\mathbb{R}^{d_{1}}$
(the first $d_{1}$ coordinates of $\boldsymbol{\theta}_{i}^{k}$)
and $\boldsymbol{\theta}_{i,\left(d_{1}+1\right):d}^{k}\in\mathbb{R}^{d_{2}}$
(the last $d_{2}$ coordinates of $\boldsymbol{\theta}_{i}^{k}$).
We compute the normalized squared norms of the first subspace's weights
$\left(\boldsymbol{\theta}_{i,1:d_{1}}^{k}\right)_{i\leq N}$ and
the second subspace's weights $\left(\boldsymbol{\theta}_{i,\left(d_{1}+1\right):d}^{k}\right)_{i\leq N}$
as respectively
\[
\frac{d}{d_{1}N}\sum_{i=1}^{N}\left\Vert \boldsymbol{\theta}_{i,1:d_{1}}^{k}\right\Vert _{2}^{2},\qquad\frac{d}{d_{2}N}\sum_{i=1}^{N}\left\Vert \boldsymbol{\theta}_{i,\left(d_{1}+1\right):d}^{k}\right\Vert _{2}^{2}.
\]
\item In Fig. \ref{fig:tanh_twoblks_loss_theta}, we assume the simplifications
in Appendix \ref{subsec:Simplifications-for-Setting-bdd-act} to solve
numerically the MF limiting dynamics.
\item For efficiency, we adopt the following practices in all simulations.
Firstly, we use mini-batch SGD with a batch size of 100. While this
is strictly not covered by our theory, we note that the use of a larger
batch size has the advantage of accommodating larger learning rate
$\epsilon$, while leaving the MF limiting dynamics unaltered. Secondly,
for simulations on the real data set, at each SGD iteration, we select
the mini-batch from the training set without replacement; once the
training set is scanned through, we randomly re-shuffle the training
set.
\item For Gaussian data, to estimate the statistics (such as the reconstruction
error), we perform Monte-Carlo averaging over $10^{4}$ random samples.
\item In Fig. \ref{fig:ReLU_twoblks_subsampled} and \ref{fig:MNIST_reg_subsampled},
each point on the plot is an average over 20 independent repeats of
the two-staged process for derived autoencoders.
\item In Fig. \ref{fig:MNIST_reg_loss_theta}, \ref{fig:MNIST_reg_subsampled}
and \ref{fig:MNIST_noreg_loss_theta}, on the MNIST data set, we train
on a training set of size $6\times10^{4}$ and compute all the plotted
statistics on the test set of size $10^{4}$. Each MNIST image has
size $d=28\times28=784$. To preprocess the data, we compute:
\[
\hat{\boldsymbol{\mu}}=\frac{1}{6\times10^{4}}\sum_{i\text{ in training set}}\bar{\boldsymbol{x}}_{i},\qquad\hat{\boldsymbol{S}}=\frac{1}{6\times10^{4}}\sum_{i\text{ in training set}}\left(\bar{\boldsymbol{x}}_{i}-\hat{\boldsymbol{\mu}}\right)\left(\bar{\boldsymbol{x}}_{i}-\hat{\boldsymbol{\mu}}\right)^{\top},
\]
where $\bar{\boldsymbol{x}}_{i}$ is the original MNIST image with
the pixel range $\left[0,1\right]$. Let $\hat{\boldsymbol{S}}=\boldsymbol{U}\bar{\boldsymbol{C}}\boldsymbol{U}^{\top}$
be its singular value decomposition. Its spectrum is plotted in Fig.
\ref{fig:MNIST_spectrum}. We transform each image $\bar{\boldsymbol{x}}$
into a data point $\boldsymbol{x}=\boldsymbol{U}^{\top}\left(\bar{\boldsymbol{x}}-\hat{\boldsymbol{\mu}}\right)/\sqrt{d}$,
which is to be inputted into the autoencoder. Note that this preprocessing
step is reasonable; all we have done are mean removal, which is a
common data preprocessing practice, and rotation by $\boldsymbol{U}$,
which does not affect the geometry of the data. We compute the MF
limiting dynamics by using the formulas given in Theorems \ref{res:ReLU_setting_simplified}
and \ref{res:ReLU_setting_2stage_simplified}. In particular, we let
$\boldsymbol{R}=\boldsymbol{I}_{d}$ and ${\rm diag}\left(\Sigma_{1}^{2},...,\Sigma_{d}^{2}\right)=\bar{\boldsymbol{C}}$.
For numerical stability, if $\Sigma_{i}^{2}<10^{-5}$, we replace
it with $10^{-5}$. For the non-digit test samples, we draw two from
the EMNIST data set \cite{cohen2017emnist} and two from the Fashion
MNIST data set \cite{xiao2017fashion}, and computer-generate the
other two patterned images. We preprocess these non-digit data in
a similar fashion.
\item In all simulations, we adopt a constant learning rate schedule $\xi\left(t\right)=1$,
which accords with the statements of Theorems \ref{res:ReLU_setting_simplified},
\ref{res:ReLU_setting_2stage_simplified} and \ref{res:Bdd_act_setting_simplified}.
\end{itemize}
\begin{figure}
\begin{centering}
\includegraphics[width=0.5\columnwidth]{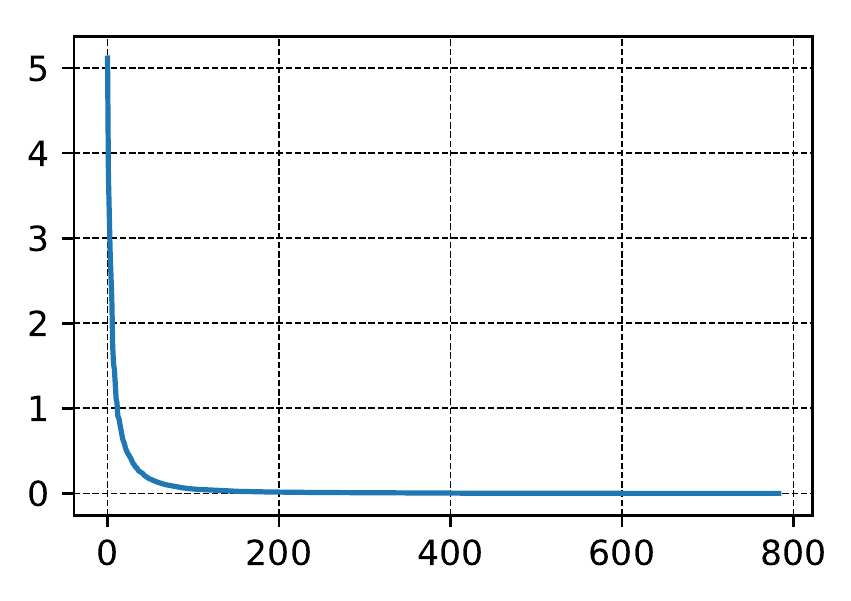}
\par\end{centering}
\caption{Spectrum of the estimated data covariance matrix of the MNIST data
set.}

\label{fig:MNIST_spectrum}
\end{figure}

In Fig. \ref{fig:MNIST_reconstr}, we visualize reconstructions of
several MNIST test images by the trained autoencoder from Fig. \ref{fig:MNIST_reg_loss_theta},
as well as its derived autoencoders constructed by the two-staged
process. This shows that the trained autoencoder is able to avoid
the common failure of producing only some average of the training
set \cite{li2018on}, although the reconstructed images are blurry
due to the regularization. The derived autoencoders, which sample
$M<N$ neurons sufficiently large from the trained autoencoder, also
incur little loss to the reconstruction quality.

\begin{figure}
\begin{centering}
\includegraphics[width=1\columnwidth]{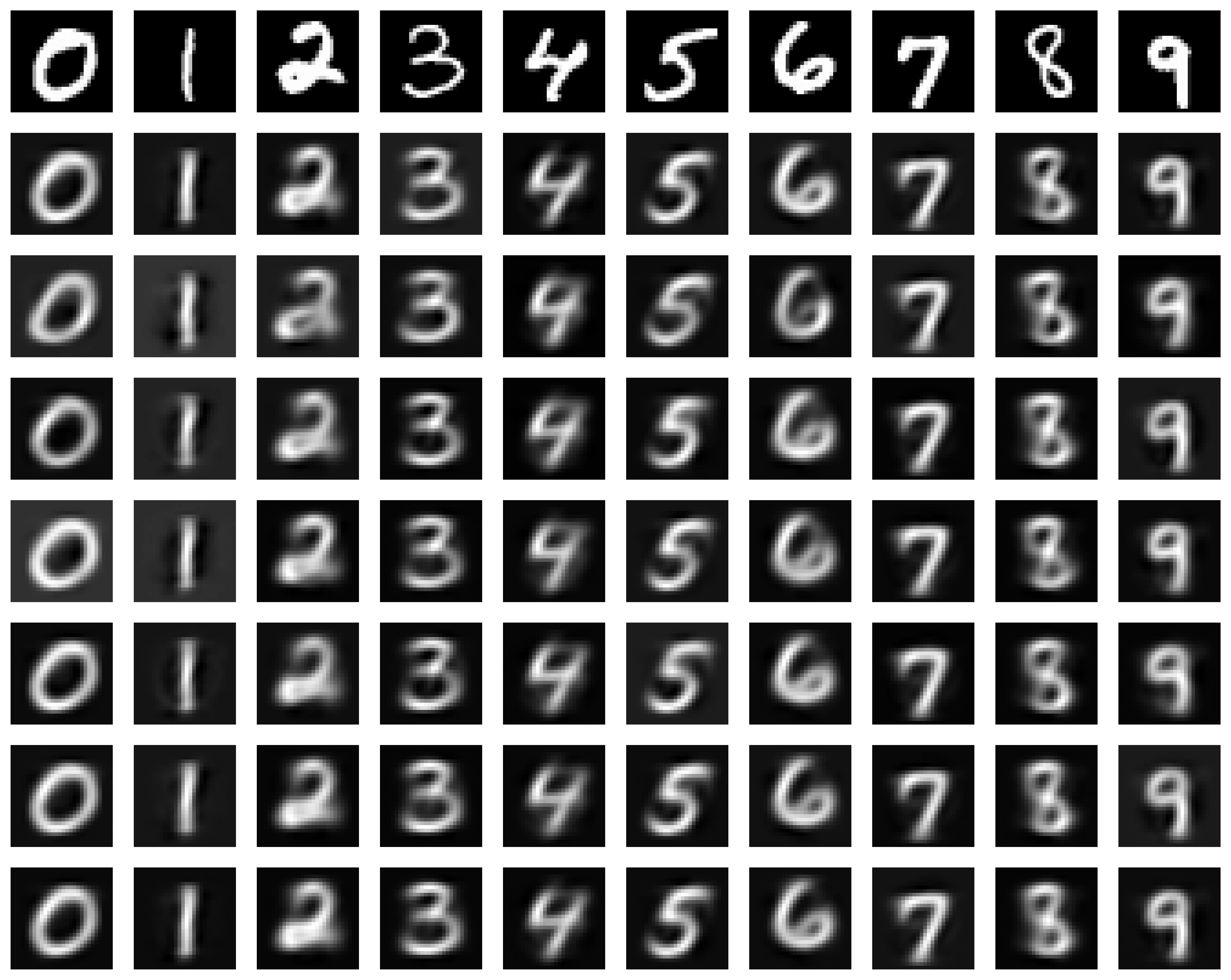}
\par\end{centering}
\caption{Reconstructed MNIST images by the trained autoencoder from Fig. \ref{fig:MNIST_reg_loss_theta}
(at iteration $10^{5}$), as well as the autoencoders derived from
the two-staged process with $M$ sampled neurons. From top: the original
MNIST test images, the reconstructions of the trained autoencoder,
the reconstructions of the derived autoencoders with $M=200$, 400,
600, 784, 2000, 10000.}

\label{fig:MNIST_reconstr}
\end{figure}

\bibliographystyle{amsalpha}
\bibliography{AE_meanfield_arXiv}

\end{document}